%% file: Thesis.tex
\documentclass[onecolumn, 12pt, fullpage, a4paper]{report}

\usepackage{amssymb}
\usepackage{amsthm}
\usepackage[cmex10]{amsmath}
\usepackage{adjustbox} 
\usepackage{arabtex} 
\usepackage{academicons}
\usepackage{breakcites} 
\usepackage{color}
\usepackage{epsfig}
\usepackage{epstopdf}
\usepackage{etoolbox}
\usepackage{enumitem}
\usepackage{float}
\usepackage{fancyhdr}
\usepackage{graphicx}
\usepackage{graphics}
\usepackage[hidelinks]{hyperref}
\usepackage{latexsym,amsfonts}
\usepackage{longtable}
\usepackage{listings}
\usepackage{lscape} 
\usepackage{lipsum}  
\usepackage{multirow}             
\usepackage{pdfpages}
\usepackage[figuresright]{rotating} 
\usepackage{sectsty}
\usepackage{setspace}
\usepackage[authoryear,round]{natbib}
\usepackage{textcomp}
\usepackage{utf8} 
\usepackage{url} 
\usepackage{wrapfig} 
\usepackage{wasysym} 
\usepackage{xcolor}
\usepackage{cleveref}
\usepackage{xspace}
\usepackage{scalefnt}
\usepackage{booktabs}
\usepackage[flushleft]{threeparttable} 
\usepackage{thmtools}
\usepackage{array}

\usepackage{caption}
\usepackage{subcaption}  

\usepackage{tabularx}
\usepackage{mathtools}

\theoremstyle{plain}
\newtheorem{theorem}{Theorem}[section]
\newtheorem{proposition}[theorem]{Proposition}
\newtheorem{lemma}[theorem]{Lemma}
\newtheorem{corollary}[theorem]{Corollary}
\theoremstyle{definition}
\newtheorem{definition}[theorem]{Definition}
\newtheorem{assumption}[theorem]{Assumption}
\theoremstyle{remark}
\newtheorem{remark}[theorem]{Remark}
\newtheorem{example}[theorem]{Example}
\newtheorem{fact}[theorem]{Fact}
\usepackage{comment}

\input{math_commands}

\chapterfont{\fontsize{14}{15}\selectfont}   
\sectionfont{\fontsize{14}{15}\selectfont}
\subsectionfont{\fontsize{14}{15}\selectfont}
\subsubsectionfont{\fontsize{14}{15}\selectfont}

\fancyheadoffset[L]{0.01mm}

\makeatletter
\patchcmd{\@makechapterhead}{50\p@}{20pt}{}{}
\patchcmd{\@makeschapterhead}{50\p@}{20pt}{}{}
\makeatother

\fancypagestyle{plain}{
\fancyhf{} 
\fancyhead[C]{\thepage} 

}

   \usepackage{ifthen,xkeyval,xfor,amsgen}
   \usepackage[acronym,toc, nogroupskip]{glossaries}
   \newglossary[slg]{symbols}{syi}{sbl}{List of Symbols}
 
   \makeglossaries

\include{Lists}

\usepackage[lmargin=40mm, rmargin=25mm, vmargin=25mm, headsep=2.5mm]{geometry}
\newcommand{\mathsym}[1]{{}}
\newcommand{\unicode}[1]{{}}
\renewcommand{\thechapter}{\arabic{chapter}}
\renewcommand\bibname{\centering BIBLIOGRAPHY}
\newcommand{\orcid}{\includegraphics[width=8pt]{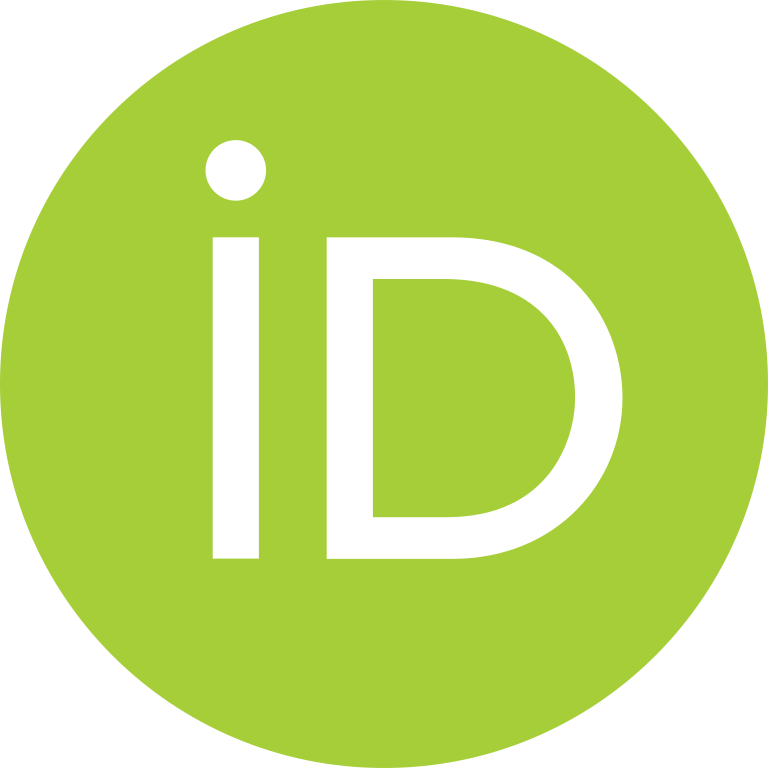}} 

\usepackage{algorithm}
\usepackage{algorithmic}

\begin{document}
\singlespacing  


\thispagestyle{empty}
\addvspace{5mm}  

\begin{singlespace}

\begin{center}
\begin{doublespace}
{\textbf{{\large Strategies for Improving Communication Efficiency\\ in Distributed and Federated Learning:\\ Compression, Local Training, and Personalization}}}
\end{doublespace}

\vspace{10mm}
{Dissertation by}\\
{Kai Yi} 

\vspace{30mm}

{ In Partial Fulfillment of the Requirements}\\[12pt]
{ For the Degree of}\\[12pt]
{Doctor of Philosophy} \vfill
{King Abdullah University of Science and Technology }\\
{Thuwal, Kingdom of Saudi Arabia}
\vfill

{\copyright April, 2025}\\
Kai Yi\\               
All rights reserved\\

\orcid{} \small \href{https://orcid.org/0000-0003-0415-3584}{https://orcid.org/0000-0003-0415-3584}\\
\href{https://kaiyi.me/}{kaiyi.me}

\end{center}
\newpage


%
\chaptertitlefont{\fontsize{14}{15}\selectfont\centering}  

\begin{singlespace}
\include{Abstract}  
\end{singlespace}

\begin{singlespace}
\include{Acknowledgment}  
\end{singlespace}


\tableofcontents
\cleardoublepage

\printglossary[type=\acronymtype,style=long3col, title=\centerline{LIST OF ABBREVIATIONS}, toctitle=List of Abbreviations, nonumberlist=true] 

\printglossary[type=symbols,style=long3col, title=\centerline{LIST OF SYMBOLS}, toctitle=List of Symbols, nonumberlist=true]

\begin{singlespace}

\cleardoublepage
\phantomsection
\addcontentsline{toc}{chapter}{\listfigurename} 
\renewcommand*\listfigurename{LIST OF FIGURES}
\listoffigures

\cleardoublepage
\phantomsection
\addcontentsline{toc}{chapter}{\listtablename}  
\renewcommand*\listtablename{LIST OF TABLES}
\listoftables
\end{singlespace}



\chaptertitlefont{\fontsize{14}{15}\selectfont}  

\begin{singlespace}
    \include{Chapter_1}
    \include{Chapter_EFBV}
    \include{Chapter_Scafflix}

\include{Chapter_2_FedP3}
    \include{Chapter_3_CohortSqueeze}
    \include{Chapter_4_SymWanda}
\include{Conclusion}
\end{singlespace}

  
	\renewcommand*\bibname{\centerline{REFERENCES}} 
    \phantomsection
	\addcontentsline{toc}{chapter}{References}
	\newcommand{\BIBdecl}{\setlength{\itemsep}{0pt}}
            \bibliographystyle{plainnat}
		\bibliography{References}


  
\appendix
		\newpage
		\begingroup
			\let\clearpage\relax
			\begin{center}
			\vspace*{2\baselineskip}
			{ \textbf{{\large APPENDICES}}} 
            \phantomsection
			\end{center}
            \include{Appendix_C3_EFBV}

            \include{Appendix_C3_Scafflix}
            \include{Appendix_C3_FedP3}

            \include{Appendix_C3_CohortSqueeze}

            \include{Appendix_C3_SymWanda}

\include{Papers}
		\endgroup

\end{singlespace}
\end{document}

%% file: math_commands.tex
\newcommand{\eqdef}{\coloneqq}
\newcommand{\lt}{\left}
\newcommand{\rt}{\right}

\renewcommand{\sb}[1]{\left[#1\right]}
\newcommand{\cb}[1]{\lt\{#1\rt\}}

\newcommand{\fr}{\frac}
\newcommand{\eps}{\varepsilon}
\newcommand{\Rd}{\mathbb{R}^d}

\newcommand{\prox}{\mathop{\mathrm{prox}}\nolimits}
\newcommand{\sumin}{\sum_{i=1}^n}

\newcommand{\sumjn}{\sum_{j=1}^n}

\newcommand{\avein}{\frac{1}{n}\sum_{i=1}^n}

\newcommand{\aveid}{\frac{1}{d}\sum_{i=1}^d}

\newcommand{\mbR}{\mathbb{R}}
\newcommand\ec[2][]{\ensuremath{\mathbb{E}_{#1} \left[#2\right]}}  
\newcommand\ev[1]{\ensuremath{\left \langle #1 \right \rangle}}  
\newcommand{\sqn}[1]{{\left\lVert#1\right\rVert}^2}
\newcommand{\norm}[1]{\left\| #1 \right\|}
\newcommand{\Norm}[1]{\left\|#1\right\|}
\newcommand{\sqN}[1]{\Norm{#1}^2}

\newcommand{\xstar}{x^{\star}}

\newcommand{\mP}{\mathbf{P}}
\newcommand{\mS}{\mathbf{S}}

\DeclareMathOperator*{\argmin}{arg\,min}
\DeclareMathOperator*{\minimize}{\mathrm{minimize}}
\usepackage{extarrows}

\usepackage{pifont}
\newcommand{\green}{\color{mydarkgreen}}
\newcommand{\red}{\color{mydarkred}}
\newcommand{\cmark}{\green{\ding{51}}}%
\newcommand{\xmark}{\red{\ding{55}}}%

\newcommand{\sqnorm}[1]{\left\| #1 \right\|^2}
\newcommand{\Exp}[1]{\mathbb{E}\!\left[ #1 \right]}

\newcommand{\oma}{\omega_{\mathrm{ran}}}
\usepackage{nicefrac}

\definecolor{mydarkgreen}{rgb}{0,0.42,0}
\definecolor{mydarkred}{rgb}{0.75,0,0}
\definecolor{mygreen2}{RGB}{0,120,20}
\newcommand{\algname}[1]{{\sf \texttt{\scalefont{0.99}{#1}}}\xspace}

\def\2{$^2$}			 
\def\3{$^3$}			 
\def\-2{$^{-2}$}		 
\def\-3{$^{-3}$}		 
\def\-1{$^{-1}$}		 

\newcommand{\Diag}{\mathrm{Diag}}

\def\gI{{\mathcal{I}}}

\def\gN{{\mathcal{N}}}

\newcommand{\mL}{\mathbf{L}}

\newcommand{\mI}{\mathbf{I}}

\DeclareMathOperator{\moL}{\overline{\mathbf{L}}}
\DeclareMathOperator{\moB}{\overline{\mathbf{B}}}

\DeclareMathOperator{\moW}{\overline{\mathbf{W}}}
\DeclareMathOperator{\ob}{\overline{b}}

\newcommand{\ft}{\algname{$R^2$-DSnoT}}
\newcommand{\vA}{{\mathbf{A}}}
\newcommand{\vB}{{\mathbf{B}}}
\newcommand{\vC}{{\mathbf{C}}}
\newcommand{\vD}{{\mathbf{D}}}

\newcommand{\vG}{{\mathbf{G}}}

\newcommand{\vI}{{\mathbf{I}}}

\newcommand{\vS}{{\mathbf{S}}}

\newcommand{\vW}{{\mathbf{W}}}
\newcommand{\vX}{{\mathbf{X}}}
\newcommand{\vY}{{\mathbf{Y}}}

\definecolor{bgcolor}{rgb}{0.93,0.99,1}
\definecolor{bgcolor2}{rgb}{0.8,1,0.8}
\definecolor{bgcolor3}{rgb}{0.50,0.90,0.50}
\definecolor{bgcolor4}{rgb}{0.94,0.97,1}

\usepackage{colortbl}
\definecolor{bgcolor}{rgb}{0.93,0.99,1}
\definecolor{bgcolor2}{rgb}{0.8,1,0.8}
\definecolor{bgcolor3}{rgb}{0.50,0.90,0.50}

\DeclareMathOperator*{\argmax}{arg\,max}

%% file: Lists.tex
\newglossaryentry{symb:Pi}{
name=$\pi$, type=symbols,
description=A mathematical constant whose value is the ratio of any circle's circumference to its diameter,
sort=symbolpi
}

\newglossaryentry{symb:Phi}{
name=$\varphi$, type=symbols,
description=An angle,
sort=symbolphi
}

\newglossaryentry{symb:Lambda}{
name=$\lambda$, type=symbols,
description=Lambda indicates usually an eigenvalue in linear algebra,
sort=symbollambda
}

\newacronym{toc}{ToC}{Table of Contents}
\newacronym{los}{LoS}{List of Symbols}
\newacronym{loa}{LoA}{List of Abbreviations}
\newacronym{phd}{PhD}{Doctoral}
\newacronym{MS}{MS}{Masters}
\newacronym{M$}{MS}{Microsoft}
\newacronym{CD}{CD}{Compact Disc}
\newacronym{kaust}{KAUST}{King Abdullah University of Science and Technology}

\newacronym{AD}{AD}{Active Directory\protect\glsadd{glos:AD}}

\newglossaryentry{glos:AD}{
name=Active Directory,
description={Active Directory is the directory service for Windows based networks, that allows central organization and administration of any network resource. It allows a single-sign-on concept independent from network topologies or network protocols. As a prerequisite you need a Windows Server acting as Domain Controller. This computer stores all necessary data, e.g.~usernames and corresponding passwords}
}

\newglossaryentry{glos:RespF}{name=response file, description={A file 
that allows unattended software installation}}

%% file: Abstract.tex
\begin{center}

\end{center}

\begin{center}
{{\bf\fontsize{14pt}{14.5pt}\selectfont \uppercase{ABSTRACT}}}
\end{center}

\doublespacing
\addcontentsline{toc}{chapter}{Abstract}

\begin{center}
        \begin{onehalfspace}
{\fontsize{14.5pt}{14.5pt}\selectfont {Strategies for Improving Communication Efficiency\\ in Distributed and Federated Learning:\\ Compression, Local Training, and Personalization}}\\
		{\fontsize{14.5pt}{14.5pt}\selectfont {Kai Yi}}\\
        \end{onehalfspace}
\end{center}

\singlespacing
Distributed and federated learning have emerged as essential paradigms for training machine learning models across decentralized data sources while preserving privacy. However, communication overhead remains a major bottleneck, particularly in large-scale, heterogeneous environments. This dissertation presents a comprehensive exploration of strategies to improve communication efficiency in distributed and federated learning systems, focusing on three key areas: model compression, local training, and personalization.

We begin by establishing a unified theoretical framework for biased and unbiased compression operators, providing convergence guarantees for both convex and non-convex settings. Building on this, we propose novel local training strategies that explicitly incorporate personalization mechanisms to accelerate convergence and mitigate client drift in federated environments. In particular, we introduce Scafflix, an adaptive local training algorithm that balances global and personalized objectives, achieving superior performance in both IID and non-IID settings.

Further, we address the challenge of communication efficiency in neural network models through federated privacy-preserving pruning frameworks that optimize global and local parameter sparsity while ensuring minimal communication costs. Our Cohort-Squeeze method extends beyond single communication rounds per cohort by leveraging hierarchical aggregation strategies, significantly reducing overall communication overhead in cross-device federated learning scenarios.

Finally, we conclude with SymWanda, a symmetric post-training pruning approach that minimizes the impact of pruning on both input activations and output layers. This strategy enhances model robustness under high sparsity and offers a training-free fine-tuning mechanism to maintain competitive performance without additional retraining.

Extensive experiments on benchmark datasets and large-scale language models demonstrate that the proposed methods consistently achieve a favorable balance between communication cost, model accuracy, and convergence speed. This dissertation provides both theoretical and practical insights for designing scalable, efficient distributed learning systems, contributing to the democratization of machine learning across diverse, resource-constrained devices.

%% file: Acknowledgment.tex

\begin{center}

\end{center}
\begin{center}

{\bf\fontsize{14pt}{14.5pt}\selectfont \uppercase{Acknowledgements}}\\\vspace{1cm}
\end{center}

\addcontentsline{toc}{chapter}{Acknowledgements} 


Time flies, and in the blink of an eye, five years have passed, bringing me to the crossroads of graduation. Completing my PhD will mark the culmination of my academic journey. As I reflect on more than 20 years of student life, especially the past five years of my master's and doctoral studies, I would like to express my deepest gratitude to everyone who has supported and guided me along the way.

First and foremost, I am profoundly grateful to my supervisor, Peter Richtárik, for his exceptional expertise, guidance, and unwavering support throughout this journey. His profound understanding, insightful feedback, and dedication to academic excellence have played a pivotal role in shaping the direction and quality of my work. I also extend my sincere thanks to my master's supervisor, Mohamed Elhoseiny, for his mentorship, his contributions to my empirical research explorations, and for giving me the opportunity to join this prestigious institute.

I am deeply appreciative of the members of my dissertation committee—Peter Richtárik, Panos Kalnis, Mikhail Moshkov and Quanquan Gu—for their valuable time and feedback. My heartfelt thanks also go to my research group colleagues and peers, whose intellectual contributions, engaging discussions, and collaborative spirit have broadened my perspectives and inspired new ideas. In particular, I am grateful to Laurent Condat, Grigory Malinovsky, Timur Kharisov, Georg Meinhardt, Konstantin Burlachenko, Egor Shulgin, Sarit Khirirat, Yury Demidovich, Kaja Gruntkowska, Artavazd Maranjyan, Hanmin Li, Abdurakhmon Sadiev, Igor Sokolov, Elnur Gasanov, Artem Riabinin, Omar Shaikh Omar, Ivan Ilin, Slavomír Hanzely, and Samuel Horváth for their support and collaboration.

Special thanks go to my internship mentors and external collaborators—Nidham Gazagnadou, Lingjuan Lyu, Yaoliang Yu, Vladimir Malinovskii, and Dan Alistarh—for their invaluable contributions to my academic growth.

Finally, I am profoundly thankful to my friends and family for their unwavering support and encouragement. Their belief in me and constant motivation have been a steady source of strength throughout this demanding journey.




%% file: Chapter_1.tex

\chapter{Introduction}
\label{chapter1}
\thispagestyle{empty}
\section{Overview}  
The rapid advancement of machine learning has led to unprecedented growth in model size and data complexity, driving the need for collaborative training paradigms. {Distributed learning} (DL) and {federated learning} (FL) have emerged as essential approaches to handle large-scale datasets and models in a decentralized fashion \citep{dean2012large, ja2016, mcmahan2017communication, liu2022distributed}. In both paradigms, multiple nodes or clients collaboratively train a shared model, with the key difference being that FL emphasizes data privacy by keeping raw data local, while DL typically assumes that data can be distributed across multiple nodes in non-private settings \citep{liu2022distributed}. However, both DL and FL face similar challenges, including high communication overhead, heterogeneous performance across nodes, and resource constraints \citep{bonawitz2019towards, kairouz2021advances}.

This dissertation focuses on \emph{improving communication efficiency} in distributed and federated learning through three interconnected strategies: \emph{model compression}, \emph{local training optimization}, and \emph{personalization}. These approaches target various aspects of the training pipeline to reduce communication costs while maintaining robust model performance.

\begin{itemize}
    \item {Model compression} techniques, such as gradient sparsification \citep{lin2017deep}, quantization \citep{hubara2018quantized}, and pruning \citep{frankle2018lottery}, aim to reduce the size of exchanged updates. While effective, they introduce trade-offs in terms of model convergence speed and accuracy, requiring careful exploration in both DL and FL settings.
    \item {Local training optimization} strategies reduce the frequency of communication rounds by increasing the number of local computations \citep{li2020federated, EF21, ProxSkip-VR}. Although this reduces communication costs, it can exacerbate model divergence in heterogeneous environments where local data distributions differ significantly.
    \item {Personalization} addresses node-level heterogeneity by tailoring the global model to individual clients or nodes \citep{fallah2020personalized, ghosh2020efficient, hanzely2021personalized}. In FL, personalization mitigates performance degradation caused by data variability, while in DL, it can improve task-specific generalization across distributed tasks.
\end{itemize}

This dissertation explores provable and efficient strategies to improve communication efficiency in distributed and federated learning systems \citep{kairouz2021advances}, balancing computational costs, communication overhead, and model performance. Next, we introduce our solutions based on these three core strategies at a high level, followed by the basic facts and notations used in subsequent sections.

\section{Distributed and federated learning}
DL and FL are two key paradigms designed to enable collaborative model training across multiple nodes or clients. This section provides an overview of their core concepts, similarities, and differences, as well as a formal definition of the problem settings considered in this dissertation.

\subsection{Distributed learning}
Distributed learning involves splitting the training process across multiple computing nodes to leverage parallelism and handle large-scale datasets or models \citep{dean2012large, verbraeken2020survey, liu2022distributed}. Each node typically has access to a partition of the training data and participates in updating the global model. The primary goal in DL is to \emph{achieve efficient parallelization, minimizing training time and ensuring that model updates are synchronized effectively}. In most cases, DL assumes that data across nodes is independently and identically distributed (i.i.d.), simplifying the aggregation process.
The global objective in distributed learning is formulated as:
\begin{equation} 
\minimize_{x\in\mathbb{R}^d}\;\underbrace{\frac{1}{n}\sum_{i=1}^n f_i(x)}_{f(x)} + R(x),\label{eqpro1}
\end{equation}
where $d\geq 1$ is the model dimension; $R:\mathbb{R}^d\rightarrow \mathbb{R}\cup\{+\infty\}$ is a  proper, closed, convex function \citep{bau17}, whose proximity operator $$\mathrm{prox}_{\gamma R} : x \mapsto \argmin_{y\in \mathbb{R}^d} \big( \gamma R(y)+\frac{1}{2}\|x-y\|^2 \big)$$ is easy to compute, for any $\gamma>0$~\citep{par14,con19,con22}; $n\geq 1$ is the number of functions; each function $f_i: \mathbb{R}^d\rightarrow \mathbb{R}$ is typically assumed to have the \emph{same smoothness $L$ and strong convexity $\mu$} across all nodes \citep{nesterov2003introductory}. In this dissertation, unless otherwise specified, we focus on the strongly convex and smooth setting.


Key challenges in DL include communication bottlenecks, node synchronization, and scalability when handling extremely large models or datasets \citep{stich2018local, verbraeken2020survey, liu2022distributed}.

\subsection{Federated learning}
Federated learning extends the distributed learning framework by incorporating privacy-preserving constraints \citep{ja2016, konevcny2016federated, mcmahan2016federated, mcmahan2017communication}. In FL, raw data remains on local nodes (clients), and only model updates or gradients are shared with a central server for aggregation. This setting is particularly relevant in privacy-sensitive applications, such as mobile devices and healthcare institutions, where data cannot be centralized \citep{hard2018federated, sheller2020federated}.

The global objective in FL is the same as DL, defined in \Cref{eqpro1}, 
but with the added constraint that the data distribution across clients may be non-i.i.d., making the optimization process more challenging \citep{ja2016, FedProx}. That is, each function $f_i:\mathbb{R}^d \rightarrow\mathbb{R}$ is convex and $L_i$-smooth, for some $L_i>0$; that is, $f_i$ is differentiable on $\mathbb{R}^d$ and its gradient $\nabla f_i$ is $L_i$-Lipschitz continuous. Here $\mu_i$ and $L_i$ is allowed to be arbitrary different. 

Similar to most FL studies, we do not include a regularizer $R$ in our formulation. Instead, we define the task as solving an empirical risk minimization (ERM) problem of the form: 

\begin{equation}\label{eq:ERM}\tag{ERM} 
    \min_{x \in \mathbb{R}^d} \left[ f(x) \eqdef \frac{1}{n} \sum_{i=1}^n f_i(x) \right], 
    \end{equation} 

where $f_i(x)$ represents the local objective for client $i$, $n$ is the total number of clients, and $x$ denotes the global model.

Key challenges in FL include communication efficiency, data heterogeneity, and client participation variability. To address these challenges, FL requires communication-efficient algorithms that balance global model performance with minimal communication overhead.

\subsection{Comparison of distributed and federated learning}
While both DL and FL aim to train a global model collaboratively, they differ in key aspects:
\begin{itemize}
    \item \textbf{Data distribution:} DL typically assumes i.i.d. data across nodes, while FL often involves non-i.i.d. data, reflecting real-world heterogeneity.
    \item \textbf{Privacy constraints:} FL enforces strict privacy by keeping data local, whereas DL generally does not impose such restrictions.
    \item \textbf{Communication frequency:} FL often has fewer communication rounds due to the high cost of transmitting updates, while DL can perform frequent communication, especially in data-center settings.
\end{itemize}

\subsection{Dissertation focus}
In this dissertation, we consider both DL and FL settings and focus on the following three key goals, framed around the core strategies of compression, local training, and personalization:

\begin{enumerate}
    \item \textbf{Communication efficiency:} Reducing the number of transmitted bits through model compression techniques such as gradient sparsification, quantization, and pruning, while maintaining model performance.
    \item \textbf{Scalability:} Ensuring that the proposed methods effectively scale to large datasets and models by optimizing local computations and minimizing communication frequency.
    \item \textbf{Robustness to heterogeneity:} Addressing non-i.i.d. data and variable client participation by incorporating personalized components that tailor the global model to local needs.
\end{enumerate}

These goals align with the three interconnected strategies presented in this dissertation: model compression, local training optimization, and personalization.


\section{Core strategies}
In this section, we detail the three core strategies for improving communication efficiency in distributed and federated learning: \emph{compression}, \emph{local training}, and \emph{personalization}. Each strategy addresses different aspects of the communication bottleneck while maintaining model performance and scalability.

\subsection{Compression}
Model compression techniques aim to reduce the size of the information exchanged during the training process \citep{choudhary2020comprehensive}. In both DL and FL, communication overhead can be significantly reduced by transmitting compressed updates instead of full gradients or model parameters.

Key approaches to compression include:
\begin{itemize}
        \item \textbf{Gradient sparsification} \citep{aji2017sparse, lin2017deep, EF21, fat21} Transmitting only the most significant gradient components, with the remaining components set to zero, thereby reducing the size of updates.
        \item \textbf{Model pruning} \citep{frankle2018lottery, RigL, SRigL, Wanda, SparseGPT, RIA} Removing unimportant weights or neurons in the model to reduce the overall model size and the corresponding communication and memory costs.

    \item \textbf{Quantization} \citep{ali17, hubara2018quantized, AQLM, PV-Tuning} Representing model updates using fewer bits, such as using fixed-point instead of floating-point representations.
\end{itemize}

In DL, compression reduces communication between nodes and the central parameter server, improving synchronization efficiency. In FL, it plays a crucial role in reducing the upload and download bandwidth required by clients, especially in real-world scenarios with limited communication resources. However, an important trade-off exists between compression ratios and model performance, as overly aggressive compression may slow convergence or degrade accuracy.

\subsection{Local training}\label{sec:local_training}
Local training optimization focuses on performing more computations on local data to reduce the frequency of communication rounds \citep{Povey2015, SparkNet2016, FL2017-AISTATS, Li-local-bounded-grad-norms--ICLR2020, LocalDescent2019, localGD, localSGD-AISTATS2020, SCAFFOLD, LSGDunified2020, FEDLIN, ProxSkip-VR, Scafflix}. By increasing the number of local updates before aggregation, this strategy can significantly reduce communication costs.

Popular local training methods include:
\begin{itemize}
    \item \textbf{Periodic aggregation:} Clients perform multiple local updates before sending their model updates to the server \citep{FedAvg, LocalDescent2019, localGD, FEDLIN, SCAFFOLD}.
    \item \textbf{Adaptive local updates:} The number of local updates is adjusted dynamically based on the current training progress or data heterogeneity \citep{stich2018local, ProxSkip, ProxSkip-VR, Scafflix}.
\end{itemize}

In DL, this strategy improves synchronization efficiency by reducing the number of gradient exchange steps. In FL, local training optimization addresses communication constraints but introduces the challenge of \textit{client drift}, where local models diverge due to differing data distributions. Properly balancing local computation and global synchronization is essential to prevent performance degradation.

Theoretical evolutions of LT in FL have been long-lasting, spanning five generations from empirical results to accelerated communication complexity. The celebrated \algname{FedAvg} algorithm proposed by \citet{FL2017-AISTATS} showed the feasibility of communication-efficient learning from decentralized data. It belongs to the first generation of LT methods, where the focus was on empirical results and practical validations \citep{Povey2015, SparkNet2016, FL2017-AISTATS}.

The second generation of studies on LT for solving (\ref{eq:ERM}) was based on homogeneity assumptions, such as bounded gradients \big($\exists c<+\infty, {\norm{\nabla f_i(x)}} \leq c, x\in\Rd$, $i\in[n]$\big) \citep{Li-local-bounded-grad-norms--ICLR2020} and bounded gradient diversity \big($\avein \sqn{\nabla f_i(x)} \leq c\sqn{\nabla f(x)}$\big) \citep{LocalDescent2019}. However, these assumptions are too restrictive and do not hold in practical FL settings~\citep{FL-big, FieldGuide2021}.

The third generation of approaches, under generic assumptions on the convexity and smoothness, 
exhibited sublinear convergence \citep{localGD, localSGD-AISTATS2020} or linear convergence to a neighborhood \citep{mal20}. 

Later, popular algorithms have emerged, such as \algname{Scaffold} \citep{SCAFFOLD}, \algname{S-Local-GD} \citep{LSGDunified2020}, and \algname{FedLin} \citep{FEDLIN}, successfully correcting for the client drift and enjoying linear convergence to an exact solution under standard assumptions. However, their communication complexity remains the same as with \algname{GD}, namely $\mathcal{O}(\kappa\log \epsilon^{-1})$, where $\kappa \eqdef L/\mu$ is the condition number.

Finally, \algname{Scaffnew} was proposed by \citet{ProxSkip}, with accelerated communication complexity $\mathcal{O}(\sqrt{\kappa}\log \epsilon^{-1})$.  This is a major achievement, which proves for the first time that LT is a communication acceleration mechanism. Thus,  \algname{Scaffnew} is the first algorithm in what can be considered the fifth generation of LT-based methods with accelerated convergence. 
Subsequent works have further extended \algname{Scaffnew} with features such as variance-reduced stochastic gradients~\citep{ProxSkip-VR}, compression \citep{CompressedScaffnew}, partial client participation \citep{con23tam2}, asynchronous communication of different clients \citep{GradSkip}, and to a general primal--dual framework \citep{con22rp}. The fifth generation of LT-based methods also includes the \algname{5GCS} algorithm \citep{mal22b}, based on a different approach: the local steps correspond to an inner loop to compute a proximity operator inexactly.  Our proposed algorithm \algname{Scafflix} generalizes  \algname{Scaffnew} and enjoys even better accelerated communication complexity, thanks to a better dependence on the possibly different condition numbers of the functions $f_i$.

\subsection{Personalization}
Personalization addresses the challenge of heterogeneity by adapting the global model to better fit the data on individual nodes or clients \citep{fallah2020personalized, ghosh2020efficient}. In FL, where data distributions across clients are often non-i.i.d., personalization improves client-specific performance while preserving the benefits of collaborative training. In DL, personalization can improve task-specific generalization when nodes handle domain-shifted or multi-task learning problems.


We can distinguish three main approaches to achieve personalization:

\begin{itemize}\label{secper}
    \item \textbf{One-stage training of a single global model using personalization algorithms.} One common scheme is to design a suitable regularizer to balance between current and past local models~\citep{MOON} or between global and local models~\citep{FedProx, hanzely2020federated}. 
    The FLIX model~\citep{FLIX} achieves explicit personalization by balancing the local and global model using interpolation. Meta-learning is also popular in this area, as evidenced by \citet{pFedMe}, who proposed a federated meta-learning framework using Moreau envelopes and a regularizer to balance personalization and generalization.

    \item \textbf{Training a global model and fine-tuning every local client or knowledge transfer/distillation.} This approach allows knowledge transfer from a source domain trained in the FL manner to target domains~\citep{FedMD}, which is especially useful for personalization in healthcare domains~\citep{FedHealth, FedSteg}.
  
    \item \textbf{Collaborative training between the global model and local models.} The basic idea behind this approach is that each local client trains some personalized parts of a large model, such as the last few layers of a neural network. Parameter decoupling enables learning of task-specific representations for better personalization~\citep{arivazhagan2019federated, bui2019federated}, while channel sparsity encourages each local client to train the neural network with sparsity based on their limited computation resources~\citep{FjORD, FedRolex, FLANC}.
\end{itemize}


In both DL and FL, the challenge lies in balancing model personalization with generalization. While highly personalized models may excel on individual clients, they can lose the collaborative benefits of global training. This dissertation proposes techniques that strike a balance by introducing efficient personalized updates while maintaining a shared model structure.

\section{Chapter overview and contributions}
\subsection{Chapter 2: unified theory of compressors}
In distributed or federated optimization and learning, communication between the different computing units is often the bottleneck and gradient compression is widely used to reduce the number of bits sent within each communication round
of iterative methods. There are two classes of compression operators and separate algorithms making use of them. In the case of unbiased random compressors with bounded variance (e.g., rand-k), the \algname{DIANA} algorithm of \cite{DIANA}, which implements a variance reduction technique for handling the variance introduced by compression, is the current state of the art. In the case of biased and contractive compressors (e.g., top-k), the \algname{EF21} algorithm of \cite{EF21}, which instead implements an error-feedback mechanism, is the current state of the art. These two classes of compression schemes and algorithms are distinct, with different analyses and proof techniques. In this paper, we unify them into a single framework and propose a new algorithm, recovering  \algname{DIANA} and  \algname{EF21} as particular cases. Our general approach works with a new, larger class of compressors, which has two parameters, the bias and the variance, and includes unbiased and biased compressors as particular cases. This allows us to inherit the best of the two worlds: like  \algname{EF21} and unlike  \algname{DIANA}, biased compressors, like top-k, whose good performance in practice is recognized, can be used. And like DIANA and unlike EF21, independent randomness at the compressors allows to mitigate the effects of compression, with the convergence rate improving when the number of parallel workers is large. This is the first time that an algorithm with all these features is proposed. We prove its linear convergence under certain
conditions. Our approach takes a step towards better understanding of two so-far distinct worlds of communication-efficient distributed learning.

This chapter is based on:

{\small
\noindent\algname{[EF-BV]} Condat, Laurent, Kai Yi, and Peter Richtárik. ``EF-BV: A unified theory of error feedback and variance reduction mechanisms for biased and unbiased compression in distributed optimization." Advances in Neural Information Processing Systems 35 (2022): 17501-17514.
}

\subsection{Chapter 3: personalized accelerated local training}
Federated Learning is an evolving machine learning paradigm, in which multiple clients perform computations based on their individual private data, interspersed by communication with a remote server. A common strategy to curtail communication costs is Local Training, which consists in performing multiple local stochastic gradient descent steps between successive communication rounds. However, the conventional approach to local training overlooks the practical necessity for client-specific personalization, a technique to tailor local models to individual needs. We introduce \algname{Scafflix}, a novel algorithm that efficiently integrates explicit personalization with local training. This innovative approach benefits from these two techniques, thereby achieving doubly accelerated communication, as we demonstrate both in theory and practice.

This chapter is based on:

{\small
\noindent\algname{[Scafflix]} Kai Yi, Laurent Condat, and Peter Richtárik. ``Explicit personalization and local training: Double communication acceleration in federated learning." Transactions on Machine Learning Research (TMLR), 2025.
}

\subsection{Chapter 4: personalized privacy-aware pruning}
The interest in federated learning has surged in recent research due to its unique ability to train a global model using privacy-secured information held locally on each client. This paper pays particular attention to the issue of client-side model heterogeneity, a pervasive challenge in the practical implementation of FL that escalates its complexity. Assuming a scenario where each client possesses varied memory storage, processing capabilities and network bandwidth - a phenomenon referred to as system heterogeneity - there is a pressing need to customize a unique model for each client. In response to this, we present an effective and adaptable federated framework  \algname{FedP3}, representing Federated Personalized and Privacy-friendly network Pruning, tailored for model heterogeneity scenarios. Our proposed methodology can incorporate and adapt well-established techniques to its specific instances. We offer a theoretical interpretation of  \algname{FedP3} and its locally differential-private variant, DP-FedP3, and theoretically validate their efficiencies.

This chapter is based on:

{\small
\noindent\algname{[FedP3]} Kai Yi, Nidham Gazagnadou, Peter Richtárik, and Lingjuan Lyu. ``FedP3: Federated Personalized and Privacy-friendly Network Pruning under Model Heterogeneity." In The Twelfth International Conference on Learning Representations.
}

\subsection{Chapter 5: beyond single communication round per cohort}
Virtually all FL methods, including \algname{FedAvg}, operate in the following manner: i) an orchestrating server sends the current model parameters to a cohort of clients selected via certain rule, ii) these clients then independently perform a local training procedure (e.g., via  \algname{SGD} or  \algname{Adam}) using their own training data, and iii) the resulting models are shipped to the server for aggregation. This process is repeated until a model of suitable quality is found. A notable feature of these methods is that each cohort is involved in a single communication round with the server only. In this work we challenge this algorithmic design primitive and investigate whether it is possible to ``squeeze more juice" out of each cohort than what is possible in a single communication round. Surprisingly, we find that this is indeed the case, and our approach leads to up to 74\% reduction in the total communication cost needed to train a FL model in the cross-device setting. Our method is based on a novel variant of the stochastic proximal point method (\algname{SPPM-AS}) which supports a large collection of client sampling procedures some of which lead to further gains when compared to classical client selection approaches.

This chapter is based on:

{\small
\noindent\algname{[Cohort-Squeeze]} Kai Yi, Timur Kharisov, Igor Sokolov, and Peter Richtárik. ``Cohort Squeeze: Beyond a Single Communication Round per Cohort in Cross-Device Federated Learning." arXiv preprint arXiv:2406.01115 (2024). {Oral} presentation at International Workshop on Federated Foundation Models In Conjunction with NeurIPS 2024 (FL@FM-NeurIPS'24).
}

\subsection{Chapter 6: symmetric post-training pruning}
Popular post-training pruning methods such as \algname{Wanda} \citep{Wanda} and \algname{RIA} \citep{RIA} are known for their simple, yet effective, designs that have shown exceptional empirical performance. \algname{Wanda} optimizes performance through calibrated activations during pruning, while \algname{RIA} emphasizes the relative, rather than absolute, importance of weight elements. Despite their practical success, a thorough theoretical foundation explaining these outcomes has been lacking. This paper introduces new theoretical insights that redefine the standard minimization objective for pruning, offering a deeper understanding of the factors contributing to their success. Our study extends beyond these insights by proposing complementary strategies that consider both input activations and weight significance. We validate these approaches through rigorous experiments, demonstrating substantial enhancements over existing methods. Furthermore, we introduce a novel training-free fine-tuning approach \ft that incorporates relative weight importance and a regularized decision boundary within a dynamic pruning-and-growing framework, significantly outperforming strong baselines and establishing a new state-of-the-art. 

This chapter is based on:

{\small
\noindent\algname{[SymWanda]} Kai Yi, Peter Richtárik. ``Symmetric Pruning for Large Language Models." arXiv preprint arXiv:2501.18980 (2025). ICLR 2025 Workshop on Sparsity in LLMs (SLLM).
}

\subsection{Chapter takeaway}
Each chapter in the subsequent section explores our approach to a specific challenging yet promising problem. It should be noted that the majority of our work focuses on developing strategies to enhance communication efficiency in distributed and federated learning environments. Specifically, we concentrate on three key areas: compression, local training, and personalization. In \Cref{tab:overview}, we provide a comparative overview of the main papers discussed in each chapter.

\begin{table}[!tb]
    \centering
    \begin{threeparttable}
    \caption{Comprehensive overview of discussed projects.}
    \label{tab:overview}
    \begin{tabular}{m{2cm} m{6cm} m{1cm} >{\centering\arraybackslash}m{1.5cm} m{0.6cm} m{0.6cm}}
    \toprule
     Paper  & Main Question & Result & Comp?\tnote{\color{blue}(a)} & LT? & Pers.? \\ \midrule
     \algname{EF-BV} (\Cref{chapter_ef_bv}) & \small Can we provide a unified theory for both biased (error feedback) and unbiased (variance reduction) compressors in distributed training? & Yes & \cmark & \xmark & \xmark \\ \midrule
     \algname{Scafflix} (\Cref{chapter_scafflix}) & \small Is it possible to achieve provable double acceleration through accelerated local training coupled with explicit personalization? & Yes & \xmark & \cmark & \cmark \\ \midrule
     \algname{FedP3} (\Cref{chapter_fedp3}) & \small Can we develop a comprehensive federated, personalized, and privacy-preserving pruning framework to enhance FL efficiency? & Yes & \cmark & \cmark & \cmark \\ \midrule
     \algname{Cohort}-\algname{Squeeze} (\Cref{chapter_cohort_squeeze}) & \small Are there provable benefits to incorporating multiple local communication rounds in cross-device FL? & Yes & \xmark & \cmark\tnote{\color{blue}(b)} & \xmark \\ \midrule
     \algname{SymWanda} (\Cref{chapter_symwanda}) & \small Can we provide theoretical support for post-training pruning methods and derive more efficient algorithms? & Yes & \cmark & \xmark & \xmark \\ \bottomrule
    \end{tabular}
    \begin{tablenotes}
    \footnotesize
    \item [{\color{blue}(a)}] ``Comp." ``LT." and ``Pers." stand for Compression, Local Training, and Personalization, respectively.
    \item [{\color{blue}(b)}] In the context of \algname{Cohort-Squeeze}, the term ``LT" deviates from the conventional definition of local training. Here, it specifically refers to multiple local communication rounds, rather than the usual multiple local computation rounds.
    \end{tablenotes}
    \end{threeparttable}
\end{table}


\subsection{Excluded Papers}  
During my PhD, I co-authored 11 additional papers that are not included in this dissertation. Most of these works focus on model compression and communication efficiency, aligning closely with my primary research interests. Others explore data-efficient model training and downstream tasks. The list includes:

\begin{itemize}
    \item \emph{Data-efficient multimodal language models:} Three works in this area, including \algname{DACZSL} \citep{DACZSL}, \algname{HGR-Net} \citep{HGR-Net}, and \algname{VisualGPT} \citep{VisualGPT}.
    \item \emph{Post-training compression of LLMs:} One paper focusing on extreme quantization (\algname{PV-Tuning}) \citep{PV-Tuning}.
    \item \emph{Efficient and accelerated FL:} A paper on accelerated sparse training (\algname{SparseProxSkip}) \citep{SparseProxSkip} and another on variance-reduced accelerated LT methods (\algname{ProxSkip-VR}) \citep{ProxSkip-VR}.
    \item \emph{Generative and creative learning:} Papers on generative data-efficient continual zero-shot learning (\algname{IGCZSL}) \citep{IGCZSL}, creative novel art generation (\algname{CWAN}) \citep{Jha2022CreativeWA}, and creativity-inspired generative zero-shot learning (\algname{CIZSL++}) \citep{elhoseiny2021cizsl++}.
    \item \emph{Representation learning and domain adaptation:} A study on semantic image feature disentanglement (\algname{3DSpVAE}) \citep{3DSpVAE} and a paper on unsupervised domain alignment for open-set structural recognition (\algname{MLUDA}) \citep{MLUDA}.
\end{itemize}

\section{Basic facts and notations}
Before presenting the main results, we will first clarify the key notations frequently used throughout this dissertation and provide relevant theoretical background to support the subsequent analysis.

\subsection{Convexity and smoothness}\label{sec:convex_smooth}
We outline the fundamental properties including convexity and smoothness of $f_i$ and $f$ in the objective function \Cref{eqpro1}.

\begin{definition}[$\mu$-strong convexity]\label{def:convexity}  
A differentiable function $f: \mathbb{R}^d \rightarrow \mathbb{R}$ is $\mu$-strongly convex if there exists $\mu > 0$ such that  
\begin{equation}\label{eqn:convexity}  
f(y) \geq f(x) + \langle \nabla f(x), y - x \rangle + \frac{\mu}{2} \| y - x \|^2, \quad \forall x, y \in \mathbb{R}^d.  
\end{equation}  
\end{definition}  

The function $f$ is considered convex if it satisfies \eqref{eqn:convexity} with $\mu = 0$. By default, we assume that each function $f_i$ in \Cref{eqpro1} is $\mu_i$-strongly convex and $L_i$-smooth, where $\mu_i, L_i > 0$. We define $L_{\max} \eqdef \max_i L_i$ and $\Tilde{L} \eqdef \sqrt{\avein L_i^2}$. The average function $f \eqdef \avein f_i$ is $\mu$-strongly convex and $L$-smooth, where $L \leq \Tilde{L} \leq L_{\max}$. Additionally, we assume that a minimizer of $f + R$ exists.

\begin{definition}[Smoothness]\label{def:smoothness}  
    A differentiable function $f: \mathbb{R}^d \rightarrow \mathbb{R}$ is said to be $L$-smooth if  
    $$
    \|\nabla f(x) - \nabla f(y)\| \leq L\|x - y\|, \quad \forall x, y \in \mathbb{R}^d.
    $$  
\end{definition}  


\subsection{Biased and unbiased compressors}
A compression operator is defined as a randomized map $\mathcal{C}: \mathbb{R}^d \to \mathbb{R}^d$ applicable to all $x \in \mathbb{R}^d$. Compressors can be broadly categorized based on their statistical properties into biased and unbiased types. Unbiased compressors are particularly notable for their ability to provide unbiased estimations. In the realm of biased compressors, we focus on the powerful classes known as biased contractive compressors, which offer specific advantages in data and model compression strategies.

\begin{definition}[Unbiased compressors]\label{def:unbiased_compressor}
    For every $\omega \geq 0$, we introduce the set $\mathbb{U}(\omega)$ of unbiased compressors, which are randomized operators of the form $\mathcal{C}: \mathbb{R}^d \rightarrow \mathbb{R}^d$, satisfying
    \begin{equation}\label{eqn:unbiased_compressor}
    \mathbb{E}[\mathcal{C}(x)]=x \quad \text { and } \quad \mathbb{E}\left[\|\mathcal{C}(x)-x\|^2\right] \leq \omega\|x\|^2, \quad \forall x \in \mathbb{R}^d.
    \end{equation}
    where $\mathbb{E}[\cdot]$ denotes the expectation. 
\end{definition}

The smaller $\omega$, the better, and $\omega=0$ if and only if $\mathcal{C}=\mathbf{I}_d$, the identity operator, which does not compress. We can remark that if $\mathcal{C} \in \mathbb{U}(\omega)$ is deterministic, then $\mathcal{C}=\mathbf{I}_d$. So, unbiased compressors are random ones. A classical unbiased compressor is \algname{rand- $k$}, for some $k \in \mathcal{I}_d$, which keeps $k$ elements chosen uniformly at random, multiplied by $\nicefrac{d}{k}$, and sets the other elements to 0. It is easy to see that \algname{rand-$k$} belongs to $\mathbb{U}(\omega)$ with $\omega=\nicefrac{d}{k}-1$ \citep{beznosikov2023biased}.

\begin{definition}[Biased contractive compressors]\label{def:biased_contractive_compressor}
    For every $\alpha \in(0,1]$, we introduce the set $\mathbb{B}(\alpha)$ of biased contractive compressors, which are possibly randomized operators of the form $\mathcal{C}: \mathbb{R}^d \rightarrow \mathbb{R}^d$, satisfying
    \begin{equation}\label{eqn:biased_contractive_compressor}
        \mathbb{E}\left[\|\mathcal{C}(x)-x\|^2\right] \leq(1-\alpha)\|x\|^2, \quad \forall x \in \mathbb{R}^d.
    \end{equation}
\end{definition}

We use the term \emph{contractive} to reflect the fact that the squared norm in the left hand side of (\ref{eqn:biased_contractive_compressor}) is smaller, in expectation, than the one in the right hand side, since $1-\alpha<1$. This is not the case in (\ref{eqn:unbiased_compressor}), where $\omega$ can be arbitrarily large. The larger $\alpha$, the better, and $\alpha=1$ if and only if $\mathcal{C}=\mathbf{I}_d$. Biased compressors need not be random: a classical biased and deterministic compressor is \algname{top-$k$}, for some $k \in \mathcal{I}_d$, which keeps the $k$ elements with largest absolute values unchanged and sets the other elements to 0. It is easy to see that \algname{top-$k$} belongs to $\mathbb{B}(\alpha)$ with $\alpha=\nicefrac{k}{d}$ \citep{beznosikov2023biased}.

\subsection{Differential privacy}
\begin{definition}[Local differential privacy (LDP)]\label{def:ldp}
    A randomized algorithm $\mathcal{A}: \mathcal{D}\to \mathcal{F}$, where $\mathcal{D}$ is the dataset domain and $\mathcal{F}$ the domain of possible outcomes, is $(\epsilon, \delta)$-locally differentially private for client $i$ if, for all neighboring datasets ${D}_i, {D}_i^\prime\in \mathcal{D}$ on client $i$ and for all events $\mathcal{S}\in \mathcal{F}$ within the range of $\mathcal{A}$, it holds that:
    \begin{align*}
    \mathrm{Pr}{\mathcal{A}(D_i)\in \mathcal{S}} \leq e^{\epsilon} \mathrm{Pr}{\mathcal{A}(D^\prime_i) \in \mathcal{S}} + \delta.
    \end{align*}
\end{definition}

This LDP definition (\ref{def:ldp}) closely resembles the original concept of $(\epsilon, \delta)$-DP \citep{dwork2014algorithmic, dwork2006calibrating}, but in the FL context, it emphasizes each client's responsibility to safeguard its privacy. This is done by locally encoding and processing sensitive data, followed by transmitting the encoded information to the server, without any coordination or information sharing among clients.

%% file: Chapter_EFBV.tex
\chapter{Unified Theory of Biased and Unbiased Compressors}
\label{chapter_ef_bv}
\thispagestyle{empty}

\section{Introduction}
In this paper, we focus on the standard distributed optimization problem in FL, where the global objective follows the finite-sum structure defined in \Cref{eqpro1}. Specifically, we assume a convex objective with basic smoothness and regularization properties as outlined in \Cref{sec:convex_smooth}.

We propose a stochastic gradient descent (\algname{SGD})-type method that leverages possibly \emph{biased} and randomized compression operators to reduce communication costs. Our approach incorporates variance reduction~\citep{han19,gor202,gow20a}, ensuring convergence to the exact solution with fixed stepsizes under standard assumptions, without requiring additional restrictive conditions on the functions being minimized.
	
\paragraph{Algorithms and Prior Work.}
Distributed proximal \algname{SGD} solves the problem \eqref{eqpro1} by iterating
\begin{equation}
x^{t+1} \eqdef \mathrm{prox}_{\gamma R} \big(x^t -  \frac{\gamma}{n} \sum_{i=1}^n g_i^t\big), 
\end{equation}  
 
where $\gamma$ is a stepsize and the vectors $g_i^t$ are possibly stochastic estimates of the gradients $\nabla f_i(x^t)$, which are cheap to compute or communicate. Compression is typically performed by the application of a possibly randomized operator $\mathcal{C}:\mathbb{R}^d\rightarrow \mathbb{R}^d$; that is, for any $x$, $\mathcal{C}(x)$ denotes a realization of a random variable, whose probability distribution depends on $x$. Compressors have  the property that it is much easier/faster to transfer $\mathcal{C}(x)$ than the original message $x$. This can be achieved in several ways, for instance by sparsifying the input vector~\citep{ali18}, or by quantizing its entries~\citep{ali17,Cnat,gan19,may21,sah21}, or via a combination of these and other approaches~\citep{Cnat,alb20, bez20}. There are two classes of compression operators often studied in the literature: 1) unbiased compression operators, satisfying a variance bound proportional to the squared norm of the input vector, and 2) biased compression operators, whose square distortion is contractive with respect to the squared norm of the input vector; we present these two classes in Sections \ref{secun} and \ref{secbia}, respectively. \medskip
 
\paragraph{Prior work: \algname{DIANA} with unbiased compressors.}
An important contribution to the field in the recent years is the variance-reduced \algname{SGD}-type method called \algname{DIANA}~\citep{DIANA}, which uses unbiased compressors; it is shown in Fig.~\ref{fig1}.  \algname{DIANA} was analyzed and extended in several ways, including bidirectional compression and acceleration,  see, e.g., the work of \citet{hor22,mis20,con22m,phi20,li2020,gor20}, and \citet{gor202,kha20} for general theories about \algname{SGD}-type methods, including variants using unbiased compression of (stochastic) gradients.\medskip

\noindent\textbf{Prior work: Error feedback with biased contractive compressors.}\ \ 
Our understanding of distributed optimization using biased compressors is more limited. The key complication comes from the fact that their naive use within methods like  gradient descent can lead to divergence, as widely observed in practice, see also Example~1 of  \citet{bez20}. 
\emph{Error feedback} (\algname{EF}), also called error compensation, techniques were proposed to fix this issue and obtain convergence, initially as heuristics \citep{sei14}. Theoretical advances have been made in the recent years in the analysis of \algname{EF}, see the discussions and references in \citet{ric21} and \citet{chu22}. But the question of whether it is possible to obtain a linearly convergent \algname{EF} method in the general heterogeneous data setting, relying on biased compressors only, was still an open problem; until last year, 2021, when \citet{ric21} re-engineered the classical \algname{EF} mechanism and came up with a new algorithm, called \algname{EF21}. 
 It was then extended in several ways, including by considering server-side compression, and the support of a regularizer $R$ in \eqref{eqpro1}, by \citet{fat21}. \algname{EF21}  is shown in Fig.~\ref{fig1}.\medskip

\noindent\textbf{Motivation and challenge.}\ \ 
While \algname{EF21} resolved an important theoretical problem  in the field of distributed optimization with contractive compression, there are still several open questions. In particular,  \algname{DIANA} with independent random compressors has a $\frac{1}{n}$ factor in its iteration complexity; that is, it converges faster when the number $n$ of workers is larger. 
\algname{EF21} does not have this property: its convergence rate does not depend on $n$. Also, the convergence analysis and proof techniques for the two algorithms are different: the linear convergence analysis of  \algname{DIANA} relies on $\|x^t-x^\star\|^2$ and $\|h_i^t-\nabla f_i(x^\star)\|^2$ tending to zero, where $x^t$ is the estimate of the solution $x^\star$ at iteration $t$ and $h_i^t$ is the control variate maintained at node $i$, whereas the analysis of \algname{EF21} relies on $(f+R)(x^t)-(f+R)(x^\star)$ and $\|h_i^t-\nabla f_i(x^t)\|^2$ tending to zero, and under different assumptions. This work aims at filling this gap. 
That is, we want to address the following open problem:

\begin{table*}[t]
\caption{Desirable properties of a distributed compressed gradient descent algorithm converging to an exact solution of \eqref{eqpro1} and whether they are satisfied by the state-of-the-art algorithms \algname{DIANA} and \algname{EF21} and their currently-known analysis, and the proposed algorithm \algname{EF-BV}.}  
\label{tab1}
\centering
\resizebox{\textwidth}{!}{
\begin{tabular}{ccccc}
\toprule
& \algname{DIANA}&\algname{EF21}&\algname{EF-BV}\\
\midrule
handles unbiased compressors in $\mathbb{U}(\omega)$ for any $\omega\geq 0$ & \cmark &\cmark\color{blue}${}^{(a)}$&\cmark\\
\midrule
handles biased contractive compressors in $\mathbb{B}(\alpha)$ for any $\alpha\in (0,1]$&\xmark&\cmark&\cmark\\
\midrule
handles  compressors in $\mathbb{C}(\eta,\omega)$ for any $\eta\in [0,1)$, $\omega\geq 0$&\xmark&\cmark\color{blue}${}^{(a)}$&\cmark\\
\midrule
recovers  \algname{DIANA} and \algname{EF21} as particular cases&\xmark&\xmark&\cmark\\
\midrule
the convergence rate improves when $n$ is large&\cmark&\xmark&\cmark\\
\bottomrule
\end{tabular}}
{\footnotesize {\color{blue} ${}^{(a)}$} with pre-scaling with $\lambda<1$, so that $\mathcal{C}'= \lambda\mathcal{C}\in\mathbb{B}(\alpha)$ is used instead of $\mathcal{C}$}
\end{table*}

{\itshape
Is it possible to design an algorithm, which combines the advantages of  \algname{DIANA} and \algname{EF21}? That is, such that:
\begin{enumerate}
	\item[a.]
It deals with unbiased compressors, biased contractive compressors, and possibly even more.
	\item[b.]
It recovers  \algname{DIANA} and \algname{EF21} as particular cases.
\item[c.]
Its convergence rate improves with $n$ large.  
	\end{enumerate}%
}\medskip
	
\noindent\textbf{Contributions.}\ \ We answer positively this question and propose a new algorithm, which we name \algname{EF-BV}, for \emph{Error Feedback with Bias-Variance decomposition}, which for the first time satisfies the three aforementioned properties. This is illustrated in Tab.~\ref{tab1}. More precisely, our contributions are:
\begin{enumerate}
	\item We propose a new, larger class of compressors,
	which includes unbiased and biased contractive compressors as particular cases, and has two parameters, the \textbf{bias} $\eta$ and the \textbf{variance} $\omega$. A third parameter $\oma$ describes  the resulting variance from the parallel compressors after aggregation, and is key to getting faster convergence with large $n$, by allowing larger stepsizes than in \algname{EF21} in our framework.
	\item We propose a new algorithm, named \algname{EF-BV}, which exploits the properties of the compressors in the new class using two scaling parameters $\lambda$ and $\nu$. For particular values of $\lambda$ and $\nu$, \algname{EF21} and  \algname{DIANA} are recovered as particular cases. But by setting the values of $\lambda$ and $\nu$ optimally with respect to $\eta$, $\omega$, $\oma$ in \algname{EF-BV}, faster convergence can be obtained.
	\item We prove linear convergence of \algname{EF-BV} under a Kurdyka--{\L}ojasiewicz condition of $f+R$, which is weaker than strong convexity of $f+R$. 	Even for \algname{EF21} and  \algname{DIANA}, this is new.
	
	\item We provide new insights on \algname{EF21} and  \algname{DIANA}; for instance, we prove linear convergence of  \algname{DIANA} with biased compressors.

	\end{enumerate}%

\section{Compressors and their properties}
We introduce two of the most widely used types of compressors: unbiased compressors (\Cref{def:unbiased_compressor}) and biased contractive compressors (\Cref{def:biased_contractive_compressor}). In the subsequent section, we propose a new, more general class of compressors, which forms the foundation of our method.

\subsection{New general class of compressors}\label{sec23}

We refer to \citet{bez20}, Table 1 in \citet{saf21}, \citet{zha21}, \citet{sze22}, for examples of compressors in $\mathbb{U}(\omega)$ or $\mathbb{B}(\alpha)$, and to \citet{xu20} for a system-oriented survey.

In this work, we introduce a new, more general class of compressors, ruled by 2 parameters, to allow for a finer characterization of their properties. Indeed, with any compressor $\mathcal{C}$, we can do a {\bf bias-variance decomposition} of the compression error: 
for every $x\in\mathbb{R}^d$,
\begin{equation}
 \mathbb{E}\big[\|\mathcal{C}(x)-x\|^2\big] = {\underbrace{\big\| \mathbb{E}[\mathcal{C}(x)]-x\big\|}_{\text{bias}}}^2 + \underbrace{\mathbb{E}\Big[\big\|\mathcal{C}(x)-\mathbb{E}[\mathcal{C}(x)]\big\|^2\Big]}_{\text{variance}}.\label{eqbiva}
 \end{equation}
Therefore, to better characterize the properties of compressors, we propose to parameterize these two parts, instead of only their sum: for every  $\eta \in [0,1)$ and $\omega\geq 0$, 
we introduce the new class $\mathbb{C}(\eta,\omega)$ of possibly random and biased operators, which are randomized operators of the form $\mathcal{C}:\mathbb{R}^d\rightarrow \mathbb{R}^d$, satisfying, for every $x\in\mathbb{R}^d$, the two properties:
\begin{align*}
\mathrm{(i)}\quad &\big\| \mathbb{E}[\mathcal{C}(x)]-x\big\|\leq \eta \|x\|,\\
\mathrm{(ii)}\quad& \mathbb{E}\Big[\big\|\mathcal{C}(x)-\mathbb{E}[\mathcal{C}(x)]\big\|^2\Big]\leq \omega\|x\|^2.
\end{align*}
Thus, $\eta$ and $\omega$ control the relative bias and variance of the compressor, respectively. 
Note that $\omega$ can be arbitrarily large, but the compressors will be scaled in order to control the compression error, as we discuss  in Sect.~\eqref{secsca}. 
 On the other hand, we must have $\eta<1$, since otherwise, no scaling can keep the compressor's discrepancy under control.\medskip

We have the following properties:
\begin{enumerate}
\item  $\mathbb{C}(\eta,0)$ is the class of deterministic compressors in $\mathbb{B}(\alpha)$, with $1-\alpha=\eta^2$.

\item  $\mathbb{C}(0,\omega)=\mathbb{U}(\omega)$, for every $\omega\geq 0$. In words, if its bias $\eta$ is zero,  the compressor is unbiased with relative variance $\omega$.

\item  Because of the bias-variance decomposition \eqref{eqbiva},  if $\mathcal{C}\in\mathbb{C}(\eta,\omega)$ with $\eta^2+\omega < 1$, then $\mathcal{C}\in\mathbb{B}(\alpha)$ with 
\begin{equation}
1-\alpha = \eta^2+\omega.\label{eqalpha}
\end{equation}

\item  Conversely, if $\mathcal{C}\in\mathbb{B}(\alpha)$, one easily sees from \eqref{eqbiva} that there exist $\eta \leq \sqrt{1-\alpha}$ and $\omega \leq 1-\alpha$ such that $\mathcal{C}\in\mathbb{C}(\eta,\omega)$.
\end{enumerate}

Thus, the new class $\mathbb{C}(\eta,\omega)$  generalizes the two previously known classes $\mathbb{U}(\omega)$ and $\mathbb{B}(\alpha)$. Actually, for compressors in $\mathbb{U}(\omega)$ and $\mathbb{B}(\alpha)$, we can just use  \algname{DIANA} and  \algname{EF21}, and our proposed algorithm \algname{EF-BV} will stand out when the compressors are neither in $\mathbb{U}(\omega)$ nor in $\mathbb{B}(\alpha)$; that is why the strictly larger class $\mathbb{C}(\eta,\omega)$ is needed for our purpose.

We present new compressors in the class $\mathbb{C}(\eta,\omega)$ in Appendix~\ref{secappa}.

\subsection{Average variance of several compressors}\label{secavv}

Given $n$ compressors $\mathcal{C}_i$, $i\in \mathcal{I}_n$, we are interested in how they behave in average. Indeed distributed algorithms consist, at every iteration, in compressing vectors in parallel, and then averaging them. Thus, we 
 introduce the \textbf{average relative variance} 
$\oma\geq 0$ of the compressors,  
such that, for every $x_i\in\mathbb{R}^d$, $i\in\mathcal{I}_n$, 
\begin{equation}
\Exp{ \sqnorm{ \frac{1}{n}\sum_{i=1}^n \big(\mathcal{C}_i(x_i)-\Exp{\mathcal{C}_i(x_i)}\big)} } \leq \frac{\oma}{n} \sum_{i=1}^n \sqnorm{x_i }.
\label{eqbo}
\end{equation}
When every $\mathcal{C}_i$ is in $\mathbb{C}(\eta,\omega)$, for some $\eta \in [0,1)$ and $\omega\geq 0$, then 
 $\oma \leq \omega$; but $\oma$ can be much smaller than  $\omega$, and we will exploit this property in \algname{EF-BV}. We can also remark that $
\frac{1}{n} \sum_{i=1}^n \mathcal{C}^i \in \mathbb{C}(\eta,\oma)$.

An important property is the following: if the $\mathcal{C}_i$ are mutually independent, since the variance of a sum of random variables is the sum of their variances, then
\begin{equation*}
\oma=\frac{\omega}{n}.
\end{equation*}
There are other cases where the compressors are dependent but $\oma$ is much smaller than $\omega$. Notably, the following setting can be used to model partial participation of $m$ among $n$ workers at every iteration of a distributed algorithm. For some $m \in \mathcal{I}_n$, where $\mathcal{I}_n \eqdef \{1, \ldots, n\}$ represents the index set, the $\mathcal{C}_i$ are defined jointly as follows: for every $i \in \mathcal{I}_n$ and $x_i \in \mathbb{R}^d$,
\begin{equation*}
\mathcal{C}_i(x_i) =\begin{cases} \;\frac{n}{m} x_i & \text{ if }\;  i\in\Omega \\ 
\;0 & \text{ otherwise} \end{cases},
\end{equation*}
where $\Omega$ is a subset of $\mathcal{I}_n$ of size $m$ chosen uniformly at random. 
 This is sometimes called $m$-nice sampling \citep{ric16,gow20}. Then  every $\mathcal{C}_i$ belongs to $\mathbb{U}(\omega)$, with $\omega=\frac{n-m}{m}$, and, as shown for instance in   \citet{qia19} and Proposition~1 in \citet{con22m},  \eqref{eqbo} is satisfied with
\begin{equation*}
\oma=\frac{n-m}{m(n-1)}=\frac{\omega}{n-1}\quad\mbox{($=0$ if $n=m=1$)}. 
\end{equation*}

\subsection{Scaling compressors}\label{secsca}

A compressor $\mathcal{C}\in\mathbb{C}(\eta,\omega)$ does not necessarily belong to $\mathbb{B}(\alpha)$ for any $\alpha \in (0,1]$, since $\omega$ can be arbitrarily large. 
Fortunately, the compression error can be kept under control 
by \emph{scaling} the compressor; that is, using  $\lambda \mathcal{C}$ instead of $\mathcal{C}$, for some scaling parameter $\lambda\leq 1$. We have:
\begin{proposition}
\label{prop3}
Let $\mathcal{C}\in\mathbb{C}(\eta,\omega)$, for some  $\eta \in [0,1)$ and $\omega\geq 0$, and $\lambda\in (0,1]$. Then
$\lambda\mathcal{C}\in\mathbb{C}(\eta',\omega')$ with $\omega'= \lambda^2 \omega$ and $\eta '=\lambda\eta+1-\lambda\in(0,1]$.\end{proposition}

\begin{proof}Let $x\in\mathbb{R}^d$. Then 
$$\mathbb{E}\Big[\big\|\lambda\mathcal{C}(x)-\mathbb{E}[\lambda\mathcal{C}(x)]\big\|^2\Big]  = \lambda^2\mathbb{E}\Big[\big\|\mathcal{C}(x)-\mathbb{E}[\mathcal{C}(x)]\big\|^2\Big] \leq  \lambda^2\omega\|x\|^2,
$$ 
 and 
$$\big\| \mathbb{E}[\lambda\mathcal{C}(x)]-x\big\| 
 \leq  \lambda
\big\| \mathbb{E}[\mathcal{C}(x)]-x\big\|+(1-\lambda) \|x\| \leq  (\lambda\eta+1-\lambda)\|x\| .
$$
\end{proof}

So, scaling deteriorates the bias, with $\eta'\geq \eta$, but linearly, whereas it reduces the variance $\omega$ quadratically. This is key, since 
the total error factor $(\eta')^2+\omega'$ can be made smaller than 1 by choosing $\lambda$ sufficiently small:
\begin{proposition}
\label{propsmall}
Let $\mathcal{C}\in\mathbb{C}(\eta,\omega)$, for some  $\eta \in [0,1)$ and $\omega\geq 0$. There exists $\lambda\in(0,1]$ such that 
$\lambda\mathcal{C}\in \mathbb{B}(\alpha)$, for some $\alpha = 1- (1-\lambda+\lambda\eta)^2-{\lambda}^2\omega \in (0,1]$, and the best such $\lambda$, maximizing $\alpha$, is
\begin{equation*}
\lambda^\star=\min\left(\frac{1-\eta}{(1-\eta)^2+\omega},1\right).
\end{equation*}
\end{proposition}

\begin{proof}
We define the polynomial $P:\lambda\mapsto(1-\lambda+\lambda\eta)^2+\lambda^2\omega$. 
After Proposition~\ref{prop3} and the discussion in Sect.~\ref{sec23}, we have to find $\lambda\in(0,1]$ such that $P(\lambda)<1$. Then 
$\lambda\mathcal{C}\in \mathbb{B}(\alpha)$, with $1-\alpha=P(\lambda)$. Since $P$ is a strictly convex quadratic function on $[0,1]$ with value $1$ and negative derivative $\eta-1$ at $\lambda=0$,  its minimum value on $[0,1]$ is smaller than 1 and is attained at $\lambda^\star$, which either satisfies the first-order condition $0 = P'(\lambda) = -2(1-\eta) +2\lambda\big((1-\eta)^2 + \omega\big)$, or, if this value is larger than 1, is equal to 1.
\end{proof}

In particular, if $\eta=0$, Proposition~\ref{propsmall} recovers Lemma 8 of \citet{ric21}, according to which, for  $\mathcal{C}\in\mathbb{U}(\omega)$, $\lambda^\star \mathcal{C}\in \mathbb{B}(\frac{1}{\omega+1})$, with $\lambda^\star=\frac{1}{\omega+1}$. 
For instance, the scaled \texttt{rand-}$k$ compressor, 
which keeps $k$ elements chosen uniformly at random unchanged and sets the other elements to 0, corresponds to scaling the unbiased \texttt{rand-}$k$ compressor, seen in \Cref{def:unbiased_compressor}, by $\lambda=\frac{k}{d}$.

We can remark that scaling is used to mitigate the randomness of a compressor, but cannot be used to reduce its bias: if $\omega=0$, $\lambda^\star=1$.

Our new algorithm \algname{EF-BV} will have two scaling parameters: $\lambda$, to mitigate the compression error in the control variates used for variance reduction, just like above, and $\nu$, to mitigate the error in the stochastic gradient estimate, in a similar way but with $\omega$ replaced by $\oma$, since we have seen in Sect.~\ref{secavv} that $\oma$ characterizes the randomness after averaging 
the outputs of several compressors.

\section{Proposed algorithm \algname{EF-BV}}

We propose the algorithm \algname{EF-BV}, shown in Fig.~\ref{fig1}. It makes use of compressors $\mathcal{C}_i^t \in \mathbb{C}(\eta,\omega)$, for some  $\eta \in [0,1)$ and $\omega\geq 0$, and we introduce $\oma\leq \omega$ such that \eqref{eqbo} is satisfied. That is, for any $x\in\mathbb{R}^d$, the $\mathcal{C}_i^t(x)$, for $i\in\mathcal{I}_n$ and $t\geq 0$, are distinct random variables; their laws might be the same or not, but they all lie in the class $\mathbb{C}(\eta,\omega)$. Also, $\mathcal{C}_i^t(x)$ and $\mathcal{C}_{i'}^{t'}(x')$, for $t\neq t'$, are independent. 

The compressors have the property that if their input is the zero vector, the compression error is zero, so we want to compress vectors that are close to zero, or at least converge to zero, to make the method variance-reduced. That is why each worker maintains a control variate $h_i^t$, converging, like $\nabla f_i(x^t)$, to $\nabla f_i(x^\star)$, for 
some solution $x^\star$. This way, the difference vectors $\nabla f_i(x^t)-h_i^t$ converge to zero, and these are the vectors that are going to be compressed. 
Thus,  \algname{EF-BV} takes the form of Distributed proximal \algname{SGD}, with $$g_i^t = h_i^t + \nu  \mathcal{C}_i^t\big(\nabla f_i(x^t)-h_i^t\big),$$ where the scaling parameter $\nu$ will be used to make the compression error, averaged over $i$, small; that is, to make $g^{t+1}=\frac{1}{n}\sum_{i=1}^n g_i^t $ close to $\nabla f(x^t)$. 
In parallel, the control variates are updated similarly as $$h_i^{t+1}= h_i^t + \lambda \mathcal{C}_i^t\big(\nabla f_i(x^t)-h_i^t\big),$$ where the scaling parameter $\lambda$ is used to make the compression error small, individually for each $i$; that is, to make $h_i^{t+1}$ close to $\nabla f_i(x^t)$.

\begin{figure*}[t]
\begin{minipage}{.30\textwidth}
	\begin{algorithm}[H]
	\scalefont{0.9}
            \caption{\algname{EF-BV}}
		\begin{algorithmic}
			\STATE
			\noindent \textbf{Input:} $x^0, h_1^0, \dots, h_n^0 \in \mathbb{R}^d$,  
			$h^0= \frac{1}{n}\sum_{i=1}^n h_i^0$, $\gamma>0$, \\
			 {\color{blue}$\lambda\in(0,1]$}, {\color{red}$\nu\in(0,1]$}
			\FOR{$t=0, 1, \ldots$}
			\FOR{$i=1, 2, \ldots, n$ in parallel}
			\STATE $d_i^t \coloneqq \mathcal{C}_i^t\big(\nabla f_i(x^t)-h_i^t\big)$
			\STATE $h_i^{t+1} \coloneqq h_i^t + {\color{blue}\lambda} d_i^t$
			\STATE send $d_i^t$ to master
			\ENDFOR
			\STATE at master:
			\STATE $d^t \coloneqq \frac{1}{n}\sum_{i=1}^n d_i^t$
			\STATE $h^{t+1} \coloneqq h^t + {\color{blue}\lambda} d^t$
			\STATE $g^{t+1} \coloneqq h^t + {\color{red}\nu} d^t$ 
			\STATE $x^{t+1}\!\coloneqq\!\mathrm{prox}_{\gamma R}(x^t - \gamma g^{t+1})$%
			\STATE broadcast $x^{t+1}$ to all workers
			\ENDFOR
		\end{algorithmic}
	\end{algorithm}
	\end{minipage}
	\ \ \ \ \ \ \begin{minipage}{.30\textwidth}
	\begin{algorithm}[H]
	\scalefont{0.9}
        \caption{\algname{EF21}}
		\begin{algorithmic}
			\STATE
			\noindent \textbf{Input:} $x^0, h_1^0, \dots, h_n^0 \in \mathbb{R}^d$,  
			$h^0= \frac{1}{n}\sum_{i=1}^n h_i^0$,  $\gamma>0$, \\
			\phantom{XXX}
			\FOR{$t=0, 1, \ldots$}
			\FOR{$i=1, 2, \ldots, n$ in parallel}
			\STATE $d_i^t \coloneqq \mathcal{C}_i^t\big(\nabla f_i(x^t)-h_i^t\big)$
			\STATE $h_i^{t+1} \coloneqq h_i^t +  d_i^t$
			\STATE send $d_i^t$ to master
			\ENDFOR
			\STATE at master:
			\STATE $d^t \coloneqq \frac{1}{n}\sum_{i=1}^n d_i^t$
			\STATE $h^{t+1} \coloneqq h^t +  d^t$
			\STATE $g^{t+1} \coloneqq h^t +  d^t$ 
			\STATE $x^{t+1}\!\coloneqq\!\mathrm{prox}_{\gamma R}(x^t - \gamma g^{t+1})$%
			\STATE broadcast $x^{t+1}$ to all workers
			\ENDFOR
		\end{algorithmic}
	\end{algorithm}
	\end{minipage}
	\ \ \ \ \ \ \begin{minipage}{.30\textwidth}
	\begin{algorithm}[H]
	\scalefont{0.9}
		\caption{\algname{DIANA}}
		\begin{algorithmic}
			\STATE
			\noindent \textbf{Input:}  $x^0, h_1^0, \dots, h_n^0 \in \mathbb{R}^d$,  $h^0= \frac{1}{n}\sum_{i=1}^n h_i^0$, 
			  $\gamma>0$, \\
			{\color{blue}$\lambda\in(0,1]$}
			\FOR{$t=0, 1, \ldots$}
			\FOR{$i=1, 2, \ldots, n$ in parallel}
			\STATE $d_i^t \coloneqq \mathcal{C}_i^t\big(\nabla f_i(x^t)-h_i^t\big)$
			\STATE $h_i^{t+1} \coloneqq h_i^t + {\color{blue}\lambda} d_i^t$
			\STATE send $d_i^t$ to master
			\ENDFOR
			\STATE at master:
			\STATE $d^t \coloneqq \frac{1}{n}\sum_{i=1}^n d_i^t$
			\STATE $h^{t+1} \coloneqq h^t + {\color{blue}\lambda} d^t$ 			
			\STATE $g^{t+1} \coloneqq h^t +  d^t$ 
			\STATE $x^{t+1}\!\coloneqq\!\mathrm{prox}_{\gamma R}(x^t - \gamma g^{t+1})$%
			\STATE broadcast $x^{t+1}$ to all workers
			\ENDFOR
		\end{algorithmic}
	\end{algorithm}
	\end{minipage}
	\caption{\label{fig1}In the three algorithms, $g^{t+1}$ is an estimate of $\nabla f(x^t)$, the $h_i^t$ are control variates converging to 
	$\nabla f_i(x^\star)$, and their average $h^t= \frac{1}{n}\sum_{i=1}^n h_i^t$ is maintained and updated by the master. 
	\algname{EF21} is a particular case of \algname{EF-BV}, when $\nu=\lambda=1$ and the compressors are in $\mathbb{B}(\alpha)$; 	then $g^{t+1}$ is simply equal to $h^{t+1}$ for every $t\geq 0$. 
	\algname{DIANA} is a particular case of \algname{EF-BV}, when $\nu=1$ and  the compressors are in $\mathbb{U}(\omega)$; then $g^{t}$ is an unbiased estimate of $\nabla f(x^t)$.}
	\end{figure*}

	\subsection{\algname{EF21} as a particular case of \algname{EF-BV}}\label{sec31}
	
	There are two ways to recover \algname{EF21} as a particular case of \algname{EF-BV}:
	
\begin{enumerate}
\item If the compressors $\mathcal{C}_i^t$ are in $\mathbb{B}(\alpha)$, for some $\alpha\in (0,1]$, there is no need for scaling the compressors, and we can use \algname{EF-BV} with $\lambda=\nu=1$. Then the variable $h^t$ in \algname{EF-BV} becomes redundant with the gradient estimate $g^t$ and we can only keep the latter, which yields  \algname{EF21}, as shown in Fig.~\ref{fig1}.
	
\item  If the scaled compressors $\lambda\mathcal{C}_i^t$ are in $\mathbb{B}(\alpha)$, for some $\alpha\in (0,1]$ and $\lambda\in (0,1)$
	(see Proposition~\ref{propsmall}), one  can simply 	use these scaled compressors in \algname{EF21}. This is equivalent to using  \algname{EF-BV} with the original compressors $\mathcal{C}_i^t$, the scaling with $\lambda$ taking place inside the algorithm. But we must have $\nu=\lambda$ for this equivalence to hold.
\end{enumerate}	
	
	Therefore, we consider thereafter that \algname{EF21} corresponds to the particular case of  \algname{EF-BV} with $\nu=\lambda \in (0,1]$ and $\lambda\mathcal{C}_i^t \in \mathbb{B}(\alpha)$, for some $\alpha\in (0,1]$, and is not only the original algorithm shown in Fig.~\ref{fig1},  which has no scaling parameter (but scaling might have been applied beforehand to make the compressors in  $\mathbb{B}(\alpha)$).

	\subsection{\algname{DIANA} as a particular case of \algname{EF-BV}}\label{sec32}
	
	\algname{EF-BV} with $\nu=1$  yields exactly  \algname{DIANA}, as shown in Fig.~\ref{fig1}.  \algname{DIANA} was only studied with unbiased compressors $\mathcal{C}_i^t\in\mathbb{U}(\omega)$, for some $\omega\geq 0$. In that case, $\Exp{g^{t+1}}=\nabla f(x^t)$, so that $g^{t+1}$ is an unbiased stochastic gradient estimate; this is not the case in \algname{EF21} and \algname{EF-BV}, in general. Also, 
	$\lambda=\frac{1}{1+\omega}$ is the usual choice in  \algname{DIANA}, which is consistent with Proposition~\ref{propsmall}.

\section{Linear convergence results}\label{sec5}

We will prove linear convergence of \algname{EF-BV} under conditions weaker than strong convexity of $f+R$.

When $R=0$, we will consider the  Polyak--{\L}ojasiewicz (P{\L}) condition on $f$: $f$ is said to satisfy the  P{\L}  condition with constant $\mu>0$ if, 
 for every $x\in\mathbb{R}^d$,
$\|\nabla f(x)\|^2\geq 2\mu\big(f(x)-f^\star\big)$,
where $f^\star = f(x^\star)$, for any minimizer $x^\star$ of $f$. This holds if, for instance, $f$ is $\mu$-strongly convex; that is, $f-\frac{\mu}{2}\|\cdot\|^2$ is convex.  In the general case, we will consider the  Kurdyka--{\L}ojasiewicz (K{\L}) condition with exponent $1/2$ \citep{att09,kar16} on $f+R$: $f+R$ is said to satisfy the  K{\L}  condition with constant $\mu>0$ if,  for every $x\in\mathbb{R}^d$ and $u\in  \partial R(x)$, 
\begin{equation}
\|\nabla f(x)+u\|^2\geq 2\mu\big(f(x)+R(x)-f^\star-R^\star\big),\label{eqKL}
\end{equation} 
where $f^\star = f(x^\star)$ and $R^\star= R(x^\star)$, for any minimizer $x^\star$ of $f+R$. This holds if, for instance, $R=0$ and $f$ satisfies the P{\L}  condition with constant $\mu$, so that the K{\L}  condition generalizes the P{\L}  condition to the general case $R\neq 0$. The K{\L}  condition also holds if $f+R$ is $\mu$-strongly convex \citep{kar16}, for which it is sufficient that $f$ is $\mu$-strongly convex, or $R$ is $\mu$-strongly convex.

In the rest of this section, we assume that $\mathcal{C}_i^t \in \mathbb{C}(\eta,\omega)$, for some  $\eta \in [0,1)$ and $\omega\geq 0$, and we introduce $\oma\leq \omega$ such that \eqref{eqbo} is satisfied. 
According to the discussion in Sect.~\ref{secsca} (see also Remark~\ref{rem1} below), we define the optimal values for the scaling parameters $\lambda$ and $\nu$:
\begin{align*}
\lambda^\star\eqdef\min\left(\frac{1-\eta}{(1-\eta)^2+\omega},1\right), \qquad \nu^\star\eqdef\min\left(\frac{1-\eta}{(1-\eta)^2+\oma},1\right).
\end{align*}
Given $\lambda \in (0,1]$ and $\nu \in (0,1]$, 
we define for convenience
$r \eqdef (1-\lambda+\lambda\eta)^2+{\lambda}^2\omega$, $r_{\mathrm{av}} \eqdef (1-\nu+\nu\eta)^2+{\nu}^2\oma$, 
as well as
$s^\star \eqdef \sqrt{\frac{1+r}{2r}}-1$ and $\theta^\star \eqdef s^\star(1+s^\star)\frac{r}{r_{\mathrm{av}}}$.

Note that if $r<1$, according to Proposition~\ref{prop3} and \eqref{eqalpha}, $\lambda \mathcal{C}_i^t \in \mathbb{B}(\alpha)$, with $\alpha=1-r$.

Our linear convergence results for \algname{EF-BV} are the following: 

\begin{theorem}\label{theo1}Suppose that $R=0$ and $f$ satisfies the P{\L}  condition with some constant  $\mu>0$. 
In \algname{EF-BV}, suppose that $\nu \in (0,1]$, $\lambda \in (0,1]$ is such that $r<1$, and 
\begin{equation}
0<\gamma \leq \frac{1}{L+\tilde{L}\sqrt{\frac{r_{\mathrm{av}}}{r}}\frac{1}{s^\star}}.\label{equpb1}
\end{equation}
For every $t\geq 0$, define the Lyapunov function
\begin{equation*}
\displaystyle\Psi^t \eqdef f(x^t)-f^\star + \frac{\gamma}{2\theta^\star}  \frac{1}{n}\sum_{i=1}^n \sqnorm{\nabla f_i(x^t)-h_i^{t}}, 
\end{equation*}
where $f^\star \eqdef f(x^\star)$, for any minimizer $x^\star$ of $f$. 
Then, for every $t\geq 0$,
\begin{align}
\Exp{\Psi^{t}} 
&\leq \left(\max\left(1-\gamma\mu, {\frac{r+1}{2}}\right) \right)^t\Psi^0.\label{eqsdgerg}
\end{align}
\end{theorem}

\begin{theorem}\label{theo2}
Suppose that $f+R$ satisfies the  the K{\L}  condition with some constant $\mu>0$. 
In \algname{EF-BV}, suppose that $\nu \in (0,1]$, $\lambda \in (0,1]$ is such that $r<1$, and
\begin{equation}
0<\gamma \leq \frac{1}{2L+\tilde{L}\sqrt{\frac{r_{\mathrm{av}}}{r}}\frac{1}{s^\star}}.\label{equpb2}
\end{equation}
 $\forall t\geq 0$, define the Lyapunov function
\begin{align*}
\displaystyle\Psi^t \eqdef f(x^t)+R(x^t)-f^\star - R^\star  + \frac{\gamma}{2\theta^\star}  \frac{1}{n}\sum_{i=1}^n \sqnorm{\nabla f_i(x^t)-h_i^{t}}, 
\end{align*}
where $f^\star \eqdef f(x^\star)$ and $R^\star \eqdef R(x^\star)$, for any minimizer $x^\star$ of $f+R$. 
Then, for every $t\geq 0$,
\begin{align}
\Exp{\Psi^{t}}  &\leq \left(\max\left({\frac{1}{1+\frac{1}{2}\gamma\mu}},\frac{r+1}{2}\right)\right)^t\Psi^0.\label{eqsdgerg2}
\end{align}
\end{theorem}

\begin{remark}[choice of $\lambda$, $\nu$, $\gamma$ in \algname{EF-BV}]\label{rem1}
\normalfont In Theorems \ref{theo1} and \ref{theo2}, the rate is better if $r$ is small and $\gamma$ is large. So, we should take $\gamma$ equal to the upper bound in \eqref{equpb1} and \eqref{equpb2}, since there is no reason to choose it smaller. Also, this upper bound is large 
if $r$ and $r_{\mathrm{av}}$ are small. As discussed in Sect.~\ref{secsca}, $r$ and $r_{\mathrm{av}}$ are minimized with $\lambda=\lambda^\star$ and $\nu=\nu^\star$ (which implies that $ r_{\mathrm{av}}\leq r<1$), so this is the recommended choice. Also, with this choice of $\lambda$, $\nu$, $\gamma$, there is no parameter left to tune in the algorithm, which is a nice feature.
\end{remark}

\begin{remark}[low noise regime]\label{rem2}
When the compression error tends to zero, i.e.\ $\eta$ and $\omega$ tend to zero, and we use accordingly $\lambda \rightarrow 1$, $\nu \rightarrow 1$, such that $r_{\mathrm{av}}/r$ remains bounded, 
 then $\mathcal{C}_i^t\rightarrow\mathrm{Id}$, $r\rightarrow 0$, and $\frac{1}{s^\star}\rightarrow 0$. Hence, \algname{EF-BV} reverts to proximal gradient descent $x^{t+1} = \mathrm{prox}_{\gamma R} \big(x^t -\nabla f(x^t)\big)$. 
\end{remark}

\begin{remark}[high noise regime]\label{rem3}
When the compression error becomes large,  i.e.\ $\eta\rightarrow 1$ or $\omega\rightarrow +\infty$, then $r\rightarrow 1$ and $\frac{1}{s^\star}\sim \frac{4}{1-r}$. Hence, the asymptotic complexity of  \algname{EF-BV} to achieve $\epsilon$-accuracy, when $\gamma=\Theta\Big(\frac{1}{L+\tilde{L}\sqrt{\frac{r_{\mathrm{av}}}{r}}\frac{1}{s^\star}}\Big)$, is
\begin{equation}
\mathcal{O}\left(\left(\frac{L}{\mu}+\left(\frac{\tilde{L}}{\mu}\sqrt{\frac{r_{\mathrm{av}}}{r}}+1\right)\frac{1}{1-r}\right)
\log \frac{1}{\epsilon}\right).\label{eqasy1}
\end{equation}
\end{remark}

\subsection{Implications for \algname{EF21}}

Let us assume that $\nu=\lambda$, so that  \algname{EF-BV} reverts to  \algname{EF21}, as explained in Sect.~\ref{sec31}. 
Then, if we don't assume the prior knowledge of $\oma$, or equivalently if $\oma=\omega$, Theorem~\ref{theo1} with $r=r_{\mathrm{av}}$ recovers  the linear convergence result of  \algname{EF21} due to \citet{ric21}, up to slightly different constants.

However,  in these same conditions, Theorem~\ref{theo2} is new:  linear convergence of \algname{EF21} with  $R\neq 0$ was only shown in Theorem 13 of \citet{fat21}, under the assumption that there exists $\mu>0$, such that
for every $x\in\mathbb{R}^d$, 
$\frac{1}{\gamma^2} \sqnorm{x-\mathrm{prox}_{\gamma R}\big(x-\gamma \nabla f(x)
\big)} \geq 2\mu\big(f(x)+R(x)-f^\star-R^\star\big)$. 
This condition 
generalizes the P{\L}  condition, since it reverts to it when $R=0$, but it is different from the K{\L}  condition, and it is not clear when it is satisfied, in particular whether it is implied by strong convexity of $f+R$.

The asymptotic complexity  to achieve $\epsilon$-accuracy of \algname{EF21} with 
$\gamma=\Theta\big(\frac{1}{L+\tilde{L}/s^\star}\big)$ is
$\mathcal{O}\big(\frac{\tilde{L}}{\mu}\frac{1}{1-r}
\log \frac{1}{\epsilon}\big)$ 
(where we recall that $1-r=\alpha$, with the scaled compressors in $\mathbb{B}(\alpha)$).
Thus,  for a given problem and compressors, the improvement of  \algname{EF-BV} over  \algname{EF21} is the factor
$\sqrt{\frac{r_{\mathrm{av}}}{r}}$ in 
\eqref{eqasy1}, which can be small if $n$ is large.

Theorems \ref{theo1} and \ref{theo2} provide a new  insight about  \algname{EF21}: if we exploit the knowledge that $\mathcal{C}_i^t \in \mathbb{C}(\eta,\omega)$ and the corresponding constant $\oma$, and if $\oma<\omega$, then $r_{\mathrm{av}}<r$, so that, based on \eqref{equpb1} and \eqref{equpb2}, $\gamma$ can be chosen larger than with the default assumption  that $r_{\mathrm{av}}=r$. As a consequence, convergence will be faster. This illustrates the interest of our new finer parameterization of compressors with $\eta$, $\omega$, $\oma$. However, it is only half the battle to make use of the factor $\frac{r_{\mathrm{av}}}{r}$ in \algname{EF21}: the property $\oma<\omega$ is only really exploited if $\nu=\nu^\star$ in \algname{EF-BV} (since $r_{\mathrm{av}}$ is minimized this way). In other words, there is no reason to set $\nu=\lambda$ in \algname{EF-BV}, when a larger value of $\nu$ is allowed in Theorems \ref{theo1} and \ref{theo2} and yields faster convergence.

\subsection{Implications for  \algname{DIANA}}

Let us assume that $\nu=1$, so that  \algname{EF-BV} reverts to   \algname{DIANA}, as explained in Sect.~\ref{sec32}. This choice is allowed in Theorems \ref{theo1} and \ref{theo2}, so that they provide new convergence results for  \algname{DIANA}. Assuming that the compressors are unbiased, i.e.\ $\mathcal{C}_i^t\in \mathbb{U}(\omega)$ for some $\omega\geq 0$, we have the following result on   \algname{DIANA} \citep[Theorem 5 with $b=\sqrt{2}$]{con22m}:

\begin{proposition}\label{propdiana}
Suppose that $f$ is $\mu$-strongly convex, for some $\mu>0$, and that in  \algname{DIANA}, $\lambda=\frac{1}{1+\omega}$, 
$0<\gamma \leq \frac{1}{L_{\max}+L_{\max}(1+\sqrt{2})^2\oma}$. For every $t\geq 0$, define the Lyapunov function
$$\Phi^t \eqdef \sqnorm{x^t-x^\star} 
+  (2+\sqrt{2})\gamma^2\oma (1+\omega)\frac{1}{n}\sum_{i=1}^n \sqnorm{\nabla f_i(x^\star)-h_i^t},$$ 
where $x^\star$ is the minimizer of $f+R$, which exists and is unique. 
Then, for every $t\geq 0$, we have 
$$\Exp{\Phi^{t}} 
\leq \left(\max\left(1-\gamma\mu, \frac{\frac{1}{2}+\omega}{1+\omega}\right) \right)^t\Phi^0.$$
\end{proposition}
Thus, noting that $r=\frac{\omega}{1+\omega}$, so that $\frac{r+1}{2}=\frac{\frac{1}{2}+\omega}{1+\omega}$, the rate is exactly the same as in Theorem~\ref{theo1}, but with a different Lyapunov function.  
Theorems \ref{theo1} and \ref{theo2} have the advantage over Proposition \ref{propdiana}, that linear convergence is guaranteed under the P{\L}  or K{\L}  assumptions, which are weaker than strong convexity of $f$. Also, the constants $L$ and $\tilde{L}$ appear instead of $L_{\max}$. This shows a better dependence with respect to the problem. However, noting that $r=\frac{\omega}{1+\omega}$, $r_{\mathrm{av}}=\oma$, $\frac{1}{s^\star}\sim 4\omega$, the factor $\sqrt{\frac{r_{\mathrm{av}}}{r}}\frac{1}{s^\star}$ scales like $\sqrt{\oma}\omega$, which is worse that $\oma$. This means that $\gamma$ can certainly be chosen larger in Proposition \ref{propdiana} than in Theorems \ref{theo1} and \ref{theo2}, leading to faster convergence.

However, Theorems \ref{theo1} and \ref{theo2} bring a major highlight: for the first time, they establish convergence of  \algname{DIANA}, which is  \algname{EF-BV} with $\nu=1$, with biased compressors. We state the results in Appendix~\ref{secappb}, by lack of space. 
In any case, 
with biased compressors, it is better to use \algname{EF-BV} than  \algname{DIANA}: there is no interest in choosing $\nu=1$ instead of $\nu=\nu^\star$, which minimizes $r_{\mathrm{av}}$ and allows for a larger $\gamma$, for faster convergence.

Finally, we can remark that for unbiased compressors with $\oma \ll 1$, for instance if $\oma \approx \frac{\omega}{n}$ with $n$ larger than $\omega$, then $\nu^\star=\frac{1}{1+\oma}\approx 1$. Thus, in this particular case,  \algname{EF-BV} with $\nu=\nu^\star$ and  \algname{DIANA} are essentially the same algorithm.  This is another sign that  \algname{EF-BV} with $\lambda=\lambda^\star$ and $\nu=\nu^\star$ is a generic and robust choice, since it recovers   \algname{EF21} and   \algname{DIANA} in settings where these algorithms shine.

\section{Sublinear convergence in the nonconvex case}

In this section, we consider the general nonconvex setting. In \eqref{eqpro1}, every function $f_i$ is supposed $L_i$-smooth, for some $L_i>0$. For simplicity, we suppose that $R=0$.
As previously, we set $\tilde{L}\eqdef\sqrt{\frac{1}{n}\sum_{i=1}^n L_i^2}$. The average function 
	$f\eqdef \frac{1}{n}\sum_{i=1}^n f_i$
	is $L$-smooth, for some $L\leq \tilde{L}$. We also suppose that $f$ is bounded from below; that is, $f^{\inf}\eqdef \inf_{x\in \mathbb{R}^d} f(x) > -\infty$.

 Given $\lambda \in (0,1]$ and $\nu \in (0,1]$, 
we define for convenience
$r \eqdef (1-\lambda+\lambda\eta)^2+{\lambda}^2\omega$, $r_{\mathrm{av}} \eqdef (1-\nu+\nu\eta)^2+{\nu}^2\oma$, 
as well as
$s \eqdef \frac{1}{\sqrt{r}}-1$ 
and $\theta \eqdef s(1+s)\frac{r}{r_{\mathrm{av}}}$. Our convergence result is the following:\smallskip

\begin{theorem}\label{thm:noncvx}
In \algname{EF-BV}, suppose that $\nu \in (0,1]$, $\lambda \in (0,1]$ is such that $r<1$, and 
       \begin{equation}0< \gamma \leq \frac{1}{L + \tilde{L}\sqrt{\frac{r_{\mathrm{av}}}{r}}\frac{1}{s}}.\end{equation}
     For every $t\geq 1$, let $\hat{x}^t$ be chosen from the iterates $x^0, x^1, \cdots, x^{t-1}$ uniformly at random. Then 
      \begin{equation}
         \mathbb{E}\left[\left\|\nabla f(\hat{x}^{t})\right\|^{2}\right] \leq \frac{2\big(f(x^{0})-f^{\inf}\big)}{\gamma t}+\frac{G^0}{\theta t},
         \end{equation}
      where $G^0\eqdef\frac{1}{n}\sum_{i=1}^n \sqnorm{\nabla f_i(x^0) - h_i^0}$.
   \end{theorem}

\section{Experiments}
We conducted comprehensive experiments to illustrate the efficiency of  \algname{EF-BV} compared to \algname{EF21} (we use biased compressors, so we don't include  \algname{DIANA} in the comparison). The settings and results are detailed in Appendix~\ref{appexp} and some results are shown in Fig.~\ref{fig:0007}; we can see the speedup obtained with \algname{EF-BV}, which exploits the randomness of the compressors.
 
\begin{figure}[!htbp]
	\centering
	\begin{subfigure}[b]{0.24\textwidth}
		\centering
		\includegraphics[width=\textwidth]{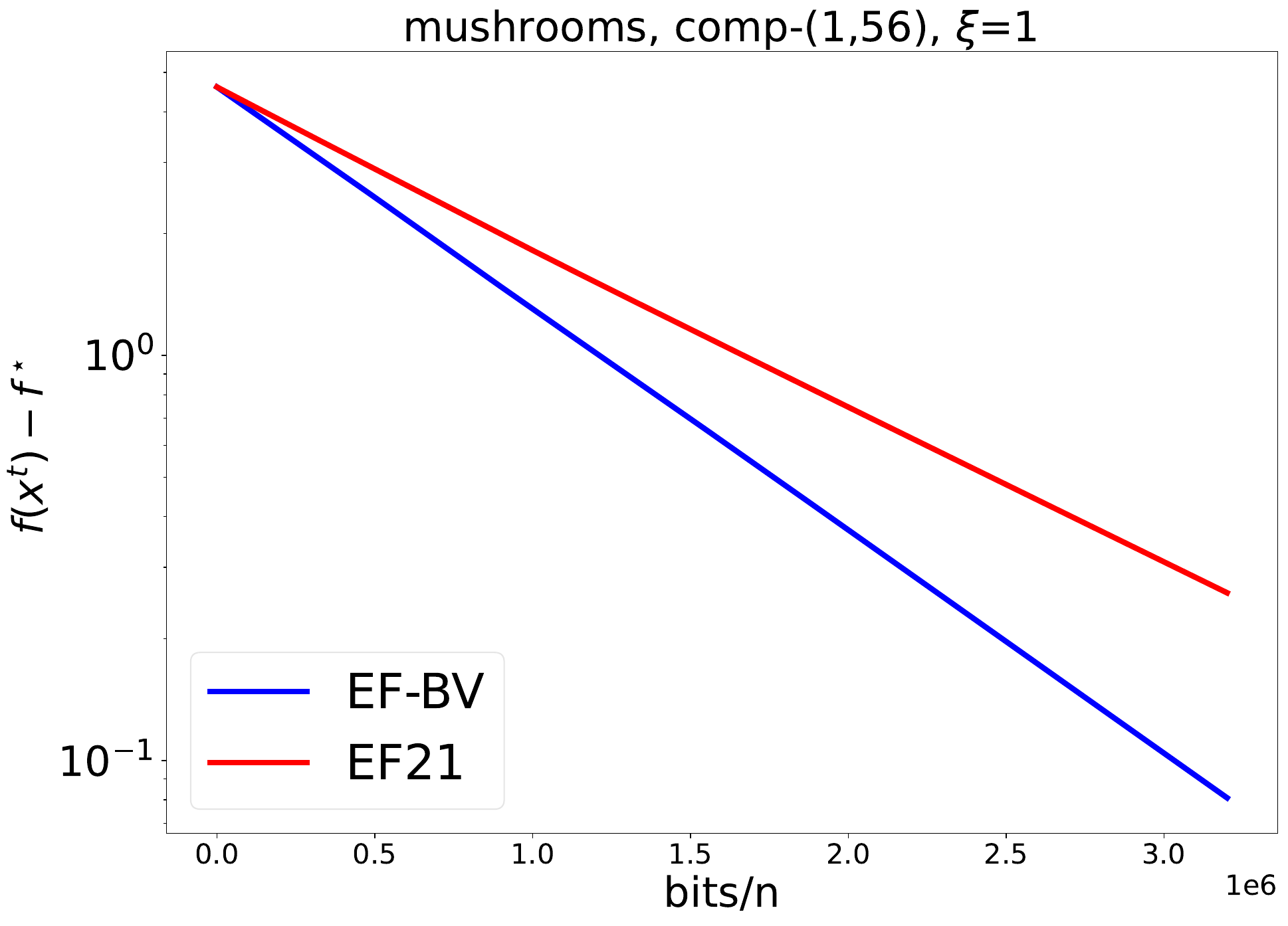}
	\end{subfigure}
	\hfill 
	\begin{subfigure}[b]{0.24\textwidth}
		\centering
		\includegraphics[width=\textwidth]{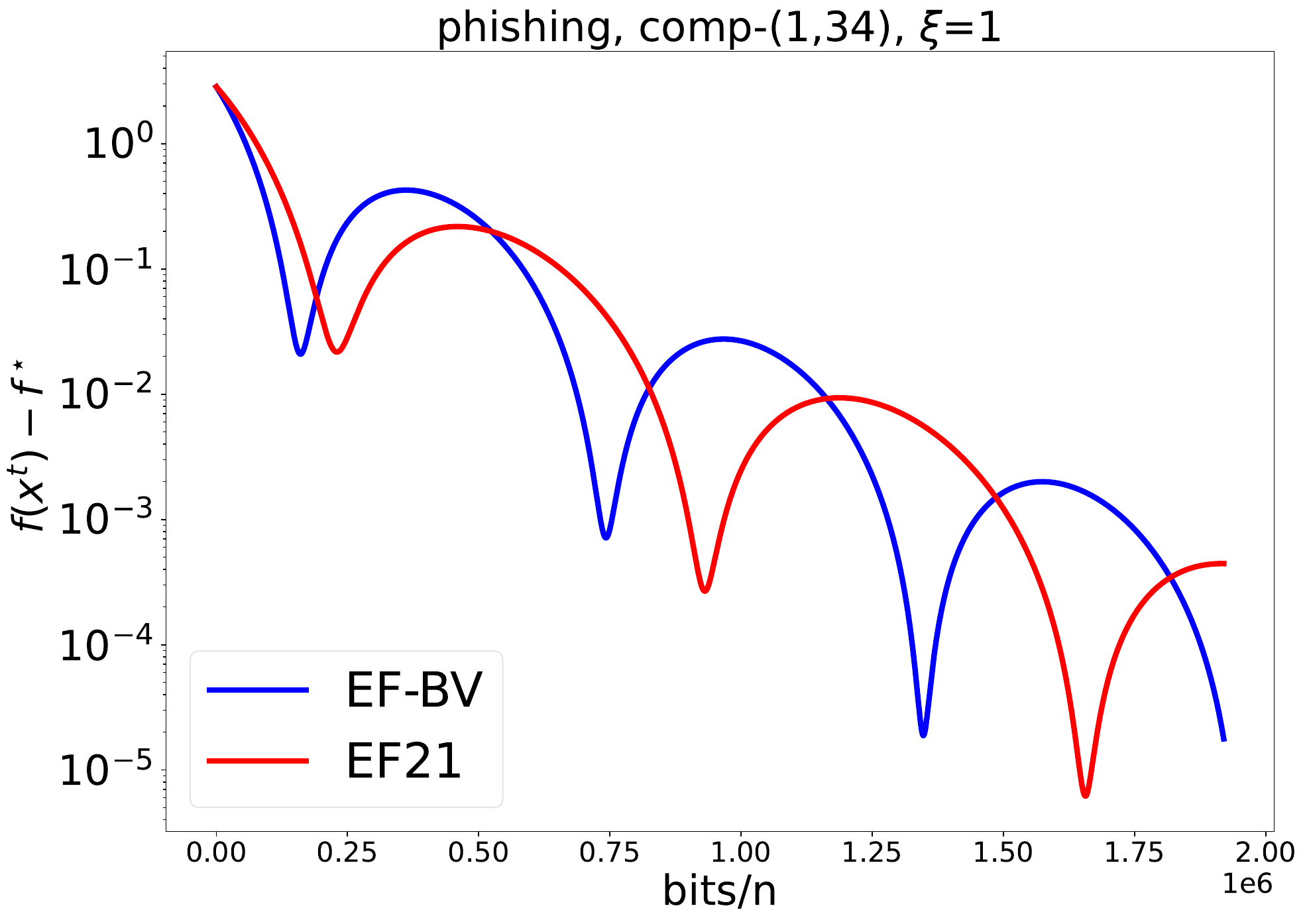}
	\end{subfigure}
	\hfill  
	\begin{subfigure}[b]{0.24\textwidth}
		\centering
		\includegraphics[width=\textwidth]{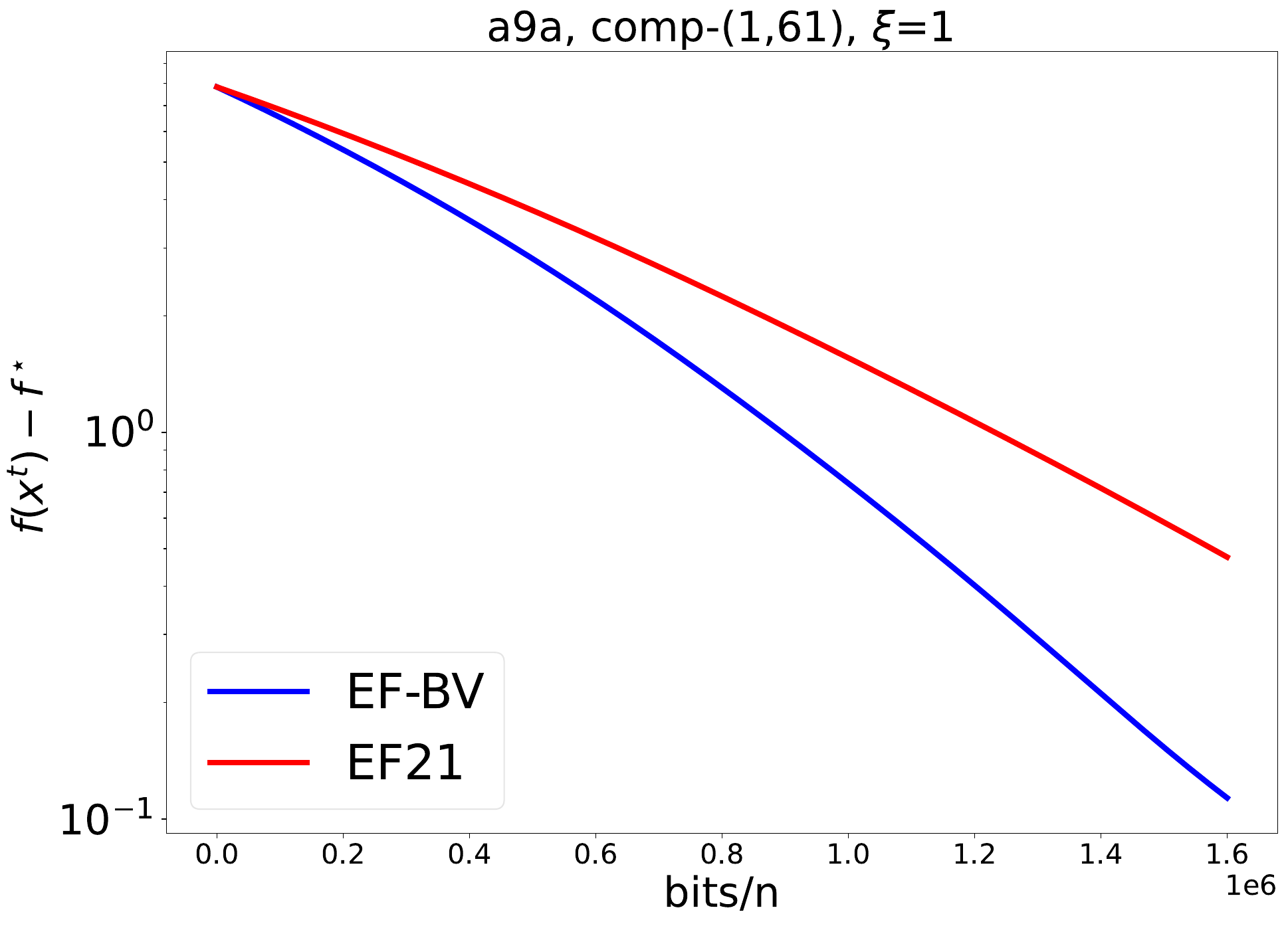}
	\end{subfigure}
	\begin{subfigure}[b]{0.24\textwidth}
		\centering
		\includegraphics[width=\textwidth]{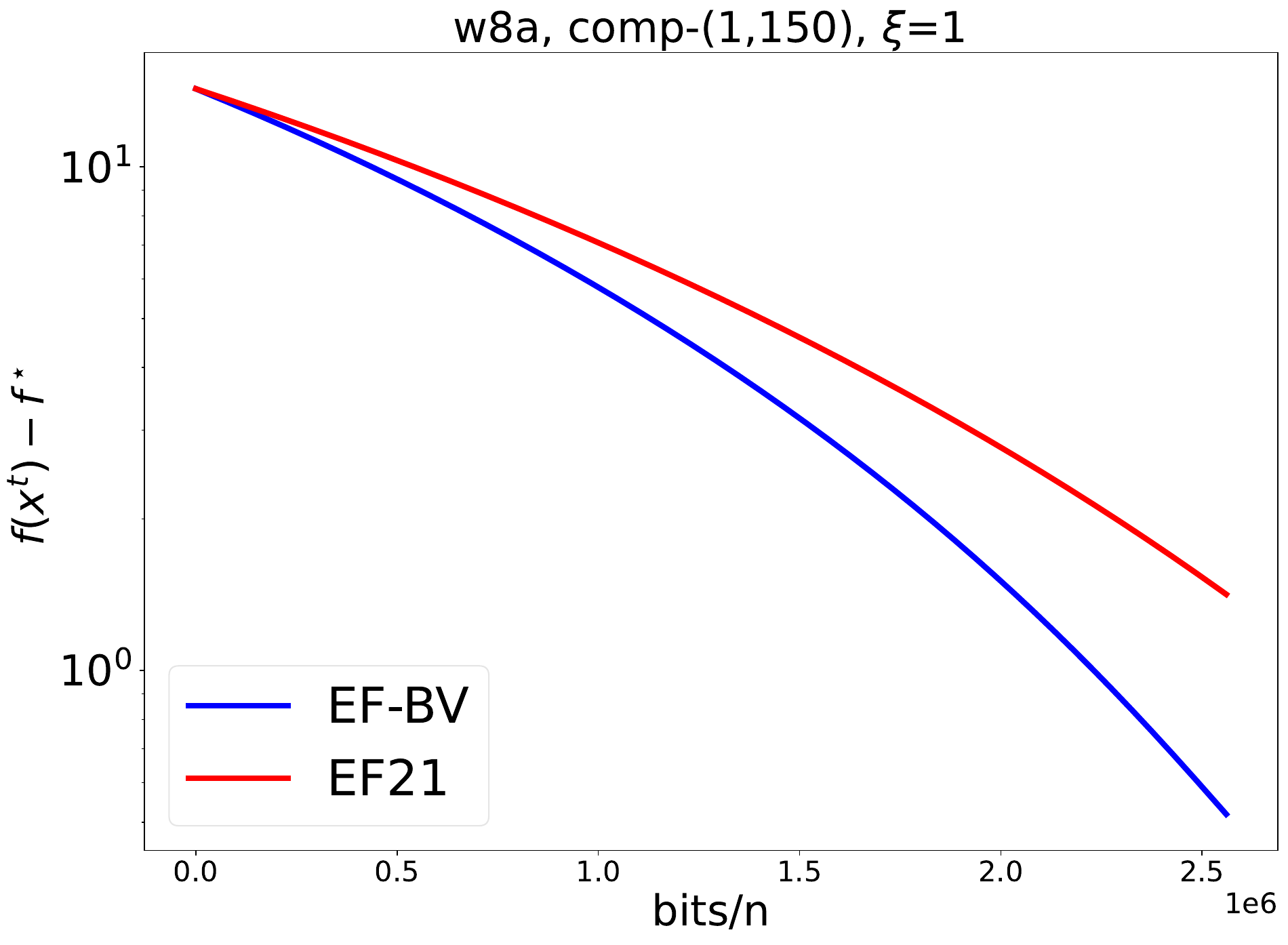}
	\end{subfigure}
	\hfill 
	\begin{subfigure}[b]{0.24\textwidth}
		\centering
		\includegraphics[width=\textwidth]{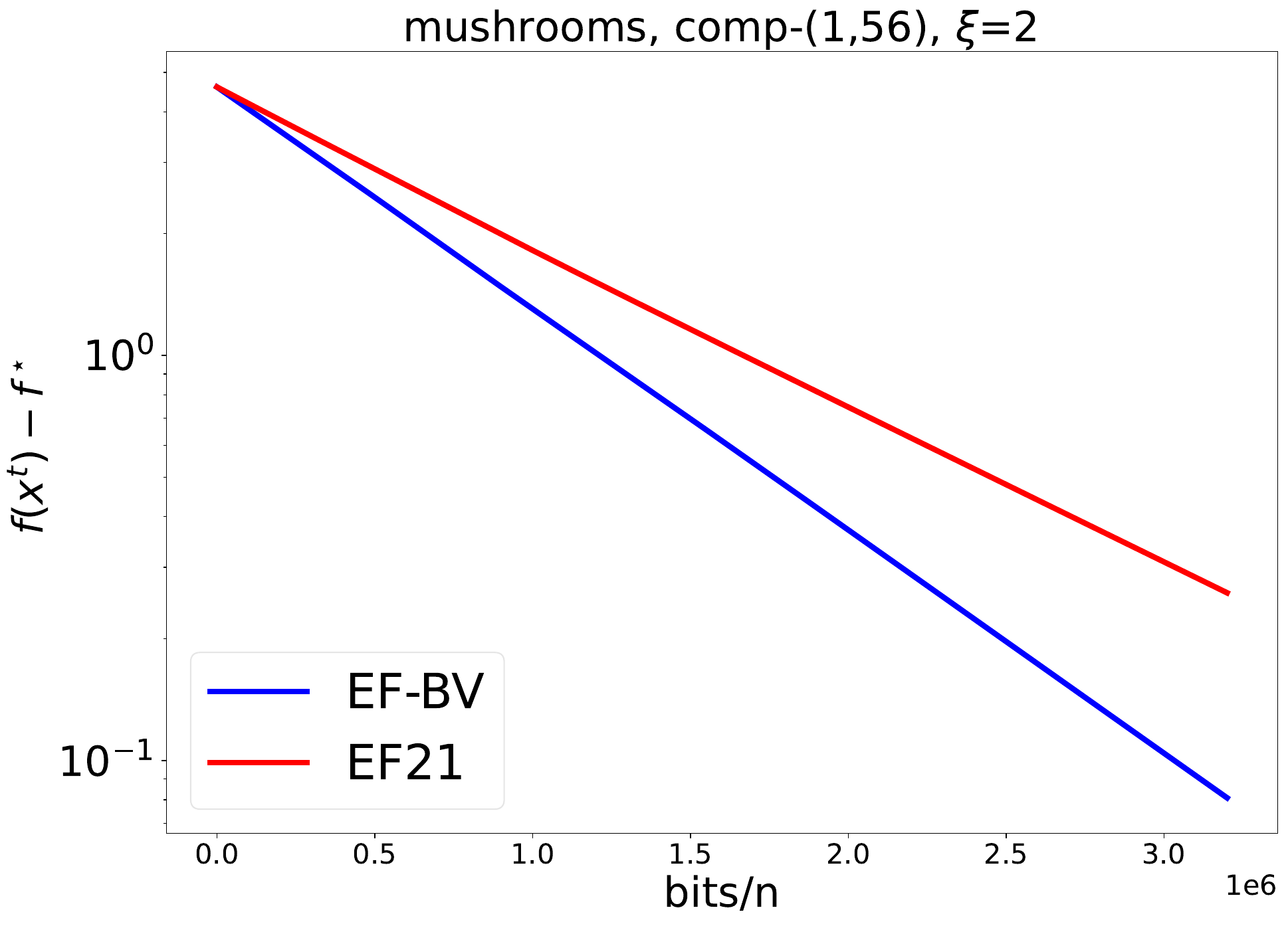}
	\end{subfigure}
	\hfill
	\begin{subfigure}[b]{0.24\textwidth}
		\centering
		\includegraphics[width=\textwidth]{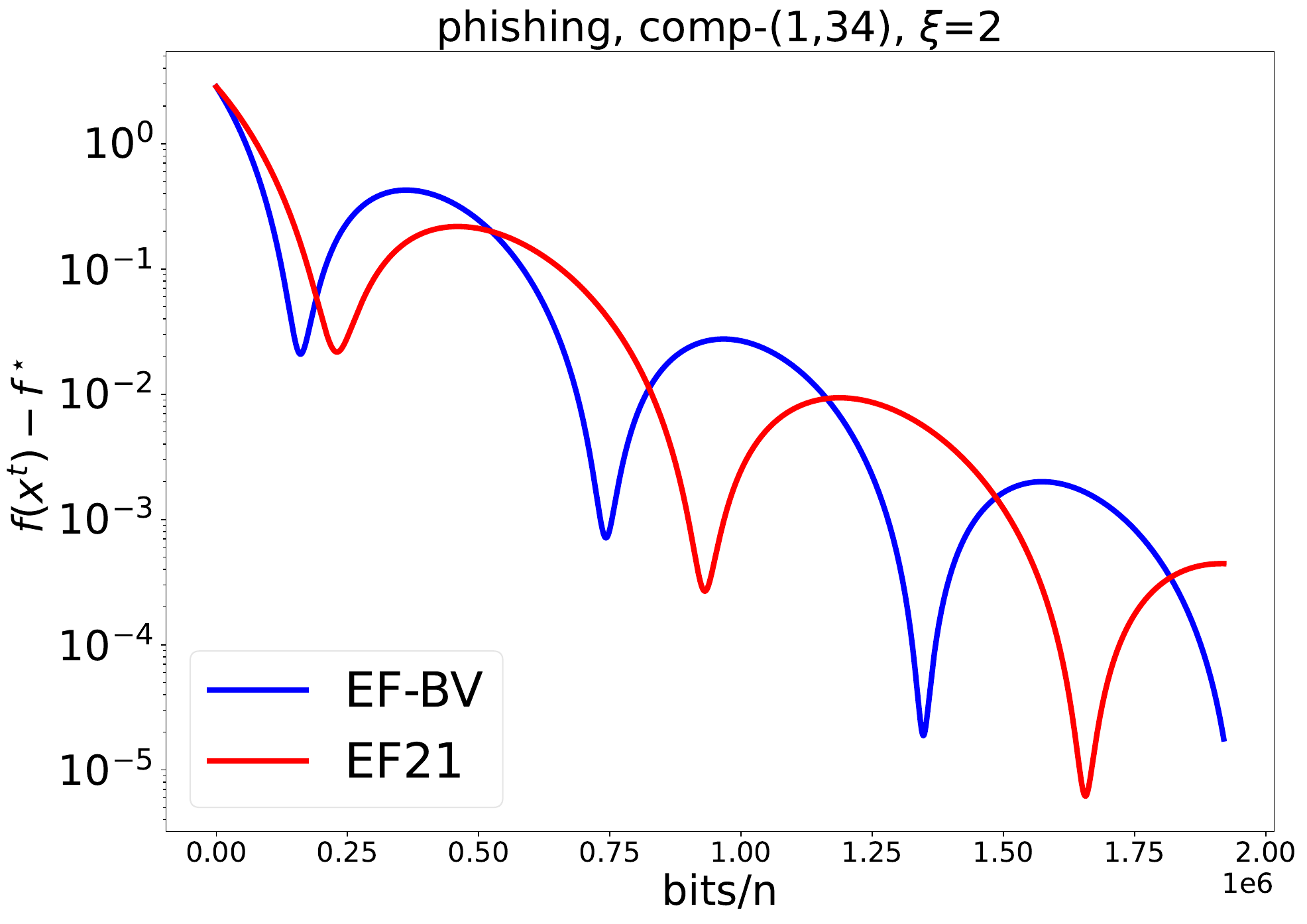}
	\end{subfigure}
	\hfill
	\begin{subfigure}[b]{0.24\textwidth}
		\centering
		\includegraphics[width=\textwidth]{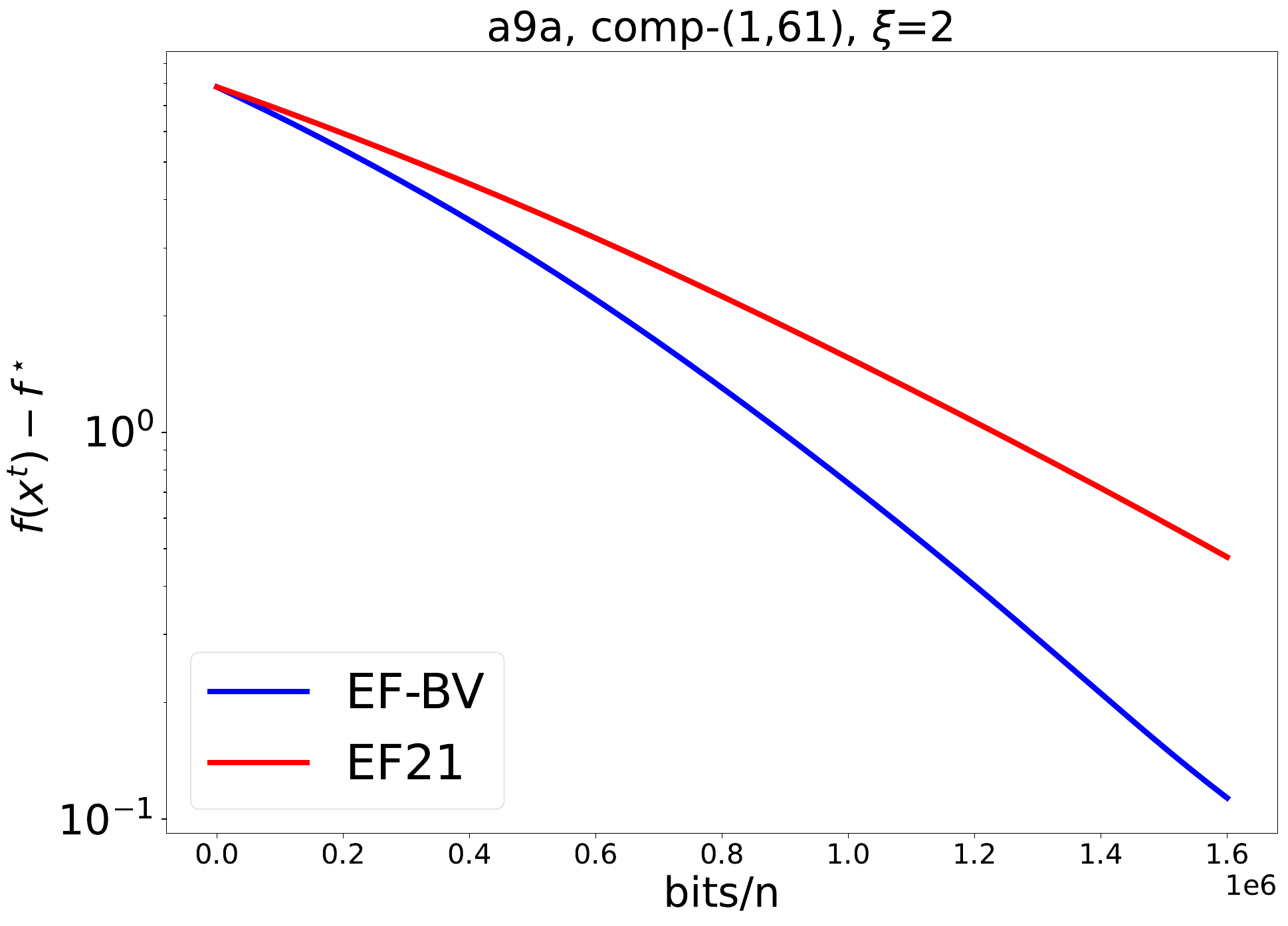}
	\end{subfigure}
	\begin{subfigure}[b]{0.24\textwidth}
		\centering
		\includegraphics[width=\textwidth]{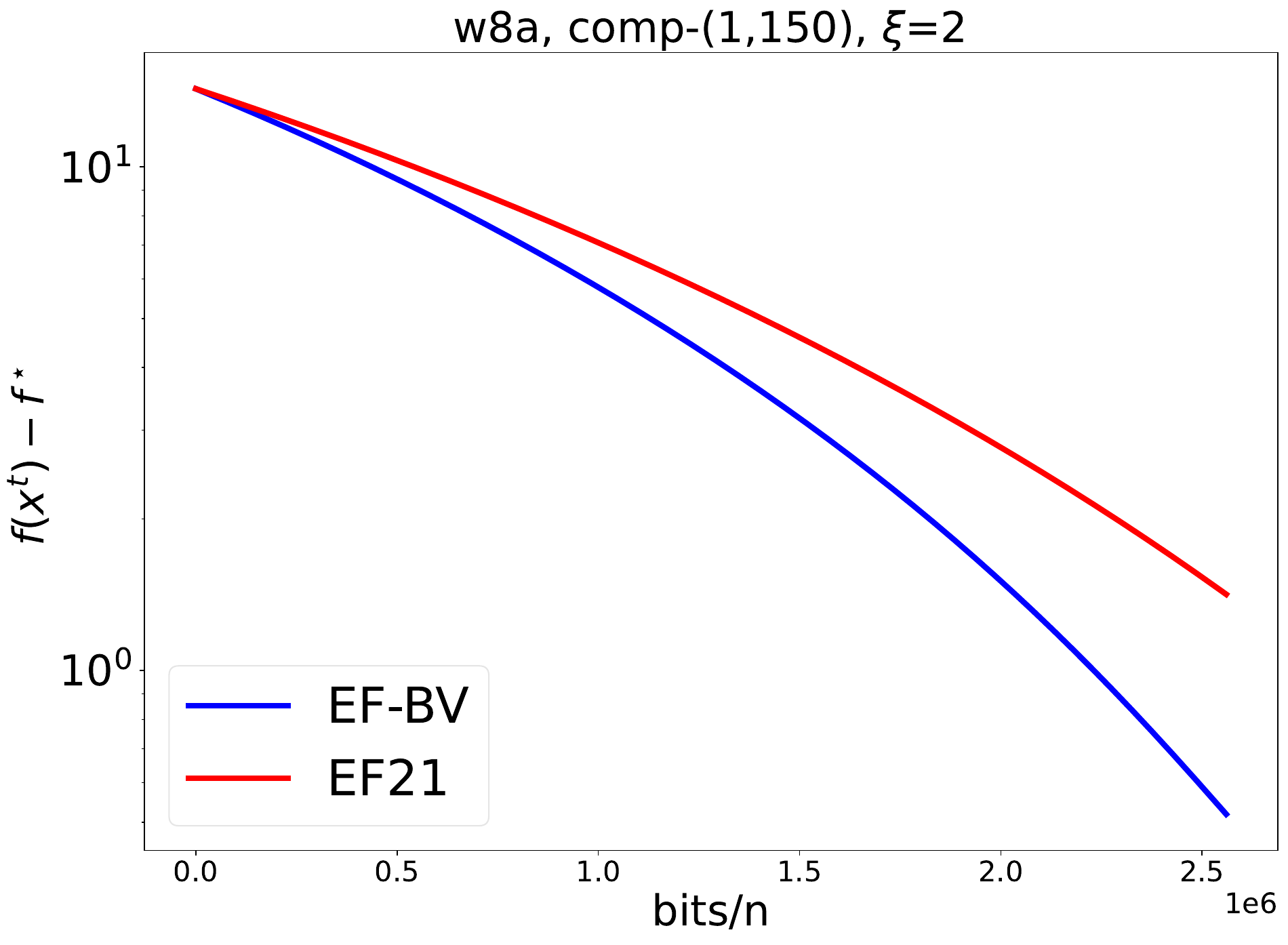}
	\end{subfigure}
		\hfill 
	\begin{subfigure}[b]{0.24\textwidth}
		\centering
		\includegraphics[width=\textwidth]{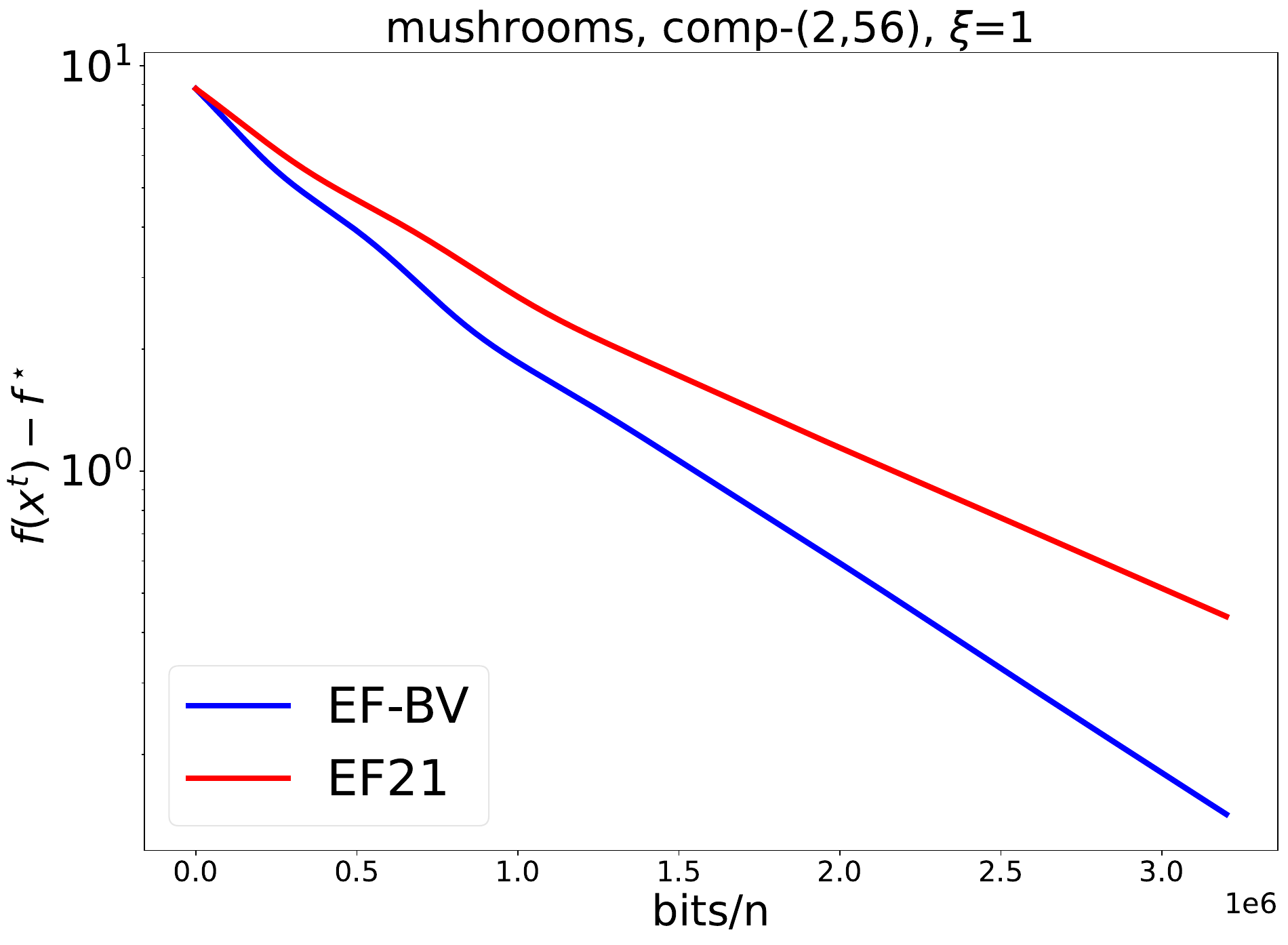}
	\end{subfigure}
	\hfill  
	\begin{subfigure}[b]{0.24\textwidth}
		\centering
		\includegraphics[width=\textwidth]{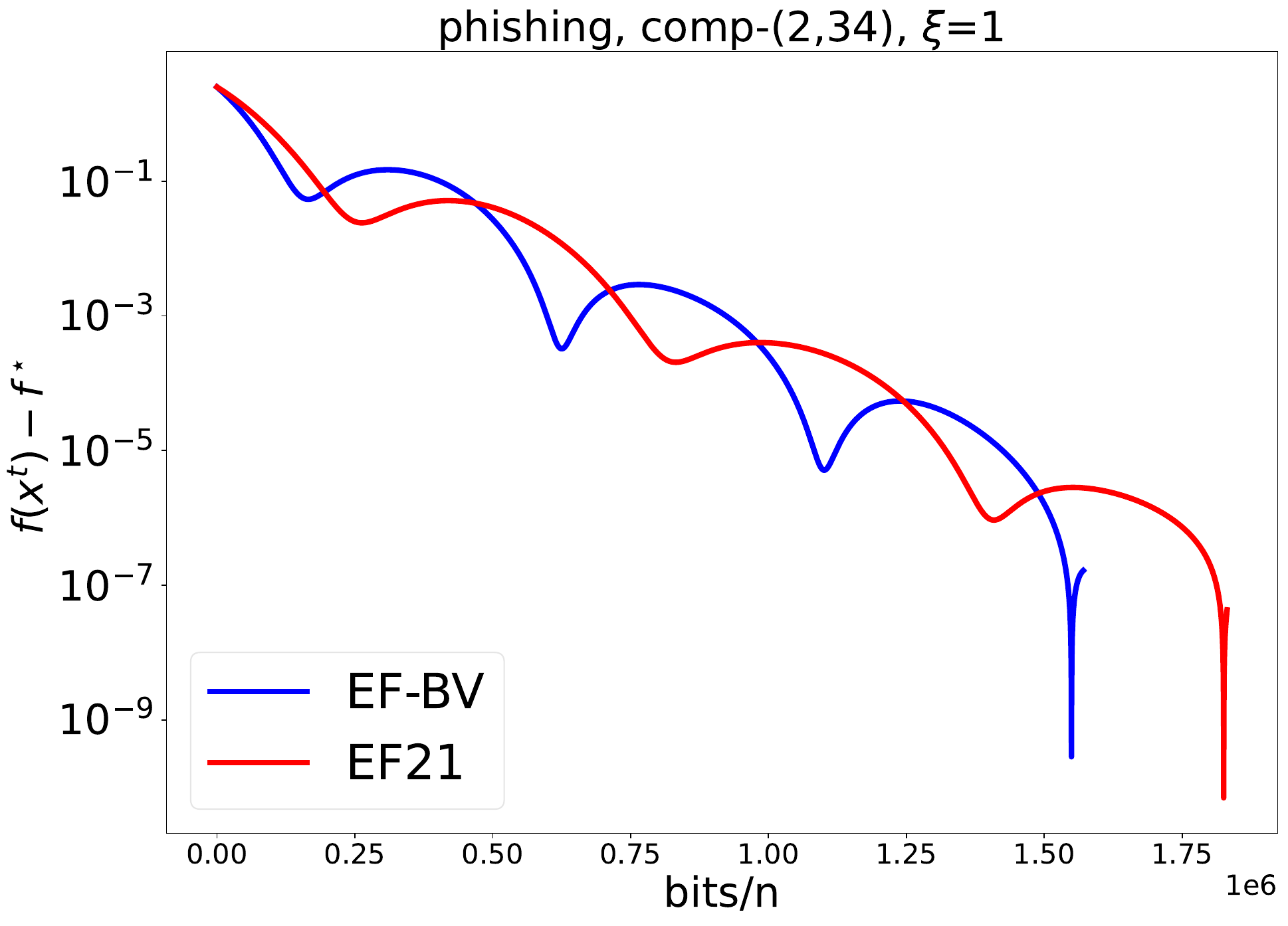}
	\end{subfigure}
	\hfill
	\begin{subfigure}[b]{0.24\textwidth}
		\centering
		\includegraphics[width=\textwidth]{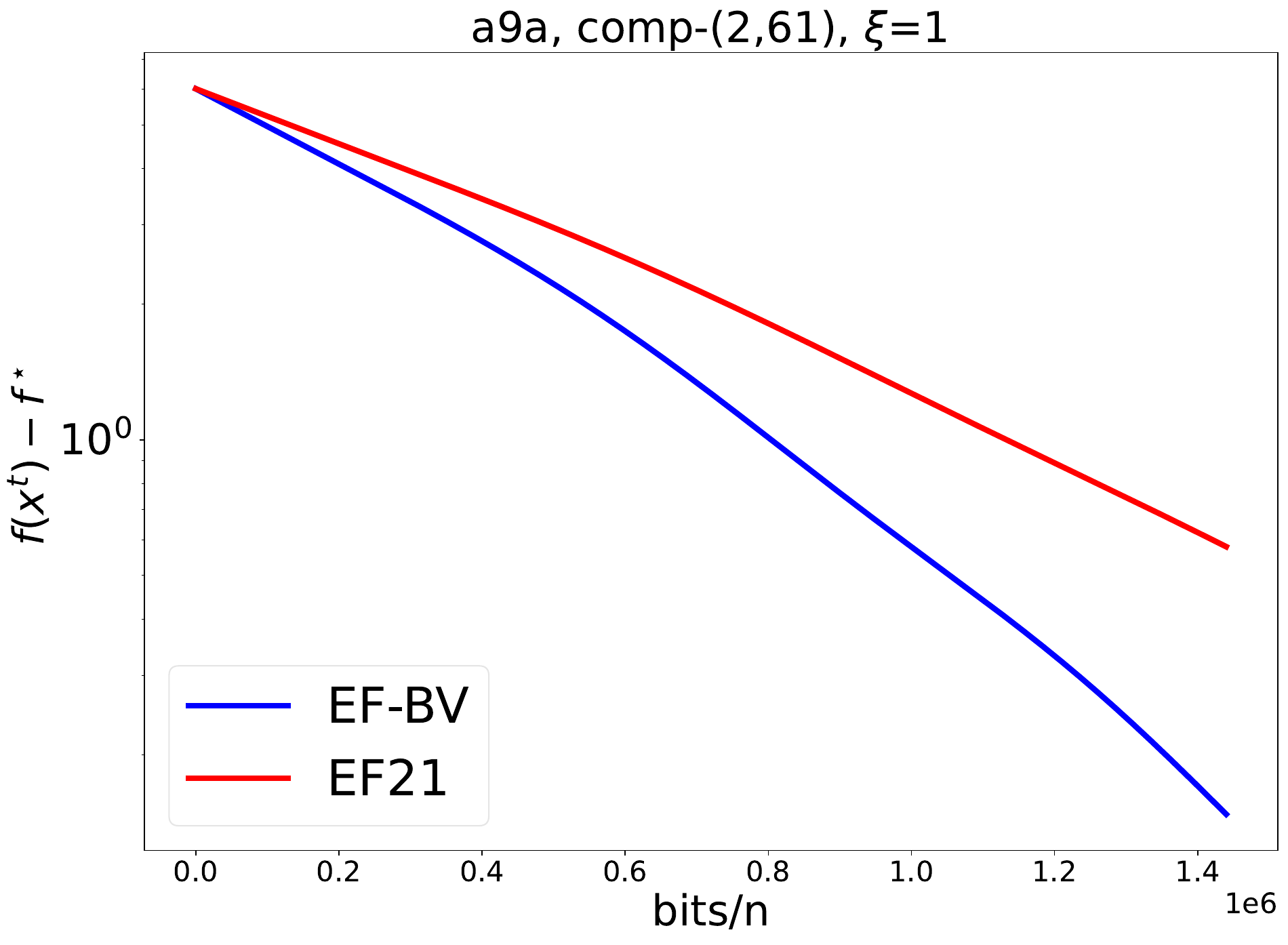}
	\end{subfigure}
	\begin{subfigure}[b]{0.24\textwidth}
		\centering
		\includegraphics[width=\textwidth]{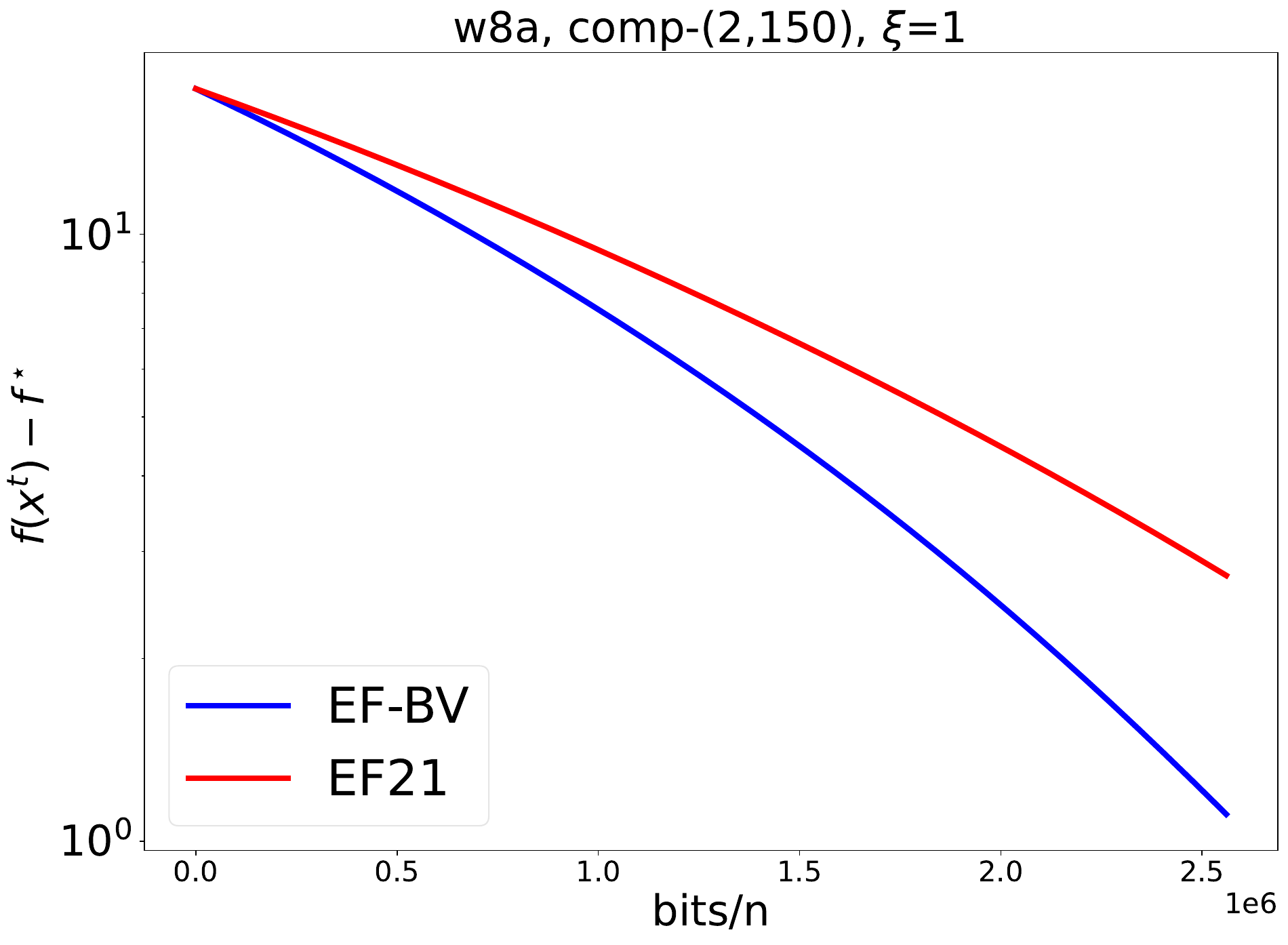}
	\end{subfigure}
	\caption{Experimental results. We plot $f(x^t)-f^\star$ with respect to the number of bits sent by each node during the learning process, which is proportional to $tk$.
	Top row: \texttt{comp-}$(1,d/2)$, overlapping $\xi=1$. Middle row: \texttt{comp-}$(1,d/2)$, overlapping $\xi=2$. Bottom row: \texttt{comp-}$(2,d/2)$, overlapping $\xi=1$.
}\label{fig:0007}
\end{figure}

%% file: Chapter_Scafflix.tex
\chapter{Accelerated Local Training with Explicit Personalization}
\label{chapter_scafflix}
\thispagestyle{empty}

\section{Introduction}
FL is classically formulated as an empirical risk minimization (ERM) problem, as defined in (\ref{eq:ERM}).
Thus, the usual approach is to solve (\ref{eq:ERM}) and then to deploy the obtained globally optimal model $\xstar \eqdef {\arg\min_{x\in \Rd}} \,f(x)$  to all clients. To reduce communication costs between the server and the clients, the practice of updating the local parameters multiple times before aggregation, known as \textbf{Local Training (LT)}~\citep{Povey2015, SparkNet2016, FL2017-AISTATS, Li-local-bounded-grad-norms--ICLR2020, LocalDescent2019, localGD, localSGD-AISTATS2020, SCAFFOLD, LSGDunified2020, FEDLIN}, is widely used in FL. 
LT, in its most modern form, is a communication-acceleration mechanism, as we detail in Section~\ref{sec:local_training}.

Meanwhile, there is a growing interest in providing \textbf{personalization} to the clients, by providing them more-or-less customized models tailored to their individual needs and heterogeneous data, instead of the one-size-fits-all model $x^\star$. We review existing approaches to personalization in Section~\ref{secper}. If personalization is pushed to the extreme, every client just uses its private data to learn its own locally-optimal model
\begin{equation*}
\xstar_i \eqdef {\arg\min_{x\in \Rd}}\, f_i(x)
\end{equation*}
and no communication at all is needed. Thus, intuitively, more personalization means less communication needed to reach a given accuracy. In other words, personalization is a communication-acceleration mechanism, like LT.

Therefore, we raise the following question:
 \emph{  
 Is it possible to achieve double communication acceleration in FL by jointly leveraging the acceleration potential of personalization and local training?}

For this purpose, we first have to formulate personalized FL as an optimization problem. 
A compelling interpretation of LT~\citep{hanzely2020federated} is that it amounts to solve an implicit personalization objective of the form:
\begin{equation}\label{eq:L2GD}
    \min_{x_1, \ldots, x_n\in \Rd} 
    \avein f_i(x_i) + \frac{\lambda}{2n}\sumin \sqn{\bar{x} - x_i},
\end{equation}
where $x_i \in \Rd$ denotes the local model at client $i\in [n]\eqdef \{1,\ldots,n\}$, $\bar{x} \eqdef \avein x_i$ is the average of these local models, and 
$\lambda \geq 0$ is the implicit personalization parameter that controls the amount of personalization. When $\lambda$ is small, the 
local models tend to be trained locally.
On the other hand, a larger $\lambda$ puts more penalty on making the local models $x_i$ close to their mean $\bar{x}$, or equivalently in making all models close to each other, by pushing towards averaging over all clients. Thus, LT is not only compatible with personalization, but can be actually used to implement it, though implicitly: there is a unique parameter $\lambda$ in \eqref{eq:L2GD} and it is difficult evaluate the amount of personalization for a given value of $\lambda$.

The more accurate FLIX model for personalized FL was proposed by \citet{FLIX}.
%
It consists for every client $i$ to first compute locally its personally-optimal model $\xstar_i$, 
and then to solve the problem
\begin{equation}\label{eq:FLIX}\tag{FLIX}
\min_{x\in \Rd} \,\tilde{f}(x) \eqdef \avein f_i\big(\alpha_i x+ (1 - \alpha_i)\xstar_i\big),
\end{equation}
where $\alpha_i \in [0,1]$ is the explicit and individual personalization factor for client $i$. At the end, the personalized model used by client $i$ is 
the explicit mixture 
\begin{equation*}
\tilde{x}_i^\star\eqdef \alpha_i \xstar+ (1 - \alpha_i)\xstar_i,
\end{equation*}
where $\xstar$ is the solution to (\ref{eq:FLIX}).
A smaller value of $\alpha_i$ gives more weight to $\xstar_i$, which means more personalization. On the other hand, if $\alpha_i=1$, 
the client $i$ uses the global model $\xstar$ without personalization. Thus, if all $\alpha_i$ are equal to 1, there is no personalization at all and (\ref{eq:FLIX}) reverts to (\ref{eq:ERM}). So, (\ref{eq:FLIX}) is a more general formulation of FL than (\ref{eq:ERM}). The functions  in (\ref{eq:FLIX}) inherit smoothness and strong convexity from
the $f_i$, so 
every algorithm appropriate for (\ref{eq:ERM}) can also be applied to solve (\ref{eq:FLIX}). \citet{FLIX} proposed an algorithm also called \algname{FLIX} to solve (\ref{eq:FLIX}), which is simply vanilla distributed gradient descent (\algname{GD}) applied to (\ref{eq:FLIX}).

In this paper, we first redesign and generalize the recent \algname{Scaffnew} algorithm~\citep{ProxSkip}, which features LT and has an accelerated communication complexity, and propose Individualized-Scaffnew (\algname{i-Scaffnew}), wherein the clients can have different properties. 
We then apply and tune \algname{i-Scaffnew} for the problem (\ref{eq:FLIX}) and propose our new algorithm for personalized FL, which we call \algname{Scafflix}. We answer positively to the above question and prove that \algname{Scafflix} enjoys a doubly accelerated communication complexity, by jointly harnessing 
the acceleration potential of LT and personalization. That is, its communication complexity depends on the square root of the condition number of the functions $f_i$ and on the $\alpha_i$. 
In addition to establishing the new state of the art  for personalized FL with our theoretical guarantees, we show by extensive experiments that \algname{Scafflix} is efficient in real-world learning setups and outperforms existing algorithms.

Our approach is novel and its good performance is built on a solid theoretical foundation. We stress that our convergence theorem for \algname{Scafflix} holds under standard assumptions, without bounded variance or any other restriction. 
By way of comparison with recent works, \algname{pFedGate} \citep{pFedGate} bases its theorem on the bounded diversity assumption, which is often unrealistic for non-iid FL. Neither \algname{FedCR}~\citep{FedCR} nor \algname{FedGMM}~\citep{FedGMM} comes with a conventional convergence theory. \algname{pFedGraph}~\citep{pFedGraph} and \algname{FED-PUB}~\citep{FED-PUB} also lack a solid convergence analysis. 

\begin{algorithm}[t]
	\caption{\algname{Scafflix} for (\ref{eq:FLIX})}
	\label{alg:scafflix}
	\begin{algorithmic}[1]
		\STATE \textbf{input:}  stepsizes $\gamma_1>0,\ldots,\gamma_n>0$; probability $p \in (0,1]$; initial estimates $x_1^0,\ldots,x_n^0 \in \mathbb{R}^d$ and ${\red h_1^0, \ldots, h_n^0 }\in \mathbb{R}^d$ such that $\sum_{i=1}^n {\red h_i^0}=0$, personalization weights $\alpha_1,\ldots,\alpha_n$
		\STATE at the server, $\gamma \eqdef \left(\frac{1}{n}\sum_{i=1}^n \alpha_i^2\gamma_i^{-1}\right)^{-1}$ 
		\hfill $\diamond$ {\small\color{gray} $\gamma$ is used by the server at Step 11}
		\STATE at clients in parallel, $x_i^\star\eqdef \arg\min f_i$\hfill $\diamond$ {\small\color{gray} not needed if $\alpha_i=1$}
		\FOR{$t=0,1,\ldots$}
		\STATE flip a coin $\theta^t \eqdef \{1$ with probability $p$, 0 otherwise$\}$
		\FOR{$i=1,\ldots,n$, at clients in parallel,}
		\STATE $\tilde{x}_i^t \eqdef \alpha_i x_i^t + (1-\alpha_i) x_i^\star$ \hfill $\diamond$ {\small\color{gray} estimate of the personalized model $\tilde{x}_i^\star$}
		\STATE compute an estimate $g_i^t$ of $\nabla f_i(\tilde{x}_i^t)$
\STATE $\hat{x}_i^t\eqdef x_i^t -\frac{\gamma_i}{\alpha_i} \big( g_i^t - {\red h_i^t}\big)$ 
\hfill $\diamond$ {\small\color{gray} local SGD step}
\IF{$\theta^t=1$}
\STATE send $\frac{\alpha_i^2}{\gamma_i}\hat{x}_i^t$ to the server, which aggregates $\bar{x}^t\eqdef \frac{\gamma}{n}\sum_{j=1}^n  \frac{\alpha_i^2}{\gamma_i}\hat{x}_{j}^t $ and broadcasts it to all clients \hfill $\diamond$ {\small\color{gray} communication, but only with small probability $p$}
\STATE $x_i^{t+1}\eqdef \bar{x}^{t}$
\STATE ${\red h_i^{t+1}}\eqdef {\red h_i^t} + \frac{p \alpha_i}{\gamma_i}\big(\bar{x}^{t}-\hat{x}_i^t\big)$\hfill $\diamond$ {\small\color{gray}update of the local control variate $\red h_i^t$}
\ELSE
\STATE $x_i^{t+1}\eqdef \hat{x}_i^t$
\STATE ${\red h_i^{t+1}}\eqdef {\red h_i^t}$
\ENDIF
\ENDFOR
		\ENDFOR
	\end{algorithmic}
\end{algorithm}

\section{Proposed algorithm and convergence analysis}

We generalize \algname{Scaffnew}~\citep{ProxSkip} and propose Individualized-Scaffnew (\algname{i-Scaffnew}), shown as Algorithm~\ref{alg1} in the Appendix. Its novelty with respect to \algname{Scaffnew} is to make use of different stepsizes $\gamma_i$ for the local SGD steps, in order to exploit the possibly different values of $L_i$ and $\mu_i$, as well as the different properties $A_i$ and $C_i$ of the stochastic gradients. This change is not straightforward and requires to rederive the whole proof with a different Lyapunov function and to formally endow $\mathbb{R}^d$ with a different inner product at every client. 

We then apply and tune \algname{i-Scaffnew} for the problem (\ref{eq:FLIX}) and propose our new algorithm for personalized FL, which we call \algname{Scafflix}, shown as Algorithm \ref{alg:scafflix}.

We analyze \algname{Scafflix} 
in the strongly convex case, 
because the analysis of linear convergence rates in this setting gives clear insights and allows us to deepen our theoretical understanding of LT and personalization. And to the best of our knowledge, there is no analysis of \algname{Scaffnew} in the nonconvex setting. 
But we conduct several nonconvex deep learning experiments  to show that our theoretical findings also hold in practice.



Our work builds upon the strong convexity assumption in \cref{def:convexity} and the smoothness assumption in \cref{def:smoothness}.
We also make the two following assumptions on the stochastic gradients $g_i^t$ used in \algname{Scafflix} (and \algname{i-Scaffnew} as a particular case with $\alpha_i\equiv 1)$.
\begin{assumption}[Unbiasedness]\label{ass:unbiasedness}
 We assume that for every $t\geq 0$ and $i\in[n]$, $g_i^t$ is an unbiased estimate of $\nabla f_i(\tilde{x}_i^t)$; that is,
%
    \begin{equation*}
        \Exp{g_{i}^t \;|\; \tilde{x}_i^{t}} = \nabla f_i (\tilde{x}_i^{t}).
    \end{equation*}
\end{assumption}

    
To characterize  unbiased stochastic gradient estimates, the modern notion of \emph{expected smoothness} is well suited  \citep{gow19,gor202}:
\begin{assumption}[Expected smoothness]\label{ass:expected_smoothness}
We assume that, for every $i\in [n]$, there exist constants
 $A_i\geq L_i $
 \footnote{We can suppose $A_i\geq L_i$. Indeed, we have the bias-variance decomposition $\Exp{\sqn{g_i^{t}- \nabla f_i(\tilde{x}_i^\star)}\;|\; \tilde{x}_i^{t}} =  \sqn{\nabla f_i(\tilde{x}_i^t)- \nabla f_i(\tilde{x}_i^\star)} + \Exp{\sqn{g_i^{t}- \nabla f_i(\tilde{x}_i^t)}\;|\; \tilde{x}_i^{t}}\geq \sqn{\nabla f_i(\tilde{x}_i^t)- \nabla f_i(\tilde{x}_i^\star)}$. Assuming that $L_i$ is the best known smoothness constant of $f_i$, we cannot improve the constant $L_i$ such that for every $x\in\mathbb{R}^d$, $\sqn{\nabla f_i(x)- \nabla f_i(\tilde{x}_i^\star)}\leq 2L_i D_{f_i}(x, \tilde{x}_i^\star)$. Therefore, $A_i$ in \eqref{ass:expected_smoothness_n} has to be $\geq L_i$.
}
 and $C_i \geq 0$ such that, for every $t\geq 0$,
    \begin{equation}\label{ass:expected_smoothness_n}
        \Exp{\sqn{g_i^{t}- \nabla f_i(\tilde{x}_i^\star)}\;|\; \tilde{x}_i^{t}} \leq 2A_i D_{f_i}(\tilde{x}_i^{t}, \tilde{x}_i^\star) + C_i,
    \end{equation}
    where $D_\varphi(x,x')\eqdef f(x)-f(x')-\langle \nabla f(x'),x-x'\rangle \geq 0$ denotes the Bregman divergence of a function $\varphi$ at points $x,x' \in\mathbb{R}^d$. 

%
\end{assumption}
Thus, unlike the analysis in~\citet[Assumption 4.1]{ProxSkip}, where the same constants are assumed for all clients, 
since we consider personalization, we individualize the analysis: we consider that each client can be different and use stochastic gradients characterized by its own constants $A_i$ and $C_i$. This is more representative of practical settings. 
%
Assumption~\ref{ass:expected_smoothness} is general and covers in particular the following two important cases  \citep{gow19}:
\begin{enumerate}
\item(bounded variance)\ \ If $g_i^{t}$ is equal to $\nabla f_i(\tilde{x}_i^t)$ plus a zero-mean random error of variance $\sigma_i^2$ (this covers the case of the exact gradient $g_i^{t}=\nabla f_i(\tilde{x}_i^t)$ with $\sigma_i=0$), then Assumption~\ref{ass:expected_smoothness} is satisfied with $A_i = L_i$ and $C_i = \sigma_i^2$.
\item(sampling)\ \  If $f_i=\frac{1}{n_i}\sum_{j=1}^{n_i}f_{i,j}$ for some $L_i$-smooth functions $f_{i,j}$ and $g_i^t = \nabla f_{i,j^t}(\tilde{x}_i^t)$ for some $j^t$ chosen uniformly at random in $[n_i]$, then Assumption~\ref{ass:expected_smoothness} is satisfied with $A_i = 2L_i$ and $C_i = \big(\frac{2}{n_i}\sum_{j=1}^{n_i}\sqnorm{\nabla f_{i,j}(\tilde{x}_i^\star)}\big)-2\sqnorm{\nabla f_{i}(\tilde{x}_i^\star)}$ (this can be extended to minibatch and nonuniform sampling).
\end{enumerate}




We now present our main convergence result:

\begin{theorem}[fast linear convergence]\label{theo2}
In (\ref{eq:FLIX}) and \algname{Scafflix}, suppose that Assumptions~\ref{def:convexity}, \ref{def:smoothness}, \ref{ass:unbiasedness}, \ref{ass:expected_smoothness} hold and that for every $i\in[n]$, $0<\gamma_i \leq \frac{1}{ A_i}$.
For every $t\geq 0$, define the Lyapunov function
\begin{align}
\Psi^{t}&\eqdef  
\frac{1}{n}\sum_{i=1}^n \frac{\gamma_{\min}}{\gamma_i} \sqnorm{\tilde{x}_i^t-\tilde{x}_i^\star}
+ \frac{\gamma_{\min}}{p^2}\frac{1}{n}\sum_{i=1}^n \gamma_i \sqnorm{h_i^t-\nabla f_i(\tilde{x}_i^\star)},\label{eqlya1j}
\end{align}
where $\gamma_{\min} \eqdef \min_{i\in[n]} \gamma_i$.
Then 
\algname{Scafflix}
converges linearly:  for every $t\geq 0$, 
\begin{equation}
\Exp{\Psi^{t}}\leq (1-\zeta)^t \Psi^0 + \frac{\gamma_{\min}}{\zeta} \frac{1}{n}\sum_{i=1}^n \gamma_i  C_i,\label{eqr1}
\end{equation}
where 
\begin{equation}
\zeta = \min\left(\min_{i\in[n]} \gamma_i\mu_i,p^2\right).\label{eqrate2j1}
\end{equation}
\end{theorem}

It is important to note that the range of the stepsizes $\gamma_i$, the Lyapunov function $\Psi^t$ and the convergence rate in \eqref{eqr1}--\eqref{eqrate2j1} do not depend on the personalization weights $\alpha_i$; they only play a role in the definition of the personalized models 
$\tilde{x}_i^t$ and $\tilde{x}_i^\star$. 
 Indeed, the convergence speed essentially depends on the conditioning of the functions $x\mapsto f_i\big(\alpha_i x + (1-\alpha_i) x_i^\star\big)$, which are independent from the $\alpha_i$. More precisely, let us define, for every $i\in[n]$,
\begin{equation*}
\kappa_i \eqdef \frac{L_i}{\mu_i} \geq 1\quad \mbox{and}\quad \kappa_{\max} = \max_{i\in[n]} \kappa_i,
\end{equation*}
and let us study the complexity of of \algname{Scafflix} to reach $\epsilon$-accuracy, i.e.\ $\Exp{\Psi^{t}}\leq \epsilon$.
If, for every $i\in[n]$, $C_i=0$, $A_i=\Theta(L_i)$, and $\gamma_i=\Theta(\frac{1}{A_i})=\Theta(\frac{1}{L_i})$, 
the iteration complexity of \algname{Scafflix} is
%
%
\begin{equation}
\mathcal{O}\left(\left(\kappa_{\max}+\frac{1}{p^2}\right)\log (\Psi^0\epsilon^{-1})
\right).
\end{equation}
And since communication occurs with probability $p$, the communication complexity of \algname{Scafflix} is
\begin{equation}
\mathcal{O}\left(\left(p\kappa_{\max}+\frac{1}{p}\right)\log (\Psi^0\epsilon^{-1})
\right).
\end{equation}
Note that $\kappa_{\max}$ can be much smaller than $\kappa_\mathrm{global}\eqdef\frac{\max_i L_i}{\min_i \mu_i}$, which is the condition number that appears in the rate of \algname{Scaffnew} with $\gamma = \frac{1}{\max_i A_i}$. Thus, \algname{Scafflix} is much more versatile and adapted to FL with heterogeneous data than \algname{Scaffnew}.

\begin{corollary}[case $C_i\equiv 0$]\label{cor1}
In the conditions of Theorem~\ref{theo2}, if $p=\Theta\big(\frac{1}{\sqrt{\kappa_{\max}}}\big)$ and, for every $i\in[n]$, $C_i=0$, $A_i=\Theta(L_i)$, and $\gamma_i=\Theta(\frac{1}{A_i})=\Theta(\frac{1}{L_i})$,  the communication complexity of \algname{Scafflix} 
is
\begin{equation}
\mathcal{O}\left(\sqrt{\kappa_{\max}}\log (\Psi^0\epsilon^{-1})
\right).
\end{equation}
\end{corollary}

\begin{corollary}[general stochastic gradients]\label{cor2}
In the conditions of Theorem~\ref{theo2}, if $p=\sqrt{\min_{i\in[n]} \gamma_i \mu_i}$ and, for every $i\in[n]$,
\begin{equation}
\gamma_i=\min \left(\frac{1}{A_i},\frac{\epsilon \mu_{\min}}{2 C_i} \right)
\end{equation}
(or $\gamma_i\eqdef \frac{1}{A_i}$ if $C_i=0$),
where $\mu_{\min}\eqdef \min_{j\in[n]} \mu_j$, the iteration complexity of \algname{Scafflix} is
\begin{equation}
	\begin{aligned}
		&\mathcal{O}\!\left(\!\left(\max_{i\in[n]}  \max\left(\frac{A_i}{\mu_i},\frac{C_i}{\epsilon \mu_{\min}\mu_i}\right)\right)\log(\Psi^0\epsilon^{-1})\!\right)\\
		&\!=\!\mathcal{O}\!\left(\!\max\left(\max_{i\in[n]}  \frac{A_i}{\mu_i},\max_{i\in[n]} \frac{C_i}{\epsilon \mu_{\min}\mu_i}\right)\log(\Psi^0\epsilon^{-1})\!\right)
	\end{aligned}
\end{equation}
and its communication complexity is
\begin{equation}
\mathcal{O}\left(\!\max\left(\max_{i\in[n]}  \sqrt{\frac{A_i}{\mu_i}},\max_{i\in[n]} \sqrt{\frac{C_i}{\epsilon \mu_{\min}\mu_i}}\right)\log(\Psi^0\epsilon^{-1})\right).
\end{equation}
\end{corollary}

If $A_i=\Theta(L_i)$ uniformly, we have $\max_{i\in[n]}  \sqrt{\frac{A_i}{\mu_i}} = \Theta(\sqrt{\kappa_{\max}})$. Thus, we see that thanks to LT, the communication complexity of \algname{Scafflix} is accelerated, as it depends on $\sqrt{\kappa_{\max}}$ and $\frac{1}{\sqrt{\epsilon}}$.

In the expressions above, the acceleration effect of personalization is not visible: it is ``hidden'' in $\Psi^0$, because every client computes $x_i^t$ but what matters is its personalized model $\tilde{x}_i^t$, and $\sqnorm{\tilde{x}_i^t-\tilde{x}_i^\star}=\alpha_i^2 \sqnorm{x_i^t-x^\star}$.  In particular, assuming that 
$x_1^0 = \cdots = x_n^0 = x^0$ and 
$h_i^0 = \nabla f_i(\tilde{x}_i^0)$, we have 
\begin{align*}
\Psi^{0} &\leq \frac{\gamma_{\min}}{n}\sqnorm{x^0-x^\star}\sum_{i=1}^n \alpha_i^2 \left(\frac{1}{\gamma_i}+\frac{\gamma_i L_i^2 }{p^2}\right)\\
&\leq \big(\max_i \alpha_i^2\big)\frac{\gamma_{\min}}{n}\sqnorm{x^0-x^\star}\sum_{i=1}^n  \left(\frac{1}{\gamma_i}+\frac{\gamma_i L_i^2 }{p^2}\right), 
\end{align*}
and we see that the contribution of every client to the initial gap $\Psi^0$ is weighted by $\alpha_i^2$. Thus,  the smaller the $\alpha_i$, the smaller $\Psi^0$ and the faster the convergence. This is why personalization is an acceleration mechanism in our setting.

\section{Experiments}

\begin{figure*}[t]
	\centering
	\begin{subfigure}[b]{0.32\textwidth}
		\centering
		\includegraphics[width=\textwidth]{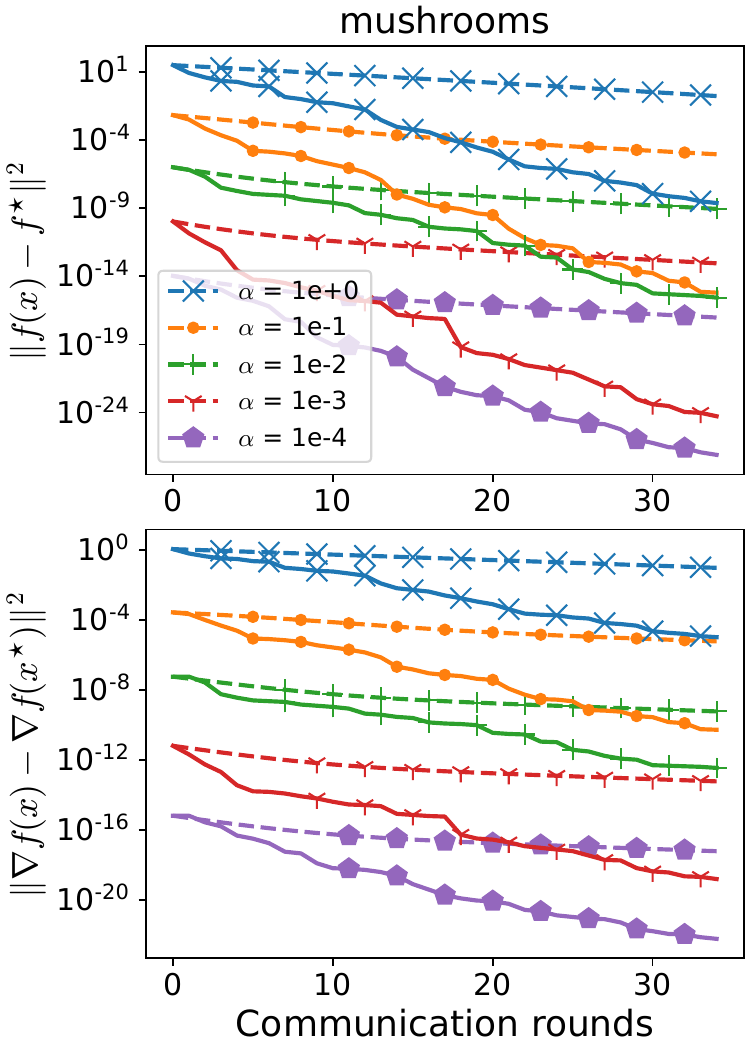}
	\end{subfigure}
	\hfill
	\begin{subfigure}[b]{0.32\textwidth}
		\centering
		\includegraphics[width=\textwidth]{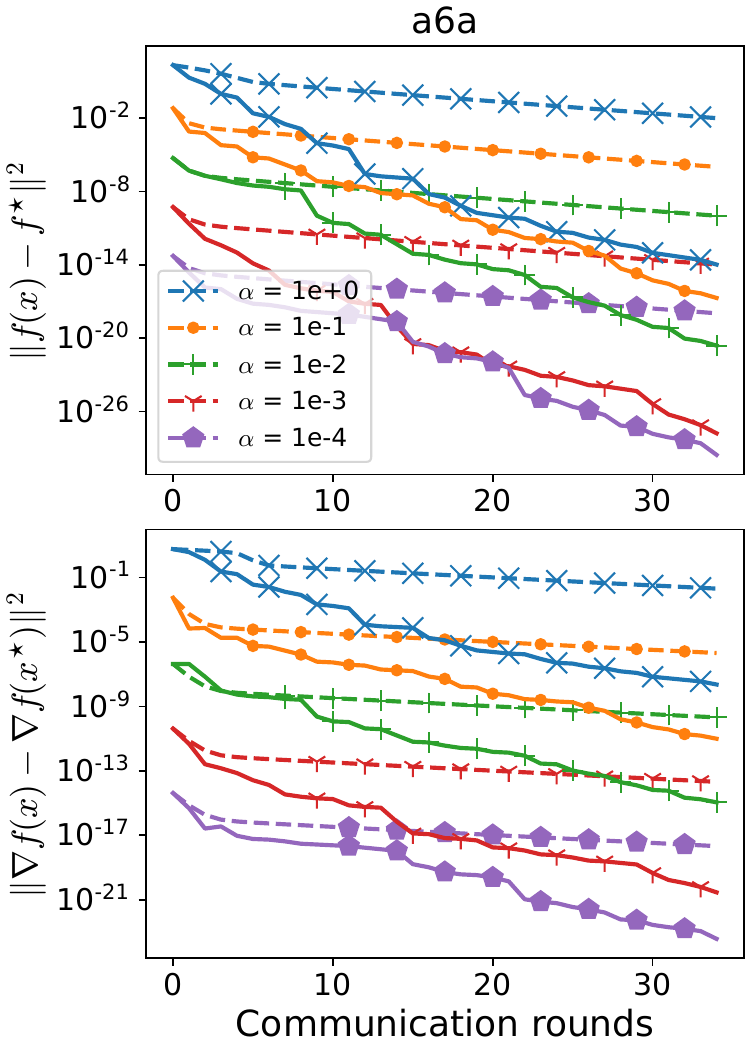}
	\end{subfigure}
	\hfill
	\begin{subfigure}[b]{0.32\textwidth}
		\centering
		\includegraphics[width=\textwidth]{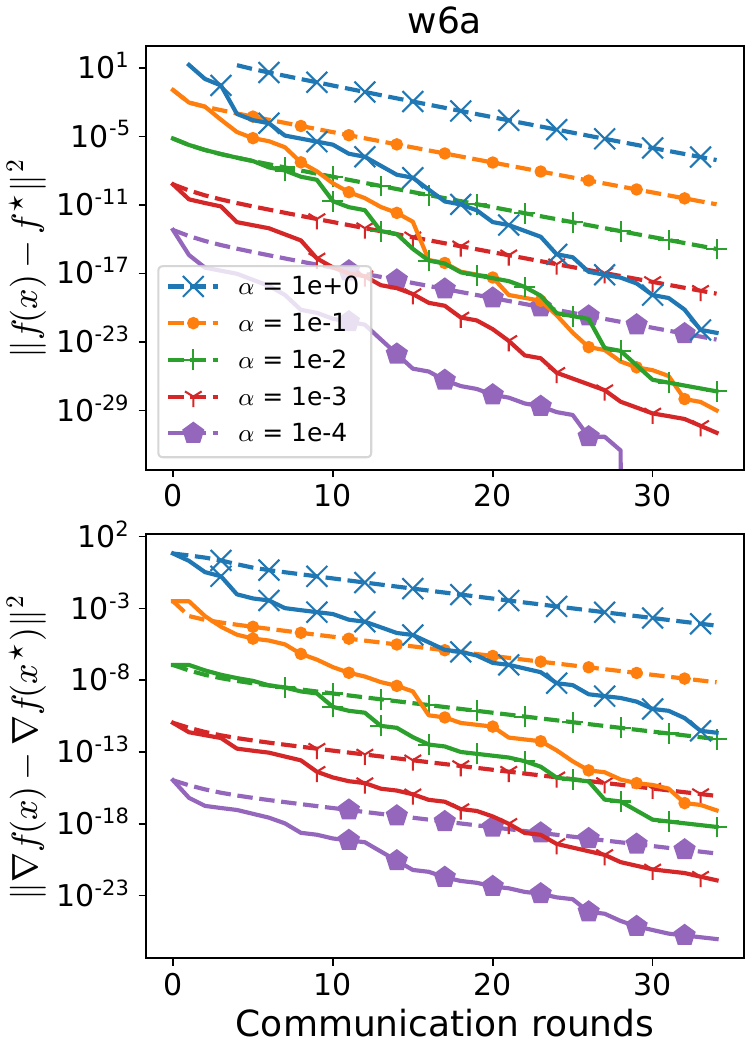}
	\end{subfigure}
	   \caption{The objective gap $f(x^k) - f^\star$ and the squared gradient norm $\sqn{\nabla f(x^k)}$ 
	   against the number $k$ of communication rounds for \algname{Scafflix} and \algname{GD} on the problem (\ref{eq:FLIX}) on class-wise non-iid FL setting. We set all $\alpha_i$ to the same value for simplicity. The dashed line represents \algname{GD}, while the solid line represents \algname{Scafflix}. We observe the double communication acceleration achieved through explicit personalization and local training. Specifically, (a) 
	   for a given algorithm,
	   smaller $\alpha_i$s (i.e.\ more personalized models) lead to faster convergence; (b) comparing 
	   the two algorithms,  \algname{Scafflix} is faster than \algname{GD}, thanks to its  local training mechanism.}
	   \label{fig:tissue_figure}
\end{figure*}

We first consider a convex logistic regression problem to show that the empirical behavior of \algname{Scafflix} is in accordance with the theoretical convergence guarantees available in the convex case. Then, we make extensive experiments of training neural networks on large-scale distributed datasets. 
 

\subsection{Prelude: convex logistic regression}
We begin our evaluation by considering the standard convex logistic regression problem with an $l_2$ regularizer. This benchmark problem is takes the form (\ref{eq:ERM}) with
\begin{equation*}
{f}_i{(x)} \eqdef \frac{1}{n_i}\sum_{j=1}^{n_i} \log \left(1 + \exp(-b_{i, j}x^T a_{i, j})\right) + \frac{\mu}{2}\sqn{x},
\end{equation*}
where $\mu$ represents the regularization parameter, $n_i$ is the total number of data points present at client $i$; $a_{i, j}$ are the training vectors and the $b_{i, j} \in \{-1, 1\}$ are the corresponding labels. Every function $f_i$ is $\mu$-strongly convex and $L_i$-smooth with $L_i=\frac{1}{4n_i}\sum_{j=1}^{n_i}\sqn{a_{i,j}} + \mu$. We set $\mu$ to $0.1$ for this experiment. 
We employ the \texttt{mushrooms}, \texttt{a6a}, and \texttt{w6a} datasets from the LibSVM library~\citep{chang2011libsvm} to conduct these tests. 
We consider several non-iid splits and present the results on feature-wise non-iid in Figure~\ref{fig:tissue_figure}. We discuss the difference among non-iid settings and complementary results in Appendix~\ref{sec:logistic_noniid}. 

The data is distributed evenly across all clients, and the $\alpha_i$  are set to the same value. The results are shown in Figure~\ref{fig:tissue_figure}. We can observe the double acceleration effect of our approach, which combines explicit personalization and accelerated local training. Lower $\alpha_i$ values, i.e.\ more personalization, yield faster convergence for both \algname{GD} and \algname{Scafflix}. Moreover, \algname{Scafflix} is much faster than \algname{GD}, thanks to its specialized local training mechanism.

\subsection{Neural network datasets and baselines}
To assess the generalization capabilities of \algname{Scafflix}, we undertake a comprehensive evaluation involving the training of neural networks using two widely-recognized large-scale FL datasets.

\textbf{Datasets.} Our selection comprises two notable large-scale FL datasets: Federated Extended MNIST (FEMNIST) \citep{leaf}, and Shakespeare \citep{FL2017-AISTATS}. FEMNIST is a character recognition dataset consisting of 671,585 samples. In line with the methodology described in FedJax~\citep{ro2021fedjax}, we distributed these samples across 3,400 devices, with each device exhibiting a naturally non-IID characteristic. For all algorithms, we employ a CNN model, featuring two convolutional layers and one fully connected layer. The Shakespeare dataset, used for next character prediction tasks, contains a total of 16,068 samples, which we distribute randomly across 1,129 devices. For all algorithms applied to this dataset, we use a RNN model, comprising two LSTM layers and one fully connected layer. 
    
\begin{figure}[!tb] 
	\centering
	\begin{subfigure}[b]{0.48\textwidth}
		\centering
		\includegraphics[trim=0 0 0 0, clip, width=\textwidth]{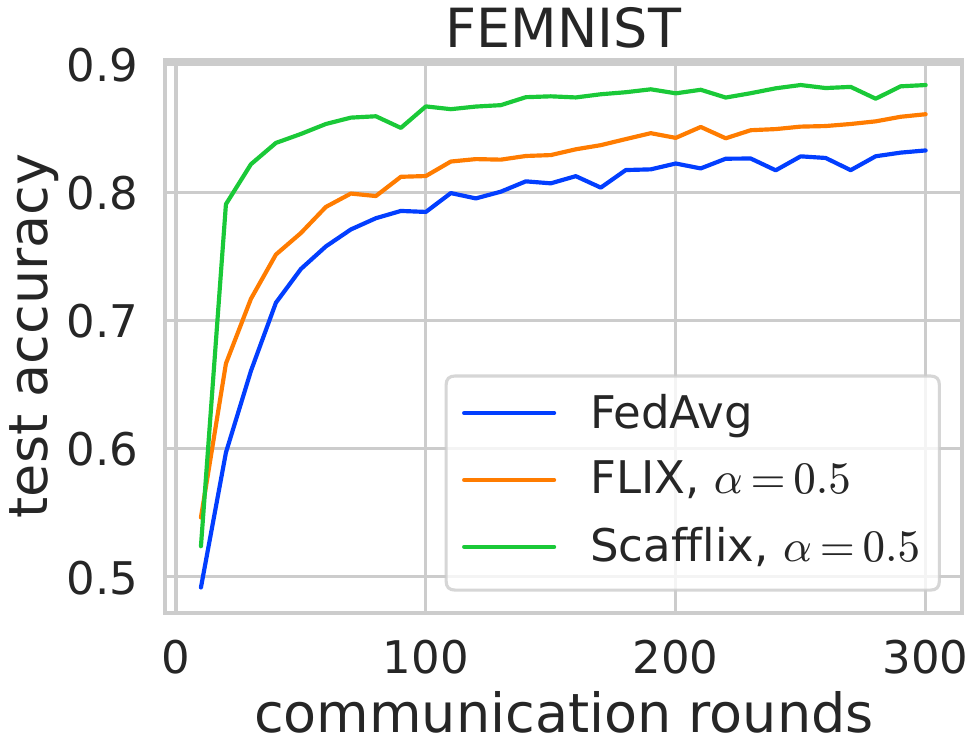}
	\end{subfigure}
	\hfill 
	\begin{subfigure}[b]{0.48\textwidth}
		\centering
		\includegraphics[trim=0 0 0 0, clip, width=\textwidth]{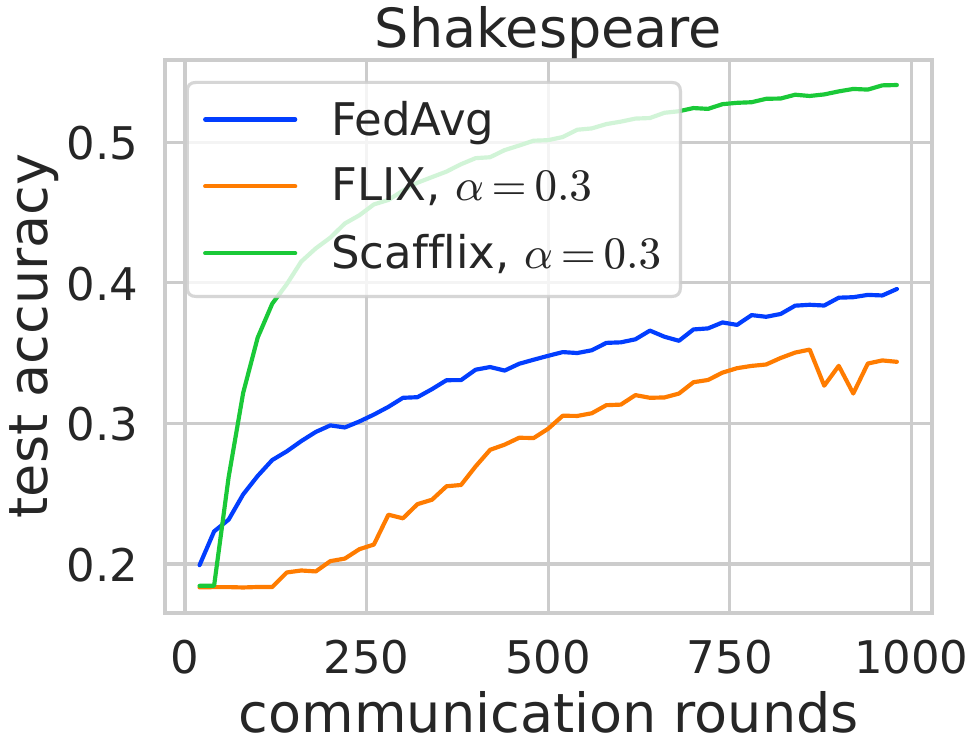}
		\end{subfigure}  
	\caption{Comparative generalization analysis with baselines. We set the communication probability to $p=0.2$. The left figure corresponds to the FEMNIST dataset with $\alpha=0.5$, while the right figure corresponds to the Shakespeare dataset with $\alpha=0.3$.}\label{fig:abs01}
	\vspace{-3.5mm}
\end{figure} 

\textbf{Baselines.} The performance of our proposed \algname{Scafflix} algorithm is benchmarked against prominent baseline algorithms, specifically \algname{FLIX}\citep{FLIX} and \algname{FedAvg}\citep{FedAvg2016}. The \algname{FLIX} algorithm optimizes the \ref{eq:FLIX} objective utilizing the \algname{SGD} method, while \algname{FedAvg} is designed to optimize the \ref{eq:ERM} objective. We employ the official implementations for these benchmark algorithms. Comprehensive hyperparameter tuning is carried out for all algorithms, including \algname{Scafflix}, to ensure optimal results.
For both \algname{FLIX} and \algname{Scafflix}, local training is required to achieve the local minima for each client. By default, we set the local training batch size at $100$ and employ \algname{SGD} with a learning rate selected from the set $C_s \eqdef \{10^{-5}, 10^{-4}, \cdots, 1\}$. Upon obtaining the local optimum, we execute each algorithm with a batch size of $20$ for 1000 communication rounds. The model's learning rate is also selected from the set $C_s$. All the experiments were conducted on a single NVIDIA A100 GPU with 80GB of memory.




\begin{figure*}[!t]
	\centering
	\begin{subfigure}[b]{0.325\textwidth}
		\centering
		\includegraphics[trim=0 0 0 0, clip, width=\textwidth]{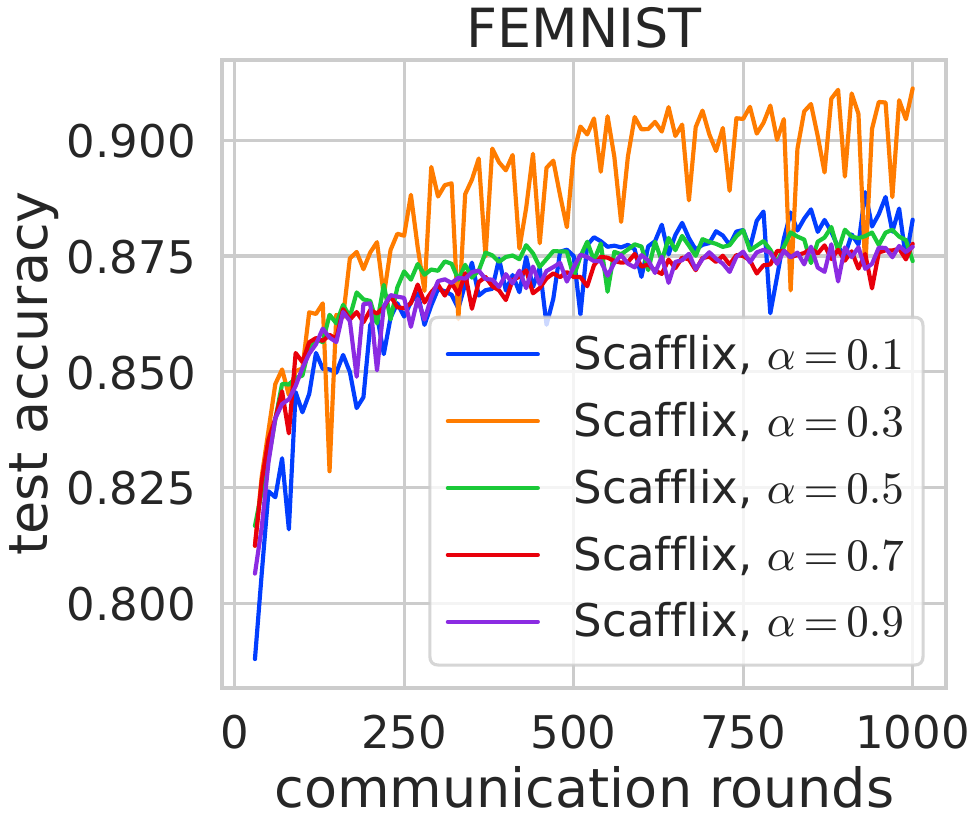}
		\caption{$\alpha$s}\label{fig:abs07_a}
	\end{subfigure}
	\begin{subfigure}[b]{0.325\textwidth}
		\centering  
		\includegraphics[trim=0 0 0 0, clip, width=\textwidth]{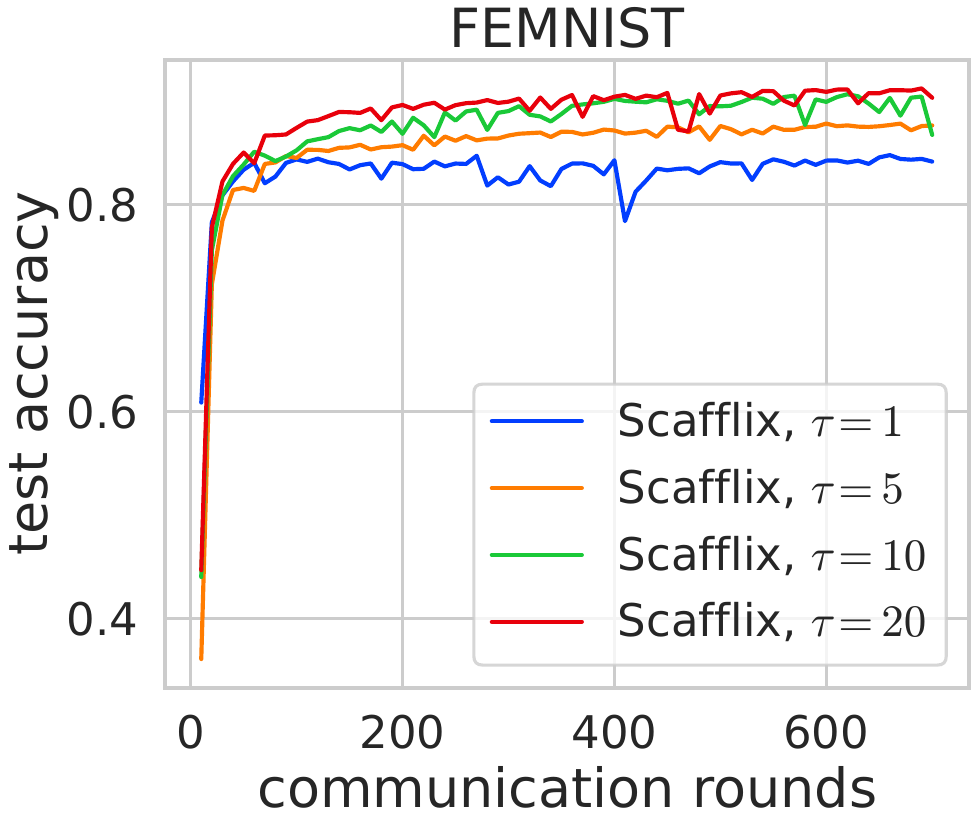}
		\caption{$\tau$s}\label{fig:abs07_b}
		\end{subfigure}
    \begin{subfigure}[b]{0.325\textwidth}
        \centering
        \includegraphics[trim=0 0 0 0, clip, width=\textwidth]{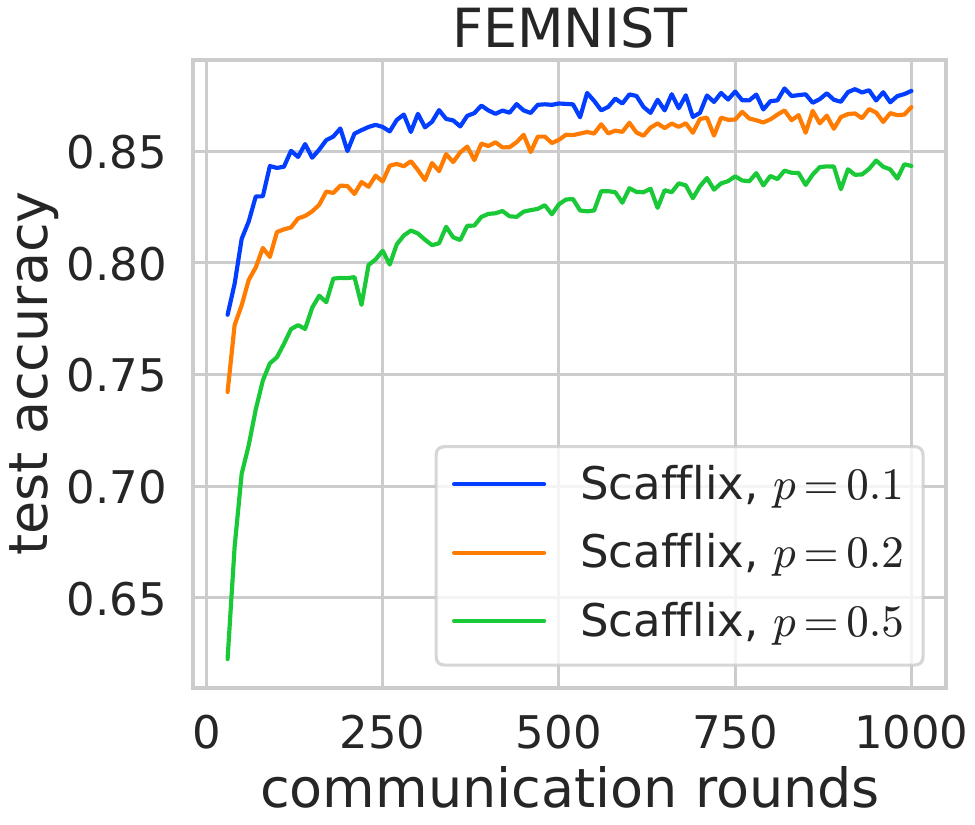}
        \caption{$p$s}\label{fig:abs07_c}
        \end{subfigure}
	\caption{Key ablation studies: (a) evaluate the influence on personalization factor $\alpha$, (b) examinate the effect of different numbers of clients participating to communication, (c) compare different values of the communication probability $p$.} 
	\label{fig:abs07}
\end{figure*} 
  
\subsection{Generalization analysis}
In this section, we perform an in-depth examination of the generalization performance of \algname{Scafflix}, particularly in scenarios with a limited number of training epochs. This investigation is motivated by our theoretical evidence of the double acceleration property of \algname{Scafflix}. 
To that aim, we conduct experiments on both FEMNIST and Shakespeare. These two datasets offer a varied landscape of complexity, allowing for a comprehensive evaluation of our algorithm. In order to ensure a fair comparison with other baseline algorithms, we conducted an extensive search of the optimal hyperparameters for each algorithm. The performance assessment of the generalization capabilities was then carried out on a separate, held-out validation dataset. The hyperparameters that gave the best results in these assessments were selected as the most optimal set. 

In order to examine the impact of personalization, we assume that all clients have same $\alpha_i \equiv \alpha$ and we select $\alpha$ in $\{0.1, 0.3, 0.5, 0.7, 0.9\}$. 
We present the results corresponding to $\alpha=0.1$ in Figure~\ref{fig:abs01}. Additional comparative analyses with other values of $\alpha$ are available in the Appendix. As shown in Figure~\ref{fig:abs01}, it is clear that \algname{Scafflix} outperforms the other algorithms in terms of generalization on both the FEMNIST and Shakespeare datasets.
Interestingly, the Shakespeare dataset (next-word prediction) poses a greater challenge compared to the FEMNIST dataset (digit recognition). Despite the increased complexity of the task, \algname{Scafflix} not only delivers significantly better results but also achieves this faster. 
Thus, \algname{Scafflix} is superior both in speed and accuracy.   

\begin{figure*}[!t]
	\centering
	\begin{minipage}[b]{0.302\textwidth}
	  \centering
	  \begin{subfigure}[b]{1.0\textwidth}
		\centering
		\includegraphics[trim=0 20 0 0, clip, width=\textwidth]{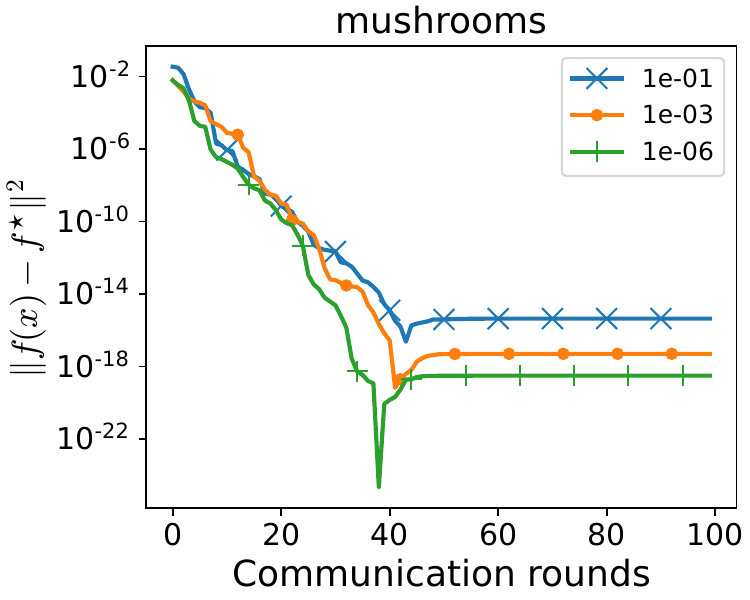}
	  \end{subfigure}
	  \begin{subfigure}[b]{1.0\textwidth}
		\centering
		\includegraphics[trim=0 0 0 20, clip, width=\textwidth]{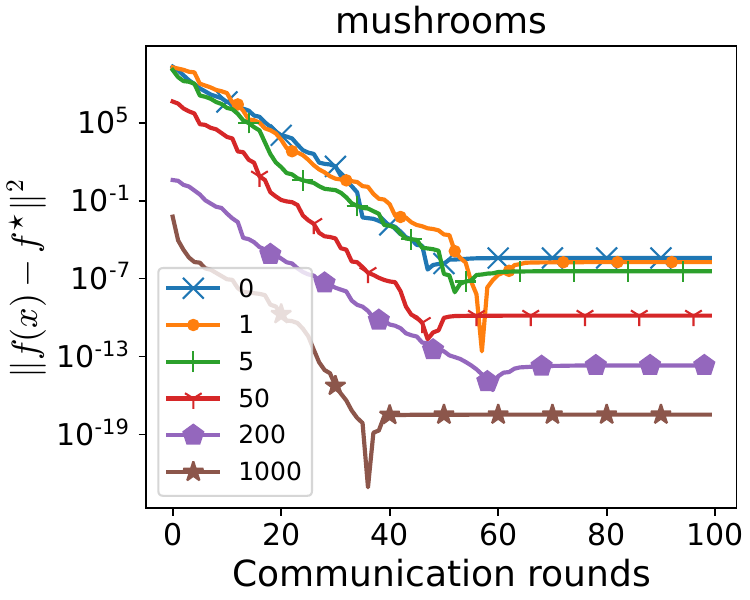}
	  \end{subfigure}
        \caption{Inexact local optimum approx.}
	  \label{fig:abs22}
	\end{minipage}
	\hfill
	\begin{minipage}[b]{0.67\textwidth}
	  \centering
	  \begin{subfigure}[b]{0.48\textwidth}
		\centering
		\includegraphics[trim=0 0 0 0, clip, width=\textwidth]{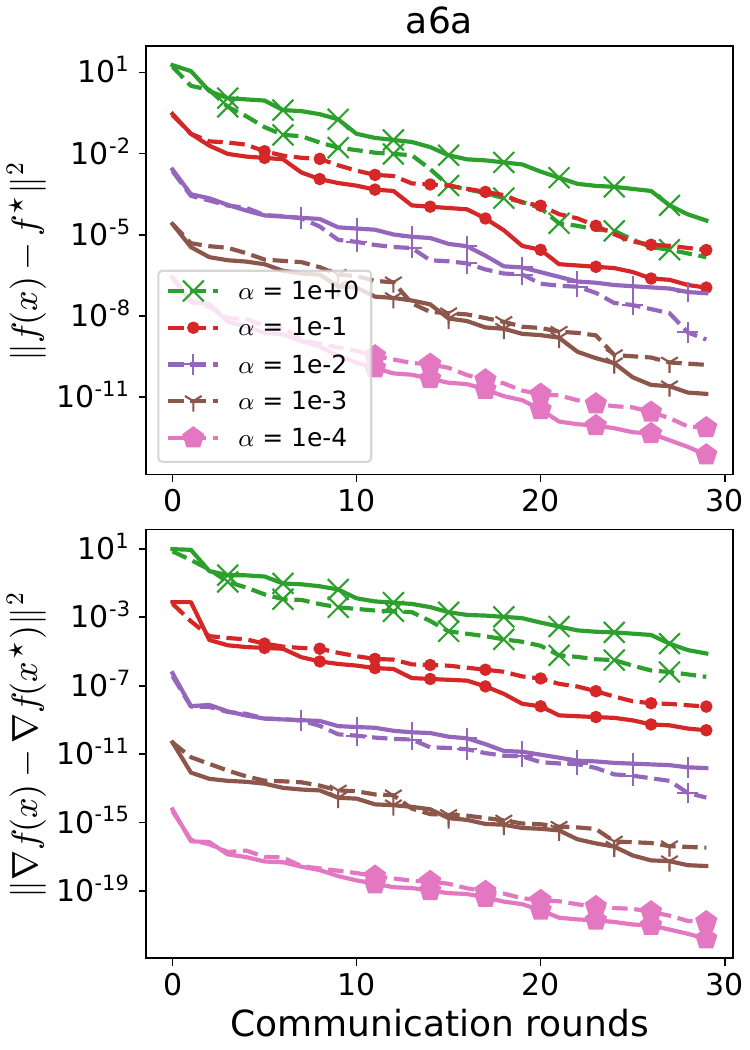}
	  \end{subfigure}  
	  \hfill
	  \begin{subfigure}[b]{0.48\textwidth}
		\centering
		\includegraphics[trim=0 0 0 0, clip, width=\textwidth]{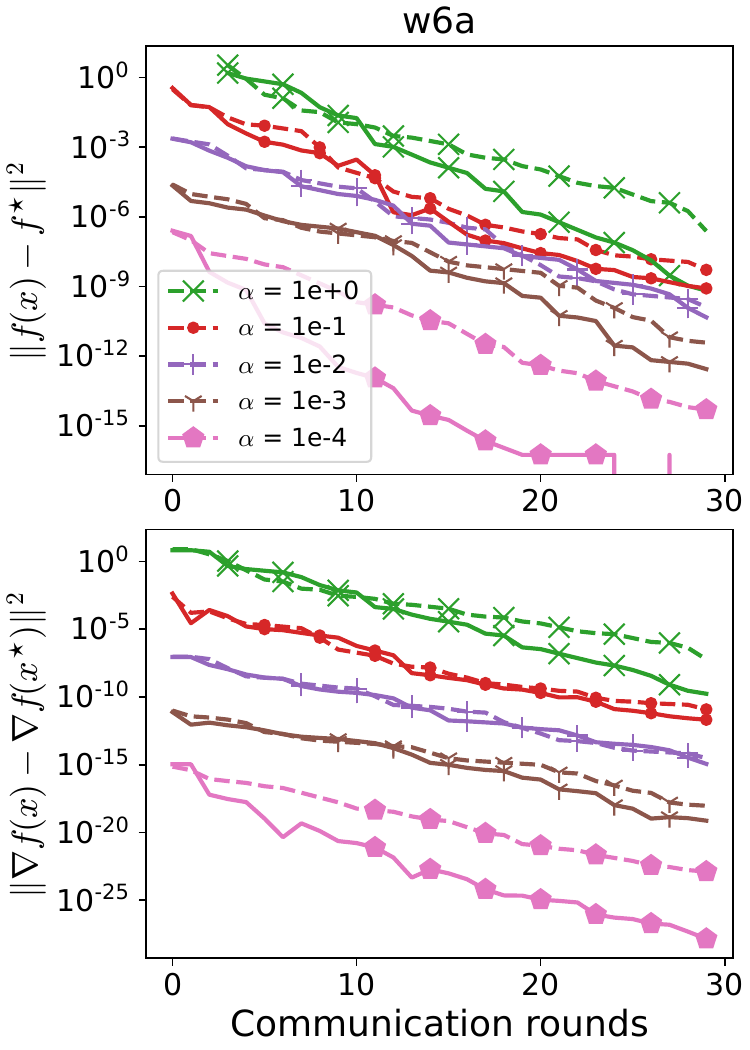}
	  \end{subfigure}
	  \caption{Comparison between global stepsize (dashed lines) and individual stepsizes (solid lines).}
	  \label{fig:abs10}
	\end{minipage}
  \end{figure*}

\subsection{Key ablation studies}
In this section, we conduct several critical ablation studies to verify the efficacy of our proposed \algname{Scafflix} method. These studies investigate the optimal personalization factor for \algname{Scafflix}, assess the impact of the number of clients per communication round, and examine the influence of the communication probability $p$ in \algname{Scafflix}.
  
\paragraph{Optimal personalization factor.}
In this experiment, we explore the effect of varying personalization factors on the FEMNIST dataset. The results are presented in Figure~\ref{fig:abs07_a}. We set the batch size to 128 and determine the most suitable learning rate through a hyperparameter search. We consider linearly increasing personalization factors within the set $\{0.1, 0.3, 0.5, 0.7, 0.9\}$. An exponential scale for $\alpha$ is also considered in the Appendix, but the conclusion remains the same.

We note that the optimal personalization factor for the FEMNIST dataset is $0.3$. Interestingly, personalization factors that yield higher accuracy also display a slightly larger variance. However, the overall average performance remains superior. This is consistent with expectations as effective personalization may emphasize the representation of local data, and thus, could be impacted by minor biases in the model parameters received from the server.
   
\paragraph{Number of clients communicating per round.}
In this ablation study, we examine the impact of varying the number of participating clients in each communication round within the \algname{Scafflix} framework. By default, we set this number to 10. Here, we conduct extensive experiments with different client numbers per round, choosing $\tau$ from  $\{1, 5, 10, 20\}$. The results are presented in Figure~\ref{fig:abs07_b}.
We can observe that \algname{Scafflix} shows that for larger batch sizes, specifically $\tau=10$ and $20$, demonstrate slightly improved generalization performance.

\paragraph{Selection of communication probability $p$.}
In this ablation study, we explore the effects of varying the communication probability $p$ in \algname{Scafflix}.
We select $p$ from  $\{0.1, 0.2, 0.5\}$, and the corresponding results are shown in Figure~\ref{fig:abs07_c}. 
We can clearly see that a smaller value of $p$, indicating reduced communication, facilitates faster convergence and superior generalization performance. This highlights the benefits of LT, which not only makes FL faster and more communication-efficient, but also improves the learning quality.

\paragraph{Inexact local Optimal.}o
In FL, the primary challenge lies in minimizing communication overhead while effectively managing local computation times. Attaining a satisfactory local optimum (or approximation) for each client is both practical and similar to \emph{pretraining} for finding a good initialization, a common practice in fields like computer vision and natural language processing. For instance, in our study of the Shakespeare dataset, distributed across 1,129 devices with over 16,000 samples, a mere \emph{50} epochs of local training per client were necessary to achieve optimal results, as demonstrated in Figure~\ref{fig:abs01}. This efficiency stands in stark contrast to traditional methods, which often require more than 800 communication rounds, each involving multiple local updates.

We further conducted detailed ablation studies on logistic regression to assess the impact of inexact local optimum approximation. A threshold was set such that $\|\nabla f_i(x)\|<\epsilon$ indicates a client has reached its local optimum, with the default $\epsilon$ set to $1e-6$. Our investigation focused on the consequences of using higher $\epsilon$ values. Appendix Figure~\ref{fig:abs11} details the expected number of local iterations for 100 clients. Notably, an $\epsilon$ value of $1e-1$ is found to be 23.55 times more efficient than $\epsilon = 1e-6$. Additional results for 8 workers with $\alpha=0.1$ are presented in Figure~\ref{fig:abs22}, showing that $\epsilon=1e-1$ provides a satisfactory approximation. (We anticipate an even lower computational cost for finding a local optimum approximation when the data per client is smaller.) Opting for $\epsilon=1e-1$ is a viable strategy to reduce computation, while smaller $\epsilon$ values are advantageous for greater precision. To ensure that our initial $x_i^0$ is not already near the optimum, we initialized each element of $x_i^0$ to 100. Additionally, we explored the number of local iterations required for achieving the optimal setting, ranging from $[0, 1, 5, 200, 1000]$, as depicted in the right panel of Figure~\ref{fig:abs22}. These findings underscore the need for a balance between performance and computational costs. More comprehensive insights and results are provided in Appendix~\ref{sec:inexact_approx_local_optimal}.


\paragraph{Individual stepsizes for each client.}
 In our experiments, we initially assumed a uniform learning rate for all clients for simplicity. However, to more accurately represent the personalized approach of our method and to align closely with Algorithm~\ref{alg:scafflix}, we explored different stepsizes for each client. Specifically, we set $\gamma_i = 1 / L_i$, where $L_i$ denotes the smoothness constant of the function $f_i$ optimizing (\ref{eq:FLIX}). The impact of this variation is demonstrated in Figure~\ref{fig:abs10}, which presents results using the mushrooms dataset. We observed that employing individual stepsizes generally enhances performance. This approach, along with a global stepsize (indicated by dashed lines in the figure), both contribute to improved outcomes.

%% file: Chapter_2_FedP3.tex
\chapter{Federated Personalized Privacy-friendly Pruning}
\label{chapter_fedp3}
\thispagestyle{empty}

\section{Introduction}


Standard FL is typically formulated as an optimization problem, specifically the Empirical Risk Minimization defined in \Cref{eq:ERM}. To better reflect that our focus is on neural networks, we reformulate the objective as:  
\begin{equation}\label{eqn:objective1}
    \min_{W \in \mathbb{R}^d} f(W) \eqdef \frac{1}{n} \sum_{i=1}^n f_i(W),
\end{equation}  
where \(W\) represents the shared global network parameters, \(f_i(W)\) denotes the local objective for client \(i\), and \(n\) is the total number of clients.

Distinguishing it from conventional distributed learning, FL predominantly addresses heterogeneity stemming from both data and model aspects. 
Data heterogeneity characterizes the fact that the local data distribution across clients can vary widely. 
Such variation is rooted in real-world scenarios where clients or users exhibit marked differences in their data, reflective of the variety of sensors or software~\cite{jiang2020federated}, of users' unique preferences, etc.~\cite{li2020federated2}. 
Recent works~\cite{zhao2018federated} showed how detrimental the non-iidness of the local data could be on the training of a FL model. 
This phenomenon known as client-drift, is intensively studied to develop methods limiting its impact on the performance~\citep{karimireddy2020scaffold, Gao_2022_CVPR, Mendieta_2022_CVPR}.
  
Furthermore, given disparities among clients in device resources, e.g., energy consumption, computational capacities, memory storage or network bandwidths, model heterogeneity becomes a pivotal consideration. 
To avoid restricting the global model's architecture to the largest that is compatible with all clients, recent methods aim at reducing its size differently for each client to extract the utmost of their capacities.
This can be referred to as constraint-based local model personalization~\citep{gao2022survey}.
In such a context, clients often train a pruned version of the global model~\citep{jiang2022model,HeteroFL} before transmitting it to the server for aggregation~\citep{li2021model}. 
A contemporary and influential offshoot of this is Independent Subnetwork Training (IST)~\citep{yuan2022distributed}. 
It hinges on the concept that each client trains a subset of the main server-side model, subsequently forwarding the pruned model to the server. 
Such an approach significantly trims local computational burdens in FL~\citep{dun2023efficient}. 

Our research, while aligning with the IST premise, brings to light some key distinctions. A significant observation from our study is the potential privacy implications of continuously sending the complete model back to the server. Presently, even pruned networks tend to preserve the overarching structure of the global model. In this paper, we present an innovative approach to privacy-friendly pruning. Our method involves transmitting only select segments of the global model back to the server. This technique effectively conceals the true structure of the global model, thus achieving a delicate balance between utility and confidentiality.
As highlighted in \cite{zeiler2014visualizing}, different layers within networks demonstrate varied capacities for representation and semantic interpretation. The challenge of securely transferring knowledge from client to server, particularly amidst notable model heterogeneity, is an area that has not been thoroughly explored. { It's pertinent to acknowledge that the concept of gradient pruning as a means of preserving privacy was initially popularized by the foundational work of \cite{zhu2019deep}. Following this, studies such as \cite{huang2020privacy} have further investigated the efficacy of DNN pruning in maintaining privacy.}

Besides, large language models (LLMs) have garnered significant attention and have been applied to a plethora of real-world scenarios~\citep{brown2020language, chowdhery2022palm, touvron2023llama} recently. 
However, the parameter count of modern LLMs often reaches the billion scale, making it challenging to utilize user or client information and communicate within a FL framework. 
We aim to explore the feasibility of training a more compact local model and transmitting only a subset of the global network parameters to the server, while still achieving commendable performance.
   
From a formulation standpoint, our goal is to optimize the following objective, thereby crafting a global model under conditions of model heterogeneity:

\begin{equation}
    \min_{W_1, \dotsc, W_n\in \Rd} f(W) \eqdef h\left( f_1(W_1), f_2(W_2), \dotsc, f_n(W_n)\right) \enspace,
\end{equation}
where $W_i$ denotes the model downloaded from client $i$ to the server, which can differ as we allow global pruning or other sparsification strategies. The global model $W$ is a function of $\{W_1, W_2, \dotsc, W_n\}$, $f_i$ the local objective for client $i$ and $n$ the total number of clients. 
    Function $h$ is the aggregation mapping from the clients to the server. 
In conventional FL, it's assumed that function $h$ is the average and all $W_1=\dotsc W_n = W$, which means the full global model is downloaded from the server to every client. When maintaining a global model $W$, this gives us $f(x) \eqdef \frac{1}{n} \sum_{i=1}^{n} f_i(W)$, which aligns with the standard empirical risk minimization (ERM).

\subsection{Summary of contributions}
In this paper, we introduce an efficient and adaptable federated network pruning framework tailored to address model heterogeneity. 
The main contributions of our framework, denoted as \algname{FedP3} (\textbf{Fed}erated \textbf{P}ersonalized and \textbf{P}rivacy-friendly network \textbf{P}runing) algorithm, are:

{
\noindent $\bullet$ \textit{Versatile framework:} 
Our framework allows personalization based on each client's unique constraints (computational, memory, and communication).

\noindent $\bullet$ \textit{Dual-pruning method:} Incorporates both global (server to client) and local (client-specific) pruning strategies for enhanced efficiency.

\noindent $\bullet$ \textit{Privacy-friendly approach:} Ensures privacy-friendly to each client by limiting the data shared with the server to only select layers post-local training.

\noindent $\bullet$ \textit{Managing heterogeneity:} Effectively tackles data and model diversity, supporting non-iid data distributions and various client-model architectures.

\noindent $\bullet$ \textit{Theoretical interpretation:} Provides a comprehensive analysis of global pruning and personalized model aggregation. Discusses convergence theories, communication costs, and the advantages over existing methodologies.

\noindent $\bullet$ \textit{Local differential-privacy algorithm:} Introduces \algname{LDP-FedP3}, a novel local differential privacy algorithm. Outlines privacy guarantees, utility, and communication efficiency.
}

\section{Approach}

We focus on the training of neural networks within the  FL paradigm. Consider a global model 
\begin{equation*}
    W \coloneqq \{W^0, W^1, \dots, W^L, W^{\mathrm{out}}\} \enspace,
\end{equation*}
where \(W^0\) represents the weights of the input layer, \(W^{\mathrm{out}}\) the weights of the final output layer, and \(L\) the number of hidden layers. 
Each \(W^l\), for all \(l \in \mathcal{L} \coloneqq \{0, 1, \dots, L\}\), denotes the model parameters for layer \(l\). 
We distribute the complete dataset \(X\) across \(n\) clients following a specific distribution, which can be non-iid. 
Each client then conducts local training on its local data denoted by \(X_i\). 

\paragraph{Algorithmic overview.} In Algorithm~\ref{alg:framework}, we introduce the details of our proposed general framework called \textbf{Fed}erated \textbf{P}ersonalized and \textbf{P}rivacy-friendly network \textbf{P}runing (\algname{FedP3}).
For every client \(i \in [n]\), we assign predefined pruning mechanisms \(P_i\) and \(Q_i\), determined by the client's computational capacity and network bandwidth (see Line~\ref{line:pruning_mechanism}). 
Here, \(P_i\) denotes the maximum capacity of a pruned global model \(W\) sent to client \(i\), signifying server-client global pruning. On the other hand, \(Q_i\) stands for the local pruning mechanism, enhancing both the speed of local computation and the robustness (allowing more dynamics) of local network training.

\begin{algorithm*}[!t]
	\caption{\algname{FedP3}}
	\label{alg:framework}
	\begin{algorithmic}[1]
		\STATE {\bf Input:} Client $i$ has data $X_i$ for $i\in [n]$, the number of local updates $K$, the number of communication rounds $T$, initial model weights $W_t = \{W_t^0, W_t^1, \ldots, W_t^L\}$ on the server for $t=0$
            \STATE Server specifies the server pruning mechanism $P_i$, the client pruning mechanism $Q_i$, and the set of layers to train $L_i \subseteq [L]$ for each client $i\in [n]$  \label{line:pruning_mechanism}
            \vspace{-\baselineskip}
		\FOR{$t=0,1,\dotsc,T-1$}  
		      \STATE Server samples a subset of participating clients $\mathcal{C}_t\ \subset [n]$  \label{line:sample_participating_clients}
                \STATE Server sends the layer weights $W_t^l$ for $l \in L_i$ to client $i\in \mathcal{C}_t$ for training \label{line:server_sends_layers}
                \STATE Server sends the pruned weights $P_i \odot W_t^l$ for $l \notin L_i$ to client $i\in \mathcal{C}_t$  \label{line:server_sends_weigths}
                \FOR{each client $i \in \mathcal{C}_t$ in parallel}
                    \STATE Initialize $W_{t,0}^l = W_t^l$ for all $l \in [L_i]$ and $W_{t, 0}^{l} = P_i \odot W_t^{l}$ for all $l \notin [L_i]$ \label{line:local_training_begin}
                    \FOR{$k = 0, 1, \dotsc, K-1$}
                        \STATE Compute $W_{t,k+1} \leftarrow \texttt{LocalUpdate}(W_{t,k}, X_i, L_i, Q_i, k)$,\\ \quad where $W_{t, k}\eqdef \{W_{t,k}^0, W_{t,k}^1, \ldots, W_{t,k}^L\}$ 
                    \ENDFOR
                \STATE Send $\cup_{l \in L_i} W_{t,K}^l$ to the server  \label{line:send_weights_to_server}
                \ENDFOR
                \STATE Server aggregates $W_{t+1} = \texttt{Aggregation}(\cup_{i \in [n]} \cup_{l \in L_i} W_{t,K}^l)$ 
		\ENDFOR
            \\
            \STATE {\bf Output:} $W_T$
	\end{algorithmic}
\end{algorithm*}

In Line~\ref{line:sample_participating_clients}, we opt for partial client participation by selecting a subset of clients \(\mathcal{C}_t\) from the total pool \( [n] \). 
Unlike the independent subnetwork training approach, Lines~\ref{line:server_sends_layers}--\ref{line:server_sends_weigths} employ a personalized server-client pruning strategy. 
This aligns with the concept of collaborative training. 
Under this approach, we envision each client learning a subset of layers, sticking to smaller neural network architectures of the global model. 
Due to the efficient and privacy-friendly communication, such a method is not only practical but also paves a promising path for future research in FL-type training and large language models.

The server chooses a layer subset \(L_i\) for client \(i\) and dispatches the pruned weights, conditioned by \(P_i\), for the remaining layers. Local training spans \(K\) steps (Lines~\ref{line:local_training_begin}--\ref{line:send_weights_to_server}), detailed in Algorithm~\ref{alg:local_update}. To uphold a privacy-friendly framework, only weights \(\cup_{l \in L_i} W^l_{t, K}\) necessary for training of each client \(i\) are transmitted to the server (Line~\ref{line:send_weights_to_server}). The server concludes by aggregating the weights received from every client to forge the updated model \(W_{t+1}\), as described in Algorithm~\ref{alg:aggregation}. We also provide an intuitive pipeline in \Cref{fig:tissue_figure}.

\begin{figure}[!t]
    \centering
    \includegraphics[width=1.0\textwidth]{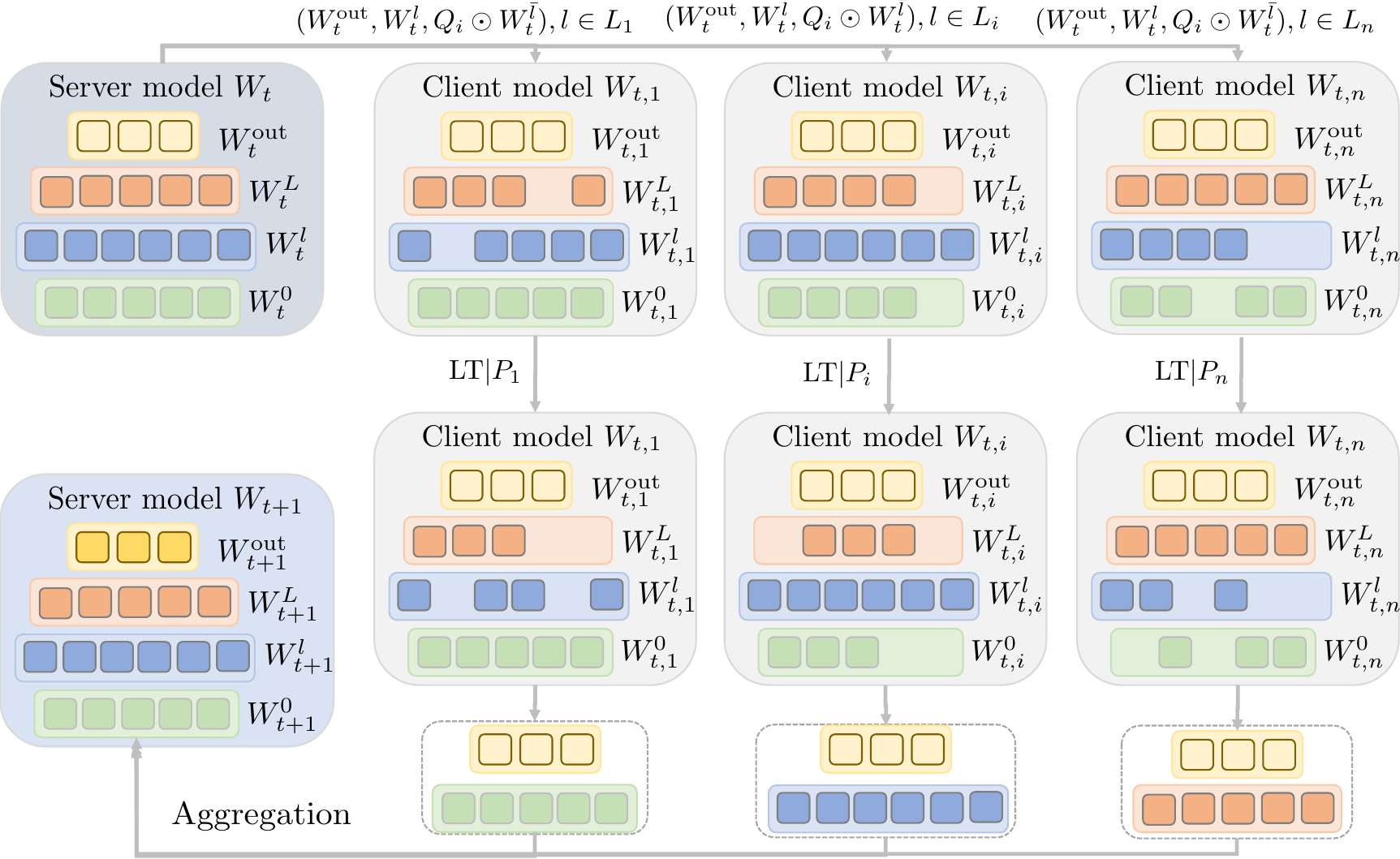}
    \caption{Pipeline illustration of our proposed framework \algname{FedP3}.}
    \label{fig:tissue_figure}
\end{figure}   

\paragraph{Local update.} Our proposed framework, \algname{FedP3}, incorporates dynamic network pruning. In addition to personalized task assignments for each client $i$, our local update mechanism supports diverse pruning strategies. Although efficient pruning strategies in FL remain an active research area~\citep{FjORD, FedRolex, Flado}, we aim to determine if our framework can accommodate various strategies and yield significant insights. In this context, we examine different local update rules as described in Algorithm~\ref{alg:local_update}. We evaluate three distinct strategies: \textit{fixed without pruning}, \textit{uniform pruning}, and \textit{uniform ordered dropout}.

Assuming our current focus is on $W^l_{t, k}$, where $l\notin L_i$, after procuring the pruned model conditioned on $P_i$ from the server, we denote the sparse model we obtain by $P_i\odot W^l_{t, 0}$. Here: 
    
\begin{itemize}
    \item \textit{Fixed without pruning} implies that we conduct multiple steps of the local update without additional local pruning, resulting in $P_i\odot W^l_{t, K}$.
    \item \textit{Uniform pruning} dictates that for every local iteration $k$, we randomly generate the probability $q_{i, k}$ and train the model $q_{i, k}\odot P_i\odot W^l_{t, K}$.
    \item \textit{Uniform ordered dropout} is inspired by~\cite{FjORD}. In essence, if $P_i\odot W^l_{t, 0} \in \mathbb{R}^{d_1\times d_2}$ (extendable to 4D convolutional weights; however, we reference 2D fully connected layer weights here), we retain only the subset $P_i\odot W^l_{t, 0}[:q_{i, k}d_1, :q_{i, k}d_2]$ for training purposes. $[:q_{i, k}d_1]$ represents we select the first $q_{i, k}\times d_1$ elements from the total $d_1$ elements. 
\end{itemize}

Regardless of the chosen method, the locally deployed model is given by $\left(\cup_{l\in L_i}W_{t, k}^{l}\right)\cup  \left( \cup_{l \not \in L_i} q_{i,k} \odot P_i \odot W_{t, k}^l \right)$, as highlighted in Algorithm~\ref{alg:local_update} Line~\ref{line:local_update}.

\begin{algorithm*}[!tb]
	\caption{\texttt{LocalUpdate}}
	\label{alg:local_update}
	\begin{algorithmic}[1]
		\STATE {\bf Input:} $W_{t,k}, X_i, L_i, Q_i, k$
            \STATE Generate the step-wise local pruning ratio $q_{i, k}$ conditioned on $P_i$ and $Q_i$
            \STATE Local training $\left(\cup_{l\in L_i}W_{t, k}^{l}\right)\cup  \left( \cup_{l \not \in L_i} q_{i,k} \odot P_i \odot W_t^l \right)$ using local data $X_i$  \label{line:local_update}
            \STATE {\bf Output:} $W_{t,k+1}$
	\end{algorithmic}
\end{algorithm*}
     
\paragraph{Layer-wise aggregation.} Our Algorithm~\ref{alg:framework} distinctively deviates from existing methods in Line~\ref{line:send_weights_to_server} as each client forwards only a portion of information to the server, thus prompting an investigation into optimal aggregation techniques. 
In Algorithm~\ref{alg:aggregation} we evaluate three aggregation methodologies: 
\begin{itemize}
    \item \textit{Simple averaging} computes the mean of all client contributions that include a specific layer $l$. 
    This option is presented in Line~\ref{line:simple_avg}.
    \item \textit{Weighted averaging} adopts a weighting scheme based on the number of layers client $i$ is designated to train. 
    Specifically, the weight for aggregating $W_{t, K, i}^l$ from client $i$ is given by $|L_i| / \sum_{j=1}^n |L_j|$, analogous to importance sampling.
    This option is presented in Line~\ref{line:weighted_avg}
    \item \textit{Attention-based averaging} introduces an adaptive mechanism where an attention layer is learned specifically for layer-wise aggregation.
    This option is presented in Line~\ref{line:attention_avg}.
\end{itemize}


\begin{algorithm*}[!tb]
	\caption{\texttt{Aggregation}}
	\label{alg:aggregation}
	\begin{algorithmic}[1]
		\STATE {\bf Input:} $\cup_{i \in [n]} \cup_{l \in L_i} W_{t,K}^l$
            \STATE {\textit{Simple Averaging:}}
            \STATE \qquad $W_{t+1}^l \leftarrow \texttt{Avg}\left( W_{t, K, i}^l\right)$ for all nodes with $l\in L_i$  \label{line:simple_avg}
            \STATE {\textit{Weighted Averaging:}}
            \STATE \qquad Construct the aggregation weighting $\alpha_i$ for each client $i$  \label{line:weighted_avg}
            \STATE \qquad $W_{t+1}^l \leftarrow \texttt{Avg}\left(\alpha_i W_{t, K, i}^l\right)$ for all nodes with $l\in L_i$
            \STATE \textit{Attention Averaging:}
            \STATE \qquad Construct an attention mapping layer annoted by function $h$
            \STATE \qquad $W_{t+1}^l \leftarrow h\left(W_{t, K, i}^l\right)$ for all nodes with $l\in L_i$  \label{line:attention_avg}
            \STATE {\bf Output:} $W_{t+1}$
	\end{algorithmic}
\end{algorithm*}



\section{Theoretical Analysis}
Our work refines independent subnetwork training (IST) by adding personalization and layer-level sampling, areas yet to be fully explored (see Appendix \ref{sec:subnetwork_training} for related work). Drawing on the sketch-based analysis from \cite{shulgin2023towards}, we aim to thoroughly analyze \algname{FedP3}, enhancing the sketch-type design concept in both scope and depth.

Consider a global model denoted as $w \in \mathbb{R}^d$. In \cite{shulgin2023towards}, a sketch $\mathcal{C}_i^k \in \mathbb{R}^{d \times d}$ represents submodel computations by weights permutations. We extend this idea to a more general case encompassing both global pruning, denoted as $\mP\in \mbR^{d\times d}$, and personalized model aggregations, denoted as $\mS\in \mbR^{d\times d}$. Now we first present the formal definitions. 

\begin{definition}[Global Pruning Sketch $\mP$]\label{def:sketch1}
    Let a random subset $\mathcal{S}$ of $[d]$ is a proper sampling such that the probability $c_j \eqdef \mathrm{Prob}(j\in S) > 0$ for all $j\in [d]$. Then the biased diagonal sketch with $\mathcal{S}$ is $\mP \eqdef \Diag(p^1_s, p^2_s, \cdots, p^d_s)$, where $p^j_s = 1$ if $j\in S$ otherwise $0$. 
\end{definition}

Unlike \cite{shulgin2023towards}, we assume client-specific sampling with potential weight overlap. For simplicity, we consider all layers pruned from the server to the client, a more challenging case than the partial pruning in \algname{FedP3} (Algorithm~\ref{alg:framework}). The convergence analysis of this global pruning sketch is in Appendix~\ref{sec:theory_global_pruning}.

\begin{definition}[Personalized Model Aggregation Sketch $\mS$]\label{def:sketch2}
    Assume $d\geq n$, $d=sn$, where $s\geq 1$ is an integer. Let $\pi = (\pi_1, \cdots, \pi_d)$ be a random permutation of the set $[d]$. The number of parameters per layer $n_l$, assume $s$ can be divided by $n_l$. Then, for all $x\in \Rd$ and each $i\in [n]$, we define $\mS$ as $\mS\eqdef n\sum^{si}_{j = s(i-1)+1}e_{\pi_j} e^\top_{\pi_j}$. 
\end{definition}
  
Sketch $\mS$ is based on the permutation compressor technique from \cite{szlendak2021permutation}. Extending this idea to scenarios where $d$ is not divisible by $n$ follows a similar approach as outlined in \cite{szlendak2021permutation}. To facilitate analysis, we apply a uniform parameter count $n_l$ across layers, preserving layer heterogeneity. For layers with fewer parameters than $d_L$, zero-padding ensures operational consistency. This uniform distribution assumption maintains our findings' generality and simplifies the discussion. 
Our method assumes $s$ divides $d_l$, streamlining layer selection over individual elements. The variable $v$ denotes the number of layers chosen per client, shaping a more analytically conducive framework for \algname{FedP3}, detailed in Algorithm~\ref{alg:IST} in the Appendix.

\begin{restatable}[Personalized Model Aggregation]{theorem}{modelaggregationtheorem}\label{thm:model_aggregation}
    Let Assumption~\ref{asm:smoothness} holds. Iterations $K$, choose stepsize $\gamma \leq \left\{ \nicefrac{1}{L_{\max}}, \nicefrac{1}{\sqrt{\hat{L}L_{\max} K}}\right\} $. Denote $\Delta_0 \eqdef f(w^0) - f^{\inf}$. Then for any $K\geq 1$, the iterates ${w^k}$ of \algname{FedP3} in Algorithm~\ref{alg:IST} satisfy
    \begin{align}\label{eqn:thm_model_aggregation}
        \min_{0\leq k\leq K -1}\ec{\sqN{\nabla f(w^k)}} \leq \frac{2(1 + \bar{L}L_{\max}\gamma^2)^K}{\gamma K}\Delta_0.
    \end{align}
\end{restatable}    

We have achieved a total communication cost of $\mathcal{O}\left(\nicefrac{d}{\epsilon^2}\right)$, marking a significant improvement over unpruned methods. This enhancement is particularly crucial in FL for scalable deployments, especially with a large number of clients. Our approach demonstrates a reduction in communication costs by a factor of $\mathcal{O}\left(\nicefrac{n}{\epsilon}\right)$. In the deterministic setting of unpruned methods, we compute the exact gradient, in contrast to bounding the gradient as in Lemma~\ref{lem:l_smooth_bound}. Remarkably, by applying the smoothness-based bound condition (Lemma~\ref{lem:l_smooth_bound}) to both \algname{FedP3} and the unpruned method, we achieve a communication cost reduction by a factor of $\mathcal{O}(\nicefrac{d}{n})$ for free. This indicates that identifying a tighter upper gradient bound could potentially lead to even more substantial theoretical improvements in communication efficiency. A detailed analysis is available in Appendix~\ref{sec:theory_model_aggregation}. We have also presented an analysis of the locally differential-private variant of \algname{FedP3}, termed \algname{LDP-FedP3}, in Theorem~\ref{thm:convergence_dp_fedp3}.

\begin{restatable}[\algname{LDP-FedP3} Convergence]{theorem}{dpfedpconvergencetheorem}\label{thm:convergence_dp_fedp3}
    Under Assumptions~\ref{asm:smoothness} and \ref{asm:bounded_gradient}, with the use of Algorithm~\ref{alg:DP_FedP3}, consider the number of samples per client to be $m$ and the number of steps to be $K$. Let the local sampling probability be $q\equiv b/m$. For constants $c^\prime$ and $c$, and for any $\epsilon < c^\prime q^2K$ and $\delta\in (0, 1)$, \algname{LDP-FedP3} achieves $(\epsilon, \delta)$-LDP with $\sigma^2 = \frac{cKC^2\log(1/\epsilon)}{m^2\epsilon^2}$. 

    Set $K=\max\Big\{\frac{m\epsilon \sqrt{L \Delta_0}}{C\sqrt{cd\log(1/\delta)}}, \frac{m^2\epsilon^2}{cd\log(1/\delta)}\Big\}$ and $\gamma = \min\Big\{\frac{1}{L}, \frac{\sqrt{\Delta_0 cd \log(1/\delta)}}{C m\epsilon\sqrt{L}}\Big\}$, we have:
    $$
    \frac{1}{K}\sum_{k=0}^{K-1} \ec{\sqn{\nabla f(w^t)}} \leq \frac{2C\sqrt{Lcd\log(1/\sigma)}}{m\epsilon} = \mathcal{O}\left(\frac{C\sqrt{Ld\log(1/\delta)}}{m\epsilon} \right).
    $$
    Consequently, the total communication cost is:
    $$
    C_{\mathrm{LDP-FedP3}} = \mathcal{O}\left( \frac{m\epsilon \sqrt{dL \Delta_0}}{C\sqrt{\log(1/\delta)}} + \frac{m^2\epsilon^2}{\log(1/\delta)}\right).
    $$
    \end{restatable}

We establish the privacy guarantee and communication cost of \algname{LDP-FedP3}. Our analysis aligns with the communication complexity in \cite{li2022soteriafl} while providing a more precise convergence bound. Further details and comparisons with existing work are discussed in Appendix \ref{sec:theory_dp_aggregation}.

\color{black}
\section{Experiments}
  
  
\subsection{Datasets and splitting techniques}


We utilize benchmark datasets CIFAR10/100~\cite{CIFAR}, a subset of EMNIST labeled EMNIST-L~\cite{EMNIST}, and FashionMNIST~\cite{FashionMNIST}, maintaining standard train/test splits as in \cite{mcmahan2017communication} and \cite{li2020federated}. While CIFAR100 has 100 labels, the others have 10, with a consistent data split of 70\% for training and 30\% for testing. Details on these splits are in Table~\ref{tab:statistics} in the Appendix. For non-iid splits in these datasets, we employ class-wise and Dirichlet non-iid strategies, detailed in Appendix~\ref{sec:data_distributions}.
   
\begin{figure}[!tb]
     \centering
     \begin{subfigure}[b]{0.48\textwidth}
         \centering
         \includegraphics[width=1.0\textwidth, trim=0 13 0 0, clip]{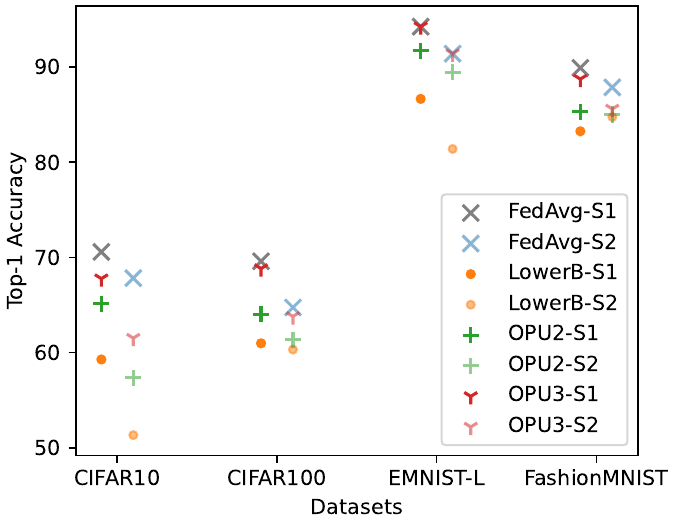}
     \end{subfigure}
     \hfill
     \begin{subfigure}[b]{0.48\textwidth}
         \centering
         \includegraphics[width=1.0\textwidth, trim=0 13 0 0, clip]{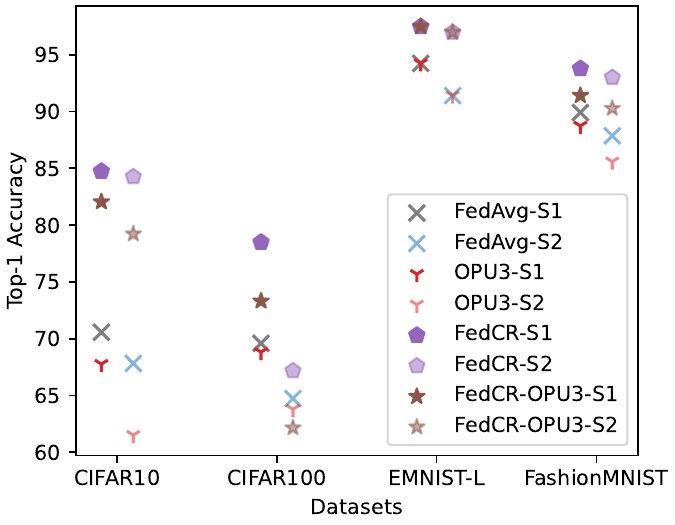}
     \end{subfigure}
        \caption{ Comparative Analysis of Layer Overlap Strategies: The left figure presents a comparative study of different overlapping layer configurations across four major datasets. On the right, we extend this comparison to include the state-of-the-art personalized FL method, \algname{FedCR}. In this context, \texttt{S1} refers to a class-wise non-iid distribution, while \texttt{S2} indicates a Dirichlet non-iid distribution.}
        \label{fig:overlapping_main}
\end{figure}

\subsection{Optimal layer overlapping among clients}
\paragraph{Datasets and models specifications.} 
{ In this section, our objective is to develop a communication-efficient architecture that also preserves accuracy. We conducted extensive experiments on recognized datasets like CIFAR10/100 and FashionMNIST, using a neural network with two convolutional layers (denoted as \texttt{Conv}) and four fully-connected layers (\texttt{FC}). For EMNIST-L, our model includes four \texttt{FC} layers including the output layer. This approach simplifies the identification of optimal layer overlaps among clients. We provide the details of network architectures in Appendix~\ref{sec:network_architecture}.}

 \paragraph{Layer overlapping analysis.} {Figure~\ref{fig:overlapping_main} presents a comparison of different layer overlapping strategies. For Optional Pruning Uniformly with selection of 2 layers (\algname{OPU2})  represents the selection of two uniformly chosen layers from the entire network for training, while \algname{OPU3} involves 3 such layers. \algname{LowerB} denotes the scenario where only one layer's parameters are trained per client, serving as a potential lower bound benchmark. All clients participate in training the final \texttt{FC} layer (denoted as \texttt{FFC}). ``\texttt{S1}" and ``\texttt{S2}" signify class-wise and Dirichlet data distributions, respectively. For example, \texttt{FedAvg-S1} shows the performance of \algname{FedAvg} under a class-wise non-iid setting. Given that a few layers are randomly assigned for each client to train, we assess the communication cost on average. 
In CIFAR10/100 and FashionMNIST training, by design, 
we obtain a 20\% communication reduction for \algname{OPU3}, 40\% for \algname{OPU2}, and 60\% for \algname{LowerB}. }
{ Remarkably, \algname{OPU3} shows comparable performance to \algname{FedAvg}, with only 80\% of the parameters communicated. Computational results in the Appendix~\ref{sec:quantitative_parameters} (\Cref{fig:parameter_comp}) elucidate the outcomes of randomly sampling a single layer (\algname{LowerB}). Particularly in CIFAR10, clients training on \texttt{FC2+FFC} layers face communication costs more than 10,815 times higher than those training on \texttt{Conv1+FFC} layers, indicating significant model heterogeneity.}

{ Beyond validating \algname{FedAvg}, we compare with the state-of-the-art personalized FL method \algname{FedCR}~\cite{FedCR} (details in Appendix~\ref{sec:training_details}), as shown on the right of Figure~\ref{fig:overlapping_main}. Our method (\algname{FedCR-OPU3}), despite 20\% lower communication costs, achieves promising performance with only a 2.56\% drop on \texttt{S1} and a 3.20\% drop on \texttt{S2} across four datasets. Additionally, Figure~\ref{fig:overlapping_main} highlights the performance differences between the two non-iid data distribution strategies, \texttt{S1} and \texttt{S2}. The average performance gap across \algname{LowerB}, \algname{OPU2}, and \algname{OPU3} is 3.55\%. This minimal reduction in performance across all datasets underscores the robustness and stability of our \algname{FedP3} pruning strategy in diverse data distributions within FL.}

\paragraph{Larger network verifications.}
Our assessment extends beyond shallow networks to the more complex ResNet18 model~\cite{ResNet}, tested with CIFAR10 and CIFAR100 datasets. Figure \ref{fig:resnet18_arch} illustrates the ResNet18 architecture, composed of four blocks, each containing four layers with skip connections, plus an input and an output layer, totaling 18 layers. A key focus of our study is to evaluate the efficiency of training this heterogeneous model using only a partial set of its layers. We performed layer ablations in blocks 2 and 3 (\texttt{B2} and \texttt{B3}), as shown in Figure~\ref{tab:resnet18_results}. The notation \texttt{-B2-B3(full)} indicates complete random pruning of \texttt{B2} or \texttt{B3}, with the remaining structure sent to the server. \texttt{-B2(part)} refers to pruning the first or last two layers in \texttt{B2}. { We default the global pruning ratio from server to client at 0.9, implying that the locally deployed model is approximately 10\% smaller than the global model.} Results in Figure~\ref{tab:resnet18_results} demonstrate that dropping random layers from ResNet18 does not significantly impact performance, sometimes even enhancing it.
{ Compared with \texttt{Full}, \texttt{-B2(part)} and \texttt{-B3(part)} achieved a 6.25\% reduction in communication costs 
with only a 1.03\% average decrease in performance. Compared to the standard \algname{FedAvg} without pruning, this is a 16.63\% reduction, showcasing the efficiency of our \algname{FedP3} method. Remarkably, \texttt{-B3(part)} even surpassed the \texttt{Full} model in performance. Additionally, \texttt{-B2-B3(full)} resulted in a 12.5\% average reduction in communication costs (21.25\% less compared to unpruned \algname{FedAvg}), with just a 6.32\% performance drop on CIFAR10 and CIFAR100. These results demonstrate the potential of \algname{FedP3} for effective learning in LLMs.}    
     
\begin{table}[!tb]
\begin{minipage}[!t]{.46\linewidth}
    \centering
    \includegraphics[width=1.0\textwidth]{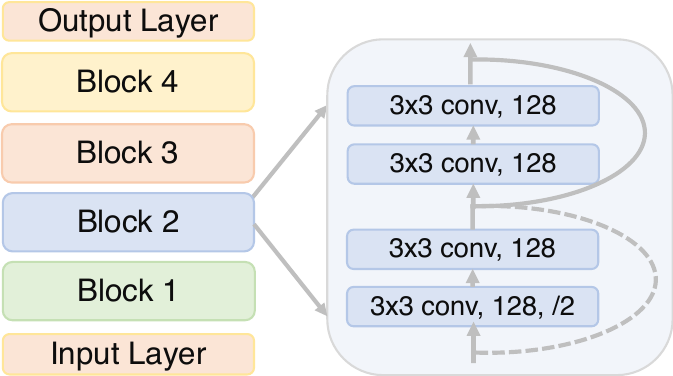}
    \captionof{figure}{ResNet18 architecture.}\label{fig:resnet18_arch}
\end{minipage}
\hfill
\begin{minipage}[!t]{.49\linewidth}
    \centering
    \captionof{table}{Performance of ResNet18 under class-wise non-iid conditions. { The global pruning ratio from server to client is maintained at 0.9 for all baseline comparisons by default.}}
    \label{tab:resnet18_results}
    \def\arraystretch{1.1}
    \begin{tabular}{l|c|c}
    \toprule
    \bf Method &{\bf CIFAR10} & {\bf CIFAR100}\\ \midrule
    Full & 73.25 & 63.33\\ \midrule
    -B2-B3 (full) & 65.68 & 58.26\\
    -B2 (part) & 72.09 & 61.11\\
    -B3 (part) & 73.47 & 62.39\\  
    \bottomrule
    \end{tabular}
\end{minipage}
\end{table}

\begin{figure}
     \centering
     \begin{subfigure}[b]{0.48\textwidth}
         \centering
         \includegraphics[width=1.0\textwidth, trim=0 0 0 0, clip]{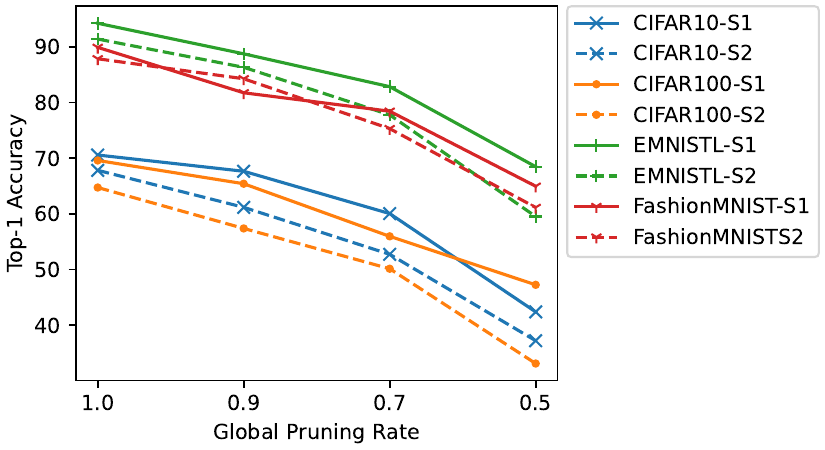}
     \end{subfigure}
     \hfill  
     \begin{subfigure}[b]{0.48\textwidth}
         \centering
         \includegraphics[width=1.0\textwidth, trim=0 0 0 0, clip]{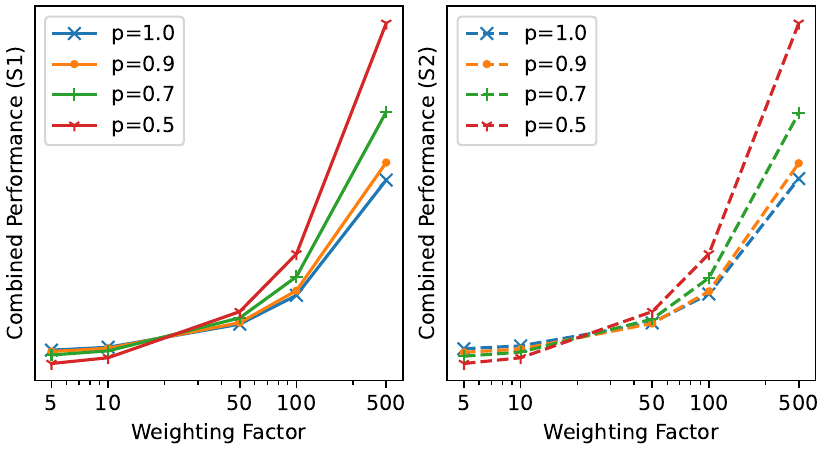}
     \end{subfigure}
        \caption{ Comparative Analysis of Server to Client Global Pruning Strategies: The left portion displays Top-1 accuracy across four major datasets and two distinct non-IID distributions, varying with different global pruning rates. On the right, we quantitatively assess the trade-off between model size and accuracy.
        }
        \label{fig:global_pruning_main}
\end{figure}  
  
{
    \subsection{Key ablation studies}
    Our framework, detailed in Algorithm \ref{alg:framework}, critically depends on the choice of pruning strategies. The \algname{FedP3} algorithm integrates both server-to-client global pruning and client-specific local pruning. Global pruning aims to minimize the size of the model deployed locally, while local pruning focuses on efficient training and enhanced robustness.
    
    \paragraph{Exploring server to client global pruning strategies}
    We investigate various global pruning ratios and their impacts, as shown in the left part of Figure~\ref{fig:global_pruning_main}. A global pruning rate of 0.9 implies the local model has 10\% fewer parameters than the global model. When comparing unpruned (rate 1.0) scenarios, we note an average performance drop of 5.32\% when reducing the rate to 0.9, 12.86\% to 0.7, and a significant 27.76\% to 0.5 across four major datasets and two data distributions. The performance decline is more pronounced at a 0.5 pruning ratio, indicating substantial compromises in performance for halving the model parameters.
    
    In the right part of Figure~\ref{fig:global_pruning_main}, we evaluate the trade-off between model size and accuracy. Assuming the total global model parameters as $N$ and accuracy as $\text{Acc}$, the global pruning ratio as $r$, we weigh the local model parameters against accuracy using a factor $\alpha \eqdef N/\text{Acc} > 0$, where the x-axis represents $\text{Acc}+\alpha / r$. A higher $\alpha$ indicates a focus on reducing parameter numbers for large global models, accepting some performance loss. This becomes increasingly advantageous with higher $\alpha$ values, suggesting a promising area for future exploration, especially with larger-scale models.
}

\paragraph{Exploring client-wise local pruning strategies}
Next, we are interested in exploring the influence of different local pruning strategies. Building upon our initial analysis, we investigate scenarios where our framework permits varying levels of local network pruning ratios. Noteworthy implementations in this domain resemble \algname{FjORD}~\citep{FjORD}, \algname{FedRolex}~\citep{FedRolex}, and \algname{Flado}~\citep{Flado}. Given that the only partially open-source code available is from \algname{FjORD}, we employ their layer-wise approach to network sparsity. The subsequent comparisons and their outcomes are presented in Table\ref{tab:abs_local_pruning}. The details of different pruning strategies, including \texttt{Fixed}, \texttt{Uniform} and \texttt{Ordered Dropout} are presented in the above Approach section. { "Fixed", "Uniform", "Ordered Dropout" represents \textit{Fixed without pruning}, \textit{Uniform pruning}, and \textit{Uniform order dropout} in the Approach section, respectively.} From the results in Table.~\ref{tab:abs_local_pruning}, we can see the difference between \texttt{Uniform} and \texttt{Ordered Dropout} strategies will be smaller with small global pruning ratio $p$ from 0.9 to 0.7. Besides, in our experiments, \texttt{Ordered Dropout} is no better than the simple \texttt{Uniform} strategy for local pruning.   

\begin{table*}[!tb]
    \centering
    \caption{Comparison of different network local pruning strategies.  Global pruning ratio $p$ is 0.9.}
    \label{tab:abs_local_pruning}
    \begin{threeparttable}
    \def\arraystretch{1.15}
    \resizebox{0.98\textwidth}{!}{%
    \begin{tabular}{lcccc}
    \toprule
    \bf Strategies &{\bf CIFAR10} & {\bf CIFAR100} & {\bf EMNIST-L} & \bf {\bf FashionMNIST} \\ \midrule
        \algname{Fixed} & 67.65 / 61.17 & 65.41 / 57.38 & 88.75 / 86.33 & 81.75 / 84.27\\ \midrule
        \algname{Uniform} ($p=0.9$) & 65.51 / 60.10 & 64.33 / 58.20 & 85.14 / 84.29 & 78.81 / 77.24\\
        \algname{Ordered Dropout} ($p=0.9$) & 61.73 / 58.82 & 61.11 / 53.28 & 82.54 / 80.18 & 75.45 / 73.27\\ \midrule
        \algname{Uniform} ($p=0.7$) & 60.78 / 56.41 & 60.35 / 54.88 & 77.39 / 75.82 & 72.66 / 70.37\\
        \algname{Ordered Dropout} ($p=0.7$) & 58.90 / 53.38 & 59.72 / 50.03 & 72.19 / 70.30 & 70.21 / 67.58\\
        \bottomrule
    \end{tabular}}
    \end{threeparttable}
\end{table*}

\paragraph{Exploring adaptive model aggregation strategies} { In this section, we explore a range of weighting strategies, including both simple and advanced averaging methods, primarily focusing on the CIFAR10/100 datasets. We assign clients with $1-3$ layers (\texttt{OPU1-2-3}) or $2-3$ layers (\texttt{OPU2-3}) randomly. In Algorithm~\ref{alg:aggregation}, we implement two aggregation approaches: \texttt{simple} and \texttt{weighted} aggregation.}

\begin{figure}[h!]
    \centering
    \begin{minipage}[t]{0.50\textwidth}
        \raggedright
        \vspace{-65mm}
        \noindent Let $L^l$ denote the set of clients involved in training the $l$-th layer, where $l\in \mathcal{L}$. The server's received weights for layer $l$ from client $i$ are represented as $W_{t, K, i}^l$. The general form of model aggregation is thus defined as:
            $$
            W_{t+1}^l = \sum_{j=1}^{L^l} \alpha_i W_{t, K, i}^l.
            $$   
            If $\alpha_i$ is initialized as $\nicefrac{1}{|L^l|}$, this constitutes \texttt{simple} mean averaging. Considering $N_i$ as the total number of layers for client $i$ and $n$ as the total number of clients, if $\alpha_i = \nicefrac{N_i}{\sumjn N_j}$, this method is termed \texttt{weighted} averaging. 
    \end{minipage}
    \hfill
    \begin{minipage}[t]{0.45\textwidth}
        \centering
        \includegraphics[width=\textwidth]{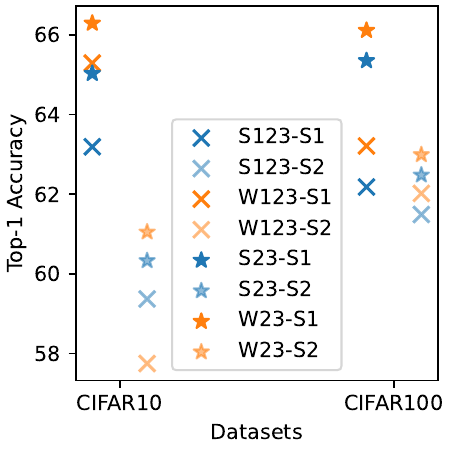}
        \caption{Comparison of various model aggregation strategies. $p=0.9$.}  
        \label{fig:abs_aggregation}  
    \end{minipage}
\end{figure}

The underlying idea is that clients with more comprehensive network information should have greater weight in parameter contribution. A more flexible approach is \texttt{attention} averaging, where $\alpha_i$ is learnable, encompassing \texttt{simple} and \texttt{weighted} averaging as specific cases. 
{ Future research may delve into a broader range of aggregation strategies.      
Our findings, shown in \Cref{fig:abs_aggregation}, include \texttt{S123-S1} for the \texttt{OPU1-2-3} method with simple aggregation in class-wise non-iid distributions, and \texttt{W23-S2} for \texttt{OPU2-3} with weighted aggregation in Dirichlet non-iid. The data illustrates that \texttt{weighted} averaging relatively improves over \texttt{simple} averaging by 1.01\% on CIFAR10 and 1.05\% on CIFAR100. Furthermore, \texttt{OPU-2-3} consistently surpasses \texttt{OPU1-2-3} by 1.89\%, empirically validating our hypotheses.} 


%% file: Chapter_3_CohortSqueeze.tex

\chapter{Beyond Single Communication Round per Cohort}
\label{chapter_cohort_squeeze}
\thispagestyle{empty}

\section{Introduction}

In this paper, we focus on cross-device FL, which involves the coordination of millions of mobile devices by a central server for training purposes \citep{FL-big}. This setting is characterized by intermittent connectivity and limited resources. As a result, only a subset of client devices participates in each communication round. Typically, the server samples a batch of clients (referred to as a \emph{cohort} in FL), and each selected client trains the model received from the server using its local data. The server then aggregates the results sent by the selected cohort.

A key limitation of this approach is that client devices operate in a stateless regime, meaning they cannot store states between communication rounds. This restriction prevents the use of variance reduction techniques, which require memory across iterations.

To address this, we reformulate the cross-device objective by assuming a finite number of workers selected with uniform probability, as defined in \eqref{eq:ERM}. This reformulation better aligns with empirical observations and provides a clearer illustration of the underlying process.
The extension of the proposed theory to the expectation-based formulation is presented in \Cref{sec:exp-formulation}.

Current representative approaches in the cross-device setting include \algname{FedAvg} and \algname{FedProx}. In our work, we introduce a method by generalizing stochastic proximal point method with arbitray sampling and term as \algname{SPPM-AS}. This new method is inspired by the stochastic proximal point method (\algname{SPPM}), a technique notable for its ability to converge under arbitrarily large learning rates and its flexibility in incorporating various solvers to perform proximal steps. This adaptability makes \algname{SPPM} highly suitable for cross-device FL \citep{FedProx, yuan2022convergence, yuan2023sharper, SPPM, lin2024stochastic}.  
Additionally, we introduce support for an arbitrary cohort sampling strategy, accompanied by a theoretical analysis. We present novel strategies that include support for client clustering, which demonstrate both theoretical and practical improvements.

Another interesting parameter that allows for control is the number of local communications.
Two distinct types of communication, \emph{global} and \emph{local}, are considered.
A \emph{global} iteration is defined as a single round of communication between the server and all participating clients.
On the other hand, \emph{local} communication rounds are synchronizations that take place within a chosen cohort.
Additionally, we introduce the concept of total communication cost, which includes both local and global communication iterations, to measure the overall efficiency of the communication process.
The total communication cost naturally depends on several factors. These include the local algorithm used to calculate the prox, the global stepsize, and the sampling technique.
    
\begin{figure}[!tb]
	\centering
	\begin{subfigure}[b]{0.45\textwidth}
		\centering
		\includegraphics[width=\textwidth]{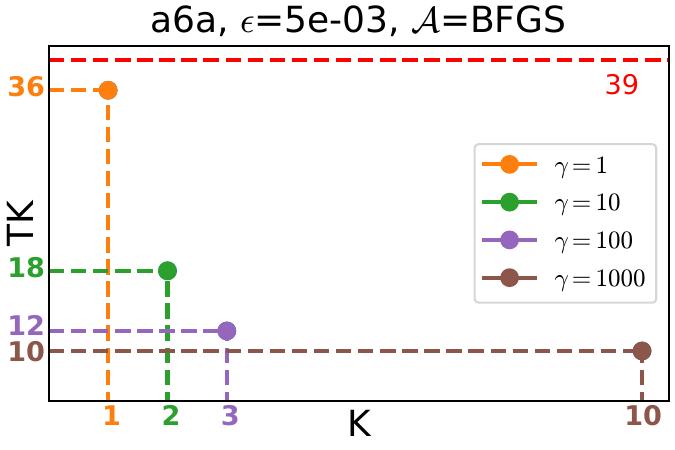}
	\end{subfigure}
	\hspace{5mm} 
	\begin{subfigure}[b]{0.45\textwidth}
		\centering
		\includegraphics[width=\textwidth]{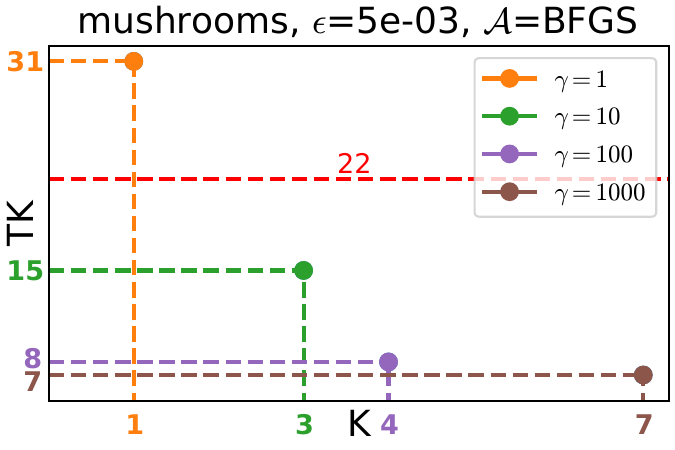}
	\end{subfigure}    
	\caption{The total communication cost (defined as $TK$) with the number of local communication rounds $K$ needed to reach the target accuracy $\epsilon$ for the chosen cohort in each global iteration. The dashed red line depicts the communication cost of the \texttt{FedAvg} algorithm. Markers indicate the $TK$ value for different learning rates $\gamma$ of our algorithm \texttt{SPPM-AS}.}
	\label{fig:tissue0}
\end{figure}

\subsection{Motivation}
  Previous results on cross-device settings consider only one local communication round for the selected cohort \citep{Li-local-bounded-grad-norms--ICLR2020, reddi2020adaptive, FedProx, wang2021novel, wang2021local, xu2021fedcm, malinovsky2023server, jhunjhunwala2023fedexp, FedSpeed, sun2024role}. 
 Our experimental findings reveal that \emph{increasing the number of local communication rounds within a chosen cohort per global iteration can indeed lower the total communication cost needed to reach a desired global accuracy level}, which we denote as $\eps$.
\Cref{fig:tissue0} illustrates the relationship between total communication costs and the number of local communication rounds.
Assume that the cost of communication per round is $1$ unit. $K$ represents the number of local communication rounds per global iteration for the selected cohort, while $T$ signifies the \emph{minimum} number of global iterations needed to achieve the accuracy threshold $\epsilon$. Then, the total cost incurred by our method can be expressed as $TK$.
For comparison, the dashed line in the figure shows the total cost for the \algname{FedAvg} algorithm, which always sets $K$ to $1$, directly equating the number of global iterations to total costs. Our results across various datasets identify the optimal $K$ for each learning rate to achieve $\epsilon$-accuracy.
 \Cref{fig:tissue0} shows that adding more local communication rounds within each global iteration can lead to a significant reduction in the overall communication cost. For example, when the learning rate is set to $1000$, the optimal cost is reached with $10$ local communication rounds, making $K=10$ a more efficient choice compared to a smaller number. On the other hand, at a lower learning rate of $100$, the optimal cost of $12$ is reached with $K=3$. This pattern indicates that as we increase the number of local communication rounds, the total cost can be reduced, and the optimal number of local communication rounds tends to increase with higher learning rates.

\subsection{Summary of contributions}
Our key \emph{contributions} are summarized as follows:
 
\noindent $\bullet$  We present and analyze \algname{SPPM-AS}, a novel approach within the stochastic proximal point method framework tailored for cross-device federated learning, which supports arbitrary sampling strategies. Additionally, we provide an analysis of standard sampling techniques and introduce new techniques based on clustering approaches. These novel techniques are theoretically analyzed, offering a thorough comparison between different methods.

\noindent $\bullet$  Our numerical experiments, conducted on both convex logistic regression models and non-convex neural networks, demonstrate that the introduced framework enables fine-tuning of parameters to surpass existing state-of-the-art cross-device algorithms. Most notably, we found that increasing the number of local communication rounds within the selected cohort is an effective strategy for reducing the overall communication costs necessary to achieve a specified target accuracy threshold.

\noindent $\bullet$  We offer practical guidance on the proper selection of parameters for federated learning applications. Specifically, we examine the potential choices of solvers for proximal operations, considering both convex and non-convex optimization regimes. Our experiments compare first-order and second-order solvers to identify the most effective ones.
   

\section{Related work}
\subsection{Cross-device federated learning}
In FL, two predominant settings are recognized: cross-silo and cross-device scenarios, as detailed in Table 1 of \citealp{FL-big}. The primary distinction lies in the nature of the clients: cross-silo FL typically involves various organizations holding substantial data, whereas cross-device FL engages a vast array of mobile or IoT devices.
In cross-device FL, the complexity is heightened by the inability to maintain a persistent hidden state for each client, unlike in cross-silo environments. This factor renders certain approaches impractical, particularly those reliant on stateful clients participating consistently across all rounds. Given the sheer volume of clients in cross-device FL, formulating and analyzing outcomes in an expectation form is more appropriate, but more complex than in finite-sum scenarios.
 
The pioneering and perhaps most renowned algorithm in cross-device FL is \algname{FedAvg} \citep{FedAvg} and implemented in applications like Google's mobile keyboard \citep{hard2018federated, yang2018applied, ramaswamy2019federated}. However, it is noteworthy that popular accelerated training algorithms such as \algname{Scaffold} \citep{SCAFFOLD} and \algname{ProxSkip} \citep{ProxSkip} are not aligned with our focus due to their reliance on memorizing the hidden state for each client, which is applicable for cross-device FL.
Our research pivots on a novel variant within the cross-device framework. Once the cohort are selected for each global communication round, these cohorts engage in what we term as `local communications' multiple times. The crux of our study is to investigate whether increasing the number of local communication rounds can effectively reduce the total communication cost to converge to a targeted accuracy.

\subsection{Stochastic proximal point method}
Our exploration in this paper centers on the Stochastic Proximal Point Method (\algname{SPPM}), a method extensively studied for its convergence properties. Initially termed as the incremental proximal point method by \cite{bertsekas2011incremental}, it was shown to converge nonasymptotically under the assumption of Lipschitz continuity for each $f_i$. Following this, \cite{RyuBoy:16} examined the convergence rates of \algname{SPPM}, noting its resilience to inaccuracies in learning rate settings, contrasting with the behavior of Stochastic Gradient Descent (\algname{SGD}).
Further developments in \algname{\algname{SPPM}}'s application were seen in the works of \cite{patrascu2018nonasymptotic}, who analyzed its effectiveness in constrained optimization, incorporating random projections. \cite{asi2019stochastic} expanded the scope of \algname{SPPM} by studying a generalized method, \algname{AProx}, providing insights into its stability and convergence rates under convex conditions. The research by \cite{asi2020minibatch} and \cite{chadha2022accelerated} further extended these findings, focusing on minibatching and convergence under interpolation in the \algname{AProx} framework.

In the realm of federated learning, particularly concerning non-convex optimization, \algname{SPPM} is also known as \algname{FedProx}, as discussed in works like those of \cite{FedProx} and \cite{yuan2022convergence}. However, it is noted that in non-convex scenarios, the performance of \algname{FedProx/SPPM} in terms of convergence rates does not surpass that of \algname{SGD}. Beyond federated learning, the versatility of \algname{SPPM} is evident in its application to matrix and tensor completion such as in the work of \cite{bumin2021efficient}. Moreover, \algname{SPPM} has been adapted for efficient implementation in a variety of optimization problems, as shown by \cite{shtoff2022efficient}.
While non-convex \algname{SPPM} analysis presents significant challenges, with a full understanding of its convex counterpart still unfolding, recent studies such as the one by \cite{SPPM} have reported enhanced convergence by leveraging second-order similarity. Diverging from this approach, our contribution is the development of an efficient minibatch \algname{SPPM} method \algname{SPPM-AS} that shows improved results without depending on such assumptions. Significantly, we also provide the first empirical evidence that increasing local communication rounds in finding the proximal point can lead to a reduction in total communication costs. 
   
\section{Method}\label{sec:method}
In this section, we explore efficient stochastic proximal point methods with arbitrary sampling for cross-device FL to optimize the objective \eqref{eq:ERM}. Throughout the paper, we denote $[n]\eqdef \cb{1,\dots,n}$. Our approach builds on the following assumptions.
  
\begin{assumption}\label{asm:differential}
    Function $f_{i}: \Rd \to \mbR$  is differentiable for all samples $i \in [n]$.
\end{assumption}

This implies that the function $f$ is differentiable. The order of differentiation and summation can be interchanged due to the additive property of the gradient operator.
\begin{align*}
    \nabla f(x) \stackrel{Eqn.~(\ref{eq:ERM})}{=} \nabla \left[\frac{1}{n}\sum_{i=1}^n f_{i}(x)\right] = \frac{1}{n}\sum_{i=1}^n\nabla f_{i}(x).
\end{align*}

\begin{assumption}\label{asm:strongly_convex}
    Function $f_i: \Rd \to \mbR$  is $\mu$-strongly convex for all samples $i\in \left[n\right]$, where $\mu > 0$. That is, 
    $
        f_{i}(y) + \ev{\nabla f_i(y), x - y} + \frac{\mu}{2}\sqn{x - y} \leq f_{i}(x),
    $ for all $x, y \in \Rd$.
\end{assumption}
This implies that $f$ is $\mu$-strongly convex and hence has a unique minimizer, which we denote by $x_\star$. We know that $\nabla f(x_\star) = 0$. Notably, we do \emph{not} assume $f$ to be $L$-smooth.

\subsection{Sampling distribution}
Let $\mathcal{S}$ be a probability distribution over the $2^n$ subsets of $[n]$. Given a random set $S\sim \mathcal{S}$, we define 
\begin{align*}
	p_i \eqdef \operatorname{Prob}(i \in {S}), \quad i \in [n].
\end{align*}

We restrict our attention to proper and nonvacuous random sets. 

\begin{assumption}\label{asm:valid_sampling}
	$\mathcal{S}$ is proper (i.e., $p_i > 0$ for all $i\in [n]$) and nonvacuous (i.e., $\operatorname{Prob}({S} = \emptyset) = 0$).
\end{assumption}

Let $C$ be the selected cohort. Given $\emptyset \neq C \subseteq[n]$ and $i \in[n]$, we define
\begin{align}\label{eqn:main_8001}
	v_i(C):= \begin{cases}\frac{1}{p_i} & i \in C \\ 0 & i \notin C\end{cases}\ \Rightarrow\ f_C(x):=\frac{1}{n} \sum_{i=1}^n v_i(C) f_i(x) {=} \sum_{i \in C} \frac{1}{n p_i} f_i(x) .
\end{align}

Note that $v_i(S)$ is a random variable and $f_S$ is a random function. By construction, $\mathrm{E}_{S \sim \mathcal{S}}\left[v_i(S)\right]=1$ for all $i \in[n]$, and hence
{
	\begin{align*}
		\mathrm{E}_{{S} \sim \mathcal{S}}\left[f_{{S}}(x)\right] =\mathrm{E}_{{S} \sim \mathcal{S}}\left[\frac{1}{n} \sum_{i=1}^n v_i(S) f_i(x)\right] =\frac{1}{n} \sum_{i=1}^n \mathrm{E}_{S \sim \mathcal{S}}\left[v_i(S)\right] f_i(x)=\frac{1}{n} \sum_{i=1}^n f_i(x)=f(x).
\end{align*}}

Therefore, the optimization problem in \Cref{eq:ERM} is equivalent to the stochastic optimization problem
\begin{align}\label{eqn:obj_exp}
	\min _{x \in \mathbb{R}^d}\left\{f(x):=\mathrm{E}_{S \sim \mathcal{S}}\left[f_S(x)\right]\right\} .
\end{align}

Further, if for each $C \subset[n]$ we let $p_C:=\operatorname{Prob}(S=C)$, then $f$ can be written in the equivalent form
{\small
	\begin{align}\label{eqn:8010}
		f(x)=\ec[S \sim \mathcal{S}]{f_S(x)}=\sum_{C \subseteq[n]} p_C f_C(x)=\sum_{C \subseteq[n], p_C>0} p_C f_C(x).
\end{align}}

\subsection{Core algorithm}\label{sec:algorithm}
\noindent
\begin{minipage}[t]{0.44\textwidth}
Applying \algname{SPPM}~\citep{SPPM} to \Cref{eqn:obj_exp}, we arrive at stochastic proximal point method with arbitrary sampling (\algname{SPPM-AS}, \Cref{alg:sppm_as}):
\[
x_{t+1}=\operatorname{prox}_{\gamma f_{S_t}}\left(x_t\right),
\]
where $S_t \sim \mathcal{S}$.
\end{minipage}%
\hfill
\begin{minipage}[t]{0.54\textwidth}
\vspace{-14mm}
\begin{algorithm}[H]
    \caption{Stochastic Proximal Point Method with Arbitrary Sampling (\algname{SPPM-AS})}\label{alg:sppm_as}
    \begin{algorithmic}
        \STATE \textbf{Input:} starting point $x^0\in \mathbb{R}^d$, distribution $\mathcal{S}$ over the subsets of $[n]$, learning rate $\gamma > 0$
        \FOR{$t = 0, 1, 2, \ldots$}
            \STATE Sample $S_t \sim \mathcal{S}$
            \STATE $x_{t+1} = \prox_{\gamma f_{S_t}}(x_t)$ 
        \ENDFOR
    \end{algorithmic}
\end{algorithm}
\end{minipage}
   
\begin{restatable}[Convergence of \algname{SPPM-AS}]{shadedtheorem}{maintheorem}\label{thm:sppm_as}
	Let \Cref{asm:differential} (differentiability) and \Cref{asm:strongly_convex} (strong convexity) hold. Let $\mathcal{S}$ be a sampling satisfying \Cref{asm:valid_sampling}, and define
	{
		\begin{align}\label{eqn:main_8003}
			\mu_{\mathrm{AS}}:=\min _{C \subseteq[n], p_C>0} \sum_{i \in C} \frac{\mu_i}{n p_i}, \quad 			\sigma_{\star, \mathrm{AS}}^2 :=\sum_{C \subseteq[n], p_C>0} p_C\left\|\nabla f_C\left(x_{\star}\right)\right\|^2 .
	\end{align}}
	
	Let $x_0 \in \mathbb{R}^d$ be an arbitrary starting point. Then for any $t \geq 0$ and any $\gamma>0$, the iterates of \algname{SPPM-AS} (\Cref{alg:sppm_as}) satisfy
	{\small
		$$
		\mathrm{E}\left[\left\|x_t-x_{\star}\right\|^2\right] \leq\left(\frac{1}{1+\gamma \mu_{\mathrm{AS}}}\right)^{2t}\left\|x_0-x_{\star}\right\|^2+\frac{\gamma \sigma_{\star, \mathrm{AS}}^2}{\gamma \mu_{\mathrm{AS}}^2+2 \mu_{\mathrm{AS}}} .
		$$
	}
\end{restatable}

\paragraph{Theorem interpretation.}
In the theorem presented above, there are two main terms: $\left(\nicefrac{1}{(1+\gamma \mu_{\mathrm{AS}})}\right)^{2t}$ and $\nicefrac{\gamma \sigma_{\star, \mathrm{AS}}^2}{(\gamma \mu_{\mathrm{AS}}^2+2 \mu_{\mathrm{AS}})}$, which define the convergence speed and neighborhood, respectively. Additionally, there are three hyperparameters to control the behavior: $\gamma$ (the global learning rate), $\mathrm{AS}$ (the sampling type), and $T$ (the number of global iterations). In the following paragraphs, we will explore special cases to provide a clear intuition of how the \algname{SPPM-AS} theory works.

\paragraph{Interpolation regime.} Consider the interpolation regime, characterized by $\sigma_{\star, \mathrm{AS}}^2 = 0$ . Since we can use arbitrarily large $\gamma > 0$, we obtain an arbitrarily fast convergence rate.
Indeed, $\left(\nicefrac{1}{(1+\gamma \mu_{\mathrm{AS}})} \right)^{2t}$ can be made arbitrarily small for any fixed $t \geq 1$, even $t=1$, by choosing $\gamma$ large enough. However, this is not surprising, since now $f$ and all functions $f_{\xi}$ share a single minimizer, $x_\star$, and hence it is possible to find it by sampling a small batch of functions even a single function $f_{\xi}$, and minimizing it, which is what the prox does, as long as $\gamma$ is large enough.

\paragraph{A single step travels far.} Observe that for $\gamma = \nicefrac{1}{\mu_{\mathrm{AS}}}$, we have $\nicefrac{\gamma \sigma^2_{\star, \mathrm{AS}}}{(\gamma \mu_{\mathrm{AS}}^2 + 2\mu_{\mathrm{AS}})} = \nicefrac{\sigma^2_{\star, \mathrm{AS}}}{3\mu_{\mathrm{AS}}^2}$. In fact, the convergence neighborhood $\nicefrac{\gamma \sigma^2_{\star, \mathrm{AS}}}{(\gamma \mu_{\mathrm{AS}}^2 + 2\mu_{\mathrm{AS}})}$ is bounded above by three times this quantity irrespective of the choice of the stepsize. Indeed, 
$
	\frac{\gamma \sigma^2_{\star, \mathrm{AS}}}{\gamma \mu_{\mathrm{AS}}^2 + 2\mu_{\mathrm{AS}}} \leq \min\left\{\frac{\sigma^2_{\star, \mathrm{AS}}}{\mu_{\mathrm{AS}}^2}, \frac{\gamma\sigma^2_{\star, \mathrm{AS}}}{\mu_{\mathrm{AS}}}\right\} \leq \frac{\sigma^2_{\star, \mathrm{AS}}}{\mu_{\mathrm{AS}}^2}.
$
That means that no matter how far the starting point $x_0$ is from the optimal solution $x_\star$, if we choose the stepsize $\gamma$ to be large enough, then we can get a decent-quality solution after a single iteration of \algname{SPPM-AS} already! Indeed, if we choose $\gamma$ large enough so that
$
	\left(\nicefrac{1}{1+\gamma\mu_{\mathrm{AS}}} \right)^2 \sqn{x_0 - x_\star} \leq \delta, 
$
where $\delta > 0$ is chosen arbitrarily, then for $t=1$ we get 
$
	\ec{\sqn{x_1 - x_\star}} \leq \delta + \nicefrac{\sigma^2_{\star, \mathrm{AS}}}{\mu_{\mathrm{AS}}^2}.
$

\paragraph{Iteration complexity.} We have seen above that an accuracy arbitrarily close to (but not reaching) $\nicefrac{\sigma^2_{\star, \mathrm{AS}}}{\mu_{\mathrm{AS}}^2}$ can be achieved via a single step of the method, provided that the stepsize $\gamma$ is large enough. 
Assume now that we aim for $\epsilon$ accuracy, where $\epsilon \leq \nicefrac{\sigma^2_{\star, \mathrm{AS}}}{\mu_{\mathrm{AS}}^2}$. We can show that with the stepsize $\gamma= \nicefrac{\varepsilon\mu_{\mathrm{AS}}}{\sigma^2_{\star, \mathrm{AS}}}$, we get
$
\mathrm{E}\left[\left\|x_t-x_{\star}\right\|^2\right] \leq \varepsilon
$ 
provided that 
$
t \geq \left(\frac{\sigma^2_{\star, \mathrm{AS}}}{2 \varepsilon \mu_{\mathrm{AS}}^2}+\frac{1}{2}\right) \log \left(\frac{2\left\|x_0-x_{\star}\right\|^2}{\varepsilon}\right).
$
We provide the proof in \Cref{sec:proof_iteration_complexity}.
To ensure thoroughness, we present in \Cref{sec:fedavg-sppm} the lemma of the inexact formulation for \algname{SPPM-AS}, which offers greater practicality for empirical experimentation. Further insights are provided in the subsequent experimental section.

\paragraph{General framework.}
With freedom to choose arbitrary algorithms for solving the proximal operator one can see that \algname{SPPM-AS} is generalization for such renowned methods as \algname{FedProx} \citep{FedProx} and \algname{FedAvg} \citep{FedAvg2016}. A more particular overview of \algname{FedProx-SPPM-AS} is presented in further \Cref{sec:Avg-SPPM-baselines}.

\subsection{Arbitrary sampling examples}\label{sec:as}
Details on simple Full Sampling (FS) and Nonuniform Sampling (NS) are provided in \Cref{sec:samplings_table}. In this section, we focus more intently on the sampling strategies that are of particular interest to us.

\paragraph{Nice sampling (NICE).} Choose $\tau \in[n]$ and let $S$ be a random subset of $[n]$ of size $\tau$ chosen uniformly at random. Then $p_i=\nicefrac{\tau}{n}$ for all $i \in[n]$. Moreover, let $\binom{n}{\tau}$ represents the number of combinations of $n$ taken $\tau$ at a time, $p_C=\frac{1}{\binom{n}{\tau}}$ whenever $|C|=\tau$ and $p_C=0$ otherwise. So,
{\small
$$
\mu_{\mathrm{AS}}=\mu_{\mathrm{NICE}}(\tau):=\min _{C \subseteq[n], p_C>0} \sum_{i \in C} \frac{\mu_i}{n p_i}=\min _{C \subseteq[n],|C|=\tau} \frac{1}{\tau} \sum_{i \in C} \mu_i,
$$}

\begin{align*}
    \sigma_{\star, \mathrm{AS}}^2&=\sigma_{\star, \mathrm{NICE}}^2(\tau):=\sum_{C \subseteq[n], p_C>0} p_C\left\|\nabla f_C\left(x_{\star}\right)\right\|^2\stackrel{Eqn.~(\ref{eqn:main_8001})}{=} \sum_{C \subseteq[n],|C|=\tau} \frac{1}{\binom{n}{\tau}}\left\|\frac{1}{\tau} \sum_{i \in C} \nabla f_i\left(x_{\star}\right)\right\|^2 .
\end{align*}

It can be shown that $\mu_{\mathrm{NICE}}(\tau)$ is a \emph{nondecreasing} function of $\tau$ (\Cref{sec:proof_nice}). So, as the minibatch size $\tau$ increases, the strong convexity constant $\mu_{\mathrm{NICE}}(\tau)$ can only improve. Since $\mu_{\mathrm{NICE}}(1)=\min _i \mu_i$ and $\mu_{\mathrm{NICE}}(n)=$ $\frac{1}{n} \sum_{i=1}^n \mu_i$, the value of $\mu_{\mathrm{NICE}}(\tau)$ interpolates these two extreme cases as $\tau$ varies between 1 and $n$. Conversely, $\sigma_{\star, \mathrm{NICE}}^2(\tau)=\frac{\nicefrac{n}{\tau}-1}{n-1}\sigma_{\star, \mathrm{NICE}}^2(1)$ is a nonincreasing function, reaching a value of $\sigma_{\star, \mathrm{NICE}}^2(n) = 0$, as explained in \Cref{sec:proof_nice}.

\paragraph{Block Sampling (BS).} Let $C_1, \ldots, C_b$ be a partition of $[n]$ into $b$ nonempty blocks. For each $i \in[n]$, let $B(i)$ indicate which block $i$ belongs to. In other words, $i \in C_j$ if $B(i)=j$. Let $S=C_j$ with probability $q_j>0$, where $\sum_j q_j=1$. Then $p_i=q_{B(i)}$, and hence \Cref{eqn:main_8003} takes on the form
\begin{align*}
    \mu_{\mathrm{AS}}=\mu_{\mathrm{BS}}:=\min _{j \in[b]} \frac{1}{n q_j} \sum_{i \in C_j} \mu_i,\quad
    \sigma_{\star, \mathrm{AS}}^2=\sigma_{\star, \mathrm{BS}}^2:=\sum_{j \in[b]} q_j\left\|\sum_{i \in C_j} \frac{1}{n p_i} \nabla f_i\left(x_{\star}\right)\right\|^2 .
\end{align*}

\emph{Considering two extreme cases:} If $b=1$, then \algname{SPPM-BS} = \algname{SPPM-FS} = \algname{PPM}. So, indeed, we recover the same rate as \algname{SPPM-FS}. If $b=n$, then \algname{SPPM-BS} = \algname{SPPM-NS}. So, indeed, we recover the same rate as \algname{SPPM-NS}. We provide the detailed analysis in \Cref{sec:appendix_extreme_bs}.

\paragraph{Stratified Sampling (SS).} Let $C_1, \ldots, C_b$ be a partition of $[n]$ into $b$ nonempty blocks, as before. For each $i \in[n]$, let $B(i)$ indicate which block does $i$ belong to. In other words, $i \in C_j$ iff $B(i)=j$. Now, for each $j \in[b]$ pick $\xi_j \in C_j$ uniformly at random, and define $S=\cup_{j \in[b]}\left\{\xi_j\right\}$. Clearly, $p_i=\frac{1}{\left|C_{B(i)}\right|}$. Let's denote $\mathbf{i}_b\eqdef (i_1, \cdots, i_b), \mathbf{C}_b \eqdef C_1\times\cdots \times C_b$. Then, \Cref{eqn:main_8003} take on the form
{\small
\begin{align*}
\mu_{\mathrm{AS}}=\mu_{\mathrm{SS}}:=\min_{\mathbf{i}_b \in \mathbf{C}_b} \sum_{j=1}^b \frac{\mu_{i_j}\left|C_j\right|}{n}, \quad
    \sigma_{\star, \mathrm{AS}}^2=\sigma_{\star, \mathrm{SS}}^2:=\sum_{\mathbf{i}_b \in \mathbf{C}_b}\left(\prod_{j=1}^b \frac{1}{\left|C_j\right|}\right)\left\|\sum_{j=1}^b \frac{\left|C_j\right|}{n} \nabla f_{i_j}\left(x_{\star}\right)\right\|^2.
\end{align*}
}
 
\begin{restatable}[Stratified Sampling Variance Bounds]{shadedlemma}{lemma3}\label{lem:vr_ss}
        Consider the stratified sampling. For each $j \in[b]$, define
        \[
        \sigma_j^2:=\max _{i \in C_j}\left\|\nabla f_i\left(x_{\star}\right)-\frac{1}{\left|C_j\right|} \sum_{l \in C_j} \nabla f_l\left(x_{\star}\right)\right\|^2 .
        \]
    
        In words, $\sigma_j^2$ is the maximal squared distance of a gradient (at the optimum) from the mean of the gradients (at optimum) within cluster $C_j$. Then 
    \begin{align*}
        \sigma_{\star, \mathrm{SS}}^2 \leq \frac{b}{n^2} \sum_{j=1}^b\left|C_j\right|^2 \sigma_j^2 \leq b \max \left\{\sigma_1^2, \ldots, \sigma_b^2\right\}.
    \end{align*}
    \end{restatable}

\emph{Considering two extreme cases:} If $b=1$, then \algname{SPPM-SS} = \algname{SPPM-US}. So, indeed, we recover the same rate as \algname{SPPM-US}. If $b=n$, then \algname{SPPM-SS} = \algname{SPPM-FS}. So, indeed, we recover the same rate as \algname{SPPM-FS}. We provide the detailed analysis in \Cref{sec:appendix_extreme_ss}.

Note that Lemma~\ref{lem:vr_ss} provides insights into how the variance might be reduced through stratified sampling. For instance, in a scenario of complete inter-cluster homogeneity, where $\sigma_j^2 = 0$ for all $j$, both bounds imply that $0 = \sigma^2_{\star, \mathrm{SS}} \leq \sigma^2_{\star, \mathrm{BS}}.$
Thus, in this scenario, the convergence neighborhood of stratified sampling is better than that of block sampling.

\paragraph{Stratified sampling outperforms block sampling and nice sampling in convergence neighborhood.}

We theoretically compare stratified sampling with block sampling and nice sampling, advocating for stratified sampling as the superior method for future clustering experiments due to its optimal variance properties. We begin with the assumption of $b$ clusters of uniform size $b$ (\Cref{asm:uniform-clustering}), which simplifies the analysis by enabling comparisons of various sampling methods, all with the same sampling size, $b$: $b$-nice sampling, stratified sampling with $b$ clusters, and block sampling where all clusters are of uniform size $b$. Furthermore, we introduce the concept of optimal clustering for stratified sampling (noted as $\mathcal{C}_{b, \mathrm{SS}}$, \Cref{def:SS-clustering}) in response to a counterexample where block sampling and nice sampling achieve lower variance than stratified sampling (\Cref{example:SS_worse_than_BS_NICE}). Finally, we compare neighborhoods using the stated assumption.

\begin{lemma}\label{lem:SS_vs_NICE}
	Given \Cref{asm:uniform-clustering}, the following holds: $\sigma_{\star, \mathrm{SS}}^2\left(\mathcal{C}_{b, \mathrm{SS}}\right) \leq \sigma_{\star, \mathrm{NICE}}^2$ for arbitrary $b$. Moreover, the variance within the convergence neighborhood of stratified sampling is less than or equal to that of nice sampling: $\frac{\gamma \sigma^2_{\star, \mathrm{SS}}}{\gamma \mu^2_{\mathrm{SS}} + 2\mu_{\mathrm{SS}}}\left(\mathcal{C}_{b, \mathrm{SS}}\right) \leq \frac{\gamma \sigma^2_{\star, \mathrm{NICE}}}{\gamma \mu^2_{\mathrm{NICE}} + 2\mu_{\mathrm{NICE}}}.$
\end{lemma}

\Cref{lem:SS_vs_NICE} demonstrates that, under specific conditions, the stratified sampling neighborhood is preferable to that of nice sampling. One might assume that, under the same assumptions, a similar assertion could be made for showing that block sampling is inferior to stratified sampling . However, this has only been verified for the simplified case where both the block size and the number of blocks are $b=2$, as detailed in  \Cref{sec:ss_vs_bs_nice}.

\section{Experiments}
\begin{table}[!b]
	\centering
	\caption{$KT(\epsilon, \mathcal{S}, \gamma, \mathcal{A}\left(K\right))$.}
	\label{tab:alg_comparison}
	\resizebox{0.9\textwidth}{!}{%
		\begin{threeparttable}
            \centering
			\begin{tabular}{c m{0.5\textwidth} m{0.25\textwidth} c}
                    \toprule
				HP & Control & $KT(\cdots)$ & Experiment\\ 
				\hline
                \addlinespace
				\multirow{2}{*}{$\gamma$} & $\gamma\uparrow$ & $KT\downarrow$, $\epsilon\uparrow\tnote{\color{blue}(a)} $ & \ref{sec:trade_off} \\  \cline{2-4} \addlinespace
				~ & optimal $\left(\gamma, K\right)\uparrow$ & $\downarrow$  & \ref{sec:exp1} \\ 
				\hline \addlinespace
				\multirow{3}{*}{$\mathcal{A}$} & $\mu$-convex + BFGS/CG & $\downarrow$ compared to \algname{LocalGD} & \ref{sec:exp1}\\  \cline{2-4} \addlinespace
				~ & NonCVX and Hierarchical FL + ADAM with tuned lr & $\downarrow$ compared to \algname{LocalGD} & \ref{sec:nn}\\  
				\bottomrule
			\end{tabular}
			\begin{tablenotes}
				\item  [{\color{blue}(a)}] $\epsilon$ is a convergence neigbourhood or accuracy. 
			\end{tablenotes}
	\end{threeparttable}}
\end{table}
  
\paragraph{Practical decision-making with \algname{SPPM-AS}.}
In our analysis of \algname{SPPM-AS}, guided by theoretical foundations of Theorem \ref{thm:sppm_as} and empirical evidence summarized in \Cref{tab:alg_comparison}, we explore practical decision-making for varying scenarios. This includes adjustments in hyperparameters within the framework $KT(\epsilon, \mathcal{S}, \gamma, \mathcal{A}\left(K\right))$. Here, $\epsilon$ represents accuracy goal, $\mathcal{S}$ represents the sampling distribution, $\gamma$ is representing global learning rate (proximal operator parameter), $\mathcal{A}$ denotes the proximal optimization algorithm, while $K$ denotes the number of local communication rounds. In table \ref{tab:alg_comparison} we summarize how changes on following hyperparameters will influence target metric. With increasing learning rate $\gamma$ one achieves faster convergence with smaller accuracy, also noted as accuracy-rate tradeoff. Our primary observation that with an increase in both the learning rate, $\gamma$, and the number of local steps, $K$, leads to an improvement in the convergence rate. Employing various local solvers for proximal operators also shows an improvement in the convergence rate compared to \algname{FedAvg} in both convex and non-convex cases. 

\paragraph{Objective and datasets.}
Our analysis begins with logistic regression with a convex $l_2$ regularizer, which can be represented as:
  
$$
f_i(x) \eqdef \frac{1}{n_i}\sum_{j=1}^{n_i} \log \left(1 + \exp(-b_{i, j}x^T a_{i, j})\right) + \frac{\mu}{2}\|x\|^2,
$$

where $\mu$ is the regularization parameter, $n_i$ denotes the total number of data points at client $i$, $a_{i, j}$ are the feature vectors, and $b_{i, j} \in \{-1, 1\}$ are the corresponding labels. Each function $f_i$ exhibits $\mu$-strong convexity and $L_i$-smoothness, with $L_i$ computed as $\frac{1}{4n_i}\sum_{j=1}^{n_i}\|a_{i,j}\|^2 + \mu$. For our experiments, we set $\mu$ to 0.1.

Our study utilized datasets from the LibSVM repository \citep{chang2011libsvm}, including \texttt{mushrooms}, \texttt{a6a}, \texttt{ijcnn1.bz2}, and \texttt{a9a}. We divided these into feature-wise heterogeneous non-iid splits for FL, detailed in \Cref{sec:data_generation}, with a default cohort size of $10$.
We primarily examined logistic regression, finding results consistent with our theoretical framework, as discussed extensively in \Cref{sec:exp1} through \Cref{sec:trade_off}. Additional neural network experiments are detailed in \Cref{sec:nn} and \Cref{sec:nn_additional}. 

\begin{figure*}[!tb]
    \centering
    \begin{subfigure}[b]{0.24\textwidth}
        \centering
        \includegraphics[width=\textwidth, trim=0 0 0 0, clip]{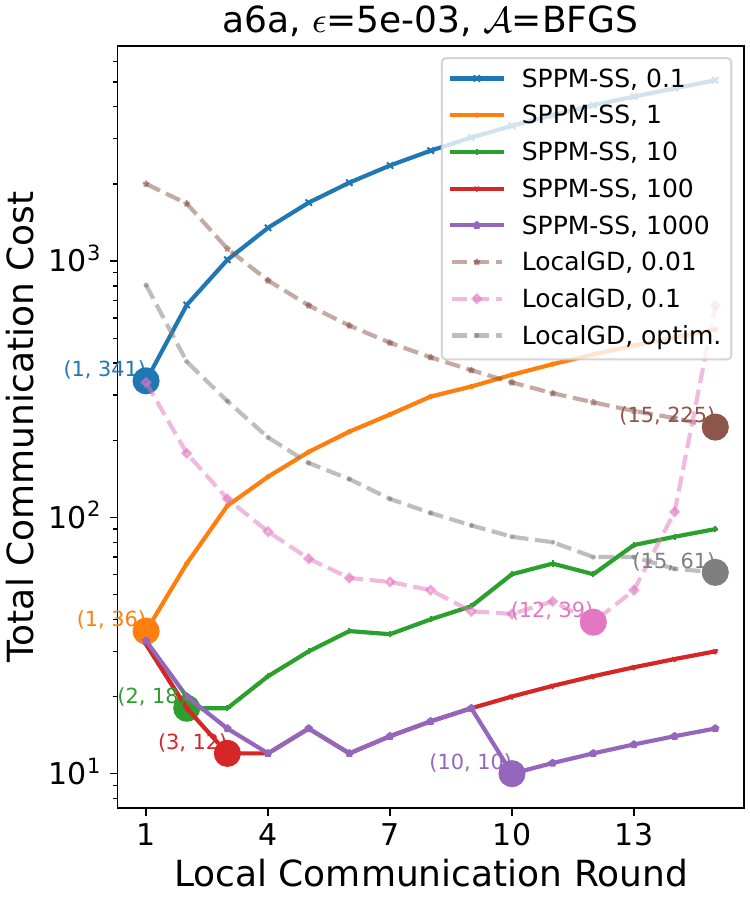}
        \caption{base}\label{fig:abs1_base}
    \end{subfigure}
    \hfill
    \begin{subfigure}[b]{0.24\textwidth}
        \centering
        \includegraphics[width=\textwidth, trim=0 0 0 0, clip]{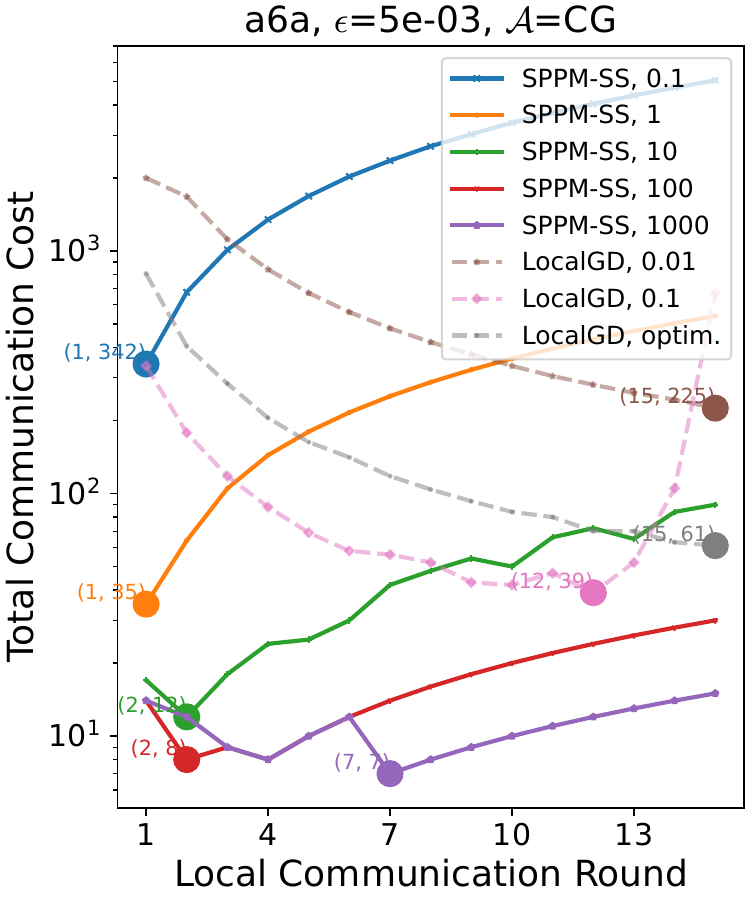}
        \caption{diff. prox solver}\label{fig:abs1_solver}
    \end{subfigure}
    \hfill
    \begin{subfigure}[b]{0.24\textwidth}
        \centering
        \includegraphics[width=\textwidth, trim=0 0 0 0, clip]{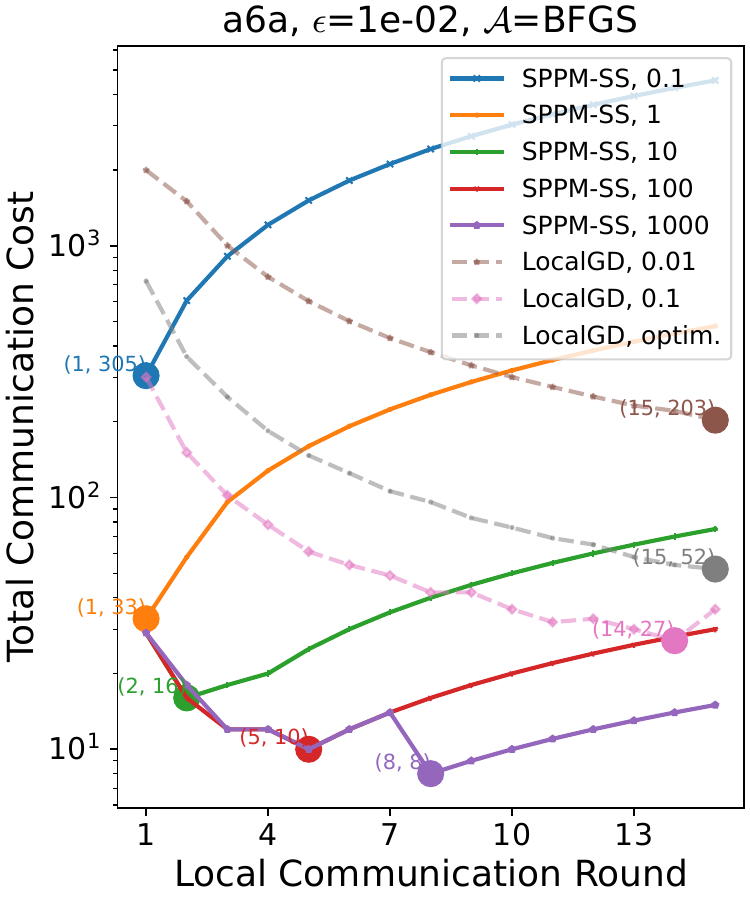}
        \caption{varying $\epsilon$}\label{fig:abs1_epsilon}
    \end{subfigure}
    \begin{subfigure}[b]{0.24\textwidth}
        \centering
        \includegraphics[width=\textwidth, trim=0 0 0 0, clip]{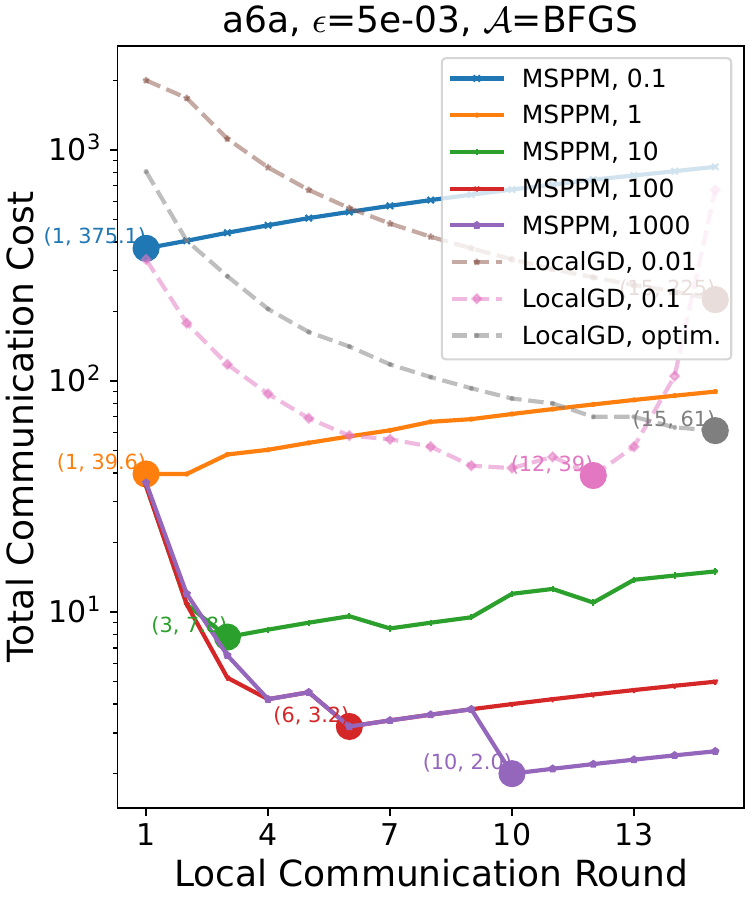}
        \caption{hierarchical}\label{fig:abs1_hierarchical}
    \end{subfigure}
    \caption{Analysis of total communication costs against local communication rounds for computing the proximal operator. For \algname{LocalGD}, we align the x-axis to the total local iterations, highlighting the absence of local communication. The aim is to minimize total communication for achieving a predefined global accuracy $\epsilon$, where $\sqn{x_T - x_\star}<\epsilon$. The optimal step size and minibatch sampling setup for \algname{LocalGD} are denoted as \algname{LocalGD, optim}. This showcases a comparison across varying $\epsilon$ values and proximal operator solvers (\texttt{CG} and \texttt{BFGS}).}
    \label{fig:abs1}
\end{figure*}

\subsection{On choosing sampling strategy}
As shown in Section \ref{sec:as}, multiple sampling techniques exist. We propose using clustering approach in conjuction with \algname{SPPM-SS} as the default sampling strategy for all our experiments. The Stratified Sampling Optimal Clustering is impractical due to the difficulty in finding $x_\star$; therefore, we employ a clustering heuristic that aligns with the concept of creating homogeneous worker groups. One such method is \algname{K-means}, which we use by default. More details on our clustering approach can be found in the Appendix \refeq{sec:data_generation}.
We compare various sampling techniques in \Cref{fig:sampling-logreg}. Extensive ablations verified the efficiency of stratified sampling over other strategies, due to variance reduction (Lemma \ref{lem:vr_ss}).

\begin{figure*}[!tb]
    \centering
    \begin{minipage}{0.315\textwidth}
        \centering
        \includegraphics[width=0.94\textwidth, trim=0 0 0 0, clip]{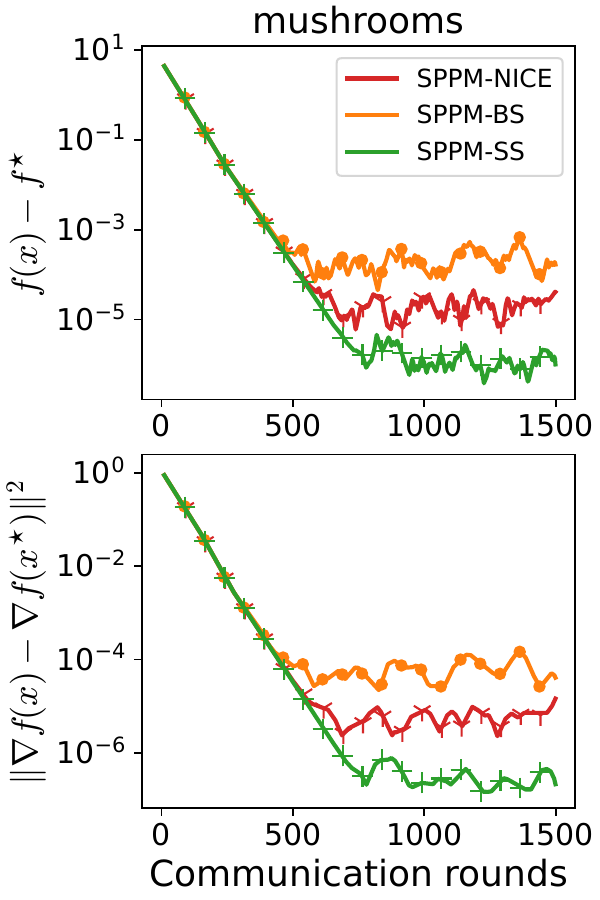}
        \hfill
        \caption{Sampling method comparison.}
        \label{fig:sampling-logreg}
    \end{minipage}
    \hfill  
    \begin{minipage}{0.665\textwidth}
        \centering
        \includegraphics[width=0.47\textwidth, trim=0 0 0 0, clip]{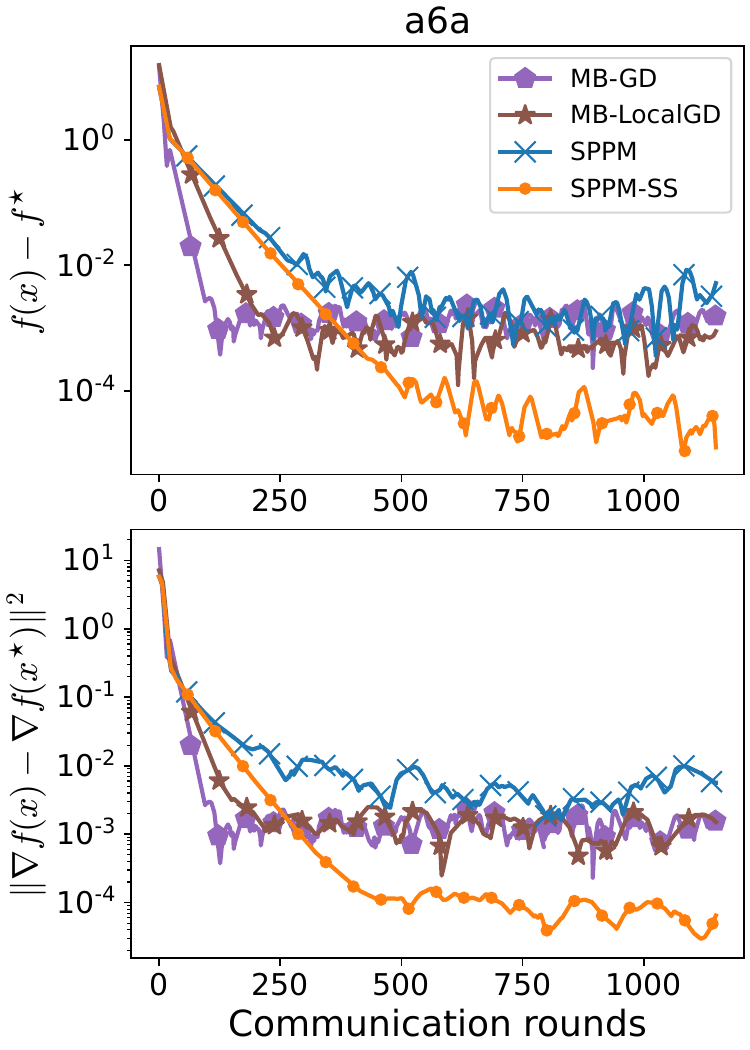}
        \hfill
        \includegraphics[width=0.47\textwidth, trim=0 0 0 0, clip]{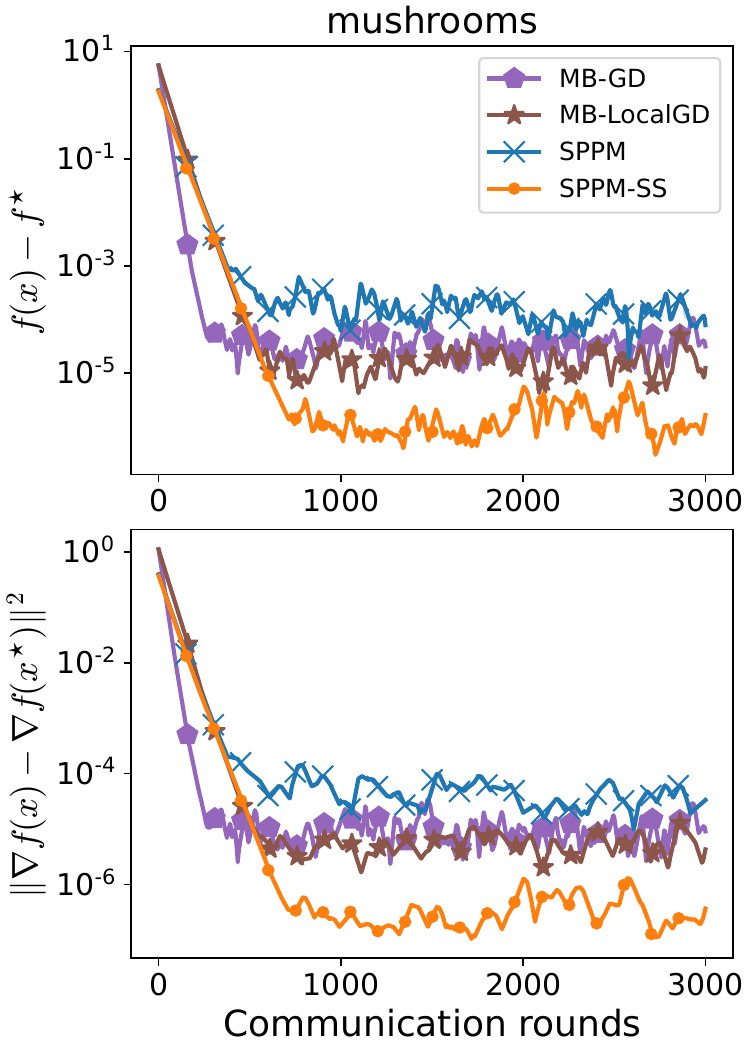}
        \caption{Convergence analysis compared to popular baselines. $\gamma=1.0$.}
        \label{fig:abs2}
    \end{minipage}
\end{figure*}

\subsection{Reducing communication cost via local rounds}\label{sec:exp1}
In this study, we investigate whether increasing the number of local communication rounds, denoted as $K$, in our proposed algorithm \algname{SPPM-SS}, can lead to a decrease in the total communication cost required to converge to a predetermined global accuracy $\epsilon > 0$. In \Cref{fig:tissue0}, we analyzed various datasets, including \texttt{a6a} and \texttt{mushrooms}, confirming that higher local communication rounds reduce communication costs, especially with larger learning rates. Our study includes both self-ablation of \algname{SPPM-SS} across different learning rate scales and comparisons with the widely-used cross-device FL method \algname{LocalGD} (or \algname{FedAvg}) on the selected cohort. Ablation studies were conducted with a large empirical learning rate of $0.1$, a smaller rate of $0.01$, and an optimal rate as per \cite{khaled2020better}, alongside minibatch sampling following \cite{gower2019sgd}.

In \Cref{fig:abs1}, we present more extensive ablations. Specifically, we set the \texttt{base} method (\Cref{fig:abs1_base}) using the dataset a6a, a proximal solver \texttt{BFGS}, and $\epsilon=5\cdot10^{-3}$. In \Cref{fig:abs1_solver}, we explore the use of an alternative solver, \texttt{CG} (Conjugate Gradient), noting some differences in outcomes. For instance, with a learning rate $\gamma=1000$, the optimal $K$ with \texttt{CG} becomes 7, lower than 10 in the \texttt{base} setting using \texttt{BFGS}. In \Cref{fig:abs1_epsilon}, we investigate the impact of varying $\epsilon=10^{-2}$. Our findings consistently show \algname{SPPM-SS}'s significant performance superiority over \algname{LocalGD}. 
  
\subsection{Impact of different solver $\mathcal{A}$}
We further explore the impact of various solvers on optimizing the proximal operators, showcasing representative methods in Table \ref{tab:comparison-solvers}. A detailed overview and comparison of local optimizers listed in the table are provided in Section \ref{sec:local_solvers}, given the extensive range of candidate options available. To highlight critical factors, we compare the performance of first-order methods, such as the Conjugate Gradient (\algname{CG}) method \citep{conjugate_gradients}, against second-order methods, like the Broyden-Fletcher-Goldfarb-Shanno (\algname{BFGS}) algorithm \citep{broyden1967quasi, shanno1970conditioning}, in the context of strongly convex settings. For non-convex settings, where first-order methods are prevalent in deep learning experiments, we examine an ablation among popular first-order local solvers, specifically choosing \algname{Mime-Adam} \citep{MIME} and \algname{FedAdam-AdaGrad} \citep{wang2021local}. The comparisons of different solvers for strongly convex settings are presented in \Cref{fig:abs1_solver}, with the non-convex comparison included in the appendix. Upon comparing first-order and second-order solvers in strongly convex settings, we observed that \algname{CG} outperforms \algname{BFGS} for our specific problem. In neural network experiments, \algname{FedAdam-AdaGrad} was found to be more effective than \algname{Mime-Adam}. However, it is important to note that all these solvers are viable options that have led to impressive performance outcomes.

\subsection{Comparative analysis with baseline algorithms}
In this section, we conduct an extensive comparison with several established cross-device FL baseline algorithms. Specifically, we examine \algname{MB-GD} (MiniBatch Gradient Descent with partial client participation), and \algname{MB-LocalGD}, which is the local gradient descent variant of \algname{MB-GD}. We default the number of local iterations to 5 and adopt the optimal learning rate as suggested by \cite{gower2019sgd}. To ensure a fair comparison, the cohort size $\left|C\right|$ is fixed at 10 for all minibatch methods, including our proposed \algname{SPPM-SS}. The results of this comparative analysis are depicted in \Cref{fig:abs2}.
Our findings reveal that \algname{SPPM-SS} consistently achieves convergence within a significantly smaller neighborhood when compared to the existing baselines. Notably, in contrast to \algname{MB-GD} and \algname{MB-LocalGD}, \algname{SPPM-SS} is capable of utilizing arbitrarily large learning rates. This attribute allows for faster convergence, although it does result in a larger neighborhood size. 

\begin{figure}[!tb]
    \centering
    \begin{minipage}{0.40\textwidth}
        \centering
        \includegraphics[width=\textwidth, trim=0 0 0 0, clip]{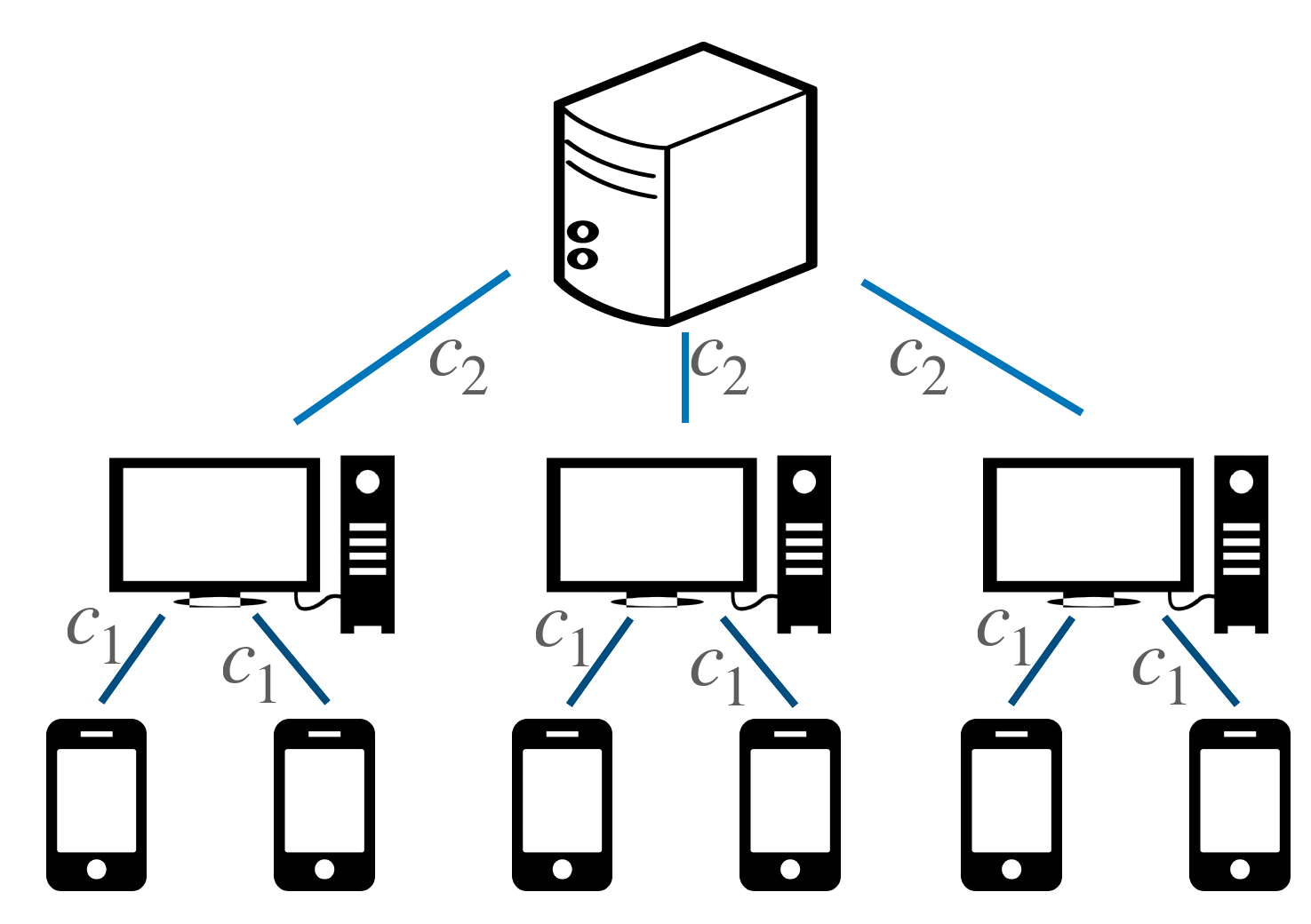}
        \caption{Server-hub-client hierarchical FL architecture.}
        \label{fig:hub}
    \end{minipage}
    \hfill 
    \begin{minipage}{0.58\textwidth}
        \centering
        \captionof{table}{Local optimizers for solving the proximal subproblem.}
        \label{tab:comparison-solvers}
        \resizebox{\textwidth}{!}{
        \begin{threeparttable}
            \begin{tabular}{lll}
                \toprule
                \bf Setting & \textbf{1st order} & \textbf{2nd order} \\
                \midrule
                Strongly-Convex & \begin{tabular}{@{}l@{}}
                        \algname{Conjugate Gradients (CG)} \\
                        \algname{Accelerated GD} \\
                        \algname{Local GD} \\
                        \algname{Scaffnew} \\
                    \end{tabular}  &  \begin{tabular}{@{}l@{}}
                        \algname{BFGS} \\
                        \algname{AICN} \\
                        \algname{LocalNewton} \\
                    \end{tabular}\\
                \midrule
                Nonconvex & \begin{tabular}{@{}l@{}}
                        \algname{Mime-Adam}\\
                        \algname{FedAdam-AdaGrad}\\
                        \algname{FedSpeed}\\ 
                    \end{tabular}  & \begin{tabular}{@{}l@{}}
                        \algname{Apollo}\\
                        \algname{OASIS}\\
                    \end{tabular} \\
                \bottomrule
            \end{tabular}
        \end{threeparttable}
        }
    \end{minipage}
\end{figure}

\subsection{Hierarchical federated learning}
We extend our analysis to a hub-based hierarchical FL structure, as conceptualized in \Cref{fig:hub}. This structure envisions a cluster directly connected to $m$ hubs, with each hub $m_i$ serving $n_i$ clients. The clients, grouped based on criteria such as region, communicate exclusively with their respective regional hub, which in turn communicates with the central server. Given the inherent nature of this hierarchical model, the communication cost $c_1$ from each client to its hub is consistently lower than the cost $c_2$ from each hub to the server. We define communication from clients to hubs as \textit{local communication} and from hubs to the server as \textit{global communication}. Under \algname{SPPM-SS}, the total cost is expressed as $(c_1 K + c_2) T_{\operatorname{SPPM-SS}}$, while for \algname{LocalGD}, it is $(c_1 + c_2)T_{\operatorname{LocalGD}}$. As established in Section \ref{sec:exp1}, $T_{\operatorname{SPPM-SS}}$ demonstrates significant improvement in total communication costs compared to \algname{LocalGD} within a hierarchical setting. Our objective is to illustrate this by contrasting the standard FL setting, depicted in \Cref{fig:abs1_base} with parameters $c_1=1$ and $c_2=0$, against the hierarchical FL structure, which assumes $c_1=0.1$ and $c_2=1$, as shown in \Cref{fig:abs1_hierarchical}. Given the variation in $c_1$ and $c_2$ values between these settings, a direct comparison of absolute communication costs is impractical. Therefore, our analysis focuses on the ratio of communication cost reduction in comparison to \algname{LocalGD}. For the \texttt{base} setting, \algname{LocalGD}'s optimal total communication cost is 39 with 12 local iterations, whereas for \algname{SPPM-SS} ($\gamma=1000$), it is reduced to 10 with 10 local and 1 global communication rounds, amounting to a 74.36\% reduction. With the hierarchical FL structure in \Cref{fig:abs1_hierarchical}, \algname{SPPM-SS} achieves an even more remarkable communication cost reduction of 94.87\%. Further ablation studies on varying local communication cost $c_1$ in the \Cref{sec:hierarchy-fl} corroborate these findings.

\subsection{Neural network evaluations}\label{sec:nn}
Our empirical analysis includes experiments on Convolutional Neural Networks (CNNs) using the FEMNIST dataset, as described in \cite{leaf}. We designed the experiments to include a total of 100 clients, with each client representing data from a unique user, thereby introducing natural heterogeneity into our study. We employed the \algname{Nice} sampling strategy with a cohort size of 10. In contrast to logistic regression models, here we utilize training accuracy as a surrogate for the target accuracy $\epsilon$. For the optimization of the proximal operator, we selected the Adam optimizer, with the learning rate meticulously fine-tuned over a linear grid. Detailed descriptions of the training procedures and the CNN architecture are provided in the Appendix \ref{sec:nn_additional}.

Our analysis primarily focuses on the hierarchical FL structure. Initially, we draw a comparison between our proposed method, \algname{SPPM-AS}, and \algname{LocalGD}. The crux of our investigation is the total communication cost required to achieve a predetermined level of accuracy, with findings detailed in \Cref{fig:dl1}. Significantly, \algname{SPPM-AS} demonstrates enhanced performance with the integration of multiple local communication rounds. Notably, the optimal number of these rounds tends to increase alongside the parameter $\gamma$. For each configuration, the convergence patterns corresponding to the sets of optimally tuned hyperparameters are depicted in \Cref{fig:dl1-convergence}.

\begin{figure*}[!tb]
    \centering
    \begin{minipage}{0.42\textwidth}
        \centering
        \includegraphics[width=\textwidth, trim=0 0 0 0, clip]{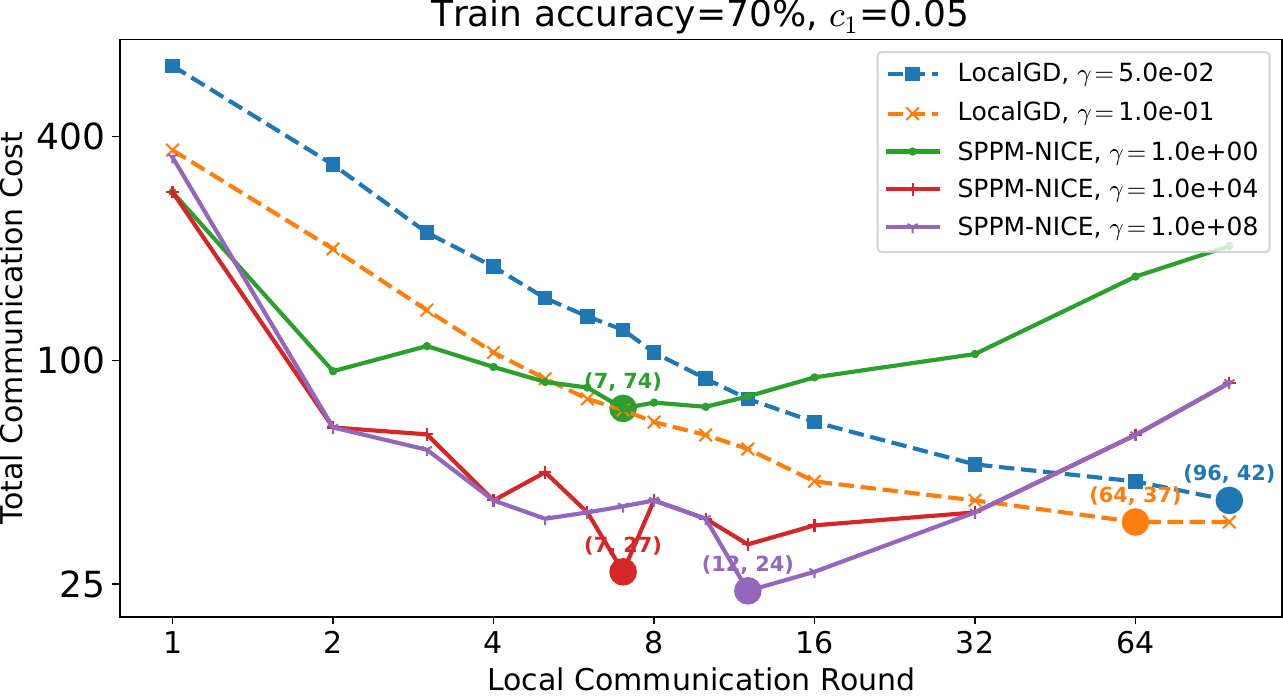}
        \caption{Communication cost for achieving 70\% accuracy in hierarchical FL ($c_1=0.05$, $c_2=1$).}
        \label{fig:dl1}
    \end{minipage}
    \hfill 
    \begin{minipage}{0.56\textwidth}
        \centering
        \includegraphics[width=\textwidth, trim=0 0 0 0, clip]{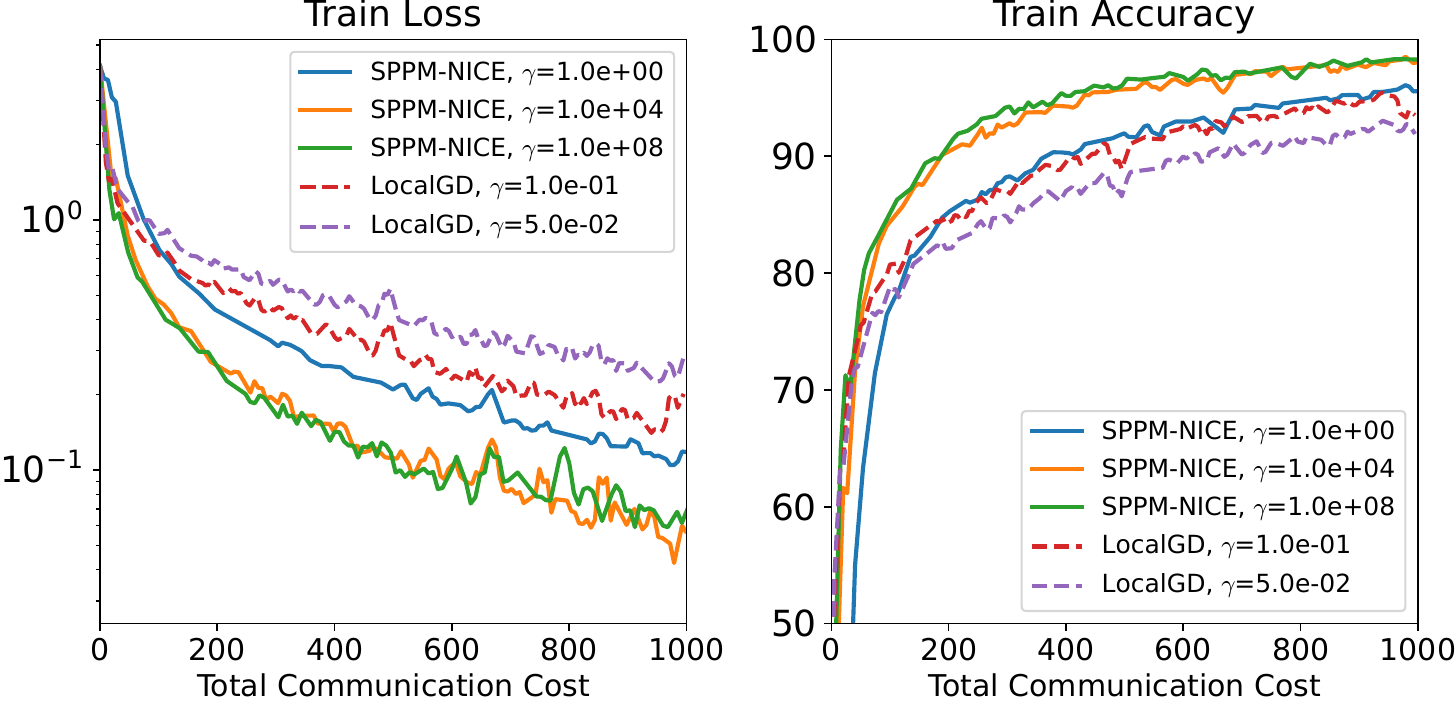}
        \caption{Convergence with optimal hyperparameters. $c_1$ is $0.05$, $c_2=1$.}
        \label{fig:dl1-convergence}
    \end{minipage}
\end{figure*}


%% file: Chapter_4_SymWanda.tex
\chapter{Symmetric Post-Training Compression}
\label{chapter_symwanda}
\thispagestyle{empty}

\section{Introduction} Large Language Models (LLMs) \citep{OPT, LlaMA, LlaMA2, Phi2} have demonstrated remarkable capabilities across a variety of tasks. However, their extensive size often hinders practical deployment. Interest in LLM compression has surged in recent years, driven by the need to reduce model sizes while maintaining performance \citep{SmoothQuant, SparseGPT, Wanda, RIA, PV-Tuning}. This paper focuses on LLM \textbf{post-training pruning (PTP)}, a prevalent method for reducing the footprint of pre-trained weights.

A common approach to pruning is magnitude-based pruning, where elements of each layer's weights with smaller absolute values are set to zero. In contrast, \algname{Wanda} \citep{Wanda} introduced an innovative method that scales the weights by the activations of each layer, demonstrating promising performance on standard benchmarks. Building upon this, \algname{RIA} \citep{RIA} further improved the approach by evaluating the relative importance of each weight across its corresponding row and column before pruning. While their empirical results are encouraging, the underlying mechanisms remain poorly understood. This leads us to our first question:

\emph{Can we provide theoretical support for post-training pruning methods and derive more efficient algorithms with minimal adaptations to the existing framework?}

To deepen our understanding of these popular PTP methods, we introduce a novel formulation—referred to as \textbf{Sym}metric \textbf{W}eight \textbf{And} \textbf{A}ctivation (\algname{SymWanda}), which aims to efficiently leverage \textit{both} the input activation of a layer and the output for that layer. This symmetric and generalized approach provides theoretical insights into the mechanisms of established empirical methods such as \algname{Wanda} and \algname{RIA}. 

Intrinsic PTP methods have demonstrated remarkable performance, as reflected by perplexity scores and zero-shot accuracy. However, their performance can degrade significantly when the sparsity ratio is high. This is due to the intrinsic reconstruction error between the pruned weights and the original pre-trained weights. Minimizing this reconstruction error is particularly important for efficient post-training pruning. Beyond LLM pruning, we explore further fine-tuning to enhance model efficiency and performance. This brings us to our second problem:

\emph{Can we fine-tune pruned LLMs without further training and outperforms state-of-the-art methods with minimal effort?}

\textbf{Dynamic sparse training (DST)} has gained attention for selectively updating and maintaining a subset of network parameters throughout the training process while dynamically adapting the sparse topology through weight operations. Its proven efficiency in enabling effective training suggests DST could be a promising approach for fine-tuning LLMs in an efficient manner. However, DST inherently requires backpropagation to train subnetworks, and its effectiveness heavily depends on a sufficient number of weight updates \citep{liu2021we}. 

Interestingly, the pruning-and-growing step within DST offers a training-free methodology, where sparse mask adaptation is based solely on weight properties such as magnitude \citep{mocanu2018scalable}. This opens up a potential alternative for addressing the challenge: Instead of relying on computationally intensive backpropagation for fine-tuning sparse LLMs, we can explore the iterative updating of sparse masks in a training-free manner. Motivated by this insight, we focus on training-free fine-tuning approaches.

\algname{DSnoT} \citep{zhang2023dynamic} introduced a straightforward yet effective method for pruning and growing weights using their values and statistical metrics (e.g., expectation and variance) for each ongoing pruning row. Inspired by \algname{Wanda}, \algname{DSnoT} achieves simplicity but falls short of fully leveraging relative weight information, particularly in scenarios where weight distributions are highly non-uniform and contain many outliers \citep{RIA}. To address these limitations, we propose incorporating relative weight importance into the growing criterion design. Furthermore, we observe that directly optimizing for reconstruction error is suboptimal. To improve performance, we introduce a regularization term that relaxes the decision boundary.
Our new designs demonstrate significant efficiency and consistently achieve promising performance, paving the way for more effective and computationally feasible fine-tuning methods for sparse LLMs.





Our \textbf{contributions} are summarized as follows: 
\begin{itemize}
    \item We propose a novel formulation, \algname{SymWanda}, which minimizes the impact of pruning on both input activations and output influences of weights. This approach provides theoretical insights into the empirical successes of methods such as \algname{Wanda} and \algname{RIA}. 
    \item Building on this formulation, we introduce a series of innovative pruning strategies. Extensive experiments validate the effectiveness of our methods. Notably, we incorporate an efficient stochastic approach for manipulating relative importance, which achieves superior performance with highly reduced sampling cost.
    \item We present a novel training-free fine-tuning method \ft that leverages relative weight importance and a regularized decision boundary within a pruning-and-growing framework. This approach significantly outperforms strong baselines, achieving remarkable results.
\end{itemize}
\section{Related Work}
\paragraph{Traditional model pruning.} 
Pruning has emerged as a powerful strategy to compress and accelerate deep neural networks by removing redundant connections while preserving overall performance \citep{han2015learning, frankle2018lottery, hoefler2021sparsity}. Early works introduced iterative pruning-and-retraining approaches, which iteratively identify unimportant weights, discard them, and retrain the resulting sparse network to recover accuracy \citep{lecun1989optimal, han2015learning}. More recent dynamic sparse training techniques \citep{mocanu2018scalable, bellec2018deep, lee2018snip, mostafa2019parameter} start from a sparse initialization and continuously prune and grow connections throughout training. These methods integrate sparsification into the training loop, yielding promising trade-offs between model size and performance. A prominent line of work has leveraged learnable thresholds to realize non-uniform sparsity \citep{kusupati2020soft} or combined magnitude-based pruning with periodic connectivity updates to regrow valuable weights \citep{RigL, SRigL}. However, most of these methods still rely on standard back-propagation over the full parameter set, which can be prohibitively expensive when scaling up to LLMs.
 
\paragraph{LLM post-training pruning.} 
The substantial computational demands of LLMs have raised the development of pruning methods tailored to reduce parameters counts without compromising performance \citep{li2023sparse, zhu2024survey}. Among these methods, post-training pruning eliminates redundant parameters in a pre-training network without requiring resource-intensive fine-tuning. For instance, \algname{SparseGPT} \citep{SparseGPT} leverages second-order information to solve layer-wise reconstruction problems, supporting both unstructured and N:M structured sparsity \citep{zhou2021learning}. \algname{Wanda} \citep{Wanda} introduces a pruning metric that incorporates both weight magnitudes and corresponding input activations, achieving perplexity performance comparable to \algname{SparseGPT} while surpassing simple magnitude-based pruning. The \algname{RIA} method \citep{RIA} builds on \algname{Wanda} by considering relative weight importance, offering performance improvements at minimal additional cost. Moreover, \algname{DSnoT} \citep{zhang2023dynamic} proposes pruning and regrowing weights based on statistical properties (e.g., mean and variance) in each pruning row, obviating the need for retraining.

\section{Symmetric Wanda}
\subsection{Prerequisites}
Post-training pruning is defined as follows: consider a target sparsity ratio $\epsilon \in [0, 1)$, a set of calibration inputs $\mathbf{X} \in \mathbb{R}^{a \times b}$, and pre-trained weights $\mathbf{W} \in \mathbb{R}^{b \times c}$. For clarity in the mathematical framework, we abstract the dimensions of inputs and weights. Specifically, in the context of large language models, let $a \eqdef C_{\text{in}}$, $b \eqdef N \times L$, and $c \equiv C_{\text{out}}$, where $N$ and $L$ denote the batch size and sequence length, respectively. The objective is to identify an optimal pruned weight matrix $\widetilde{\mathbf{W}} \in \mathbb{R}^{b \times c}$ that minimizes:

\begin{equation}\label{obj1}\tag{InpRecon}
    f(\widetilde{\mathbf{W}}) \eqdef \|\mathbf{X} (\widetilde{\mathbf{W}} - \mathbf{W})\|_F^2,
\end{equation}

where the optimization challenge is:
\begin{equation}
    \minimize f(\widetilde{\mathbf{W}}) \ \ s.t. \ \ \text{Mem}(\widetilde{\mathbf{W}}) \leq (1 - \epsilon) \text{Mem}(\mathbf{W}), \notag
\end{equation}

where $\text{Mem}(\cdot)$ denotes the memory consumption associated with a weight matrix, and (\ref{obj1}) quantifies the input reconstruction error.

This formulation applies to various post-training compression techniques, including both pruning \citep{SparseGPT, Wanda, RIA} and quantization \citep{GPTQ, AQLM}. Our focus here is specifically on post-training pruning.

\begin{table*}[t]
    \centering
    \scriptsize
    \caption{\footnotesize Comparison of LLM post-training pruning algorithms.}
    \label{tab:comparison}
    \begin{threeparttable}
    \resizebox{\textwidth}{!}{
    \renewcommand{\arraystretch}{1.8}
\begin{tabular}{lllllll}
\toprule
{\bf Algorithm}  & \bf W? & \bf Act.? & {$\vX$} & $\vY$ & {\bf $\vS_{jk}$}\tnote{\color{blue}(a)} & \bf Comment \\
\midrule
\rowcolor{bgcolor4}
\algname{General Sym.} & \cmark & \cmark & $\vX$ & $\vY$ & $|\vW_{jk}| \left(\|\mathbf{X}_{:j}\|_2 + \|\mathbf{Y}_{k:}\|_2 \right)$ & \Cref{lemma:lm1}\\
\algname{Marginal} & \cmark & \xmark & \bf $\vI$ & \bf 0 & $|\vW_{jk}|$ & - \\
\algname{Wanda} & \cmark & \cmark & \bf X & \bf 0 & $|\vW_{jk}|\norm{\vX_{:j}}_2$ & \Cref{corollary1_2}\\
\rowcolor{bgcolor4}
\algname{OWanda}& \cmark & \cmark & \bf 0 & \bf Y & $|\vW_{jk}|\norm{\vY_{k:}}_2$ & \Cref{corollary1}\\
\rowcolor{bgcolor4}
\algname{Symmetric}& \cmark & \cmark & $\vW^T$ & $\vW^T$ & $|\vW_{jk}| \sqrt{\sqn{\vW_{j:}}_2 + \sqn{\vW_{:k}}}_2$ & \Cref{corollary2}\\

\algname{RI (v1)} & \multirow{1}{*}{\cmark} & \multirow{1}{*}{\xmark} & $t_j(1;, \cdots;, 1)$, $t_j = ({\sqrt{b} \norm{\vW_{j:}}}_1)^{-1}$\tnote{\color{blue}(a)} & $s_k(1, \cdots, 1)$, $s_k = \left({\sqrt{c}\norm{\vW_{:k}}_1}\right)^{-1}$ & \multirow{1}{*}{$\norm{\vW_{j:}}_1^{-1} + \norm{\vW_{:k}}_1^{-1}$} & \Cref{thm:main2}\\  
\algname{RI (v2)}  & \cmark & \xmark & $\Diag(\|\mathbf{W}_{1:}\|^{-1}_1, \ldots, \|\mathbf{W}_{b:}\|^{-1}_1)$ & $\Diag(\|\mathbf{W}_{:1}\|^{-1}_1, \ldots, \|\mathbf{W}_{:c}\|^{-1}_1)$ & $\norm{\vW_{j:}}_1^{-1} + \norm{\vW_{:k}}_1^{-1}$ & \Cref{thm:main2}\\

\algname{RIA} & \cmark & \cmark & $\delta_{u=j} \delta_{v=p} {\left\|\mathbf{C}_{: j}\right\|_2^\alpha}{\left\|\mathbf{W}_{j:}\right\|_1^{-1}}$\tnote{\color{blue}(c)} & $\delta_{u=s} \delta_{v=k} {\left\|\mathbf{C}_{: j}\right\|_2^\alpha}{\left\|\mathbf{W}_{:k}\right\|_1^{-1}}$ & $\left({\left\|\mathbf{W}_{j:}\right\|_1^{-1}}+{\left\|\mathbf{W}_{:k}\right\|_1^{-1}}\right) \left\|\mathbf{\vX}_{:j}\right\|_2^{\alpha}$ & \Cref{lem:ria}\\

\rowcolor{bgcolor4}
\algname{General (diag.)} & \cmark & \cmark & $\vA \vD_{\vX}$\tnote{\color{blue}(d)} & $\vD_{\vY}\vB$ & ${\left\|\mathbf{A}_{:j}\right\|_2}{\left\|\mathbf{W}_{j:}\right\|_1^{-1}} + {\left\|\mathbf{B}_{k:}\right\|_2}{\left\|\mathbf{W}_{:k}\right\|_1^{-1}}$ & \Cref{lem:general2}\\

\rowcolor{bgcolor4}
\algname{$\ell_p$-norm (v1)} & \cmark & \xmark\tnote{\color{blue}(e)} & ${\left\|\mathbf{W}_{j:}\right\|_p^{-1} \cdot\left\|\mathbf{W}_{j:}\right\|_2^{-1}} \cdot \mathbf{W}_{j:}^{\top}$ & ${\left\|\mathbf{W}_{:k}\right\|_p^{-1} \cdot\left\|\mathbf{W}_{:k}\right\|_2^{-1}} \cdot \mathbf{W}_{:k}^{\top}$ & \multirow{1}{*}{$|\mathbf{W}_{jk}| (\|\mathbf{W}_{j:}\|^{-1}_p + \|\mathbf{W}_{:k}\|^{-1}_p)$} &\Cref{lem:generalized_p_norm}\\

\rowcolor{bgcolor4}
\algname{$\ell_p$-norm (v2)} & \cmark & \xmark & $\left\|\mathbf{W}_{j:}\right\|_p^{-1} \cdot \mathbf{u}$ & $\left\|\mathbf{W}_{:k}\right\|_p^{-1} \cdot \mathbf{v}$ & $|\mathbf{W}_{jk}| (\|\mathbf{W}_{j:}\|^{-1}_p + \|\mathbf{W}_{:k}\|^{-1}_p)$ & \Cref{lem:random_unit_vector_scaling}\\

\rowcolor{bgcolor4}
\algname{StochRIA} & \cmark & \xmark & ${\mathbf{1}_{\{i \in S_j\}}}\left({\|\mathbf{W}_{j:S_j}\|_{1}\sqrt{\tau}}\right)^{-1}$ & ${\mathbf{1}_{\{i \in S_k\}}} \left({\|\mathbf{W}_{S_k:k}\|_{1}\sqrt{\tau}}\right)^{-1}$ & $|\mathbf{W}_{jk}| (\|\mathbf{W}_{j:S_j}\|_1^{-1} + \|\mathbf{W}_{S_k:k}\|_1^{-1})$ & \Cref{lem:stochria}\\
\bottomrule
\end{tabular}}
  \begin{tablenotes}
        {\tiny
        \item [{\color{blue}(a)}] \parbox[t]{0.58\linewidth}{Without loss of generality, we consider the elimination of a single weight, $\vW_{jk}$. The detailed explanation can be found in \Cref{lemma:lm1} and \Cref{sec:new_form}.}
        \item [{\color{blue}(b)}] \parbox[t]{0.58\linewidth}{For simplicity, instead of displaying the entire matrices $\mathbf{X}$ and $\mathbf{Y}$, we present the columns $\mathbf{X}_{:j}$ and the rows $\mathbf{Y}_{k:}$. This design is employed in the algorithms \algname{RI}, \algname{RIA}, $\ell_p$-norm, and \algname{StochRIA}.}
        \item [{\color{blue}(c)}] The Kronecker delta, denoted by $\delta_{i j}$, is a function of two indices $i$ and $j$ that equals 1 if $i=j$ and 0 otherwise.
        \item [{\color{blue}(d)}] \parbox[t]{0.58\linewidth}{$\mathbf{D}_{\mathbf{X}}$ and $\mathbf{D}_{\mathbf{Y}}$ are the diagonal matrices associated with $\mathbf{W}$, as defined in \Cref{sec:general_solution}. }
        \item [{\color{blue}(e)}] \parbox[t]{0.58\linewidth}{By default, for \algname{$\ell_p$-norm} and \algname{StochRIA}, we do not consider the input activation. However, the design is similar to the transition from \algname{RI} to \algname{RIA}, as described in \Cref{sec:ria}.} 
        }  
    \end{tablenotes}
    \end{threeparttable}
\end{table*}
\subsection{Symmetric Wanda: new formulations}\label{sec:new_form}
Building upon the methods introduced in \algname{Wanda} \citep{Wanda}, which considered both weights and activations, and later improvements by \algname{RIA} \citep{RIA}, which analyzed the relative importance of weights by summing over corresponding rows and columns, we provide new insights by redefining our optimization objective. Apart from the previous defined input calibration $\vX$, we particularly introduce the output calibration $\vY \in \mathbb{R}^{c\times d}$. Considering both the input and output dependencies, we express the objective as:

\begin{equation}\label{obj2}\tag{Sym}
    g(\widetilde{\mathbf{W}}) \eqdef \|\mathbf{X} (\widetilde{\mathbf{W}} - \mathbf{W})\|_F + \| (\widetilde{\mathbf{W}} - \mathbf{W})\vY\|_F,
\end{equation}

and propose to solve:
\begin{equation}
    \operatorname{minimize} \ g(\widetilde{\mathbf{W}}), \ \ s.t. \ \ \text{Mem}(\widetilde{\mathbf{W}}) \leq (1 - \epsilon) \text{Mem}(\mathbf{W}). \notag
\end{equation}

We refer to the method that utilizes the general matrix in (\ref{obj2}) without instantiation as \algname{SymWanda}, which is designed to minimize the reconstruction error affected by both the input $\mathbf{X}$ and the output $\mathbf{Y}$. It is important to note that this formulation employs non-squared Frobenius norms to facilitate better theoretical interpretations. It is important to note that this formulation employs \emph{non-squared} Frobenius norms to facilitate better theoretical interpretations. A squared norm version is also provided in \Cref{sec:squared_frobenius} for comparison. We elucidate the efficacy of both approaches and provide new theoretical insights into the performance advantages previously observed with \algname{Wanda} and \algname{RIA}.

\begin{lemma}\label{lemma:lm1}
    Assume we aim to eliminate a single weight $\vW_{jk}$, setting $\widetilde{\mathbf{W}}_{jk} = 0$ and keeping all other weights unchanged. The simplified expression for $g(\widetilde{\mathbf{W}})$ becomes:

    \begin{equation}\label{eqn0}
        g(\widetilde{\mathbf{W}}) = |\vW_{jk}| \left(\|\mathbf{X}_{:j}\|_2 + \|\mathbf{Y}_{k:}\|_2 \right) \eqdef \vS_{jk}, 
    \end{equation}

    where $\mathbf{X}_{:j}$ and $\mathbf{Y}_{k:}$ represent the j-th column and k-th row of $\mathbf{X}$ and $\mathbf{Y}$, respectively. 
\end{lemma}

This formulation (\ref{eqn0}) underscores the impact of individual weights on the error metrics and guides the pruning process. While Lemma \ref{lemma:lm1} simplifies the formulation for pruning a single weight, the general approach can be extended to multiple weights iteratively. This method facilitates a robust pruning strategy that is backed by both empirical results and theoretical foundations, bridging the gap in understanding observed in prior studies such as \algname{Wanda} \citep{Wanda} and \algname{RIA} \citep{RIA}.


\begin{corollary}\label{corollary1_2}
    Setting $\mathbf{Y} = \mathbf{0} \in \mathbb{R}^{c \times d}$ transitions our method to \emph{input} \algname{Wanda}, described by $\mathbf{S}_{jk} \eqdef |\mathbf{W}_{jk}| \|\mathbf{X}_{:j}\|_2$. 
\end{corollary}

This directly aligns with the objective in \cite{Wanda}, demonstrating that \algname{Wanda} is a specific case under our broader framework.

\begin{corollary}\label{corollary1} 
    Conversely, choosing $\mathbf{X} = \mathbf{0} \in \mathbb{R}^{a \times b}$ simplifies our pruning method to what we term \emph{output} \algname{Wanda} (denoted as \algname{OWanda}), where the score matrix becomes $\mathbf{S}_{jk} \eqdef |\mathbf{W}_{jk}| \|\mathbf{Y}_{k:}\|_2$. 
\end{corollary}

\begin{corollary}\label{corollary2}
By setting $\mathbf{X} = \mathbf{W}^\top \in \mathbb{R}^{c \times b} (a = c)$ and $\mathbf{Y} = \mathbf{W}^\top \in \mathbb{R}^{c \times b} (d=b)$, the score matrix $\mathbf{S}_{jk}$ is redefined as $|\mathbf{W}_{jk}| (\|\mathbf{W}_{j:}\|_2 + \|\mathbf{W}_{:k}\|_2)$. 
\end{corollary} 

This configuration suggests an alternative masking approach and segues into a further analysis on how our method encompasses both \algname{Wanda} and \algname{RIA} as special cases. The following theorem provides a provable construction to recover the relative importance design in \cite{RIA}.

\begin{theorem}\label{thm:main2}
    Assuming $a = b$ and $c = d$, consider one of the following strategies:
    \begin{itemize}
        \item $\mathbf{X}_{:j} \eqdef t_{j} (1; \ldots; 1) \in \mathbb{R}^{b \times 1}$ and $\mathbf{Y}_{k:} \eqdef s_k (1, \ldots, 1) \in \mathbb{R}^{1 \times c}$, where $t_j = ({\sqrt{b} \|\mathbf{W}_{j:}\|_1})^{-1}$ and $s_k = ({\sqrt{c} \|\mathbf{W}_{:k}\|_1})^{-1}$.
        \item $\mathbf{X} = \Diag(\|\mathbf{W}_{1:}\|^{-1}_1, \ldots, \|\mathbf{W}_{b:}\|^{-1}_1)$ and $\mathbf{Y} = \Diag(\|\mathbf{W}_{:1}\|^{-1}_1, \ldots, \|\mathbf{W}_{:c}\|^{-1}_1)$.
    \end{itemize}
    For these configurations, the condition $\|\mathbf{X}_{:j}\|_2 + \|\mathbf{Y}_{k:}\|_2 = \alpha_{jk} \eqdef {\|\mathbf{W}_{j:}\|_1^{-1}} + {\|\mathbf{W}_{:k}\|_1^{-1}}$ holds for all $j, k$.
\end{theorem}

This theorem elucidates that our methodology can invariably reconstruct the framework of relative importance \algname{RI} in \citep{RIA}, validating the adaptability and breadth of our proposed pruning strategy. 

\subsection{From relative importance (RI) to RI activation}\label{sec:ria}
In \Cref{thm:main2}, we revisit the concept of Relative Importance (\algname{RI}). Specifically, we represent \algname{RI} by the following equation:
$$
\vS_{jk} = {|\vW_{jk}|}{\norm{\vW_{j:}}_1^{-1}} + {|\vW_{jk}|}{\norm{\vW_{:k}}_1^{-1}} \eqdef \algname{RI}_{jk}.
$$

\cite{RIA} also introduces an enhanced version of \algname{RI}, termed RI with Activation (\algname{RIA}), which incorporates the $\ell_2$-norm of activations:

\begin{equation}\label{eqn:ria}
    \algname{RIA}_{jk} = \algname{RI}_{jk} \cdot \norm{\vX_{:j}}_2^\alpha,
\end{equation}

where \(\alpha\) is controlling the strength of activations.

This section aims to explore the derivation of \(\algname{RIA}\) with theoretical grounding in \(\algname{RI}\). To clarify our notation and avoid confusion, we are aiming at finding the suitable $\vA \in \mathbb{R}^{a\times b}$ and $\vB \in \mathbb{R}^{c\times d}$ such as:
$$
\left\|\mathbf{A}_{j:}\right\|_2+\left\|\mathbf{B}_{:k}\right\|_2=\left({\left\|\mathbf{W}_{j:}\right\|_1^{-1}}+{\left\|\mathbf{W}_{:k}\right\|_1^{-1}}\right) \cdot\left\|\mathbf{C}_{:j}\right\|_2^{\alpha},
$$
where $\vC_{:j}$ will be instantiated as $\vX_{:j}$ to satisfy \Cref{eqn:ria}.

    
\begin{lemma}\label{lem:ria}
Let $p$ be a valid column index for $\mathbf{A}$. Define $\mathbf{A}_{uv} = 0$ for all $(u,v)\neq (j,p)$, and 
$
\mathbf{A}_{j,p} 
= {\|\mathbf{C}_{:j}\|_2^\alpha}{\|\mathbf{W}_{j:}\|_1^{-1}}.
$
Similarly, let $s$ be a valid row index for $\mathbf{B}$. Define $\mathbf{B}_{uv} = 0$ for all $(u,v)\neq (s,k)$, and 
$
\mathbf{B}_{s,k}
= {\|\mathbf{C}_{:j}\|_2^\alpha}{\|\mathbf{W}_{:k}\|_1^{-1}}.
$
Then we recover \Cref{eqn:ria}.
\end{lemma}

The nonzero element in \(\mathbf{A}\) ensures that the \(\ell_2\)-norm of the \(j\)-th row of \(\mathbf{A}\) is:
$
\left\|\mathbf{A}_{j:}\right\|_2={\left\|\mathbf{W}_{j:}\right\|_1^{-1}} \cdot \left\|\mathbf{C}_{:j}\right\|_2^\alpha.
$
Similarly, the nonzero element in \(\mathbf{B}\) ensures that the \(\ell_2\)-norm of the \(k\)-th column of \(\mathbf{B}\) is:
$
\left\|\mathbf{B}_{:k}\right\|_2={\left\|\mathbf{W}_{:k}\right\|_1^{-1}} \cdot\left\|\mathbf{C}_{:j}\right\|_2^\alpha.
$
Combining these norms fulfills the intended equation.

\subsection{General solution}\label{sec:general_solution}
In \Cref{thm:main2}, we presented two distinct strategies for recovering the relative importance as described in \cite{RIA}. Following this, in \Cref{lem:ria}, we constructed a method that accounts for both the weights and the input activations. Inspired by the diagonal design in \Cref{thm:main2}, we now propose a general variant that considers both the weights and the activations.

Given that $\mathbf{D}_{\mathbf{X}} \in \mathbb{R}^{b\times b}$ and $\mathbf{D}_{\mathbf{Y}} \in \mathbb{R}^{c\times c}$ are diagonal matrices with entries defined as $\left(\mathbf{D}_{\mathbf{X}}\right)_{ii} = x_i = \left\|\mathbf{W}_{i:}\right\|_1^{-1}$ and $\left(\mathbf{D}_{\mathbf{Y}}\right)_{ii} = y_i = \left\|\mathbf{W}_{:i}\right\|_1^{-1}$ respectively, and $\mathbf{A}\in \mathbb{R}^{a\times b}$ and $\mathbf{B}\in \mbR^{c\times d}$ are arbitrary matrices, our objective is to compute the sum of norms:
$
\left\|\left(\mathbf{A} \mathbf{D}_{\mathbf{X}}\right)_{:j}\right\|_2 + \left\|\left(\mathbf{D}_{\mathbf{Y}} \mathbf{B}\right)_{k:}\right\|_2.
$

\begin{lemma}\label{lem:general2}
    Given the above definition, we show 
    $$
    \left\|\left(\mathbf{A} \mathbf{D}_{\mathbf{X}}\right)_{:j}\right\|_2 + \left\|\left(\mathbf{D}_{\mathbf{Y}} \mathbf{B}\right)_{k:}\right\|_2 = \frac{\left\|\mathbf{A}_{:j}\right\|_2}{\left\|\mathbf{W}_{j:}\right\|_1} + \frac{\left\|\mathbf{B}_{k:}\right\|_2}{\left\|\mathbf{W}_{:k}\right\|_1}.
    $$
\end{lemma}

The utilization of the diagonal matrices $\mathbf{D}_{\mathbf{X}}$ and $\mathbf{D}_{\mathbf{Y}}$ simplifies the sum of the norms to the expressions derived above, offering insights into the influence of the weight matrix $\mathbf{W}$ on the norms of matrix transformations.

\subsection{Enhanced relative importance strategies}\label{sec:strategies}
Beyond \algname{RIA}, we propose several alternative strategies for relative importance that aim to minimize $\vS_{jk}$ in \Cref{eqn0}.

\paragraph{Generalized $\ell_p$-norm.}\label{sec:p_norm}
Expanding beyond the conventional $\ell_1$-norm, we explore the utility of the $\ell_p$-norm in designing score matrices. 
In our approach, mirroring the strategy outlined in Theorem \ref{thm:main2} for reconstructing \algname{RIA} outcomes, we define the score as:

\begin{equation}\label{eqn:pnorm2}
    \mathbf{S}_{jk} = |\mathbf{W}_{jk}| (\|\mathbf{W}_{j:}\|^{-1}_p + \|\mathbf{W}_{:k}\|^{-1}_p).
\end{equation}

Next, we are interested in finding the explicit formulation of $\vX$ and $\vY$ instead of the norm representation when constructing the general $\ell_p$-norm. 


\begin{lemma}[Generalized $\ell_p$-norm]\label{lem:generalized_p_norm}
    
    Let $\mathbf{X}_{: j}={\left\|\mathbf{W}_{j:}\right\|_p^{-1} \cdot\left\|\mathbf{W}_{j:}\right\|_2^{-1}} \cdot \mathbf{W}_{j:}^{\top}$ and $\mathbf{Y}_{k:}={\left\|\mathbf{W}_{:k}\right\|_p^{-1} \cdot\left\|\mathbf{W}_{:k}\right\|_2^{-1}} \cdot \mathbf{W}_{:k}^{\top}$, we recover \Cref{eqn:pnorm2}.
\end{lemma}


Since the equation only requires $\left\|\mathbf{X}_{: j}\right\|_2=\left\|\mathbf{W}_j\right\|_p^{-1}$, \emph{any} vector with this $\ell_2$-norm will satisfy the condition. Inspired by this fact, we can consider the random unit vector scaling in the below lemma. 

\begin{lemma}[Random unit vector scaling]\label{lem:random_unit_vector_scaling}
    Choose any unit vector $\mathbf{u}, \mathbf{v}$ (i.e., $\|\mathbf{u}\|_2=1, \norm{\mathbf{v}}_2=1$) and set $\mathbf{X}_{: j}=\left\|\mathbf{W}_{j:}\right\|_p^{-1} \cdot \mathbf{u}$ and $\mathbf{Y}_{k:}=\left\|\mathbf{W}_{:k}\right\|_p^{-1} \cdot \mathbf{v}$ ensuring \Cref{eqn:pnorm2}.
\end{lemma}


\paragraph{Stochastic relative importance.}\label{sec:StochRIA}
Considering the computational and noise challenges associated with summing all elements across the full rows and columns of large matrices, we introduce a stochastic approach that involves sampling a subset of each row and column. This method assesses the effects of varying subset sizes, denoted by $\tau$, where $\tau < \min(b, c)$, on the overall performance. 

Specifically, we aim to:

a) Evaluate the sensitivity of the final performance to the size of $\tau$ when $\tau$ is reasonably large.

b) Determine if random sampling can enhance the results compared to a deterministic approach.

For this, we define the score matrix for a randomly sampled subset as:
\begin{equation}\label{eqn:stochria}
    \mathbf{S}_{jk} = |\mathbf{W}_{jk}| (\|\mathbf{W}_{j:S_j}\|_1^{-1} + \|\mathbf{W}_{S_k:k}\|_1^{-1}),
\end{equation}

where $S_j$ and $S_k$ represent the sampled indices from the $j$-th row and $k$-th column, respectively, each with a cardinality of $\tau$. This approach builds on the \algname{RIA}-inspired framework, adapting it for practical scenarios involving large-scale data.

For \algname{RIA} in each weight layer, the reweighting sampling complexity is $O(b + c)$. In LLMs, $b$ and $c$ are always very large. Let's say the selection ratio is $\beta$, then for the stochastic relative importance design, the sampling complexity can be reduced to $O(\beta \min (b, c))$, which has been highly reduced. 



\begin{lemma}\label{lem:stochria}
Let $S_j$ and $S_k$ be index sets, and let $\tau > 0$. Define the vectors $\vX_{:j}$ and $\vY_{k:}$ by
$$
\vX_{:j}(i) 
= \frac{\mathbf{1}_{\{i \in S_j\}}}{\|\mathbf{W}_{j:S_j}\|_{1}\sqrt{\tau}},
\quad
\vY_{k:}(i) 
= \frac{\mathbf{1}_{\{i \in S_k\}}}{\|\mathbf{W}_{S_k:k}\|_{1}\sqrt{\tau}}.
$$
Then these vectors satisfy \Cref{eqn:stochria}.
\end{lemma}

\subsection{Training-free fine-tuning}\label{sec:training_free_fine_tuning}
We explore training-free fine-tuning within the context of the pruning-and-growing framework. Specifically, for the pruned weight matrix $\widetilde{\vW}$, we aim to minimize the reconstruction error as defined in (\ref{obj2}). Initially, we identify the growth index, followed by the pruning index, to maintain a consistent sparsity ratio. \algname{DSnoT} \citep{zhang2023dynamic} developed a growing criterion based on the expected change in reconstruction error when reinstating a weight. Particularly, for any given weight row $q\in [1, b]$, the index $i$ is determined as follows:
$$
i = \argmax_r \ \mathrm{sign}(\ec{\epsilon_q}) \cdot \widetilde{\vW}_{q, r} \cdot {\ec{\vX_q}}/{\mathrm{Var}(\vX_q)},
$$

where $\epsilon_q \coloneqq \vW_{q:}\vX - \widetilde{\vW}_{q:}\vX$ denotes the reconstruction error of the $q$-th row across different input activations. It is important to note that for simplicity, output activations are not considered here, which may provide an interesting avenue for future exploration. The functions $\mathrm{sign}(\cdot)$, $\ec{\cdot}$, and $\mathrm{Var}(\cdot)$ denote the standard sign function, expectation, and variance of given inputs over $N \times L$ tokens, respectively. Drawing inspiration from the \algname{Wanda} metric, the \algname{DSnoT} model defines the pruning index $j$ as:
$$
j=\underset{r: \Delta(q, r)<0}{\arg \min } |\widetilde{\vW}_{q, r}|\left\|\vX_q\right\|_2,
$$
where $\Delta(q, r) \eqdef \mathrm{sign}\!\bigl(\ec{\epsilon_q}\bigr)\,\bigl(\widetilde{\vW}_{q, r}\cdot\ec{\vX_q}\bigr)$. 

Several simple yet effective modifications have been incorporated into the pruning-and-growing framework:

\textbf{a) Relative weight importance.} Both in determining the growing index $i$ and the pruning index $j$, we incorporate global information, emphasizing the relative importance of weights in neuron selection.

\textbf{b) Squared activation.} Our extensive experiments demonstrate the widespread benefits of using squared activation, which we utilize in determining the pruning index $j$.
 
\textbf{c) Regularized objective.} The method \algname{MagR} \citep{zhang2024magr} found that adding an $\ell_{\infty}$ norm helps reduce the magnitude of weights during quantization. Here, we adopt a more general regularizer, considering a general $\ell_p$ norm and focusing on specific rows rather than entire layers to reduce communication costs.

Define $\vD_{q, r}\eqdef \|\widetilde{\mathbf{W}}_{q,:}\|_1^{-1}+\|\widetilde{\mathbf{W}}_{:, r}\|_1^{-1}$. The updated rule for identifying the growing index $i$ is formalized as:

\begin{equation}\label{eqn:ft01}
\begin{aligned}
    i &= \argmax_r \left\{ \mathrm{sign}(\mathbb{E}[\epsilon_q]) \cdot \vD_{q, r} \cdot \frac{\mathbb{E}[\mathbf{X}_q]}{\mathrm{Var}(\mathbf{X}_q)} + \gamma_1 \|\widetilde{\mathbf{W}}_q\|_p \right\},
\end{aligned}
\end{equation}

where $\gamma_1$ is the regularization parameter, striking a balance between fidelity and the $\ell_p$ regularizer. Similarly, the pruning index $j$ is now defined as:

\begin{equation}\label{eqn:ft02}
\begin{aligned}
    j&=\underset{r: \Delta(q, r)<0}{\arg \min }\left\{ |\widetilde{\vW}_{q, r}| \cdot \vD_{q,r} \cdot \left\|\vX_q\right\|_2^{\alpha} + \gamma_2 \|\widetilde{\vW}_q\|_p \right\}, 
\end{aligned}
\end{equation}

where
$\Delta(q, r) \eqdef \mathrm{sign}\!\bigl(\ec{\epsilon_q}\bigr)\,\left(\widetilde{\vW}_{q, r} \cdot \vD_{q, r} \cdot \ec{\vX_q}\right). $

This approach allows for effective fine-tuning of the network without the need for retraining, preserving computational resources while optimizing performance.

\section{Experiments}
\paragraph{Setup and configurations.}
We assess the proposed methods across a broad spectrum of popular LLMs, including LlaMA2 (7b-13b) \citep{LlaMA2}, LlaMA3-8b \citep{LlaMA3}, OPT-1.3b \citep{OPT}. We utilize publicly available model checkpoints from the HuggingFace Transformers library \citep{wolf2020transformers} for our evaluations. Each experiment, focused on post-training pruning, is conducted on an NVIDIA A100-80G GPU.
The effectiveness of each pruned model is primarily measured using the perplexity score on the Wikitext-2 dataset \citep{WikiText2}. 
For calibration, we use 128 samples from the C4 dataset \citep{C4}, with each sample comprising 2048 tokens. This approach ensures consistency with the settings used in baseline methods, enabling a fair comparison.

\subsection{Efficiency of stochastic methods}\label{sec:efficiency_stochastic_methods}
We begin by examining two key designs discussed in \Cref{sec:strategies}: the generalized $\ell_p$ norm and stochastic relative importance. The results for the $\ell_p$ norm are presented in \Cref{sec:optimal_p}, where we confirm that $p = 1$ is indeed optimal. We also compare various $\ell_p$ norm reweighting strategies, with the results presented in \Cref{sec:norm_p_reweighting}. Our primary focus, however, is on the findings related to stochastic relative importance, which, to the best of our knowledge, represents the first approach to incorporating stochasticity into LLM post-training pruning.  

We analyze the impact of stochastic relative importance, with the results summarized in \Cref{tab:stoch_res}. The \algname{stochRIA} results correspond to a sampling ratio of $\beta = 0.1$. Each reported value represents the mean performance across five trials with different random seeds. Notably, even with less than only 10\% of the samples used to estimate relative importance, the results remain sufficiently representative, leading to promising outcomes.  

In addition to unstructured pruning with a sparsity ratio of $0.5$, we also explore structured pruning using the N:M pattern \citep{zhou2021learning, zhang2022learning}. The results are presented in \Cref{tab:stoch_res}. Noticed that here for intuitive comparison between \algname{RIA} and \algname{stochRIA}, we use the plain N:M structural pruning without channel permutation. These results consistently demonstrate the benefits and efficiency of our proposed method, \algname{stochRIA}. 

\begin{table}[!tb]
    \centering
    \caption{Comparison of \algname{StochRIA} ($\beta=0.1$) and \algname{RIA} on the Wikitext-2 dataset, using perplexity scores with $\alpha=1$. For \algname{StochRIA}, the mean perplexity over 5 trials is shown in dark, with variance in {\green green}. Improvements and declines relative to \algname{RIA} are indicated in {\color{blue} blue} and {\red red}, respectively.}
    \label{tab:stoch_res}
    \resizebox{0.98\textwidth}{!}{
    \renewcommand{\arraystretch}{1.2}
    \begin{tabular}{l|lcccccc}
    \toprule
    Sparsity & Method & Sampling & LlaMA2-7b & LlaMA2-13b & LlaMA3-8b & OPT-1.3b\\ \midrule
    - & Dense & - & 5.47 & 4.88 & 6.14 & 14.62 \\ \midrule
    \multirow{4}{*}{50\%} & \algname{Magnitude} & - & 16.03 & 6.83 & 205.44 & 1712.39\\
    ~ & \algname{Wanda} & - & 7.79 & 6.28 & 10.81 & 22.19\\ \cmidrule{2-7}
    ~ & \algname{RIA} & Full & 6.88 & 5.95 & 9.44 & 18.94\\
    ~ & \cellcolor{bgcolor4}\algname{stochRIA} & \cellcolor{bgcolor4}$10\%$ & \cellcolor{bgcolor4} $6.91_{{\color{red} -0.03}}^{{\green \pm 0.0032}}$ &  \cellcolor{bgcolor4} $5.95_{{\color{blue} +0}}^{{\green \pm 0.0033}}$ &  \cellcolor{bgcolor4} $9.46_{{\color{red} -0.02}}^{{\green \pm 0.025}}$ &  \cellcolor{bgcolor4} $18.78_{{\color{blue} +0.16}}^{{\green \pm 0.050}}$\\ \midrule
    \multirow{2}{*}{2:4} 
      & \algname{RIA} & Full & 11.31 & 8.40 & 22.89 & 27.43\\
    ~ & \cellcolor{bgcolor4}\algname{stochRIA} & \cellcolor{bgcolor4}$10\%$ & \cellcolor{bgcolor4} $11.41_{{\color{red} -0.10}}^{{\green \pm 0.046}}$ & \cellcolor{bgcolor4} $8.44_{{\color{red} -0.04}}^{{\green \pm 0.016}}$ & \cellcolor{bgcolor4} $23.74_{{\color{blue} +0.15}}^{{\green \pm 0.230}}$ & \cellcolor{bgcolor4} $26.78_{{\color{blue} +0.65}}^{{\green \pm 0.127}}$\\ \midrule
    \multirow{2}{*}{4:8} 
    ~ & \algname{RIA} & Full & 8.39 & 6.74 & 13.77 & 21.59\\
    ~ & \cellcolor{bgcolor4}\algname{stochRIA} & \cellcolor{bgcolor4}$10\%$ & \cellcolor{bgcolor4} $8.44_{{\color{red} -0.05}}^{{\green \pm 0.014}}$ & \cellcolor{bgcolor4} $6.74_{{\color{blue} +0}}^{{\green \pm 0.013}}$ & \cellcolor{bgcolor4} $13.93_{{\color{red} -0.16}}^{{\green \pm 0.095}}$ & \cellcolor{bgcolor4} $21.49_{{\color{blue} +0.10}}^{{\green \pm 0.089}}$\\ \bottomrule
    \end{tabular}}
\end{table}

\begin{table*}[!tb]
\centering
\caption{Perplexity scores on Wikitext-2, accounting for various norm $\alpha$ values and column \& row sensitivity, with a sparsity ratio $50\%$.}\label{tab:norm_alpha_col_row}
\resizebox{1.0\textwidth}{!}{
\begin{tabular}{l|cccc|cccc|cccc|cccc}
\toprule
{Model} & \multicolumn{4}{c|}{{LlaMA2-7b}} & \multicolumn{4}{c|}{{LlaMA2-13b}} & \multicolumn{4}{c|}{{LlaMA3-8b}} & \multicolumn{4}{c}{{OPT-1.3b}} \\ \midrule
\textbf{$\alpha$}        & {0} & {0.5} & {1} & {2} & {0} & {0.5} & {1} & {2} & {0} & {0.5} & {1} & {2} & {0} & {0.5} & {1} & {2} \\
\midrule
Dense       & \multicolumn{4}{c}{ 5.47} & \multicolumn{4}{c}{ 4.88} & \multicolumn{4}{c}{ 6.14} & \multicolumn{4}{c}{ 14.62}\\ \midrule
\algname{Wanda}       & 16.03 & 7.60 & 7.79 & 8.66 & 6.83 & 6.17 & 6.28 & 7.15 & 205.44 & 10.66 & 10.81 & 12.98 & 1712.39 & 22.14 & 22.19 & 24.74 \\
\algname{Col-Sum}    & 11.59 & \color{blue} 6.83 & 6.91 & 7.46 & 6.39 & \bf \color{blue} 5.87 & 5.96 & 6.55 & 59.41 & 9.53 & 9.69 & 12.01 & 1062.66 & \color{blue} 18.28 & 18.41 & 22.25 \\
\algname{Row-Sum}    & 14.93 & 7.49 & 7.51 & 8.01 & 6.74 & 6.13 & 6.24 & 7.01 & 17.80 & 10.50 & 10.55 & 11.79 & 141.92 & 22.09 & 22.47 & 26.62 \\
\algname{RIA}         & 7.39 & \bf \color{blue} 6.81 & 6.88 & 7.37 & 5.95 & \color{blue} 5.93 & 5.95 & 6.56 & 12.07 & \bf \color{blue} 9.34 & \color{blue} 9.44 & 10.67 & 64.70 & \bf \color{blue} 18.08 & 18.94 & 23.39 \\
\bottomrule
\end{tabular}}
\end{table*}

Furthermore, when aggregating results across all examined models and baselines, \algname{stochRIA} achieves an accumulated perplexity that is 0.66 lower than \algname{RIA}, demonstrating the effectiveness of a stochastic design. 
This stochastic sampling preserves the diversity needed to handle subpopulations that rely on lower-average-importance weights while also helping preserve generalization by avoiding the dilution of salient features.

We also evaluate the performance across different sampling ratios, as shown in \Cref{sec:sampling_ratios}. Our main takeaway is that \algname{stochRIA} exhibits stable and competitive performance relative to \algname{RIA}, particularly when the sampling ratio $\tau \geq 0.05$. At or above this threshold, the performance remains robust and occasionally surpasses less noisy sampling configurations. However, at an extremely low sampling ratio of $\tau = 0.01$, a significant performance drop is observed. Consequently, we adopt $\tau = 0.1$ as the default setting for our experiments.


\subsection{Insights on sensitivity, activation, and sparsity}
\paragraph{Column and row sensitivity.}
Compared with the \algname{Wanda} design, \algname{RIA} accounts for the relative importance of both rows and columns. However, it remains unclear whether columns and rows contribute equally to \algname{RIA}'s performance improvements. To investigate this, we conducted an extensive analysis of the significance of column-wise and row-wise relative importance, with the results shown in \Cref{tab:norm_alpha_col_row}. A key finding is that the sum of the columns has more impact on performance, indicating greater importance.

To provide further insights, we visualized the heatmap of a randomly selected dense weight matrix from LLaMA2-7b, as illustrated in \Cref{fig:vis}. The heatmap displays stripe-like patterns, indicating column-specific structures where certain columns show significantly higher activations, forming distinct stripes. This observation suggests that normalizing by rows effectively balances these disparities. In cases where the rows within a specific column already exhibit relatively uniform distributions, normalization over rows may not be necessary. Thus, column normalization alone might suffice to balance the contributions of output neurons, especially when some columns dominate due to large absolute values.

\begin{figure}[!tb]
    \centering
        \includegraphics[width=0.95\linewidth, trim = 88 40 100 68, clip]{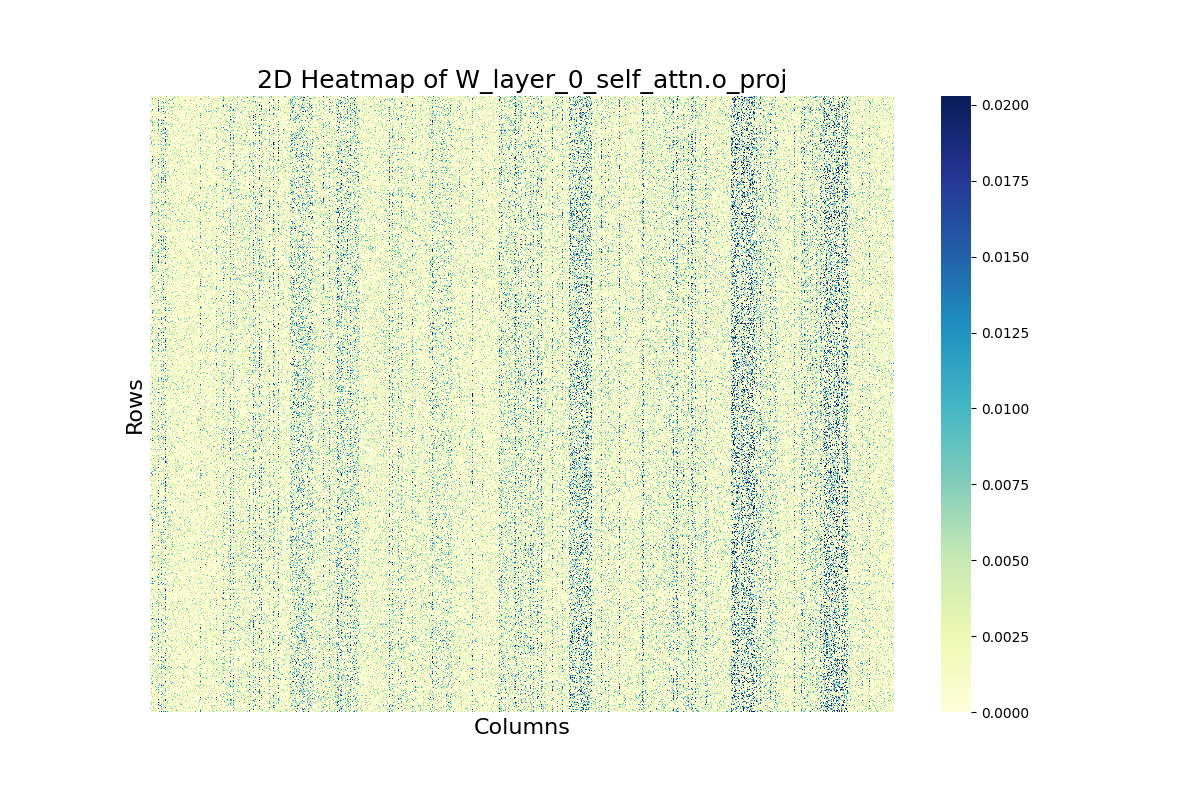}
    \caption{Visualization of the dense weight matrix in LLaMA2-7b.}
    \label{fig:vis}
\end{figure}

\paragraph{Benefits of squared input activation.} 
In the design of \algname{Wanda} \citep{Wanda}, the power factor $\alpha$ applied to input activations is set to 1, whereas in \algname{RIA} \citep{RIA}, $\alpha$ is adjusted to $0.5$.  In this study, we systematically explore the impact of varying the power factor on input activations, with detailed results presented in \Cref{tab:norm_alpha_col_row}. An $\alpha$ value of 0 implies that no activation is considered in generating the pruning matrix. Our findings consistently show that incorporating input activation improves performance in terms of perplexity. Notably, $\alpha = 0.5$ proved optimal across various methods, underscoring the advantages of reducing the magnitude of input activations. We attribute this improvement to the mitigation of outliers in the input activations, where smoothing these values provides more meaningful guidance for pruning.

\begin{table}[!tb]
    \centering
    \caption{Perplexity on Wikitext-2 with different sparsity. $\alpha=1.0$.}
    \label{tab:various_sparsity}
    \resizebox{0.8\textwidth}{!}{
    \begin{tabular}{l|ccccccc}
    \toprule
    Sparsity & Method & Sampling & L2-7b & L2-13b & L3-8b & OPT-1.3b\\ \midrule
    Dense  & - & - & 5.47 & 4.88 & 6.14 & 14.62 \\ \hline
    \multirow{3}{*}{$50\%$} & \algname{Wanda} & - & 7.79 & 6.28 & 10.81 & 22.19\\
    ~ & \algname{RIA} & Full & \color{blue} \bf 6.88 & \color{blue} \bf 5.95 & \color{blue} \bf 9.44 & 18.94\\
    ~ & \cellcolor{bgcolor4}\algname{stochRIA} & \cellcolor{bgcolor4}$10\%$ & \cellcolor{bgcolor4}6.91 & \cellcolor{bgcolor4}\color{blue} \bf 5.95 & \cellcolor{bgcolor4}9.46 & \cellcolor{bgcolor4}\color{blue} \bf 18.78\\ \midrule
    \multirow{3}{*}{$60\%$} & \algname{Wanda} & -  & 15.30 & 9.63 & 27.55 & 38.81\\
    ~ & \algname{RIA} & Full & \color{blue} \bf 10.39 & \color{blue} \bf 7.84 & 19.52 & 26.22\\
    ~ & \cellcolor{bgcolor4}\algname{stochRIA} & \cellcolor{bgcolor4}$10\%$ & \cellcolor{bgcolor4}10.62 & \cellcolor{bgcolor4}7.97 & \cellcolor{bgcolor4}\color{blue} \bf 19.04 & \cellcolor{bgcolor4}\color{blue} \bf 25.93\\\midrule
    \multirow{3}{*}{$70\%$} & \algname{Wanda} & -  & 214.93 & 104.97 & 412.90 & 231.15\\ 
    ~ & \algname{RIA} & Full & \color{blue} \bf 68.75 & \color{blue} \bf 51.96 & 169.51 & 98.52\\ 
    ~ & \cellcolor{bgcolor4}\algname{stochRIA} & \cellcolor{bgcolor4}$10\%$ & \cellcolor{bgcolor4}72.85 & \cellcolor{bgcolor4}62.15 & \cellcolor{bgcolor4}\color{blue} \bf 155.34 & \cellcolor{bgcolor4}\color{blue} \bf 93.29\\ \bottomrule
    \end{tabular}}
\end{table}

\paragraph{Various unstructured sparsity ratios.} 
We established a default unstructured sparsity ratio of 50\%. In this section, we investigate the impact of varying sparsity ratios, as detailed in \Cref{tab:various_sparsity}. For \algname{stochRIA}, we report the mean average perplexity after three trials. Given that \algname{stochRIA} has been shown to be stable, with variance examined in \Cref{tab:comparison}, we omit the variance to focus on performance. Our findings reveal that \algname{Wanda} is particularly sensitive to higher sparsity ratios, whereas both \algname{RIA} and our proposed \algname{stochRIA} demonstrate robustness to increased sparsity, maintaining stable performance across a broader range of conditions.
Interestingly, we observed that on LLaMA3-8b and OPT1.3b, \algname{stochRIA} consistently outperforms \algname{RIA}, whereas on LLaMA2-7b and LLaMA2-13b, the reverse is true. This intriguing phenomenon may be attributed to the heavy noise present in the sampling process for LLaMA3-8b and OPT1.3b. In such cases, selecting a subset of weights through \algname{stochRIA} may yield more reliable relative weight information, resulting in improved performance.

\subsection{Training-free fine-tuning comparisons} 

The intrinsic gap between pruned weights and the original, unpruned pretrained weights underscores the importance of minimizing reconstruction loss to achieve promising results. We introduced \ft, which incorporates relative weight reweighting and a regularized decision boundary during the dynamic sparse refinement step, all without additional training. Perplexity scores, as shown in \Cref{tab:ft_all}, reveal that our \ft approach consistently surpasses baseline methods and the previous state-of-the-art \algname{DSnoT} without fine-tuning. For instance, \algname{Magnitude} exhibited subpar perplexity scores on LlaMA2-7b and LlaMA3-8b; however, our \ft achieved perplexity reductions of 96.5\% and 96.4\%, respectively. These results not only validate \ft's efficacy but also offer guidance for scenarios involving high sparsity or underperforming pruned models, with minimal effort and no additional training.

\paragraph{Zero-shot performance.} To provide a comprehensive evaluation, we also conducted zero-shot classification tests using seven well-regarded datasets. These tests assess the pruned models' ability to accurately categorize objects or data points into previously unseen categories. We employed the methodology described by \cite{Wanda} and utilized tasks from the EleutherAI LM Harness \citep{gao2021framework}, including BoolQ \citep{clark2019boolq}, RTE \citep{wang2018glue}, HellaSwag \citep{zellers2019hellaswag}, WinoGrande \citep{sakaguchi2021winogrande}, ARC (Easy and Challenge) \citep{clark2018think}, and OpenbookQA \citep{mihaylov2018can}. The results, presented in \Cref{tab:zsl_per_task_with_wandag_noDenseCompare}, show that \ft consistently outperforms \algname{DSnoT} in zero-shot tasks, confirming its effectiveness. 
To the best of our knowledge, \ft establishes a new state-of-the-art for training-free pruning and fine-tuning methods in zero-shot performance.

\begin{table}[!tb]
    \centering
    \caption{Perplexity scores on Wikitext-2 after training-free fine-tuning. The sparsity ratio is set to $60\%$ and $\alpha = 0.5$.}  
    \label{tab:ft_all}
    \begin{tabular}{lcccccc}
    \toprule
    Base & FT & LlaMA2-7b & LlaMA2-13b & LlaMA3-8b\\ \midrule
    \algname{Dense} & - & 5.47 & 4.88 & 6.14\\ \hline
    \algname{Magnitude} & - & 6.9e3 & 10.10 & 4.05e5\\
    \algname{Magnitude} & \algname{DSnoT} & 4.1e3 & 10.19 & 4.18e4\\
    \rowcolor{bgcolor4}
    \algname{Magnitude} & \ft & \color{blue} \bf2.4e2 & \color{blue} \bf10.09 & \color{blue} \bf1.44e4\\ \midrule
    \algname{Wanda} & - & \color{blue} \bf 9.72 & 7.75 & 21.36\\
    \algname{Wanda} & \algname{DSnoT} & 10.23 & \color{blue} \bf 7.69 & 20.70\\
    \rowcolor{bgcolor4}
    \algname{Wanda} & \ft & 10.08 & \color{blue} \bf7.69 & \color{blue} \bf20.50\\ \midrule
    \algname{RIA} & - & 10.29 & 7.85 & 21.09\\
    \algname{RIA} & \algname{DSnoT} & 9.97 & 7.82 & 19.51\\
    \rowcolor{bgcolor4}
    \algname{RIA} & \ft & \color{blue} \bf9.96 & \color{blue} \bf7.78 & \color{blue} \bf18.99\\ \bottomrule
    \end{tabular}
\end{table}

\begin{table*}[!tb]
\centering
\caption{Accuracies (\%) for LLaMA2 models on 7 zero-shot tasks at 60\% unstructured sparsity.}
\label{tab:zsl_per_task_with_wandag_noDenseCompare}
\resizebox{1.0\textwidth}{!}{
\begin{tabular}{llcccccccc}
\toprule
Params & Method & BoolQ & RTE & HellaSwag & WinoGrande & ARC-e & ARC-c & OBQA & \cellcolor{bgcolor} Mean \\
\midrule
\multirow{7}{*}{LlaMA2-{7b}}
 & Dense      
   & 77.7 & 62.8 & 57.2 & 69.2 & 76.4 & 43.4 & 31.4 & \cellcolor{bgcolor} 57.9 \\ \cmidrule{2-10}
 & \algname{Magnitude} & 41.2 & 51.3 & 37.0 & 55.7 & 50.0 & 27.0 & 16.2 & \cellcolor{bgcolor} 39.3\\
 & \algname{w. DSnoT} & 43.2 & \color{blue}54.2 & 38.4 & 56.4 & 53.3 & 27.7 & 20.6 & \cellcolor{bgcolor}  41.1 \\
 &\cellcolor{bgcolor4}\bf \algname{w. $R^2$-DSnoT} & \cellcolor{bgcolor4}\color{blue}50.9 & \cellcolor{bgcolor4} 52.0 & \cellcolor{bgcolor4}\color{blue}39.8& \cellcolor{bgcolor4}\color{blue}56.8 & \cellcolor{bgcolor4}\color{blue}\color{blue}56.6 & \cellcolor{bgcolor4}\color{blue}28.3 & \cellcolor{bgcolor4}\color{blue}23.4 & \cellcolor{bgcolor}\color{blue} \bf43.4\\ \cmidrule{2-10}
 & \algname{RIA} & \color{blue} 66.1 & 53.1 & 43.5 & 63.2 & 64.6 & 30.2 & 26.0 & \cellcolor{bgcolor} 49.5 \\
  & \algname{w. DSnoT} & 65.5 & 53.4 & \color{blue} 44.7 & 64.6 & \color{blue} 65.3 & \color{blue} 31.7 & 26.4 & \cellcolor{bgcolor} 50.2 \\
 &\cellcolor{bgcolor4}\bf \algname{w. $R^2$-DSnoT} & \cellcolor{bgcolor4}65.2 & \cellcolor{bgcolor4}\color{blue} 53.8 & \cellcolor{bgcolor4}\color{blue} 44.7 & \cellcolor{bgcolor4}\color{blue} 65.1 & \cellcolor{bgcolor4}65.0 & \cellcolor{bgcolor4}31.6 & \cellcolor{bgcolor4}\color{blue} 27.0 & \cellcolor{bgcolor4}\cellcolor{bgcolor} \color{blue} \bf 50.3 \\ 
\hline
 \multirow{7}{*}{LlaMA3-{8b}} 
 & Dense      
   & 81.3 & 69.7 & 60.1 & 73.0 & 80.1 & 50.4 & 34.8 & \cellcolor{bgcolor} 64.2 \\ \cmidrule{2-10}
 & \algname{Magnitude} & 37.8 & 52.7 & 30.7 & 51.0 & 39.7 & 23.4 & 14.4 & \cellcolor{bgcolor} 35.7 \\
 & \algname{w. DSnoT} & 37.8 & 52.7 & \color{blue}33.4 & 49.9 & 43.5 & 23.0 & \color{blue}14.8 & \cellcolor{bgcolor} 36.4 \\
 &\cellcolor{bgcolor4}\bf \algname{w. $R^2$-DSnoT} & \cellcolor{bgcolor4}37.8 & \cellcolor{bgcolor4}52.7 & \cellcolor{bgcolor4}33.1 & \cellcolor{bgcolor4}\color{blue}52.1 & \cellcolor{bgcolor4}\color{blue}43.9 & \cellcolor{bgcolor4}\color{blue}23.6 & \cellcolor{bgcolor4}\color{blue}14.8 & \cellcolor{bgcolor} \color{blue} \bf 37.1\\ \cmidrule{2-10}
 & \algname{RIA} & 70.2 & 53.4 & 39.7 & 61.7 & 61.1 & \color{blue} 28.6 & 20.4 & \cellcolor{bgcolor} 47.9 \\
  & \algname{w. DSnoT} & \color{blue} 70.7 & 53.4 & \color{blue} 40.3 & 61.3 & \color{blue} 61.7 & 28.0 & 20.0 & \cellcolor{bgcolor} 47.9 \\
 &\cellcolor{bgcolor4}\bf \algname{w. $R^2$-DSnoT} & \cellcolor{bgcolor4}70.4 & \cellcolor{bgcolor4}53.4 & \cellcolor{bgcolor4}\color{blue} \cellcolor{bgcolor4}40.3 & \cellcolor{bgcolor4}\color{blue} 61.9 & \cellcolor{bgcolor4}61.2 & \cellcolor{bgcolor4}28.3 & \cellcolor{bgcolor4}\color{blue} 21.0 & \cellcolor{bgcolor4} \cellcolor{bgcolor}\color{blue} \bf 48.1 \\
\bottomrule
\end{tabular}
} 
\end{table*}


\section{Discussion and Future Work}
\textbf{Beyond pruning.} We initiated our exploration by assessing the efficacy of \algname{Wanda} and \algname{RIA}, introducing the symmetric objective in (\ref{obj2}). Although initially aimed at post-training pruning for LLMs, our approach can extend to post-training quantization and training-aware compression \citep{GPTQ, AQLM, PV-Tuning}, promising areas for future exploration.

\textbf{Better sampling.} In \Cref{sec:efficiency_stochastic_methods}, we showed that selective sampling of matrix rows and columns enhances performance and efficiency over full sampling. This improvement is credited to stochastic sampling maintaining diversity in lower-importance weights and preventing loss of key features. Future research could investigate asymmetric or non-uniform sampling within the (\ref{obj2}) framework to further optimize performance.


\textbf{Exploring symmetric designs.} 
\Cref{tab:comparison} introduces general and diagonal-specific symmetric designs for LLM compression. These initial findings underscore the potential benefits of further exploring symmetric designs in weights and activations to enhance LLM compression techniques. Extending these approaches into distributed and federated settings \citep{kai2023fedp3, ye2024fedllm} could also prove promising.

%% file: Appendix_C3_EFBV.tex
\chapter{Appendix to Chapter \ref{chapter_ef_bv}}
\label{chapter_appendix_ef_bv}
\thispagestyle{empty}

\section{New compressors}\label{secappa}

We propose new compressors in our class $\mathbb{C}(\eta,\omega)$.

\subsection{\texttt{mix-}(k,k'): Mixture of \texttt{top-}k and \texttt{rand-}k}

Let $k\in \mathcal{I}_d$ and $k'\in \mathcal{I}_d$, with $k+k'\leq d$. We propose the compressor \texttt{mix-}$(k,k')$. It maps $x\in\mathbb{R}^d$ to $x'\in\mathbb{R}^d$, defined as follows. 
Let $i_1,\ldots,i_k$ be distinct indexes in $\mathcal{I}_d$ such that $|x_{i_1}|,\ldots,|x_{i_k}|$ are the $k$ largest elements of $|x|$ (if this selection is not unique, we can choose any one). These coordinates are kept: $x'_{i_j}=x_{i_j}$, $j=1,\ldots,k$. In addition, $k'$ other coordinates chosen at random in the remaining ones are kept: $x'_{i_j}=x_{i_j}$, $j=k+1,\ldots,k+k'$, where $\{i_j : j=k+1,\ldots,k+k'\}$ is a subset of size $k'$ of $\mathcal{I}_d \backslash \{i_1,\ldots,i_k\}$ chosen uniformly at random. The other coordinates of $x'$ are set to zero.

\begin{proposition}
\label{prop1}\emph{\texttt{mix-}}$(k,k')\in \mathbb{C}(\eta,\omega)$ with $\eta =\frac{d-k-k'}{\sqrt{(d-k)d}}$ and $\omega=\frac{k'(d-k-k')}{(d-k)d}$.
\end{proposition}
As a consequence, \texttt{mix-}$(k,k')\in \mathbb{B}(\alpha)$ with $\alpha=1-\eta^2-\omega = 1-\frac{(d-k-k')^2}{(d-k)d}-\frac{k'(d-k-k')}{(d-k)d}=
\frac{k+k'}{d}$. This is the same $\alpha$ as for \texttt{top-}$(k+k')$ and scaled \texttt{rand-}$(k+k')$.

The proof is given in Appendix~\ref{secproofp4}.

\subsection{\texttt{comp-}(k,k'): Composition of \texttt{top-}k and \texttt{rand-}k}

Let $k\in \mathcal{I}_d$ and $k'\in \mathcal{I}_d$, with $k\leq k'$. We consider the compressor \texttt{comp-}$(k,k')$, proposed in \citet{bar20}, 
which is the composition of \texttt{top-}$k'$ and \texttt{rand-}$k$:
\texttt{top-}$k'$  is applied first, then \texttt{rand-}$k$ is applied to the $k'$ selected (largest) elements. 
That is,  \texttt{comp-}$(k,k')$ maps $x\in\mathbb{R}^d$ to $x'\in\mathbb{R}^d$, defined as follows. Let $i_1,\ldots,i_{k'}$ be distinct indexes in $\mathcal{I}_d$ such that $|x_{i_1}|,\ldots,|x_{i_{k'}}|$ are the $k'$ largest elements of $|x|$ (if this selection is not unique, we can choose any one). Then 
$x'_{i_j}=\frac{k'}{k} x_{i_j}$, $j=1,\ldots,k$, where $\{i_j : j=1,\ldots,k\}$ is a subset of size $k$ of $\{i_1,\ldots,i_{k'}\}$ chosen uniformly at random. The other coordinates of $x'$ are set to zero.

 \texttt{comp-}$(k,k')$ sends $k$ coordinates of its input vector, like \texttt{top-}$k$ and \texttt{rand-}$k$, whatever $k'$. We can note that  \texttt{comp-}$(k,d)={}$\texttt{rand-}$k$ and \texttt{comp-}$(k,k)={}$\texttt{top-}$k$. We have:

\begin{proposition}
\label{prop2}
 \emph{\texttt{comp-}}$(k,k')\in \mathbb{C}(\eta,\omega)$ with $\eta =\sqrt{\frac{d-k'}{d}}$ and $\omega=\frac{k'-k}{k}$.
\end{proposition}

The proof is given in Appendix~\ref{secproofp5}.

\section{New results on \algname{DIANA}}\label{secappb}

We suppose that the compressors $\mathcal{C}_i^t$ are in $\mathbb{C}(\eta,\omega)$, for some $\eta\in[0,1)$ and $\omega\geq 0$. Viewing  \algname{DIANA} as \algname{EF-BV} with $\nu=1$, we define $r$, $s^\star$, $\theta^\star$ as before, as well as 
$r_{\mathrm{av}} \eqdef \eta^2+\oma$. 
We obtain, as corollaries of Theorems \ref{theo1} and \ref{theo2}:

\begin{theorem}\label{coro1}Suppose that $R=0$ and $f$ satisfies the P{\L}  condition with some constant  $\mu>0$. 
In  \algname{DIANA}, suppose that $\lambda \in (0,1]$ is such that $r<1$, and
\begin{equation*}
0<\gamma \leq \frac{1}{L+\tilde{L}\sqrt{\frac{r_{\mathrm{av}}}{r}}\frac{1}{s^\star}}.
\end{equation*}
For every $t\geq 0$, define the Lyapunov function
\begin{equation*}
\Psi^t \eqdef f(x^t)-f^\star + \frac{\gamma}{2\theta^\star}  \frac{1}{n}\sum_{i=1}^n \sqnorm{\nabla f_i(x^t)-h_i^{t}}, 
\end{equation*}
where $f^\star \eqdef f(x^\star)$, for any minimizer $x^\star$ of $f$. 
Then, for every $t\geq 0$,
\begin{align*}
\Exp{\Psi^{t}} 
&\leq \left(\max\left(1-\gamma\mu, {\frac{r+1}{2}}\right) \right)^t\Psi^0.
\end{align*}
\end{theorem}

\begin{theorem}\label{coro2}
Suppose that $f+R$ satisfies the  the K{\L}  condition with some constant $\mu>0$. 
In  \algname{DIANA}, suppose that $\lambda \in (0,1]$ is such that $r<1$, and
\begin{equation*}
0<\gamma \leq \frac{1}{2L+\tilde{L}\sqrt{\frac{r_{\mathrm{av}}}{r}}\frac{1}{s^\star}}.
\end{equation*}
$\forall t\geq 0$, define the Lyapunov function
\begin{align*}
\Psi^t \eqdef f(x^t)+R(x^t)-f^\star - R^\star  + \frac{\gamma}{2\theta^\star}  \frac{1}{n}\sum_{i=1}^n \sqnorm{\nabla f_i(x^t)-h_i^{t}},
\end{align*}
where $f^\star \eqdef f(x^\star)$ and $R^\star \eqdef R(x^\star)$, for any minimizer $x^\star$ of $f+R$. 
Then, for every $t\geq 0$,
\begin{align*}
\Exp{\Psi^{t}}  &\leq \left(\max\left({\frac{1}{1+\frac{1}{2}\gamma\mu}},\frac{r+1}{2}\right)\right)^t\Psi^0.
\end{align*}
\end{theorem}\smallskip

Interestingly,  \algname{DIANA}, used beyond its initial setting with compressors in $\mathbb{B}(\alpha)$ with $\lambda=1$, just reverts to (the original) 
 \algname{EF21}, as shown in Fig.~\ref{fig1}. This shows how our unified framework reveals connections between these two algorithms and  unleashes their potential.

\section{Experiments}\label{appexp}
\subsection{Datasets and experimental setup}
We consider the heterogeneous data distributed regime, which means that all parallel nodes store different data points, but use the same type of learning function. We adopt the datasets from LibSVM~\citep{chang2011libsvm} and  we split them, after random shuffling, into 
$n\leq N$ 
blocks, where $N$ is the total number of data points (the left-out data points from the integer division of $N$ by $n$ are stored at the last node). The corresponding values are shown in Tab.~\ref{tab:logistic_datasets}. 
To make our  setting more realistic, we consider that different nodes partially share some data:
we set the overlapping factor to be $\xi\in\{1, 2\}$, where $\xi=1$ means no overlap and $\xi=2$ means that the data is partially shared among the nodes, with a redundancy factor of 2; this is achieved by  sequentially assigning 2 blocks of data to every node.
The experiments were conducted using 24 NVIDIA-A100-80G GPUs, each with  
80GB memory. 

\begin{table}[b]
 \caption{Values of $d$ and $N$ for the considered datasets.}
    \label{tab:logistic_datasets}
    \centering
    \begin{tabular}{c|c|c}
        \toprule
        \multirow{2}{*}{Dataset} & \multirow{2}{*}{$N$ (total \# of datapoints)} & \multirow{2}{*}{$d$ (\# of features)} \\
        ~ & ~ &~ 
        \\\hline
        \multirow{1}{*}{\texttt{mushrooms}} & \multirow{1}{*}{8,124} & \multirow{1}{*}{112} 
        \\ \hline
        \multirow{1}{*}{\texttt{phishing}} & \multirow{1}{*}{11,055} & \multirow{1}{*}{68} 
        \\ \hline
        \multirow{1}{*}{\texttt{a9a}} & \multirow{1}{*}{32,561} & \multirow{1}{*}{123} 
        \\ \hline
        \multirow{1}{*}{\texttt{w8a}} & \multirow{1}{*}{49,749} & \multirow{1}{*}{300} 
        \\  
        \bottomrule
    \end{tabular}
   \end{table}

We consider logistic regression, which consists in minimizing the $\mu$-strongly convex function 
\begin{equation*}
f=\frac{1}{n} \sum_{i=1}^{n}f_i,
\end{equation*}
with, for every $i\in\mathcal{I}_n$,
\begin{equation*}
        f_i(x) =\frac{1}{N_i} \sum_{j=1}^{N_i} \log\!\Big(1+\exp\!\left(-b_{i,j} x^{\top} a_{i,j}\right)\Big)+ \frac{\mu}{2} \|x\|^2, 
    \end{equation*}
where $\mu$, set to $0.1$, 
is the strong convexity constant; $N_i$ is the number of data points at node $i$; the $a_{i,j}$ are the training vectors and the $b_{i,j} \in\{-1,1\}$ the corresponding labels. 
Note that there is no regularizer in this problem; that is, $R=0$.

We set $L=\tilde{L}=\sqrt{\sum_{i=1}^n L_i^2}$, with $L_i =\mu +  \frac{1}{4N_i} \sum_{j=1}^{N_i} \|a_{i,j}\|^2$. We use independent compressors of type \texttt{comp-}$(k,k')$ at every node, for some small $k$ and large $k'<d$. These compressors are biased ($\eta>0$) and have a variance $\omega>1$, so they are not contractive: they don't belong to $\mathbb{B}(\alpha)$ for any $\alpha$. We have $\oma=\frac{\omega}{n}$.
Thus, we place ourselves in the conditions of Theorem~\ref{theo1}, and we compare \algname{EF-BV} with 
\begin{equation*}
\lambda=\lambda^\star,\quad \nu=\nu^\star,\quad \gamma=\frac{1}{L+\tilde{L}\sqrt{\frac{r_{\mathrm{av}}}{r}}\frac{1}{s^\star}}
\end{equation*}
to \algname{EF21}, which corresponds to the particular case of \algname{EF-BV} with
\begin{equation*}
\nu = \lambda=\lambda^\star,\quad \gamma=\frac{1}{L+\tilde{L}\frac{1}{s^\star}}.
\end{equation*}

\begin{table}[t]
	\caption{Parameter values of \algname{EF-BV} and \algname{EF21} in the different settings.  $k'$ in \texttt{comp-}$(k, k')$ is set to $d/2$ and $n=1000$. In pairs of values like (1,2), the first value is $k$ and the second value is $\xi$.}\label{tab10}
	\centering
	\resizebox{1.0\textwidth}{!}{
	\begin{tabular}{c|c|ccc|ccc|ccc|ccc}
		\toprule
		\multirow{2}{*}{Method} & \multirow{2}{*}{Params} & \multicolumn{3}{c}{mushrooms} & \multicolumn{3}{c}{phishing} & \multicolumn{3}{c}{a9a} & \multicolumn{3}{c}{w8a}\\ \cmidrule{3-5} \cmidrule{6-8} \cmidrule{9-11} \cmidrule{12-14}
		~ & ~ & (1,1) & (1,2) & (2,1) & (1,1) & (1,2) & (2,1) & (1,1) & (1,2) & (2,1) & (1,1) & (1,2) & (2,1)\\ \hline
		 & \multirow{2}{*}{$\eta$} & \multirow{2}{*}{0.707} & \multirow{2}{*}{0.707} & \multirow{2}{*}{0.707} & \multirow{2}{*}{0.707} & \multirow{2}{*}{0.707} & \multirow{2}{*}{0.707} & \multirow{2}{*}{0.710} & \multirow{2}{*}{0.710} & \multirow{2}{*}{0.710}  & \multirow{2}{*}{0.707} & \multirow{2}{*}{0.707} & \multirow{2}{*}{0.707} \\
		 & ~ & ~ & ~\\ \hline

		& \multirow{2}{*}{$\omega$} & \multirow{2}{*}{55} & \multirow{2}{*}{55} & \multirow{2}{*}{27} & \multirow{2}{*}{33} & \multirow{2}{*}{33} & \multirow{2}{*}{16} & \multirow{2}{*}{60} & \multirow{2}{*}{60} & \multirow{2}{*}{29.5} & \multirow{2}{*}{149} & \multirow{2}{*}{149} & \multirow{2}{*}{74}\\
		 & ~ & ~ & ~\\ \hline  

		 & \multirow{2}{*}{$\omega_{\mathrm{av}}$} & \multirow{2}{*}{0.055} & \multirow{2}{*}{0.055} & \multirow{2}{*}{0.027} & \multirow{2}{*}{0.033} & \multirow{2}{*}{0.033} & \multirow{2}{*}{0.016}  & \multirow{2}{*}{0.06} & \multirow{2}{*}{0.06} & \multirow{2}{*}{0.295} & \multirow{2}{*}{0.149} & \multirow{2}{*}{0.149} & \multirow{2}{*}{0.074}\\
		 & ~ & ~ & ~\\ \hline 

		EF-BV & \multirow{2}{*}{$\lambda$} & 5.32e-3 & 5.32e-3 & 1.08e-2 & 8.85e-3 & 8.85e-3 & 1.82e-2 & 4.83e-3 & 4.83e-3 & 9.8e-3 & 1.96e-3 & 1.96e-3 & 3.95e-3\\
		EF21 & ~ & 5.32e-3 & 5.32e-4 & 1.08e-2 & 8.85e-3 & 8.85e-3 & 1.82e-2 & 4.83e-3 & 4.83e-3 & 9.8e-3 & 1.96e-3 & 1.96e-3 & 3.95e-3\\ \hline

		EF-BV & \multirow{2}{*}{$\nu$} & 1 & 1 & 1 & 1 & 1 & 1 & 1 & 1 & 1 & 1 & 1 & 1 \\
		EF21 & ~ & 5.32e-3 & 5.32e-4 & 1.08e-2 & 8.85e-3 & 8.85e-3 & 1.82e-2 & 4.83e-3 & 4.83e-3 & 9.8e-3 & 1.96e-3 & 1.96e-3 & 3.95e-3\\ \hline

		EF-BV & \multirow{2}{*}{$r$} & 0.998 & 0.998 & 0.997& 0.997 & 0.997 & 0.994 & 0.999 & 0.999 & 0.997 & 0.999 & 0.999 & 0.999\\
		EF21 & ~ & 0.998 & 0.998 & 0.997 & 0.997 & 0.997 & 0.994 &0.999 & 0.999 & 0.997 & 0.999 & 0.999 & 0.999\\ \hline

		EF-BV & \multirow{2}{*}{$r_{\mathrm{av}}$} & 0.555 & 0.555 & 0.527 & 0.533 & 0.533 & 0.516 & 0.564 & 0.564 & 0.534 & 0.649 & 0.649 & 0.574\\
		EF21 & ~ & 0.998 & 0.998 & 0.997 & 0.997 & 0.997 & 0.994 & 0.999 & 0.999 & 0.997 & 0.999 & 0.999 & 0.999\\ \hline

		EF-BV & \multirow{2}{*}{$\sqrt{\frac{r_{\mathrm{av}}}{r}}$} & 0.746 & 0.746 & 0.727 & 0.731 & 0.731 & 0.720 & 0.752 & 0.752 & 0.731 & 0.806 & 0.806 & 0.758\\
		EF21 & ~ & 1 & 1 & 1 & 1 & 1 & 1 & 1 & 1 & 1 & 1 & 1 & 1\\ \hline  

		EF-BV & \multirow{2}{*}{$s^\star$} & 3.90e-4 & 3.90e-4 & 7.94e-4 & 6.50e-4 & 6.50e-4 & 1.34e-3 & 3.5e-4 & 3.5e-4 & 7.13e-4 & 1.44e-4 & 1.44e-4 & 2.90e-4\\
		EF21 & ~ & 3.90e-4 & 3.90e-4 & 7.94e-4 & 6.50e-4 & 6.50e-4 & 1.34e-3 & 3.5e-4 & 3.5e-4 & 7.13e-4 & 1.44e-4 & 1.44e-4 & 2.90e-4\\ \hline

		EF-BV & \multirow{2}{*}{$\gamma$} & 1.38e-4 & 1.43e-4 & 2.87e-4 & 2.33e-3 & 2.36e-3 & 4.80e-3 & 2.53e-4 & 2.58e-4 & 5.28e-4 & 1.01e-4 & 1.15e-4 & 2.15e-4\\
		EF21 & ~ & 1.03e-4 & 1.06e-4 & 2.10e-4 & 1.71e-3 & 1.73e-3 & 3.49e-3 & 1.91e-4 & 1.84e-4 & 3.87e-4 & 8.12e-5 & 9.31e-5 & 1.63e-4\\ \hline
	\end{tabular}}
\end{table}
  
\subsection{Experimental results and analysis}  
We show in Fig.~\ref{fig:0007} the results with $k=1$ or $k=2$ in the compressors  \texttt{comp-}$(k,k')$, and overlapping factor $\xi=1$ or $\xi=2$.
We chose $k'=\frac{d}{2}$ and $n=1000$. The corresponding values of $\eta$, $\omega$, $\oma$, and the parameter values used in the algorithms are shown in Tab.~\ref{tab10}. We can see that there is essentially no difference between the two choices $\xi=1$ and $\xi=2$, and the qualitative behavior for $k=1$ and $k=2$ is similar. Thus, we observe that  \algname{EF-BV} converges always faster than \algname{EF21}; this is consistent with our analysis. 

We tried other values of $n$, including the largest value $n=N$, for which there is only one data point at every node. The behavior of \algname{EF21} and \algname{EF-BV} was the same as for $n=1000$, so we don't show the results.

We tried other values of $k'$. The behavior of \algname{EF21} and \algname{EF-BV} was the same as for $k'= \frac{d}{2}$ overall, so we don't show the results. We noticed that the difference between the two algorithms was smaller when $k'$ was smaller; this is expected, since for $k'=k$, the compressors revert to \texttt{top-}$k$, for which \algname{EF21} and \algname{EF-BV} are the same algorithm.

To sum up, the experiments confirm our analysis: when $\omega$ and $n$ are large,  so that the key factor $\sqrt{\frac{r_{\mathrm{av}}}{r}}$ is small, randomness is exploited in \algname{EF-BV}, with larger values of $\nu$ and $\gamma$ allowed than in \algname{EF21}, and this yields faster convergence. 

In future work, we will design and compare other compressors in our new class $\mathbb{C}(\eta,\omega)$, performing well in both homogeneous and heterogeneous regimes.

\subsection{Additional experiments in the nonconvex setting} 
We consider the logistic regression  problem with a nonconvex regularizer:
\begin{equation}
   f(x)=\frac{1}{n} \sum_{i=1}^{n} \log \left(1+\exp \left(-y_{i} a_{i}^{\top} x\right)\right)+\lambda \sum_{j=1}^{d} \frac{x_{j}^{2}}{1+x_{j}^{2}}, 
\end{equation}
where $a_{i} \in \mathbb{R}^{d}, y_{i} \in\{-1,1\}$ are the training data, and $\lambda>0$ is the regularizer parameter. We used $\lambda=0.1$ in all experiments. We present the results in Fig.~\ref{fig13}.
   
\begin{figure}[!htbp]
   \centering
   \begin{subfigure}[b]{0.32\textwidth}
      \centering
      \includegraphics[width=\textwidth]{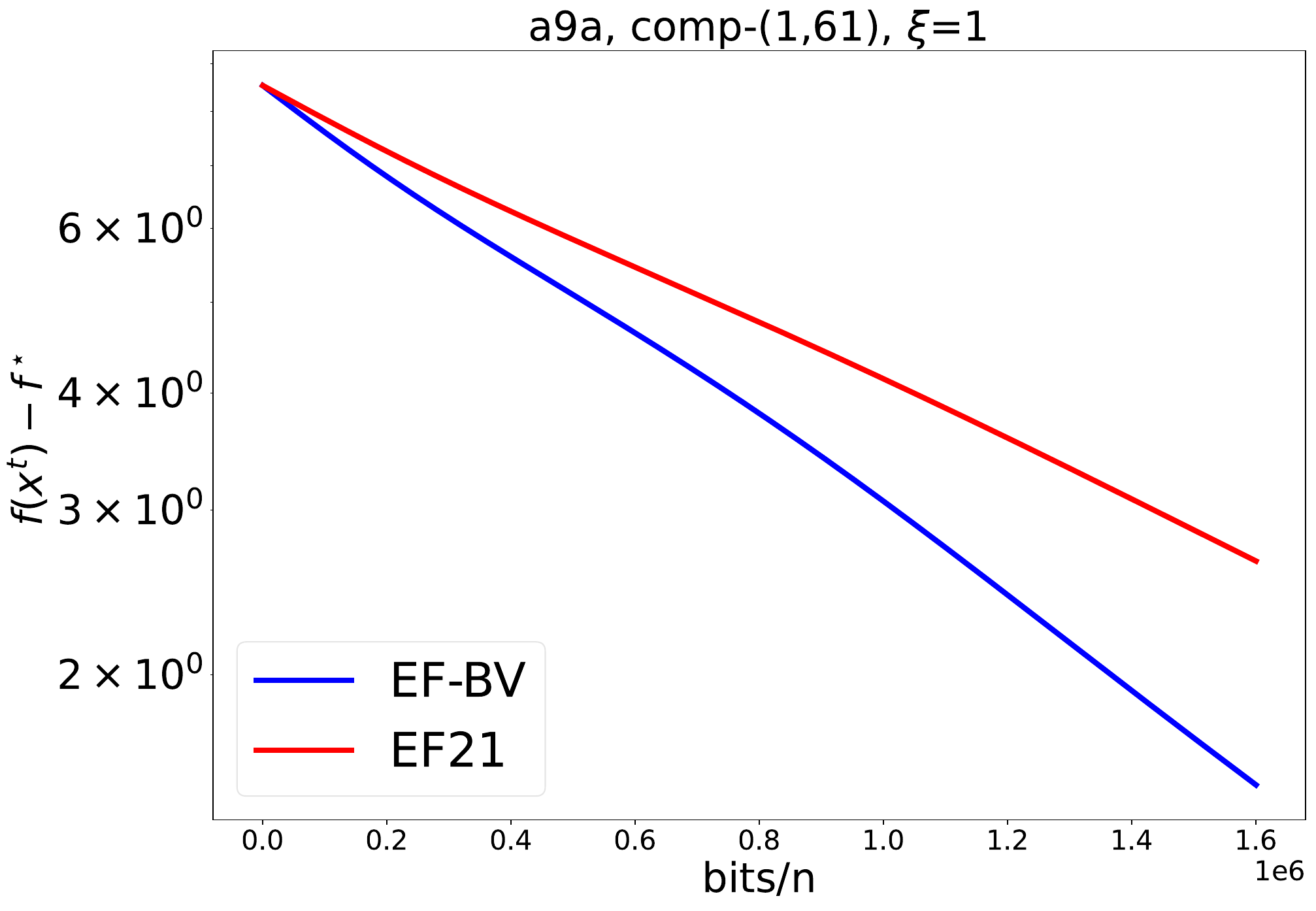}
   \end{subfigure}
   \hfill
   \begin{subfigure}[b]{0.32\textwidth}
      \centering
      \includegraphics[width=\textwidth]{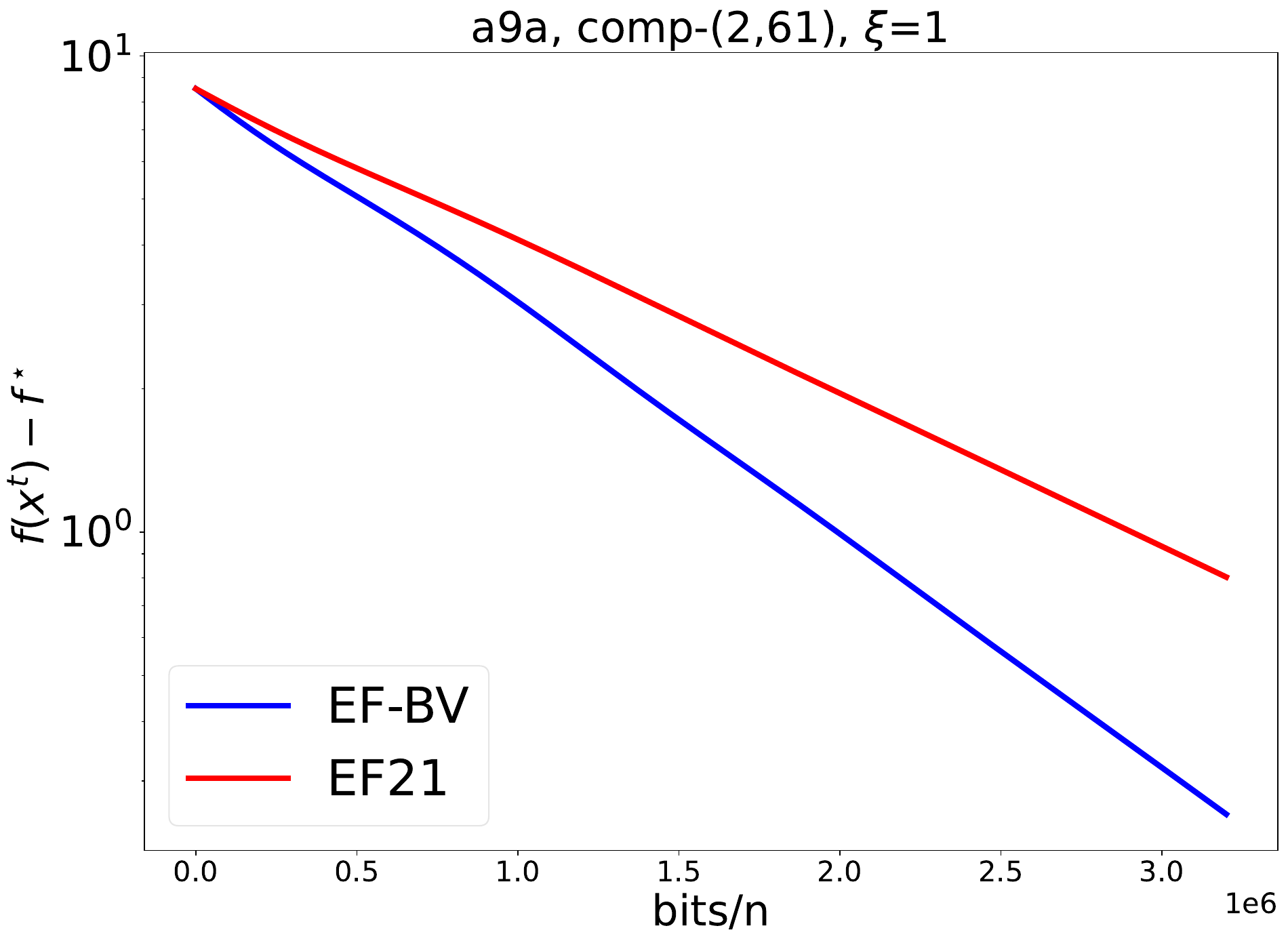}
   \end{subfigure}
   \hfill
   \begin{subfigure}[b]{0.32\textwidth}
      \centering
      \includegraphics[width=\textwidth]{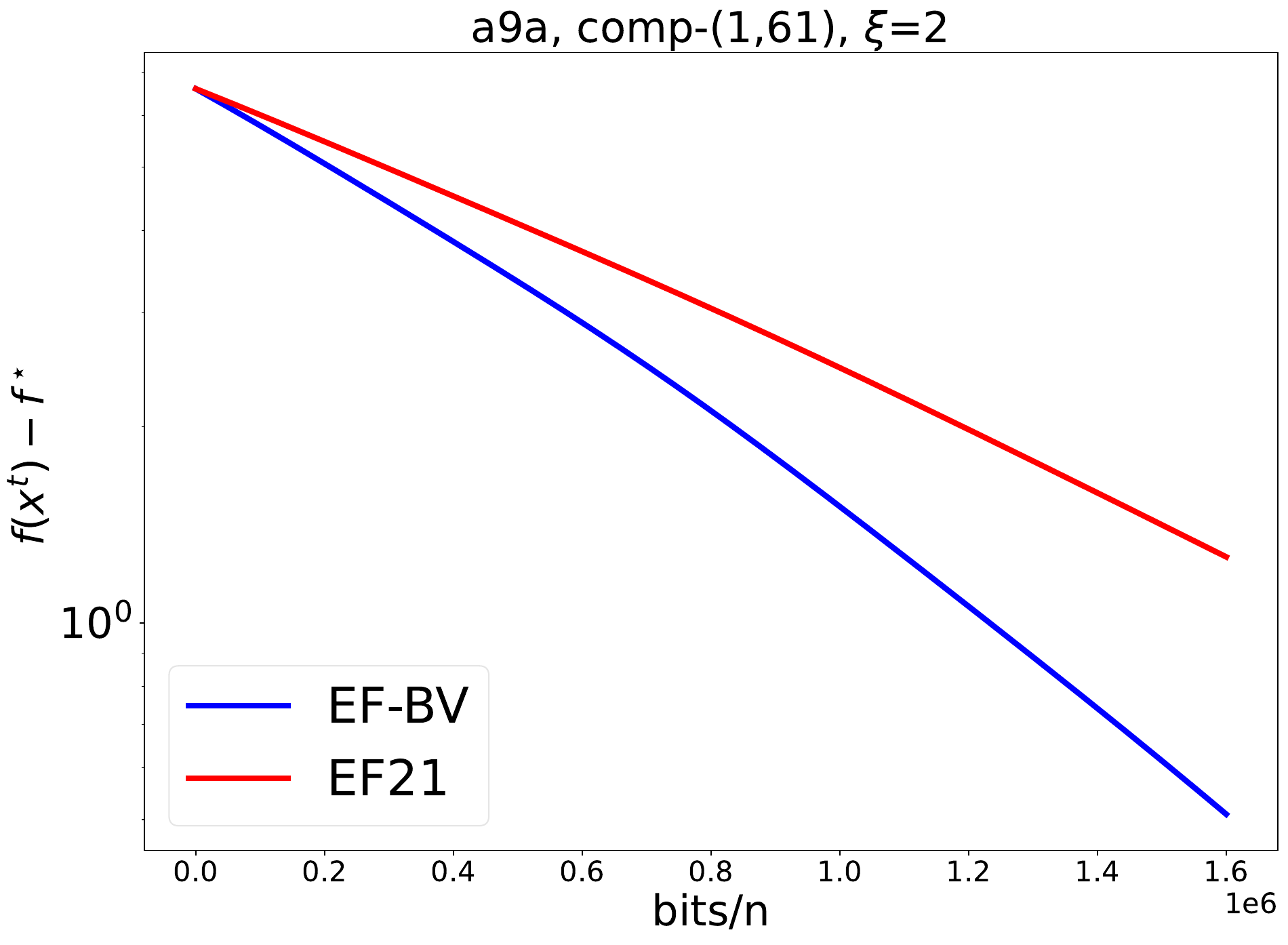}
   \end{subfigure}
   \hfill
   \begin{subfigure}[b]{0.32\textwidth}
      \centering
      \includegraphics[width=\textwidth]{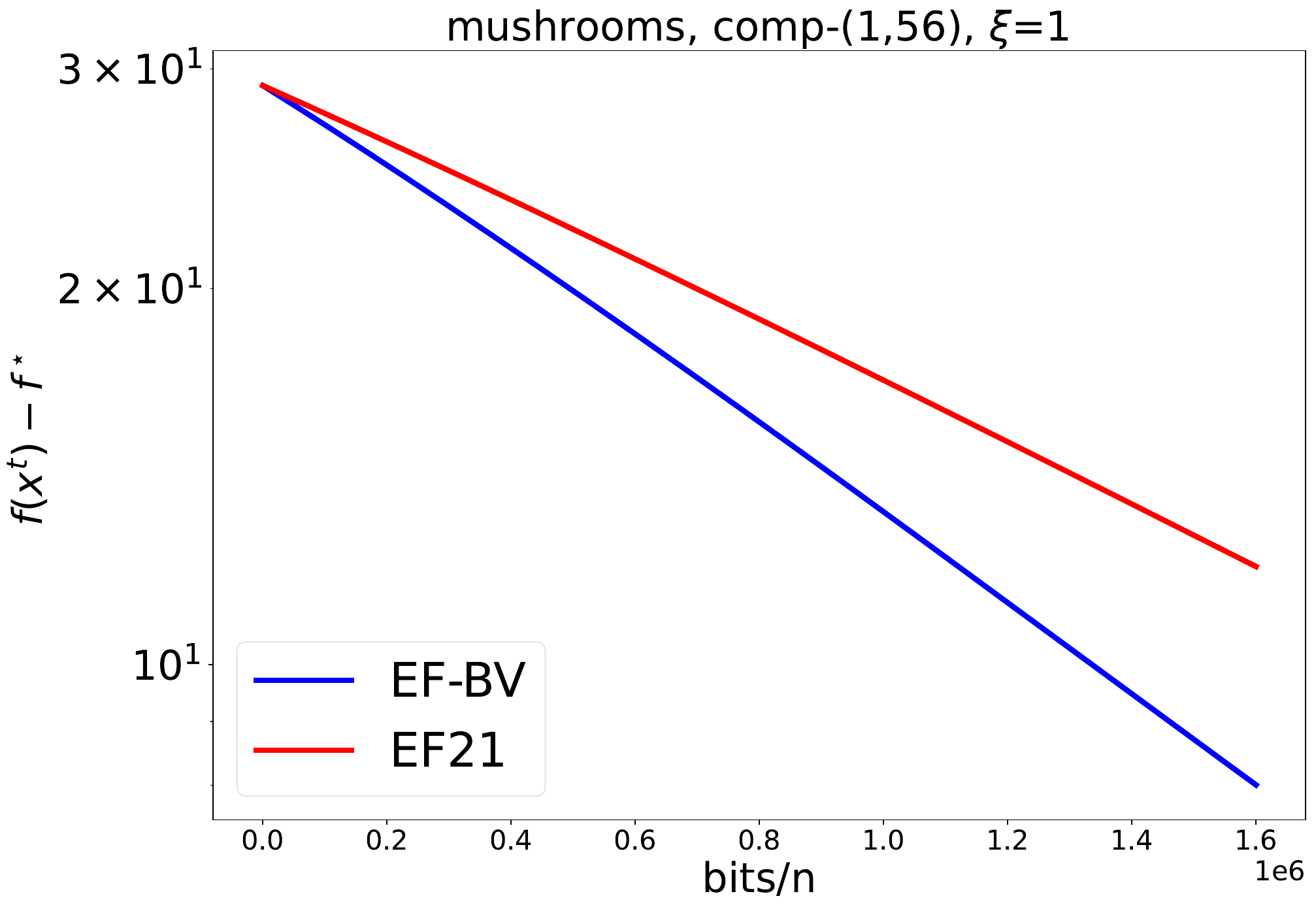}
   \end{subfigure}
   \hfill
   \begin{subfigure}[b]{0.32\textwidth}
      \centering
      \includegraphics[width=\textwidth]{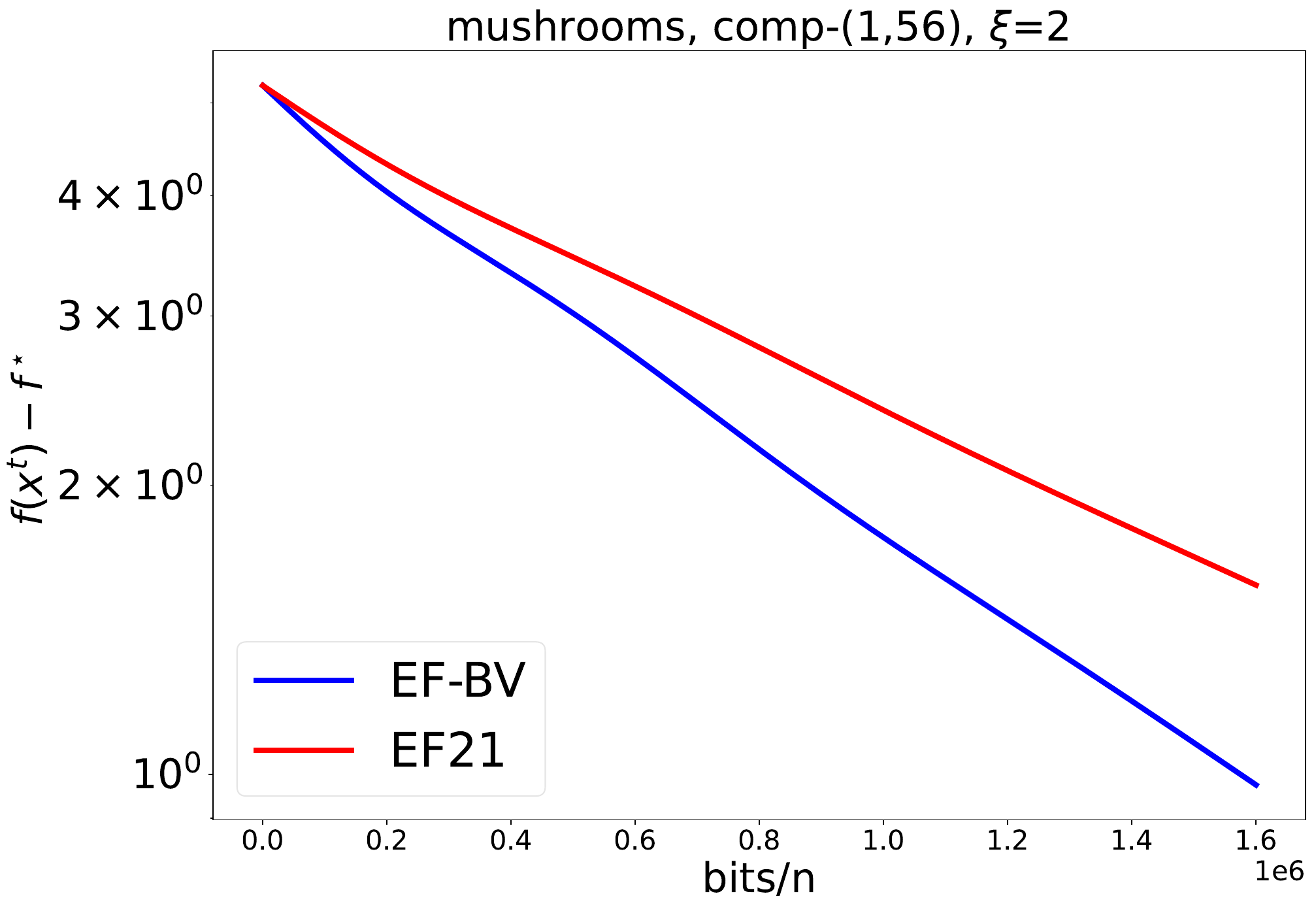}
   \end{subfigure}
   \hfill
   \begin{subfigure}[b]{0.32\textwidth}
      \centering
      \includegraphics[width=\textwidth]{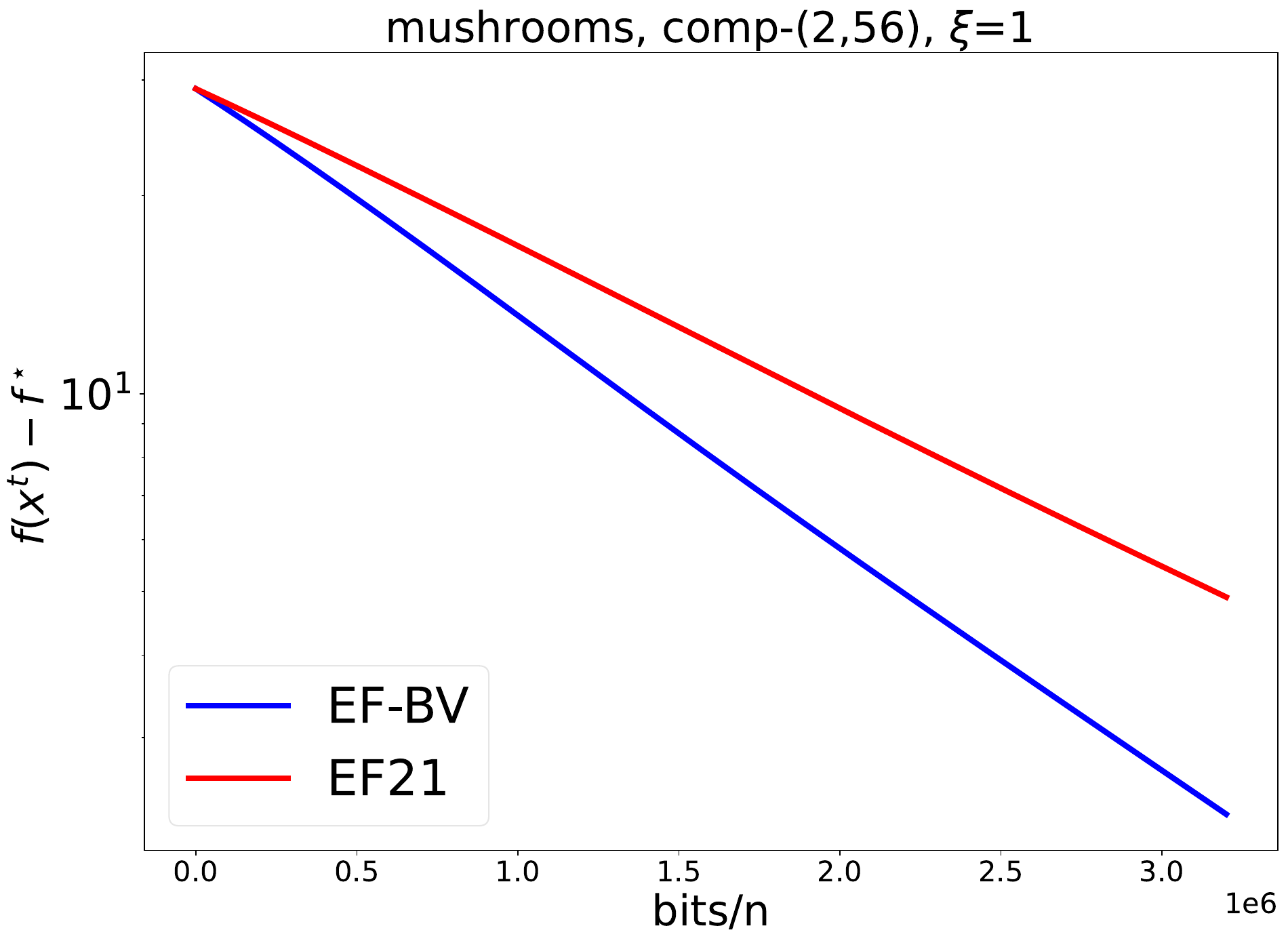}
   \end{subfigure}
   \hfill
   \begin{subfigure}[b]{0.32\textwidth}
      \centering
      \includegraphics[width=\textwidth]{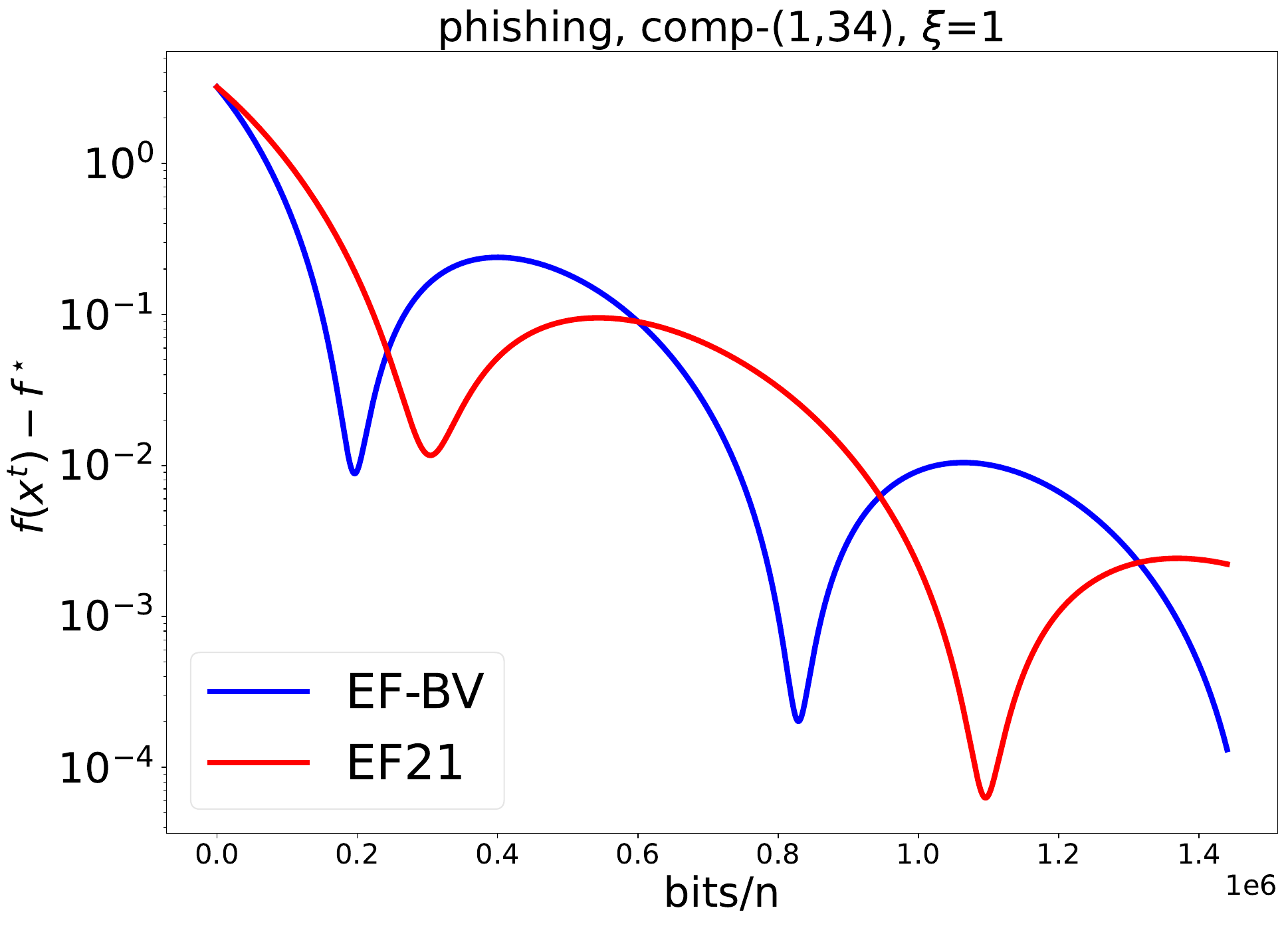}
   \end{subfigure}
   \hfill
   \begin{subfigure}[b]{0.32\textwidth}
      \centering
      \includegraphics[width=\textwidth]{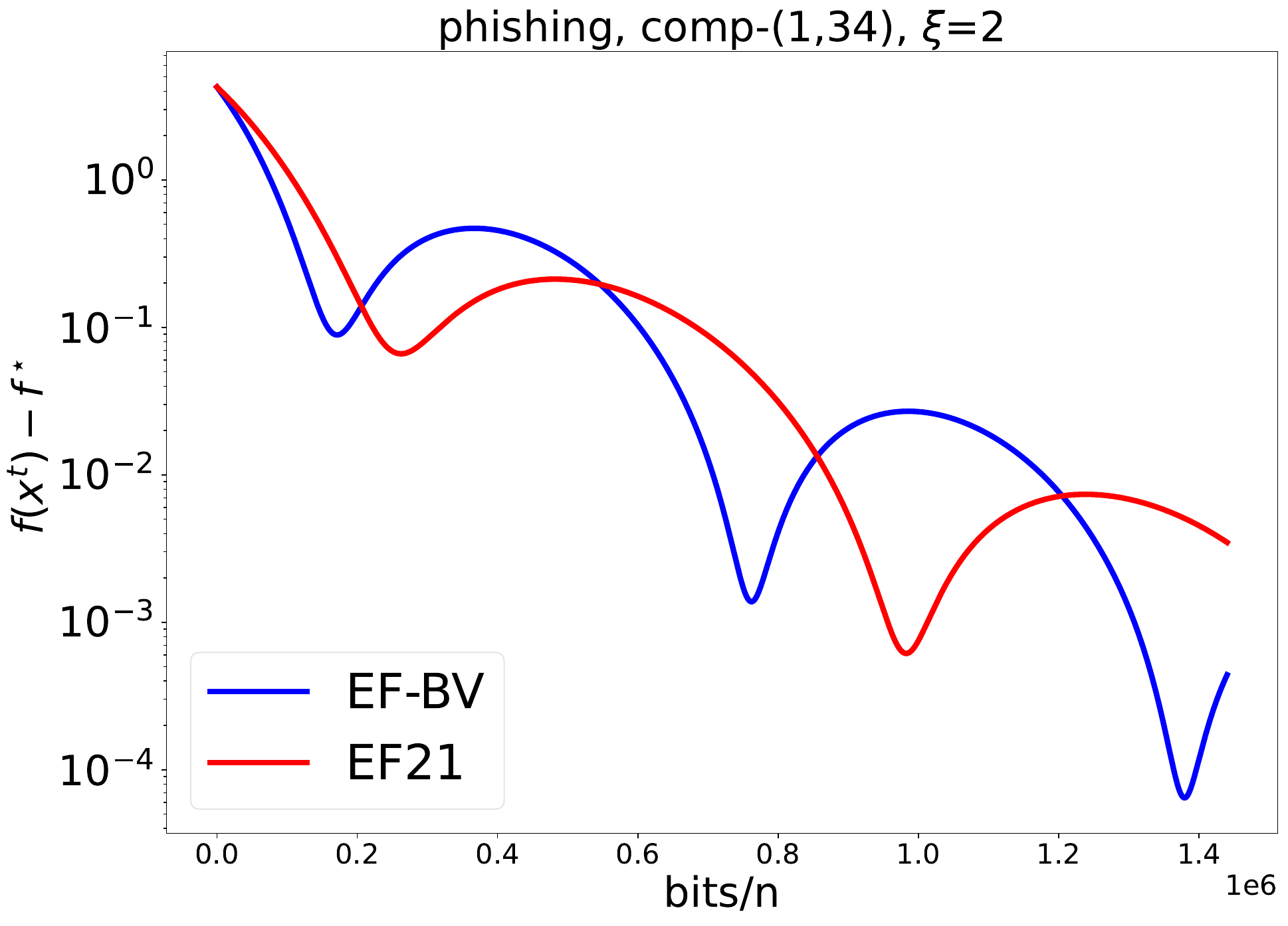}
   \end{subfigure}
   \hfill
   \begin{subfigure}[b]{0.32\textwidth}
      \centering
      \includegraphics[width=\textwidth]{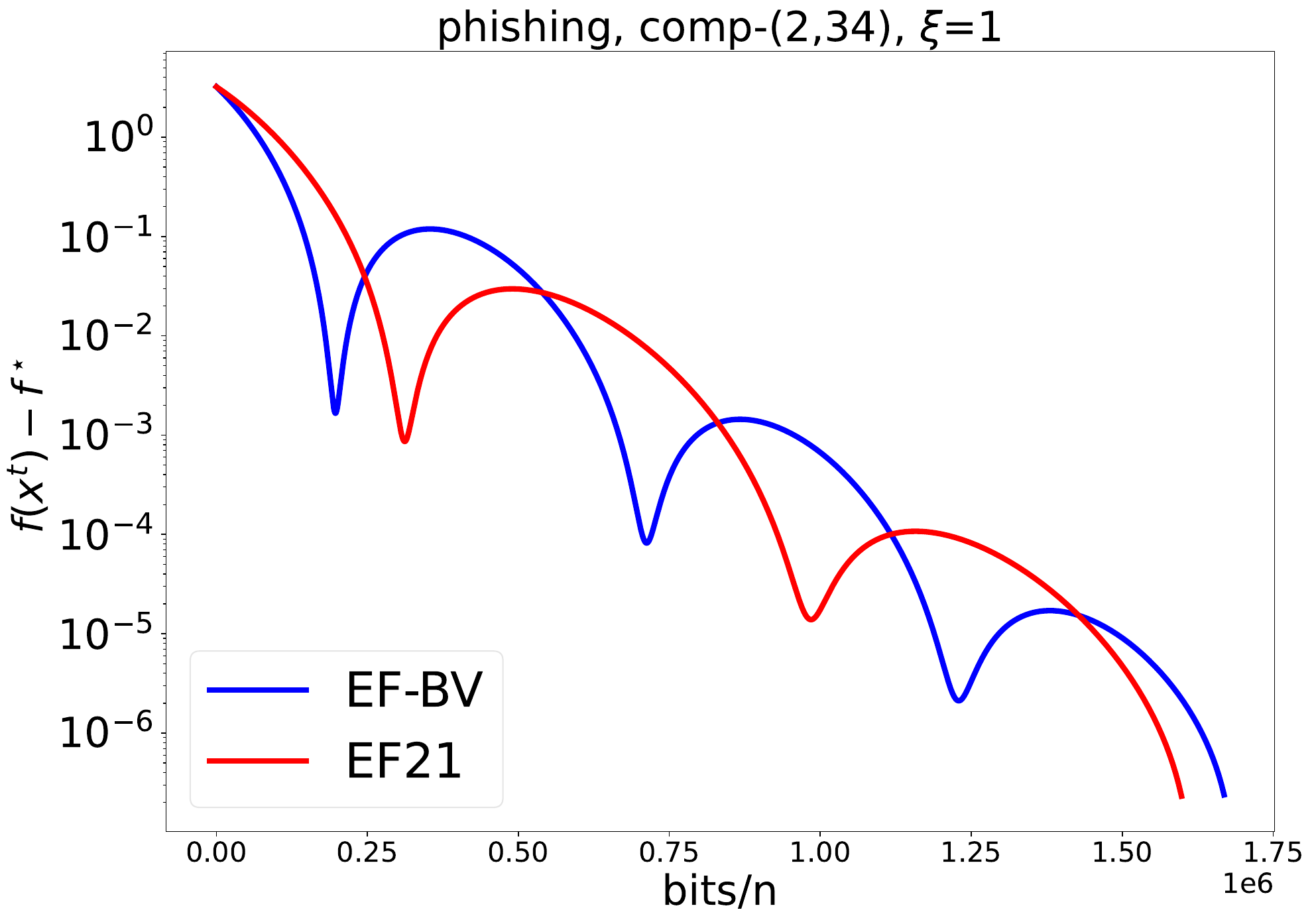}
   \end{subfigure}
      \caption{Comparison between \algname{EF21} and \algname{EF-BV} in the nonconvex setting. We see that \algname{EF-BV} outperforms \algname{EF21} for all datasets.}
      \label{fig13}
\end{figure}

\section{Proof of Proposition~\ref{prop1}}\label{secproofp4}
We first calculate $\omega$. Let $x\in\mathbb{R}^d$.
\begin{align*}
\big\|\mathcal{C}(x)- \mathbb{E}[\mathcal{C}(x)]\big\|^2&=\sum_{i\in \mathcal{I}_d\backslash \{i_1,\ldots, i_{k+k'}\}} \left(\frac{k'}{d-k}\right)^2|x_i|^2
+ \sum_{j=k+1}^{k+k'} \left(\frac{d-k-k'}{d-k}\right)^2|x_{i_j}|^2.
\end{align*}
Therefore, by taking the expectation over the random indexes $i_{k+1},\ldots,i_{2k}$,
\begin{align*}
\Exp{\big\|\mathcal{C}(x)- \mathbb{E}[\mathcal{C}(x)]\big\|^2} &=\sum_{i\in \mathcal{I}_d\backslash \{i_1,\ldots, i_{k}\}} 
\left( \frac{d-k-k'}{d-k}\left(\frac{k'}{d-k}\right)^2 +\frac{k'}{d-k} \left(\frac{d-k-k'}{d-k}\right)^2 
\right)|x_i|^2\\
&=\frac{k'(d-k-k')}{(d-k)^2}\sum_{i\in \mathcal{I}_d\backslash \{i_1,\ldots, i_{k}\}}  |x_i|^2.
\end{align*}
Moreover, since the $|x_{i_j}|$ are the largest elements of $|x|$, for every $j=1,\ldots,k$, 
\begin{equation*}
|x_{i_j}|^2\geq \frac{1}{d-k}\sum_{i\in \mathcal{I}_d\backslash \{i_1,\ldots, i_{k}\}}  |x_i|^2,
\end{equation*}
so that
\begin{equation*}
\|x\|^2 = \sum_{i\in \mathcal{I}_d}  |x_i|^2 \geq \left(1+\frac{k}{d-k}\right) \sum_{i\in \mathcal{I}_d\backslash \{i_1,\ldots, i_{k}\}}  |x_i|^2.
\end{equation*}
Hence, 
\begin{align*}
\Exp{\big\|\mathcal{C}(x)- \mathbb{E}[\mathcal{C}(x)]\big\|^2} &\leq
\frac{k'(d-k-k')}{(d-k)^2}\frac{d-k}{d}\|x\|^2 = \frac{k'(d-k-k')}{(d-k)d}\|x\|^2.
\end{align*}
Then, let us calculate $\eta$.
\begin{align*}
\big\| \mathbb{E}[\mathcal{C}(x)]-x\big\|^2 &=\sum_{i\in \mathcal{I}_d\backslash \{i_1,\ldots, i_{k}\}} \left(\frac{d-k-k'}{d-k}\right)^2|x_i|^2\\
&\leq  \frac{(d-k-k')^2}{(d-k)d}\|x\|^2.
\end{align*}
Thus, $\eta =\frac{d-k-k'}{\sqrt{(d-k)d}}$.

\section{Proof of Proposition~\ref{prop2}}\label{secproofp5}
We first calculate $\omega$. Let $x\in\mathbb{R}^d$.
\begin{align*}
\big\|\mathcal{C}(x)- \mathbb{E}[\mathcal{C}(x)]\big\|^2&=\sum_{j\in\{j_1,\ldots,j_{k}\}}  \left(\frac{k'-k}{k}\right)^2|x_{i_j}|^2+
\sum_{i\in  \{i_1,\ldots, i_{k'}\}\backslash \{i_{j_1},\ldots, i_{j_{k}}\}} |x_i|^2
\end{align*}
Therefore, by taking the expectation over the random indexes $i_{j_1},\ldots,i_{j_{k}}$,
\begin{align*}
\Exp{\big\|\mathcal{C}(x)- \mathbb{E}[\mathcal{C}(x)]\big\|^2} &=\sum_{j=1}^{k'}
\left( \frac{k}{k'}\left(\frac{k'-k}{k}\right)^2 +\frac{k'-k}{k'}\right)|x_{i_j}|^2\\
&=\frac{k'-k}{k}\sum_{j=1}^{k'}
|x_{i_j}|^2\\
&\leq\frac{k'-k}{k}\|x\|^2
\end{align*}
Then, let us calculate $\eta$:
\begin{align*}
\big\| \mathbb{E}[\mathcal{C}(x)]-x\big\|^2 =\sum_{i\in \mathcal{I}_d\backslash \{i_1,\ldots, i_{k'}\}} |x_i|^2 \leq  \frac{d-k'}{d}\|x\|^2.
\end{align*}

\section{Proof of Theorem~\ref{theo1}}
We have the descent property \citep[Lemma 4]{ric21}, for every $t\geq 0$,
\begin{align}
f(x^{t+1}) -f^\star &\leq f(x^t)  -f^\star -\frac{\gamma}{2} \sqnorm{\nabla f(x^t)} +\frac{ \gamma }{2}\sqnorm{g^{t+1}-\nabla f(x^t)}\notag\\
&\quad+ \left(\frac{ L}{2}-\frac{1}{2\gamma}\right)\sqnorm{x^{t+1}-x^t}\label{eqgergg}\\
&\leq (1-\gamma\mu) \big(f(x^t)  -f^\star\big)  +\frac{ \gamma }{2}\sqnorm{g^{t+1}-\nabla f(x^t)}+ \left(\frac{ L}{2}-\frac{1}{2\gamma}\right)\sqnorm{x^{t+1}-x^t}.\notag
\end{align}
Then, for every $t\geq 0$, conditionally on $x^t$, $h^t$ and $(h_i^t)_{i=1}^n$,
\begin{align*}
\Exp{\sqnorm{g^{t+1}-\nabla f(x^t)}} &=\Exp{\sqnorm{\frac{1}{n}\sum_{i=1}^n \Big(h_i^{t}-\nabla f_i(x^t) +\nu \mathcal{C}_i^t\big(\nabla f_i(x^t)-h_i^t\big) \Big) }}\\
&=\sqnorm{\frac{1}{n}\sum_{i=1}^n \Big(h_i^{t}-\nabla f_i(x^t) +\nu \Exp{\mathcal{C}_i^t\big(\nabla f_i(x^t)-h_i^t\big)} \Big)}\\
&\quad+\nu^2\Exp{\sqnorm{\frac{1}{n}\sum_{i=1}^n \Big( \mathcal{C}_i^t\big(\nabla f_i(x^t)-h_i^t\big)-\Exp{ \mathcal{C}_i^t\big(\nabla f_i(x^t)-h_i^t\big) } \Big) }}\\
&\leq \sqnorm{\frac{1}{n}\sum_{i=1}^n \Big(h_i^{t}-\nabla f_i(x^t) +\nu \Exp{\mathcal{C}_i^t\big(\nabla f_i(x^t)-h_i^t\big)} \Big)}\\
&\quad+\nu^2 \frac{\oma}{n}\sum_{i=1}^n \sqnorm{\nabla f_i(x^t)-h_i^t },
\end{align*}
where the last inequality follows from \eqref{eqbo}. In addition,
\begin{align*}
&\left\|\frac{1}{n}\sum_{i=1}^n \Big(h_i^{t}-\nabla f_i(x^t) +\nu \Exp{\mathcal{C}_i^t\big(\nabla f_i(x^t)-h_i^t\big)} \Big)\right\|\\
&\quad \leq \left\|\frac{1}{n}\sum_{i=1}^n \Big(\nu\big(h_i^{t}-\nabla f_i(x^t)\big) +\nu \Exp{\mathcal{C}_i^t\big(\nabla f_i(x^t)-h_i^t\big)} \Big)\right\|\\
&\quad\quad + (1-\nu)\left\|\frac{1}{n}\sum_{i=1}^n \big(h_i^{t}-\nabla f_i(x^t)\big)\right\|\\
&\quad \leq \frac{\nu}{n}\sum_{i=1}^n\left\|h_i^{t}-\nabla f_i(x^t)+ \Exp{\mathcal{C}_i^t\big(\nabla f_i(x^t)-h_i^t\big)}\right\|\\
&\quad\quad + \frac{1-\nu}{n}\sum_{i=1}^n \left\|h_i^{t}-\nabla f_i(x^t)\right\|\\
&\quad \leq  \frac{\nu\eta}{n}\sum_{i=1}^n \left\|\nabla f_i(x^t)-h_i^{t}\right\|+ \frac{1-\nu}{n}\sum_{i=1}^n \left\|\nabla f_i(x^t)-h_i^{t}\right\|\\
& \quad = \frac{1-\nu+\nu\eta}{n}\sum_{i=1}^n \left\|\nabla f_i(x^t)-h_i^{t}\right\|.
\end{align*}
Therefore, 
\begin{align*}
\sqnorm{\frac{1}{n}\sum_{i=1}^n \Big(h_i^{t}-\nabla f_i(x^t) +\nu \Exp{\mathcal{C}_i^t\big(\nabla f_i(x^t)-h_i^t\big)} \Big)} \leq \frac{(1-\nu+\nu\eta)^2}{n}\sum_{i=1}^n \sqnorm{\nabla f_i(x^t)-h_i^{t}},
\end{align*}
and, 
conditionally on $x^t$, $h^t$ and $(h_i^t)_{i=1}^n$,
\begin{align*}
\Exp{\sqnorm{g^{t+1}-\nabla f(x^t)}} &\leq 
\left((1-\nu+\nu\eta)^2+\nu^2\oma\right)\frac{1}{n}\sum_{i=1}^n \sqnorm{\nabla f_i(x^t)-h_i^{t}}.
\end{align*}

Thus, for every $t\geq 0$, conditionally on $x^t$, $h^t$ and $(h_i^t)_{i=1}^n$,
 \begin{align*}
\Exp{f(x^{t+1}) -f^\star} 
&\leq (1-\gamma\mu) \big(f(x^t)  -f^\star \big)  +\frac{ \gamma }{2}\big((1-\nu+\nu\eta)^2+\nu^2\oma\big)\frac{1}{n}\sum_{i=1}^n \sqnorm{\nabla f_i(x^t)-h_i^{t}}\\
&\quad+ \left(\frac{ L}{2}-\frac{1}{2\gamma}\right)\Exp{\sqnorm{x^{t+1}-x^t}}.
\end{align*}

Now, let us study the control variates $h_i^t$. Let $s>0$. Using the Peter--Paul inequality $\|a+b\|^2 \leq (1+s) \|a\|^2 + (1+s^{-1}) \|b\|^2$, for any vectors $a$ and $b$, 
we have, for every $t\geq 0$ and $i\in\mathcal{I}_n$,
\begin{align*}
\sqnorm{\nabla f_i(x^{t+1})-h_i^{t+1}}&=\sqnorm{h_i^{t}-\nabla f_i(x^{t+1}) +\lambda \mathcal{C}_i^t\big(\nabla f_i(x^t)-h_i^t\big)  }\\
&\leq (1+s)\sqnorm{h_i^{t}-\nabla f_i(x^{t}) +\lambda \mathcal{C}_i^t\big(\nabla f_i(x^t)-h_i^t\big)  }\\ 
&\quad+(1+s^{-1})\sqnorm{\nabla f_i(x^{t+1})-\nabla f_i(x^{t})}\\
&\leq (1+s)\sqnorm{h_i^{t}-\nabla f_i(x^{t}) +\lambda \mathcal{C}_i^t\big(\nabla f_i(x^t)-h_i^t\big)  } \\
&\quad+(1+s^{-1})L_i^2\sqnorm{x^{t+1}-x^{t}}.
\end{align*}
Moreover,  conditionally on $x^t$, $h^t$ and $(h_i^t)_{i=1}^n$, 
\begin{align*}
\Exp{\sqnorm{h_i^{t}-\nabla f_i(x^{t}) +\lambda \mathcal{C}_i^t\big(\nabla f_i(x^t)-h_i^t\big)  }}&=\sqnorm{h_i^{t}-\nabla f_i(x^{t}) +\lambda \Exp{\mathcal{C}_i^t\big(\nabla f_i(x^t)-h_i^t\big)} }\\
&\quad+\lambda^2\Exp{\sqnorm{ \mathcal{C}_i^t\big(\nabla f_i(x^t)-h_i^t\big)-\Exp{ \mathcal{C}_i^t\big(\nabla f_i(x^t)-h_i^t\big) }  }}\\
&\leq\sqnorm{h_i^{t}-\nabla f_i(x^{t}) +\lambda \Exp{\mathcal{C}_i^t\big(\nabla f_i(x^t)-h_i^t\big)} }\\
&\quad+\lambda^2\omega \sqnorm{\nabla f_i(x^t)-h_i^t}.
\end{align*}
In addition,
 \begin{align*}
\left\|h_i^{t}-\nabla f_i(x^{t}) +\lambda \Exp{\mathcal{C}_i^t\big(\nabla f_i(x^t)-h_i^t\big)}\right\|& \leq \left\|\lambda\big(h_i^{t}-\nabla f_i(x^{t})\big) +\lambda \Exp{\mathcal{C}_i^t\big(\nabla f_i(x^t)-h_i^t\big)}\right\|\notag\\
&\quad + (1-\lambda)\left\|h_i^{t}-\nabla f_i(x^t)\right\|\\
&\leq  \lambda\eta\left\|\nabla f_i(x^t)-h_i^{t}\right\|+ (1-\lambda)\left\|\nabla f_i(x^t)-h_i^{t}\right\|\notag\\
&=(1-\lambda+\lambda\eta)\left\|\nabla f_i(x^t)-h_i^t\right\|.
\end{align*}
Therefore, conditionally on $x^t$, $h^t$ and $(h_i^t)_{i=1}^n$, 
\begin{align*}
\Exp{\sqnorm{h_i^{t}-\nabla f_i(x^{t}) +\lambda \mathcal{C}_i^t\big(\nabla f_i(x^t)-h_i^t\big)  }}\leq\big((1-\lambda+\lambda\eta)^2+\lambda^2\omega\big) \sqnorm{\nabla f_i(x^t)-h_i^t}
\end{align*}
and
\begin{align*}
\Exp{\sqnorm{\nabla f_i(x^{t+1})-h_i^{t+1}}}&\leq (1+s)\big((1-\lambda+\lambda\eta)^2+\lambda^2\omega\big) \sqnorm{\nabla f_i(x^{t})-h_i^{t}}\\
&\quad+(1+s^{-1})L_i^2\Exp
{\sqnorm{x^{t+1}-x^{t}}},
\end{align*}
so that
\begin{align*}
\Exp{\frac{1}{n}\sum_{i=1}^n \sqnorm{\nabla f_i(x^{t+1})-h_i^{t+1}}}&
\leq (1+s)\big((1-\lambda+\lambda\eta)^2+\lambda^2\omega\big) \frac{1}{n}\sum_{i=1}^n \sqnorm{\nabla f_i(x^t)-h_i^{t}}\\
&\quad+(1+s^{-1})\tilde{L}^2\Exp
{\sqnorm{x^{t+1}-x^{t}}}.
\end{align*}

Let $\theta>0$; its value will be set to $\theta^\star$ later on. We introduce the Lyapunov function, for every $t\geq 0$,
\begin{equation*}
\Psi^t \eqdef f(x^t)-f^\star + \frac{\gamma}{2\theta}  \frac{1}{n}\sum_{i=1}^n \sqnorm{\nabla f_i(x^t)-h_i^{t}}.
\end{equation*}
Hence, for every $t\geq 0$, conditionally on $x^t$, $h^t$ and $(h_i^t)_{i=1}^n$, we have
\begin{align}
\Exp{\Psi^{t+1}} &\leq (1-\gamma\mu) \big(f(x^t)  -f^\star \big)\notag \\
&\quad +\frac{ \gamma }{2\theta}\Big( \theta\big((1-\nu+\nu\eta)^2+\nu^2\oma\big)\notag\\
&\quad+(1+s)\big((1-\lambda+\lambda\eta)^2+\lambda^2\omega\big)
\Big) \frac{1}{n}\sum_{i=1}^n \sqnorm{\nabla f_i(x^t)-h_i^{t}}\label{eqq1}\\
&\quad+ \left(\frac{ L}{2}-\frac{1}{2\gamma}+\frac{\gamma}{2\theta}(1+s^{-1})\tilde{L}^2\right)\!\Exp{\sqnorm{x^{t+1}-x^t}}.\notag
\end{align}
Making use of $r$ and $r_{\mathrm{av}}$ and setting 
$\theta = s(1+s)\frac{r}{r_{\mathrm{av}}}$, 
we can rewrite \eqref{eqq1} as:
\begin{align*}
\Exp{\Psi^{t+1}} &\leq (1-\gamma\mu) \big(f(x^t)  -f^\star \big) +\frac{ \gamma }{2\theta}\Big( \theta r_{\mathrm{av}}
+(1+s)r
\Big) \frac{1}{n}\sum_{i=1}^n \sqnorm{\nabla f_i(x^t)-h_i^{t}}\\
&\quad+ \left(\frac{ L}{2}-\frac{1}{2\gamma}+\frac{\gamma}{2\theta}(1+s^{-1})\tilde{L}^2\right)\!\Exp{\sqnorm{x^{t+1}-x^t}}\\
&=(1-\gamma\mu) \big(f(x^t)  -f^\star \big) +\frac{ \gamma }{2\theta} (1+s)^2 
 \frac{r}{n}\sum_{i=1}^n \sqnorm{\nabla f_i(x^t)-h_i^{t}}\\
&\quad+ \left(\frac{ L}{2}-\frac{1}{2\gamma}+\frac{\gamma}{2s^2}\frac{r_{\mathrm{av}}}{r}\tilde{L}^2\right)\!\Exp{\sqnorm{x^{t+1}-x^t}}.
\end{align*}
We now choose $\gamma$ small enough so that 
 \begin{equation}
L-\frac{1}{\gamma}+\frac{\gamma}{s^2}\frac{r_{\mathrm{av}}}{r}\tilde{L}^2 \leq 0.\label{eqgreg}
\end{equation}
A sufficient condition for \eqref{eqgreg} to hold is \citep[Lemma 5]{ric21}:
\begin{equation}
0<\gamma \leq \frac{1}{L+\tilde{L}\sqrt{\frac{r_{\mathrm{av}}}{r}}\frac{1}{s}}.\label{eqgamfek}
\end{equation}
Then, assuming that \eqref{eqgamfek} holds, we have, for every $t\geq 0$, conditionally on $x^t$, $h^t$ and $(h_i^t)_{i=1}^n$,
\begin{align*}
\Exp{\Psi^{t+1}} &\leq (1-\gamma\mu) \big(f(x^t)  -f^\star \big) +\frac{ \gamma }{2\theta} (1+s)^2 
 \frac{r}{n}\sum_{i=1}^n \sqnorm{\nabla f_i(x^t)-h_i^{t}}\\
&\leq \max\big(1-\gamma\mu,(1+s)^2 r\big)  \Psi^t.
\end{align*}

We see that $s$ must be small enough so that $(1+s)^2 r <1$; this is the case with 
$s=s^\star$, so that $(1+s^\star)^2 r = \frac{r+1}{2}<1$. 
Therefore, we set $s=s^\star$, and, accordingly, $\theta=\theta^\star$. Then, for every $t\geq 0$,
 conditionally on $x^t$, $h^t$ and $(h_i^t)_{i=1}^n$,
\begin{align*}
\Exp{\Psi^{t+1}} 
&\leq \max\big(1-\gamma\mu, {\frac{r+1}{2}}\big) \Psi^t.
\end{align*}
Unrolling the recursion using the tower rule yields \eqref{eqsdgerg}.

\section{Proof of Theorem~\ref{theo2}}
Using $L$-smoothness of $f$, we have, for every $t\geq 0$,
\begin{equation*}
f(x^{t+1})\leq f(x^t) + \langle \nabla f(x^t),x^{t+1}-x^t\rangle + \frac{L}{2}\|x^{t+1}-x^t\|^2.
\end{equation*}
Moreover, using convexity of $R$, we have, for every subgradient $u^{t+1}\in \partial R(x^{t+1})$,
\begin{equation}
R(x^t)\geq R(x^{t+1}) + \langle u^{t+1}, x^{t}-x^{t+1}\rangle.\label{khfjehgg}
\end{equation}
From the property that $\mathrm{prox}_{\gamma R}=(\mathrm{Id}+\gamma \partial R)^{-1}$~\citep{bau17}, it follows from 
$x^{t+1} = \mathrm{prox}_{\gamma R}(x^t - \gamma  g^{t+1})$ that
\begin{equation*}
0\in  \partial R(x^{t+1})+ \frac{1}{\gamma}(x^{t+1} - x^t+\gamma g^{t+1}).
\end{equation*}
So, we set $u^{t+1}\eqdef \frac{1}{\gamma}(x^{t} - x^{t+1})-g^{t+1}$. Using this subgradient in \eqref{khfjehgg} and replacing
$x^{t}-x^{t+1}$ by $\gamma(u^{t+1}+g^{t+1})$, we get, for every $t\geq 0$,
\begin{align*}
f(x^{t+1})+R(x^{t+1}) &\leq f(x^t) +R(x^t)  + \langle \nabla f(x^t)+u^{t+1},x^{t+1}-x^t\rangle + \frac{L}{2}\|x^{t+1}-x^t\|^2\\
&=f(x^t) +R(x^t)  - \gamma \langle \nabla f(x^t)+u^{t+1},g^{t+1}+u^{t+1}\rangle + \frac{L}{2}\gamma^2\|g^{t+1}+u^{t+1}\|^2\\
&=f(x^t) +R(x^t)  +\frac{ \gamma }{2}\|\nabla f(x^t)-g^{t+1}\|^2+ \left(\frac{\gamma^2 L}{2}-\frac{\gamma}{2}\right)\|g^{t+1}+u^{t+1}\|^2\\
&\quad -\frac{\gamma}{2} \|\nabla f(x^t)+u^{t+1}\|^2\\
&=f(x^t) +R(x^t)  +\frac{\gamma }{2}\|\nabla f(x^t)-g^{t+1}\|^2+ \left(\frac{ L}{2}-\frac{1}{2\gamma}\right)\|x^{t+1}-x^t\|^2\\
&\quad -\frac{\gamma}{2} \|\nabla f(x^t)+u^{t+1}\|^2
\end{align*}
Note that we recover \eqref{eqgergg} if $R=0$ and $u^t \equiv 0$.

Using the fact that for any vectors $a$ and $b$, $-\|a+b\|^2 \leq -\frac{1}{2} \|a\|^2 + \|b\|^2$, we have, for every $t\geq 0$,
\begin{align*}
-\frac{\gamma}{2} \|\nabla f(x^t)+u^{t+1}\|^2 &\leq -\frac{\gamma}{4} \|\nabla f(x^{t+1})+u^{t+1}\|^2 + \frac{\gamma}{2}  \|\nabla f(x^{t+1})-\nabla f(x^{t})\|^2\\
&\leq -\frac{\gamma}{4} \|\nabla f(x^{t+1})+u^{t+1}\|^2 + \frac{\gamma L^2}{2}  \|x^{t+1}-x^t\|^2.
\end{align*}
Hence, for every $t\geq 0$,
\begin{align*}
f(x^{t+1})+R(x^{t+1}) 
&\leq  f(x^t) +R(x^t)  +\frac{\gamma}{2}\|\nabla f(x^t)-g^{t+1}\|^2+ \left(\frac{ L}{2}-\frac{1}{2\gamma}+\frac{\gamma L^2}{2}\right)\|x^{t+1}-x^t\|^2\\
&\quad -\frac{\gamma}{4} \|\nabla f(x^{t+1})+u^{t+1}\|^2.
\end{align*}
It follows from the K{\L}  assumption \eqref{eqKL} that
\begin{align*}
f(x^{t+1})+R(x^{t+1}) -f^\star - R^\star
&\leq f(x^{t})+R(x^{t}) -f^\star - R^\star +\frac{\gamma}{2}\|\nabla f(x^t)-g^{t+1}\|^2\\
&\quad + \left(\frac{ L}{2}-\frac{1}{2\gamma}+\frac{\gamma L^2}{2}\right)\|x^{t+1}-x^t\|^2\\
&\quad-2\mu\frac{\gamma}{4}  \left(f(x^{t+1})+R(x^{t+1}) -f^\star - R^\star\right),
\end{align*}
so that
\begin{align*}
 \Big(1+\frac{\gamma\mu}{2}\Big)\left(f(x^{t+1})+R(x^{t+1}) -f^\star - R^\star\right)
&\leq f(x^{t})+R(x^{t}) -f^\star - R^\star +\frac{\gamma}{2}\|\nabla f(x^t)-g^{t+1}\|^2\\
&\quad+ \left(\frac{ L}{2}-\frac{1}{2\gamma}+\frac{\gamma L^2}{2}\right)\|x^{t+1}-x^t\|^2,
\end{align*}
and
\begin{align*}
f(x^{t+1})+R(x^{t+1}) -f^\star - R^\star
&\leq \Big(1+\frac{\gamma\mu}{2}\Big)^{-1}\big(f(x^{t})+R(x^{t}) -f^\star - R^\star\big) +\frac{\gamma}{2}\|\nabla f(x^t)-g^{t+1}\|^2\\
&\quad+ \left(\frac{ L}{2}-\frac{1}{2\gamma}+\frac{\gamma L^2}{2}\right)\|x^{t+1}-x^t\|^2.
\end{align*}
Let $s>0$. Like in the proof of Theorem~\ref{theo1}, we have
\begin{align*}
\Exp{\frac{1}{n}\sum_{i=1}^n \sqnorm{\nabla f_i(x^{t+1})-h_i^{t+1}}}&
\leq (1+s)\big((1-\lambda+\lambda\eta)^2+\lambda^2\omega\big) \frac{1}{n}\sum_{i=1}^n \sqnorm{\nabla f_i(x^t)-h_i^{t}}\\
&\quad+(1+s^{-1})\tilde{L}^2\Exp
{\sqnorm{x^{t+1}-x^{t}}}
\end{align*}
and
\begin{align*}
\Exp{\sqnorm{g^{t+1}-\nabla f(x^t)}} &\leq 
\left((1-\nu+\nu\eta)^2+\nu^2\oma\right)\frac{1}{n}\sum_{i=1}^n \sqnorm{\nabla f_i(x^t)-h_i^{t}}.
\end{align*}
We introduce the Lyapunov function, for every $t\geq 0$,
\begin{align*}
\Psi^t &\eqdef f(x^t)+R(x^t)-f^\star - R^\star + \frac{\gamma}{2\theta}  \frac{1}{n}\sum_{i=1}^n \sqnorm{\nabla f_i(x^t)-h_i^{t}},
\end{align*}
where 
$\theta = s(1+s)\frac{r}{r_{\mathrm{av}}}$.

Following the  same derivations as in the proof of Theorem~\ref{theo1}, we obtain that, for every $t\geq 0$, conditionally on $x^t$, $h^t$ and $(h_i^t)_{i=1}^n$, 
\begin{align*}
\Exp{\Psi^{t+1}} &\leq \Big(1+\frac{\gamma\mu}{2}\Big)^{-1}\big(f(x^{t})+R(x^{t}) -f^\star - R^\star\big) \\
&\quad +\frac{ \gamma }{2\theta}\Big( \theta\big((1-\nu+\nu\eta)^2+\nu^2\oma\big)\\
&\quad+(1+s)\big((1-\lambda+\lambda\eta)^2+\lambda^2\omega\big)
\Big)\frac{1}{n}\sum_{i=1}^n \sqnorm{\nabla f_i(x^{t})-h_i^{t}}\\
&\quad+ \left(\frac{ L}{2}-\frac{1}{2\gamma}+\frac{\gamma L^2}{2}+\frac{\gamma}{2\theta}(1+s^{-1})\tilde{L}^2\right)\!\Exp{\sqnorm{x^{t+1}-x^t}}\\
&= \Big(1+\frac{\gamma\mu}{2}\Big)^{-1}\big(f(x^{t})+R(x^{t}) -f^\star - R^\star\big) \\
&\quad+\frac{ \gamma }{2\theta}\Big( \theta r_{\mathrm{av}}
+(1+s)r
\Big)\frac{1}{n}\sum_{i=1}^n \sqnorm{\nabla f_i(x^{t})-h_i^{t}}\\
&\quad+ \left(\frac{ L}{2}-\frac{1}{2\gamma}+\frac{\gamma L^2}{2}+\frac{\gamma}{2\theta}(1+s^{-1})\tilde{L}^2\right)\!\Exp{\sqnorm{x^{t+1}-x^t}}\\
&=\Big(1+\frac{\gamma\mu}{2}\Big)^{-1}\big(f(x^{t})+R(x^{t}) -f^\star - R^\star\big) +\frac{ \gamma }{2\theta} (1+s)^2 
\frac{r}{n}\sum_{i=1}^n \sqnorm{\nabla f_i(x^{t})-h_i^{t}}\\
&\quad+ \left(\frac{ L}{2}-\frac{1}{2\gamma}+\frac{\gamma L^2}{2}+\frac{\gamma}{2s^2}\frac{r_{\mathrm{av}}}{r}\tilde{L}^2\right)\!\Exp{\sqnorm{x^{t+1}-x^t}}.
\end{align*}
We now choose $\gamma$ small enough so that 
 \begin{equation*}
L-\frac{1}{\gamma}+\gamma L^2+\frac{\gamma}{s^2}\frac{r_{\mathrm{av}}}{r}\tilde{L}^2 \leq 0.
\end{equation*}
If we assume $\gamma\leq \frac{1}{L}$, a sufficient condition is 
 \begin{equation}
2L-\frac{1}{\gamma}+\frac{\gamma}{s^2}\frac{r_{\mathrm{av}}}{r}\tilde{L}^2 \leq 0.\label{eqgreg3}
\end{equation}
A sufficient condition for \eqref{eqgreg3} to hold is \citep[Lemma 5]{ric21}:
\begin{equation}
0<\gamma \leq \frac{1}{2L+\tilde{L}\sqrt{\frac{r_{\mathrm{av}}}{r}}\frac{1}{s}}.\label{eqgamfek2}
\end{equation}
Then, assuming that \eqref{eqgamfek2} holds, we have, for every $t\geq 0$, conditionally on $x^t$, $h^t$ and $(h_i^t)_{i=1}^n$,
\begin{align*}
\Exp{\Psi^{t+1}} &\leq \Big(1+\frac{\gamma\mu}{2}\Big)^{-1} \big(f(x^t)  +R(x^t)-f^\star-R^\star \big) +\frac{ \gamma }{2\theta} (1+s)^2 
 \frac{r}{n}\sum_{i=1}^n \sqnorm{\nabla f_i(x^t)-h_i^{t}}\\
&\leq \max\Big({\textstyle\frac{1}{1+\frac{1}{2}\gamma\mu}},(1+s)^2 r\Big)  \Psi^t.
\end{align*}
We set $s=s^\star$ and, accordingly, $\theta=\theta^\star$, so that  $(1+s^\star)^2 r = \frac{r+1}{2}<1$. Then, 
for every $t\geq 0$, conditionally on $x^t$, $h^t$ and $(h_i^t)_{i=1}^n$,
\begin{align*}
\Exp{\Psi^{t+1}} &\leq \max\left({\frac{1}{1+\frac{1}{2}\gamma\mu}},\frac{r+1}{2}\right)  \Psi^t.
\end{align*}
Unrolling the recursion using the tower rule yields \eqref{eqsdgerg2}.

\section{Proof of Theorem~\ref{thm:noncvx}}
    Let $\theta>0$; its value will be set to the prescribed value later on. We introduce the Lyapunov function, for every $t\geq 0$,
   \begin{equation*}
   \Psi^t \eqdef f(x^t)-f^{\inf} + \frac{\gamma}{2\theta}  \frac{1}{n}\sum_{i=1}^n \sqnorm{\nabla f_i(x^t)-h_i^{t}}.
   \end{equation*}
 According to \citep[Lemma 4]{ric21}, we have, for every $t\geq 0$,
   \begin{align}
   f(x^{t+1}) -f^{\inf} &\leq f(x^t)  -f^{\inf} -\frac{\gamma}{2} \sqnorm{\nabla f(x^t)} +\frac{ \gamma }{2}\sqnorm{g^{t+1}-\nabla f(x^t)} + \left(\frac{ L}{2}-\frac{1}{2\gamma}\right)\sqnorm{x^{t+1}-x^t}.\notag
   \end{align}
   As shown in the proof of Theorem~\ref{theo1}, we have, conditionally on $x^t$, $h^t$ and $(h_i^t)_{i=1}^n$,
\begin{align*}
\Exp{\sqnorm{g^{t+1}-\nabla f(x^t)}} &\leq 
\left((1-\nu+\nu\eta)^2+\nu^2\oma\right)\frac{1}{n}\sum_{i=1}^n \sqnorm{\nabla f_i(x^t)-h_i^{t}}.
\end{align*}
   As for the control variates $h_i^t$, as shown in the proof of Theorem \ref{theo1}, we have, conditionally on $x^t$, $h^t$ and $(h_i^t)_{i=1}^n$, 
   \begin{align*}
   \Exp{\frac{1}{n}\sum_{i=1}^n \sqnorm{\nabla f_i(x^{t+1})-h_i^{t+1}}}&
   \leq (1+s)\big((1-\lambda+\lambda\eta)^2+\lambda^2\omega\big) \frac{1}{n}\sum_{i=1}^n \sqnorm{\nabla f_i(x^t)-h_i^{t}}\\
   &\quad+(1+s^{-1})\tilde{L}^2\Exp
   {\sqnorm{x^{t+1}-x^{t}}}.
   \end{align*}

    Hence, for every $t\geq 0$, conditionally on $x^t$, $h^t$ and $(h_i^t)_{i=1}^n$, we have
   \begin{align}
   \Exp{\Psi^{t+1}} &\leq f(x^t) - f^{\inf} - \frac{\gamma}{2} \sqnorm{\nabla f(x^t)}\notag \\
   &\quad +\frac{ \gamma }{2\theta}\Big( \theta\big((1-\nu+\nu\eta)^2+\nu^2\oma\big)\notag+(1+s)\big((1-\lambda+\lambda\eta)^2+\lambda^2\omega\big)
   \Big) \frac{1}{n}\sum_{i=1}^n \sqnorm{\nabla f_i(x^t)-h_i^{t}}\notag\\
   &\quad+ \left(\frac{ L}{2}-\frac{1}{2\gamma}+\frac{\gamma}{2\theta}(1+s^{-1})\tilde{L}^2\right)\!\Exp{\sqnorm{x^{t+1}-x^t}}.\label{eqq1z}
   \end{align}
   Let $r \eqdef (1 - \lambda + \lambda \eta)^2 + \lambda^2 \omega, r_{\mathrm{av}}\eqdef (1 - \nu + \nu\eta)^2 + \nu^2 \oma$. Set $\theta \eqdef s(1+s)\frac{r}{r_{\mathrm{av}}}$. We can rewrite \eqref{eqq1z} as:
   \begin{align*}
   \Exp{\Psi^{t+1}} &\leq f(x^t) - f^{\inf} - \frac{\gamma}{2} \sqnorm{\nabla f(x^t)} +\frac{ \gamma }{2\theta}\Big( \theta r_{\mathrm{av}}
   +(1+s)r
   \Big) \frac{1}{n}\sum_{i=1}^n \sqnorm{\nabla f_i(x^t)-h_i^{t}}\\
   &\quad+ \left(\frac{ L}{2}-\frac{1}{2\gamma}+\frac{\gamma}{2\theta}(1+s^{-1})\tilde{L}^2\right)\!\Exp{\sqnorm{x^{t+1}-x^t}}\\
   &=f(x^t) - f^{\inf} - \frac{\gamma}{2} \sqnorm{\nabla f(x^t)} +\frac{ \gamma }{2\theta} (1+s)^2 
   \frac{r}{n}\sum_{i=1}^n \sqnorm{\nabla f_i(x^t)-h_i^{t}}\\
   &\quad+ \left(\frac{ L}{2}-\frac{1}{2\gamma}+\frac{\gamma}{2s^2}\frac{r_{\mathrm{av}}}{r}\tilde{L}^2\right)\!\Exp{\sqnorm{x^{t+1}-x^t}}.
   \end{align*}
   We now choose $\gamma$ small enough so that 
   \begin{equation}
   L-\frac{1}{\gamma}+\frac{\gamma}{s^2}\frac{r_{\mathrm{av}}}{r}\tilde{L}^2 \leq 0.\label{eqgreg5}
   \end{equation}
   A sufficient condition for \eqref{eqgreg5} to hold is \citep[Lemma 5]{ric21}:
   \begin{equation}
   0<\gamma \leq \frac{1}{L+\tilde{L}\sqrt{\frac{r_{\mathrm{av}}}{r}}\frac{1}{s}}.\label{eqgamfek5}
   \end{equation}
   Then, assuming that \eqref{eqgamfek5} holds, we have, for every $t\geq 0$, conditionally on $x^t$, $h^t$ and $(h_i^t)_{i=1}^n$,
   \begin{align*}
   \Exp{\Psi^{t+1}} &\leq f(x^t) - f^{\inf} - \frac{\gamma}{2} \sqnorm{\nabla f(x^t)} +\frac{ \gamma }{2\theta} (1+s)^2 
   \frac{r}{n}\sum_{i=1}^n \sqnorm{\nabla f_i(x^t)-h_i^{t}}.
   \end{align*}
 We have chosen $s$ so that $(1+s)^2 r = 1$. 
 Hence, using the tower rule, we have, for every $t\geq 0$,
   \begin{align*}
      \Exp{\Psi^{t+1}} \leq \Exp{\Psi^t} - \frac{\gamma}{2}\Exp{\sqnorm{\nabla f(x^t)}}.
   \end{align*}
   Let $T\geq 1$. By summing up the inequalities for $t=0, \cdots, T-1$, we get 
   \begin{align*}
      0 \leq \Exp{\Psi^T} \leq \Psi^0 - \frac{\gamma}{2} \sum_{t=0}^{T-1} \Exp{\sqnorm{\nabla f(x^t)}}.
   \end{align*}
   Multiplying both sides by $\frac{2}{\gamma T}$ and rearranging the terms, we get
   \begin{align*}
      \frac{1}{T}\sum_{t=0}^{T-1} \Exp{\sqnorm{\nabla f(x^t)}} \leq \frac{2}{\gamma T} \Psi^0,
   \end{align*}
   where the left hand side can be interpreted as $\mathbb{E}\left[\left\|\nabla f(\hat{x}^{T})\right\|^{2}\right]$, where $\hat{x}^{T}$ is chosen from $x^{0}, x^{1}, \ldots, x^{T-1}$ uniformly at random.

%% file: Appendix_C3_Scafflix.tex
\chapter{Appendix to Chapter \ref{chapter_scafflix}}
\label{chapter_appendix_scafflix}
\thispagestyle{empty}

\section{Proposed \algname{i-Scaffnew} algorithm}
We consider solving (\ref{eq:ERM}) with the proposed \algname{i-Scaffnew} algorithm, shown as Algorithm~\ref{alg1} (applying  \algname{i-Scaffnew} to (\ref{eq:FLIX}) yields \algname{Scafflix}, as we discuss subsequently in Section~\ref{secalg3}). 

\begin{algorithm}[t]
	\caption{\algname{i-Scaffnew} for (\ref{eq:ERM})}
	\label{alg1}
	\begin{algorithmic}[1]
		\STATE \textbf{input:}  stepsizes $\gamma_1>0,\ldots,\gamma_n>0$; probability $p \in (0,1]$; initial estimates $x_1^0,\ldots,x_n^0 \in \mathbb{R}^d$ and ${\red h_1^0, \ldots, h_n^0} \in \mathbb{R}^d$ such that $\sum_{i=1}^n {\red h_i^0}=0$.
		\STATE at the server, $\gamma \eqdef \left(\frac{1}{n}\sum_{i=1}^n \gamma_i^{-1}\right)^{-1}$ 
		\hfill $\diamond$ {\small\color{gray} $\gamma$ is used by the server for Step 9}
		\FOR{$t=0,1,\ldots$}
		\STATE flip a coin $\theta^t \eqdef \{1$ with probability $p$, 0 otherwise$\}$
		\FOR{$i=1,\ldots,n$, at clients in parallel,}
		\STATE compute an estimate $g_i^t$  of $\nabla f_i(x_i^t)$
\STATE $\hat{x}_i^t\eqdef x_i^t -\gamma_i \big(g_i^t - {\red h_i^t}\big)$
\hfill $\diamond$ {\small\color{gray} local SGD step}
\IF{$\theta^t=1$}
\STATE send $\frac{1}{\gamma_i}\hat{x}_i^t$ to the server, which aggregates $\bar{x}^t\eqdef \frac{\gamma}{n}\sum_{j=1}^n  \frac{1}{\gamma_i}\hat{x}_{j}^t $ and broadcasts it to all clients \hfill $\diamond$ {\small\color{gray} communication, but only with small probability $p$}
\STATE $x_i^{t+1}\eqdef \bar{x}^{t}$
\STATE ${\red h_i^{t+1}}\eqdef {\red h_i^t} + \frac{p}{\gamma_i}\big(\bar{x}^{t}-\hat{x}_i^t\big)$\hfill $\diamond$ {\small\color{gray}update of the local control variate $\red h_i^t$}
\ELSE
\STATE $x_i^{t+1}\eqdef \hat{x}_i^t$
\STATE ${\red h_i^{t+1}}\eqdef {\red h_i^t}$
\ENDIF
\ENDFOR
		\ENDFOR
	\end{algorithmic}
\end{algorithm}

\begin{theorem}[fast linear convergence]\label{scafflix_theo1}
In (\ref{eq:ERM}) and \algname{i-Scaffnew}, suppose that Assumptions~\ref{ass:convex_smooth}, \ref{ass:unbiasedness}, \ref{ass:expected_smoothness} hold and that for every $i\in[n]$, $0<\gamma_i \leq \frac{1}{ A_i}$.
For every $t\geq 0$, define the Lyapunov function
\begin{equation}
\Psi^{t}\eqdef  \sum_{i=1}^n \frac{1}{\gamma_i} \sqnorm{x_i^t-x^\star}+ \frac{1}{p^2}\sum_{i=1}^n \gamma_i \sqnorm{h_i^t-\nabla f_i(x^\star)}
.
\end{equation}
Then 
\algname{i-Scaffnew}
converges linearly:  for every $t\geq 0$, 
\begin{equation}
\Exp{\Psi^{t}}\leq (1-\zeta)^t \Psi^0 + \frac{1}{\zeta} \sum_{i=1}^n \gamma_i C_i,
\end{equation}
where 
\begin{equation}
\zeta = \min\left(\min_{i\in[n]} \gamma_i\mu_i,p^2\right).\label{eqrate2j}
\end{equation}
\end{theorem}
  
\begin{proof}  
%
%
%
To simplify the analysis of \algname{i-Scaffnew}, we introduce vector notations: the problem (\ref{eq:ERM}) can be written as
\begin{equation}
\mathrm{find}\ \mathbf{x}^\star =\argmin_{\mathbf{x}\in\mathcal{X}}\  \mathbf{f}(\mathbf{x})\quad\mbox{s.t.}\quad W\mathbf{x}=0,\label{eqpro2}
\end{equation}
where $\mathcal{X}\eqdef\mathbb{R}^{d\times n}$, an element 
$\mathbf{x}=(x_i)_{i=1}^n \in \mathcal{X}$ is a collection of vectors $x_i\in \mathbb{R}^d$, $\mathbf{f}:\mathbf{x}\in \mathcal{X}\mapsto \sum_{i=1}^n f_i(x_i)$, the linear operator $W:\mathcal{X}\rightarrow \mathcal{X}$ maps $\mathbf{x}=(x_i)_{i=1}^n $ to $(x_i-\frac{1}{n}\sum_{j=1}^n \frac{\gamma}{\gamma_j}x_j)_{i=1}^n$, for given values $\gamma_1>0,\ldots,\gamma_n>0$ and their harmonic mean $\gamma = \left(\frac{1}{n}\sum_{i=1}^n \gamma_i^{-1}\right)^{-1}$.
The constraint $W\mathbf{x}=0$ means that $\mathbf{x}$ minus its weighted average is zero; that is, $\mathbf{x}$ has identical components $x_1 = \cdots = x_n$. Thus, \eqref{eqpro2} is indeed equivalent to (\ref{eq:ERM}). $\mathbf{x}^\star\eqdef (x^\star)_{i=1}^n \in \mathcal{X}$ is the unique solution to  \eqref{eqpro2}, where $x^\star$ is the unique solution to (\ref{eq:ERM}). 

Moreover, we introduce the weighted inner product in $\mathcal{X}$: $(\mathbf{x},\mathbf{y})\mapsto \langle \mathbf{x},\mathbf{y}\rangle_{\boldsymbol{\gamma}}\eqdef \sum_{i=1}^n \frac{1}{\gamma_i} \langle x_i,y_i\rangle$. Then, the orthogonal projector $P$ onto the hyperspace $\{\mathbf{y} \in \mathcal{X}\ :\ y_1=\cdots=y_n\}$, with respect to this weighted inner product, is  $P:\mathbf{x}\in\mathcal{X} \mapsto \mathbf{\bar{x}}=(\bar{x})_{i=1}^n$ with $\bar{x}=\frac{\gamma}{n}\sum_{i=1}^n \frac{1}{\gamma_i}  x_i$ (because $\bar{x}$ minimizes $\sqnorm{\mathbf{\bar{x}}-\mathbf{x}}_{\boldsymbol{\gamma}}$, so that $\frac{1}{n}\sum_{i=1}^n \frac{1}{\gamma_i}(\bar{x} - x_i) = 0$). Thus, $P$, as well as $W=\mathrm{Id}-P$, where $\mathrm{Id}$ denotes the identity, are self-adjoint and positive linear operators with respect to the weighted inner product. Moreover, for every $\mathbf{x}\in \mathcal{X}$,
\begin{equation*}
\sqnorm{\mathbf{x}}_{\boldsymbol{\gamma}}=\sqnorm{P\mathbf{x}}_{\boldsymbol{\gamma}}+\sqnorm{W\mathbf{x}}_{\boldsymbol{\gamma}}
=\sqnorm{\mathbf{\bar{x}}}_{\boldsymbol{\gamma}}+\sqnorm{W\mathbf{x}}_{\boldsymbol{\gamma}}
=\frac{n}{\gamma} \sqnorm{\bar{x}}+\sqnorm{W\mathbf{x}}_{\boldsymbol{\gamma}},
\end{equation*}
where $\mathbf{\bar{x}}=(\bar{x})_{i=1}^n$ and $\bar{x}=\frac{\gamma}{n}\sum_{i=1}^n \frac{1}{\gamma_i}  x_i$.

Let us introduce  further vector notations for the variables of \algname{i-Scaffnew}: for every $t\geq 0$, we define the \emph{scaled} concatenated control variate $\mathbf{h}^t\eqdef (\gamma_i h_i^t)_{i=1}^n$, $\mathbf{h}^\star\eqdef (\gamma_i h_i^\star)_{i=1}^n$, with $h_i^\star \eqdef \nabla f_i(x^\star)$, $\mathbf{\bar{x}}^t\eqdef (\bar{x}^t)_{i=1}^n$,  
$\mathbf{w}^t\eqdef (w_i^t)_{i=1}^n$, with $w_i^t\eqdef x_i^t-\gamma_i g_i^t$, $\mathbf{w}^\star\eqdef (w_i^\star)_{i=1}^n$, with $w_i^\star\eqdef x_i^\star-\gamma_i \nabla f_i(x_i^\star)$,
$\mathbf{\hat{h}}^{t}\eqdef  \mathbf{h}^t - p W  \mathbf{\hat{x}}^{t}$.
Finally, 
we denote by $\mathcal{F}_0^t$ the $\sigma$-algebra generated by the collection of $\mathcal{X}$-valued random variables $\mathbf{x}^0,\mathbf{h}^0,\ldots, \mathbf{x}^t,\mathbf{h}^t$ 
and by $\mathcal{F}^t$ the $\sigma$-algebra generated by these variables, as well as the stochastic gradients $g_i^t$. 

We can then rewrite the iteration of \algname{i-Scaffnew} as:
	\noindent	\begin{algorithmic}
			\STATE $\mathbf{\hat{x}}^{t} \eqdef  \mathbf{w}^t +  \mathbf{h}^t$
			\IF{$\theta^t=1$}
			\STATE $\mathbf{x}^{t+1}\eqdef \mathbf{\bar{x}}^{t}$
			\STATE $\mathbf{h}^{t+1}\eqdef \mathbf{h}^t -pW\mathbf{\hat{x}}^{t}$
			\ELSE
			\STATE $\mathbf{x}^{t+1}\eqdef \mathbf{\hat{x}}^{t}$
			\STATE $\mathbf{h}^{t+1}\eqdef \mathbf{h}^t$
			\ENDIF
		\end{algorithmic}\medskip

We suppose that $\sum_{i=1}^n h^0_i = 0$. Then, it follows from the definition of $\bar{x}^t$ that $\frac{\gamma}{n}\sum_{j=1}^n  \frac{1}{\gamma_i}(\bar{x}^t-\hat{x}_{j}^t) = 0$, so that 
 for every $t\geq 0$, 
$\sum_{i=1}^n  h^t_i = 0$; that is, $W \mathbf{h}^t=\mathbf{h}^t$.\medskip

Let $t\geq 0$. We have 
\begin{align*}
\Exp{\sqnorm{\mathbf{x}^{t+1}-\mathbf{x}^\star}_{\boldsymbol{\gamma}}\;|\;\mathcal{F}^t}&=
p \sqnorm{\mathbf{\bar{x}}^{t}-\mathbf{x}^\star}_{\boldsymbol{\gamma}}+(1-p)\sqnorm{\mathbf{\hat{x}}^{t}-\mathbf{x}^\star}_{\boldsymbol{\gamma}},
\end{align*}
with 
\begin{equation*}
\sqnorm{\mathbf{\bar{x}}^{t}-\mathbf{x}^\star}_{\boldsymbol{\gamma}}=\sqnorm{\mathbf{\hat{x}}^{t}-\mathbf{x}^\star}_{\boldsymbol{\gamma}}-\sqnorm{W  \mathbf{\hat{x}}^{t}}_{\boldsymbol{\gamma}}.
\end{equation*}
Moreover,
\begin{align*}
\sqnorm{ \mathbf{\hat{x}}^{t}-\mathbf{x}^\star}_{\boldsymbol{\gamma}}
&= \sqnorm{\mathbf{w}^t-\mathbf{w}^\star}_{\boldsymbol{\gamma}}+\sqnorm{ \mathbf{h}^t- \mathbf{h}^\star}_{\boldsymbol{\gamma}}
 +2\langle \mathbf{w}^t-\mathbf{w}^\star, \mathbf{h}^t- \mathbf{h}^\star\rangle_{\boldsymbol{\gamma}}
\\
&= \sqnorm{\mathbf{w}^t-\mathbf{w}^\star}_{\boldsymbol{\gamma}}-\sqnorm{ \mathbf{h}^t- \mathbf{h}^\star}_{\boldsymbol{\gamma}}
 +2\langle\mathbf{\hat{x}}^{t}-\mathbf{x}^\star, \mathbf{h}^t- \mathbf{h}^\star\rangle_{\boldsymbol{\gamma}}\\
 &= \sqnorm{\mathbf{w}^t-\mathbf{w}^\star}_{\boldsymbol{\gamma}}-\sqnorm{ \mathbf{h}^t- \mathbf{h}^\star}_{\boldsymbol{\gamma}}
 +2\langle\mathbf{\hat{x}}^{t}-\mathbf{x}^\star, \mathbf{\hat{h}}^{t}- \mathbf{h}^\star\rangle_{\boldsymbol{\gamma}}
  -2\langle\mathbf{\hat{x}}^{t}-\mathbf{x}^\star, \mathbf{\hat{h}}^{t}- \mathbf{h}^t\rangle_{\boldsymbol{\gamma}}\\
   &= \sqnorm{\mathbf{w}^t-\mathbf{w}^\star}_{\boldsymbol{\gamma}}-\sqnorm{ \mathbf{h}^t- \mathbf{h}^\star}_{\boldsymbol{\gamma}}
 +2\langle\mathbf{\hat{x}}^{t}-\mathbf{x}^\star, \mathbf{\hat{h}}^{t}- \mathbf{h}^\star\rangle_{\boldsymbol{\gamma}}
  +2p \langle\mathbf{\hat{x}}^{t}-\mathbf{x}^\star, W\mathbf{\hat{x}}^t\rangle_{\boldsymbol{\gamma}}\\
    &= \sqnorm{\mathbf{w}^t-\mathbf{w}^\star}_{\boldsymbol{\gamma}}-\sqnorm{ \mathbf{h}^t- \mathbf{h}^\star}_{\boldsymbol{\gamma}}
 +2\langle\mathbf{\hat{x}}^{t}-\mathbf{x}^\star, \mathbf{\hat{h}}^{t}- \mathbf{h}^\star\rangle_{\boldsymbol{\gamma}}
  +2p \sqnorm{W\mathbf{\hat{x}}^t}_{\boldsymbol{\gamma}}.
 \end{align*}
 Hence,
\begin{align*}
\Exp{\sqnorm{\mathbf{x}^{t+1}-\mathbf{x}^\star}_{\boldsymbol{\gamma}}\;|\;\mathcal{F}^t}&=\sqnorm{\mathbf{\hat{x}}^{t}-\mathbf{x}^\star}_{\boldsymbol{\gamma}}-p\sqnorm{W  \mathbf{\hat{x}}^{t}}_{\boldsymbol{\gamma}}\\
&=\sqnorm{\mathbf{w}^t-\mathbf{w}^\star}_{\boldsymbol{\gamma}}-\sqnorm{ \mathbf{h}^t- \mathbf{h}^\star}_{\boldsymbol{\gamma}}
 +2\langle\mathbf{\hat{x}}^{t}-\mathbf{x}^\star, \mathbf{\hat{h}}^{t}- \mathbf{h}^\star\rangle_{\boldsymbol{\gamma}}+p\sqnorm{W\mathbf{\hat{x}}^t}_{\boldsymbol{\gamma}}.
\end{align*}

On the other hand, we have
\begin{align*}
\Exp{\sqnorm{\mathbf{h}^{t+1}- \mathbf{h}^\star}_{\boldsymbol{\gamma}}\;|\;\mathcal{F}^t}
&=p\sqnorm{\mathbf{\hat{h}}^{t} - \mathbf{h}^\star}_{\boldsymbol{\gamma}}+(1-p)\sqnorm{\mathbf{h}^t - \mathbf{h}^\star}_{\boldsymbol{\gamma}}
\end{align*}
and
\begin{align*}
\sqnorm{\mathbf{\hat{h}}^{t}- \mathbf{h}^\star}_{\boldsymbol{\gamma}}&=\sqnorm{( \mathbf{h}^t- \mathbf{h}^\star)+(\mathbf{\hat{h}}^{t}- \mathbf{h}^t)}_{\boldsymbol{\gamma}}\\
&=\sqnorm{ \mathbf{h}^t- \mathbf{h}^\star}_{\boldsymbol{\gamma}}+\sqnorm{\mathbf{\hat{h}}^{t}- \mathbf{h}^t}_{\boldsymbol{\gamma}}+2\langle  \mathbf{h}^t- \mathbf{h}^\star,\mathbf{\hat{h}}^{t}- \mathbf{h}^t\rangle_{\boldsymbol{\gamma}}\\
&=\sqnorm{ \mathbf{h}^t- \mathbf{h}^\star}_{\boldsymbol{\gamma}}- \sqnorm{\mathbf{\hat{h}}^{t}- \mathbf{h}^t}_{\boldsymbol{\gamma}}+2 \langle \mathbf{\hat{h}}^{t}- \mathbf{h}^\star,\mathbf{\hat{h}}^{t}- \mathbf{h}^t\rangle_{\boldsymbol{\gamma}} \\
&=\sqnorm{ \mathbf{h}^t- \mathbf{h}^\star}_{\boldsymbol{\gamma}} - \sqnorm{\mathbf{\hat{h}}^{t}- \mathbf{h}^t}_{\boldsymbol{\gamma}}-2 p\langle \mathbf{\hat{h}}^{t}- \mathbf{h}^\star, W( \mathbf{\hat{x}}^{t}-\mathbf{x}^\star)\rangle_{\boldsymbol{\gamma}}\\
&=\sqnorm{ \mathbf{h}^t- \mathbf{h}^\star}_{\boldsymbol{\gamma}} - p^2\sqnorm{W\mathbf{\hat{x}}^t}_{\boldsymbol{\gamma}} -2p\langle W(\mathbf{\hat{h}}^{t}- \mathbf{h}^\star),  \mathbf{\hat{x}}^{t}-\mathbf{x}^\star\rangle_{\boldsymbol{\gamma}}\\
&=\sqnorm{ \mathbf{h}^t- \mathbf{h}^\star}_{\boldsymbol{\gamma}} - p^2\sqnorm{W\mathbf{\hat{x}}^t}_{\boldsymbol{\gamma}} -2 p\langle \mathbf{\hat{h}}^{t}- \mathbf{h}^\star,  \mathbf{\hat{x}}^{t}-\mathbf{x}^\star\rangle_{\boldsymbol{\gamma}}.
\end{align*}

Hence, 
\begin{align}
\Exp{\sqnorm{\mathbf{x}^{t+1}-\mathbf{x}^\star}_{\boldsymbol{\gamma}}\;|\;\mathcal{F}^t}&+\frac{1}{p^2}\Exp{\sqnorm{\mathbf{h}^{t+1}- \mathbf{h}^\star}_{\boldsymbol{\gamma}}\;|\;\mathcal{F}^t}\notag\\
&=\sqnorm{\mathbf{w}^t-\mathbf{w}^\star}_{\boldsymbol{\gamma}}-\sqnorm{ \mathbf{h}^t- \mathbf{h}^\star}_{\boldsymbol{\gamma}}
 +2\langle\mathbf{\hat{x}}^{t}-\mathbf{x}^\star, \mathbf{\hat{h}}^{t}- \mathbf{h}^\star\rangle_{\boldsymbol{\gamma}}+p\sqnorm{W\mathbf{\hat{x}}^t}_{\boldsymbol{\gamma}}\notag\\
&\quad+\frac{1}{p^2}\sqnorm{ \mathbf{h}^t- \mathbf{h}^\star}_{\boldsymbol{\gamma}}- p\sqnorm{W\mathbf{\hat{x}}^t}_{\boldsymbol{\gamma}} -2 \langle \mathbf{\hat{h}}^{t}- \mathbf{h}^\star,  \mathbf{\hat{x}}^{t}-\mathbf{x}^\star\rangle_{\boldsymbol{\gamma}} \notag\\
&=  \sqnorm{\mathbf{w}^t-\mathbf{w}^\star}_{\boldsymbol{\gamma}}+\frac{1}{p^2}\left(1-p^2\right)\sqnorm{ \mathbf{h}^t- \mathbf{h}^\star}_{\boldsymbol{\gamma}}.\label{eq55}
\end{align}

Moreover, for every $i\in[n]$,
\begin{align*}
\sqnorm{w_i^t - w_i^\star}&=\sqnorm{x_i^t - x^\star -\gamma_i \big(g_i^t -\nabla f_i(x^\star)\big)}\\
&=\sqnorm{x_i^t - x^\star } - 2\gamma_i \langle x_i^t - x^\star, g_i^t -\nabla f_i(x^\star)\rangle + \gamma_i^2 \sqnorm{g_i^t -\nabla f_i(x^\star)},
\end{align*}
and, by unbiasedness of $g_i^t$ and Assumption~\ref{ass:expected_smoothness_n},
\begin{align*}
\Exp{\sqnorm{w_i^t - w_i^\star}\;|\;\mathcal{F}_0^t}&=\sqnorm{x_i^t - x^\star } - 2\gamma_i \langle x_i^t - x^\star, \nabla f_i(x_i^t) -\nabla f_i(x^\star)\rangle \\
&\quad + \gamma_i^2\Exp{ \sqnorm{g_i^t -\nabla f_i(x^\star)}\;|\;\mathcal{F}^t}\\
&\leq \sqnorm{x_i^t - x^\star } - 2\gamma_i \langle x_i^t - x^\star, \nabla f_i(x_i^t) -\nabla f_i(x^\star)\rangle + 2\gamma_i^2 A_i D_{f_i}(x_i^t, \xstar) \\
&\quad+ \gamma_i^2  C_i.
\end{align*}
It is easy to see that $\langle x_i^t - x^\star, \nabla f_i(x_i^t) -\nabla f_i(x^\star)\rangle = D_{f_i}(x_i^t, \xstar) + D_{f_i}(\xstar,x_i^t)$. This yields \begin{align*}
\Exp{\sqnorm{w_i^t - w_i^\star}\;|\;\mathcal{F}_0^t}
&\leq \sqnorm{x_i^t - x^\star } -2\gamma_i D_{f_i}(\xstar,x_i^t) - 2\gamma_i D_{f_i}(x_i^t, \xstar) + 2\gamma_i^2  A_i D_{f_i}(x_i^t, \xstar) \\
&\quad+ \gamma_i^2  C_i.
\end{align*}
In addition, the strong convexity of $f_i$ implies that $D_{f_i}(\xstar,x_i^t)\geq \frac{\mu_i}{2}\sqnorm{x_i^t - x^\star }$, so that 
\begin{align*}
\Exp{\sqnorm{w_i^t - w_i^\star}\;|\;\mathcal{F}_0^t}
&\leq (1-\gamma_i\mu_i)\sqnorm{x_i^t - x^\star }  - 2\gamma_i (1-\gamma_i A_i) D_{f_i}(x_i^t, \xstar) + \gamma_i^2  C_i,
\end{align*}
and since we have supposed $\gamma_i\leq \frac{1}{A_i}$,
\begin{align*}
\Exp{\sqnorm{w_i^t - w_i^\star}\;|\;\mathcal{F}_0^t}
&\leq (1-\gamma_i\mu_i)\sqnorm{x_i^t - x^\star } + \gamma_i^2  C_i.
\end{align*}
Therefore,
\begin{equation*}
\Exp{\sqnorm{\mathbf{w}^t-\mathbf{w}^\star}_{\boldsymbol{\gamma}}\;|\;\mathcal{F}_0^t}\leq \max_{i\in[n]}(1-\gamma_i\mu_i)\sqnorm{\mathbf{x}^t-\mathbf{x}^\star}_{\boldsymbol{\gamma}}+\sum_{i=1}^n \gamma_i C_i
\end{equation*}
and
\begin{align}
\Exp{\Psi^{t+1}\;|\;\mathcal{F}_0^t}&=\Exp{\sqnorm{\mathbf{x}^{t+1}-\mathbf{x}^\star}_{\boldsymbol{\gamma}}\;|\;\mathcal{F}^t_0}+\frac{1}{p^2}\Exp{\sqnorm{\mathbf{h}^{t+1}- \mathbf{h}^\star}_{\boldsymbol{\gamma}}\;|\;\mathcal{F}^t_0}\notag\\
&\leq  \max_{i\in[n]}(1-\gamma_i\mu_i)\sqnorm{\mathbf{x}^t-\mathbf{x}^\star}_{\boldsymbol{\gamma}}+\frac{1}{p^2}\left(1-p^2\right)\sqnorm{ \mathbf{h}^t- \mathbf{h}^\star}_{\boldsymbol{\gamma}}+\sum_{i=1}^n \gamma_i C_i\notag\\
&\leq  (1-\zeta)\left(\sqnorm{\mathbf{x}^t-\mathbf{x}^\star}_{\boldsymbol{\gamma}}+\frac{1}{p^2}\sqnorm{ \mathbf{h}^t- \mathbf{h}^\star}_{\boldsymbol{\gamma}}\right)+\sum_{i=1}^n \gamma_i C_i\notag\\
&=(1-\zeta) \Psi^t +\sum_{i=1}^n \gamma_i C_i,\label{eqrec2b}
\end{align}
where
\begin{equation*}
\zeta = \min\left(\min_{i\in[n]} \gamma_i\mu_i,p^2\right).
\end{equation*}
%
Using the tower rule, we can unroll the recursion in \eqref{eqrec2b} to obtain the unconditional expectation of $\Psi^{t+1}$. 
\end{proof}

\section{From \algname{i-Scaffnew} to \algname{Scafflix}}\label{secalg3}

We suppose that Assumptions~\ref{ass:convex_smooth}, \ref{ass:unbiasedness}, \ref{ass:expected_smoothness} hold. 
We define for every $i\in[n]$
the function
$\tilde{f}_i:x\in\mathbb{R}^d\mapsto f_i\big(\alpha_i x+ (1 - \alpha_i)\xstar_i\big)$. Thus, (\ref{eq:FLIX}) takes the form of  (\ref{eq:ERM}) with $f_i$ replaced by $\tilde{f}_i$. 

We want to derive \algname{Scafflix} from \algname{i-Scaffnew} applied to (\ref{eq:ERM}) with $f_i$ replaced by $\tilde{f}_i$. For this, we first observe that for every $i\in[n]$, $\tilde{f}_i$ is $\alpha_i^2 L_i$-smooth and $\alpha_i^2 \mu_i$-strongly convex. This follows easily from the fact that 
$\nabla \tilde{f}_i (x) = \alpha_i \nabla f_i\big(\alpha_i x+ (1 - \alpha_i)\xstar_i\big)$.

Second, for every $t\geq 0$ and $i\in[n]$, $g_i^t$ is an unbiased estimate of $\nabla f_i(\tilde{x}_i^t)=\alpha_i^{-1}\nabla \tilde{f}_i(x_i^t)$. Therefore, $\alpha_i g_i^t$ is an unbiased estimate of $\nabla \tilde{f}_i(x_i^t)$ satisfying
    \begin{equation*}
        \Exp{\sqn{\alpha_i g_i^{t}- \nabla \tilde{f}_i(x^\star)}\;|\; x_i^{t}} 
        = \alpha_i^2\Exp{\sqn{g_i^{t}- \nabla f_i(\tilde{x}_i^\star)}\;|\; \tilde{x}_i^{t}} 
        \leq 2\alpha_i^2 A_i D_{f_i}(\tilde{x}_i^{t}, \tilde{x}_i^\star) + \alpha_i^2 C_i.
    \end{equation*}
Moreover, 
        \begin{align*}
D_{f_i}(\tilde{x}_i^{t}, \tilde{x}_i^\star) &= f_i(\tilde{x}_i^t)-f_i(\tilde{x}_i^\star)-\langle \nabla f_i(\tilde{x}_i^\star),\tilde{x}_i^t-\tilde{x}_i^\star\rangle
\\
&= \tilde{f}_i(x_i^t) - \tilde{f}_i(x^\star) -\langle \alpha_i^{-1}\nabla \tilde{f}_i(x^\star),\alpha_i(x_i^t-x^\star)\rangle\\
&= \tilde{f}_i(x_i^t) - \tilde{f}_i(x^\star) -\langle \nabla \tilde{f}_i(x^\star),x_i^t-x^\star\rangle\\
&=D_{\tilde{f}_i}(x_i^{t}, x^\star).
    \end{align*}

Thus, we obtain \algname{Scafflix} by applying \algname{i-Scaffnew} to solve (\ref{eq:FLIX}), viewed as  (\ref{eq:ERM}) with $f_i$ replaced by $\tilde{f}_i$, and further making the following substitutions in the algorithm: $g_i^t$ is replaced by $\alpha_i g_i^t$, $h_i^t$ is replaced by $\alpha_i h_i^t$ (so that $h_i^t$ in \algname{Scafflix} converges to  $\nabla f_i (\tilde{x}_i^\star)$ instead of $\nabla \tilde{f}_i (x^\star)=\alpha_i\nabla f_i (\tilde{x}_i^\star)$), $\gamma_i$ is replaced by $\alpha_i^{-2}\gamma_i$ (so that the $\alpha_i$ disappear in the theorem).

Accordingly, Theorem~\ref{theo2} follows from Theorem~\ref{scafflix_theo1}, with the same substitutions and with $A_i$, $C_i$ and $\mu_i$ replaced by  $\alpha_i^2 A_i$, $\alpha_i^2 C_i$ and $\alpha_i^2\mu_i$, respectively. Finally, the Lyapunov function is multiplied by $\gamma_{\min}/n$ to make it independent from $\epsilon$ when scaling the $\gamma_i$ by $\epsilon$ in Corollary~\ref{cor2}.\medskip


We note that 
\algname{i-Scaffnew} is recovered as a particular case of \algname{Scafflix} if $\alpha_i \equiv 1$, so that \algname{Scafflix} is indeed more general.

\section{Proof of Corollary~\ref{cor2}}

We place ourselves in the conditions of Theorem~\ref{theo2}. 
Let $\epsilon>0$. We want to choose the $\gamma_i$ and the number of iterations $T\geq 0$ such that $\Exp{\Psi^T}\leq \epsilon$. For this, we bound the two terms $(1-\zeta)^T \Psi^0$ and $\frac{\gamma_{\min}}{\zeta n} \sum_{i=1}^n \gamma_i  C_i$ in \eqref{eqr1} by $\epsilon/2$.

We set $p=\sqrt{\min_{i\in[n]} \gamma_i \mu_i}$, so that $\zeta = \min_{i\in[n]} \gamma_i \mu_i$. We have
\begin{equation}
T\geq \frac{1}{\zeta}\log(2\Psi^0\epsilon^{-1}) \Rightarrow (1-\zeta)^T\Psi^0 \leq \frac{\epsilon}{2}.\label{eqgrergg}
\end{equation}
Moreover, 
\begin{equation*}
(\forall i\in[n] \mbox{ s.t.\ } C_i>0)\ \gamma_i \leq \frac{\epsilon \mu_{\min}}{2 C_i} \Rightarrow \frac{\gamma_{\min}}{\zeta n} \sum_{i=1}^n \gamma_i  C_i \leq \frac{\epsilon}{2} \frac{\left(\min_{j\in[n]} \gamma_j \right)\left( \min_{j\in[n]} \mu_j\right)}{\min_{j\in[n]} \gamma_j \mu_j }\leq \frac{\epsilon}{2}.
\end{equation*}
Therefore, we set for every $i\in[n]$
\begin{equation*}
\gamma_i\eqdef \min \left(\frac{1}{A_i},\frac{\epsilon \mu_{\min}}{2 C_i} \right)
\end{equation*}
(or $\gamma_i\eqdef \frac{1}{A_i}$ if $C_i=0$), 
and we get from \eqref{eqgrergg} that $\Exp{\Psi^T}\leq \epsilon$ after
\begin{equation*}
\mathcal{O}\left(\left(\max_{i\in[n]}  \max\left(\frac{A_i}{\mu_i},\frac{C_i}{\epsilon \mu_{\min}\mu_i}\right)\right)\log(\Psi^0\epsilon^{-1})\right)
\end{equation*}
iterations.

\section{Additional experimental results}
\begin{figure}[!htbp]
	\centering
	\begin{subfigure}[b]{0.4\textwidth}
		\centering
		\includegraphics[trim=0 0 0 0, clip, width=\textwidth]{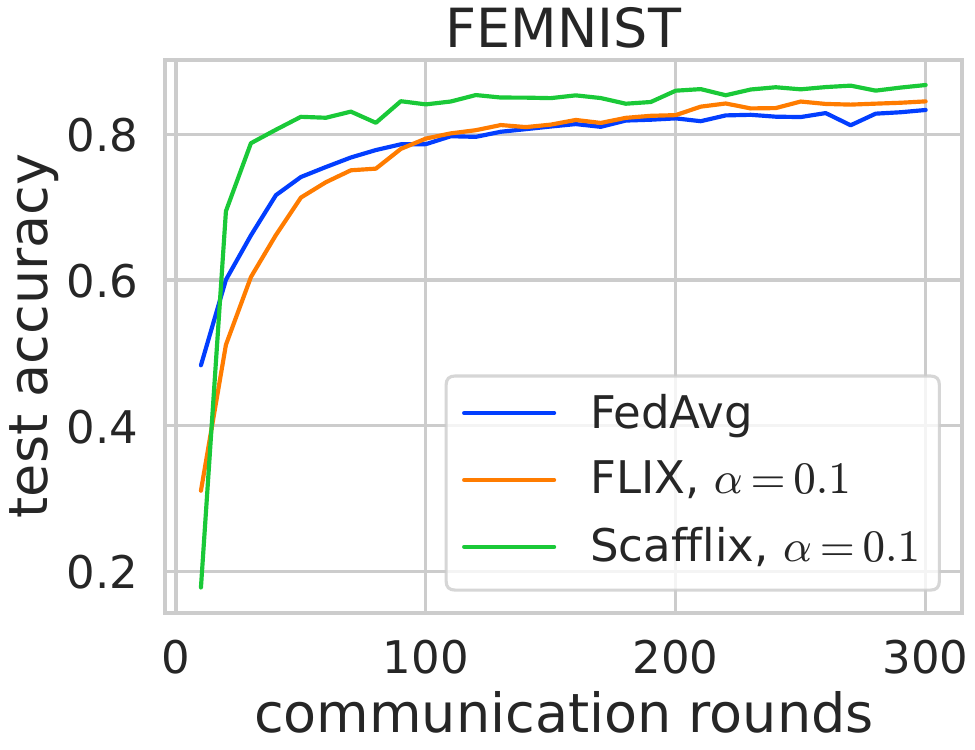}
	\end{subfigure}
	\qquad\qquad
	\begin{subfigure}[b]{0.4\textwidth}
		\centering
		\includegraphics[trim=0 0 0 0, clip, width=\textwidth]{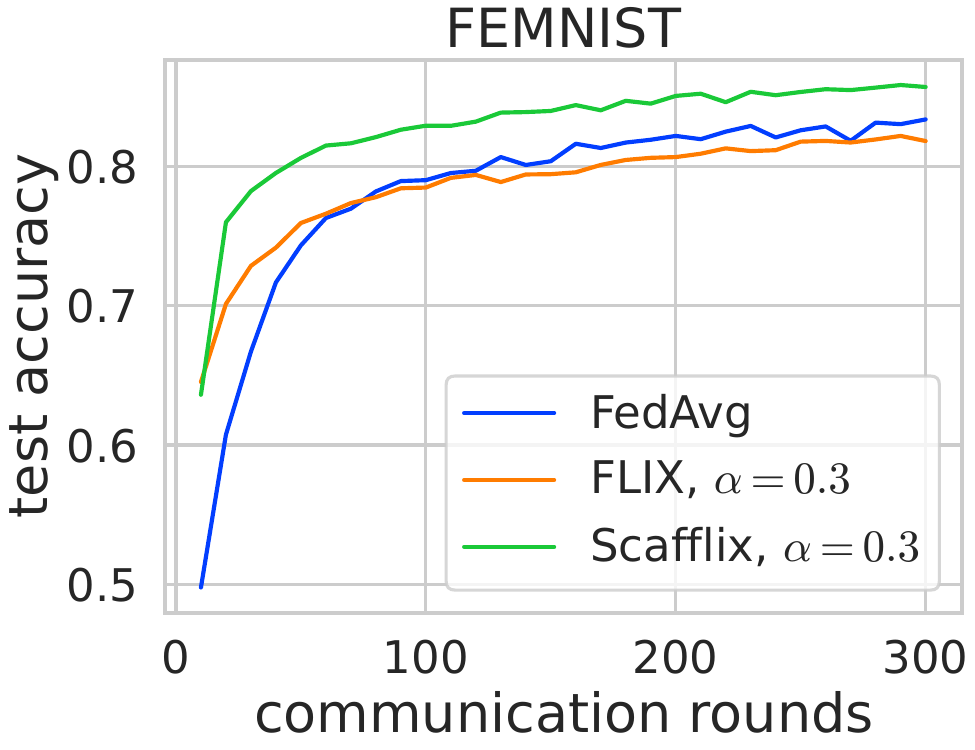}
		\end{subfigure}
	\begin{subfigure}[b]{0.4\textwidth}
		\centering
		\includegraphics[trim=0 0 0 0, clip, width=\textwidth]{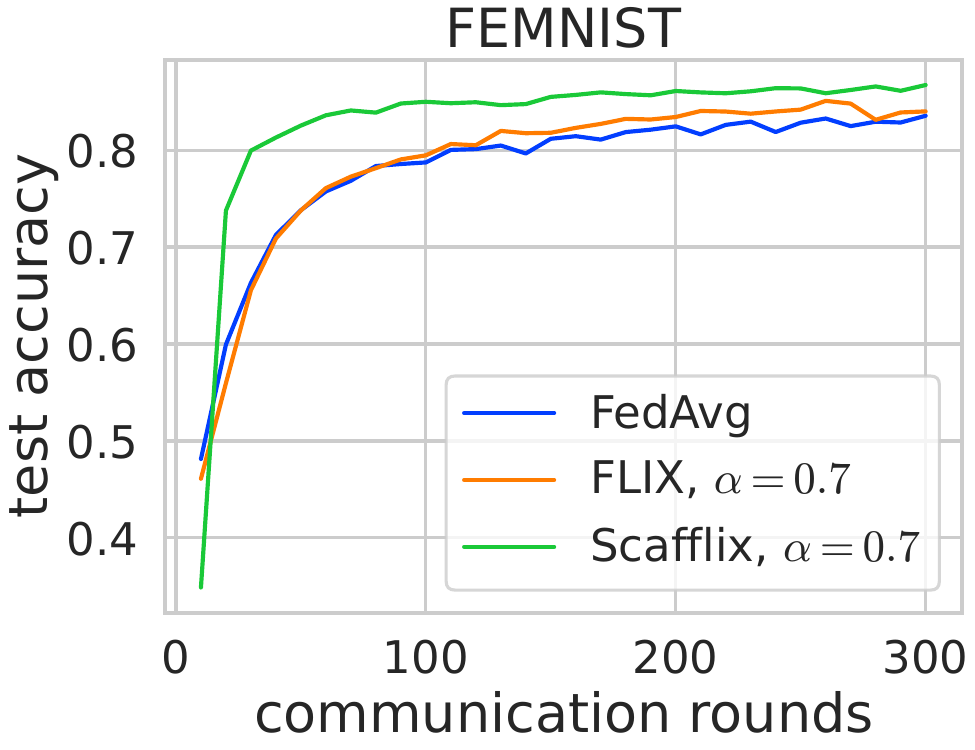}
	\end{subfigure}
	\qquad\qquad
	\begin{subfigure}[b]{0.4\textwidth}
		\centering
		\includegraphics[trim=0 0 0 0, clip, width=\textwidth]{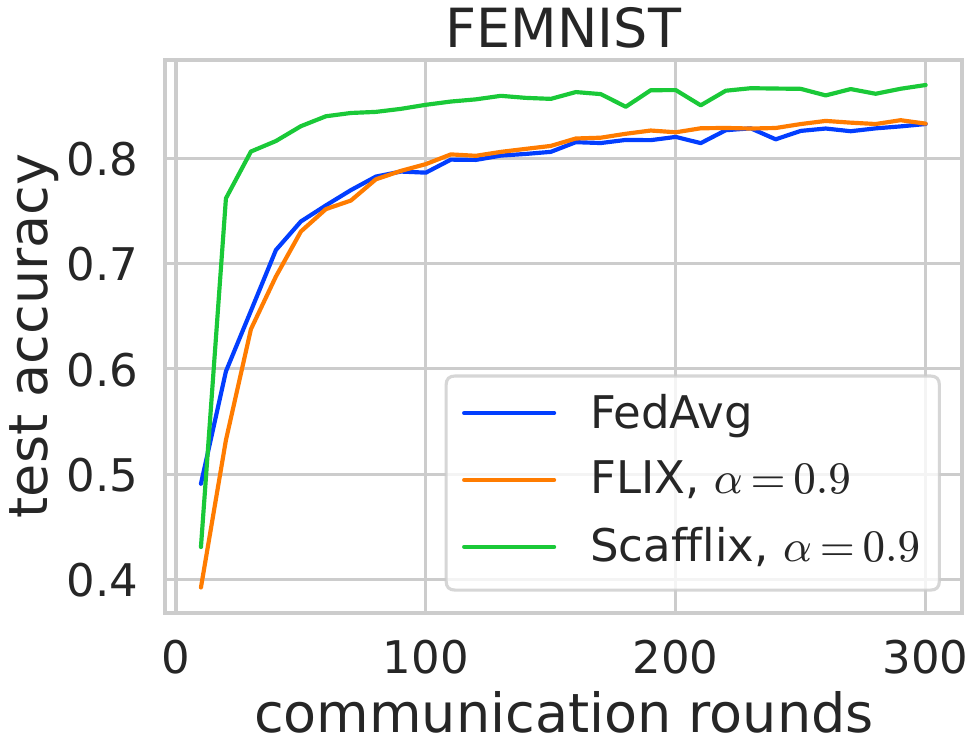}
		\end{subfigure}
	\caption{As part of our experimentation on the FEMNIST dataset, we performed complementary ablations by incorporating various personalization factors, represented as $\alpha$. In the main section, we present the results obtained specifically with $\alpha = 0.5$. Furthermore, we extend our analysis by highlighting the outcomes achieved with $\alpha$ values spanning from 0.1 to 0.9.} 
	\label{fig:abs03}
\end{figure}

\begin{figure}[!htbp]
	\centering
	\begin{subfigure}[b]{0.4\textwidth}
		\centering
		\includegraphics[trim=0 0 0 0, clip, width=\textwidth]{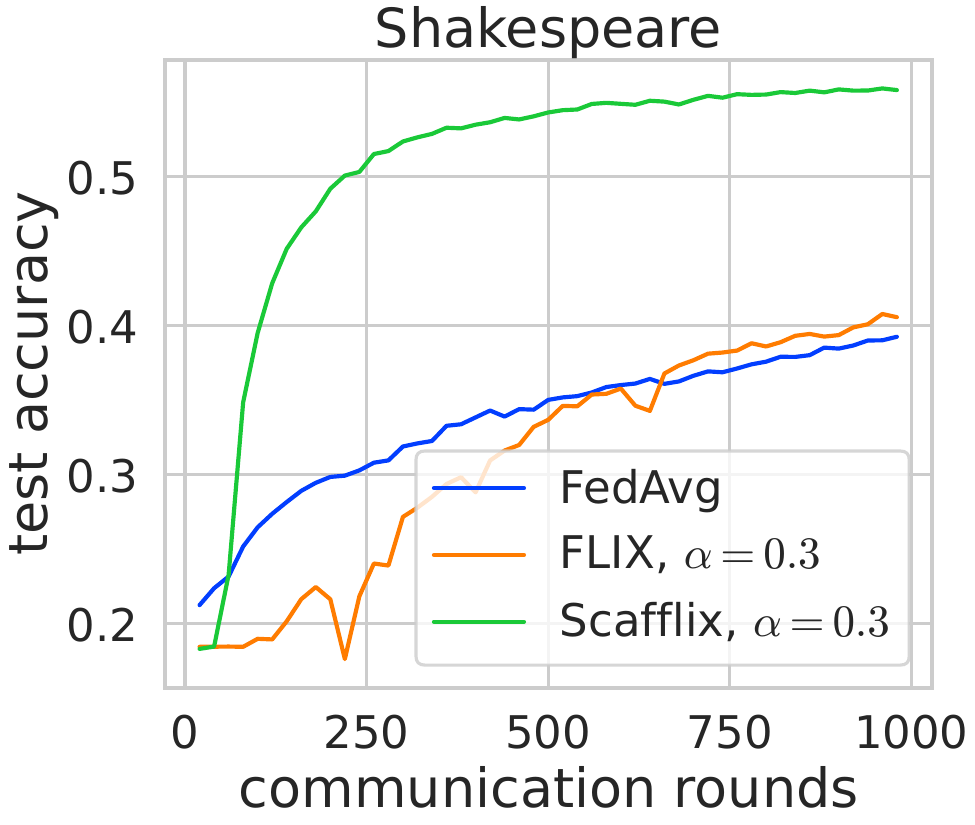}
	\end{subfigure}
	\qquad\qquad
		\begin{subfigure}[b]{0.4\textwidth}
		\centering
		\includegraphics[trim=0 0 0 0, clip, width=\textwidth]{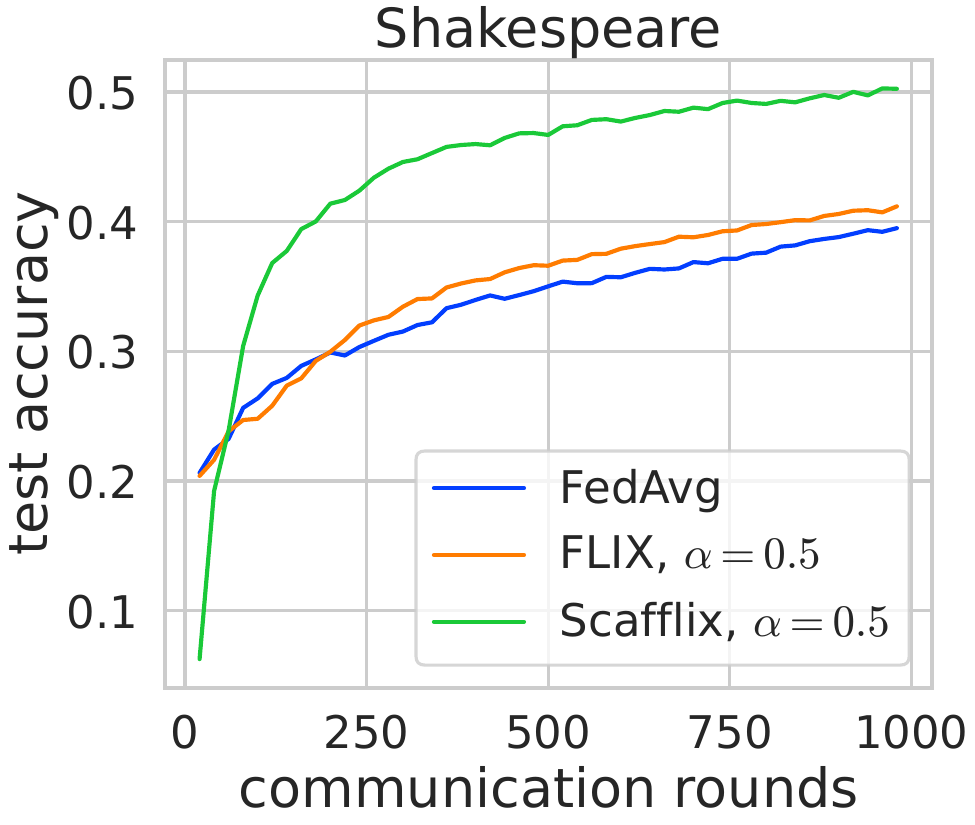}
		\end{subfigure}
	\begin{subfigure}[b]{0.4\textwidth}
		\centering
		\includegraphics[trim=0 0 0 0, clip, width=\textwidth]{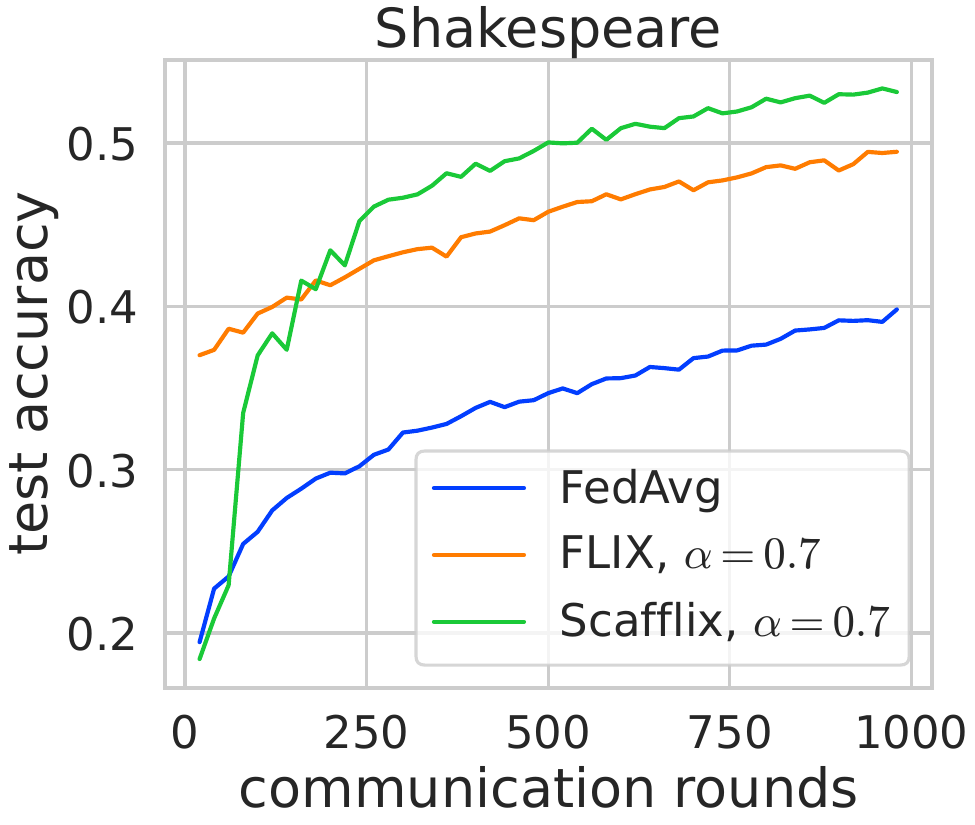}
	\end{subfigure}
	\qquad\qquad
	\begin{subfigure}[b]{0.4\textwidth}
		\centering
		\includegraphics[trim=0 0 0 0, clip, width=\textwidth]{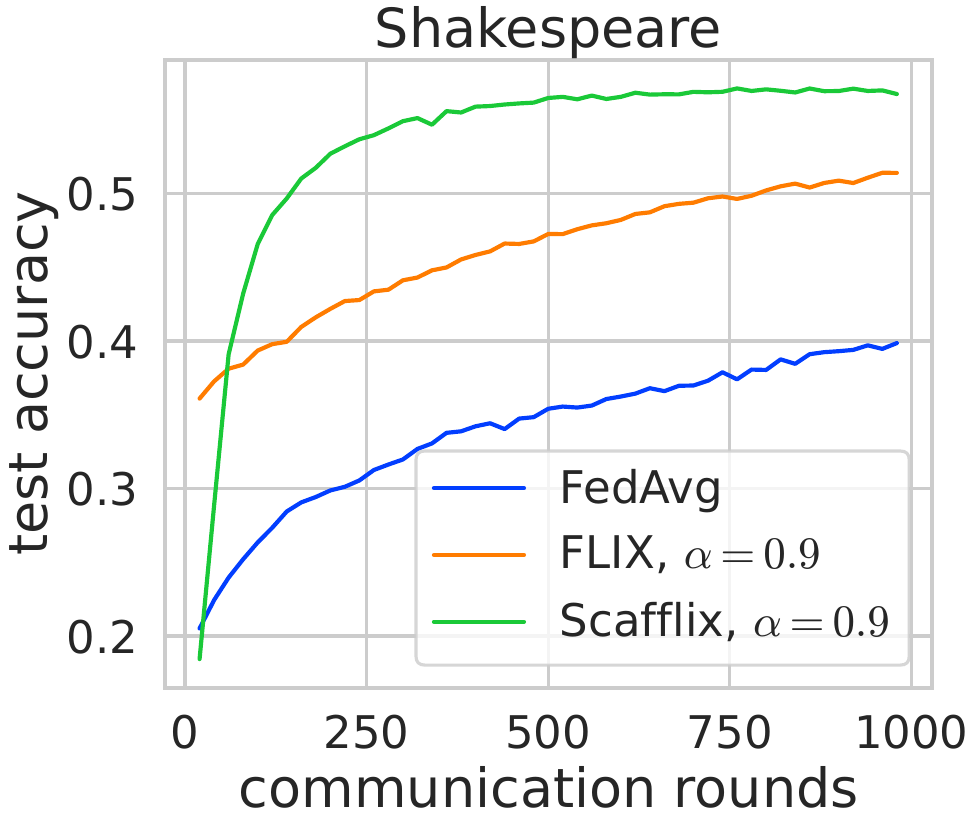}
		\end{subfigure}
	\caption{In our investigation of the Shakespeare dataset, we carried out complementary ablations, considering a range of personalization factors denoted as $\alpha$. The selection strategy for determining the appropriate $\alpha$ values remains consistent with the methodology described in the above figure.} 
	\label{fig:abs09}
\end{figure} 

\begin{figure}[!htbp]
	\centering
	\begin{subfigure}[b]{0.4\textwidth}
		\centering
		\includegraphics[trim=0 0 0 0, clip, width=\textwidth]{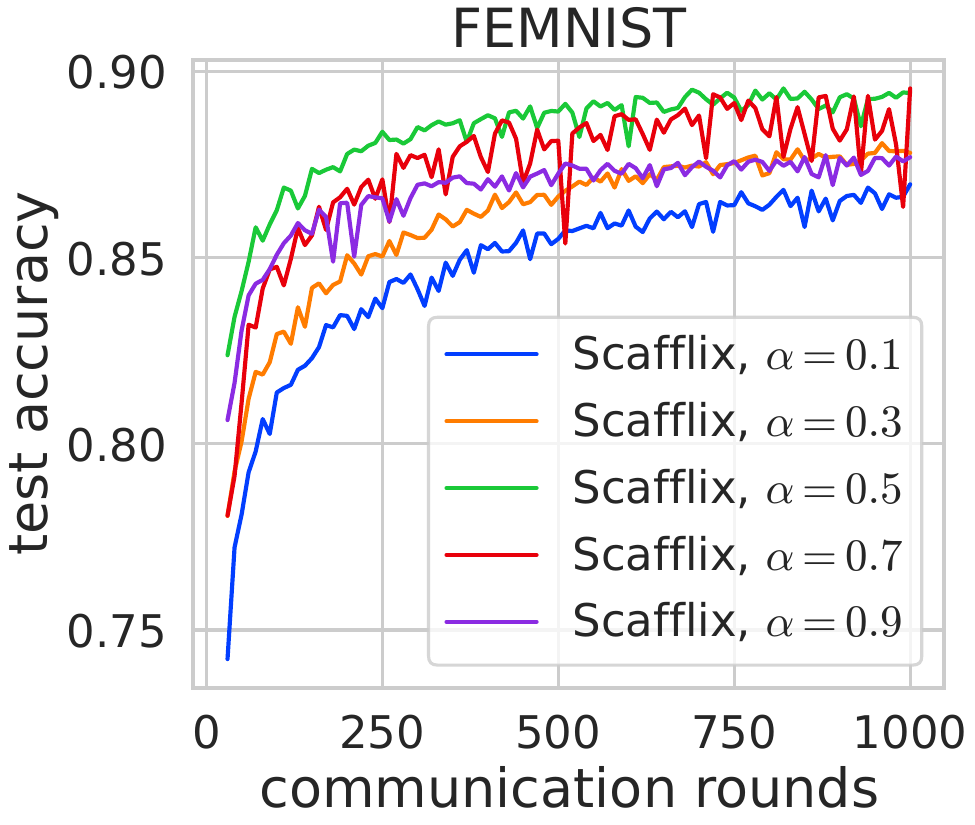}
	\end{subfigure}
	\qquad\qquad
	\begin{subfigure}[b]{0.4\textwidth}
		\centering
		\includegraphics[trim=0 0 0 0, clip, width=\textwidth]{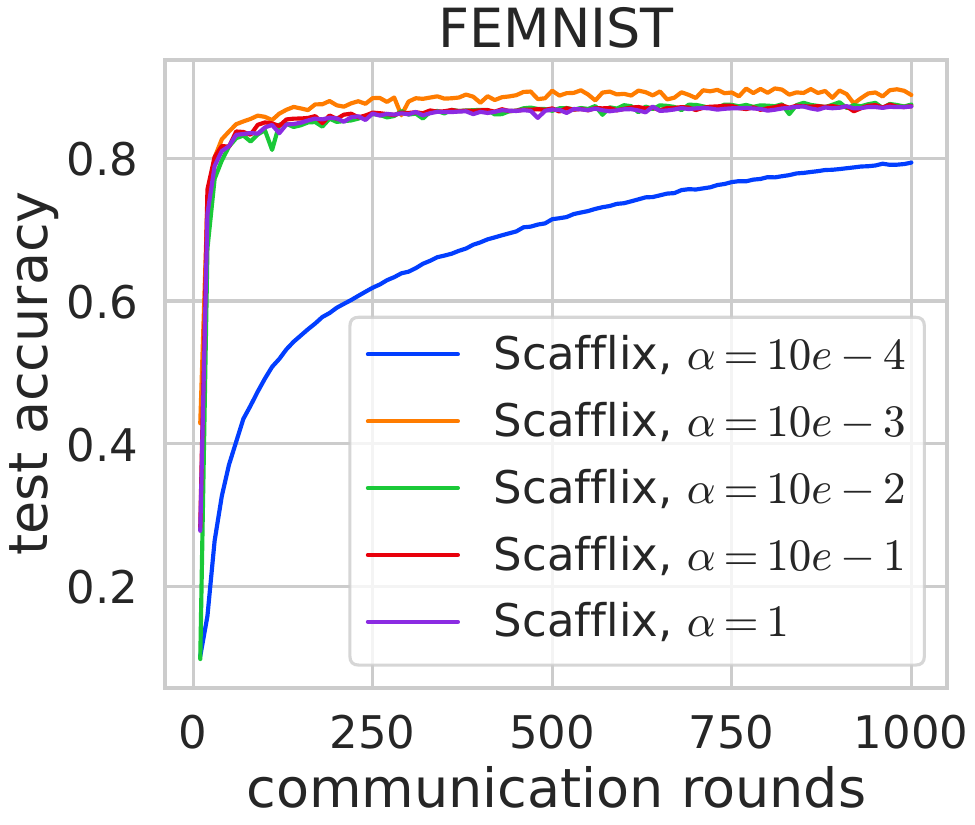}
		\end{subfigure}
	\caption{Ablation studies with different values of the personalization factor $\alpha$. The left figure is the complementary experiment of linearly increasing $\alpha$ with full batch size; the right is the figure with exponentially increasing $\alpha$ with default batch size of 20. } 
	\label{fig:abs08}
\end{figure} 

\subsection{Additional baselines}
While our research primarily seeks to ascertain the impact of explicit personalization and local training on communication costs, we recognize the interest of the community for a broad comparative scope. Accordingly, we have included extensive baseline comparisons with other recent FL and particularly personalized FL (pFL) methodologies. A comparative performance analysis on popular datasets like CIFAR100 and FMNIST is presented below:
    
\begin{table}[!htbp]
    \centering
    \renewcommand{\arraystretch}{1.05}\caption{Results of additional baselines.}\label{tab:datasets_compare} 
    \begin{adjustbox}{max width=\textwidth}
    \begin{tabular}{c|c|c|c|c|c}
        \toprule
        Method & Ditto & FedSR-FT & FedPAC & FedCR & Scafflix\\ \hline
        CIFAR100 & 58.87 & 69.95 & 69.31 & 78.49 & 72.37\\ \hline
        FMNIST & 85.97 & 87.08 & 89.49 & 93.77 & 89.62\\
        \bottomrule
    \end{tabular}
    \end{adjustbox}  
\end{table}     
     
We utilized the public code and adopted the optimal hyper-parameters from \algname{FedCR}~\cite{FedCR}, subsequently re-running and documenting all baseline performances under the `non-iid' setting. Our proposed \algname{Scafflix} algorithm was reported with a communication probability of $p=0.3$ and spanned $500$ communication rounds. We set the personalization factor $\alpha$ at 0.3. Based on the results, when focusing solely on the generalization (testing) performance of the final epoch, our method is on par with state-of-the-art approaches such as \algname{FedPAC}~\cite{FedPAC} and \algname{FedCR}~\cite{FedCR}. However, our primary emphasis lies in demonstrating accelerated convergence. 

\subsection{Logistic regression under non-IID conditions}\label{sec:logistic_noniid}
Our thorough evaluation investigates the potential for achieving double acceleration through both explicit personalization and efficient local training under varying data distributions. We consider the scenarios outlined below:

\begin{itemize}
\item \textit{IID:} Data is uniformly distributed across all clients with identical weighting factors, denoted as $\alpha_i$.
\item \textit{Label-wise Non-IID:} We induce imbalances in label distribution among clients. The data is bifurcated into positive and negative samples, followed by a tailored sampling technique that incrementally augments the ratio of positive samples relative to negative ones. We define these ratios as $r_{\mathrm{pos}} = (i + 1) / n$ and $r_{\mathrm{neg}} = 1 - r_{\mathrm{pos}}$, where $i$ represents the client index, and $n$ is the number of clients.
\item \textit{Feature-wise Non-IID:} Variations in feature distribution across clients are introduced by segmenting the features into clusters with the k-means algorithm. The number of clusters corresponds to the client count.
\item \textit{Quantity-wise Non-IID:} Data volume variance among clients is realized. The distribution of data samples per client follows a Dirichlet distribution, with a default setting of $\alpha = 0.5$. Notably, a higher $\alpha$ leads to a more uniform distribution. At $\alpha = 1$, it resembles a uniform distribution, while at $\alpha < 1$, the distribution becomes skewed, resulting in a disparate data volume across workers.
\end{itemize}

In the main text, Figure~\ref{fig:tissue_figure} illustrates the outcomes for \textit{label-wise non-IID}. For the sake of completeness, we also include results in Figure~\ref{fig:tissue_figure2}, Figure~\ref{fig:tissue_figure3}, and Figure~\ref{fig:tissue_figure4} depicting various data partitioning strategies. Across these figures, we consistently observe that \algname{Scafflix} successfully achieves double acceleration.

\begin{figure*}[!htbp]
	\centering
	\begin{subfigure}[b]{0.32\textwidth}
		\centering
		\includegraphics[width=\textwidth]{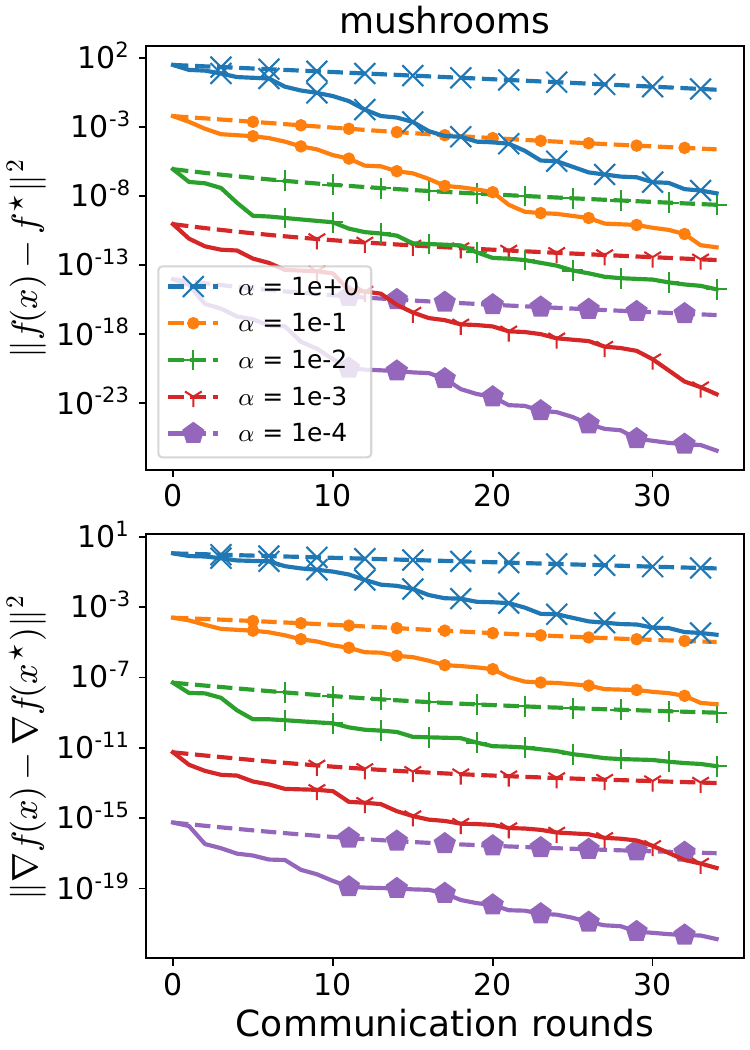}
	\end{subfigure}
	\hfill
	\begin{subfigure}[b]{0.32\textwidth}
		\centering
		\includegraphics[width=\textwidth]{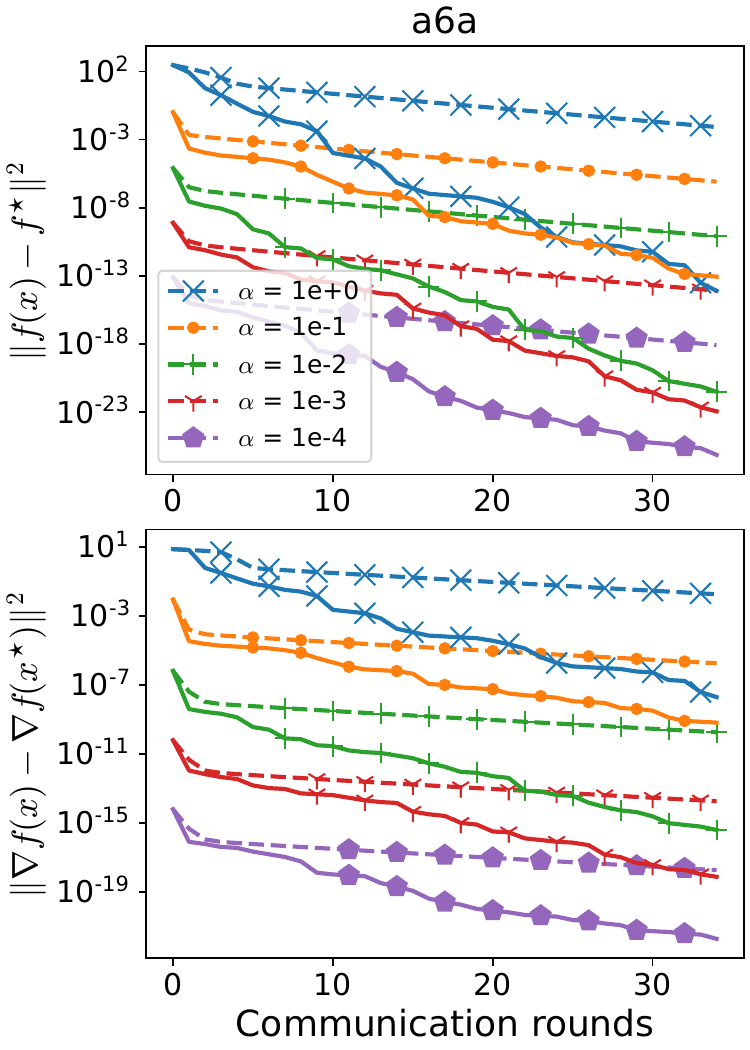}
	\end{subfigure}
	\hfill
	\begin{subfigure}[b]{0.32\textwidth}
		\centering
		\includegraphics[width=\textwidth]{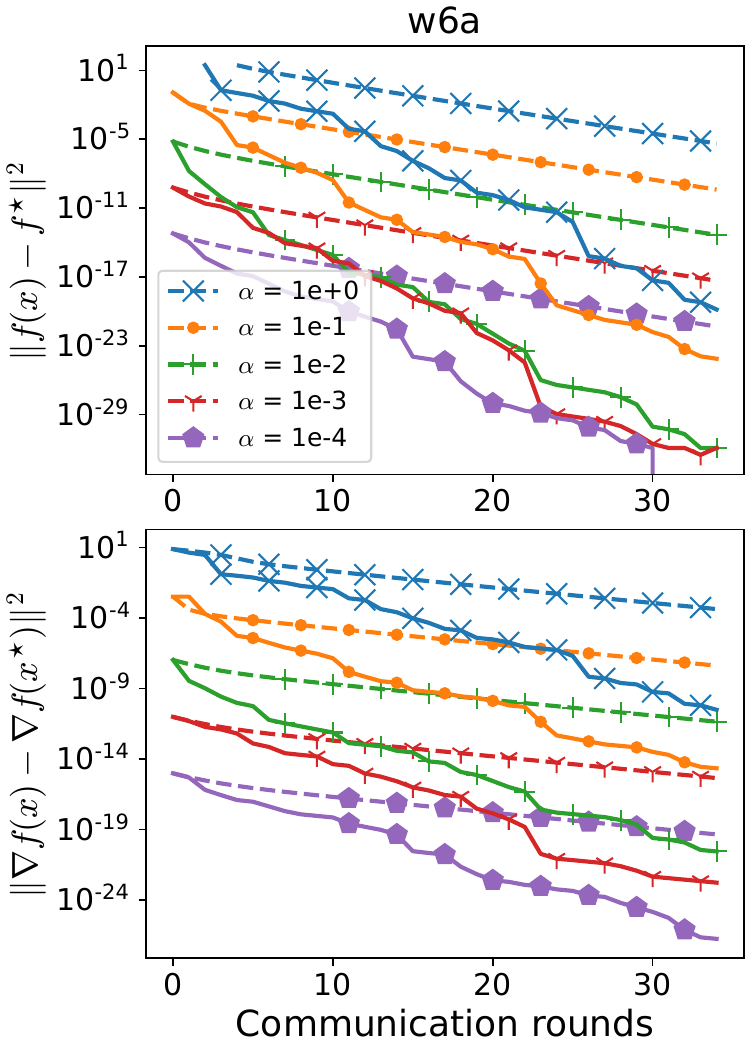}
	\end{subfigure}
	\caption{Results on IID splits.}
	\label{fig:tissue_figure2}
\end{figure*}  

\begin{figure*}[!htbp]
	\centering
	\begin{subfigure}[b]{0.32\textwidth}
		\centering
		\includegraphics[width=\textwidth]{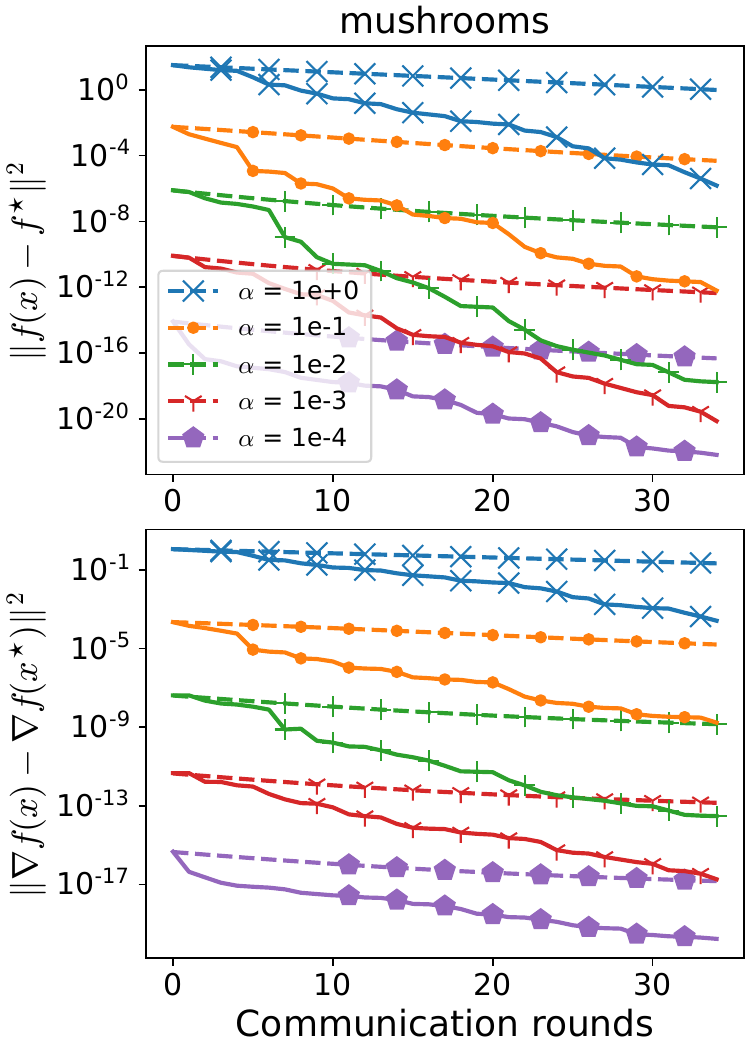}
	\end{subfigure}
	\hfill
	\begin{subfigure}[b]{0.32\textwidth}
		\centering
		\includegraphics[width=\textwidth]{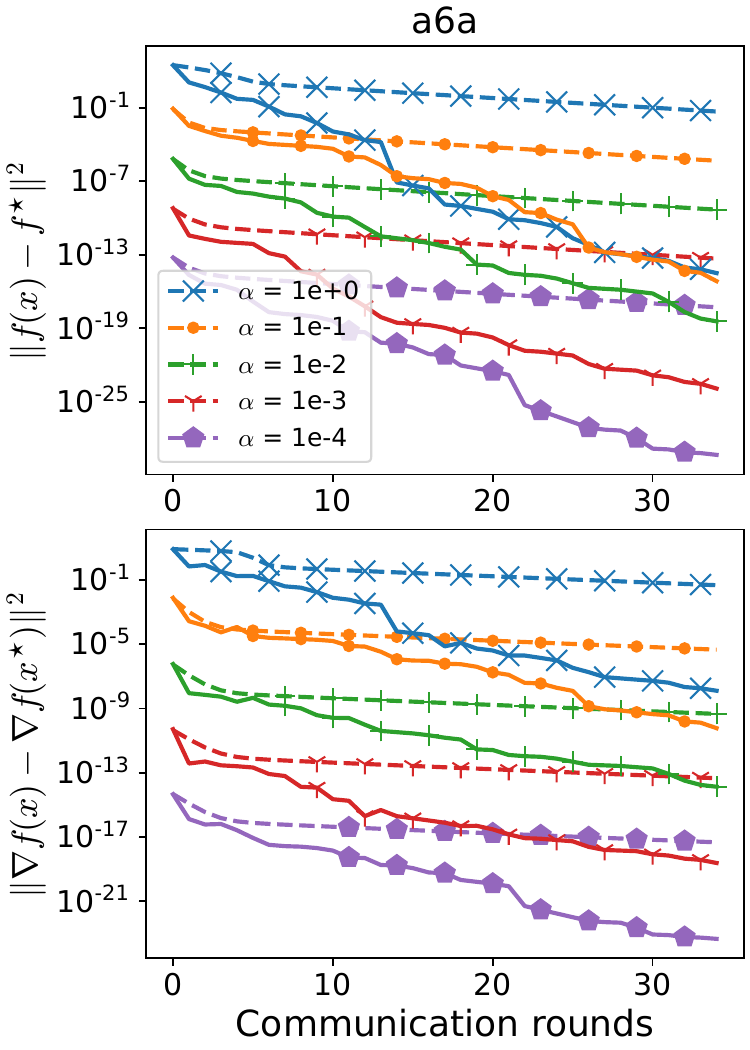}
	\end{subfigure}
	\hfill
	\begin{subfigure}[b]{0.32\textwidth}
		\centering
		\includegraphics[width=\textwidth]{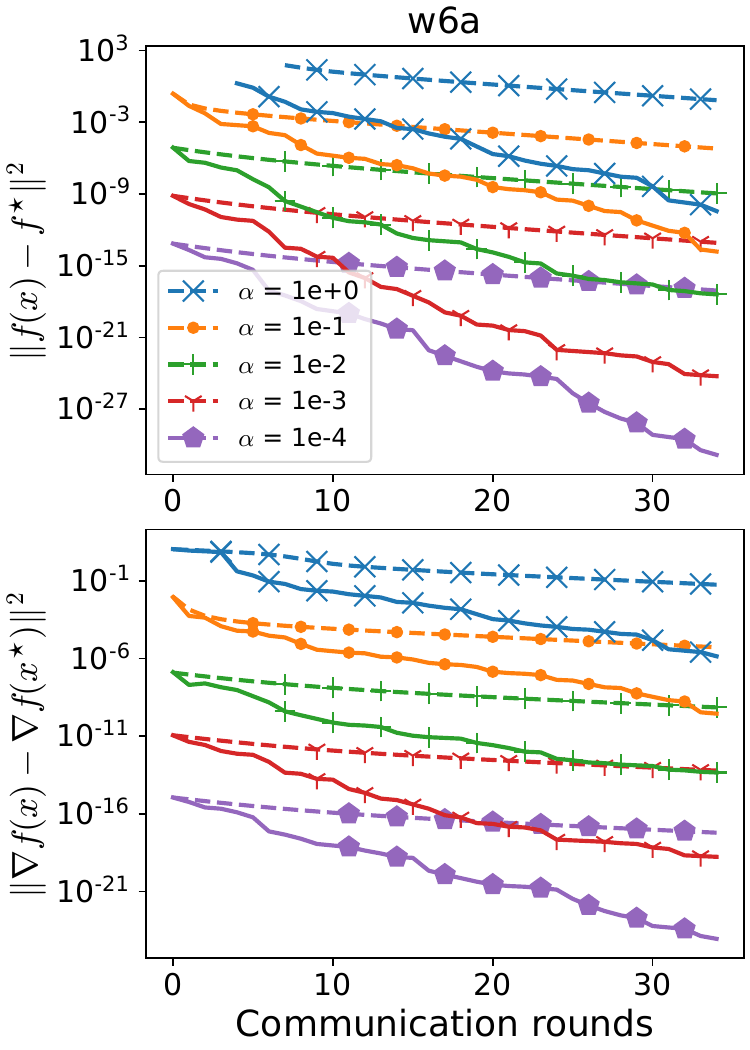}
	\end{subfigure}
	   \caption{Feature-wise non-IID.}
	   \label{fig:tissue_figure3}
\end{figure*}  
 
\begin{figure*}[!htbp]
	\centering
	\begin{subfigure}[b]{0.32\textwidth}
		\centering
		\includegraphics[width=\textwidth]{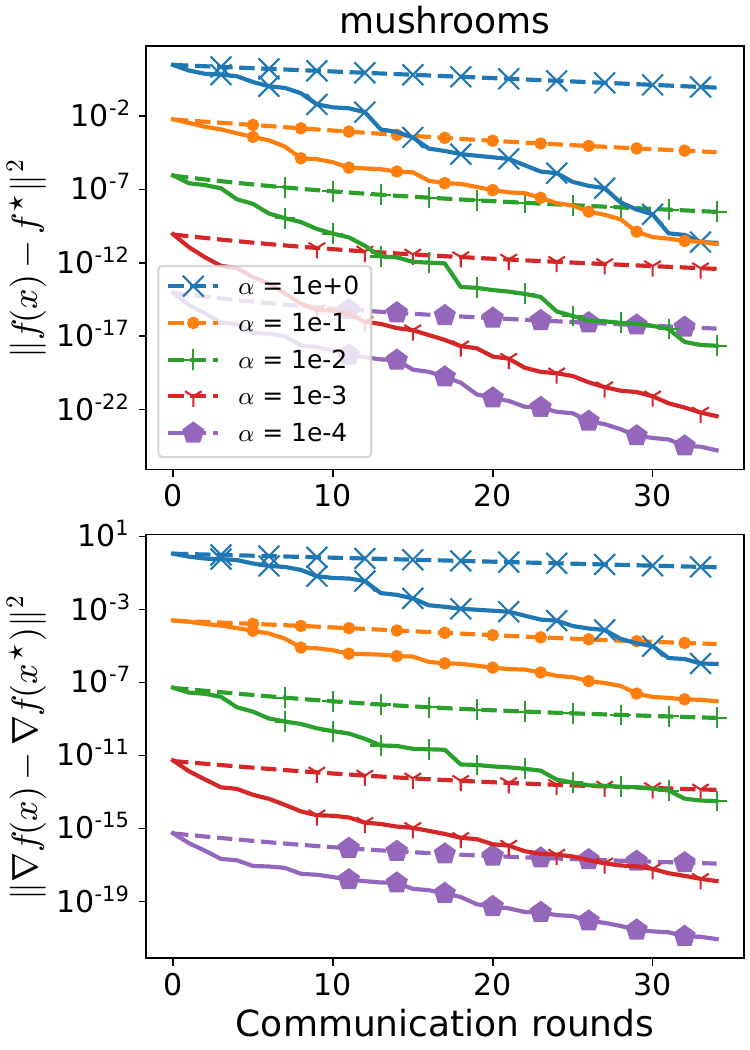}
	\end{subfigure}
	\hfill
	\begin{subfigure}[b]{0.32\textwidth}
		\centering
		\includegraphics[width=\textwidth]{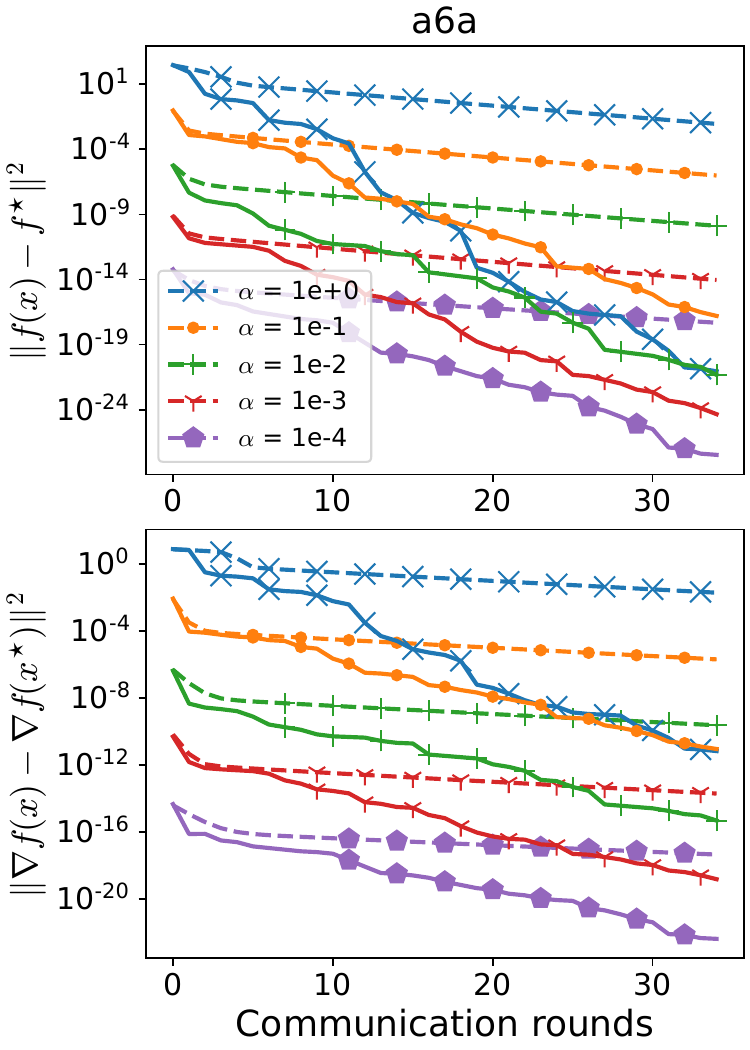}
	\end{subfigure}
	\hfill
	\begin{subfigure}[b]{0.32\textwidth}
		\centering
		\includegraphics[width=\textwidth]{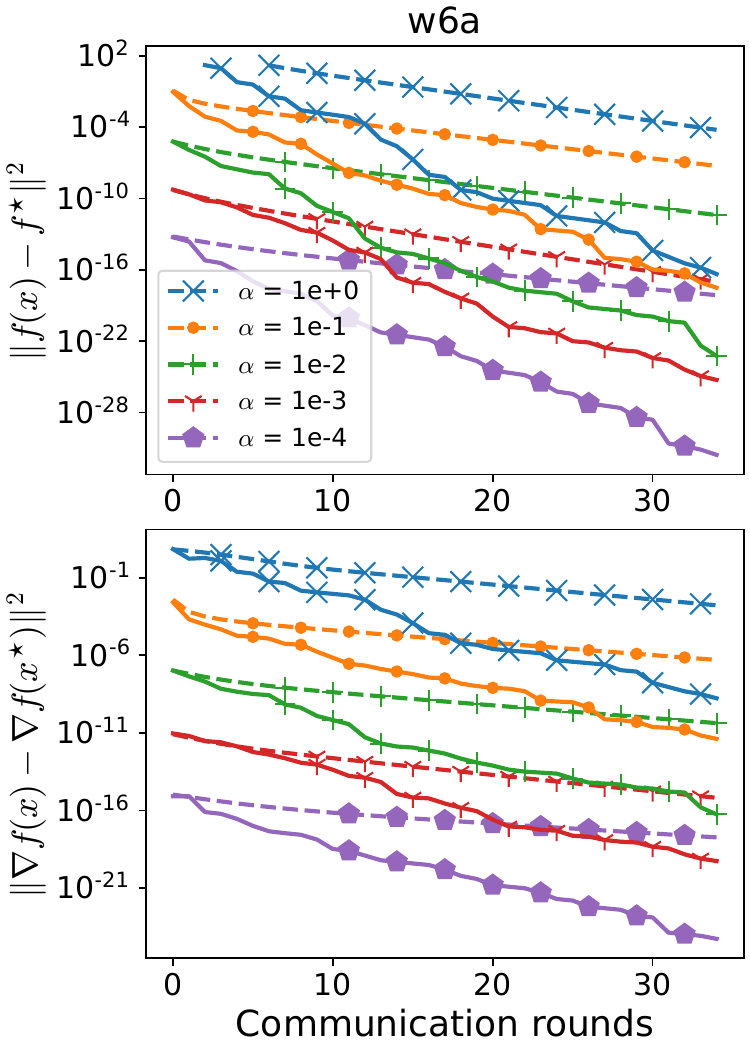}
	\end{subfigure}
	   \caption{Quantity-wise non-IID.}
	   \label{fig:tissue_figure4}
\end{figure*}  

\newpage
\clearpage
\subsection{Inexact approximation of local optimal}\label{sec:inexact_approx_local_optimal}
To visualize the cost of local communications, we present the expected number of local iterations to achieve an epsilon such that $\|\nabla f_i(x)\|<\epsilon$. We present the results in Figure~\ref{fig:abs11}. We can see there is a huge difference with respect to the different $\epsilon$. Since in FL, the communication cost is always the bottleneck, for scenarios that local computation is not that expensive, we can run more local iterations to obtain a smaller $\epsilon$. In Figure~\ref{fig:abs22}, we show on ablations that even choose $\epsilon=1e-1$, which can still provide guidance leading to acceptable neighborhood. In general, there is a neighborhood here. Since in Figure~\ref{fig:abs22}, we consider the personalization factor $\alpha=0.1$, here we conduct further ablations with $\alpha=0.01$ with the results presented in Figure~\ref{fig:abs32}.

\begin{figure}[!htbp]
	\centering
	\begin{subfigure}[b]{0.325\textwidth}
	\centering
	\includegraphics[trim=0 0 0 0, clip, width=\textwidth]{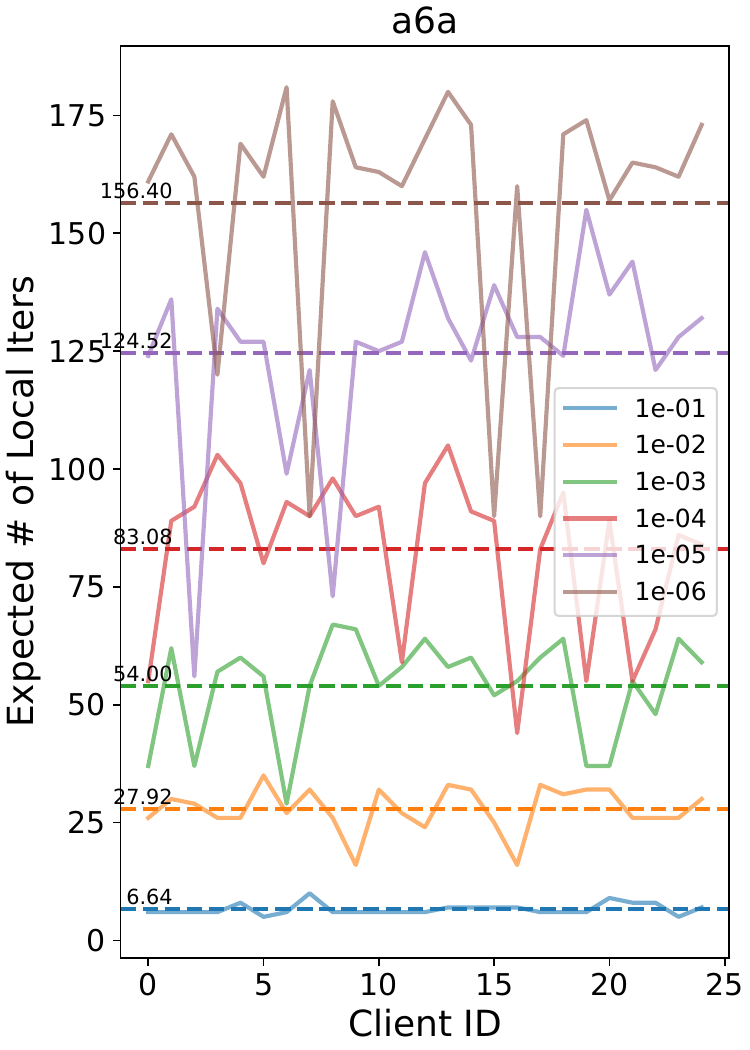}
	\end{subfigure}
	\hfill
	\begin{subfigure}[b]{0.325\textwidth}
	\centering
	\includegraphics[trim=0 0 0 0, clip, width=\textwidth]{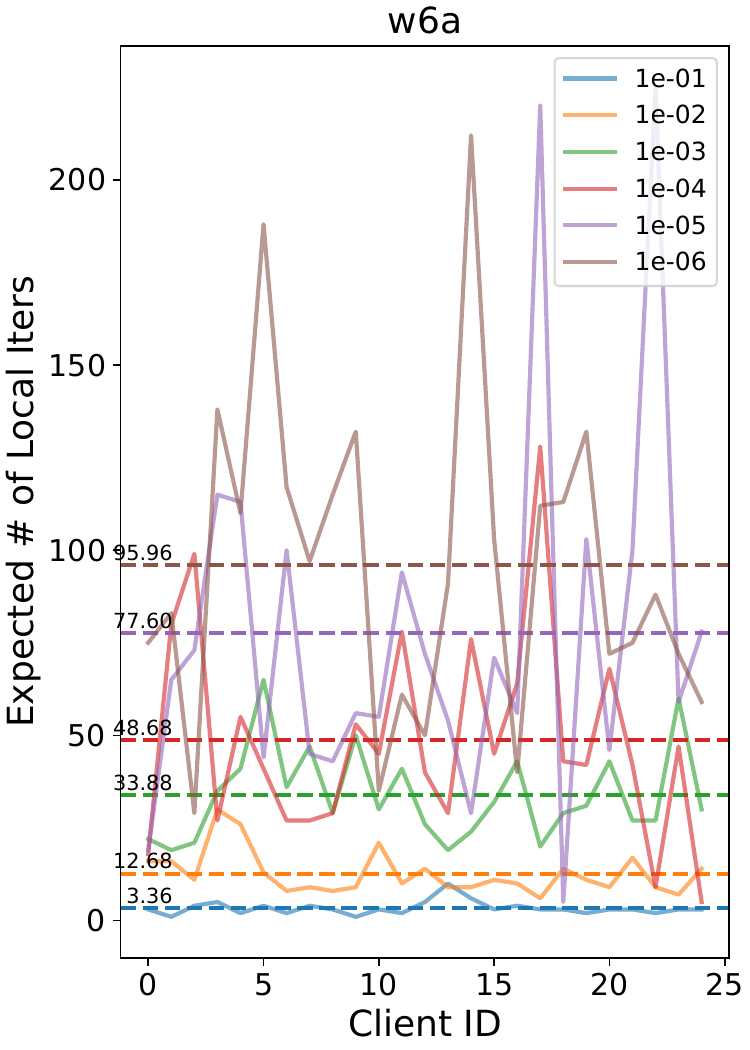}
	\end{subfigure}
	\hfill
	\begin{subfigure}[b]{0.325\textwidth}
		\centering
		\includegraphics[trim=0 0 0 0, clip, width=\textwidth]{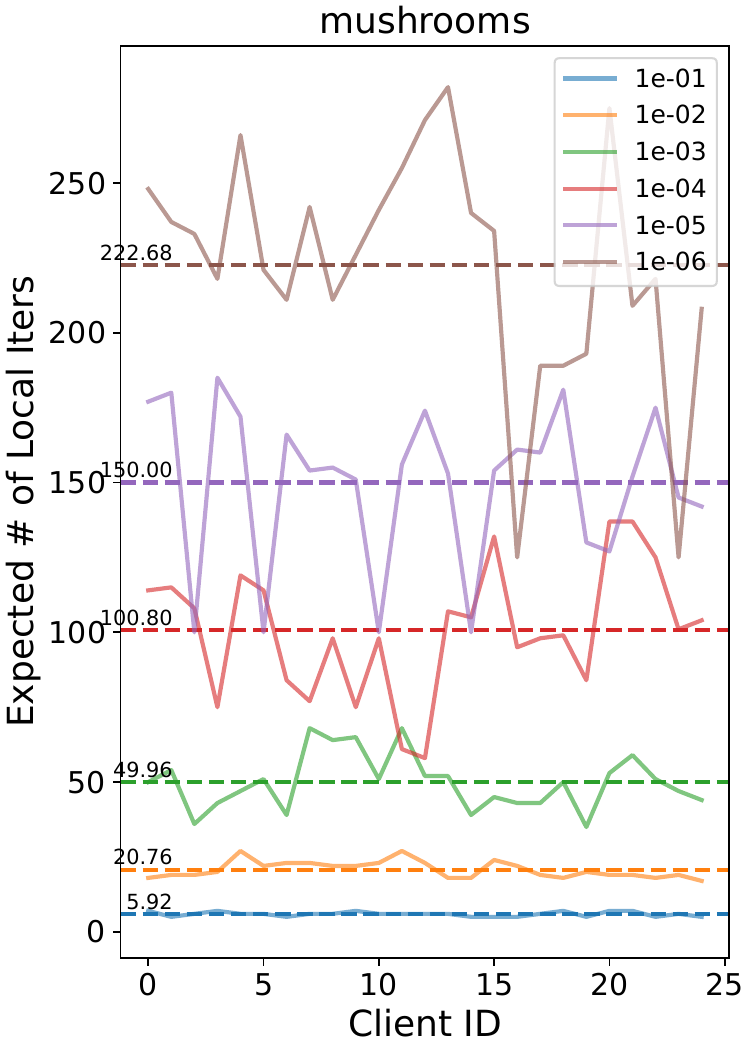}
		\end{subfigure}
	\caption{Number of local iterations per client for find an approximation $\bar{x}_i^\star$ of the local optimal $x_i^\star$ such that $\|\nabla f_i(x)\|<\epsilon$. The legend is $\epsilon$.}
	\label{fig:abs11}
	\end{figure} 
 
	\begin{figure}[!htbp]
		\centering
		\begin{subfigure}[b]{0.48\textwidth}
		\centering
		\includegraphics[trim=0 0 0 0, clip, width=\textwidth]{img/scafflix/inexact/scafflix_quantity_noniid_mushrooms_1e-06_log_reg_mushrooms_epsilons_inexact_.pdf}
		\end{subfigure}
		\hfill  
		\begin{subfigure}[b]{0.48\textwidth}
		\centering
		\includegraphics[trim=0 0 0 0, clip, width=\textwidth]{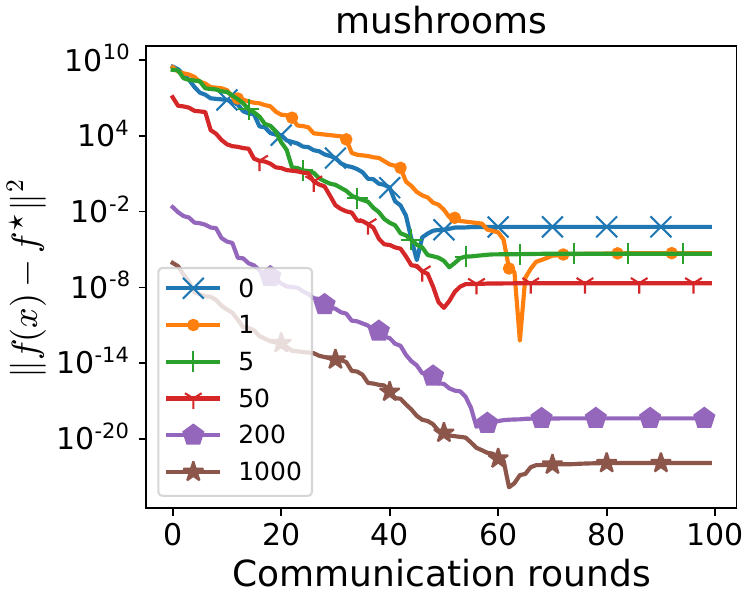}
		\end{subfigure}
		\caption{Inexact local optimal approximation with $\alpha=0.01$.}
		\label{fig:abs32}
		\end{figure}  

%% file: Appendix_C3_FedP3.tex
\chapter{Appendix to Chapter \ref{chapter_fedp3}}
\label{chapter_appendix_fedp3}
\thispagestyle{empty}

\section{Extended related work}\label{ref:app_extended_related_work}

\color{black}
\subsection{Federated network pruning}\label{sec:federated_network_pruning}

We introduce two distinct types of network pruning within our study: 1) global pruning, which extends from server to client, and 2) local pruning, where each client's network is pruned based on its own specific data. In our setting, we assume federated pruning is the scenario with both possible global and local pruning. 
Federated network pruning, a closely related field, pursues the objective of identifying the optimal or near-optimal pruned neural network at each communication from the server to the clients, as documented in works of \citet{PruneFL} and \citet{FedTiny}, for example.

During the initial phase of global pruning, \citep{PruneFL} isolates a single potent and reliable client to initiate model pruning. The subsequent stage of local pruning incorporates all clients, advancing the adaptive pruning process. This process involves not only parameter removal but also the reintroduction of parameters, complemented by the standard {FedAvg}~\citep{mcmahan2017communication}. However, the need for substantial local memory to record the updated relevance measures of all parameters in the full-scale model poses a challenge. As a solution to this problem, \cite{FedTiny} proposes an adaptive batch normalization and progressive pruning modules that utilize sparse local computation. Yet, these methods overlook explicit considerations for constraints related to client-side computational resources and communication bandwidth.

Our primary attention gravitates towards designing distinct local pruning methods, such as \citep{FjORD}, \citep{FedRolex}, and \citep{Flado}. Instead of learning the optimal or suboptimal pruned local network, each client attempts to identify the optimal adaptive sparsity method. The work of \cite{FjORD} has been groundbreaking, as they introduced Ordered Dropout to navigate this issue, achieving commendable results. It's noteworthy that our overarching framework is compatible with these methods, facilitating straightforward integration of diverse local pruning methods.
There are other noticeable methods, such as \citep{HeteroFL}, which focuses on reducing the size of each layer in neural networks. In contrast, our approach contemplates a more comprehensive layer-wise selection and emphasizes neuron-oriented sparsity.

As of our current knowledge, no existing literature directly aligns with our approach, despite its practicality and generality. Even the standard literature regarding federated network pruning appears to be rather constrained.

\subsection{Subnetwork training}\label{sec:subnetwork_training}
Our research aligns with the rising interest in Independent Subnetwork Training (IST), a technique that partitions a neural network into smaller components. Each component is trained in a distributed parallel manner, and the results are subsequently aggregated to update the weights of the entire model. The decoupling in IST enables each subnetwork to operate autonomously, using fewer parameters than the complete model. 
This not only diminishes the computational cost on individual nodes but also expedites synchronization.

This approach was introduced by \cite{yuan2022distributed} for networks with fully connected layers and was later extended to ResNets \cite{dun2022resist} and Graph architectures \cite{wolfe2023gist}. Empirical evaluations have consistently posited IST as an attractive strategy, proficiently melding data and model parallelism to train expansive models even with restricted computational resources.

Further theoretical insights into IST for overparameterized single hidden layer neural networks with ReLU activations were presented by \cite{liao2022convergence}. Concurrently, \cite{shulgin2023towards} revisited IST, exploring it through the lens of sketch-type compression.

While acknowledging the adaptation of IST to FL using asynchronous distributed dropout techniques \cite{dun2023efficient}, our approach diverges significantly from prior works. We advocate that clients should not relay the entirety of their subnetworks to the central server—both to curb excessive networking costs and to safeguard privacy. 
Moreover, our model envisions each client akin to an assembly line component: each specializes in a fraction of the complete neural network, guided by its intrinsic resources and computational prowess.

In Section~\ref{sec:federated_network_pruning} and \ref{sec:subnetwork_training}, we compared our study with pivotal existing research, focusing on federated network pruning and subnetwork training. Responding to reviewer feedback, we have broadened the scope of our related work section to include a more extensive comparison with other significant studies.
    
\subsection{Model heterogeneity} Model heterogeneity denotes the variation in local models trained across diverse clients, as highlighted in previous research~\citep{kairouz2021advances, ye2023heterogeneous}. A seminal work by \cite{smith2017federated} extended the well-known COCOA method~\citep{jaggi2014communication, ma2015adding}, incorporating system heterogeneity by randomly selecting the number of local iterations or mini-batch sizes. However, this approach did not account for variations in client-specific model architectures or sizes. Knowledge distillation has emerged as a prominent strategy for addressing model heterogeneity in Federated Learning (FL). \cite{li2019fedmd} demonstrated training local models with distinct architectures through knowledge distillation, but their method assumes access to a large public dataset for each client, a premise not typically found in current FL scenarios. Additionally, their approach, which shares model outputs, contrasts with our method of sharing pruned local models. Building on this concept, \cite{lin2020ensemble} proposed local parameter fusion based on model prototypes, fusing outputs of clients with similar architectures and employing ensemble training on additional unlabeled datasets. \cite{tan2022fedproto} introduced an approach where clients transmit the mean values of embedding vectors for specific classes, enabling the server to aggregate and redistribute global prototypes to minimize the local-global prototype distance. \cite{he2021fednas} developed FedNAS, where clients collaboratively train a global model by searching for optimal architectures, but this requires transmitting both full network weights and additional architecture parameters. Our method diverges from these approaches by transmitting only weights from a subset of neural network layers from client to server.
 
\color{black}
\section{Experimental details}\label{sec:experimental_details}
\subsection{Statistics of datasets}\label{sec:stats_of_datasets}
We provide the statistics of our adopted datasets in Table.~\ref{tab:statistics}.

\begin{table}[!htbp]
    \centering
    \resizebox{\textwidth}{!}{ 
    \begin{tabular}{c|c|c|c}
    \toprule
    Dataset & \# data & \# train per client & \# test per client \\
    \midrule
    EMNIST-L~\citep{EMNIST} & 48K+8K & 392 & 168 \\
    FashionMNIST~\citep{FashionMNIST} & 60K+10K & 490 & 210 \\
    CIFAR10~\citep{CIFAR} & 50K+10K & 420 & 180 \\
    CIFAR100~\citep{CIFAR} & 50K+10K & 420 & 180 \\
    \bottomrule
    \end{tabular}
    }
    \caption{Dataset statistics, with data uniformly divided among 100 clients by default.}
    \label{tab:statistics}
\end{table}

\subsection{Data distributions}\label{sec:data_distributions}
We emulated non-iid data distribution among clients using both class-wise and Dirichlet non-iid scenarios.

\begin{itemize}
    \item Class-wise: we designate fixed classes directly to every client, ensuring uniform data volume per class. 
    As specifics, EMNIST-L, FashionMNIST, and CIFAR10 assign 5 classes per client, while CIFAR100 allocates 15 classes for each client.
    \item Dirichlet: following an approach similar to \algname{FedCR}~\citep{FedCR}, we use a Dirichlet distribution over dataset labels to create a heterogeneous dataset. 
    Each client is assigned a vector (based on the Dirichlet distribution) that corresponds to class preferences, dictating how labels--and consequently images--are selected without repetition. 
    This method continues until every data point is allocated to a client. 
    The Dirichlet factor indicates the level of data non-iidness. 
    With a Dirichlet parameter of 0.5, about 80\% of the samples for each client on EMNIST-L, FashionMNIST, and CIFAR10 are concentrated in four classes. For 
    CIFAR100, the parameter is set to 0.3. 
\end{itemize}

\subsection{Network architectures}\label{sec:network_architecture}
Our primary experiments utilize four widely recognized datasets, with detailed descriptions provided in the Experiments section. For the CIFAR10/100 and FashionMNIST experiments, we opt for CNNs comprising two convolutional layers and four fully-connected layers as our standard network architecture. In contrast, for the EMNIST-L experiments, we employ a four-layer MLP architecture. The specifics of these architectures are outlined in Table~\ref{tab:network_arch}. Additionally, the default ResNet18 network architecture is selected for our layer-overlapping experiments.
    
\begin{table}[!htbp]
    \centering
    \caption{ The top figure depicts the neural network architecture employed for the CIFAR10/100 and FashionMNIST experiments. Conversely, the bottom figure illustrates the default MLP (Multi-Layer Perceptron) architecture used specifically for the EMNIST-L experiments.}\label{tab:network_arch}
    \begin{tabular}{c|c|c}
        \bf Layer Type & \bf Size & \bf \# of Params.\\ \hline
        Conv + ReLu & $5\times 5\times 64$ & 4,864 / 1,664\\
        Max Pool & $2\times 2$ & 0\\
        Conv + ReLu & $5\times 5\times 64$ & 102, 464\\
        Max Pool & $2\times 2$ & 0\\
        FC + ReLu & $1600\times 1024$ & 1,638,400\\
        FC + ReLu & $1024\times 1024$ & 1,048,576\\
        FC + ReLu & $1024\times 10/100$ & 10,240 / 102,400\\
        \bottomrule
    \end{tabular}
    \\
    \vspace{2mm}
    \begin{tabular}{c|c|c}
       \bf Layer Type  & \bf Size & \bf \# of Params.\\ \hline
       FC + ReLu & $784\times 1024$ & 802,816\\
       FC + ReLu & $1024\times 1024$ & 1,048,576\\
       FC + ReLu & $1024\times 1024$ & 1,048,576\\
       FC & $1024\times 10$ & 10,240\\
       \bottomrule
    \end{tabular}
\end{table}
 
\subsection{Training details}\label{sec:training_details}
Our experiments were conducted on NVIDIA A100 or V100 GPUs, depending on their availability in our cluster. The framework was implemented in PyTorch 1.4.0 and torchvision 0.5.0 within a Python 3.8 environment. Our initial code, based on \algname{FedCR}~\citep{FedCR}, was refined to include hyper-parameter fine-tuning. A significant modification was the use of an MLP network with four \texttt{FC} layers for EMNIST-L performance evaluation. We standardized the experiments to 500 epochs with a local training batch size of 48. The number of local updates was set at 10 to assess final performance. For the learning rate, we conducted a grid search, exploring a range from $10^{-5}$ to 0.1, with a fivefold increase at each step. In adapting FedCR, we used their default settings and fine-tuned the $\beta$ parameter across values ${0.0001, 0.0005, 0.001, 0.005, 0.01}$ for all datasets.

\subsection{Quantitative analysis of reduced parameters}\label{sec:quantitative_parameters}
We provide a quantitative analysis of parameter reduction across four datasets, as shown in \Cref{fig:parameter_comp}. The x-axis represents different global pruning ratios, and the y-axis indicates the number of parameters. For simplicity, we consider a scenario where, aside from the final fully-connected layer, each client trains only one additional layer, akin to the \algname{LowerB} method used in our earlier experiments. For instance, the label \texttt{FC} refers to a condition where only \texttt{FC2} and the final layer are fully trained, with other layers being pruned during server-to-client transfer and dropped in server communication.

With a constant global pruning ratio, the left part of the figure shows the total number of parameters in the locally deployed model post server-to-client pruning, while the right part illustrates the communication cost for each scenario. The numbers atop each bar indicate the relative differences between the largest and smallest elements under various conditions. Across all datasets, we note that higher global pruning ratios result in progressively smaller deployed models. For example, at a 0.5 global pruning ratio, the model size for clients training the \texttt{Conv1} layer is 57.93\% smaller than those training \texttt{FC2}. Moreover, there is a significant disparity in communication costs among clients. The ratios of communication costs are 10815 for CIFAR10, 1522.91 for CIFAR100, 13749.46 for FashionMNIST, and 30.23 for EMNIST-L. 
 
\begin{figure}[!tb]
     \centering
     \begin{subfigure}[b]{0.48\textwidth}
         \centering
         \includegraphics[width=1.0\textwidth, trim=0 0 0 0, clip]{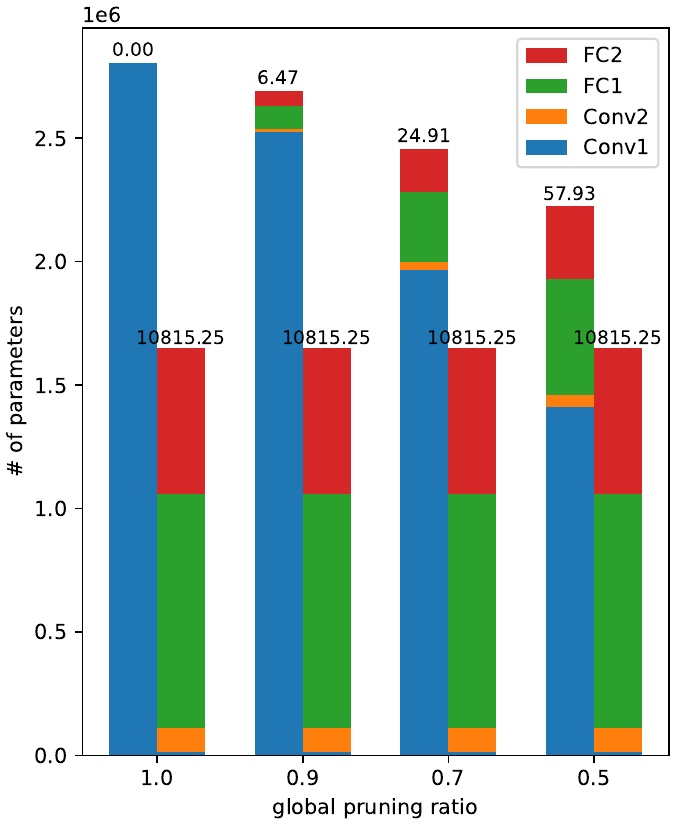}
         \caption{CIFAR10}
     \end{subfigure}
     \hfill
     \begin{subfigure}[b]{0.48\textwidth}
         \centering
         \includegraphics[width=1.0\textwidth, trim=0 0 0 0, clip]{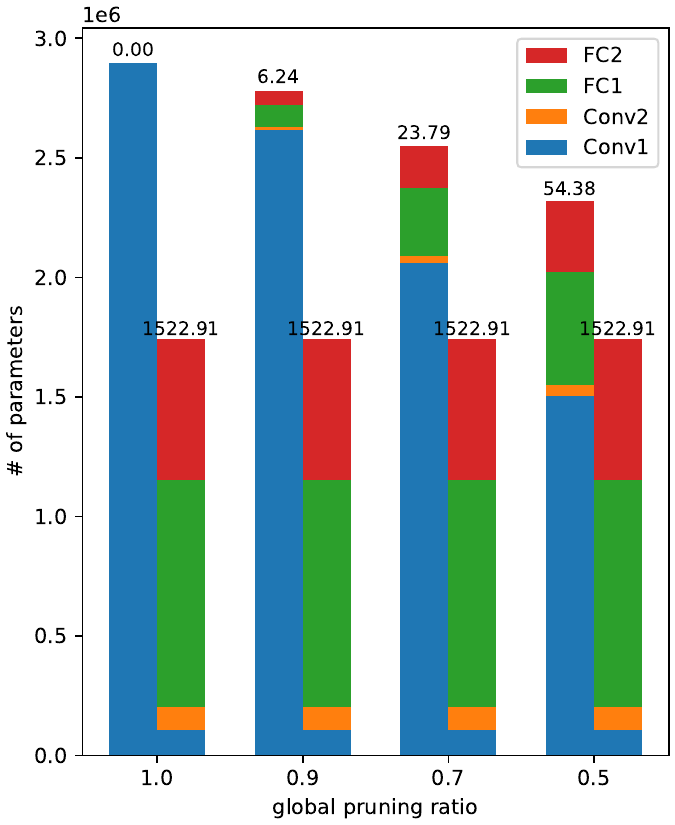}
         \caption{CIFAR100}
     \end{subfigure}
     \hfill
          \centering
     \begin{subfigure}[b]{0.48\textwidth}
         \centering
         \includegraphics[width=1.0\textwidth, trim=0 0 0 0, clip]{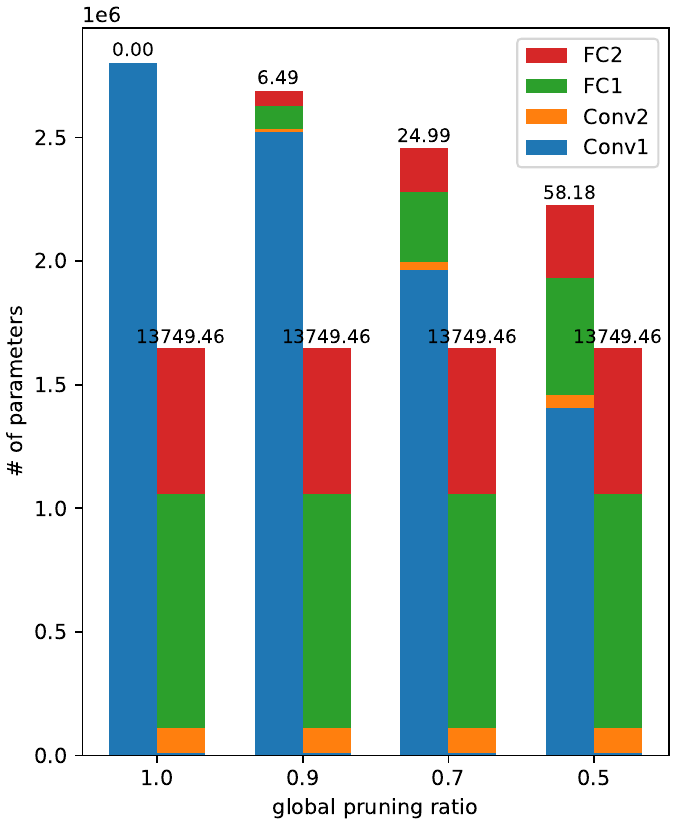}
         \caption{FashionMNIST}
     \end{subfigure}
     \hfill
     \begin{subfigure}[b]{0.48\textwidth}
         \centering
         \includegraphics[width=1.0\textwidth, trim=0 0 0 0, clip]{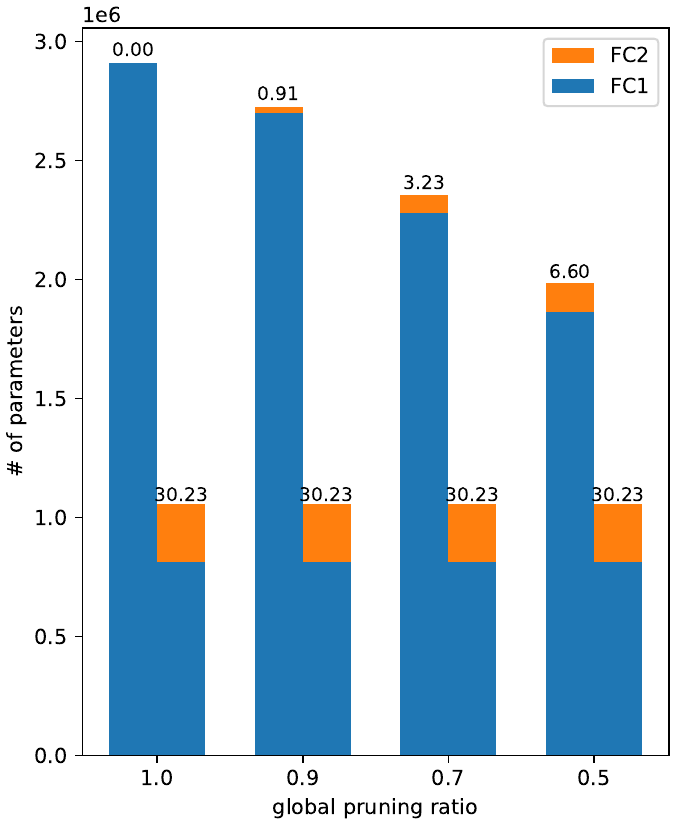}  
         \caption{EMNIST-L}
         \label{fig:parameter_comp_emnistl}
     \end{subfigure}
     \caption{
      The number of parameters across multiple layers, varying according to different global pruning ratios, spans across four distinct datasets.
     For each global pruning ratio, the left side of the bar graph shows the total number of parameters in the model after server-to-client pruning when deployed locally. Conversely, the right side details the communication cost associated with each scenario. 
     Atop each bar, we indicate the relative ratio between the layers with the largest and smallest number of parameters, \textit{i.e.,} $\mathrm{value} = \nicefrac{\mathrm{(largest - smallest)}} {\mathrm{smallest}}$. 
     For (d), since the size of parameters of FC2 and FC3 are the same, we omit plotting FC3 to avoid overlapping.} \label{fig:parameter_comp}
\end{figure}

 
\section{Extended theoretical analysis}
\subsection{Analysis of the general FedP3 theoretical framework}
\begin{algorithm*}[!tb]
\begin{algorithmic}[1] 
\caption{\algname{FedP3} theoretical framework} \label{alg:IST}
    \STATE \textbf{Parameters:} learning rate $\gamma>0$, number of iterations $K$, sequence of global pruning sketches $\left(\mP_1^k, \ldots, \mP_n^k\right)_{k\leq K}$, aggregation sketches $\left(\mS_1^k, \ldots, \mS_n^k\right)_{k\leq K}$; initial model $w^0 \in \mbR^d$  
    \FOR{$k = 0, 1, \cdots, K$}
        \STATE Conduct global pruning $\mP_i^k w^k$ for $i \in [n]$ and broadcast to all computing nodes
        \FOR{$i = 1, \ldots, n$ in parallel}
            \STATE Compute local (stochastic) gradient w.r.t. personalized model: $\mP_i^k \nabla f_i(\mP_i^k w^k)$
            \STATE Take (maybe multiple) gradient descent step $u_i^k = \mP_i^k w^k - \gamma \mP_i^k \nabla f_i(\mP_i^k w^k)$
            \STATE Send $v_i^k = \mS_i^k u_i^k$ to the server 
        \ENDFOR
        \STATE Aggregate received subset of layers: $w^{k+1} = \frac{1}{n} \sum_{i=1}^n v_i^k$
    \ENDFOR
\end{algorithmic}
\end{algorithm*}

We introduce the theoretical foundation of \algname{FedP3}, detailed in Algorithm~\ref{alg:IST}. Line 3 demonstrates the global pruning process, employing a biased sketch over randomized sketches $P_i$ for each client $i \in [n]$, as in \Cref{def:sketch1}. The procedure from Lines 4 to 8 details the local training methods, though we exclude further local pruning for brevity. Notably, our framework could potentially integrate various local pruning techniques, an aspect that merits future exploration.

Our approach uniquely compresses both the weights $w^k$ and their gradients $\nabla f_i(\mP^k_i w^k)$. For the sake of clarity, we assume in Line 5 that each client $i$ calculates the pruned full gradient $\mP_i^k \nabla f_i(\mP_i^k w^k)$, a concept that could be expanded to encompass stochastic gradient computations.

In alignment with Line 6, our subsequent theoretical analysis presumes that each client performs a single-step gradient descent. This assumption stems from observations that local steps have not demonstrated theoretical efficiency gains in heterogeneous environments until very recent studies, such as \cite{ProxSkip} and its extensions like \cite{ProxSkip-VR, Scafflix}, which required extra control variables not always viable in settings with limited resources.

Diverging from the method in \cite{shulgin2023towards}, our model involves explicitly sending a selected subset of layers $v_i^k$ from each client $i$ to the server. The aggregation of these layer subsets is meticulously described in Line 9.

Our expanded theoretical analysis is structured as follows: Section~\ref{sec:theory_model_aggregation} focuses on analyzing the convergence rate of our innovative model aggregation method. In Section~\ref{sec:theory_dp_aggregation}, we introduce \algname{LDP-FedP3}, a novel differential-private variant of \algname{FedP3}, and discuss its communication complexity in a local differential privacy setting. Section~\ref{sec:theory_global_pruning} then delves into the analysis of global pruning, as detailed in Algorithm~\ref{alg:IST}.
  
\subsection{Model aggregation analysis}\label{sec:theory_model_aggregation}
In this section, our objective is to examine the potential advantages of model aggregation and to present the convergence analysis of our proposed \algname{FedP3}. Our subsequent analysis adheres to the standard nonconvex optimization framework, with the goal of identifying an $\epsilon$-stationary point where:
\begin{align}\label{eqn:epsilon_stationary_point}
\ec{\sqN{\nabla f(w)}} \leq \epsilon,
\end{align}

Here, $\ec{\cdot}$ represents the expectation over the inherent randomness in $w\in \Rd$.
Moving forward, our analysis will focus primarily on the convergence rate of our innovative model aggregation strategy. To begin, we establish the smoothness assumption for each local client's model.

\begin{assumption}[Smoothness]\label{asm:smoothness}
    There exists some $L_i\geq 0$, such that for all $i\in [n]$, the function $f_i$ is $L_i$-smooth, i.e.,
    $$
    \norm{\nabla f_i(x) - \nabla f_i(y)} \leq L_i\norm{x - y}, \qquad \forall x, y \in \Rd.
    $$
\end{assumption}

This smoothness assumption is very standard for the convergence analysis~\citep{nesterov2003introductory, ghadimi2013stochastic, ProxSkip, ProxSkip-VR, li2022simple, Scafflix}. The smoothness of function $f$ is $\bar{L} =\avein L_i$, we denote $L_{\max}\eqdef \max_{i\in n} L_i$.

We demonstrate the convergence of our proposed \algname{FedP3}, with a detailed proof presented in Section~\ref{sec:proof_model_aggregation}. Here, we restate Theorem~\ref{thm:model_aggregation} for clarity:
\modelaggregationtheorem*

Next, we interpret the results. Utilizing the inequality $1+w \leq \exp(w)$ and assuming $\gamma \leq \frac{1}{\sqrt{\bar{L}L_{\max}K}}$, we derive the following:
\begin{align*}
(1+\bar{L}L_{\max} \gamma^2)^K \leq \exp(\bar{L}L_{\max} \gamma^2 K) \leq \exp(1) \leq 3.
\end{align*}

Incorporating this into the equation from Theorem~\ref{thm:model_aggregation}, we ascertain:
\begin{align*}
\min_{0\leq k\leq K-1} \ec{\sqN{\nabla f(w^k)}} \leq \frac{6}{\gamma K}\Delta_0.
\end{align*}

To ensure the right-hand side of the above equation is less than $\epsilon$, the condition becomes:
\begin{align*}
\frac{6\Delta_0}{\gamma K} \leq \epsilon \Rightarrow K\geq \frac{6\Delta_0}{\gamma \epsilon}.
\end{align*}

Given $\gamma\leq \frac{1}{\sqrt{\bar{L}L_{\max}K}}$, it follows that $K\geq \frac{36(\Delta_0)^2}{\bar{L}L_{\max}\epsilon^2} = \mathcal{O}\left(\frac{1}{\epsilon^2}\right)$.

Considering the communication cost per iteration is $n\times v = n\times \frac{d}{n} = d$, the total communication cost is:
\begin{align*}
C_{\mathrm{FedP3}} = \mathcal{O}\left(\frac{d}{\epsilon^2} \right).
\end{align*}

We compare this performance with an algorithm lacking our specific model aggregation design, namely Distributed Gradient Descent (\algname{DGD}). When \algname{DGD} satisfies Assumption~\ref{asm:abc_assumption} with $A=C=0, B=1$ as per Theorem~\ref{thm_abc_convergence}, the total iteration complexity to achieve an $\epsilon$-stationary point is $\mathcal{O}\left(\frac{1}{\epsilon}\right)$. Given that the communication cost per iteration is $nd$, the total communication cost for \algname{DGD} is:
\begin{align*}
C_{\mathrm{DGD}} = \mathcal{O}\left( \frac{nd}{\epsilon}\right).
\end{align*}

We observe that the communication cost of \algname{FedP3} is more efficient than \algname{DGD} by a factor of $\mathcal{O}(n/\epsilon)$. This is particularly advantageous in practical Federated Learning (FL) scenarios, where a large number of clients are distributed, highlighting the suitability of our method for such environments. This efficiency also opens avenues for further exploration in large language models.

Although we have demonstrated provable advantages in communication costs for large client numbers, we anticipate that our method's performance exceeds our current theoretical predictions. This expectation is based on the comparison of \algname{FedP3} and \algname{DGD} under Lemma~\ref{lem:l_smooth_bound}. For \algname{DGD}, with parameters $A=\bar{L}, B=C=0$, the iteration complexity aligns with $\mathcal{O}(\frac{1}{\epsilon^2})$, leading to a communication cost of:

\begin{align*}
C_{\mathrm{DGD}}^{\prime} = \mathcal{O}\left(\frac{nd}{\epsilon^2}\right).
\end{align*}

This indicates a significant reduction in communication costs by a factor of $n$ without additional requirements. It implies that if we could establish a tighter bound on $\sqN{\nabla f_i (w)}$, beyond the scope of Lemma~\ref{lem:l_smooth_bound}, our theoretical results could be further enhanced.

\subsection{Differential-private \algname{FedP3} analysis}\label{sec:theory_dp_aggregation}
The integration of gradient pruning as a privacy preservation method was first brought to prominence by \cite{zhu2019deep}. Further studies, such as \cite{huang2020privacy}, have delved into the effectiveness of DNN pruning in protecting privacy.

In our setting, we ensure that our training process focuses on extracting partial features without relying on all layers to memorize local training data. This is achieved by transmitting only a select subset of layers from the client to the server in each iteration. By transmitting fewer layers—effectively implementing greater pruning from clients to the server—we enhance the privacy-friendliness of our framework.

This section aims to provide a theoretical exploration of the "privacy-friendly" aspect of our work. Specifically, we introduce a differential-private version of our method, \algname{LDP-FedP3}, and discuss its privacy guarantees, utility, and communication cost, supported by substantial evidence and rigorous proof.

Local differential privacy is crucial in our context. We aim not only to train machine learning models with reduced communication bits but also to preserve each client's local privacy, an essential element in FL applications. Following the principles of local differential privacy (LDP) as outlined in works like \cite{andres2013geo, chatzikokolakis2013broadening, zhao2020local, li2022soteriafl}, we define two datasets ${D}$ and ${D}^\prime$ as neighbors if they differ by just one entry. We provide the following definition for LDP:

\begin{algorithm*}[!tb]
\begin{algorithmic}[1] 
\caption{Differential-Private FedP3 (\algname{LDP-FedP3})} \label{alg:DP_FedP3}
    \STATE \textbf{Parameters:} learning rate $\gamma>0$, number of iterations $K$, sequence of aggregation sketches $\left(\mS_1^k, \ldots, \mS_n^k\right)_{k\leq K}$, perturbation variance $\sigma^2$, minibatch size $b$
    \FOR{$k = 0, 1, 2 \ldots$}
        \STATE Server broadcasts $w^k$ to all clients
        \FOR{each client $i = 1, \ldots, n$ in parallel}
            \STATE Sample a random minibatch $\mathcal{I}_b$ with size $b$ from lcoal dataset $D_i$
            \STATE Compute local stochastic gradient ${g}_i^k = \frac{1}{b}\sum_{j\in \gI_b} \nabla f_{i, j} (w^k)$
            \STATE Take (maybe multiple) gradient descent step $u_i^k = w^k - \gamma {g}_i^k$
            \STATE Gaussian perturbation to achieve LDP: $\tilde{u}^k_i = u_i^k + \zeta_i^k$, where $\zeta_i^k \sim \gN(\mathbf{0}, \sigma^2\vI)$
            \STATE Send $v_i^k = \mS_i^k \tilde{u}_i^k$ to the server 
        \ENDFOR
        \STATE Server aggregates received subset of layers: $w^{k+1} = \frac{1}{n} \sum_{i=1}^n v_i^k$
    \ENDFOR
\end{algorithmic}
\end{algorithm*}

\begin{definition}\label{def:ldp}
    A randomized algorithm $\mathcal{A}: \mathcal{D}\to \mathcal{F}$, where $\mathcal{D}$ is the dataset domain and $\mathcal{F}$ the domain of possible outcomes, is $(\epsilon, \delta)$-locally differentially private for client $i$ if, for all neighboring datasets ${D}_i, {D}_i^\prime\in \mathcal{D}$ on client $i$ and for all events $\mathcal{S}\in \mathcal{F}$ within the range of $\mathcal{A}$, it holds that:
    \begin{align*}
    \mathrm{Pr}{\mathcal{A}(D_i)\in \mathcal{S}} \leq e^{\epsilon} \mathrm{Pr}{\mathcal{A}(D^\prime_i) \in \mathcal{S}} + \delta.
    \end{align*}
\end{definition}

This LDP definition (Definition \ref{def:ldp}) closely resembles the original concept of $(\epsilon, \delta)$-DP \citep{dwork2014algorithmic, dwork2006calibrating}, but in the FL context, it emphasizes each client's responsibility to safeguard its privacy. This is done by locally encoding and processing sensitive data, followed by transmitting the encoded information to the server, without any coordination or information sharing among clients.

Similar to our previous analysis of \algname{FedP3}, we base our discussion here on the smoothness assumption outlined in Assumption~\ref{asm:smoothness}. For simplicity, and because our primary focus in this section is on privacy concerns, we assume uniform smoothness across all clients, i.e., $L_i \equiv L$.

Our analysis also relies on the bounded gradient assumption, which is a common consideration in differential privacy analyses:

\begin{assumption}[Bounded gradient]\label{asm:bounded_gradient}
There exists some constant $C\geq 0$, such that for all clients $i \in [n]$ and for any $x\in \Rd$, the gradient norm satisfies $\norm{\nabla f_i(x)} \leq C$.
\end{assumption}

This bounded gradient assumption aligns with standard practices in differential privacy analysis, as evidenced in works such as \citep{bassily2014private, wang2017differentially, iyengar2019towards, feldman2020private, li2022soteriafl}.

We introduce a locally differentially private version of \algname{FedP3}, termed \algname{LDP-FedP3}, with detailed algorithmic steps provided in Algorithm~\ref{alg:DP_FedP3}. This variant differs from \algname{FedP3} in Algorithm~\ref{alg:IST} primarily by incorporating the Gaussian mechanism, as per \cite{abadi2016deep}, to ensure local differential privacy (as implemented in Line 8 of Algorithm~\ref{alg:DP_FedP3}). Another distinction is the allowance for minibatch sampling per client in \algname{LDP-FedP3}. Given that our primary focus in this section is on privacy, we set aside the global pruning aspect for now, considering it orthogonal to our current analysis and not central on our privacy considerations. In Theorem~\ref{thm:convergence_dp_fedp3}, we encapsulate the following theorem:
\dpfedpconvergencetheorem*

\begin{table}[!t]
    \centering
    \caption{Comparison of communication complexity in LDP Algorithms for nonconvex problems across distributed settings with $n$ nodes.}\label{tab:results}
    \renewcommand{\arraystretch}{2}
    \resizebox{\textwidth}{!}{
    \begin{tabular}{ccc}
    \toprule
     \bf Algorithm  & \bf Privacy & \bf Communication Complexity \\ \midrule
     \algname{Q-DPSGD} {\small \citep{ding2021differentially}} & $(\epsilon,\delta)$-LDP & $\frac{(1 +{n}/(m{\tilde{\sigma}^2})) m^2\epsilon^2}{d \log(1/\delta)}$\\ \midrule
     \algname{LDP SVRG/SPIDER} {\small \citep{lowy2023private}} & $(\epsilon, \delta)$-LDP & $\frac{n^{3/2}m\epsilon\sqrt{d}}{\sqrt{\log(1/\delta)}}$\\ \midrule
     \algname{SDM-DSGD} {\small \citep{zhang2020private}} & $(\epsilon, \delta)$-LDP & $\frac{{n^{7/2}} m\epsilon\sqrt{d}}{{(1+\omega)^{3/2}}\sqrt{\log(1/\delta)}} + \frac{{n} m^2\epsilon^2}{{(1+\omega)}\log(1/\delta)}$\\ \midrule
     \algname{CDP-SGD} {\small \citep{li2022soteriafl}} & $(\epsilon, \delta)$-LDP & $\frac{{n^{3/2}} m\epsilon\sqrt{d}}{{(1+\omega)^{3/2}}\sqrt{\log(1/\delta)}} + \frac{{n} m^2\epsilon^2}{{(1+\omega)}\log(1/\delta)}$\\ \midrule
     \cellcolor{bgcolor4} \algname{LDP-FedP3} (Ours) & \cellcolor{bgcolor4}$(\epsilon, \delta)$-LDP & \cellcolor{bgcolor4} $\frac{ m\epsilon\sqrt{d}}{\sqrt{\log(1/\delta)}} + \frac{m^2\epsilon^2}{\log(1/\delta)}$\\ \bottomrule
    \end{tabular}}
\end{table}

In Section~\ref{sec:proof_dp_fedp3_convergence_analysis}, we provide the proof for our analysis. This section primarily focuses on analyzing and comparing our results with existing literature. Our proof pertains to local differentially-private Stochastic Gradient Descent (SGD). We note that \cite{li2022soteriafl} offered a proof for \algname{CDP-SGD} using a specific set of compressors. However, our chosen compressor does not fall into that category, as discussed more comprehensively in \cite{szlendak2021permutation}. Considering the Rand-t compressor with $t=d/n$, it's established that:
\begin{align*}
\ec{\sqN{\mathcal{R}_t(w) - w}} \leq \omega \sqN{w}, \quad \text{where} \quad \omega = \frac{d}{t} - 1 = n - 1.\nonumber
\end{align*}

Setting the same $K$ and $\gamma$ and applying Theorem 1 from \cite{li2022soteriafl}, we obtain:
\begin{align*}
\frac{1}{K}\sum_{k=0}^{K-1} \ec{\sqn{\nabla f(w^t)}} \leq \frac{5C\sqrt{Lcd\log(1/\sigma)}}{m\epsilon} = \mathcal{O}\left(\frac{C\sqrt{Ld\log(1/\delta)}}{m\epsilon} \right),
\end{align*}

which aligns with our theoretical analysis. Interestingly, we observe that our bound is tighter by a factor of $2/5$, indicating a more efficient performance in our approach.

We also compare our proposed \algname{LDP-FedP3} with other existing algorithms in Algorithm~\ref{tab:results}. An intriguing finding is that our method's efficiency does not linearly increase with a higher number of clients, denoted as $n$. Notably, our communication complexity remains independent of $n$. This implies that in practical scenarios with a large $n$, our communication costs will not escalate. We then focus on methods with a similar structure, namely, \algname{SDM-DSGD} and \algname{CDP-SGD}. For these, the communication cost comprises two components. Considering a specific case, \texttt{Rand-t}, where $t$ is deliberately set to $d/n$, we derive $\omega = \nicefrac{d}{t} - 1 = n - 1$. This results in a communication complexity on par with \algname{CDP-SGD}, but significantly more efficient than \algname{SDM-DSGD}. Moreover, it's important to note that the compressor in \algname{LDP-FedP3} differs from that in \algname{CDP-SGD}. Our analysis introduces new perspectives and achieves comparable communication complexity to other well-established results.


\subsection{Global pruning analysis}\label{sec:theory_global_pruning}
Our methodology relates to independent subnetwork training (IST) but introduces distinctive features such as personalization and explicit layer-level sampling for aggregation. IST, although conceptually simple, remains underexplored with only limited studies like \cite{liao2022convergence}, which provides theoretical insights for overparameterized single hidden layer neural networks with ReLU activations, and \cite{shulgin2023towards}, which revisits IST from the perspective of sketch-type compression. In this section, we delve into the nuances of global pruning as applied in Algorithm~\ref{alg:IST}.

For our analysis here, centered on global pruning, we simplify by assuming that all personalized model aggregation sketches $\mS_i$ are identical matrices, that is, $\mS_i = \vI$. This simplification, however, does not trivialize the analysis as the pruning of both gradients and weights complicates the convergence analysis. Additionally, we adhere to the design of the global pruning sketch $\mP$ as per Definition~\ref{def:sketch1}, which results in a biased estimation, i.e., $\mathbb{E}[\mP_i w]\neq w$. Unbiased estimators, such as \texttt{Rand-t} that operates over coordinates, are more commonly studied and offer several advantages in theoretical analysis.

For \texttt{Rand-t}, consider a random subset $\mathcal{S}$ of $[d]$ representing a proper sampling with probability $c_j\eqdef \mathrm{Prob}(j\in \mathcal{S}) > 0$ for every $j\in [d]$. $\mathcal{R}_t \eqdef \Diag(r_s^1, r_s^2, \cdots, r_s^d)$, where $r_s^j = \nicefrac{1}{c_j}$ if $j\in \mathcal{S}$ and $0$ otherwise. In contrast to our case, the value on each selected coordinate in \texttt{Rand-t} is scaled by the probability $p_i$, equivalent to $|\mathcal{S}|/d$. However, the implications of using a biased estimator like ours are not as well understood.

Our theoretical focus is on FL in the context of empirical risk minimization, formulated in (\ref{eqn:objective1}) within quadratic problem frameworks. This setting involves symmetric matrices $\mL_i$, as defined in the following equation:

\begin{equation}\label{eqn:quadratic_objective}
\begin{aligned}
f(w) = \avein f_i(w), \quad \text{where} \quad f_i(w) \equiv \frac{1}{2} w^\top \mL_i w - w^\top b_i.
\end{aligned}
\end{equation}

While Equation~\ref{eqn:quadratic_objective} simplifies the loss function, the quadratic problem paradigm is extensively used in neural network analysis \citep{zhang2019algorithmic, zhu2022quadratic, shulgin2023towards}. Its inherent complexity provides valuable insights into complex optimization algorithms \citep{arjevani2020tight, cunha2022only, goujaud2022super}, thereby serving as a robust model for both theoretical examination and practical applications. In this framework, $f(x)$ is $\moL$-smooth, and $\nabla f(x) = \moL x - \ob$, where $\moL = \avein \mL_i$, and $\ob \eqdef \avein b_i$.

At this juncture, we introduce a fundamental assumption commonly applied in the theoretical analysis of coordinate descent-type methods.

\begin{assumption}[Matrix Smoothness]\label{ass:matrix-smoothness}
Consider a differentiable function \( f: \mathbb{R}^d \rightarrow \mathbb{R} \). We say that \( f \) is \( \mathbf{L} \)-smooth if there exists a positive semi-definite matrix \( \mathbf{L} \in \mathbb{R}^{d \times d} \) satisfying the following condition for all \( x, h \in \mathbb{R}^d \):
\begin{equation}\label{eq:L-matrix-smooth}
f(x + h) \leq f(x) + \langle \nabla f(x), h \rangle + \frac{1}{2} \langle \mathbf{L}h, h \rangle.
\end{equation}
\end{assumption}

The classical \( L \)-smoothness condition, where \( \mathbf{L} = L \cdot \mathbf{I} \), is a particular case of Equation~\eqref{eq:L-matrix-smooth}. The concept of matrix smoothness has been pivotal in the development of gradient sparsification methods, particularly in scenarios optimizing under communication constraints, as shown in \cite{safaryan2021smoothness, wang2022theoretically}. We then present our main theory under the interpolation regime for a quadratic problem (\ref{eqn:quadratic_objective}) with \( b_i \equiv 0 \), as detailed in Theorem~\ref{thm:global_pruning}. 

We first provide the theoretical analysis of biased global pruning as implemented in Algorithm~\ref{alg:DP_FedP3}. To the best of our knowledge, biased gradient estimators have rarely been explored in theoretical analysis. However, our approach of intrinsic submodel training or global pruning is inherently biased. \cite{shulgin2023towards} proposed using the Perm-K~\citep{szlendak2021permutation} as the global pruning sketch. Unlike their approach, which assumes a pruning connection among clients, our method considers the biased Rand-K compressor over coordinates. 

\begin{theorem}[Global pruning]\label{thm:global_pruning}
    In the interpolation regime for a quadratic problem~(\ref{eqn:quadratic_objective}) with \( \moL \succ 0 \) and \( b_i \equiv 0 \), let \( \moL^k \eqdef \avein \mP_i^k \moL \mP_i^k \). Assume that \( \moW \eqdef \frac{1}{2}\mathbb{E}[\mP^k \moL\moB^k + \mP^k\moB^k \moL] \succeq 0 \) and there exists a constant \( \theta >0 \) such that \( \mathbb{E}[\moB^k \moL\moB^k]\preceq \theta \moW \). Also, assume \( f(\mP^kw^k)\leq (1+\gamma^2 h) f(w^k) - f^{\inf} \) for some \( h>0 \). Fixing the number of iterations \( K \) and choosing the step size \( \gamma \in \min\left\{\sqrt{\frac{\log2}{hK}}, \frac{1}{\theta} \right\} \), the iterates satisfy:
    $$
    \mathbb{E}\left[\|\nabla f(w^k)\|^2_{\moL^{-1}\moW\moL^{-1}}\right] \leq  \frac{4\Delta_0}{\gamma K},
    $$
    where $\Delta_0 = f(w^0) - f^{\inf}$.
\end{theorem}

By employing the definition of $\gamma$, we demonstrate that the iteration complexity is $\mathcal{O}(1/\epsilon^2)$.
Compared with the analysis in~\cite{shulgin2023towards}, we allow personalization and do not constrain the global pruning per client to be dependent on other clients. Global pruning is essentially a biased estimator over the global model weights, a concept not widely understood. Our theorem provides insightful perspectives on the convergence of global pruning.

Our theory could also extend to the general case by applying the rescaling trick from Section 3.2 in \cite{shulgin2023towards}. This conversion of the biased estimator to an unbiased one leads to a general convergence theory. However, this is impractical for realistic global pruning analysis, as it involves pruning the global model without altering each weight's scale. Given that IST and biased gradient estimators are relatively new in theoretical analysis, we hope our analysis could provide some insights.

\section{Missing proofs}
\subsection{Proof of Theorem \ref{thm:model_aggregation}}\label{sec:proof_model_aggregation}
Building on the smoothness assumption of $L_i$ outlined in Assumption~\ref{asm:smoothness}, the following lemma is established:

\begin{lemma}\label{lem:l_smooth_bound}
    Given that a function $f_i$ satisfies Assumption~\ref{asm:smoothness} for each $i\in [n]$, then for any $w\in \Rd$, it holds that 
    \begin{align}\label{eqn:l_smooth_bound}
        \sqN{\nabla f_i(w)} \leq 2L_i (f_i(w) - f^{\inf}).
    \end{align}
\end{lemma}

\begin{proof}
    Consider $w^\prime = w - \frac{1}{L_i} \nabla f_i(w)$. By applying the $L_i$-smoothness condition of $f$ as per Assumption~\ref{asm:smoothness}, we obtain
    \begin{align*}
        f_i(w^\prime) &\leq f_i(w) + \langle \nabla f_i(w), w^\prime - w \rangle + \frac{L_i}{2}\|\nabla f_i(w)\|^2.
    \end{align*}
    Taking into account that $f^{\inf} \leq f_i(w^\prime)$, it follows that
    \begin{align*}
        f^{\inf} &\leq f_i(w^\prime) \\
        &\leq f_i(w) - \frac{1}{L_i}\|\nabla f_i(w)\|^2 + \frac{1}{2L_i}\|\nabla f_i(w)\|^2 \\
        &= f_i(w) - \frac{1}{2L_i}\|\nabla f_i(w)\|^2.
    \end{align*}
    Rearranging the terms yields the claimed result.
\end{proof}

Since in this section, we are primarily interested in exploring the convergence of our novel model aggregation design, we set $\mP_i^k \equiv \vI$ for all $i\in [n]$ and $k\in [K]$. Our analysis focuses on exploring the characteristics of $\mS$, which leads to the following theorem.

By the definition of model aggregation sketches in Definition~\ref{def:sketch2}, we have $\avein \mS_i = \vI$. Thus, the next iterate can be represented as 
\begin{align}
    w^{k+1} &= \avein \mS_i^k (w^k - \gamma \nabla f_i(w^k))\nonumber\\
    &= \avein \mS_i^k w^k - \gamma \underbrace{\avein \mS_i^k \nabla f_i(w^k)}_{g^k}\label{eqn:next_iterate}\\
    &= w^k - \gamma g^k.\nonumber
\end{align}

Bounding $g^k$ is a crucial part of our analysis. To align with existing works on non-convex optimization, numerous critical assumptions are considered. Extended reading on this can be found in \cite{khaled2020better}. Here, we choose the weakest assumption among all those listed in \cite{khaled2020better}.

\begin{assumption}[ABC Assumption]\label{asm:abc_assumption}
    For the second moment of the stochastic gradient, it holds that 
    \begin{align}\label{eqn:abc_assumption}
        \mathbb{E}\left[\|\mathbf{g}(w)\|^2\right] \leq 2A(f(w) - f^{\inf}) + B\|\nabla f(w)\|^2 + C,
    \end{align}
    for certain constants $A, B, C \geq 0$ and for all $w\in \mathbb{R}^d$.
\end{assumption}

Note that in order to accommodate heterogeneous settings, we assume a localized version of Assumption~\ref{asm:abc_assumption}. Specifically, each $g_i^k \equiv \mS_i^k \nabla f_i(w^k)$ is bounded for some constants $A_i, B_i, C_i \geq 0$ and all $w^k \in \mathbb{R}^d$.

\begin{lemma}\label{lem:individualized_abc}
    The $g^k$ defined in Eqn.~\ref{eqn:next_iterate} satisfies Assumption~\ref{asm:abc_assumption} with $A=L_{\max}$, $B=C = 0$.
\end{lemma}

\begin{proof}
    The proof is as follows:
    \begin{align}
    \mathbb{E}_k\left[\|g^k\|^2\right] &= \mathbb{E}_k\left[\|\avein S_i\nabla f_i(w^k)\|^2\right] \nonumber \\
    &= \avein \|\nabla f_i(w^k)\|^2 \nonumber \\
    &\leq \avein 2L_i(f_i(w^k) - f^{\inf}) \label{eqn:0020}\\
    &\leq 2L_{\max} (f(w^k) - f^{\inf}),\nonumber
    \end{align}
    where Equation~\ref{eqn:0020} follows from Lemma~\ref{lem:l_smooth_bound}. 
\end{proof}

We also recognize certain characteristics of the unbiasedness and upper bound of model aggregation sketches, as elaborated in Theorem~\ref{thm:unbiased_second_moment_aggregation_sketch}.

\begin{theorem}[Unbiasedness and Upper Bound of Model Aggregation Sketches]\label{thm:unbiased_second_moment_aggregation_sketch}
    For any vector $w \in \mathbb{R}^d$, the model aggregation sketch $\mS_i$, for each $i \in [n]$, is unbiased, meaning $\mathbb{E}[\mS_i w] = w$. Moreover, for any set of vectors $y_1, y_2, \ldots, y_n \in \mathbb{R}^d$, the following inequality is satisfied:
    \begin{align*}
        \mathbb{E}\left[\left\|\avein \mS_i y_i\right\|^2\right] \leq \avein \left\|y_i\right\|^2.
    \end{align*}
\end{theorem}

\begin{proof}
    Consider a vector $x \in \mathbb{R}^d$, where $x_i$ denotes the $i$-th element of $x$. We first establish the unbiasedness of the model aggregation sketch (Definition~\ref{def:sketch1}):
    
    \begin{equation}\label{eqn:unbiased_aggregation_sketch}
        \mathbb{E}[\mS_i x] = n \sum_{j=q(i-1)+1}^{qi} \mathbb{E}[x_{\pi_j}e_{\pi_j}] = n\left(\sum_{j=q(i-1)+1}^{qi} \frac{1}{d}\sum_{i=1}^d x_ie_i \right) = \frac{nq}{d} x = x.
    \end{equation}

    Next, we examine the second moment:
    $$
    \mathbb{E}\left[\|\mS_i x\|^2\right] = n^2\sum_{j=q(i-1)+1}^{qi} \aveid \left\|x_i\right\|^2 = n^2\frac{q}{d}\left\|x\right\|^2 = n\left\|x\right\|^2.
    $$
    For all vectors $y_1, y_2, \ldots, y_n \in \mathbb{R}^d$, the following inequality holds:

    \begin{equation}\label{eqn:second_moment_aggregation_sketch}
        \begin{aligned}
            \mathbb{E}\left[\left\|\avein \mS_i y_i\right\|^2\right] &= \frac{1}{n^2}\sumin \mathbb{E}\left[\left\|\mS_i y_i\right\|\right] + \sum_{i\neq j} \mathbb{E}\left[\langle \mS_i y_i, \mS_j y_j\rangle\right]\\
            &= \frac{1}{n^2}\sumin \mathbb{E}\left[\left\|\mS_i y_i\right\|\right]\\
            &= \avein \left\|y_i\right\|^2.
        \end{aligned}
    \end{equation}

    Integrating Equation~\ref{eqn:unbiased_aggregation_sketch} with Equation~\ref{eqn:second_moment_aggregation_sketch}, we also deduce:

    \begin{equation}\label{eqn:second_moment_aggregation_sketch2}
        \begin{aligned}
            \mathbb{E}\left[\left\|\avein \mS_i y_i - \avein y_i\right\|^2\right] \leq \avein \left\|y_i\right\|^2 - \left\|\avein y_i\right\|^2.
        \end{aligned}
    \end{equation}
\end{proof}

We now proceed to prove the main theorem of model aggregation, as presented in Theorem~\ref{thm:model_aggregation}. This theorem is restated below for convenience:
\modelaggregationtheorem*

Our proof draws inspiration from the analysis in Theorem 2 of \cite{khaled2020better} and is reformulated as follows:

\begin{theorem}[Theorem 2 in \cite{khaled2020better}]\label{thm_abc_convergence}
    Under the assumptions that Assumption~\ref{asm:smoothness} and \ref{asm:abc_assumption} are satisfied, let us choose a step size $\gamma > 0$ such that $\gamma \leq \frac{1}{\bar{L}B}$. Define $\Delta \equiv f(w^0) - f^{\inf}$. Then, it holds that
    \begin{align*}
        \min_{0\leq k\leq K-1}\mathbb{E}\left[\|\nabla f(w^k)\|^2\right] \leq \bar{L}C\gamma + \frac{2(1+\bar{L}\gamma^2 A)^K}{\gamma K}\Delta.
    \end{align*}
\end{theorem}

Careful control of the step size is crucial to prevent potential blow-up of the term and to ensure convergence to an $\epsilon$-stationary point. Our theory can be seen as a special case with $A=L_{\max}, B=0, C=0$, as established in Lemma~\ref{lem:individualized_abc}. Thus, we conclude our proof.

\subsection{Proof of Theorem~\ref{thm:convergence_dp_fedp3}}
\label{sec:proof_dp_fedp3_convergence_analysis}

To establish the convergence of the proposed method, we begin by presenting a crucial lemma which describes the mean and variance of the stochastic gradient. Consider the stochastic gradient \(g_i^k = \frac{1}{b}\sum_{j\in\mathcal{I}_b} \nabla f_{i, j}(w^k)\) as outlined in Line 6 of Algorithm~\ref{alg:DP_FedP3}.

\begin{lemma}[Lemma 9 in~\cite{li2022soteriafl}]
\label{lem:g_variance}
    Given Assumption~\ref{asm:bounded_gradient}, for any client \(i\), the stochastic gradient estimator \(g_i^k\) is an unbiased estimator, that is,
    \begin{align*}
        \mathbb{E}_k\left[\frac{1}{b}\sum_{j\in\mathcal{I}_b}\nabla f_{i, j}(w^k)\right] = \nabla f_i (w^k),
    \end{align*}
    where \(\mathbb{E}_k\) denotes the expectation conditioned on all history up to round \(k\). Letting \(q = \frac{b}{m}\), the following inequality holds:
    \begin{align*}
    \mathbb{E}_k \left[\left\|\frac{1}{b}\sum_{j\in\mathcal{I}_b} \nabla f_{i, j}(w^k) - \nabla f_i(w^k)\right\|^2  \right] \leq \frac{(1 - q)C^2}{b}.
    \end{align*}
\end{lemma}

Considering the definition of \(\mathcal{S}_i^k\), we observe that \(\frac{1}{n} \sum_{i=1}^n \mathcal{S}_i^k = \mathbf{I}\). According to Algorithm~\ref{alg:DP_FedP3}, the next iteration \(w^{k+1}\) of the global model is given by:
$$
w^{k+1} = \frac{1}{n} \sum_{i=1}^n \mathcal{S}_i^k \left(w^k - \gamma g^k_i + \zeta_i^k\right) = w^k - \underbrace{\frac{1}{n} \sum_{i=1}^n \mathcal{S}_i^k (\gamma g_i^k - \zeta_i^k)}_{G^k}.
$$

Employing the smoothness Assumption~\ref{asm:smoothness} and taking expectations, we derive:
\begin{equation}
\label{eqn:0000}
    \begin{aligned}
    \mathbb{E}_k[f(w^{k+1})] &\leq f(w^k) - \mathbb{E}_k \left\langle \nabla f(w^k), G^k \right\rangle + \frac{L}{2} \mathbb{E}_k\left\|G^k\right\|^2.
    \end{aligned}
\end{equation}

Given that \(\zeta_i^k \sim \mathcal{N}(\mathbf{0}, \sigma^2\mathbf{I})\), we have \(\mathbb{E}_k[\zeta_i^k] = 0\). Consequently, we can analyze \(\mathbb{E}_k \langle \nabla f(w^k), G^k \rangle\) as follows:
\begin{align}
    \mathbb{E}_k \langle \nabla f(w^k), G^k \rangle &= \mathbb{E}_k \left\langle \nabla f(w^k), \frac{1}{n} \sum_{i=1}^n \mathcal{S}_i^k(\gamma g_i^k - \zeta_i^k) \right\rangle \nonumber\\
    &\stackrel{(\ref{eqn:unbiased_aggregation_sketch})}{=} \mathbb{E}_k \left\langle \nabla f(w^k), \frac{1}{n} \sum_{i=1}^n (\gamma g_i^k - \zeta_i^k) \right\rangle \nonumber\\
    &= \mathbb{E}_k \left\langle \nabla f(w^k), \gamma \frac{1}{n} \sum_{i=1}^n g_i^k \right\rangle \nonumber\\
    &\stackrel{(\ref{lem:g_variance})}{=} \gamma \left\|\nabla f(w^k)\right\|^2. \label{eqn:0003}
\end{align}

To bound the last term \(\mathbb{E}_k \left\|G^k\right\|^2\) in Equation~\ref{eqn:0000}, we proceed as follows:

\begin{align}
    \mathbb{E}_k \left\|G^k\right\|^2 &= \mathbb{E}_k \left\|\frac{1}{n} \sum_{i=1}^n \mathcal{S}_i^k (\underbrace{\gamma g_i^k - \zeta_i^k}_{M_i^k})\right\|^2 \nonumber\\
    &\stackrel{(\ref{eqn:second_moment_aggregation_sketch})}{\leq} \frac{1}{n} \sum_{i=1}^n \mathbb{E}_k\left\|M_i^k\right\|^2 \nonumber\\
    &= \frac{1}{n} \sum_{i=1}^n \mathbb{E}_k\left\|\gamma g_i^k - \zeta_i^k\right\|^2 \nonumber\\
    &= \frac{1}{n} \sum_{i=1}^n \mathbb{E}_k\left\|\gamma g_i^k\right\|^2 + d\sigma^2 \nonumber\\
    &= \gamma^2 \frac{1}{n} \sum_{i=1}^n \mathbb{E}_k \left\|g_i^k - \nabla f_i(w^k) + \nabla f_i(w^k)\right\|^2 + d\sigma^2 \nonumber\\
    &\leq \frac{1}{n} \sum_{i=1}^n \gamma^2\left\|\nabla f_i (w^k)\right\|^2 + \gamma^2 \frac{1}{n} \sum_{i=1}^n \mathbb{E}_k\left\|g_i^k - \nabla f_i(w^k)\right\|^2 + d\sigma^2 \nonumber\\
    &\stackrel{(\ref{lem:g_variance}, \ref{asm:bounded_gradient})}{\leq} \gamma^2 C^2 + \frac{\gamma^2(1-q)C^2}{b} + d\sigma^2. \label{eqn:0002}
\end{align}

Incorporating Equations~\ref{eqn:0002} and \ref{eqn:0003} into Equation~\ref{eqn:0000}, we obtain the following inequality for the expected function value at the next iteration:

\begin{align}
    \mathbb{E}_k [f(w^{k+1})] &\leq f(w^k) - \gamma \left\|\nabla f(w^k)\right\|^2 + \frac{L}{2}\left(\gamma^2 C^2 + \frac{\gamma^2(1-q)C^2}{b} + d\sigma^2 \right).
\end{align}

Before proceeding further, it is pertinent to consider the privacy guarantees of \(\texttt{FedP3}\), which are based on the analysis of \algname{SoteriaFL} as presented in Theorem 2 of \cite{li2022soteriafl}. We reformulate this theorem as follows:

\begin{theorem}[Theorem 2 in \cite{li2022soteriafl}]
\label{thm:dp_fedp3_privacy}
    Assume each client possesses $m$ data points. Under Assumption 3 in \cite{li2022soteriafl} and given two bounding constants $C_A$ and $C_B$ for the decomposed gradient estimator, there exist constants $c$ and $c^\prime$. For any $\epsilon < c^\prime \frac{b^2 T}{m^2}$ and $\delta \in (0,1)$, \algname{SoteriaFL} satisfies $(\epsilon, \delta)$-Local Differential Privacy (LDP) if we choose
    $$
    \sigma_p^2 = \frac{c\left(C_A^2 / 4 + C_B^2\right) K \log (1 / \delta)}{m^2 \epsilon^2}.
    $$
\end{theorem}

In the absence of gradient shift consideration within \algname{SoteriaFL}, the complexity of the gradient estimator can be reduced. We simplify the analysis by substituting the two bounds $C_A$ and $C_B$ with a single constant $C$. Following a similar setting, we derive the privacy guarantee for \algname{LDP-FedP3} as:
\begin{align} 
    \sigma^2 = \frac{cC^2K\log(1/\delta)}{m^2\epsilon^2},\label{eqn:dp_fedp3_privacy_guarantee}
\end{align}
which establishes that \algname{LDP-FedP3} is $(\epsilon, \delta)$-LDP compliant under the above condition.

Substituting \(\sigma\) from Equation~\ref{eqn:dp_fedp3_privacy_guarantee} and telescoping over iterations $k = 1, \ldots, K$, we can demonstrate the following convergence bound:
\begin{align*}
    \frac{1}{K}\sum_{k=1}^K\mathbb{E}\left[\left\|\nabla f(w^k)\right\|^2\right] &\leq \frac{f(w^0) - f^\star}{\gamma K} + \frac{L}{2}\left[ \gamma C^2 + \frac{\gamma(1-q)C^2}{b} + \frac{cdC^2T\log(1/\delta)}{\gamma m^2\epsilon^2} \right] \\
    &\leq \frac{\Delta_0}{\gamma K} + \frac{L}{2}\left[\frac{\gamma(b+ 1 - q)}{b}C^2 + \frac{cd C^2 K\log(1/\delta)}{\gamma m^2\epsilon^2}\right] \\
    &\leq \frac{\Delta_0}{\gamma K} + \frac{L}{2}\left[\gamma C^2 + \frac{cd C^2 K\log(1/\delta)}{\gamma m^2\epsilon^2}\right].
\end{align*}

To harmonize our analysis with existing works, such as \(\texttt{CDP-SGD}\) proposed by \cite{li2022soteriafl}, which compresses the gradient and performs aggregation on the server over the gradients instead of directly on the weights, we reframe Algorithm~\ref{alg:DP_FedP3} accordingly. The primary modification involves defining $M_i^k \eqdef \gamma g_i^k - \gamma \zeta_i^k$, where $\zeta_i^k$ is scaled by a factor of $\gamma$. This leads to the following convergence result:

\begin{align}\label{eqn:potato}
    \frac{1}{K}\sum_{k=1}^K\mathbb{E}\left[\left\|\nabla f(w^k)\right\|^2\right] &\leq \frac{\Delta_0}{\gamma K} + \frac{\gamma LC^2}{2}\left[1 + \frac{cd K\log(1/\delta)}{m^2\epsilon^2}\right].
\end{align}

Optimal choices for $K$ and $\gamma$ that align with this convergence result can be defined as:
\begin{align}
\gamma K &= \frac{m\epsilon \sqrt{\Delta_0}}{C\sqrt{Lcd\log(1/\delta)}}, \quad 
K \geq \frac{m^2\epsilon^2}{cd \log\left(1/\delta\right)}. \label{eqn:K-cdpsgd}
\end{align}

Adhering to the relationship established in Equation~\eqref{eqn:K-cdpsgd} and considering the stepsize constraint $\gamma \leq \frac{1}{L}$, we define:
\begin{align*}
    K &= \max\left\{\frac{m\epsilon \sqrt{L \Delta_0}}{C\sqrt{cd\log(1/\delta)}}, \frac{m^2\epsilon^2}{cd\log(1/\delta)}\right\}, \\
    \gamma &= \min\left\{\frac{1}{L}, \frac{\sqrt{\Delta_0 cd \log(1/\delta)}}{C m\epsilon\sqrt{L}}\right\}.
\end{align*}

Substituting these into Equation~\ref{eqn:potato}, we obtain:
\begin{align*}
    \frac{1}{K}\sum_{t=1}^K \mathbb{E}\left[\left\|\nabla f(x^t)\right\|^2\right] &\leq \frac{\Delta_0}{\gamma K} + \frac{\gamma LC^2}{2}\left[1 + \frac{cdK\log(1/\delta)}{m^2\epsilon^2} \right]\\
    &\leq \frac{\Delta_0}{\gamma K} + \frac{\gamma LC^2cdK\log(1/\delta)}{m^2\epsilon^2}\\
    &= \frac{\Delta_0}{\gamma K} + \frac{\gamma K LC^2cd\log(1/\delta)}{m^2\epsilon^2}\\
    &\leq \frac{2C\sqrt{Lcd\log(1/\delta)}}{m\epsilon}\\
    &= \mathcal{O}\left(\frac{C\sqrt{Ld\log(1/\delta)}}{m\epsilon} \right).
\end{align*}  
\label{eqn:0012}

Neglecting the constant $c$, the total communication cost for \(\texttt{LDP-FedP3}\) is computed as:
\begin{align*}
    C_{\text{LDP-FedP3}} &= n \frac{d}{n} K = dK \\
    &= \max\left\{\frac{m\epsilon \sqrt{dL \Delta_0}}{C\sqrt{\log(1/\delta)}}, \frac{m^2\epsilon^2}{\log(1/\delta)}\right\} \\
    &= \mathcal{O}\left( \frac{m\epsilon \sqrt{dL \Delta_0}}{C\sqrt{\log(1/\delta)}} + \frac{m^2\epsilon^2}{\log(1/\delta)}\right).
\end{align*}

\subsection{Proof of Theorem~\ref{thm:global_pruning}}
We consider the scenario where $\mP_i^k$ acts as a biased random sparsifier, and $\mS_i^k\equiv \mI$. In this case, the update rule is given by:
$$
w^{k+1} = \avein\left(\mP_i^kw^k  - \gamma \mP_i^k \nabla f_i(\mP_i^k w^k)\right). 
$$

Let $w\in \Rd$ and let $S$ represent the selected number of coordinates from $d$. Then, $\mP_i$ is defined as:
$$
\mP_i = \Diag(c_s^1, c_s^2, \cdots, c_s^d), \quad \text{where} \quad c_s^j = \begin{cases}
    1 & \text{if } j \in S,\\
    0 & \text{if } j \notin S.
\end{cases}
$$

Given that $\mP_i\preceq \mI$, it follows that $\avein \mP_i \preceq \mI$.

In the context where $\mP_i$ is a biased sketch, we introduce Assumption~\ref{asm:biased_bound}:

\begin{assumption}\label{asm:biased_bound}
    For any learning rate $\gamma > 0$, there exists a constant $h > 0$ such that, for any $\mP\in \mbR^{d\times d}$, $w\in \Rd$, we have:
    \begin{align*}
        f(\mP w) \leq (1+\gamma^2 h) (f(w) - f^{\inf}).
    \end{align*}
\end{assumption}

Assumption~\ref{asm:biased_bound} assumes the pruning sketch is bounded. Given that the function value should remain finite, this assumption is reasonable and applicable.

In this section, for simplicity, we focus on the interpolation case where $f_i(x) = \frac{1}{2}w^\top \mL_i w$. The extension to scenarios with $b_i \neq 0$ is left for future work. By leveraging the $\moL$-smoothness of function $f$ and the diagonal nature of $\mP_i$, we derive the following:

\begin{equation}\label{eqn:main0}
\begin{aligned}
    f(w^{k+1}) &\eqdef f\left(\avein (\mP_i^k w^k - \gamma \mP_i^k \nabla f_i(\mP_i^k w^k))\right)\\
    &= f\left(\underbrace{\avein \mP_i^k}_{\mP^k} w^k - \gamma \underbrace{\avein\mP_i^k \moL_i \mP_i^k}_{\moB^k} w^k\right)\\
    &\leq f(\mP^k w^k) - \gamma\langle\nabla f(\mP^k w^k), \moB^k w^k\rangle + \frac{\gamma^2}{2}\norm{\moB^k w^k}^2_{\moL}\\
    &\stackrel{(\ref{asm:biased_bound})}{\leq} a f(w^k) - \gamma \langle\moL \mP^k w^k, \moB^k w^k\rangle + \frac{\gamma^2}{2}\norm{\moB^k w^k}^2_{\moL}\\
    &= a f(w^k)  - \gamma(w^k)^\top \mP^k \moL\moB^k w^k + \frac{\gamma^2}{2}(w^k)^\top \moB^k \moL\moB^k w^k\\
\end{aligned}
\end{equation}

Considering the conditional expectation and its linearity, along with the transformation properties of symmetric matrices, we obtain:

$$
w^\top \moL w = \frac{1}{2}w^\top \left(\moL + \moL^\top \right) w.
$$

By defining $\moW\eqdef \frac{1}{2}\ec{\mP^k \moL\moB^k + \mP^k \moB^k\moL}$ and setting the stepsize $\gamma$ to be less than or equal to $\frac{1}{\theta}$, we can derive the following:

\begin{equation}\label{eqn:main1}
\begin{aligned}
    \ec{f(w^{k+1})|w^k} &\leq af(w^k) - \gamma(w^k)^\top\ec{\mP^k\moL\moB^k} w^k + \frac{\gamma^2}{2}(w^k)^\top \ec{\moB^k\moL\moB^k}w^k\\
    &= af(w^k) - \gamma(w^k)^\top\moW w^k + \frac{\gamma^2}{2}(w^k)^\top \ec{\moB^k\moL\moB^k}w^k\\
    &= af(w^k) - \gamma(\nabla f(w^k))^\top \moL^{-1}\moW \moL^{-1} \nabla f(w^k)\\
    & \qquad + \frac{\gamma^2}{2}(\nabla f(w^k))^\top\moL^{-1} \ec{\moB^k\moL\moB^k}\moL^{-1}\nabla f(w^k)\\
    &\leq af(w^k) - \gamma(\nabla f(w^k))^\top \moL^{-1}\moW \moL^{-1}  \nabla f(w^k)\\
    & \qquad + \frac{\gamma^2}{2}(\nabla f(w^k))^\top\moL^{-1} \theta \moW\moL^{-1}\nabla f(w^k)\\
    &= af(w^k) - \gamma \norm{\nabla f(w^k)}^2_{\moL^{-1}\moW\moL^{-1}} + \frac{\theta \gamma^2}{2}\norm{\nabla f(w^k)}^2_{\moL^{-1}\moW\moL^{-1}}\\
    &= af(w^k) - \gamma (1 - \nicefrac{\theta\gamma}{2})\norm{\nabla f(w^k)}^2_{\moL^{-1}\moW\moL^{-1}}\\
    &\leq af(w^k) - \frac{\gamma}{2}\norm{\nabla f(w^k)}^2_{\moL^{-1}\moW\moL^{-1}}.\\
\end{aligned}
\end{equation}

Our subsequent analysis relies on the following useful lemma:

\begin{lemma}\label{lem:weighted_recursion}
    Consider two sequences $\{X_k\}_{k\geq 0}$ and $\{Y_k\}_{k\geq 0}$ of nonnegative real numbers satisfying, for each $k\geq 0$, the recursion 
    \[
    X_{k+1} \leq a X_k - Y_k + c,
    \]
    where $a > 1$ and $c \geq 0$ are constants. Let $K \geq 1$ be fixed. For each $k = 0, 1, \ldots, K-1$, define the probabilities 
    \[
    p_k \eqdef \frac{a^{K-(k+1)}}{S_K}, \quad \text{where} \quad S_K \eqdef \sum_{k=0}^{K-1} a^{K-(k+1)}.
    \]
    Define a random variable $Y$ such that $Y = Y_k$ with probability $p_k$. Then 
    \[
    \mathbb{E}[Y] \leq \frac{a^KX_0 - X_K}{S_K} + c \leq \frac{a^K}{S_K} X_0 + c.
    \]
\end{lemma}

\begin{proof}
    We start by multiplying the inequality $Y_k \leq aX_k - X_{k+1} + c$ by $a^{K- (k+1)}$ for each $k$, yielding 
    \[
    a^{K-(k+1)}Y_k \leq a^{K-k} X_k - a^{K-(k+1)}X_{k+1} + a^{K-(k+1)}c.
    \]

    Summing these inequalities for $k = 0, 1, \ldots, K-1$, we observe that many terms cancel out in a telescopic fashion, leading to
    \[
    \sum_{k=0}^{K-1}a^{K-(k+1)} Y_k \leq a^KX_0 - X_K + \sum_{k=0}^{K-1}a^{K-(k+1)}c = a^K X_0 - X_K + S_K c.
    \]
    Dividing both sides of this inequality by $S_K$, we get
    \[
    \sum_{k=0}^{K-1} p_k Y_k \leq \frac{a^KX_0 - X_K}{S_K} + c, 
    \]
    where the left-hand side represents $\mathbb{E}[Y]$. 
\end{proof}

Building upon Lemma~\ref{lem:weighted_recursion} and employing the inequality $1 + x \leq e^x$, which is valid for all $x \geq 0$, along with the fact that $S_K \geq K$, we can further refine the bound:

\begin{equation}\label{eqn:exp_blowup}
    \frac{a^K}{S_K} \leq \frac{(1 + (a - 1))^K}{K} \leq \frac{e^{(a - 1)K}}{K}.
\end{equation}

To mitigate the exponential growth observed in Eqn~\ref{eqn:exp_blowup}, we choose $a = 1 + \gamma^2 h$ for some $h > 0$. Setting the step size as
\[
    \gamma \leq \sqrt{\frac{\log 2}{hK}},
\]
ensures that $\gamma^2 hK \leq \log 2$, leading to
\[
    \frac{a^K}{S_K} \stackrel{\ref{eqn:exp_blowup}}{\leq} \frac{e^{(a - 1)K}}{K} \leq \frac{e^{\gamma^2 hK}}{K} \leq \frac{2}{K}.
\]

Incorporating Lemma~\ref{lem:weighted_recursion} into Eqn~\ref{eqn:main1} and assuming a step size $\gamma \leq \sqrt{\frac{\log 2}{hK}}$ for some $h > 0$, we establish the following result:

\begin{equation}
    \begin{aligned}
        \mathbb{E}\left[\|\nabla f(w^k)\|^2_{\moL^{-1}\moW\moL^{-1}}\right] \leq \frac{4\Delta_0}{\gamma K}.
    \end{aligned}
\end{equation}

%% file: Appendix_C3_CohortSqueeze.tex
\chapter{Appendix to Chapter \ref{chapter_cohort_squeeze}}
\label{chapter3_appendix}
\thispagestyle{empty}

\section{Extended related work}
\subsection{Local solvers}\label{sec:local_solvers}
\begin{table}[!htb]
    \centering
    \caption{Local optimizers for solving the proximal subproblem.}
    \label{tab:comparison-solvers}
    \resizebox{0.7\textwidth}{!}{%
    \begin{threeparttable}
        \begin{tabular}{lll}
            \toprule
            \textbf{Setting} & \textbf{1st order} & \textbf{2nd order} \\
            \midrule
            Strongly-Convex & \begin{tabular}{@{}l@{}}
                    \algname{Conjugate Gradients (CG)} \\
                    \algname{Accelerated GD} \\
                    \algname{Local GD} \\
                    \algname{Scaffnew} \\
                \end{tabular}  &  \begin{tabular}{@{}l@{}}
                    \algname{BFGS} \\
                    \algname{AICN} \\
                    \algname{LocalNewton} \\
                \end{tabular}\\
            \midrule
            Nonconvex & \begin{tabular}{@{}l@{}}
                    \algname{Mime-Adam}\\
                    \algname{FedAdam-AdaGrad}\\
                    \algname{FedSpeed}\\ 
                \end{tabular}  & \begin{tabular}{@{}l@{}}
                    \algname{Apollo}\\
                    \algname{OASIS}\\
                \end{tabular} \\
            \bottomrule
        \end{tabular}
    \end{threeparttable}
    }
\end{table}

In the exploration of local solvers for the \algname{SPPM-AS} algorithm, the focus is on evaluating the performance impact of various inexact proximal solvers within federated learning settings, spanning both strongly convex and non-convex objectives. Here’s a simple summary of the algorithms discussed:

\algname{FedAdagrad-AdaGrad} \citep{wang2021local}: Adapts \algname{AdaGrad} for both client and server sides within federated learning, introducing local and global corrections to address optimizer state handling and solution bias.

\algname{BFGS} \citep{broyden1967quasi, fletcher1970new, goldfarb1970family, shanno1970conditioning}: A quasi-Newton method that approximates the inverse Hessian matrix to improve optimization efficiency, particularly effective in strongly convex settings but with limitations in distributed implementations.

\algname{AICN} \citep{hanzely2022damped}: Offers a global $O(1/k^2)$ convergence rate under a semi-strong self-concordance assumption, streamlining Newton's method without the need for line searches.

\algname{LocalNewton} \citep{bischoff2023second}: Enhances local optimization steps with second-order information and global line search, showing efficacy in heterogeneous data scenarios despite a lack of extensive theoretical grounding.

\algname{Fed-LAMB} \citep{karimi2022layerwise}: Extends the \algname{LAMB} optimizer to federated settings, incorporating layer-wise and dimension-wise adaptivity to accelerate deep neural network training.

\algname{FedSpeed} \citep{FedSpeed}: Aims to overcome non-vanishing biases and client-drift in federated learning through prox-correction and gradient perturbation steps, demonstrating effectiveness in image classification tasks.

\algname{Mime-Adam} \citep{MIME}: Mitigates client drift in federated learning by integrating global optimizer states and an \algname{SVRG}-style correction term, enhancing the adaptability of \algname{Adam} to distributed settings.

\algname{OASIS} \citep{jahani2021doubly}: Utilizes local curvature information for gradient scaling, providing an adaptive, hyperparameter-light approach that excels in handling ill-conditioned problems.

\algname{Apollo} \citep{ma2020apollo}: A quasi-Newton method that dynamically incorporates curvature information, showing improved efficiency and performance over first-order methods in deep learning applications.

Each algorithm contributes uniquely to the landscape of local solvers in federated learning, ranging from enhanced adaptivity and efficiency to addressing specific challenges such as bias, drift, and computational overhead.

\section{Theoretical overview and recommendations}
\subsection{Parameter control}
We have explored the effects of changing the hyperparameters of \algname{SPPM-AS} on its theoretical properties, as summarized in \Cref{tab:theory_parameters}. This summary shows that as the learning rate increases, the number of iterations required to achieve a target accuracy decreases, though this comes with an increase in neighborhood size. Focusing on sampling strategies, for \algname{SPPM-NICE} employing NICE sampling, an increase in the sampling size $\tau_{\mathcal{S}}$ results in fewer iterations ($T$) and a smaller neighborhood. Furthermore, given that stratified sampling outperforms both block sampling and NICE sampling, we recommend adopting stratified sampling, as advised by Lemma \ref{lem:vr_ss}.

\begin{table}[!tb]
	\centering
	\caption{Theoretical summary}
	\label{tab:theory_parameters}
	\resizebox{0.9\textwidth}{!}{%
		\begin{threeparttable}
			\begin{tabular}{cm{0.4\textwidth}cm{0.25\textwidth}cm{0.25\textwidth}cm{0.25\textwidth}}
				\toprule
				Hyperparameter & Control & Rate (T) & Neighborhood \\ 
				\midrule
				$\gamma$ & $\uparrow$ & $\downarrow$ & $\uparrow$ \\ 
				\hline
				\multirow{2}{*}{$\mathcal{S}$} & $\tau_{\mathcal{S}}\uparrow$\tnote{\color{blue}(a)} & $\downarrow$ & $\downarrow$ \\ \cline{2-4}
				~ &  Stratified sampling optimal clustering instead of $\mathrm{BS}$ or $\mathrm{NICE}$ sampling & $\downarrow$ & Lemma \ref{lem:vr_ss} \\ 
				\bottomrule
			\end{tabular}
			\begin{tablenotes}
				\item [{\color{blue}(a)}] We define $\tau_{\mathcal{S}} := \mathbb{E}_{S\sim \mathcal{S}}\sb{\left|S\right|}.$
			\end{tablenotes}
	\end{threeparttable}}
\end{table}

\subsection{Comparison of sampling strategies}\label{sec:samplings_table}
\paragraph{Full Sampling (FS).} Let $S=[n]$ with probability 1. Then \algname{SPPM-AS} applied to \Cref{eqn:obj_4} becomes \algname{PPM}~\citep{moreau1965proximite, martinet1970regularisation} for minimizing $f$. 
Moreover, in this case, we have $p_i=1$ for all $i \in[n]$ and \Cref{eqn:main_8003} takes on the form 
\begin{equation*} 
	\mu_{\mathrm{AS}}=\mu_{\mathrm{FS}}:=\frac{1}{n} \sum_{i=1}^n \mu_i, \quad \sigma_{\star, \mathrm{AS}}^2=\sigma_{\star, \mathrm{FS}}^2:=0 . 
	\end{equation*}

Note that $\mu_{\text {FS }}$ is the strong convexity constant of $f$, and that the neighborhood size is zero, as we would expect.

\paragraph{Nonuniform Sampling (NS).} Let $S=\{i\}$ with probability $p_i>0$, where $\sum_i p_i=1$. Then \Cref{eqn:main_8003} takes on the form
\begin{align*}
    \mu_{\mathrm{AS}}=\mu_{\mathrm{NS}}:=\min _i \frac{\mu_i}{n p_i}, \quad
    \sigma_{\star, \mathrm{AS}}^2=\sigma_{\star, \mathrm{NS}}^2:=\frac{1}{n} \sum_{i=1}^n \frac{1}{n p_i}\left\|\nabla f_i\left(x_{\star}\right)\right\|^2.
\end{align*}

If we take $p_i=\frac{\mu_i}{\sum_{j=1}^n \mu_j}$ for all $i \in[n]$, we shall refer to Algorithm \ref{alg:sppm_as} as \algname{SPPM} with importance sampling (\algname{SPPM-IS}). In this case,
$$
\mu_{\mathrm{NS}}=\mu_{\mathrm{IS}}:=\frac{1}{n} \sum_{i=1}^n \mu_i, \quad \sigma_{\star, \mathrm{NS}}^2=\sigma_{\star, \mathrm{IS}}^2:=\frac{\sum_{i=1}^n \mu_i}{n} \sum_{i=1}^n \frac{\left\|\nabla f_i\left(x_{\star}\right)\right\|^2}{n \mu_i} .
$$

This choice maximizes the value of $\mu_{\mathrm{NS}}$ (and hence minimizes the first part of the convergence rate) over the choice of the probabilities. 

Table \ref{tab:samplings-comparison} summarizes the parameters associated with various sampling strategies, serving as a concise overview of the methodologies discussed in the main text. This summary facilitates a quick comparison and reference.

\begin{table}[!tb]
	\centering
	\caption{Arbitrary samplings comparison.}
	\label{tab:samplings-comparison}
	\renewcommand{\arraystretch}{1.5}
	\resizebox{1.0\textwidth}{!}{%
		\begin{threeparttable}
			\begin{tabular}{ccc}
				\toprule
				Setting/Requirement & $\mu_{\mathrm{AS}}$ & $\sigma_{\star, \mathrm{AS}}$ \\
				\midrule
				Full & $\frac{1}{n}\sum_{i=1}^n\mu_i$ & 0 \\ \hline
				Non-Uniform & $\min_{i}\frac{\mu_i}{np_i}$ & $ \frac{1}{n} \sum_{i=1}^n \frac{1}{n p_i}\left\|\nabla f_i\left(x_{\star}\right)\right\|^2$ \\ \hline
				Nice & $\min_{C \subseteq[n],|C|=\tau} \frac{1}{\tau} \sum_{i \in C} \mu_i$ & $\sum_{C \subseteq[n],|C|=\tau} \frac{1}{\binom{n}{\tau}}\left\|\frac{1}{\tau} \sum_{i \in C} \nabla f_i\left(x_{\star}\right)\right\|^2$ \\ \hline
				Block & $\min_{j \in[b]} \frac{1}{n q_j} \sum_{i \in C_j} \mu_i$ & $\sum_{j \in[b]} q_j\left\|\sum_{i \in C_j} \frac{1}{n p_i} \nabla f_i\left(x_{\star}\right)\right\|^2$\\ \hline
				\multirow{2}{*}{Stratified} & $\min_{\mathbf{i}_b \in \mathbf{C}_b} \sum_{j=1}^b \frac{\mu_{i_j}\left|C_j\right|}{n}$ & $\sum_{\mathbf{i}_b \in \mathbf{C}_b}\left(\prod_{j=1}^b \frac{1}{\left|C_j\right|}\right)\left\|\sum_{j=1}^b \frac{\left|C_j\right|}{n} \nabla f_{i_j}\left(x_{\star}\right)\right\|^2$\\
				~ &~& Upper bound: $\frac{b}{n^2} \sum_{j=1}^b\left|C_j\right|^2 \sigma_j^2$\\
				\bottomrule
			\end{tabular}
	\end{threeparttable}}
\end{table}

\subsection{Extreme cases of block sampling and stratified sampling}
\paragraph{Extreme cases of block sampling.}\label{sec:appendix_extreme_bs}
We now consider two extreme cases:
\begin{itemize}
	\item If $b=1$, then \algname{SPPM-BS} = \algname{SPPM-FS} = \algname{PPM}. Let's see, as a sanity check, whether we recover the right rate as well. We have $q_1=1, C_1=[n], p_i=1$ for all $i \in[n]$, and the expressions for $\mu_{\mathrm{AS}}$ and $\sigma_{\star, \text { BS }}^2$ simplify to
	\begin{align*}
		\mu_{\mathrm{BS}}=\mu_{\mathrm{FS}}:=\frac{1}{n} \sum_{i=1}^n \mu_i, \sigma_{\star, \mathrm{BS}}^2=\sigma_{\star, \mathrm{FS}}^2:=0 .
	\end{align*}
	
	So, indeed, we recover the same rate as \algname{SPPM-FS}.
	
	\item  If $b=n$, then \algname{SPPM-BS} = \algname{SPPM-NS}. Let's see, as a sanity check, whether we recover the right rate as well. We have $C_i=\{i\}$ and $q_i=p_i$ for all $i \in[n]$, and the expressions for $\mu_{\mathrm{AS}}$ and $\sigma_{\star, \mathrm{BS}}^2$ simplify to
	\begin{align*}
		\mu_{\mathrm{BS}}=\mu_{\mathrm{NS}}:=\min _{i \in[n]} \frac{\mu_i}{n p_i},\quad
		\sigma_{\star, \mathrm{BS}}^2=\sigma_{\star, \mathrm{NS}}^2:=\frac{1}{n} \sum_{i=1}^n \frac{1}{n p_i}\left\|\nabla f_i\left(x_{\star}\right)\right\|^2 .
	\end{align*}
	So, indeed, we recover the same rate as \algname{SPPM-NS}.
\end{itemize}

\paragraph{Extreme cases of stratified sampling.}\label{sec:appendix_extreme_ss}

We now consider two extreme cases:
\begin{itemize}
	\item If $b=1$, then \algname{SPPM-SS} = \algname{SPPM-US}. Let's see, as a sanity check, whether we recover the right rate as well. We have $C_1=[n],\left|C_1\right|=n,\left(\prod_{j=1}^b \frac{1}{\left|C_j\right|}\right)=\frac{1}{n}$ and hence
	\begin{align*}
		\mu_{\mathrm{SS}}=\mu_{\mathrm{US}}:=\min _i \mu_i, \quad \sigma_{\star, \mathrm{SS}}^2=\sigma_{\star, \mathrm{US}}^2:=\frac{1}{n} \sum_{i=1}^n\left\|\nabla f_i\left(x_{\star}\right)\right\|^2 .
	\end{align*} 
	So, indeed, we recover the same rate as \algname{SPPM-US}.
	
	\item If $b=n$, then \algname{SPPM-SS} = \algname{SPPM-FS}. Let's see, as a sanity check, whether we recover the right rate as well. We have $C_i=\{i\}$ for all $i \in[n],\left(\prod_{j=1}^b \frac{1}{\left|C_j\right|}\right)=1$, and hence
	\begin{align*}
		\mu_{\mathrm{SS}}=\mu_{\mathrm{FS}}:=\frac{1}{n} \sum_{i=1}^n \mu_i, \quad \sigma_{\star, \mathrm{SS}}^2=\sigma_{\star, \mathrm{FS}}^2:=0 .
	\end{align*}
	So, indeed, we recover the same rate as \algname{SPPM-FS}.
\end{itemize}

\subsection{Federated averaging SPPM baselines}\label{sec:Avg-SPPM-baselines}
In this section we propose two new algorithms based on Federated Averaging principle. Since to the best of our knowledge there are no federated averaging analyses within the same assumptions, we provide analysis of modified versions of \algname{SPPM-AS}.
\paragraph{Averaging on $\prox_{\gamma f_i}$.}
We introduce \algname{FedProx-SPPM-AS} (see  \Cref{alg:FedProx-SPPM-AS}), which is inspired by the principles of \algname{FedProx} \citep{FedProx}. Unlike the traditional approach where a proximal operator is computed for the chosen cohort as a whole, in \algname{FedProx-SPPM-AS}, we compute and then average the proximal operators calculated for each member within the cohort. However, this algorithm is not a simple case of \algname{SPPM-AS} because it does not directly estimate the proximal operator at each step.

\begin{minipage}{.46\textwidth}
	\begin{algorithm}[H]
		\caption{Proximal Averaging SPPM-AS (\algname{FedProx-SPPM-AS})}\label{alg:FedProx-SPPM-AS}
		\begin{algorithmic}[1]
			\STATE \textbf{Input:} starting point \(x_{0, 0} \in \mathbb{R}^d\), arbitrary sampling distribution \(\mathcal{S}\), learning rate \(\gamma > 0\), local communication rounds \(K\).
			\FOR{$t = 0, 1, 2, \cdots, T - 1$}
			\STATE Sample $S_t \sim \mathcal{S}$
			\FOR{$k= 0, 1, 2, \cdots K - 1$}
			\STATE $x_{k + 1, t} = \sum_{i\in S_t}\frac{1}{\left|S_t\right|}\prox_{\gamma f_{i}}(x_{k, t})$ 
			\ENDFOR
			\STATE $x_{0, t + 1} \gets x_{K, t}$
			\ENDFOR
			\STATE \textbf{Output: $x_{0, T}$}
		\end{algorithmic}
	\end{algorithm}
\end{minipage}%
\hfill
\begin{minipage}{.49\textwidth}
	\begin{algorithm}[H]
		\caption{Federated Averaging SPPM-AS (\algname{FedAvg-SPPM-AS})}\label{alg:fedavg-sppm-as}
		\begin{algorithmic}[1]
			\STATE \textbf{Input:} starting point $x_{0,0}\in \mathbb{R}^d$, arbitrary sampling distribution $\mathcal{S}$, global learning rate $\gamma > 0$, local learning rate $\alpha > 0$, local communication rounds $K$
			\FOR{$t = 0, 1, 2, \cdots, T - 1$}
			\STATE Sample $S_t \sim \mathcal{S}$
			\STATE $\forall i\in S_t \ \tilde{f}_{i,t}(x)\gets f_i(x) + \frac{1}{2\gamma}\left\|x-x_t\right\|^2$
			\FOR{$k= 0, 1, 2, \cdots K - 1$}
			\STATE $x_{k + 1, t} = \sum_{i\in S_t}\frac{1}{\left|S_t\right|}\prox_{\alpha \tilde{f}_{i,t}}(x_{k, t})$ 
			\ENDFOR
			\STATE $x_{0, t + 1} \gets x_{K, t}$
			\ENDFOR
			\STATE \textbf{Output: $x_{0, T}$}
		\end{algorithmic}
	\end{algorithm}
\end{minipage}

Here, we employ a proof technique similar to that of Theorem \ref{thm:sppm_as} and obtain the following convergence.
\begin{theorem}[FedProx-SPPM-AS convergence]\label{thm:FedProx-SPPM-AS}
		Let the number of local iterations $K=1$, and assume that  \Cref{asm:differential} (differentiability) and \Cref{asm:strongly_convex} (strong convexity) hold. Let $x_0 \in \mathbb{R}^d$ be an arbitrary starting point. Then, for any $t \geq 0$ and any $\gamma > 0$, the iterates of \algname{FedProx-SPPM} (as described in \Cref{alg:FedProx-SPPM-AS}) satisfy:
		
		\[
		\ec{\sqn{x_t - x_\star}} \leq A_{\mathcal{S}}^{t}\left\|x_0-x_{\star}\right\|^2+\frac{B_{\mathcal{S}}}{1-A_{\mathcal{S}}},
		\] where $A_{\mathcal{S}}\eqdef \ec[S_t\sim\mathcal{S}]{\frac{1}{\left|S_{t}\right|}\sum_{i\in S_{t}}\frac{1}{1+\gamma \mu_{i}}}$ and $B_{\mathcal{S}}\eqdef \ec[S_t\sim\mathcal{S}]{\frac{1}{\left|S_{t}\right|}\sum_{i\in S_{t}}\frac{\gamma}{(1 + \gamma\mu_i)\mu_i}\sqn{\nabla f_{i}(x_\star))}}$.
    \end{theorem}

	\paragraph{Federated averaging for $\prox$ approximation.}
	An alternative method involves estimating the proximal operator by averaging the proximal operators calculated for each worker's function. We call it \emph{Federated Averaging Stochastic Proximal Point Method} (\algname{FedAvg-SPPM-AS}, see \Cref{alg:fedavg-sppm-as}).
 (\algname{FedAvg-SPPM-AS}, see \Cref{alg:fedavg-sppm-as}).

After selecting and fixing a sample of workers $S_k$, the main objective is to calculate the proximal operator. This can be accomplished by approximating the proximal calculation with the goal of minimizing $\tilde{f}_S(x)=f_S(x)+\frac{2}{\gamma}\left\|x-x_t\right\|^2$. It can be observed that this approach is equivalent to \algname{FedProx-SPPM-AS}, as at each local step we calculate
\[
\prox_{\alpha\tilde{f_i}}(x_{k, t})\eqdef \argmin_{z\in\Rd}\left[\tilde{f_i}(z)+\frac{2}{\alpha}\sqn{z-x_{k, t}}\right]=\argmin_{z\in\Rd}\left[f_i(z)+\left(\frac{2}{\gamma}+\frac{2}{\alpha}\right)\sqn{z-x_{k, t}}\right].
\]
\section{Training details}
\subsection{Non-IID Data Generation}\label{sec:data_generation}
In our study, we validate performance and compare the benefits of \algname{SPPM-AS} over \algname{SPPM} using well-known datasets such as \texttt{mushrooms}, \texttt{a6a}, \texttt{w6a}, and \texttt{ijcnn1.bz2} from LibSVM~\citep{chang2011libsvm}. To ensure relevance to our research focus, we adopt a feature-wise non-IID setting, characterized by variation in feature distribution across clients. This variation is introduced by clustering the features using the K-means algorithm, with the number of clusters set to $10$ and the number of clients per cluster fixed at $10$ for simplicity. We visualize the clustered data using t-SNE in \Cref{fig:tSNE1}, where we observe that the data are divided into 10 distinct clusters with significantly spaced cluster centers.

\begin{figure}[!tb]
	\centering
	\begin{subfigure}[b]{0.475\textwidth}
		\centering
		\includegraphics[width=1.0\textwidth, trim=90 46 30 50, clip]{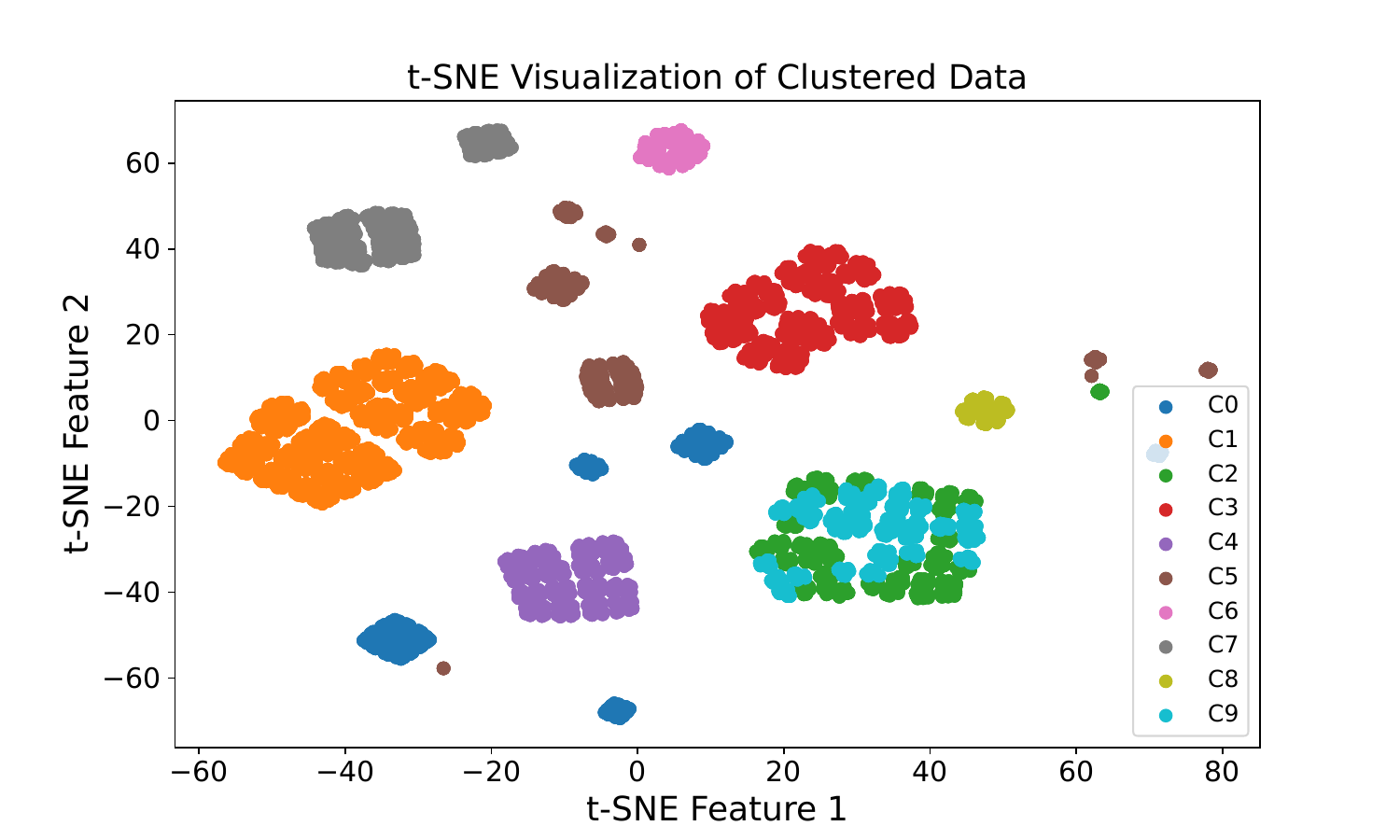}
		\caption{mushrooms, 10 clusters}
	\end{subfigure}
	\begin{subfigure}[b]{0.475\textwidth}
		\centering
		\includegraphics[width=1.0\textwidth, trim=90 46 30 50, clip]{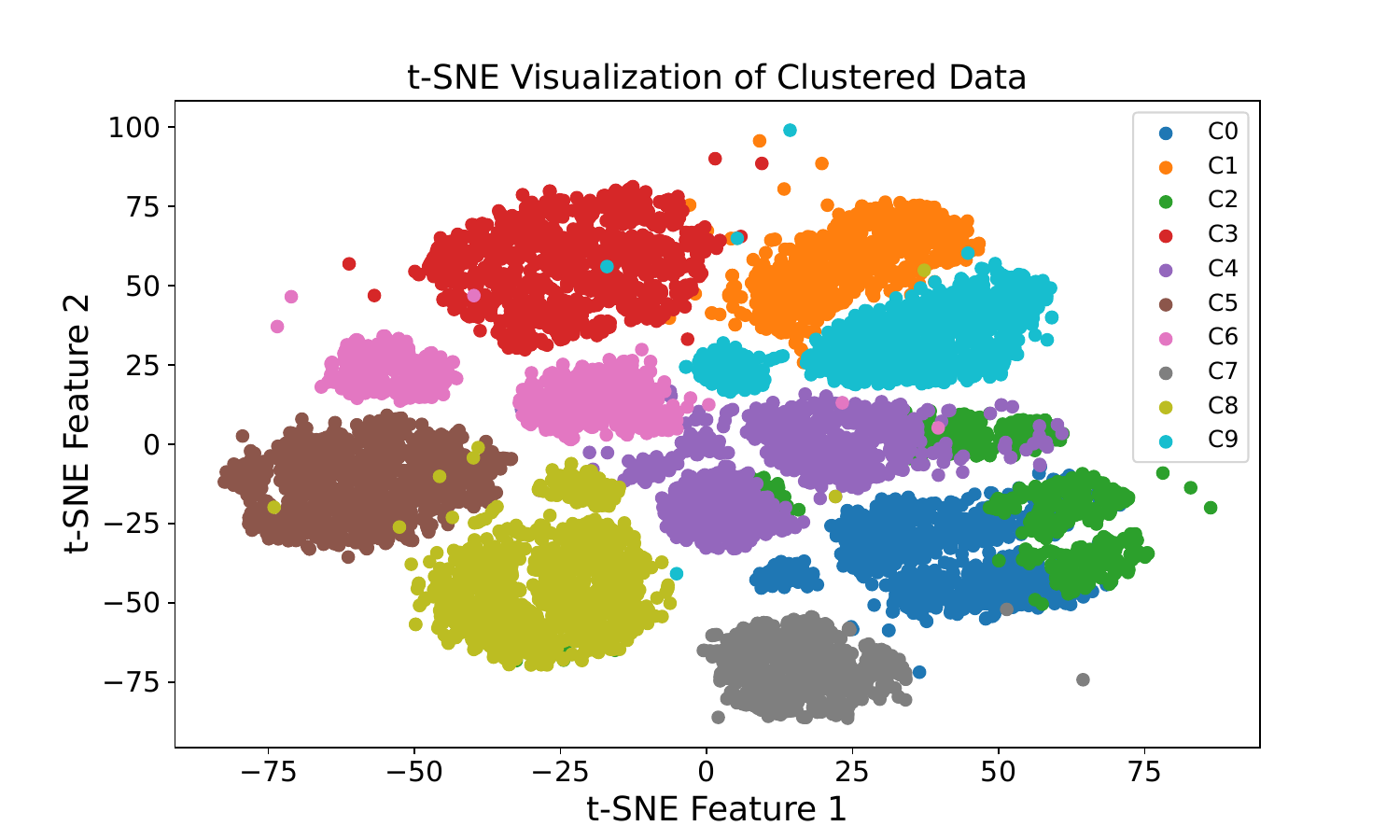}
		\caption{a6a, 10 clusters}
	\end{subfigure}
	\caption{t-SNE visualization of cluster-features across data samples on clients.} \label{fig:tSNE1}
\end{figure}

\subsection{Sampling}
To simulate random sampling among clients within these 10 clusters, where each cluster comprises 10 clients, we consider two contrasting scenarios:

\begin{itemize}
	\item \textit{Case I - \algname{SPPM-BS}:} Assuming clients within the same cluster share similar features and data distributions, sampling all clients from one cluster (i.e., $C=10$ clients) results in a homogeneous sample.
	\item \textit{Case II - \algname{SPPM-SS}:} Conversely, by traversing all 10 clusters and randomly sampling one client from each, we obtain a group of 10 clients representing maximum heterogeneity. 
\end{itemize}

We hypothesize that any random sampling from the 100 clients will yield performance metrics lying between these two scenarios. In \Cref{fig:setting_comp_1}, we examine the impact of sampling clients with varying degrees of heterogeneity using a fixed learning rate of $0.1$. Our findings indicate that heterogeneous sampling results in a significantly smaller convergence neighborhood $\sigma^2_{\star}$. This outcome is attributed to the broader global information captured through heterogeneous sampling, in contrast to homogeneous sampling, which increases the data volume without contributing additional global insights. As these two sampling strategies represent the extremes of arbitrary sampling, any random selection will fall between them in terms of performance. Given their equal cost and the superior performance of the \algname{SPPM-SS} strategy in heterogeneous FL environments, we designate \algname{SPPM-SS} as our default sampling approach.

\begin{figure}[!tb]
	\centering
	\begin{subfigure}[b]{0.31\textwidth}
		\centering
		\includegraphics[width=1.0\textwidth, trim=0 0 0 0, clip]{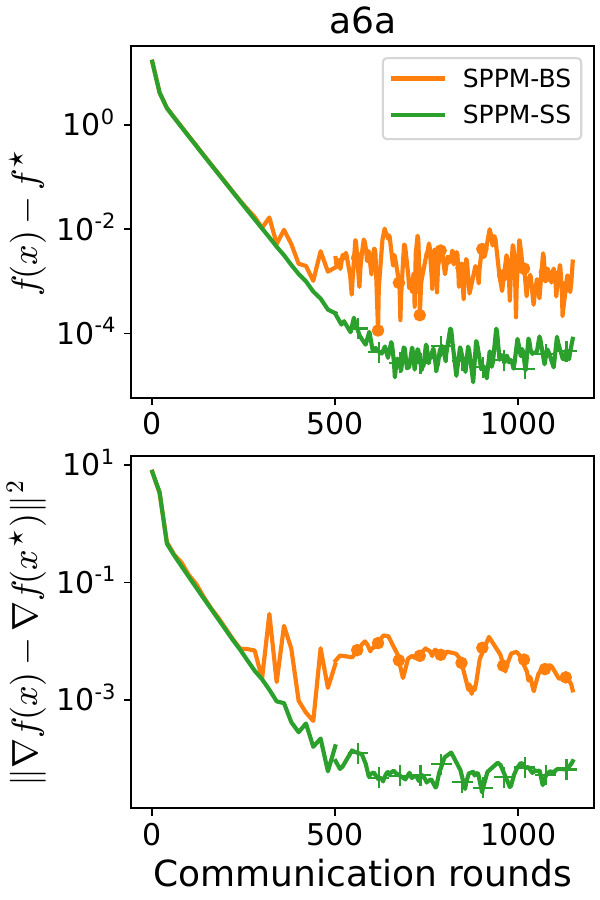}
	\end{subfigure}
	\hfill
	\begin{subfigure}[b]{0.31\textwidth}
		\centering
		\includegraphics[width=1.0\textwidth, trim=0 0 0 0, clip]{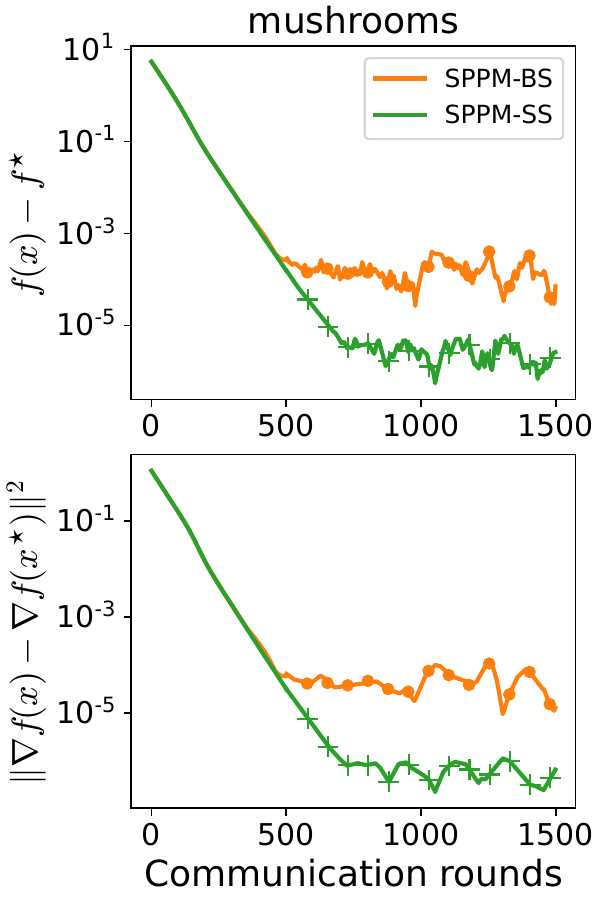}
	\end{subfigure}
	\hfill
	\begin{subfigure}[b]{0.31\textwidth}
		\centering
		\includegraphics[width=1.0\textwidth, trim=0 0 0 0, clip]{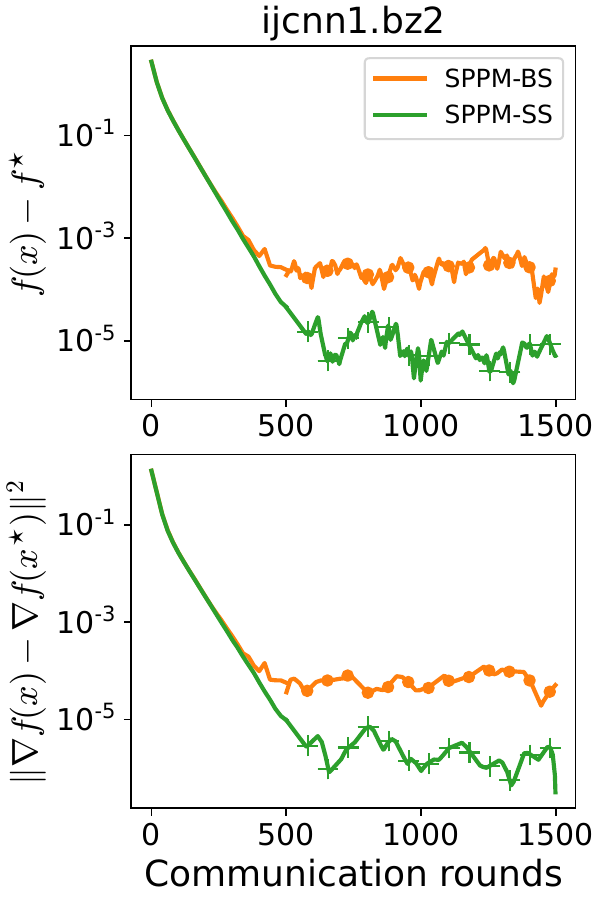}
	\end{subfigure}
	\caption{Comparison with \algname{SPPM-SS} and \algname{SPPM-BS} samplings.} \label{fig:setting_comp_1}
\end{figure}

\begin{algorithm}[!tb]
	\caption{\algname{SPPM-AS} Adaptation for Federated Learning}\label{alg:SPPMFL}
	\begin{algorithmic}[1]
		\STATE \textbf{Input:} Initial point $x^0\in \Rd$, cohort size $C\geq 1$, learning rate $\gamma > 0$, clusters $q\geq C$, local communication rounds $K$
		\FOR{$t = 0, 1, 2, \cdots$}
		\STATE \algname{SPPM-BS}:
		\STATE \quad Server samples a cluster $q_i$ from $[q]$
		\STATE \quad Server samples $C$ clients, denoted as $[C]$ from cluster $q_i$
		\STATE \algname{SPPM-SS}:
		\STATE \quad Server samples $C$ clusters from $[q]$
		\STATE \quad Server sample 1 client from each selected cluster to construct $C$ clients
		\STATE Server broadcasts the model $x_t$ to each $C_i \in [C]$
		\STATE All selected clients in parallel construct $F_{\xi_t^1, \cdots, \xi_t^C} (x_t)$
		\STATE All selected clients together evaluate the prox for $K$ local communication rounds to obtain 
		\STATE $$x_{t+1} \simeq \prox_{\gamma F_{\xi_t^1, \cdots, \xi_t^C}}(x_t)$$
		\STATE All selected clients send the updated model $x_{t+1}$ to the server
		
		\ENDFOR
	\end{algorithmic}
\end{algorithm}

\subsection{{SPPM-AS} algorithm adaptation for FL}
In the main text, \Cref{alg:sppm_as} outlines the general form of \algname{SPPM-AS}. For the convenience of implementation in FL contexts and to facilitate a better understanding, we introduce a tailored version of the \algname{SPPM-AS} algorithm specific to FL, designated as \Cref{alg:SPPMFL}. Notably, as block sampling is adopted as our default method, this adaptation of the algorithm specifically addresses the nuances of the block sampling approach. We also conducted arbitrary sampling on synthetic datasets and neural networks to demonstrate the algorithm's versatility.  

\section{Additional experiments on logistic regression}
\subsection{Communication cost on various datasets to a target accuracy}
In \Cref{fig:tissue0}, we presented the total communication cost relative to the number of rounds required to achieve the target accuracy for the selected cohort. In this section, we provide more details on how is this figure was obtained and present additional results for various datasets. 
\begin{figure*}[!htbp]
	\centering
	\begin{subfigure}[b]{0.31\textwidth}
		\centering
		\includegraphics[width=1.0\textwidth, trim=0 0 0 0, clip]{img/cohort_squeeze/main/BFGS_clusters_100_a6a_False_True_1e-06_15_1000.0_minibatch_sppm_round_0.005main.pdf}
	\end{subfigure}
	\hfill
	\centering
	\begin{subfigure}[b]{0.31\textwidth}
		\centering
		\includegraphics[width=1.0\textwidth, trim=0 0 0 0, clip]{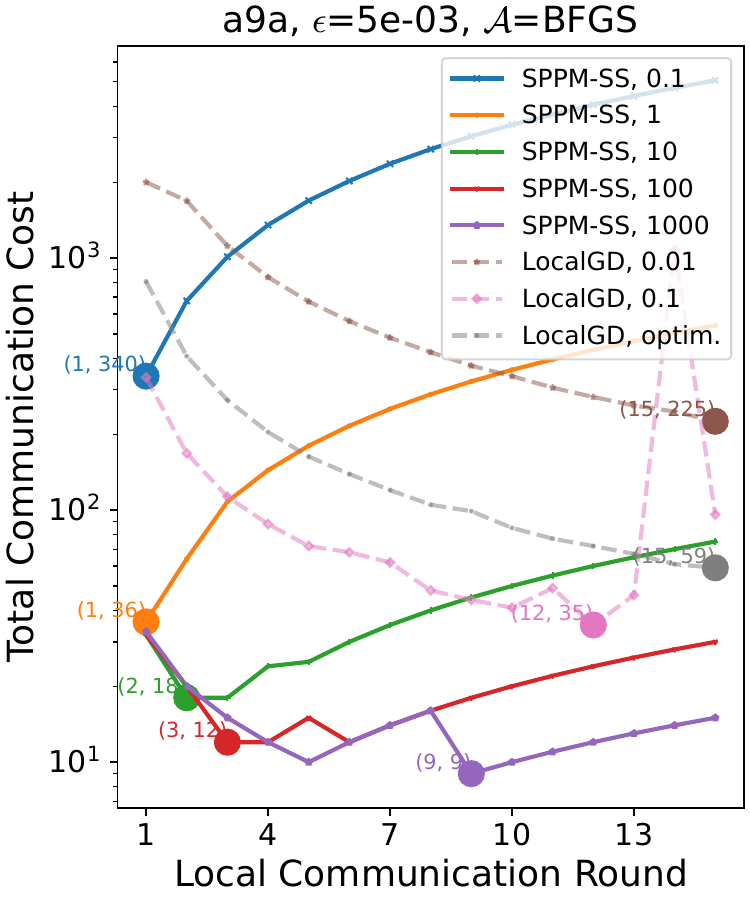}
	\end{subfigure}
	\hfill
	\begin{subfigure}[b]{0.31\textwidth}
		\centering
		\includegraphics[width=1.0\textwidth, trim=0 0 0 0, clip]{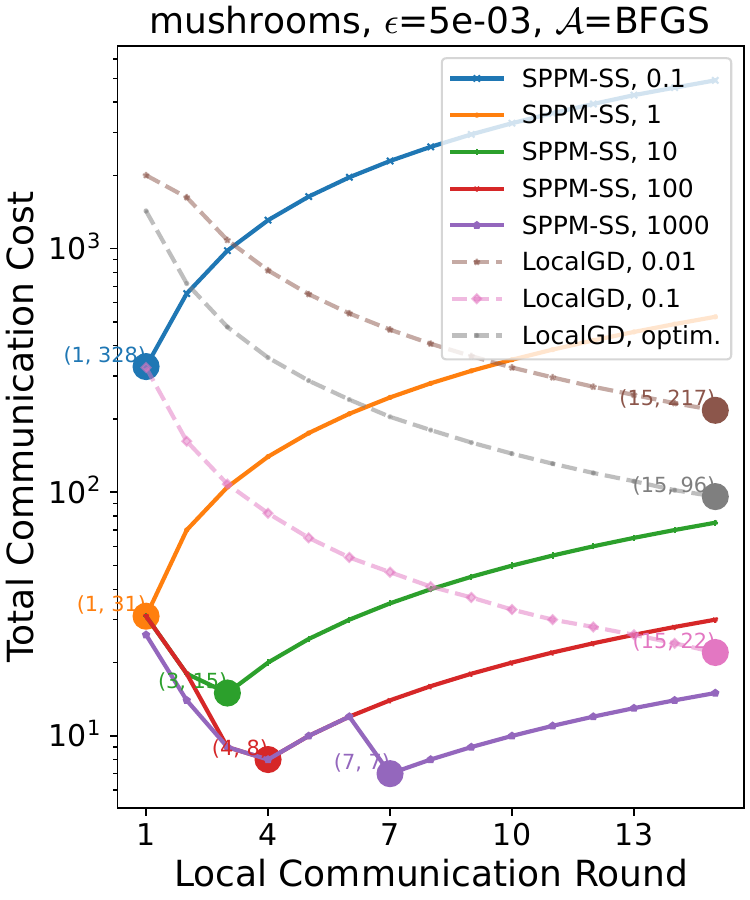}
	\end{subfigure}
	\hfill
	\caption{Total communication cost with respect to the local communication round. For \algname{LocalGD}, $K$ represents the local communication round $K$ for finding the prox of the current model. For \algname{LocalGD}, we slightly abuse the x-axis, which represents the total number of local iterations, no local communication is required. We calculate the total communication cost to reach a fixed global accuracy $\epsilon$ such that $\sqn{x_t - x_\star}<\epsilon$. \algname{LocalGD, optim} represents using the theoretical optimal stepsize of \algname{LocalGD} with minibatch sampling.}
	\label{fig:tissue1_arch}
\end{figure*} 

\begin{figure}[!htbp]
	\centering
	\begin{minipage}{0.48\textwidth}
		\centering
		\begin{subfigure}[b]{0.48\textwidth}
			\centering
			\includegraphics[width=\textwidth, trim=0 0 0 0, clip]{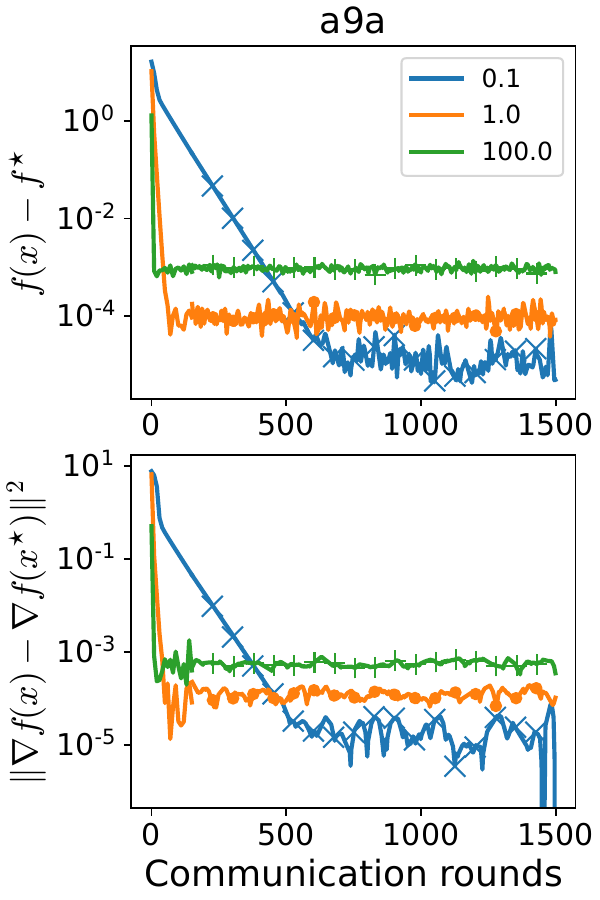}
		\end{subfigure}
		\hfill
		\begin{subfigure}[b]{0.48\textwidth}
			\centering
			\includegraphics[width=\textwidth, trim=0 0 0 0, clip]{img/cohort_squeeze/gammas/BFGS_clusters_100_a9a_False_True_1e-06_4_100.0_minibatch_sppm_round.pdf}
		\end{subfigure}
		\caption{$K=4$.}\label{fig:abs3_1}
	\end{minipage}
	\hfill
	\begin{minipage}{0.48\textwidth}
		\centering
		\begin{subfigure}[b]{0.48\textwidth}
			\centering
			\includegraphics[width=\textwidth, trim=0 0 0 0, clip]{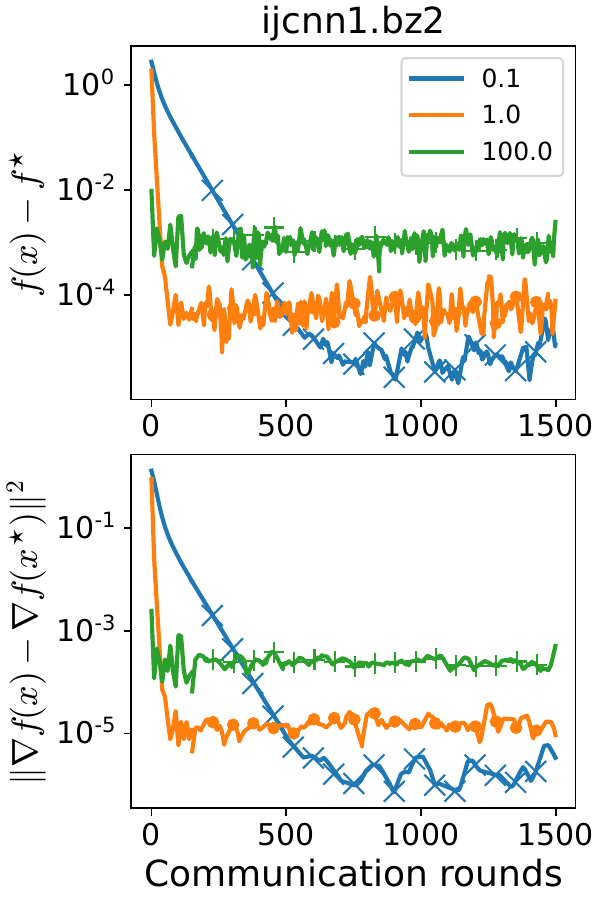}
		\end{subfigure}
		\hfill
		\begin{subfigure}[b]{0.48\textwidth}
			\centering
			\includegraphics[width=\textwidth, trim=0 0 0 0, clip]{img/cohort_squeeze/gammas/BFGS_clusters_100_ijcnn1.bz2_False_True_1e-06_16_100.0_minibatch_sppm_round.pdf}
		\end{subfigure}
		\caption{$K=16$.}\label{fig:abs3_2}
	\end{minipage}
\end{figure}

\subsection{Convergence speed and $\sigma^2_{\star, \mathrm{SS}}$ trade-off}\label{sec:trade_off}
Unlike \algname{SGD}-type methods such as \algname{MB-GD} and \algname{MB-LocalGD}, in which the largest allowed learning rate is $1/A$, where $A$ is a constant proportion to the smoothness of the function we want to optimize~\citep{gower2019sgd}. For larger learning rate, \algname{SGD}-type method may not converge and exploding. However, for stochastic proximal point methods, they have a very descent benefit of allowing arbitrary learning rate. In this section, we verify whether our proposed method can allow arbitrary learning rate and whether we can find something interesting. We considered different learning rate scale from 1e-5 to 1e+5. We randomly selected three learning rates [0.1, 1, 100] for visual representation with the results presented in \Cref{fig:abs3_1} and \Cref{fig:abs3_2}. We found that a larger learning rate leads to a faster convergence rate but results in a much larger neighborhood, \(\sigma^2_{\star, \mathrm{SS}} / \mu^2_{\mathrm{SS}}\). This can be considered a trade-off between convergence speed and neighborhood size, \(\sigma^2_{\star, \mathrm{SS}}\). By default, we consider setting the learning rate to $1.0$ which has a good balance between the convergence speed and the neighborhood size.

In this section, we extend our analysis by providing additional results across a broader range of datasets and varying learning rates. Specifically, \Cref{fig:abs3_1} illustrates the outcomes using 4 local communication rounds ($K=4$), while \Cref{fig:abs3_2} details the results for 16 local communication rounds ($K=16$). Previously, in \Cref{fig:tissue0}, we explored the advantages of larger $K$ values. Here, our focus shifts to determining if similar trends are observable across different $K$ values. Through comprehensive evaluations on various datasets and multiple $K$ settings, we have confirmed that lower learning rates in \algname{SPPM-AS} result in slower convergence speeds; however, they also lead to a smaller final convergence neighborhood.

\subsection{Additional experiments on hierarchical FL}\label{sec:hierarchy-fl}
In \Cref{fig:abs1_hierarchical} of the main text, we detail the total communication cost for hierarchical Federated Learning (FL) utilizing parameters $c_1=0.1$ and $c_2=1$ on the \texttt{a6a} dataset. Our findings reveal that \algname{SPPM-AS} achieves a significant reduction in communication costs, amounting to $94.87\%$, compared with the conventional FL setting where $c_1=1$ and $c_2=1$, which shows a 74.36\% reduction. In this section, we extend our analysis with comprehensive evaluations on additional datasets, namely \texttt{ijcnn1.bz2}, \texttt{a9a}, and \texttt{mushrooms}. Beyond considering $c_1=0.1$, we further explore the impact of reducing the local communication cost from each client to the corresponding hub to $c_1=0.05$. The results, presented in \Cref{fig:tissue11} and the continued \Cref{fig:tissue11_2}, reinforce our observation: hierarchical FL consistently leads to further reductions in communication costs. A lower $c_1$ parameter correlates with even greater savings in communication overhead. These results not only align with our expectations but also underscore the efficacy of our proposed \algname{SPPM-AS} in cross-device FL settings.

\begin{figure*}[!htbp]
	\centering
	\begin{subfigure}[b]{0.31\textwidth}
		\centering
		\includegraphics[width=1.0\textwidth, trim=0 0 0 0, clip]{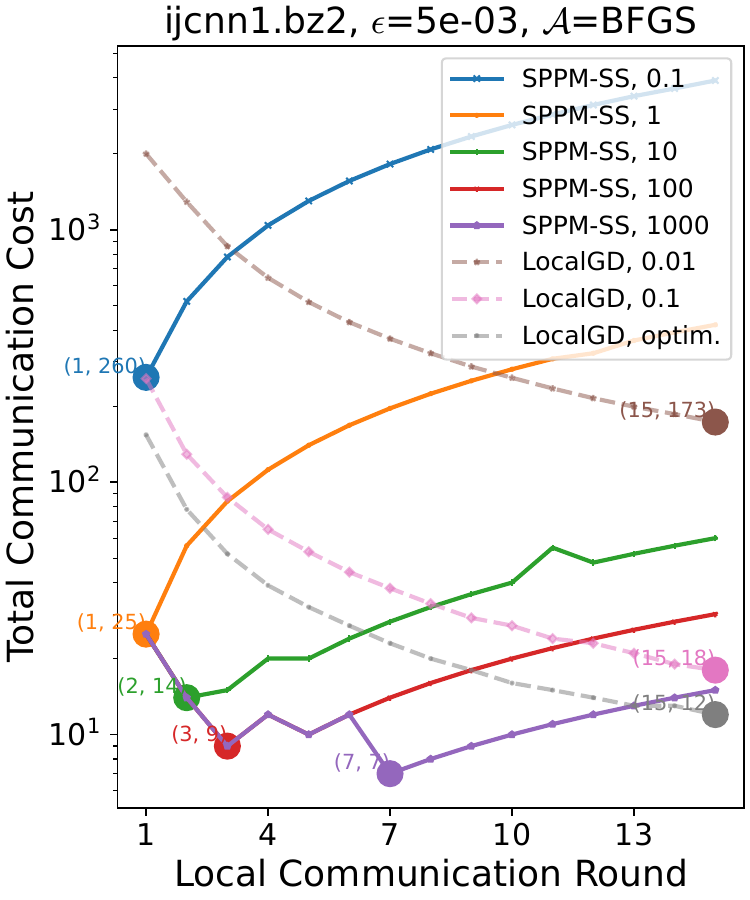}
	\end{subfigure}
	\hfill
	\centering
	\begin{subfigure}[b]{0.31\textwidth}
		\centering
		\includegraphics[width=1.0\textwidth, trim=0 0 0 0, clip]{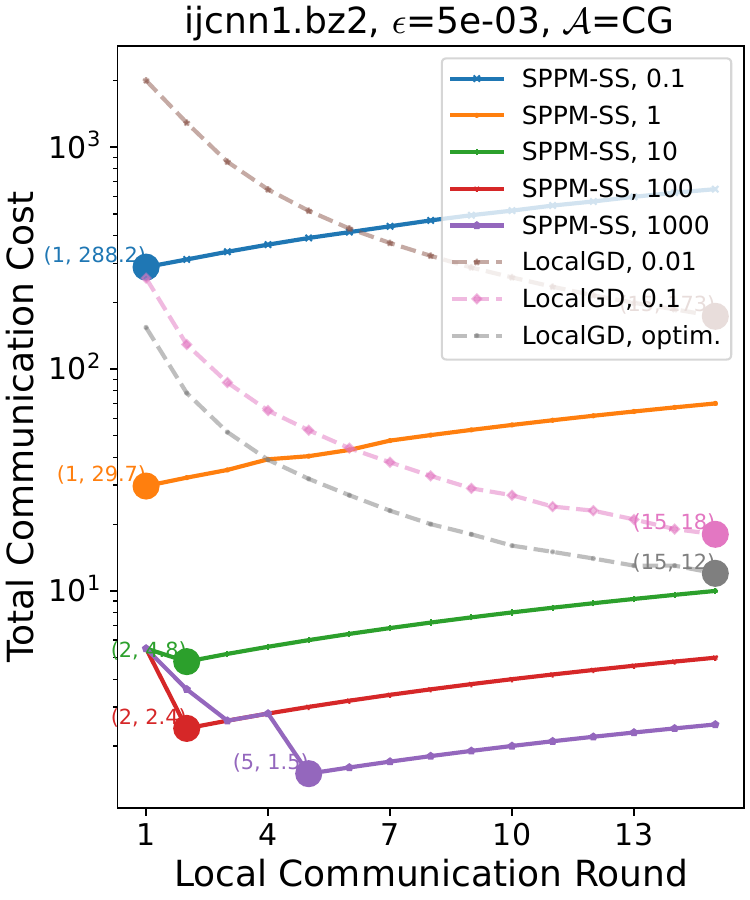}
	\end{subfigure}
	\hfill
	\begin{subfigure}[b]{0.31\textwidth}
		\centering
		\includegraphics[width=1.0\textwidth, trim=0 0 0 0, clip]{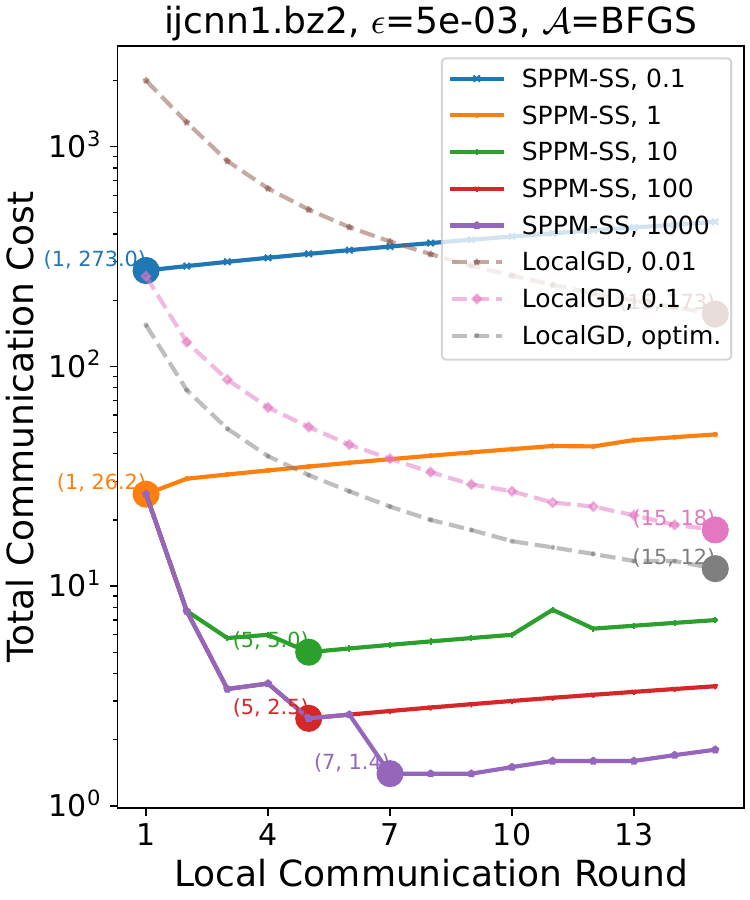}
	\end{subfigure}
	\hfill
	\caption{The total communication cost is analyzed with respect to the number of local communication rounds. For \algname{LocalGD}, $K$ represents the local communication round used for finding the prox of the current model. In the case of \algname{LocalGD}, we slightly abuse the x-axis to represent the total number of local iterations, as no local communication is required. We calculate the total communication cost needed to reach a fixed global accuracy $\epsilon$, such that $\sqn{x_t - x_\star} < \epsilon$. \algname{LocalGD, optim} denotes the use of the theoretically optimal stepsize for \algname{LocalGD} with minibatch sampling. Comparisons are made between different prox solvers (\algname{CG} and \algname{BFGS}).
	}
	\label{fig:tissue11}
\end{figure*}

\begin{figure*}[!htbp]
	\centering
	\begin{subfigure}[b]{0.30\textwidth}
		\centering
		\includegraphics[width=1.0\textwidth]{img/cohort_squeeze/main/BFGS_clusters_100_a9a_False_True_1e-06_15_1000.0_minibatch_sppm_round_0.005main.pdf}
	\end{subfigure}
	\hfill
	\centering
	\begin{subfigure}[b]{0.30\textwidth}
		\centering
		\includegraphics[width=1.0\textwidth]{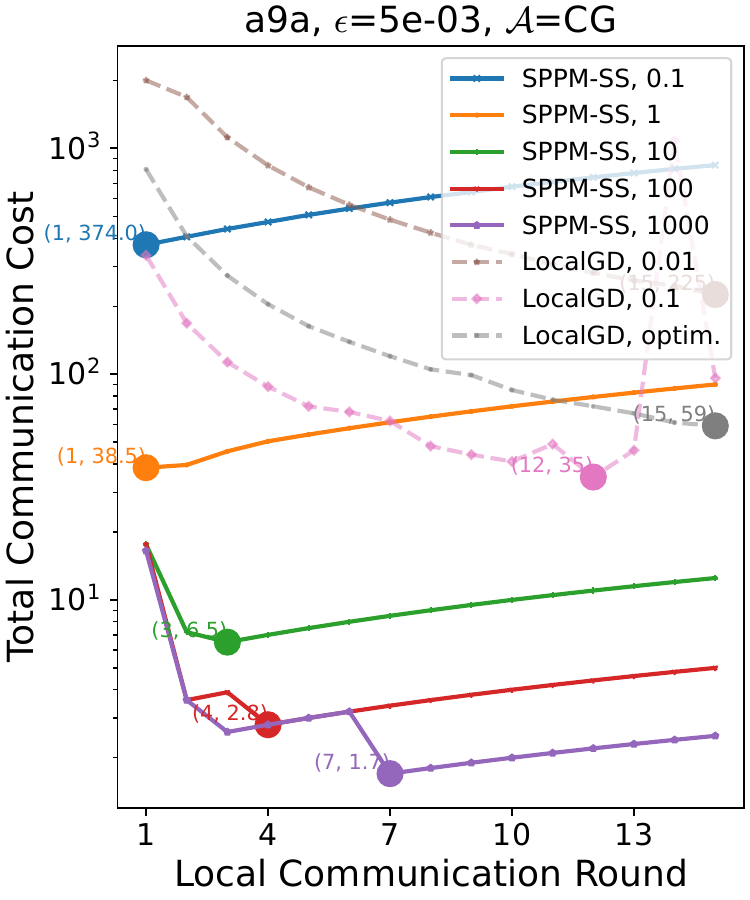}
	\end{subfigure}
	\hfill
	\begin{subfigure}[b]{0.30\textwidth}
		\centering
		\includegraphics[width=1.0\textwidth]{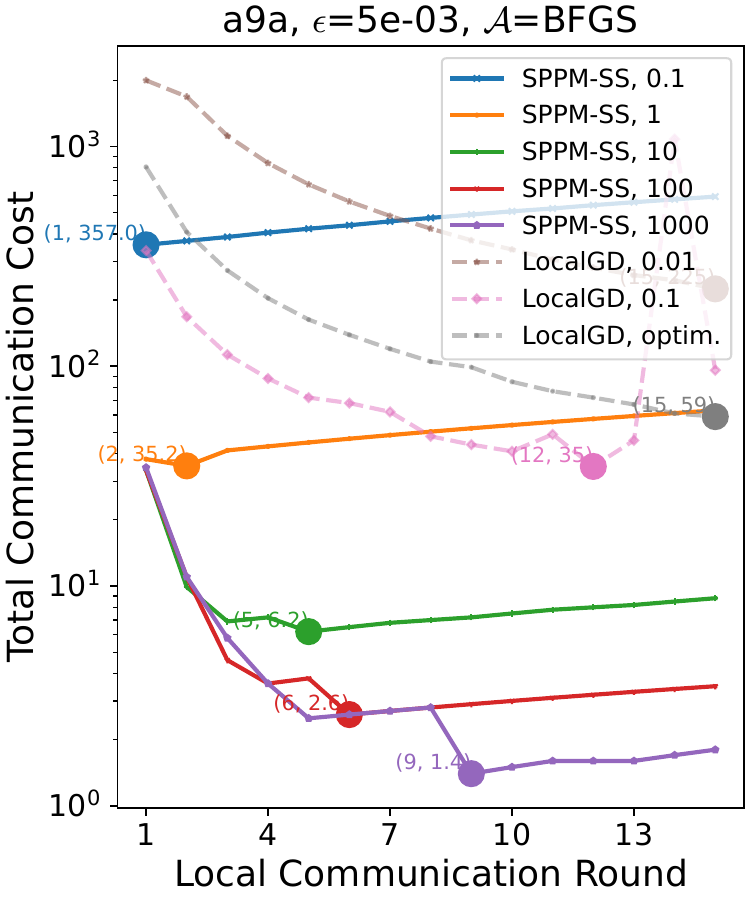}
	\end{subfigure}	
	\centering
	\begin{subfigure}[b]{0.30\textwidth}
		\centering
		\includegraphics[width=1.0\textwidth]{img/cohort_squeeze/main/BFGS_clusters_100_mushrooms_False_True_1e-06_15_1000.0_minibatch_sppm_round_0.005main.pdf}
		\caption{standard FL, $c_1=1, c_2=0$}
	\end{subfigure}
	\hfill
	\centering
	\begin{subfigure}[b]{0.30\textwidth}
		\centering
		\includegraphics[width=1.0\textwidth]{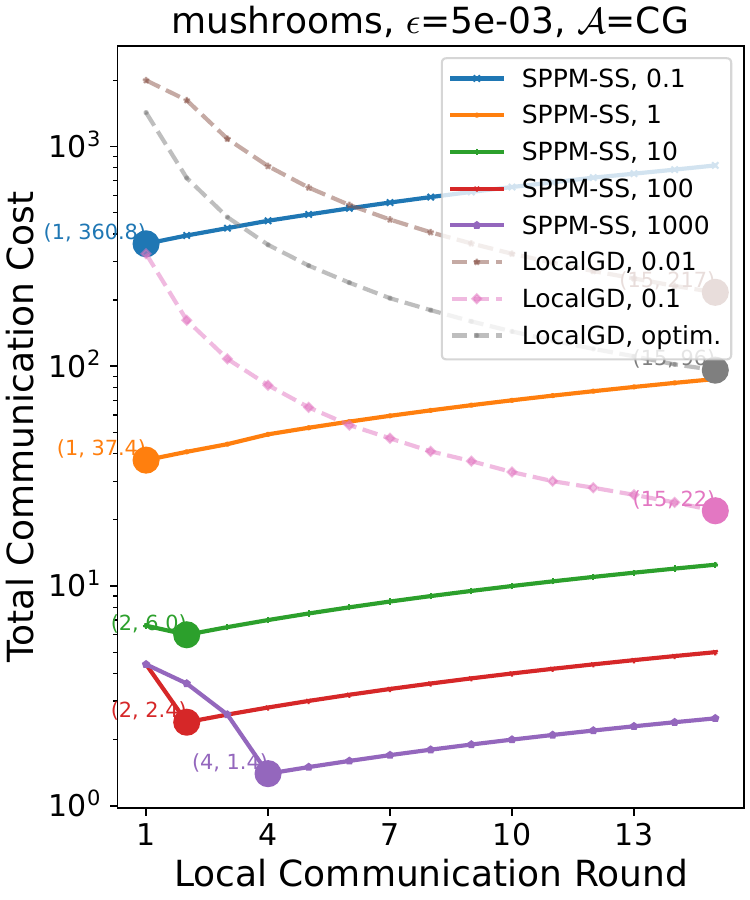}
		\caption{hierarchical FL, $c_1=0.1, c_2=1$}
	\end{subfigure}
	\hfill
	\begin{subfigure}[b]{0.30\textwidth}
		\centering
		\includegraphics[width=1.0\textwidth]{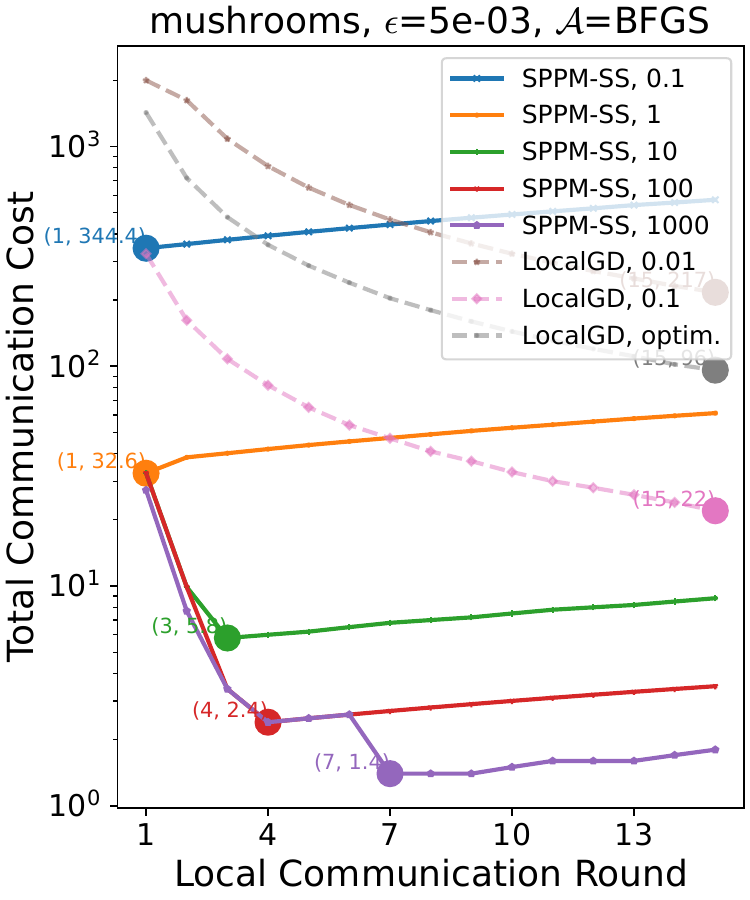}
		\caption{hierarchical FL, $c_1=0.05, c_2=1$}
	\end{subfigure}
	\hfill
	\caption{Total communication cost with respect to the local communication round.}
	\label{fig:tissue11_2}
\end{figure*}  

\section{Additional neural network experiments}\label{sec:nn_additional}
\subsection{Experiment Details} For our neural network experiments, we used the \texttt{FEMNIST} dataset \citep{leaf}. Each client was created by uniformly selecting from user from original dataset, inherently introducing heterogeneity among clients. We tracked and reported key evaluation metrics—training and testing loss and accuracy—after every 5 global communication rounds. The test dataset was prepared by dividing each user's data into a 9:1 ratio, following the partitioning approach of the FedLab framework \citep{FedLab}. For the \algname{SPPM-AS} algorithm, we selected \algname{Adam} as the optimizer for the proximal operator. The learning rate was determined through a grid search across the following range: $[0.0001, 0.0005, 0.001, 0.005, 0.01, 0.05, 0.1, 0.5]$. The model architecture comprises a CNN with the following layers: Conv2d(1, 32, 5), ReLU, Conv2d(32, 64, 5), MaxPool2d(2, 2), a fully connected (FC) layer with 128 units, ReLU, and another FC layer with 128 units, as specified in Table \ref{tab:CNN}. Dropout, learning rate scheduling, gradient clipping, etc., were not used to improve the interpretability of results.

We explore various values of targeted training accuracy, as illustrated in \Cref{fig:dl_accuracy}. This analysis helps us understand the impact of different accuracy thresholds on the model's performance. For instance, we observe that as the target accuracy changes, \algname{SPPM-NICE} consistently outperforms \algname{LocalGD} in terms of total communication cost. As the target accuracy increases, the performance gap between these two algorithms also widens. Additionally, we perform ablation studies on different values of \(c_1\), as shown in \Cref{fig:dl_c_1}, to assess their effects on the learning process. Here, we note that with \(c_2 = 0.2\), \algname{SPPM-NICE} performs similarly to \algname{LocalGD}, suggesting that an increase in \(c_2\) value could narrow the performance gap between \algname{SPPM-NICE} and \algname{LocalGD}.
	
\begin{table}[!tb]
	\centering
	\caption{Architecture of the CNN model for \texttt{FEMNIST} symbol recognition.}
	\label{tab:CNN}
	\begin{tabular}{p{0.15\textwidth}p{0.15\textwidth}p{0.15\textwidth}p{0.15\textwidth}p{0.2\textwidth}}
		\toprule
		Layer       & Output Shape   & \# of Trainable Parameters & Activation & Hyperparameters                         \\ \midrule
		Input       & (28, 28, 1)    & 0                          &            &                                          \\ 
		Conv2d      & (24, 24, 32)   & 832                        & ReLU       & kernel size = 5; strides = (1, 1)        \\ 
		Conv2d      & (10, 10, 64)   & 51,264                     & ReLU       & kernel size = 5; strides = (1, 1)        \\ 
		MaxPool2d   & (5, 5, 64)     & 0                          &            & pool size = (2, 2)                       \\ 
		Flatten     & 6400           & 0                          &            &                                          \\
		Dense       & 128            & 819,328                    & ReLU       &                                          \\ 
		Dense       & 62            & 7,998                      & softmax    &                                         \\ \bottomrule
	\end{tabular}
\end{table}

\begin{figure}[!htbp]
	\centering
	\begin{subfigure}[b]{0.31\textwidth}
		\centering
		\includegraphics[width=1.0\textwidth, trim=0 0 0 0, clip]{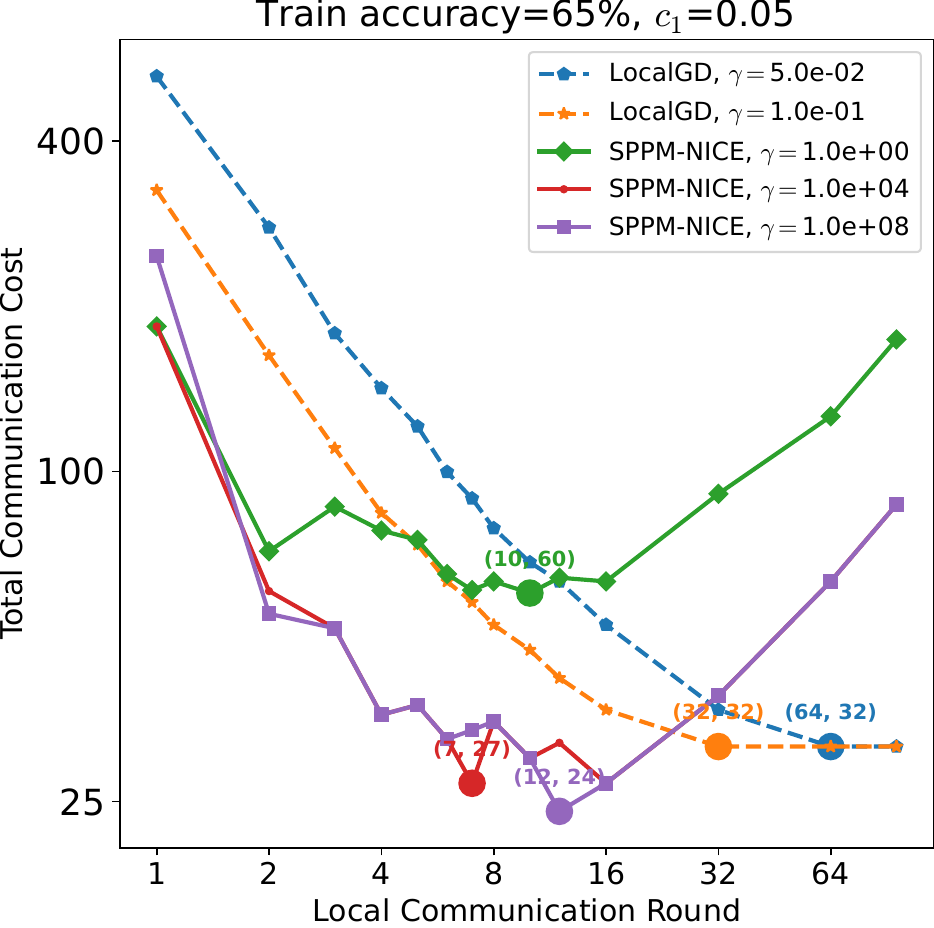}
	\end{subfigure}
	\hfill  
	\centering
	\begin{subfigure}[b]{0.31\textwidth}
		\centering
		\includegraphics[width=1.0\textwidth, trim=0 0 0 0, clip]{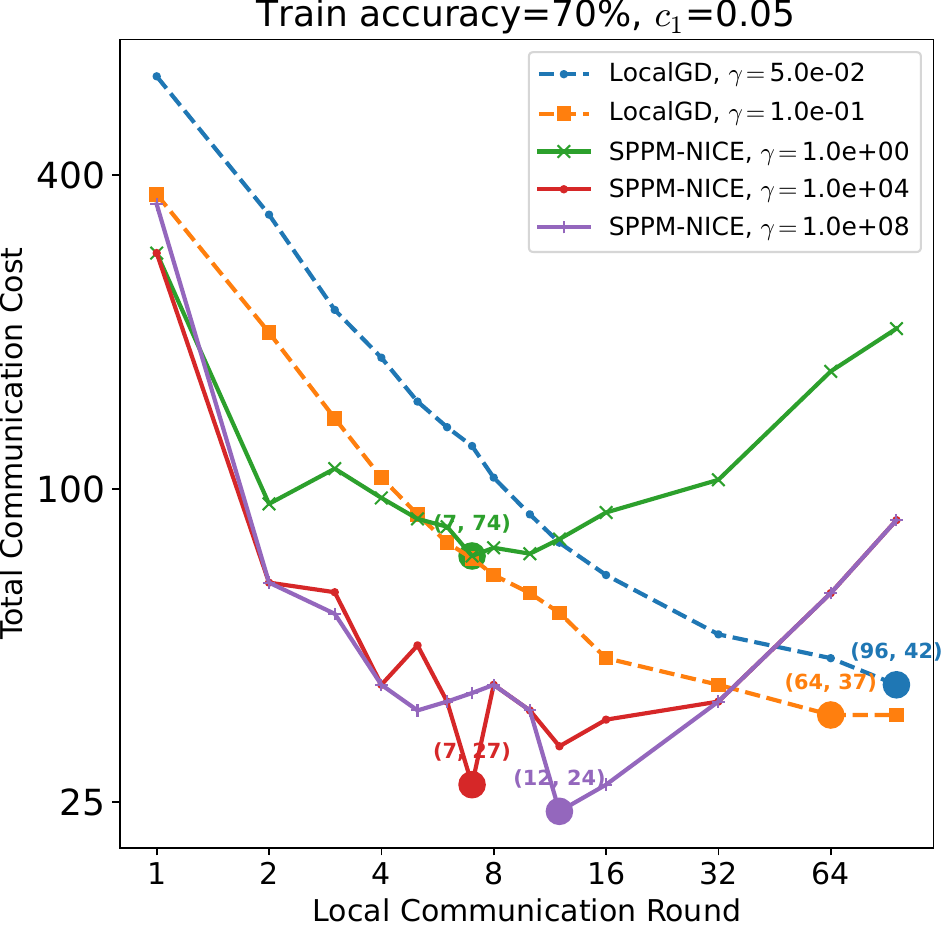}
	\end{subfigure}
	\hfill
	\begin{subfigure}[b]{0.31\textwidth}
		\centering
		\includegraphics[width=1.0\textwidth, trim=0 0 0 0, clip]{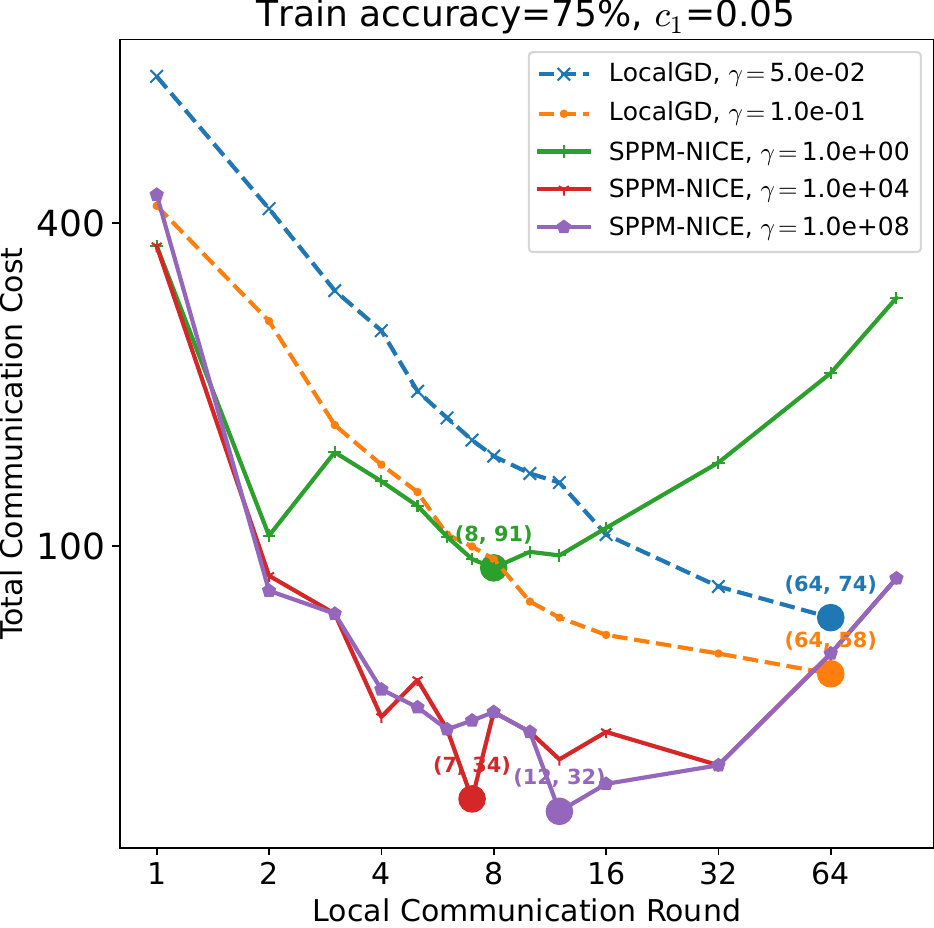}
	\end{subfigure}
	\hfill
	\caption{Varying targeted training accuracy level for \algname{SPPM-AS}.}
	\label{fig:dl_accuracy}
\end{figure}  

\begin{figure}[!htbp]
	\centering
	\begin{subfigure}[b]{0.31\textwidth}
		\centering
		\includegraphics[width=1.0\textwidth, trim=0 0 0 0, clip]{img/cohort_squeeze/deep_learning/1-main.pdf}
	\end{subfigure}
	\hfill
	\centering
	\begin{subfigure}[b]{0.31\textwidth}
		\centering
		\includegraphics[width=1.0\textwidth, trim=0 0 0 0, clip]{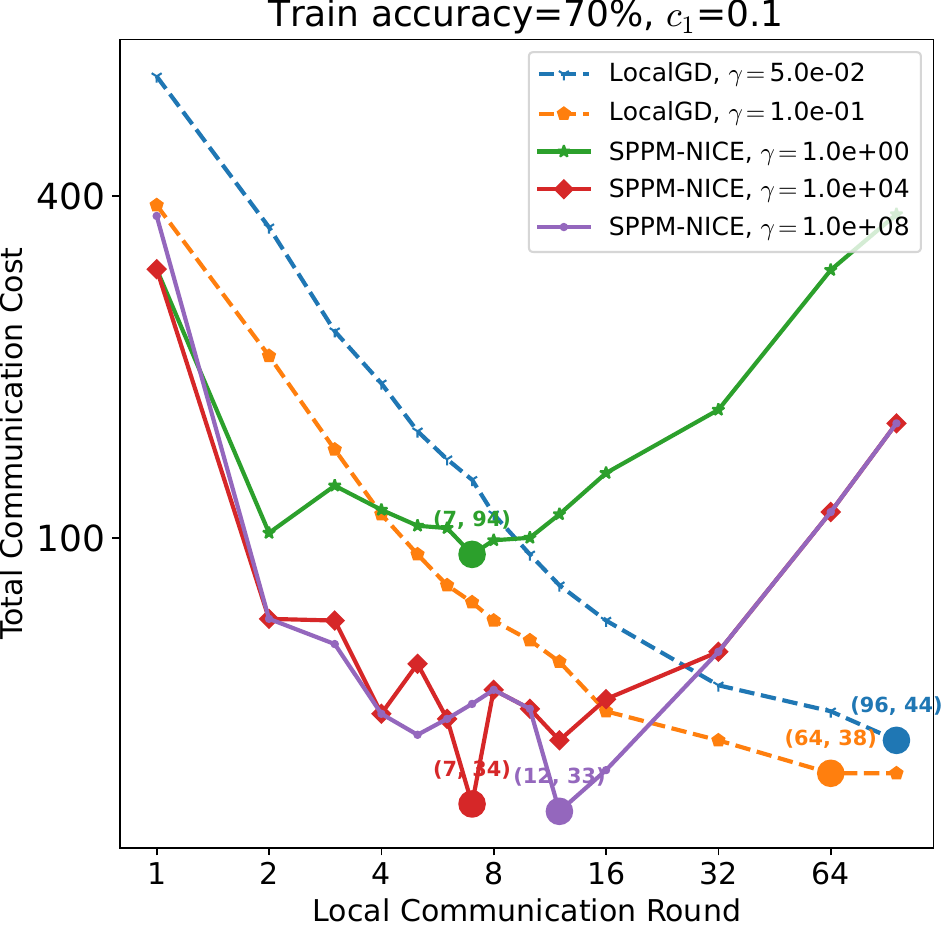}
		
	\end{subfigure}
	\hfill
	\begin{subfigure}[b]{0.31\textwidth}
		\centering
		\includegraphics[width=1.0\textwidth, trim=0 0 0 0, clip]{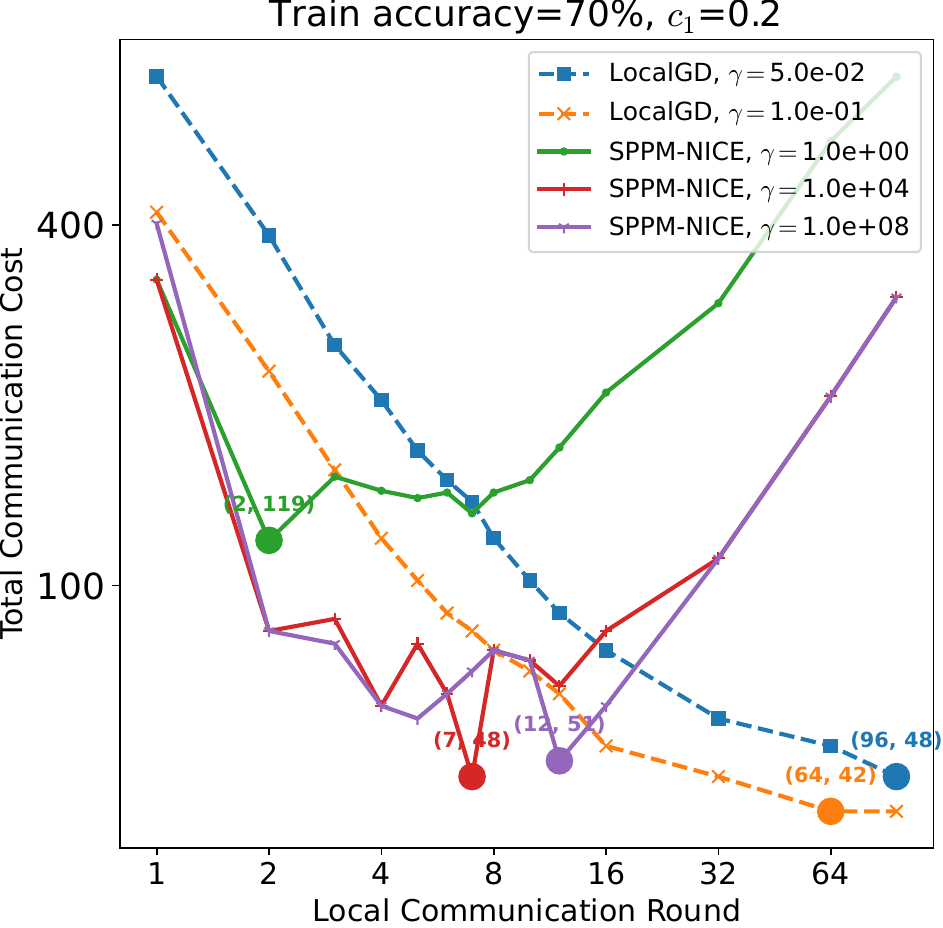}
	\end{subfigure}
	\hfill
	\caption{Varying $c_1$ cost. }
	\label{fig:dl_c_1}
\end{figure}  
\subsection{Convergence Analysis Compared with Baselines}
Further, we compare \algname{SPPM-AS}, \algname{SPPM}, and \algname{LocalGD} in \Cref{fig:dl2}, placing a particular emphasis on evaluating the total computational complexity. This measure gains importance in scenarios where communication rounds are of secondary concern, thereby shifting the focus to the assessment of computational resource expenditure.

\begin{figure}[!htbp]
	\centering
	\begin{subfigure}[b]{0.9\textwidth}
		\centering
		\includegraphics[width=1.0\textwidth, trim=0 0 0 0, clip]{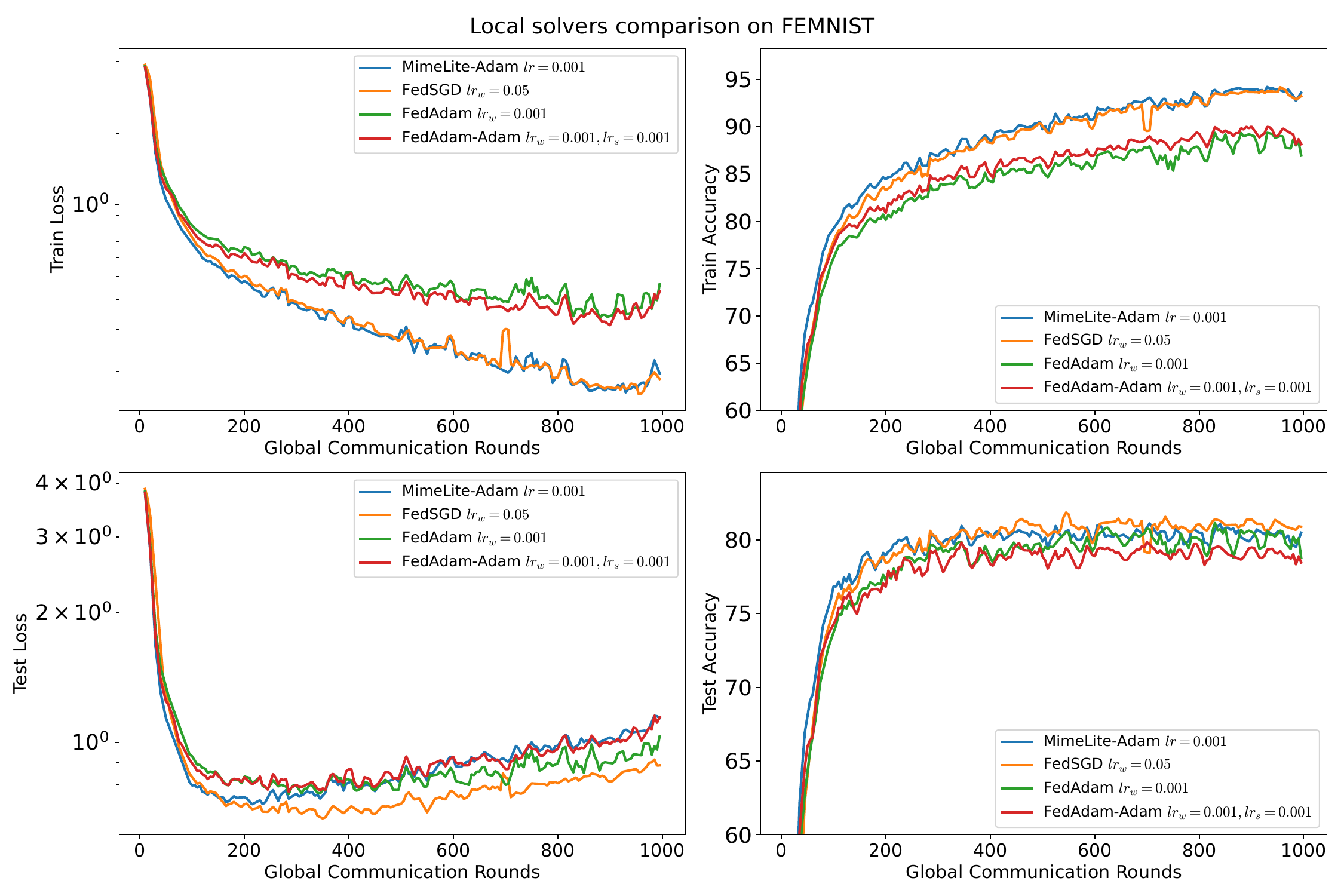}
	\end{subfigure}
	\hfill
	\caption{Different local solvers for prox baselines for training a CNN model over 100 workers using data from the \texttt{FEMNIST} dataset. The number of local communication rounds is fixed at 3 and the number of worker optimizer steps is fixed at 3. Nice sampling with a minibatch size of 10 is used. $\gamma$ is fixed at 1.0.}
	\label{fig:dl_solvers_comparison}
\end{figure}
\subsection{Prox solvers baselines}
We compare baselines from \refeq{sec:local_solvers} for training a CNN model over 100 workers using data from the \texttt{FEMNIST} dataset, as shown in \Cref{fig:dl_solvers_comparison}. The number of local communication rounds and worker optimizer steps is consistent among various solvers for the purpose of fair comparison. All local solvers optimize the local objective, which is prox on the selected cohort. The solvers compared are: \algname{LocalGD} referred as \algname{FedSGD}\citep{FedAvg} - the Federated Averaging algorithm with SGD as the worker optimizer, \algname{FedAdam} - the Federated Averaging algorithm with Adam as the worker optimizer, \algname{FedAdam-Adam} based on the FedOpt framework \citep{reddi2020adaptive}, and finally \algname{MimeLite-Adam}, which is based on the \algname{Mime}~\citep{MIME} framework and the Adam optimizer. The hyperparameter search included a double-level sweep of the optimizer learning rates: $[0.00001, 0.0001, 0.001, 0.01, 0.1]$, followed by $[0.25, 0.5, 1.0, 2.5, 5] * lr_{\text{best}}$. One can see that all methods perform similarly, with \algname{MimeLite-Adam} and \algname{FedSGD} converging better on the test data.

\begin{figure}[!tb]
	\centering
	\begin{subfigure}[b]{0.45\textwidth}
		\centering
		\includegraphics[width=1.0\textwidth, trim=0 350 0 0, clip]{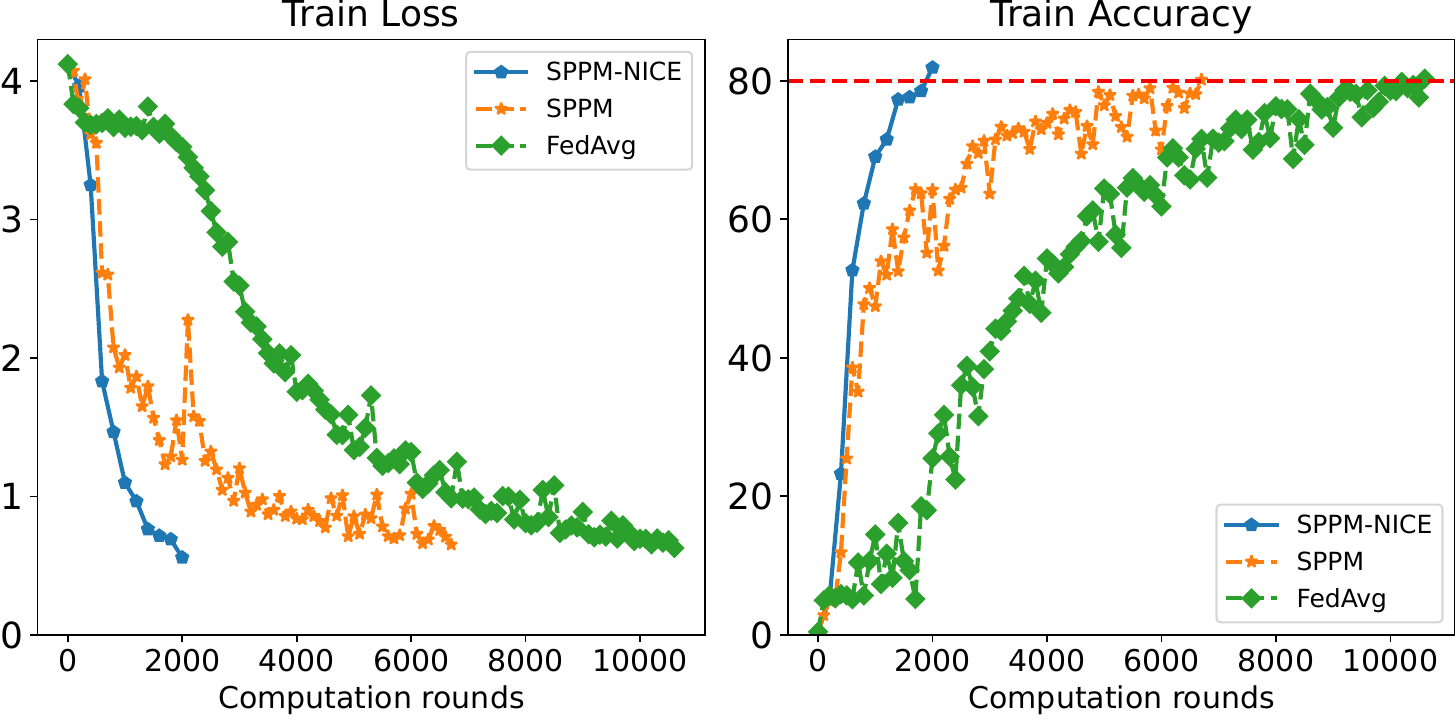}
	\end{subfigure}
	\begin{subfigure}[b]{0.45\textwidth}
		\centering
		\includegraphics[width=1.0\textwidth, trim=0 0 0 350, clip]{img/cohort_squeeze/deep_learning/2-computation_cost.pdf}
	\end{subfigure}
	\caption{Accuracy compared with baselines.}
	\label{fig:dl2}
\end{figure}  

\section{Missing proof and additional theoretical analysis}
\subsection{Facts used in the proof}
\begin{fact}[Differentiation of integral with a parameter (theorem 2.27 from \cite{Folland1984RealAM})]\label{fact:dif-parameter}
	Suppose that $f: X \times[a, b] \rightarrow \mathbb{C}(-\infty<a<b<\infty)$ and that $f(\cdot, t): X \rightarrow \mathbb{C}$ is integrable for each $t \in[a, b]$. Let $F(t)=\int_X f(x, t) d \mu(x)$.
	\begin{enumerate}[label=\alph*.]
		\item Suppose that there exists $g \in L^1(\mu)$ such that $|f(x, t)| \leq g(x)$ for all $x, t$. If $\lim _{t \rightarrow t_0} f(x, t)=f\left(x, t_0\right)$ for every $x$, then $\lim _{t \rightarrow t_0} F(t)=F\left(t_0\right)$; in particular, if $f(x, \cdot)$ is continuous for each $x$, then $F$ is continuous.
		\item Suppose that $\partial f / \partial t$ exists and there is a $g \in L^1(\mu)$ such that $|(\partial f / \partial t)(x, t)| \leq$ $g(x)$ for all $x, t$. Then $F$ is differentiable and $F^{\prime}(x)=\int(\partial f / \partial t)(x, t) d \mu(x)$.
	\end{enumerate}
\end{fact}
\begin{fact}[Tower Property]\label{fact:tower}
	For any random variables $X$ and $Y$, we have
	\begin{align*}
		\ec{\ec{X|Y}} = \ec{X}.
	\end{align*}	
\end{fact}

\begin{fact}[Every point is a fixed point \citep{SPPM}]\label{fact:fact1}
	Let $\varphi: \Rd \to \mbR$  be a convex differentiable function. Then 
	\begin{align*}
		\prox_{\gamma\varphi}(x + \gamma \nabla \varphi(x)) = x, \qquad \forall \gamma >0, \quad \forall x\in \Rd.
	\end{align*}
	In particular, if $x_{\star}$ is a minimizer of $\varphi$, then $\prox_{\gamma \varphi}(x_\star) = x_\star$.
\end{fact}

\begin{proof}
	Evaluating the proximity operator is equivalent to 
	\begin{align*}
		\prox_{\gamma \varphi}(y) = \argmin_{x\in \Rd}\left(\varphi(x) + \frac{1}{2\gamma}\sqn{x -y} \right).
	\end{align*}
	This is a strongly convex minimization problem for any $\gamma>0$, hence the (necessarily unique) minimizer $x=\prox_{\gamma \varphi} (y)$ of this problem satisfies the first-order optimality condition 
	\begin{align*}
		\nabla \varphi(x) + \frac{1}{\gamma} (x - y) = 0.
	\end{align*}
	Solving for $y$, we observe that this holds for $y=x + \gamma \nabla \phi(x)$. Therefore, $x = \prox_{\gamma \varphi}(x + \gamma \nabla \varphi (x))$.
\end{proof}

\begin{fact}[Contractivity of the prox \citep{mishchenko2022proximal}]\label{fact:fact2}
	If $\varphi$ is differentiable and $\mu$-strongly convex, then for all $\gamma >0$ and for any $x, y\in \Rd$ we have 
	\begin{align*}
		\norm{\prox_{\gamma\varphi}(x) - \prox_{\gamma\varphi}(y)}^2 \leq \frac{1}{(1+\gamma \mu)^2}\norm{x - y}^2.
	\end{align*}
\end{fact}

\begin{fact}[Recurrence (\citealp{SPPM}, Lemma 1)]\label{fact:fact3}
	Assume that a sequence $\{s_t\}_{t\geq 0}$ of positive real numbers for all $t\geq 0$ satisfies 
	\begin{align*}
		s_{t+1} \leq as_t + b,
	\end{align*}
	where $0<a<1$ and $b\geq 0$. Then the sequence for all $t\geq 0$ satisfies 
	\begin{align*}
		s_t \leq a^t s_0 + b\min\left\{t, \frac{1}{1 -a} \right\}.
	\end{align*}
\end{fact}

\begin{proof}
	Unrolling the recurrence, we get 
	\begin{align*}
		s_t \leq as_{t - 1} + b \leq a(as_{t -2} + b) + b \leq \cdots \leq a^t s_0 + b\sum_{i=0}^{t-1} a^i.
	\end{align*}
	
	We can now bound the sum $\sum_{i=0}^{t-1} a^i$ in two different ways. First, since $a<1$, we get the estimate 
	\begin{align*}
		\sum_{i=0}^{t-1} a^i \leq \sum^{t-1}_{i=0} 1 = t.
	\end{align*}
	Second, we sum a geometic series
	\begin{align*}
		\sum_{i=0}^{t-1} a^i\leq \sum_{i=0}^{\inf} a^i = \frac{1}{1-a}.
	\end{align*}
	
	Note that either of these bounds can be better. So, we apply the best of these bounds. Substituing the above two bounds gived the target inequality. 
\end{proof}

\subsection{Simplified proof of \algname{SPPM}}\label{sec:proof_simple}
We provide a simplified proof of \algname{SPPM}~\citep{SPPM} in this section. 
Using the fact that $x_\star = \prox_{\gamma f_{\xi_t}}(x_\star + \gamma \nabla f_{\xi_t}(x_\star))$ (see \Cref{fact:fact1}) and then applying contraction of the prox (\Cref{fact:fact2}), we get 

\begin{align*}
	\sqn{x_{t+1} - x_\star} &= \sqn{\prox_{\gamma f_{\xi_t}} - x_\star}\\
	&\stackrel{(\Cref{fact:fact1})}{=} \sqn{\prox_{\gamma f_{\xi_t}}(x_t) - \prox_{\gamma f_{\xi_t}}(x_\star + \gamma \nabla f_{\xi_t}(x_\star))}\\
	&\stackrel{(\Cref{fact:fact2})}{\leq} \frac{1}{(1+\gamma \mu)^2} \sqn{x_t - (x_\star + \gamma \nabla f_{\xi_t}(x_\star))}\\
	&= \frac{1}{(1+\gamma \mu)^2}\left(\sqn{x_t - x_\star} - 2\gamma \ev{\nabla f_{\xi_t}(x_\star), x_t - x_\star} + \gamma^2 \sqn{\nabla f_{\xi_t}(x_\star)} \right).
\end{align*} 

Taking expectation on both sides, conditioned on $x_t$, we get 

\begin{align*}
	\ec{\sqn{x_{t+1} - x_\star}|x_t} &\leq \frac{1}{(1+\gamma \mu)^2} \left(\sqn{x_t - x_\star} - 2\gamma \ev{\ec{\nabla f_{\xi_t}(x_\star)}, x_t -x_\star} + \gamma^2 \ec{\sqn{\nabla f_{\xi_t}(x_\star)}} \right)\\
	&=\frac{1}{(1+\gamma \mu)^2} \left(\sqn{x_t - x_\star} + \gamma^2 \sigma^2_{\star} \right),
\end{align*}

where we used the fact that $\ec{\nabla f_{\xi_t} (x_\star)} = \nabla f(x_\star) = 0$ and $\sigma^2_{\star} \eqdef \ec{\sqn{\nabla f_{\xi_t}(x_\star)}}$. Taking expectation again and applying the tower property (\Cref{fact:tower}), we get 

\begin{align*}
	\ec{\sqn{x_{t+1} - x_\star}} \leq \frac{1}{(1+\gamma\mu)^2}\left(\sqn{x_t - x_\star} + \gamma^2 \sigma^2_{\star} \right).
\end{align*}

It only remains to solve the above recursion. Luckily, that is exactly what \Cref{fact:fact3} does. In particular, we use it with $s_t = \ec{\sqn{x_t - x_\star}}, a = \frac{1}{(1+\gamma\mu)^2}$ and $b = \frac{\gamma^2 \sigma_{\star}^2}{(1+\gamma\mu)^2}$ to get 

\begin{align*}
	\ec{\sqn{x_t - x_\star}} &\stackrel{(\Cref{fact:fact3})}{\leq} \left(\frac{1}{1+\gamma\mu} \right)^{2t} \sqn{x_0 - x_\star} + \frac{\gamma^2 \sigma^2_{\star}}{(1+\gamma\mu)^2} \min\left\{t, \frac{(1+\gamma\mu)^2}{(1+\gamma\mu)^2 - 1} \right\}\\
	&\leq \left(\frac{1}{1+\gamma\mu} \right)^{2t} \sqn{x_0 - x_\star} + \frac{\gamma^2\sigma^2_{\star}}{(1+\gamma\mu)^2 - 1}\\
	&\leq \left(\frac{1}{1+\gamma\mu} \right)^{2t} \sqn{x_0 - x_\star} + \frac{\gamma\sigma^2_{\star}}{\gamma\mu^2 + 2\mu}.
\end{align*}

\subsection{Missing proof of Theorem \ref{thm:sppm_as}}
We first prove the following useful lemma. 

\begin{lemma}\label{lem:0002}
	Let $\phi_\xi: \mathbb{R}^d \rightarrow \mathbb{R}$ be differentiable functions for almost all $\xi \sim\mathcal{D}$, with $\phi_\xi$ being $\mu_\xi$-strongly convex for almost all $\xi \sim\mathcal{D}$. Further, let $w_\xi$ be positive scalars. Then the function $\phi:=\ec[\xi\sim\mathcal{D}]{ w_\xi \phi_\xi}$ is $\mu$-strongly convex with $\mu=\ec[\xi\sim\mathcal{D}]{w_\xi \mu_\xi}$.
\end{lemma}

\begin{proof}
	By assumption,
	$$
	\phi_\xi(y)+\left\langle\nabla \phi_\xi(y), x-y\right\rangle+\frac{\mu_\xi}{2}\|x-y\|^2 \leq \phi_\xi(x), \quad \text{for almost all }\xi\in\mathcal{D},\forall x, y \in \mathbb{R}^d .
	$$
	
	This means that
	$$
	\ec[\xi\sim\mathcal{D}]{w_\xi\left(\phi_\xi(y)+\left\langle\nabla \phi_\xi(y), x-y\right\rangle+\frac{\mu_\xi}{2}\|x-y\|^2\right)} \leq \ec[\xi\sim\mathcal{D}]{w_\xi \phi_\xi(x)}, \quad \forall x, y \in \mathbb{R}^d,
	$$
	which is equivalent to
	$$
	\phi(y)+\langle\nabla \phi(y), x-y\rangle+\frac{\ec[\xi\sim\mathcal{D}]{w_\xi \mu_\xi}}{2}\|x-y\|^2 \leq \phi(x), \quad \forall x, y \in \mathbb{R}^d,
	$$
	
	So, $\phi$ is $\mu$-strongly convex.
\end{proof}

Now, we are ready to prove our main Theorem \ref{thm:sppm_as}.

\begin{proof}
	Let $C$ be any (necessarily nonempty) subset of $[n]$ such that $p_C>0$. Recall that in view of \Cref{eqn:8004} we have
	$$
	f_C(x)=\ec[\xi\sim\mathcal{D}]{\frac{I\left(\xi\in C\right)}{p_\xi} f_\xi(x)}
	$$
	i.e., $f_C$ is a conic combination of the functions $\left\{f_\xi: \xi \in C\right\}$ with weights $w_\xi=\frac{I\left(\xi\in C\right)}{p_\xi}$. Since each $f_\xi$ is $\mu_\xi$-strongly convex, \Cref{lem:0002} says that $f_C$ is  $\mu_C$-strongly convex with
	$$
	\mu_C:=\ec[\xi\sim\mathcal{D}]{\frac{I\left(\xi\in C\right)\mu_\xi}{p_\xi}} .
	$$
	
	So, every such $f_C$ is $\mu$-strongly convex with
	$$
	\mu=\mu_{\mathrm{AS}}:=\min _{C \subseteq[n], p_C>0}\ec[\xi\sim\mathcal{D}]{\frac{I\left(\xi\in C\right)\mu_\xi}{p_\xi}}.
	$$
	
	Further, the quantity $\sigma_{\star}^2$ from (2.3) is equal to
	$$
	\sigma_{\star}^2:=\mathrm{E}_{\xi \sim \mathcal{D}}\left[\left\|\nabla f_{\xi}\left(x_{\star}\right)\right\|^2\right] \stackrel{Eqn.~(\ref{eqn:8010})}{=} \sum_{C \subseteq[n], p_C>0} p_C\left\|\nabla f_C\left(x_{\star}\right)\right\|^2:=\sigma_{\star, \mathrm{AS}}^2 .
	$$
	
	Incorporating \Cref{sec:proof_simple} into the above equation, we prove the theorem.
\end{proof}

\subsection{Theory for expectation formulation}\label{sec:exp-formulation}
We will formally define our optimization objective, focusing
on minimization in expectation form. We consider
\begin{align}\label{eqn:obj_1_exp}
	\min_{x\in \Rd} {f(x)\eqdef \ec[\xi\sim \mathcal{D}]{f_{\xi}(x)} },
\end{align}
where $f_{\xi}: \Rd \to \mbR$, $\xi\sim \mathcal{D}$ is a random variable following distribution $\mathcal{D}$. 
\begin{assumption}\label{asm:differential_exp}
	Function $f_{\xi}: \Rd \to \mbR$  is differentiable for almost all samples $\xi \sim \mathcal{D}$.
\end{assumption}

This implies that $f$ is differentiable. We will implicitly assume that the order of differentiation and expectation can be swapped \footnote{This assumption satisfies the conditions required for the theorem about differentiating an integral with a parameter (\Cref{fact:dif-parameter}).}, which means that 
\begin{align*}
	\nabla f(x) \stackrel{Eqn.~(\ref{eqn:obj_1})}{=} \nabla \ec[\xi\sim \mathcal{D}]{f_{\xi}(x)} = \ec[\xi\sim \mathcal{D}]{\nabla f_{\xi}(x)}.
\end{align*}

\begin{assumption}\label{asm:strongly_convex_exp}
	Function $f_\xi: \Rd \to \mbR$  is $\mu$-strongly convex for almost all samples $\xi\sim \mathcal{D}$, where $\mu > 0$. That is 
	\begin{align*}
		f_{\xi}(y) + \ev{\nabla f_{\xi}, x - y} + \frac{\mu}{2}\sqn{x - y} \leq f_{\xi}(x),
	\end{align*} for all $x, y \in \Rd$.
\end{assumption}

This implies that $f$ is $\mu$-strongly convex, and hence $f$ has a unique minimizer, which we denote by $x_\star$. We know that $\nabla f(x_\star) = 0$. Notably, we do \emph{not} assume $f$ to be $L$-smooth.

Let $\mathcal{S}$ be a probability distribution over all \emph{finite} subsets of $\mathbb{N}$. Given a random set $S\sim \mathcal{S}$, we define 
\begin{align*}
	p_i \eqdef \operatorname{Prob}(i \in {S}), \quad i \in \mathbb{N}.
\end{align*}

We will restrict our attention to proper and nonvacuous random sets. 

\begin{assumption}\label{asm:valid_sampling_exp}
	$S$ is proper (i.e., $p_i > 0$ for all $i\in \mathbb{N}$) and nonvacuous (i.e., $\operatorname{Prob}({S} = \emptyset) = 0$).
\end{assumption}

Let $C$ be the selected cohort. Given $\emptyset \neq C \subset\mathbb{N}$ and $i \in\mathbb{N}$, we define
\begin{align}\label{eqn:8001}
	v_i(C):= \begin{cases}\frac{1}{p_i} & i \in C \\ 0 & i \notin C,\end{cases}
\end{align}
and
\begin{align}\label{eqn:8004}
	f_C(x):=\ec[\xi\sim\mathcal{D}]{v_{\xi}(C) f_{\xi}(x)}\stackrel{Eqn.~(\ref{eqn:8001})}{=}\ec[\xi\sim\mathcal{D}]{\frac{I\left(\xi\in C\right)}{p_{\xi}} f_{\xi}(x)} .
\end{align}

Note that $v_i(S)$ is a random variable and $f_S$ is a random function. By construction, $\mathrm{E}_{S \sim \mathcal{S}}\left[v_i(S)\right]=1$ for all $i \in\mathbb{N}$, and hence
{
	\begin{align*}
		&\ec[{S} \sim \mathcal{S}]{f_{{S}}(x)} =\ec[{S} \sim \mathcal{S}]{\ec[\xi\sim\mathcal{D}]{v_{\xi}(C) \nabla f_{\xi}(x)}}\\
		&\qquad =\ec[\xi\sim\mathcal{D}]{\ec[{S}\sim \mathcal{S}]{v_{\xi}(S)} \nabla f_{\xi}(x)}=\ec[\xi\sim\mathcal{D}]{f_\xi(x)}=f(x).
\end{align*}}

Therefore, the optimization problem in \Cref{eqn:obj_1} is equivalent to the stochastic optimization problem
\begin{align}\label{eqn:obj_4}
	\min _{x \in \mathbb{R}^d}\left\{f(x):=\mathrm{E}_{S \sim \mathcal{S}}\left[f_S(x)\right]\right\} .
\end{align}

Further, if for each $C \subset\mathbb{N}$ we let $p_C:=\operatorname{Prob}(S=C), f$ can be written in the equivalent form
{\small
	\begin{align}\label{eqn:8010}
		f(x)=\ec[S \sim \mathcal{S}]{f_S(x)}=\sum_{C \subset\mathbb{N}} p_C f_C(x)=\sum_{C \subset\mathbb{N}, p_C>0} p_C f_C(x).
\end{align}}
\begin{theorem}[Main Theorem]\label{thm:sppm_as_exp}
	Let \Cref{asm:differential} (diferentiability) and \Cref{asm:strongly_convex} (strong convexity) hold. Let $S$ be a random set satisfying \Cref{asm:valid_sampling}, and define
	{
		\begin{align}\label{eqn:8003}
			\mu_{\mathrm{AS}}&:=\min _{C \subset\mathbb{N}, p_C>0} \ec[\xi\sim\mathcal{D}]{\frac{I\left(\xi\in C\right)\mu_\xi}{p_\xi}}, \notag\\ 
			\sigma_{\star, \mathrm{AS}}^2 &:=\sum_{C \subset\mathbb{N}, p_C>0} p_C\left\|\nabla f_C\left(x_{\star}\right)\right\|^2 .
	\end{align}}
	
	Let $x_0 \in \mathbb{R}^d$ be an arbitrary starting point. Then for any $t \geq 0$ and any $\gamma>0$, the iterates of \algname{SPPM-AS} (\Cref{alg:sppm_as}) satisfy
	{\small
		$$
		\mathrm{E}\left[\left\|x_t-x_{\star}\right\|^2\right] \leq\left(\frac{1}{1+\gamma \mu_{\mathrm{AS}}}\right)^{2t}\left\|x_0-x_{\star}\right\|^2+\frac{\gamma \sigma_{\star, \mathrm{AS}}^2}{\gamma \mu_{\mathrm{AS}}^2+2 \mu_{\mathrm{AS}}} .
		$$
	}
\end{theorem}

\subsection{Missing proof of iteration complexity of \algname{SPPM-AS}}\label{sec:proof_iteration_complexity}
We have seen above that accuracy arbitrarily close to (but not reaching) $\nicefrac{\sigma^2_{\star, \mathrm{AS}}}{\mu_{\mathrm{AS}}^2}$ can be achieved via a single step of the method, provided the stepsize $\gamma$ is large enough. Assume now that we aim for $\epsilon$ accuracy where $\epsilon \leq \nicefrac{\sigma^2_{\star, \mathrm{AS}}}{\mu_{\mathrm{AS}}^2}$. Using the inequality $1-k\leq \exp(-k)$ which holds for all $k>0$, we get 

\begin{align*}
	\left(\frac{1}{1+\gamma \mu_{\mathrm{AS}}}\right)^{2 t}=\left(1-\frac{\gamma \mu}{1+\gamma \mu_{\mathrm{AS}}}\right)^{2 t} \leq \exp \left(-\frac{2 \gamma \mu_{\mathrm{AS}}t}{1+\gamma \mu_{\mathrm{AS}}}\right)
\end{align*}

Therefore, provided that
$$
t \geq \frac{1+\gamma \mu_{\mathrm{AS}}}{2 \gamma \mu_{\mathrm{AS}}} \log \left(\frac{2\left\|x_0-x_{\star}\right\|^2}{\varepsilon}\right),
$$
we get $\left(\frac{1}{1+\gamma \mu_{\mathrm{AS}}}\right)^{2 t}\left\|x_0-x_{\star}\right\|^2 \leq \frac{\varepsilon}{2}$. Furthermore, as long as $\gamma \leq \frac{2 \varepsilon \mu_{\mathrm{AS}}}{2 \sigma_{\star, \mathrm{AS}}^2-\varepsilon \mu_{\mathrm{AS}}^2}$ (this is true provided that the more restrictive but also more elegant-looking condition $\gamma \leq \nicefrac{\varepsilon \mu_{\mathrm{AS}}}{\sigma_{\star, \mathrm{AS}}^2}$ holds), we get
$
\frac{\gamma \sigma_{\star, \mathrm{AS}}^2}{\gamma \mu_{\mathrm{AS}}^2+2 \mu_{\mathrm{AS}}} \leq \frac{\varepsilon}{2} .
$
Putting these observations together, we conclude that with the stepsize $\gamma= \nicefrac{\varepsilon\mu_{\mathrm{AS}}}{\sigma_{\star, \mathrm{AS}}^2}$, we get
$
\mathrm{E}\left[\left\|x_t-x_{\star}\right\|^2\right] \leq \varepsilon
$
provided that
\begin{align*}
	&t \geq \frac{1+\gamma \mu_{\mathrm{AS}}}{2 \gamma \mu_{\mathrm{AS}}} \log \frac{2\left\|x_0-x_{\star}\right\|^2}{\varepsilon} =\left(\frac{\sigma_{\star, \mathrm{AS}}^2}{2 \varepsilon \mu_{\mathrm{AS}}^2}+\frac{1}{2}\right) \log \left(\frac{2\left\|x_0-x_{\star}\right\|^2}{\varepsilon}\right) .
\end{align*}

\subsection{$\sigma_{\star, \mathrm{NICE}}^2(\tau)$ and $\mu_{\mathrm{NICE}}(\tau)$ are Monotonous Functions of $\tau$}\label{sec:proof_nice}
\begin{lemma}
	For all $0\leq\tau\leq n-1$: 
	\begin{enumerate}
		\item $\mu_{\mathrm{NICE}}(\tau + 1)\geq\mu_{\mathrm{NICE}}(\tau)$, 
		\item $\sigma_{\star, \mathrm{NICE}}^2(\tau)=\frac{\frac{n}{\tau}-1}{n-1}\sigma_{\star, \mathrm{NICE}}^2(1)\leq\frac{1}{\tau}\sigma_{\star, \mathrm{NICE}}^2(1)$. 
	\end{enumerate}
	\begin{proof}
		\begin{enumerate}
			\item Pick any $1 \leq \tau<n$, and consider a set $C$ for which the minimum is attained in
			$$
			\mu_{\mathrm{NICE}}(\tau+1)=\min _{C \subseteq[n],|C|=\tau+1} \frac{1}{\tau+1} \sum_{i \in C} \mu_i .
			$$
			
			Let $j=\arg \max _{i \in C} \mu_i$. That is, $\mu_j \geq \mu_i$ for all $i \in C$. Let $C_j$ be the set obtained from $C$ by removing the element $j$. Then $\left|C_j\right|=\tau$ and
			$$
			\mu_j=\max _{i \in C} \mu_i \geq \max _{i \in C_j} \mu_i \geq \frac{1}{\tau} \sum_{i \in C_j} \mu_i.
			$$
			
			By adding $\sum_{i \in C_j} \mu_i$ to the above inequality, we obtain
			$$
			\mu_j+\sum_{i \in C_j} \mu_i \geq \frac{1}{\tau} \sum_{i \in C_j} \mu_i+\sum_{i \in C_j} \mu_i .
			$$
			
			Observe that the left-hand side is equal to $\sum_{i \in C} \mu_i$, and the right-hand side is equal to $\frac{\tau+1}{\tau} \sum_{i \in C_j} \mu_i$. If we divide both sides by $\tau+1$, we obtain
			$$
			\frac{1}{\tau+1} \sum_{i \in C} \mu_i \geq \frac{1}{\tau} \sum_{i \in C_j} \mu_i.
			$$
			
			Since the left-hand side is equal to $\mu_{\mathrm{NICE}}(\tau+1)$, and the right hand side is an upper bound on $\mu_{\mathrm{NICE}}(\tau)$, we conclude that $\mu_{\mathrm{NICE}}(\tau+1) \geq \mu_{\mathrm{NICE}}(\tau)$.
			\item
			In view of \eqref{eqn:8004} we have
			\begin{eqnarray}
				f_C(x) = \sum_{i \in C} \frac{1}{n p_i} f_i(x) .
			\end{eqnarray}
			
			\begin{eqnarray}
				\sigma_{\star, \mathrm{AS}}^2 
				&=& \mathrm{E}_{S \sim \mathcal{S}} \sb{\sqnorm{\sum_{i \in S} \frac{1}{n p_i} \nabla f_i(x_{\star})}}
				= \mathrm{E}_{S \sim \mathcal{S}} \sb{\sqnorm{\sum_{i \in S} \frac{1}{\tau} \nabla f_i(x_{\star})}}\notag\\
			\end{eqnarray}
			
			Let $\chi_{i}$ be the random variable defined by
			\begin{eqnarray}
				\chi_{j}=\left\{\begin{array}{ll}
					1 & j \in S \\
					0 & j \notin S.
				\end{array}\right.
			\end{eqnarray}
			It is easy to show that
			\begin{eqnarray}
				\Exp{\chi_{j}} = \operatorname{Prob}(j \in S)=\fr{\tau}{n}.
			\end{eqnarray}
			Let fix the cohort S. Let $\chi_{i j}$ be the random variable defined by
			\begin{eqnarray}
				\chi_{ij}=\left\{\begin{array}{ll}
					1 & i \in S \text { and } j \in S \\
					0 & \text { otherwise}.
				\end{array}\right.
			\end{eqnarray}
			Note that
			\begin{eqnarray}
				\chi_{ij}=\chi_{i} \chi_{j}.
			\end{eqnarray}
			Further, it is easy to show that
			\begin{eqnarray}
				\Exp{\chi_{ij}}=\operatorname{Prob}(i \in S, j \in S)=\fr{\tau(\tau-1)}{n(n-1)}.
			\end{eqnarray}
			Denote $a_i := \nabla f_i(x_{\star}).$
			
			\begin{eqnarray*}
				\Exp{\sqnorm{\fr{1}{\tau} \sum_{i \in S} a_i}} 
				&=&\fr{1}{\tau^{2}} \Exp{\sqnorm{\sum_{i \in S} a_{i}}}\\
				&=&\fr{1}{\tau^{2}} \Exp{\sqnorm{\sumin \chi_{i} a_{i}}} \\
				&=& \fr{1}{\tau^{2}} \Exp{\sumin \sqnorm{\chi_{i} a_{i}}+\sum_{i \neq j}\left\langle\chi_{i} a_{i}, \chi_{j} a_{j}\right\rangle} \\
				&=& \fr{1}{\tau^{2}} \Exp{\sumin\sqnorm{\chi_{i} a_{i}}+\sum_{i \neq j} \chi_{ij}\left\langle a_{i}, a_{j}\right\rangle} \\
				&=& \fr{1}{\tau^{2}} \sumin \Exp{\chi_{i}}\sqnorm{a_{i}}+\sum_{i \neq j} \Exp{\chi_{ij}}\left\langle a_{i}, a_{j}\right\rangle\\
				&=& \fr{1}{\tau^{2}}\left(\fr{\tau}{n} \sumin\sqnorm{a_{i}}+\fr{\tau(\tau-1)}{n(n-1)} \sum_{i \neq j}\left\langle a_{i}, a_{j}\right\rangle\right) \\
				&=&\fr{1}{\tau n} \sumin\sqnorm{a_{i}}+\fr{\tau-1}{\tau n(n-1)} \sum_{i \neq j}\left\langle a_{i}, a_{j}\right\rangle \\
				&=& \fr{1}{\tau n} \sumin\sqnorm{a_{i}}+\fr{\tau-1}{\tau n(n-1)}\left(\sqnorm{\sumin a_{j}}-\sumin \sqnorm{a_{i}}\right) \\
				&=&\fr{n-\tau}{\tau(n-1)} \fr{1}{n} \sumin\sqnorm{a_{i}}+\fr{n(\tau-1)}{\tau(n-1)}\sqnorm{\fr{1}{n} \sumin a_{i}}\\
				&=&\fr{n-\tau}{\tau(n-1)} \fr{1}{n} \sumin\sqnorm{\nabla f_i(x_{\star})}+\fr{n(\tau-1)}{\tau(n-1)}\sqnorm{\fr{1}{n} \sumin \nabla f_i(x_{\star})}\\
				&=&\fr{n-\tau}{\tau(n-1)} \fr{1}{n} \sumin\sqnorm{\nabla f_i(x_{\star})}\\
				&\le&\fr{1}{\tau} \fr{1}{n} \sumin\sqnorm{\nabla f_i(x_{\star})}
			\end{eqnarray*}
			
		\end{enumerate}
	\end{proof}
\end{lemma}

\subsection{Missing proof of Lemma \ref{lem:vr_ss}} 
For ease of notation, let $a_i=\nabla f_i\left(x_{\star}\right)$ and $\hat{z}_j=\left|C_j\right| a_{\xi_j}$, and recall that
\begin{align}\label{eqn:412}
	\sigma_{\star, \mathrm{SS}}^2=\mathrm{E}_{\xi_1, \ldots, \xi_b}\left[\left\|\frac{1}{n} \sum_{j=1}^b \hat{z}_j\right\|^2\right].
\end{align}
where $\xi_j \in C_j$ is chosen uniformly at random. Further, for each $j \in[b]$, let $z_j=\sum_{i \in C_j} a_i$. Observe that $\sum_{j=1}^b z_j=\sum_{j=1}^b \sum_{i \in C_j} a_i=\sum_{i=1}^n a_i=\nabla f\left(x_{\star}\right)=0$. Therefore,
\begin{align}\label{eqn:413}
	\left\|\frac{1}{n} \sum_{j=1}^b \hat{z}_j\right\|^2 & =\frac{1}{n^2}\left\|\sum_{j=1}^b \hat{z}_j-\sum_{j=1}^b z_j\right\|^2\notag \\
	& =\frac{b^2}{n^2}\left\|\frac{1}{b} \sum_{j=1}^b\left(\hat{z}_j-z_j\right)\right\|^2\notag \\
	& \leq \frac{b^2}{n^2} \frac{1}{b} \sum_{j=1}^b\left\|\hat{z}_j-z_j\right\|^2\notag \\
	& =\frac{b}{n^2} \sum_{j=1}^b\left\|\hat{z}_j-z_j\right\|^2,
\end{align}

where the inequality follows from convexity of the function $u \mapsto\|u\|^2$. Next,
\begin{align}\label{eqn:414}
	\left\|\hat{z}_j-z_j\right\|^2=\left\|\left|C_j\right| a_{\xi_j}-\sum_{i \in C_j} a_i\right\|^2=\left|C_j\right|^2\left\|a_{\xi_j}-\frac{1}{\left|C_j\right|} \sum_{i \in C_j} a_i\right\|^2 \leq\left|C_j\right|^2 \sigma_j^2 .
\end{align}

By combining \Cref{eqn:412}, \Cref{eqn:413} and \Cref{eqn:414}, we get
$$
\begin{aligned}
	& \sigma_{\star, \mathrm{SS}}^2 \stackrel{Eqn.~(\ref{eqn:412})}{=} \mathrm{E}_{\xi_1, \ldots, \xi_b}\left[\left\|\frac{1}{n} \sum_{j=1}^b \hat{z}_j\right\|^2\right] \\
	& \stackrel{Eqn.~(\ref{eqn:413})}{\leq} \quad \mathrm{E}_{\xi_1, \ldots, \xi_b}\left[\frac{b}{n^2} \sum_{j=1}^b\left\|\hat{z}_j-z_j\right\|^2\right] \\
	& \stackrel{Eqn.~(\ref{eqn:414})}{\leq} \quad \mathrm{E}_{\xi_1, \ldots, \xi_b}\left[\frac{b}{n^2} \sum_{j=1}^b\left|C_j\right|^2 \sigma_j^2\right] \\
	& =\frac{b}{n^2} \sum_{j=1}^b\left|C_j\right|^2 \sigma_j^2 .
\end{aligned}
$$

The last expression can be further bounded as follows:
$$
\frac{b}{n^2} \sum_{j=1}^b\left|C_j\right|^2 \sigma_j^2 \leq \frac{b}{n^2}\left(\sum_{j=1}^b\left|C_j\right|^2\right) \max _j \sigma_j^2 \leq \frac{b}{n^2}\left(\sum_{j=1}^b\left|C_j\right|\right)^2 \max _j \sigma_j^2=b \max _j \sigma_j^2,
$$
where the second inequality follows from the relation $\|u\|_2 \leq\|u\|_1$ between the $L_2$ and $L_1$ norms, and the last identity follows from the fact that $\sum_{j=1}^b\left|C_j\right|=n$.

\subsection{Stratified sampling against block sampling and nice sampling}\label{sec:ss_vs_bs_nice}
In this section, we present a theoretical comparison of block sampling and its counterparts, providing a theoretical justification for selecting block sampling as the default clustering method in future experiments. Additionally, we compare various sampling methods, all with the same sampling size, $b$: $b$-nice sampling, block sampling with $b$ clusters, and block sampling, where all clusters are of uniform size $b$.

\begin{assumption}\label{asm:uniform-clustering}
	For simplicity of comparison, we assume $b$ clusters, each of the same size, $b$:
	\[
	\left|C_1\right|=\left|C_2\right|=\ldots=\left|C_b\right|=b.
	\]
\end{assumption}
It is crucial to acknowledge that, without specific assumptions, the comparison of different sampling methods may not provide meaningful insights. For instance, the scenario described in Lemma \ref{lem:vr_ss}, characterized by complete inter-cluster homogeneity, demonstrates that block sampling achieves a variance term, denoted as \(\sigma_{\star, \mathrm{SS}}^2\), which is lower than the variance terms associated with both block sampling and nice sampling. However, a subsequent example illustrates examples in which the variance term for block sampling surpasses those of block sampling and nice sampling.
\begin{example}\label{example:SS_worse_than_BS_NICE}
	Without imposing any additional clustering assumptions, there exist examples for any arbitrary $n$, such that $\sigma^2_{\star, \mathrm{SS}} \geq \sigma^2_{\star, \mathrm{BS}}$ and $\sigma^2_{\star, \mathrm{SS}} \geq \sigma^2_{\star, \mathrm{NICE}}$.
	\begin{proof}
		\textbf{Counterexample when SS is worse in neighborhood than BS}\\
		Assume we have such clustering and $\nabla f_i(x_\star)$ such that the centroids of each cluster are equal to zero: $\forall i \in [b]$, $\frac{1}{|C_i|}\sum_{j \in C_i}\nabla f_j(x_\star) = 0$. For instance, this can be achieved in the following case: The dimension is $d=2$, all clusters are of equal size $m$, then assign $\forall i \in [b]$, $\forall j \in C_i$, $\nabla f_j(x_\star) = \left(Re\left(\omega^{mj + i}\right), Im\left(\omega^{mj + i}\right)\right)$ where $\omega = \sqrt[n]{1} \in \mathbb{C}$. Let us calculate $\sigma_{\star, \mathrm{BS}}^2$:
		\begin{align*}
			&\sigma_{\star, \mathrm{BS}}^2 := \sum_{j=1}^b q_j \left\| \sum_{i \in C_j} \frac{1}{np_i} \nabla f_i(x_\star) \right\|^2 = \\
			&= \frac{1}{n^2} \sum_{j=1}^b \frac{|C_j|^2}{q_j} \left\| \frac{1}{|C_j|} \sum_{i \in C_j} \nabla f_i(x_\star) \right\|^2 = 0.
		\end{align*}
		As a result:
		\[\sigma_{\star, \mathrm{BS}}^2 = 0 \leq \sigma_{\star, \mathrm{SS}}^2.\]
		\textbf{Counterexample when SS is worse in neighborhood than NICE}\\
		Here, we employ a similar proof technique as in the proof of Lemma \ref{lem:SS_vs_NICE}.
		Let us choose such clustering $\mathcal{C}_{b, \mathrm{SS}, \max} = \argmax_{\mathcal{C}_b} \sigma^2_{\star, \mathrm{SS}}(\mathcal{C}_b)$. Denote $\mathbf{i}_b \eqdef (i_1, \cdots, i_b)$, $\mathbf{C}_b \eqdef C_1 \times \cdots \times C_b$, and $S_{\mathbf{i}_b} \eqdef \left\| \frac{1}{\tau} \sum_{i \in \mathbf{i}_b} \nabla f_i(x_\star) \right\|$. 
		\begin{align*}
			&\sigma_{\star, \mathrm{NICE}}^2 = \frac{1}{C(n, \tau)} \sum_{C \subseteq [n], |C| = \tau} \left\| \frac{1}{\tau} \sum_{i \in C} \nabla f_i(x_\star) \right\|^2 \\
			&= \frac{1}{C(n, b)} \sum_{\mathbf{i}_b \subseteq [n]} S_{\mathbf{i}_b} \\
			&\stackrel{1}{=} \frac{1}{\#_{\text{clusterizations}}} \sum_{\mathcal{C}_b} \frac{1}{b^b} \sum_{\mathbf{i}_b \in \mathbf{C}_b} S_{\mathbf{i}_b} \\
			&= \frac{1}{\#_{\text{clusterizations}}} \sum_{\mathcal{C}_b} \sigma^2_{\star, \mathrm{SS}}(\mathcal{C}_b) \\
			&\stackrel{2}{\leq} \sigma^2_{\star, \mathrm{SS}}(\mathcal{C}_{b, \mathrm{SS}, \max}).
		\end{align*}
		Equation 1 holds because, in every clusterization $\mathcal{C}_b$, there are $\frac{1}{b^b}$ possible sample combinations $\mathbf{i}_b$. Due to symmetry, one can conclude that each combination $S_{\mathbf{i}_b}$ is counted the same number of times. Equation 2 follows from the definition of $\mathcal{C}_{b, \mathrm{SS}, \max}$. \\
		For illustrative purposes, we can demonstrate this effect with a specific example. Let $n=4$ and define $\forall i~ a_i = \nabla f_i(x^*) \in \mathbb{R}^2$. Let $a_1 = (0, 1)^T$, $a_2 = (1, 0)^T$, $a_3 = (0, -1)^T$, and $a_4 = (-1, 0)^T$. Then fix clustering $\mathcal{C}_b = \left\{C_1 = \{a_1, a_3\}, C_2 = \{a_2, a_4\}\right\}$. Then:
		\begin{align*}
			&\sigma_{\star, \mathrm{SS}}^2 = \frac{1}{4} \sum_{\mathbf{i}_b \in \mathcal{C}_b} \left\| \frac{a_{i_1} + a_{i_2}}{2} \right\|^2 \\
			&= \frac{1}{4} \sum_{\mathbf{i}_b \in \mathcal{C}_b} \left\| (\pm\frac{1}{2}, \pm\frac{1}{2}) \right\|^2 \\
			&= \frac{1}{2}.
		\end{align*}
		\begin{align*}
			&\sigma_{\star, \mathrm{NICE}}^2 = \frac{1}{C(4, 2)} \sum_{i < j} \left\| \frac{a_{i} + a_{j}}{2} \right\|^2 \\
			&= \frac{1}{6} \sum_{i<j} \left\| \frac{a_{i} + a_{j}}{2} \right\|^2 \\
			&= \frac{1}{6} \left( \left[ \left\| \frac{a_1 + a_3}{2} \right\|^2 + \left\| \frac{a_2 + a_4}{2} \right\|^2 \right] + 2 \times \left\| \frac{a_{i_1} + a_{i_2}}{2} \right\|^2 \right) \\
			&= \frac{1}{6} \left( 0 + 2 \times 2 \times \frac{1}{2} \right) \\
			&= \frac{1}{3} \\
			&= \frac{2}{3} \times \sigma_{\star, \mathrm{SS}}^2 \\
			&\leq \sigma_{\star, \mathrm{SS}}^2
		\end{align*}
		
	\end{proof}
\end{example}
To select the optimal clustering, we will choose the clustering that minimizes $\sigma^2_{\star, \mathrm{SS}}$.
\begin{definition}[Stratified sampling optimal clustering]\label{def:SS-clustering}
	Denote the clustering of workers into blocks as $\mathcal{C}_b := \{C_1, C_2, \ldots, C_b\}$, such that the disjoint union of all clusters $C_1 \cup C_2 \cup \ldots \cup C_b = [n]$. Define \emph{block sampling Optimal Clustering} as the clustering configuration that minimizes $\sigma_{\star, \mathrm{SS}}^2$, formally given by:
	\[
	\mathcal{C}_{b, \mathrm{SS}} := \argmin_{\mathcal{C}_b} \sigma^2_{\star, \mathrm{SS}}(\mathcal{C}_b).
	\]
\end{definition}
\begin{restatable}{lemma}{lemma5}
		Given \Cref{asm:uniform-clustering}, the following holds: $\sigma_{\star, \mathrm{SS}}^2\left(\mathcal{C}_{b, \mathrm{SS}}\right) \leq \sigma_{\star, \mathrm{NICE}}^2$ ~for arbitrary $b$. Moreover, the variance within the convergence neighborhood of stratified sampling is less than or equal to that of nice sampling: $\frac{\gamma \sigma^2_{\star, \mathrm{SS}}}{\gamma \mu^2_{\mathrm{SS}} + 2\mu_{\mathrm{SS}}}\left(\mathcal{C}_{b, \mathrm{SS}}\right) \leq \frac{\gamma \sigma^2_{\star, \mathrm{NICE}}}{\gamma \mu^2_{\mathrm{NICE}} + 2\mu_{\mathrm{NICE}}}.$
		\begin{proof}
			\begin{enumerate}
			\item Denote $\mathbf{i}_b \eqdef (i_1, \cdots, i_b)$, $\mathbf{C}_b \eqdef C_1 \times \cdots \times C_b$, and $S_{\mathbf{i}_b} \eqdef \left\|\frac{1}{\tau}\sum_{i \in \mathbf{i}_b} \nabla f_i(x_\star)\right\|$. 
			\begin{align*}
				&\sigma_{\star, \mathrm{NICE}}^2 = \frac{1}{C(n, \tau)} \sum_{C \subseteq [n], |C| = \tau} \left\|\frac{1}{\tau}\sum_{i \in C} \nabla f_i(x_\star)\right\|^2 \\
				&= \frac{1}{C(n, b)} \sum_{\mathbf{i}_b \subseteq [n]} S_{\mathbf{i}_b} \\
				&\stackrel{1}{=} \frac{1}{\#_{\text{clusterizations}}} \sum_{\mathcal{C}_b} \frac{1}{b^b} \sum_{\mathbf{i}_b \in \mathbf{C}_b} S_{\mathbf{i}_b} \\
				&= \frac{1}{\#_{\text{clusterizations}}} \sum_{\mathcal{C}_b} \sigma^2_{\star, \mathrm{SS}}(\mathcal{C}_b) \\
				&\stackrel{2}{\geq} \sigma^2_{\star, \mathrm{SS}}(\mathcal{C}_{b, \mathrm{SS}, \min})
			\end{align*}
			Equation 1 holds because, in every clusterization $\mathcal{C}_b$, there are $\frac{1}{b^b}$ possible sample combinations $\mathbf{i}_b$. Due to symmetry, one can conclude that each combination $S_{\mathbf{i}_b}$ is counted the same number of times. Equation 2 follows from the definition of $\mathcal{C}_{b, \mathrm{SS}, \min}$ as the clustering that minimizes $\sigma^2_{\star, \mathrm{SS}}$, according to \Cref{def:SS-clustering}.
			\item The neighborhood size for SPPM-AS is given by $\frac{\gamma \sigma^2_{\star, \mathrm{AS}}}{\gamma \mu^2_{\mathrm{AS}} + 2\mu_{\mathrm{AS}}}$, denoted as $U_{\mathrm{AS}}$ for simplicity. 
			Define:
			\begin{align*}
				\mu_{\mathrm{NICE}(b)} &:= \min_{\substack{C \subseteq [n] \\ |C|=b}} \frac{1}{b} \sum_{i \in C} \mu_i, \\
				\mu_{\mathrm{SS}} &:= \min_{\mathbf{i}_b \in \mathbf{C}_b} \sum_{j=1}^b \frac{\mu_{i_j} |C_j|}{n} \stackrel{\text{Asm.~10}}{=} \min_{\mathbf{i}_b \in \mathbf{C}_b} \sum_{j=1}^b \frac{\mu_{i_j} b}{b^2} = \min_{\mathbf{i}_b \in \mathbf{C}_b} \frac{1}{b} \sum_{j=1}^b \mu_{i_j}.
			\end{align*}

			Using the definition of the set $\mathbf{C}_b := C_1 \times C_2 \times \cdots \times C_b$, we have $\mathbf{C}_b \subseteq \{ C \subseteq [n] \mid |C|=b \}$. Applying this fact, we obtain:
			\[
			\mu_{\mathrm{SS}} = \min_{\mathbf{i}_b \in \mathbf{C}_b} \frac{1}{b} \sum_{j \in \mathbf{i}_b} \mu_j \geq \mu_{\mathrm{NICE}(b)}.
			\]

			Combining the above with $\sigma_{\star, \mathrm{SS}}^2\left(\mathcal{C}_{b, \mathrm{SS}}\right) \leq \sigma_{\star, \mathrm{NICE}}^2$, we obtain that $U_{\mathrm{SS}}\left(\mathcal{C}_{b, \mathrm{SS}}\right) \leq U_{\mathrm{NICE}}$, demonstrating the variance reduction of SS compared to NICE.
			\end{enumerate}
		\end{proof}
	\end{restatable}
	\begin{example}
		Consider the number of clusters and the size of each cluster, with \(b=2\), under \Cref{asm:uniform-clustering}. Then, \(\sigma_{\star, \mathrm{SS}}^2\left(\mathcal{C}_{b, \mathrm{SS}}\right) \leq \sigma^2_{\star, \mathrm{BS}}\).
		\begin{proof}
			Let \(n=4\), \(b=2\). Denote \(\forall i \enspace a_i = \nabla f_i(x_*)\). Define \(S^2 := \sum_{i<j} \left\| \frac{a_i + a_j}{2} \right\|^2\).
			\begin{align*}
				&\sigma^2_{\star, \text{SS}} = \frac{1}{4} \left(S^2 - \left\| \frac{a_{C_1^1} + a_{C_1^2}}{2} \right\|^2 - \left\| \frac{a_{C_2^1} + a_{C_2^2}}{2} \right\|^2\right)\\
				&= \frac{1}{4} \left(S^2 - 2\sigma^2_{\star, \text{BS}}\right)
			\end{align*}
			$\mathcal{C}_{b, \mathrm{SS}}$ clustering minimizes \(\sigma^2_{\star, \text{SS}}\), thereby maximizing \(\sigma^2_{\star, \text{BS}}\). Thus,
			\begin{align*}
				&\sigma^2_{\star, \text{SS}} = \frac{1}{4} \left( \left[ \left\| \frac{a_{C_1^1} + a_{C_2^1}}{2} \right\|^2 + \left\| \frac{a_{C_1^2} + a_{C_2^2}}{2} \right\|^2 \right] + \left[ \left\| \frac{a_{C_1^1} + a_{C_2^2}}{2} \right\|^2 + \left\| \frac{a_{C_1^2} + a_{C_2^1}}{2} \right\|^2 \right] \right)\\
				&= \frac{1}{4} \left( 2\sigma^2_{\star, \text{BS}}\left((C_1^1, C_2^1), (C_1^2, C_2^2) \right) + 2\sigma^2_{\star, \text{BS}}\left((C_1^1, C_2^2), (C_1^2, C_2^1) \right) \right)\\
				&= \frac{1}{2} \left( \sigma^2_{\star, \text{BS}}\left((C_1^1, C_2^1), (C_1^2, C_2^2) \right) + \sigma^2_{\star, \text{BS}}\left((C_1^1, C_2^2), (C_1^2, C_2^1) \right) \right)\\
				&\leq \sigma^2_{\star, \text{BS}}.
			\end{align*}
		\end{proof}
	\end{example}
	
	However, it is possible that this relationship might hold more generally. Empirical experiments for different configurations, such as \(b=3\), support this possibility. For example, with \(n=9\), \(b=3\), and \(d=10\), Python simulations where gradients \(\nabla f_i\) are sampled from \(\mathcal{N}(0, 1)\) and \(\mathcal{N}(e, 1)\) across \(1000\) independent trials, show that \(\sigma^2_{\star, \text{SS}} \leq \sigma^2_{\star, \text{BS}}\).
	Question of finding theoretical proof for arbitraty $n$ remains open and has yet to be addressed in the existing literature.

	\subsection{Different approaches of federated averaging}\label{sec:fedavg-sppm}
	Proof of Theorem \ref{thm:FedProx-SPPM-AS}:
	\begin{proof}
		\begin{align*}
			\sqn{x_{t} - x_\star} &= \sqn{\sum_{i\in S_{t}}\frac{1}{\left|S_{t}\right|}\prox_{\gamma f_{i}}(x_{t-1}) - \frac{1}{\left|S_{t}\right|}\sum_{i\in S_{t}} x_\star}\\
			&\stackrel{(\Cref{fact:fact1})}{=} \sqn{\sum_{i\in S_{t}}\frac{1}{\left|S_{t}\right|}\left[\prox_{\gamma f_{i}}(x_{t-1}) - \prox_{\gamma f_{i}}(x_\star + \gamma \nabla f_{i}(x_\star))\right]}\\
			&\stackrel{Jensen}{\leq}\sum_{i\in S_{t}}\frac{1}{\left|S_{t}\right|}\sqn{\left[\prox_{\gamma f_{i}}(x_{t-1}) - \prox_{\gamma f_{i}}(x_\star + \gamma \nabla f_{i}(x_\star))\right]}\\
			&\stackrel{(\Cref{fact:fact2})}{\leq}\sum_{i\in S_{t}}\frac{1}{\left|S_{t}\right|}\frac{1}{(1+\gamma \mu_i)^2} \sqn{x_{t-1} - (x_\star + \gamma \nabla f_{i}(x_\star))}
		\end{align*} 
		\begin{align*}
			&\ec[S_t\sim\mathcal{S}]{\sqn{x_{t}-x_\star}|x_{t-1}}\\
			&\leq\ec[S_t\sim\mathcal{S}]{\sum_{i\in S_{t}}\frac{1}{\left|S_{t}\right|}\frac{1}{(1+\gamma \mu_{i})^2} \sqn{\left(x_{t-1} - x_\star\right) - \gamma \nabla f_{i}(x_\star))}|x_{t-1}}\\
			&\stackrel{\text{Young, }\alpha_i>0}{\leq}\ec[S_t\sim\mathcal{S}]{\sum_{i\in S_{t}}\frac{1}{\left|S_{t}\right|}\frac{1}{(1+\gamma \mu_{i})^2} \left(\left(1+\alpha_i\right)\sqn{x_{t-1} - x_\star} + \left(1+\alpha_i^{-1}\right)\sqn{\gamma \nabla f_{i}(x_\star))}\right)|x_{t-1}}\\
			&\stackrel{\alpha_i=\gamma\mu_i}{=}\ec[S_t\sim\mathcal{S}]{\sum_{i\in S_{t}}\frac{1}{\left|S_{t}\right|}\frac{1}{(1+\gamma \mu_{i})^2} \left(\left(1+\gamma\mu_i\right)\sqn{x_{t-1} - x_\star} + \left(1+\frac{1}{\gamma\mu_i}\right)\sqn{\gamma \nabla f_{i}(x_\star))}\right)|x_{t-1}}\\
			&=\ec[S_t\sim\mathcal{S}]{\sum_{i\in S_{t}}\frac{1}{\left|S_{t}\right|}\left(\frac{1}{1+\gamma \mu_{i}} \sqn{x_{t-1} - x_\star} + \frac{\gamma}{(1 + \gamma\mu_i)\mu_i}\sqn{\nabla f_{i}(x_\star))}\right)|x_{t-1}}\\
			&=\ec[S_t\sim\mathcal{S}]{\frac{1}{\left|S_{t}\right|}\sum_{i\in S_{t}}\frac{1}{1+\gamma \mu_{i}}|x_{t-1}}\sqn{x_{t-1} - x_\star} + \ec[S_t\sim\mathcal{S}]{\frac{1}{\left|S_{t}\right|}\sum_{i\in S_{t}}\frac{\gamma}{(1 + \gamma\mu_i)\mu_i}\sqn{\nabla f_{i}(x_\star))}|x_{t-1}}\\
		\end{align*}
		By applying tower property one can get the following: 
		\begin{align*}
			&\ec[S_t\sim\mathcal{S}]{\sqn{x_{t}-x_\star}}\\
			&=\ec[S_t\sim\mathcal{S}]{\frac{1}{\left|S_{t}\right|}\sum_{i\in S_{t}}\frac{1}{1+\gamma \mu_{i}}}\sqn{x_{t-1} - x_\star} + \ec[S_t\sim\mathcal{S}]{\frac{1}{\left|S_{t}\right|}\sum_{i\in S_{t}}\frac{\gamma}{(1 + \gamma\mu_i)\mu_i}\sqn{\nabla f_{i}(x_\star))}}\\
			&=A_\mathcal{S}\sqn{x_{t-1} - x_\star} + B_\mathcal{S}.
		\end{align*}
		where  $A_\mathcal{S}\eqdef\ec[S_t\sim\mathcal{S}]{\frac{1}{\left|S_{t}\right|}\sum_{i\in S_{t}}\frac{1}{1+\gamma \mu_{i}}}$ and $B_\mathcal{S}\eqdef\ec[S_t\sim\mathcal{S}]{\frac{1}{\left|S_{t}\right|}\sum_{i\in S_{t}}\frac{\gamma}{(1 + \gamma\mu_i)\mu_i}\sqn{\nabla f_{i}(x_\star))}}$.
		By directly applying \Cref{fact:fact3}:
		\[
		\ec[S_t\sim\mathcal{S}]{\sqn{x_{t}-x_\star}}\leq A_\mathcal{S}^t\sqn{x_0-x_\star}+\frac{B_\mathcal{S}}{1-A_\mathcal{S}}.
		\]
	\end{proof}

		\begin{restatable}[Inexact formulation of \algname{SPPM-AS}]{lemma}{lemma2}\label{lemma:sppm-as-inexact}
			Let $b > 0\in \mathbb{R}$ and define $\widetilde{\prox}_{\gamma f}(x)$ such that $\forall x\sqn{\widetilde{\prox}_{\gamma f}(x) - \prox_{\gamma f}(x)}\leq b$. Let \Cref{asm:differential} and \Cref{asm:strongly_convex} hold. Let $x_0\in\mathbb{R}^d$ be an arbitrary starting point. Then for any $t\geq0$ and any $\gamma>0$, $s>0$, the iterates of \algname{SPPM-AS} satisfy
			$$
			\ec{\sqn{x_{t} - x_\star}} \leq \left(\frac{1+s}{(1+\gamma\mu)^2}\right)^t\sqn{x_0-x_\star}+\frac{\left(1+s\right)\left(\gamma^2 \sigma^2_{\star} + s^{-1}b(1+\gamma\mu)^2\right)}{\gamma^2\mu^2+2\gamma\mu-s}.
			$$
		\end{restatable}
	\begin{proof}[Proof of Lemma \ref{lemma:sppm-as-inexact}]
		We provide more general version of \algname{SPPM} proof
		\begin{align*}
			\sqn{x_{t+1} - x_\star} &= \sqn{\widetilde{\prox}_{\gamma f_{\xi_t}(x_t)} - \prox_{\gamma f_{\xi_t}}(x_t)  + \prox_{\gamma f_{\xi_t}}(x_t) - x_\star}\\
			&\stackrel{Young, s>0}{\leq}(1+s^{-1})\sqn{\widetilde{\prox}_{\gamma f_{\xi_t}}(x_t) - \prox_{\gamma f_{\xi_t}}}(x_t) + (1+s)\sqn{\prox_{\gamma f_{\xi_t}}(x_t) - x_\star}\\
			&\leq (1 + s^{-1}) b + (1+s)\sqn{\prox_{\gamma f_{\xi_t}}(x_t) - x_\star}.
		\end{align*} 
		Then proof follows same path as proof Theorem \ref{thm:sppm_as} and we get
		\begin{align*}
			&\ec{\sqn{x_{t+1} - x_\star}} \leq(1+s^{-1})b + (1+s) \frac{1}{(1+\gamma\mu)^2}\left(\sqn{x_t - x_\star} + \gamma^2 \sigma^2_{\star} \right)\\
			&=\frac{1+s}{(1+\gamma\mu)^2}\left(\sqn{x_t - x_\star}  + \left[\gamma^2 \sigma^2_{\star} + s^{-1}b(1+\gamma\mu)^2\right]\right).
		\end{align*}
azc		
		It only remains to solve the above recursion. Luckily, that is exactly what \Cref{fact:fact3} does. In particular, we use it with $s_t = \ec{\sqn{x_t - x_\star}}, A = \frac{1+s}{(1+\gamma\mu)^2}$ and $B = \frac{\left(1+s\right)\left(\gamma^2 \sigma^2_{\star} + s^{-1}b(1+\gamma\mu)^2\right)}{(1+\gamma\mu)^2}$ to get 
		\begin{align*}
			\ec{\sqn{x_t - x_\star}}
			&\leq A^t\sqn{x_0 - x_\star} + B\frac{1}{1-A}\\
			&\leq A^t\sqn{x_0 - x_\star} + B\frac{(1+\gamma\mu)^2}{(1+\gamma\mu)^2-1-s}\\
			&\leq A^t\sqn{x_0-x_\star}+\frac{\left(1+s\right)\left(\gamma^2 \sigma^2_{\star} + s^{-1}b(1+\gamma\mu)^2\right)}{(1+\gamma\mu)^2-1-s}\\
			&=\left(\frac{1+s}{(1+\gamma\mu)^2}\right)^t\sqn{x_0-x_\star}+\frac{\left(1+s\right)\left(\gamma^2 \sigma^2_{\star} + s^{-1}b(1+\gamma\mu)^2\right)}{\gamma^2\mu^2+2\gamma\mu-s}.
		\end{align*}
	\end{proof}

%% file: Appendix_C3_SymWanda.tex
\chapter{Appendix to Chapter \ref{chapter_symwanda}}
\label{chapter_appendix_symwanda}
\thispagestyle{empty}

\section{Missing Proofs}
\subsection{Proof of \Cref{lemma:lm1}}
By using the definition of $g(\widetilde{\vW})$ in \Cref{obj1}, we have

$$
\begin{aligned}
g(\widetilde{\vW}) & =\sqrt{\sum_{k=1}^c\left\|\vX\left(\widetilde{\vW}_{: k}-\vW_{: k}\right)\right\|_2^2}+\sqrt{\sum_{j=1}^b\left\|\left(\widetilde{\vW}_{j:}-\vW_{j:}\right) \vY\right\|_2^2} \\
& =\sqrt{\sum_{k=1}^c \sum_{i=1}^a\left(\vX_{i:}\left(\widetilde{\vW}_{: k}-\vW_{: k}\right)\right)^2}+\sqrt{\sum_{j=1}^b \sum_{l=1}^d\left(\left(\widetilde{\vW}_{j:}-\vW_{j:}\right) \vY_{: l}\right)^2} \\
& =\sqrt{\sum_{k=1}^c \sum_{i=1}^a\left(\sum_{j=1}^b \vX_{i j}\left(\widetilde{\vW}_{j k}-\vW_{j k}\right)\right)^2}+\sqrt{\sum_{j=1}^b \sum_{l=1}^d\left(\sum_{k=1}^c\left(\widetilde{\vW}_{j k}-\vW_{j k}\right) \vY_{k l}\right)^2}
\end{aligned}
$$

Now say we want to prune away just a single weight $\vW_{j k}$. That is, we want to set $\widetilde{\vW}_{j k}=0$ and $\widetilde{\vW}_{j^{\prime} k^{\prime}}=\vW_{j^{\prime} k^{\prime}}$ for all $\left(j^{\prime}, k^{\prime}\right) \neq(j, k)$. For such a weight matrix $\widetilde{\vW}_{j k}$ the expression for $f(\widetilde{\vW})$ simplifies to

$$
\begin{aligned}
& g(\widetilde{\vW})=\sum_{i=1}^a\left(\sum_{j^{\prime}=1}^b \vX_{i j^{\prime}}\left(\widetilde{\vW}_{j^{\prime} k}-\vW_{j^{\prime} k}\right)\right)^2+\sum_{l=1}^d\left(\sum_{k^{\prime}=1}^c\left(\widetilde{\vW}_{j k^{\prime}}-\vW_{j k^{\prime}}\right) \vY_{k^{\prime} l}\right)^2 \\
& =\sqrt{\sum_{i=1}^a\left(\vX_{i j}\left(\widetilde{\vW}_{j k}-\vW_{j k}\right)+\sum_{j^{\prime} \neq j} \vX_{i j^{\prime}}\left(\widetilde{\vW}_{j^{\prime} k}-\vW_{j^{\prime} k}\right)\right)^2}\\
&\qquad +\sqrt{\sum_{l=1}^d\left(\left(\widetilde{\vW}_{j k}-\vW_{j k}\right) \vY_{k l}+\sum_{k^{\prime} \neq k}\left(\widetilde{\vW}_{j k}-\vW_{j k}\right) \vY_{k l}\right)^2} \\
& =\sqrt{\sum_{i=1}^a(\vX_{i j}\left(0-\vW_{j k}\right)+\sum_{j^{\prime} \neq j} \vX_{i j^{\prime}} \underbrace{\left(\vW_{j^{\prime} k}-\vW_{j^{\prime} k}\right)}_{=0})^2}\\
&\qquad +\sqrt{\sum_{l=1}^d(\left(0-\vW_{j k}\right) \vY_{k l}+\sum_{k^{\prime} \neq k} \underbrace{\left(\widetilde{\vW}_{j k}-\vW_{j k}\right)}_{=0} \vY_{k l})^2} \\
& =\sqrt{\sum_{i=1}^a\left(-\vX_{i j} \vW_{j k}\right)^2}+\sqrt{\sum_{l=1}^d\left(-\vW_{j k} \vY_{k l}\right)^2} \\
& =\sqrt{\sum_{i=1}^a \vX_{i j}^2 \vW_{j k}^2}+\sqrt{\sum_{l=1}^d \vW_{j k}^2 \vY_{k l}^2} \\
& =\left|\vW_{j k}\right|\left(\left\|\vX_{: j}\right\|_2+\left\|\vY_{k:}\right\|_2\right) \eqdef \vS_{j k}.
\end{aligned}
$$

\subsection{Proof of Theorem \ref{thm:main2}}
\begin{itemize}
    \item Assume it is possible to choose matrices $\vX \in \mbR^{a\times b}$ and $\vY \in \mbR^{c\times d}$ such that the identity 
            \begin{equation}\label{eqn1}
                \norm{\vX_{:k}}_2 + \norm{\vY_{j:}}_2 = \alpha_{jk} \eqdef \frac{1}{\norm{\vW_{j:}}_1}+ \frac{1}{\norm{\vW_{:k}}_1}
            \end{equation} 
            holds for all $j, k$. \emph{This is always possible!} 
            
            Indeed, if we choose $a=b$, and let the $j$-th row of $\vX$ be of the form $\vX_{:j} \eqdef t_{j} (1; \cdots; 1)\in \mbR^{b\times 1}$, where $t_j = \frac{1}{\sqrt{b} \norm{\vW_{j:}}_1}$, then $\norm{\vX_{j:}}_2 = t_j \sqrt{b} = \frac{1}{\norm{\vW_{j:}}_1}$. 
            
            Similarly, if we choose $d=c$, and let the $k$-th column of $\vY$ be of the form $\vY_{:k} \eqdef s_k(1,\cdots,1)\in \mbR^{1\times c}$, where $s_k = \frac{1}{\sqrt{c}\norm{\vW_{:k}}_1}$, then $\norm{\vY_{:k}}_2 = s_k\sqrt{c}= \frac{1}{\norm{\vW_{:k}}_1}$.
            
            So, \Cref{eqn1} holds. In this case, our score matrix \Cref{eqn0} reduces to the plug-and-play method \algname{RIA} \citep{RIA}.
\end{itemize}

\begin{itemize}
    \item Another (even simpler) possiblity for constructing matrices $\vX, \vY$ such that \Cref{eqn1} holds is as follows. Let $a=b$, and let $\vX = \Diag(\norm{\vW_{1:}}^{-1}_1, \cdots, \norm{\vW_{b:}}^{-1}_1)$. 
            Clearly, for all $j=1, \cdots, b$ we have $\norm{\vX_{j:}}_2 = \frac{1}{\norm{\vW_{j:}}_1}$. 
            
            Similarly, let $d=c$, and let $\vY = \Diag(\norm{\vW_{:1}}^{-1}_1, \cdots, \norm{\vW_{:c}}^{-1}_1)$. Clearly, for all $k=1, \cdots, c$, we have $\norm{\vY_{:k}}_2 = \frac{1}{\norm{\vW_{:k}}}_1$. 
            
            Therefore, $\norm{\vX_{:j}}_2 + \norm{\vY_{k:}}_2 = \frac{1}{\norm{\vW_{j:}}_1} + \frac{1}{\norm{\vW_{:k}}_1}$ for all $j, k$. So again, our score matrix (\ref{eqn0}) reduces to the plug-and-play method in \cite{RIA}.
\end{itemize}

\subsection{Proof of \Cref{lem:general2}}
Recall that in \Cref{sec:general_solution} $\vD_{\vX} \in \mbR^{b\times b}$ and $\vD_{\vY} \in \mbR^{c\times c}$ are diagonal matrices with entries defined as $\left(\mathbf{D}_{\mathbf{X}}\right)_{ii} = x_i = \left\|\mathbf{W}_{i:}\right\|_1^{-1}$ and $\left(\mathbf{D}_{\mathbf{Y}}\right)_{ii} = y_i = \left\|\mathbf{W}_{:i}\right\|_1^{-1}$ respectively, and $\mathbf{A}\in \mathbb{R}^{a\times b}$ and $\mathbf{B}\in \mbR^{c\times d}$ are arbitrary matrices. 
We first compute $\mathbf{A} \mathbf{D}_{\mathbf{X}}$. This product scales each column of $\mathbf{A}$ by the corresponding $x_i$. Specifically, for the $j$-th column, this operation is expressed as:
\[
\left(\mathbf{A} \mathbf{D}_{\mathbf{X}}\right)_{:j} = x_j \mathbf{A}_{:j}.
\]
The $\ell_2$-norm of this column is then given by:
\[
\left\|\left(\mathbf{A} \mathbf{D}_{\mathbf{X}}\right)_{:j}\right\|_2 = x_j \left\|\mathbf{A}_{:j}\right\|_2 = \frac{\left\|\mathbf{A}_{:j}\right\|_2}{\left\|\mathbf{W}_{j:}\right\|_1}.
\]

Next, we compute $\mathbf{D}_{\mathbf{Y}} \mathbf{B}$. In this computation, each row of $\mathbf{B}$ is scaled by the corresponding $y_i$. For the $k$-th row, the scaling is represented as:
\[
\left(\mathbf{D}_{\mathbf{Y}} \mathbf{B}\right)_{k:} = y_k \mathbf{B}_{k:}.
\]
The $\ell_2$-norm of this row is:
\[
\left\|\left(\mathbf{D}_{\mathbf{Y}} \mathbf{B}\right)_{k:}\right\|_2 = y_k \left\|\mathbf{B}_{k:}\right\|_2 = \frac{\left\|\mathbf{B}_{k:}\right\|_2}{\left\|\mathbf{W}_{:k}\right\|_1}.
\]

Finally, we consider the sum of these norms:
\[
\left\|\left(\mathbf{A} \mathbf{D}_{\mathbf{X}}\right)_{:j}\right\|_2 + \left\|\left(\mathbf{D}_{\mathbf{Y}} \mathbf{B}\right)_{k:}\right\|_2 = \frac{\left\|\mathbf{A}_{:j}\right\|_2}{\left\|\mathbf{W}_{j:}\right\|_1} + \frac{\left\|\mathbf{B}_{k:}\right\|_2}{\left\|\mathbf{W}_{:k}\right\|_1}.
\]

The first term involves scaling the $j$-th column of $\mathbf{A}$ by $x_j$, with the resulting norm being the original column norm divided by the $\ell_1$-norm of the corresponding weights in $\mathbf{W}$. Similarly, the second term scales the $k$-th row of $\mathbf{B}$ by $y_k$, with the resulting norm also being the original row norm divided by the $\ell_1$-norm of the corresponding weights in $\mathbf{W}$.

\subsection{Proof of \Cref{lem:generalized_p_norm}}
We aim to construct $\mathbf{X}_{: j}$ to be proportional to $\mathbf{W}_{j:}^{\top}$. A natural choice is to set
$$
\mathbf{X}_{: j} = c \cdot \mathbf{W}_{j:}^{\top},
$$
where $c$ is a scalar to be determined. A similar condition applies when considering $\mathbf{Y}_{k:}$. The central task is to compute the corresponding scaling factor $c$ for both $\mathbf{X}$ and $\mathbf{Y}$.

To determine $c$, we choose it such that
$$
\left\|\mathbf{X}_{: j}\right\|_2 = \left\|c \cdot \mathbf{W}_{j:}^{\top}\right\|_2 = \left\|\mathbf{W}_{j:}\right\|_p^{-1}.
$$

We now compute the $\ell_2$-norm of $\mathbf{X}_{: j}$:
$$
\left\|c \cdot \mathbf{W}_{j:}^{\top}\right\|_2 = |c| \cdot \left\|\mathbf{W}_{j:}^{\top}\right\|_2 = |c| \cdot \left\|\mathbf{W}_{j:}\right\|_2.
$$

Setting this equal to $\left\|\mathbf{W}_{j:}\right\|_p^{-1}$, we have:
$$
|c| \cdot \left\|\mathbf{W}_{j:}\right\|_2 = \left\|\mathbf{W}_{j:}\right\|_p^{-1}.
$$

Solving for $c$, we obtain:
$$
c = \frac{1}{\left\|\mathbf{W}_{j:}\right\|_p} \cdot \frac{1}{\left\|\mathbf{W}_{j:}\right\|_2}.
$$

Using this value of $c$, we define $\mathbf{X}_{: j}$ as:
$$
\mathbf{X}_{: j} = \frac{1}{\left\|\mathbf{W}_{j:}\right\|_p} \cdot \frac{1}{\left\|\mathbf{W}_{j:}\right\|_2} \cdot \mathbf{W}_{j:}^{\top}.
$$

This construction ensures that
$$
\left\|\mathbf{X}_{: j}\right\|_2 = \left\|\mathbf{W}_{j:}\right\|_p^{-1}.
$$

Similarly, for $\mathbf{Y}$, we have:
$$
\mathbf{Y}_{k:} = \frac{1}{\left\|\mathbf{W}_{:k}\right\|_p} \cdot \frac{1}{\left\|\mathbf{W}_{:k}\right\|_2} \cdot \mathbf{W}_{:k}^{\top},
$$
which satisfies \Cref{eqn:pnorm2}.

By combining these results, we conclude the proof of \Cref{lem:generalized_p_norm}.

\subsection{Proof of \Cref{lem:random_unit_vector_scaling}}

Let $\mathbf{u}$ be any unit vector in $\ell_2$-norm, i.e., $\norm{\mathbf{u}}_2 = 1$. Construct $\vX_{:j} = \norm{\vW_{j:}}_p^{-1} \mathbf{u}$. Then by using the definition of the $\ell_2$-norm, we have 

$$
\left\|\mathbf{X}_{: j}\right\|_2=\| \| \mathbf{W}_{j:}\left\|_p^{-1} \mathbf{u}\right\|_2=\left|\left\|\mathbf{W}_{j:}\right\|_p^{-1}\right|\|\mathbf{u}\|_2=\left\|\mathbf{W}_{j:}\right\|_p^{-1} \cdot 1=\left\|\mathbf{W}_{j:}\right\|_p^{-1}.
$$

Hence, we obtain $\left\|\mathbf{X}_{: j}\right\|_2=\left\|\mathbf{W}_{j:}\right\|_p^{-1}$, which is exactly as desired. 

Similarly, let $\mathbf{v}$ be any unit vector in $\ell_2$-norm, we have $|\vW_{jk}|\cdot \|\mathbf{W}_{:k}\|^{-1}_p$.

Put them together, we prove \Cref{lem:random_unit_vector_scaling}.

\subsection{Proof of \Cref{lem:stochria}}
Given that $\mathbf{X}_{:j}$ and $\mathbf{Y}_{k:}$ are vectors to be constructed, $\mathbf{W}$ is a matrix, and $S_j$ and $S_k$ are randomly sampled index sets from the $j$-th row and $k$-th column of $\mathbf{W}$, respectively, each with cardinality $\tau$, our task is to construct $\mathbf{X}_{:j}$ and $\mathbf{Y}_{k:}$ with specific norms. Specifically, the goal is to construct $\mathbf{X}_{:j}$ and $\mathbf{Y}_{k:}$ such that:
$$
\left\| \mathbf{X}_{:j} \right\|_2 + \left\| \mathbf{Y}_{k:} \right\|_2 = \frac{1}{\left\| \mathbf{W}_{j:S_j} \right\|_1} + \frac{1}{\left\| \mathbf{W}_{S_k:k} \right\|_1},
$$
where $\mathbf{W}_{j:S_j}$ denotes the entries of the $j$-th row of $\mathbf{W}$ at indices in $S_j$, and $\mathbf{W}_{S_k:k}$ denotes the entries of the $k$-th column of $\mathbf{W}$ at indices in $S_k$.


We first define the support vector $\mathbf{e}_{S_j}$ of appropriate size (equal to the number of rows in $\mathbf{X}$) as:
$$
(\mathbf{e}_{S_j})_i = \begin{cases}
    \frac{1}{\sqrt{\tau}}, & \text{if } i \in S_j, \\
    0, & \text{otherwise}.
\end{cases}
$$

The vector $\mathbf{e}_{S_j}$ has non-zero entries only at indices in $S_j$, each equal to $\frac{1}{\sqrt{\tau}}$, ensuring that the $\ell_2$-norm of $\mathbf{e}_{S_j}$ is 1:
$$
\left\| \mathbf{e}_{S_j} \right\|_2 = \sqrt{ \sum_{i \in S_j} \left( \frac{1}{\sqrt{\tau}} \right)^2 } = \sqrt{ \tau \cdot \left( \frac{1}{\sqrt{\tau}} \right)^2 } = 1.
$$

To construct $\mathbf{X}_{:j}$, we set:
$$
\mathbf{X}_{:j} = \frac{1}{\left\| \mathbf{W}_{j:S_j} \right\|_1} \cdot \mathbf{e}_{S_j}.
$$

A basic verification shows that the $\ell_2$-norm of $\mathbf{X}_{:j}$ is:
$$
\left\| \mathbf{X}_{:j} \right\|_2 = \frac{1}{\left\| \mathbf{W}_{j:S_j} \right\|_1} \cdot \left\| \mathbf{e}_{S_j} \right\|_2 = \frac{1}{\left\| \mathbf{W}_{j:S_j} \right\|_1} \cdot 1 = \frac{1}{\left\| \mathbf{W}_{j:S_j} \right\|_1}.
$$


Similarly, we define the support vector $\mathbf{e}_{S_k}$ of appropriate size (equal to the number of columns in $\mathbf{Y}$) as:
$$
(\mathbf{e}_{S_k})_i = \begin{cases}
    \frac{1}{\sqrt{\tau}}, & \text{if } i \in S_k, \\
    0, & \text{otherwise}.
\end{cases}
$$

To construct $\mathbf{Y}_{k:}$, we set:
$$
\mathbf{Y}_{k:} = \frac{1}{\left\| \mathbf{W}_{S_k:k} \right\|_1} \cdot \mathbf{e}_{S_k}^\top.
$$


Adding the norms:
$$
\left\| \mathbf{X}_{:j} \right\|_2 + \left\| \mathbf{Y}_{k:} \right\|_2 = \frac{1}{\left\| \mathbf{W}_{j:S_j} \right\|_1} + \frac{1}{\left\| \mathbf{W}_{S_k:k} \right\|_1},
$$
which matches the desired expression.

\textbf{Alternative construction using $\ell_1$ and $\ell_2$ norms.}

By definition:
$$
\left\| \mathbf{W}_{j:S_j} \right\|_1 = \sum_{i \in S_j} |w_{j i}|, \quad \left\| \mathbf{W}_{j:S_j} \right\|_2 = \sqrt{ \sum_{i \in S_j} w_{j i}^2 }.
$$

We can construct $\mathbf{X}_{:j}$ as:
$$
\mathbf{X}_{:j} = \frac{1}{\left\| \mathbf{W}_{j:S_j} \right\|_1} \cdot \frac{1}{\left\| \mathbf{W}_{j:S_j} \right\|_2} \cdot \mathbf{W}_{j:S_j}^\top,
$$
where $\mathbf{W}_{j:S_j}^\top$ is a vector with entries:
$$
(\mathbf{W}_{j:S_j}^\top)_i = \begin{cases}
    w_{j i}, & \text{if } i \in S_j, \\
    0, & \text{otherwise}.
\end{cases}
$$

Similarly, we can construct $\mathbf{Y}_{k:}$ as:
$$
\mathbf{Y}_{k:} = \frac{1}{\left\| \mathbf{W}_{S_k:k} \right\|_1} \cdot \frac{1}{\left\| \mathbf{W}_{S_k:k} \right\|_2} \cdot \mathbf{W}_{S_k:k}^\top,
$$
where $\mathbf{W}_{S_k:k}^\top$ is a vector with entries:
$$
(\mathbf{W}_{S_k:k}^\top)_i = \begin{cases}
    w_{i k}, & \text{if } i \in S_k, \\
    0, & \text{otherwise}.
\end{cases}
$$

Putting everything together, we prove \Cref{lem:stochria}.




\section{Symmetric Wanda Variant with Squared Frobenius Norms}\label{sec:squared_frobenius}
 
Choose $\varepsilon \in(0,1]$. Given $\vX \in \mathbb{R}^{a \times b}, \vW \in \mathbb{R}^{b \times c}$ and $\vY \in \mathbb{R}^{c \times d}$, define

$$
g^\prime(\widetilde{\vW}):=\|\vX(\widetilde{\vW}-\vW)\|_F^2+\|(\widetilde{\vW}-\vW) \vY\|_F^2,
$$

and consider solving the problem

$$
\begin{aligned}
\operatorname{mininimize} & g^\prime(\widetilde{\vW}) \quad \text { subject to } & \operatorname{Mem}(\widetilde{\vW}) \leq \varepsilon \operatorname{Mem}(\vW), \widetilde{\vW} \in \mathbb{R}^{b \times c}.
\end{aligned}
$$

Note that

$$
\begin{aligned}
g^\prime(\widetilde{\vW}) & =\sum_{k=1}^c\left\|\vX\left(\widetilde{\vW}_{: k}-\vW_{: k}\right)\right\|_2^2+\sum_{j=1}^b\left\|\left(\widetilde{\vW}_{j:}-\vW_{j:}\right) \vY\right\|_2^2 \\
& =\sum_{k=1}^c \sum_{i=1}^a\left(\vX_{i:}\left(\widetilde{\vW}_{: k}-\vW_{: k}\right)\right)^2+\sum_{j=1}^b \sum_{l=1}^d\left(\left(\widetilde{\vW}_{j:}-\vW_{j:}\right) Y_{: l}\right)^2 \\
& =\sum_{k=1}^c \sum_{i=1}^a\left(\sum_{j=1}^b \vX_{i j}\left(\widetilde{\vW}_{j k}-\vW_{j k}\right)\right)^2+\sum_{j=1}^b \sum_{l=1}^d\left(\sum_{k=1}^c\left(\widetilde{\vW}_{j k}-\vW_{j k}\right) \vY_{k l}\right)^2
\end{aligned}
$$

Now say we want to prune away just a single weight $\vW_{j k}$. That is, we want to set $\widetilde{\vW}_{j k}=0$ and $\widetilde{\vW}_{j^{\prime} k^{\prime}}=\vW_{j^{\prime} k^{\prime}}$ for all $\left(j^{\prime}, k^{\prime}\right) \neq(j, k)$. For such a weight matrix $\widetilde{\vW}_{j k}$ the expression for $g^\prime(\widetilde{\vW})$ simplifies to

$$
\begin{aligned}
g^\prime(\widetilde{\vW}) & =\sum_{i=1}^a\left(\sum_{j^{\prime}=1}^b \vX_{i j^{\prime}}\left(\widetilde{\vW}_{j^{\prime} k}-\vW_{j^{\prime} k}\right)\right)^2+\sum_{l=1}^d\left(\sum_{k^{\prime}=1}^c\left(\widetilde{\vW}_{j k^{\prime}}-\vW_{j k^{\prime}}\right) \vY_{k^{\prime} l}\right)^2 \\
& =\sum_{i=1}^a\left(\vX_{i j}\left(\widetilde{\vW}_{j k}-\vW_{j k}\right)+\sum_{j^{\prime} \neq j} \vX_{i j^{\prime}}\left(\widetilde{\vW}_{j^{\prime} k}-\vW_{j^{\prime} k}\right)\right)^2\\
& \qquad +\sum_{l=1}^d\left(\left(\widetilde{\vW}_{j k}-\vW_{j k}\right) \vY_{k l}+\sum_{k^{\prime} \neq k}\left(\widetilde{\vW}_{j k}-\vW_{j k}\right) \vY_{k l}\right)^2 \\
& =\sum_{i=1}^a(\vX_{i j}\left(0-\vW_{j k}\right)+\sum_{j^{\prime} \neq j} \vX_{i j^{\prime}} \underbrace{\left(\vW_{j^{\prime} k}-\vW_{j^{\prime} k}\right)}_{=0})^2\\
&\qquad +\sum_{l=1}^d(\left(0-\vW_{j k}\right) \vY_{k l}+\sum_{k^{\prime} \neq k} \underbrace{\left(\widetilde{\vW}_{j k}-\vW_{j k}\right)}_{=0} \vY_{k l})^2 \\
& =\sum_{i=1}^a\left(-\vX_{i j} \vW_{j k}\right)^2+\sum_{l=1}^d\left(-\vW_{j k} \vY_{k l}\right)^2 \\
& =\sum_{i=1}^a \vX_{i j}^2 \vW_{j k}^2+\sum_{l=1}^d \vW_{j k}^2 \vY_{k l}^2 \\
& =\vW_{j k}^2\left(\left\|\vX_{: j}\right\|_2^2+\left\|Y_{k:}\right\|_2^2\right) \eqdef \vS_{j k}^2.
\end{aligned}
$$

Our proposal is to choose entry $(j, k)$ which the smallest score $\vS_{j k}$. Special cases:

1. If we choose $\vX=\mathbf{0} \in \mathbb{R}^{a \times b}$, then our pruning method reduces to "output" \algname{Wanda}:

$$
\vS_{j k}:=\left|\vW_{j k}\right|\left\|\vY_{k:}\right\|_2
$$

2. If we choose $\vY=\mathbf{0} \in \mathbb{R}^{c \times d}$, then our pruning method reduces to "input" \algname{Wanda}:

$$
\vS_{j k}:=\left|\vW_{j k}\right|\left\|\vX_{: j}\right\|_2.
$$

3. If we choose $\vX=\vW^{\top} \in \mathbb{R}^{c \times b}(a=c)$ and $\vY=\vW^{\top} \in \mathbb{R}^{c \times b}(d=b)$, then our score matrix becomes

$$
\vS_{j k} \stackrel{(27)}{=}\left|\vW_{j k}\right| \sqrt{\left\|\vX_{: j}\right\|_2^2+\left\|\vY_{k:}\right\|_2^2}=\left|\vW_{j k}\right| \sqrt{\left\|\vW_{j:}\right\|_2^2+\left\|\vW_{: k}\right\|_2^2}
$$

Letting $\vG_{j k}^2:=\frac{1}{b+c}\left(\left\|\vW_{j:}\right\|_2^2+\left\|\vW_{: k}\right\|_2^2\right)$, note that

$$
\begin{aligned}
\|\vG\|_F^2 & =\sum_{j=1}^b \sum_{k=1}^c \vG_{j k}^2 \\
& =\frac{1}{b+c} \sum_{j=1}^b \sum_{k=1}^c\left(\left\|\vW_{j:}\right\|_2^2+\left\|\vW_{: k}\right\|_2^2\right) \\
& =\frac{1}{b+c}\left(\sum_{j=1}^b \sum_{k=1}^c\left\|\vW_{j:}\right\|_2^2+\sum_{k=1}^c \sum_{j=1}^b\left\|\vW_{: k}\right\|_2^2\right) \\
& =\frac{1}{b+c}\left(c \sum_{j=1}^b\left\|\vW_{j:}\right\|_2^2+b \sum_{k=1}^c\left\|\vW_{: k}\right\|_2^2\right) \\
& =\frac{1}{b+c}\left(c\|\vW\|_F^2+b\|\vW\|_F^2\right) \\
& =\|\vW\|_F^2
\end{aligned}
$$

Clearly,

$$
\frac{\vS_{j k}^2}{(b+c)\|\vW\|_F^2} \stackrel{}{=} \frac{\vW_{j k}^2 \vG_{j k}^2}{\|\vW\|_F^2}
$$

4. Assume it is possible to choose matrices $\vX \in \mathbb{R}^{a \times b}$ and $\vY \in \mathbb{R}^{c \times d}$ such that the identity

$$
\sqrt{\left\|\vX_{j:}\right\|_2^2+\left\|\vY_{: k}\right\|_2^2}=\alpha_{j k}:=\frac{1}{\left\|\vW_{j:}\right\|_1}+\frac{1}{\left\|\vW_{: k}\right\|_1}
$$

holds for all $j, k$ (note that this is not always possible!). In this case, our score matrix reduces to the plug-and-play method of \cite{RIA}.

\section{Additional Experiments}
\subsection{Implementation Details}
Our selected baselines are implemented using the source code from \algname{Wanda} and \algname{RIA}. The default settings remain unchanged to ensure consistency. Notably, we explicitly set the sequence length to 2048 instead of using the maximum possible length to enable a fair comparison, following the strategy outlined in \algname{RIA}.

The training-free fine-tuning component is based on \algname{DSnoT}. We configure the maximum cycle count to 50 and set the update threshold to 0.1. The default power of variance for regrowing and pruning is set to 1. Additionally, we incorporate the regularized relative design, resulting in our modified approach, \algname{DSnoT}.

The seed for sampling the calibration data is set to 0. For N:M structural pruning, to enable an intuitive comparison, we use the standard approach without employing channel reallocation or linear sum assignment, as used in \algname{RIA}.

\subsection{Optimal $\ell_p$ Norm}\label{sec:optimal_p}
In this study, we further explore the influence of the $\ell_p$ norm, considering standard norms where $p \in [1, 2, 3, 4]$, as well as the $0$-norm and $\infty$-norm. The results are presented in \Cref{tab:p_norms}. We observed that higher $p$ values degrade performance, as reflected by the perplexity scores, with $p=1$ yielding the best results. This may be due to the fact that in pruning, significantly magnifying the differences between weights is not beneficial. Additionally, we found that both the $0$-norm and $\infty$-norm do not yield promising results, as they capture only partial, and often highly biased, information about the weights.

\begin{table}[!tb]
    \centering
    \caption{Perplexity scores on Wikitext-2 for \algname{p-norm}. The sparsity ratio is $50\%$, and all results correspond to $\alpha=1$. }
    \label{tab:p_norms}
    \begin{tabular}{l|ccccccc}
    \toprule
    p & LlaMA2-7b & LlaMA2-13b & LlaMA3-8b & OPT-1.3b\\ \midrule
    1 & \textbf{6.88} & \textbf{5.95} & \textbf{9.44} & \textbf{18.95}\\
    2 & 6.90 & 5.96 & 9.48 & 19.02\\
    3 & 6.95 & 6.01 & 9.57 & 19.66\\
    4 & 7.12 & 6.08 & 9.92 & 20.77\\ \midrule
    0 & 7.78 & 6.28 & 10.81 & 22.17\\
    $\infty$ & 8.60 & 6.80 & 11.28 & 24.92\\ \bottomrule
    \end{tabular}
\end{table}

\subsection{$\ell_p$ Norm Re-weighting}\label{sec:norm_p_reweighting}
In this section, we explore different $\ell_p$ norm re-weighting strategies. Our default re-weighting approach is defined in \Cref{eqn:pnorm2} and is referred to as $\mathrm{S1}$. Additionally, we investigate alternative strategies, denoted as $\mathrm{S2}$, $\mathrm{S3}$, and $\mathrm{S4}$, as specified below:

\begin{equation}
\begin{aligned}
    \mathrm{S2} \eqdef \mathbf{S}_{jk} &= |\mathbf{W}_{jk}| / (\|\mathbf{W}_{j:}\|_p + \|\mathbf{W}_{:k}\|_p),\\
    \mathrm{S3} \eqdef \mathbf{S}_{jk} &= |\mathbf{W}_{jk}| \cdot (\|\mathbf{W}_{j:}\|_p + \|\mathbf{W}_{:k}\|_p),\\
    \mathrm{S4} \eqdef \mathbf{S}_{jk} &= |\mathbf{W}_{jk}| / (\|\mathbf{W}_{j:}\|^{-1}_p + \|\mathbf{W}_{:k}\|^{-1}_p).\\ \notag
\end{aligned}
\end{equation}

The comparative results for these strategies are presented in \Cref{tab:p_norm_reweighting}. As shown, our default strategy ($\mathrm{S1}$) achieves the best performance, while the alternative designs fail to deliver improvements.

\begin{table}[!tb]
    \centering
    \caption{Perplexity scores on Wikitext-2 for \algname{$\ell_p$-norm} re-weighting with different strategies. The sparsity ratio is $50\%$, and all results are computed with $\alpha=0.5$ and $p=1$.}
    \label{tab:p_norm_reweighting}
    \begin{tabular}{l|cccc}
    \toprule
    Strategy & LLaMA2-7b & LLaMA2-13b & LLaMA3-8b & OPT-1.3b\\ \midrule
    $\mathrm{S1}$ (default) & 6.81 & 5.83 & 9.34 & 18.08 \\
    $\mathrm{S2}$ & 6.99 & 5.91 & 9.58 & 19.01 \\
    $\mathrm{S3}$ & 9.32 & 6.87 & 17.31 & 31.66 \\
    $\mathrm{S4}$ & 14.51 & 20.78 & 30.47 & 53.17 \\ \bottomrule
    \end{tabular}
\end{table}

We hypothesize that the performance differences arise due to the relative magnitudes of the terms $\|\mathbf{W}_{j:}\|_p + \|\mathbf{W}_{:k}\|_p$ and $\|\mathbf{W}_{j:}\|^{-1}_p + \|\mathbf{W}_{:k}\|^{-1}_p$. Specifically, we assume that $\|\mathbf{W}_{j:}\|_p + \|\mathbf{W}_{:k}\|_p$ is typically large, while $\|\mathbf{W}_{j:}\|^{-1}_p + \|\mathbf{W}_{:k}\|^{-1}_p$ is generally small. Consequently, dividing by the former ($\mathrm{S2}$) or multiplying by the latter ($\mathrm{S4}$) reduces the magnitude of the pruning weights. We will provide statistical evidence to validate this assumption in subsequent sections.

\subsection{Influence of Sampling Ratios}\label{sec:sampling_ratios}
In this section, we examine the impact of varying sampling ratios in \algname{stochRIA}. It is important to note that these ratios are applied over $\min(b, c)$, where $b$ and $c$ represent the number of rows and columns in each layer, respectively. In \Cref{tab:stoch_res_diff_sampling_ratio}, we can see the performance of \algname{stochRIA} is generally stable and compares favorably to that of \algname{RIA} when sampling across entire rows and columns, particularly for $\beta \geq 0.05$. At this threshold and above, the performance is robust, occasionally even surpassing less noisy sampling configurations. However, at an extremely low ratio of $\beta = 0.01$, there is a significant performance decline. Consequently, we have set $\beta = 0.1$ as the default setting for our experiments.

\begin{table}[!tb]
    \centering
    \caption{Perplexity scores on Wikitext-2 for \algname{stochRIA} with different sampling ratios. The sparsity ratio is $50\%$, and all results correspond to $\alpha=1$. We highlight those performance drops over 0.1 as significant.}
    \label{tab:stoch_res_diff_sampling_ratio}
    \begin{tabular}{l|ccccccc}
    \toprule
    ratio ($\beta$) & LlaMA2-7b & LlaMA2-13b & LlaMA3-8b & OPT-1.3b\\ \midrule 
    1 & 6.91 & 5.95 & 9.45 & 18.88\\     \midrule
    0.9 & 6.91 & 5.95 & 9.43 & 18.87\\
    0.5 & 6.90 & 5.95 & 9.42 & 18.84\\
    0.1 & 6.91 & 5.95 & 9.46 & 18.78\\
    0.05 & 6.91 & 5.96 & 9.47 & 18.91\\
    0.01 & 6.98 & 6.00 & 9.69 {\small \color{red} -0.24} & 19.36 {\small \color{red} -0.48}\\
    \bottomrule
    \end{tabular}
\end{table}

\begin{table}[!tb]
    \centering
    \caption{\ft Hyperparameter Ablations on LLaMA3-8b. Each row shows the non-default hyperparameter values compared to the best-performing method.}
    \label{tab:ft_hyper_abl}
    \resizebox{1.0\textwidth}{!}{
    \begin{tabular}{l|c|cccccc}
    \toprule
    base & setting & $p$ & grow relative? & $\gamma_1$ & prune relative? & $\gamma_2$ & perplexity$\downarrow$\\ \midrule
    \multirow{8}{*}{\algname{Wanda}} & best & 2 & \cmark & 0 & \xmark & 0.0001 & 18.99\\ \cmidrule{2-8}
      & \multirow{2}{*}{$p$} & 1 & & & & & 19.04\\ 
      & & $\infty$ & &&&& 18.99\\ \cmidrule{2-8}
      & \multirow{2}{*}{$\gamma$} & & & & & 0 & 18.99\\ 
      & & & &&& 0.001 & 18.99\\ \cmidrule{2-8}
      & \multirow{3}{*}{relative} & & \xmark & & \xmark && 19.49\\
      & & & \xmark & & \cmark && 19.25\\
      & & & \cmark & & \cmark && 19.63\\ \cmidrule{1-8}
    \multirow{8}{*}{\algname{RIA}} & best & 2 & \xmark & 0 & \cmark & 0.001 & 20.50\\ \cmidrule{2-8}
      & \multirow{2}{*}{$p$} & 1 & & & & & 25.61\\ 
      & & $\infty$ & &&&& 20.51\\ \cmidrule{2-8}
      & \multirow{2}{*}{$\gamma$} & & & & & 0 & 20.51\\ 
      & & & &&& 0.0001 & 20.52\\ \cmidrule{2-8}
      & \multirow{3}{*}{relative} & & \xmark & & \xmark && 21.33\\
      & & & \cmark & & \xmark && 22.16\\
      & & & \cmark & & \cmark && 22.60\\
      \bottomrule
    \end{tabular}}
\end{table}
  
\subsection{Analysis of \ft Hyperparameters}\label{sec:ablation_ft_parameters}
In \Cref{sec:training_free_fine_tuning}, we introduced the equations for our proposed \ft method, specifically \Cref{eqn:ft01} and \Cref{eqn:ft02}. This method primarily involves three key hyperparameters: the regularization penalty $\gamma_1, \gamma_2$ and the norm type $p$. Additionally, we consider whether to apply relative importance reweighting during the growing or pruning phases—or during both. Given the number of hyperparameters, understanding their interactions can be computationally expensive and time-consuming.

To address this complexity, we adopt a systematic approach by performing a random search over 20 different combinations of hyperparameter settings. These combinations include: $p \in \{1, 2, \infty\}$, $\gamma_1 \in \{0, 0.0001, 0.001\}$, $\gamma_2 \in \{0, 0.0001, 0.001\}$, and binary choices for relative reweighting (True/False) during both the growing and pruning phases. For each of the 20 trials on the same model, we identify the best-performing combination and treat its hyperparameters as the "ground truth." We then evaluate the behavior under different scenarios and report the results in \Cref{tab:ft_hyper_abl}.

Our findings reveal several notable insights:

\begin{itemize}
    \item Norm type $p$: The smooth $\ell_p$-norm with $p = 2$ consistently achieves the best performance. Compared to the non-differentiable $\ell_1$-norm, which underperforms due to its non-smooth nature, and the $\ell_{\infty}$-norm, which focuses only on the largest values and ignores smaller differences, the $\ell_p$-norm with $p = 2$ balances sensitivity and robustness effectively.
    
    \item  Relative importance reweighting: Applying relative reweighting during either the growing or pruning phase improves performance significantly—yielding a 0.5 improvement on \algname{Wanda} and 0.83 on \algname{RIA}. However, applying reweighting to both phases simultaneously leads to substantial performance degradation, with a 0.64 and 2.1 drop on \algname{Wanda} and \algname{RIA}, respectively.
    
    \item  Regularization penalty $\gamma$: The impact of $\gamma$ is minimal, as variations in its value result in only marginal differences in performance. This finding highlights the greater importance of the relative reweighting strategy.
\end{itemize}

%% file: Papers.tex

\refstepcounter{chapter}%
\chapter*{\thechapter \quad Papers Accepted and Submitted}
Here is a list of papers accepted (14) and submitted (4) during my PhD.

\noindent $\bullet$ \textbf{Kai Yi}, Peter Richtárik. ``Symmetric Pruning of Large Language Models". \emph{arXiv preprint} arXiv:2501.18980 (2025). \emph{ICLR 2025 Workshop on Sparsity in LLMs} (SLLM).\\
\noindent $\bullet$ \textbf{Kai Yi}, Georg Meinhardt, Laurent Condat, and Peter Richtárik. "Fedcomloc: Communication-efficient distributed training of sparse and quantized models." \emph{arXiv preprint} arXiv:2403.09904 (2024).\\
\noindent $\bullet$ Meinhardt, Georg, \textbf{Kai Yi}, Laurent Condat, and Peter Richtárik. ``Prune at the Clients, Not the Server: Accelerated Sparse Training in Federated Learning." \emph{arXiv preprint} arXiv:2405.20623 (2024).\\
\noindent $\bullet$ Vladimir Malinovskii, Denis Mazur, Ivan Ilin, Denis Kuznedelev, Konstantin Pavlovich Burlachenko, \textbf{Kai Yi}, Dan Alistarh, Peter Richtárik. ``PV-Tuning: Beyond Straight-Through Estimation for Extreme LLM Compression." Oral presentation at \emph{The Thirty-eighth Annual Conference on Neural Information Processing Systems} (NeurIPS 2024). \\
$\bullet$ \textbf{Kai Yi}, Timur Kharisov, Igor Sokolov, and Peter Richtárik. ``Cohort Squeeze: Beyond a Single Communication Round per Cohort in Cross-Device Federated Learning." arXiv preprint arXiv:2406.01115 (2024). {Oral} presentation at \emph{International Workshop on Federated Foundation Models In Conjunction with NeurIPS 2024} (FL@FM-NeurIPS'24).\\
$\bullet$ \textbf{Kai Yi}, Nidham Gazagnadou, Peter Richtárik, and Lingjuan Lyu. ``FedP3: Federated Personalized and Privacy-friendly Network Pruning under Model Heterogeneity." In \textit{The Twelfth International Conference on Learning Representations (ICLR)}. 2024.\\
$\bullet$ Wenxuan Zhang, Paul Janson, \textbf{Kai Yi}, Ivan Skorokhodov, and Mohamed Elhoseiny. ``Continual Zero-Shot Learning through Semantically Guided Generative Random Walks." In \textit{Proceedings of the IEEE/CVF International Conference on Computer Vision (ICCV)}, pp. 11574-11585. 2023.\\
$\bullet$ \textbf{Kai Yi}, Paul Janson, and Mohamed Elhoseiny. ``Domain-aware continual zero-shot learning." In \textit{Out Of Distribution Generalization in Computer Vision Workshop of ICCV}, 2023.\\
$\bullet$ \textbf{Kai Yi}, Laurent Condat, and Peter Richtárik. ``Explicit personalization and local training: Double communication acceleration in federated learning." \emph{Transactions on Machine Learning Research} (TMLR), 2025.\\
$\bullet$ Condat Laurent, \textbf{Kai Yi}, and Peter Richtárik. ``EF-BV: A unified theory of error feedback and variance reduction mechanisms for biased and unbiased compression in distributed optimization." \textit{Advances in Neural Information Processing Systems (NeurIPS)} 35 (2022): 17501-17514.\\
$\bullet$ Grigory Malinovsky, \textbf{Kai Yi}, and Peter Richtárik. ``Variance reduced proxskip: Algorithm, theory and application to federated learning." \emph{Advances in Neural Information Processing Systems} 35 (2022): 15176-15189.\\
$\bullet$ \textbf{Kai Yi}, Xiaoqian Shen, Yunhao Gou, and Mohamed Elhoseiny. ``Exploring hierarchical graph representation for large-scale zero-shot image classification." In \emph{European Conference on Computer Vision (ECCV)}, pp. 116-132. Cham: Springer Nature Switzerland, 2022.\\
$\bullet$ Jun Chen, Han Guo, \textbf{Kai Yi}, Boyang Li, and Mohamed Elhoseiny. ``Visualgpt: Data-efficient adaptation of pretrained language models for image captioning." In \emph{Proceedings of the IEEE/CVF Conference on Computer Vision and Pattern Recognition (CVPR)}, pp. 18030-18040. 2022.\\
$\bullet$ \textbf{Kai Yi}, Divyansh Jha, Ivan Skorokhodov, and Mohamed Elhoseiny. ``Language-Guided Imaginative Walks: Generative Random Walk Deviation Loss for Unseen Class Recognition using Text Descriptions." In \textit{Learning with Limited Labelled Data for Image and Video Understanding Workshop of CVPR}, 2022.\\
$\bullet$ Divyansh Jha, \textbf{Kai Yi}, Ivan Skorokhodov, and Mohamed Elhoseiny. ``Creative Walk Adversarial Networks: Novel Art Generation with Probabilistic Random Walk Deviation from Style Norms." In \emph{13th International Conference on Computational Creativity (ICCC)}, 2022.\\
$\bullet$ \textbf{Kai Yi}, Yungeng Zhang, Jianye Pang, Xiangrui Zeng, Min Xu. ``Learning To Disentangle Semantic Features From cryo-ET with 3D Spatial Generative Network". \emph{Technical Report}, 2021.\\
$\bullet$ Yuchen Zeng, Gregory Howe, \textbf{Kai Yi}, Xiangrui Zeng, Jing Zhang, Yi-Wei Chang, and Min Xu. ``Unsupervised Domain Alignment Based Open Set Structural Recognition of Macromolecules Captured By Cryo-Electron Tomography." In \emph{2021 IEEE International Conference on Image Processing (ICIP)}, pp. 106-110. IEEE, 2021.\\
$\bullet$ Mohamed Elhoseiny*, \textbf{Kai Yi}*, and Mohamed Elfeki. ``Cizsl++: Creativity inspired generative zero-shot learning." \emph{arXiv preprint} arXiv:2101.00173 (2021).

%% file: Thesis.bbl
\begin{thebibliography}{256}
\providecommand{\natexlab}[1]{#1}
\providecommand{\url}[1]{\texttt{#1}}
\expandafter\ifx\csname urlstyle\endcsname\relax
  \providecommand{\doi}[1]{doi: #1}\else
  \providecommand{\doi}{doi: \begingroup \urlstyle{rm}\Url}\fi

\bibitem[Abadi et~al.(2016)Abadi, Chu, Goodfellow, McMahan, Mironov, Talwar, and Zhang]{abadi2016deep}
Martin Abadi, Andy Chu, Ian Goodfellow, H~Brendan McMahan, Ilya Mironov, Kunal Talwar, and Li~Zhang.
\newblock Deep learning with differential privacy.
\newblock In \emph{Proceedings of the 2016 ACM SIGSAC conference on computer and communications security}, pages 308--318, 2016.

\bibitem[Aji and Heafield(2017)]{aji2017sparse}
Alham~Fikri Aji and Kenneth Heafield.
\newblock Sparse communication for distributed gradient descent.
\newblock \emph{arXiv preprint arXiv:1704.05021}, 2017.

\bibitem[Alam et~al.(2022)Alam, Liu, Yan, and Zhang]{FedRolex}
Samiul Alam, Luyang Liu, Ming Yan, and Mi~Zhang.
\newblock Fedrolex: Model-heterogeneous federated learning with rolling sub-model extraction.
\newblock In \emph{Advances in Neural Information Processing Systems}, 2022.

\bibitem[Albasyoni et~al.(2020)Albasyoni, Safaryan, Condat, and {Richt\'arik}]{alb20}
A.~Albasyoni, M.~Safaryan, L.~Condat, and P.~{Richt\'arik}.
\newblock Optimal gradient compression for distributed and federated learning.
\newblock preprint arXiv:2010.03246, 2020.

\bibitem[Alistarh et~al.(2017)Alistarh, Grubic, Li, Tomioka, and Vojnovic]{ali17}
D.~Alistarh, D.~Grubic, J.~Li, R.~Tomioka, and M.~Vojnovic.
\newblock {QSGD: Communication-efficient SGD via gradient quantization and encoding}.
\newblock In \emph{Proc. of 31st Conf. Neural Information Processing Systems (NIPS)}, pages 1709--1720, 2017.

\bibitem[Andr{\'e}s et~al.(2013)Andr{\'e}s, Bordenabe, Chatzikokolakis, and Palamidessi]{andres2013geo}
Miguel~E Andr{\'e}s, Nicol{\'a}s~E Bordenabe, Konstantinos Chatzikokolakis, and Catuscia Palamidessi.
\newblock Geo-indistinguishability: Differential privacy for location-based systems.
\newblock In \emph{Proceedings of the 2013 ACM SIGSAC conference on Computer \& communications security}, pages 901--914, 2013.

\bibitem[Arivazhagan et~al.(2019)Arivazhagan, Aggarwal, Singh, and Choudhary]{arivazhagan2019federated}
M.~G. Arivazhagan, V.~Aggarwal, A.~K. Singh, and S.~Choudhary.
\newblock Federated learning with personalization layers.
\newblock preprint arXiv:1912.00818, 2019.

\bibitem[Arjevani et~al.(2020)Arjevani, Shamir, and Srebro]{arjevani2020tight}
Yossi Arjevani, Ohad Shamir, and Nathan Srebro.
\newblock A tight convergence analysis for stochastic gradient descent with delayed updates.
\newblock In \emph{Algorithmic Learning Theory}, pages 111--132. PMLR, 2020.

\bibitem[Asi and Duchi(2019)]{asi2019stochastic}
Hilal Asi and John~C Duchi.
\newblock {Stochastic (approximate) proximal point methods: Convergence, optimality, and adaptivity}.
\newblock \emph{SIAM Journal on Optimization}, 29\penalty0 (3):\penalty0 2257--2290, 2019.

\bibitem[Asi et~al.(2020)Asi, Chadha, Cheng, and Duchi]{asi2020minibatch}
Hilal Asi, Karan Chadha, Gary Cheng, and John~C Duchi.
\newblock {Minibatch stochastic approximate proximal point methods}.
\newblock In \emph{Advances in Neural Information Processing Systems}, volume~33, pages 21958--21968. Curran Associates, Inc., 2020.

\bibitem[Attouch and Bolte(2009)]{att09}
H.~Attouch and J.~Bolte.
\newblock On the convergence of the proximal algorithm for nonsmooth functions involving analytic features.
\newblock \emph{Math. Program.}, 116:\penalty0 5--116, 2009.

\bibitem[Baek et~al.(2023)Baek, Jeong, Jin, Yoon, and Hwang]{FED-PUB}
J.~Baek, W.~Jeong, J.~Jin, J.~Yoon, and S.~J. Hwang.
\newblock Personalized subgraph federated learning.
\newblock In \emph{{Proc.\ of} 40th Int. Conf. Machine Learning (ICML), PMLR 202}, pages 1396--1415, 2023.

\bibitem[Barnes et~al.(2020)Barnes, Inan, Isik, and {\"Ozg\"ur}]{bar20}
L.~P. Barnes, H.~A. Inan, B.~Isik, and A.~{\"Ozg\"ur}.
\newblock {rTop-k}: {A} statistical estimation approach to distributed {SGD}.
\newblock \emph{IEEE J. Sel. Areas Inf. Theory}, 1\penalty0 (3):\penalty0 897--907, November 2020.

\bibitem[Bassily et~al.(2014)Bassily, Smith, and Thakurta]{bassily2014private}
Raef Bassily, Adam Smith, and Abhradeep Thakurta.
\newblock Private empirical risk minimization: Efficient algorithms and tight error bounds.
\newblock In \emph{2014 IEEE 55th annual symposium on foundations of computer science}, pages 464--473. IEEE, 2014.

\bibitem[Bauschke and Combettes(2017)]{bau17}
H.~H. Bauschke and P.~L. Combettes.
\newblock \emph{Convex Analysis and Monotone Operator Theory in Hilbert Spaces}.
\newblock Springer, New York, 2nd edition, 2017.

\bibitem[Bellec et~al.(2018)Bellec, Kappel, Maass, and Legenstein]{bellec2018deep}
Guillaume Bellec, David Kappel, Wolfgang Maass, and Robert Legenstein.
\newblock Deep rewiring: Training very sparse deep networks.
\newblock In \emph{International Conference on Learning Representations}, 2018.

\bibitem[Bertsekas(2011)]{bertsekas2011incremental}
Dimitri~P Bertsekas.
\newblock {Incremental proximal methods for large scale convex optimization}.
\newblock \emph{Mathematical Programming}, 129\penalty0 (2):\penalty0 163--195, 2011.

\bibitem[Beznosikov et~al.(2020)Beznosikov, Horv{\'a}th, Richt{\'a}rik, and Safaryan]{bez20}
A.~Beznosikov, S.~Horv{\'a}th, P.~Richt{\'a}rik, and M.~Safaryan.
\newblock On biased compression for distributed learning.
\newblock preprint arXiv:2002.12410, 2020.

\bibitem[Beznosikov et~al.(2023)Beznosikov, Horv{\'a}th, Richt{\'a}rik, and Safaryan]{beznosikov2023biased}
Aleksandr Beznosikov, Samuel Horv{\'a}th, Peter Richt{\'a}rik, and Mher Safaryan.
\newblock On biased compression for distributed learning.
\newblock \emph{Journal of Machine Learning Research}, 24\penalty0 (276):\penalty0 1--50, 2023.

\bibitem[Bischoff et~al.(2023)Bischoff, Günnemann, Jaggi, and Stich]{bischoff2023second}
Sebastian Bischoff, Stephan Günnemann, Martin Jaggi, and Sebastian~U. Stich.
\newblock {On second-order optimization methods for federated learning}.
\newblock \emph{arXiv preprint arXiv:2303.10581}, 2023.

\bibitem[Bonawitz(2019)]{bonawitz2019towards}
Keith Bonawitz.
\newblock Towards federated learning at scale: Syste m design.
\newblock \emph{arXiv preprint arXiv:1902.01046}, 2019.

\bibitem[Brown et~al.(2020)Brown, Mann, Ryder, Subbiah, Kaplan, Dhariwal, Neelakantan, Shyam, Sastry, Askell, et~al.]{brown2020language}
Tom Brown, Benjamin Mann, Nick Ryder, Melanie Subbiah, Jared~D Kaplan, Prafulla Dhariwal, Arvind Neelakantan, Pranav Shyam, Girish Sastry, Amanda Askell, et~al.
\newblock Language models are few-shot learners.
\newblock \emph{Advances in neural information processing systems}, 33:\penalty0 1877--1901, 2020.

\bibitem[Broyden(1967)]{broyden1967quasi}
Charles~G Broyden.
\newblock {Quasi-Newton methods and their application to function minimisation}.
\newblock \emph{Mathematics of Computation}, 21\penalty0 (99):\penalty0 368--381, 1967.

\bibitem[Bui et~al.(2019)Bui, Malik, Goetz, Liu, Moon, Kumar, and Shin]{bui2019federated}
D.~Bui, K.~Malik, J.~Goetz, H.~Liu, S.~Moon, A.~Kumar, and K.~G. Shin.
\newblock Federated user representation learning.
\newblock preprint arXiv:1909.12535, 2019.

\bibitem[Bumin and Huang(2021)]{bumin2021efficient}
Aysegul Bumin and Kejun Huang.
\newblock {Efficient implementation of stochastic proximal point algorithm for matrix and tensor completion}.
\newblock In \emph{29th European Signal Processing Conference (EUSIPCO)}, pages 1050--1054. IEEE, 2021.

\bibitem[Caldas et~al.(2018)Caldas, Wua, Lia, J., McMahan, Smith, and Talwalkar]{leaf}
S.~Caldas, P.~Wua, T.~Lia, Kone\v{c}n\'{y} J., B.~McMahan, Virginia Smith, and Ameet Talwalkar.
\newblock Leaf: a benchmark for federated settings.
\newblock \emph{arXiv preprint arXiv:1812.01097}, 2018.

\bibitem[Chadha et~al.(2022)Chadha, Cheng, and Duchi]{chadha2022accelerated}
Karan Chadha, Gary Cheng, and John Duchi.
\newblock {Accelerated, optimal and parallel: Some results on model-based stochastic optimization}.
\newblock In \emph{Proceedings of the 39th International Conference on Machine Learning}, volume 162, pages 2811--2827. PMLR, 2022.

\bibitem[Chang and Lin(2011)]{chang2011libsvm}
C.-C. Chang and C.-J. Lin.
\newblock Lib{S}{V}{M}: {A} library for support vector machines.
\newblock \emph{ACM Transactions on Intelligent Systems and Technology (TIST)}, 2\penalty0 (3):\penalty0 27, 2011.

\bibitem[Chatzikokolakis et~al.(2013)Chatzikokolakis, Andr{\'e}s, Bordenabe, and Palamidessi]{chatzikokolakis2013broadening}
Konstantinos Chatzikokolakis, Miguel~E Andr{\'e}s, Nicol{\'a}s~Emilio Bordenabe, and Catuscia Palamidessi.
\newblock Broadening the scope of differential privacy using metrics.
\newblock In \emph{Privacy Enhancing Technologies: 13th International Symposium, PETS 2013, Bloomington, IN, USA, July 10-12, 2013. Proceedings 13}, pages 82--102. Springer, 2013.

\bibitem[Chen et~al.(2023)Chen, Yao, Gao, Ding, and Li]{pFedGate}
D.~Chen, L.~Yao, D.~Gao, B.~Ding, and Y.~Li.
\newblock Efficient personalized federated learning via sparse model-adaptation.
\newblock preprint arXiv:2305.02776, 2023.

\bibitem[Chen et~al.(2022)Chen, Guo, Yi, Li, and Elhoseiny]{VisualGPT}
Jun Chen, Han Guo, Kai Yi, Boyang Li, and Mohamed Elhoseiny.
\newblock Visualgpt: Data-efficient adaptation of pretrained language models for image captioning.
\newblock In \emph{Proceedings of the IEEE/CVF Conference on Computer Vision and Pattern Recognition}, pages 18030--18040, 2022.

\bibitem[Chen et~al.(2020)Chen, Qin, Wang, Yu, and Gao]{FedHealth}
Y.~Chen, X.~Qin, J.~Wang, C.~Yu, and W.~Gao.
\newblock Fedhealth: A federated transfer learning framework for wearable healthcare.
\newblock \emph{IEEE Intelligent Systems}, 35\penalty0 (4):\penalty0 83--93, 2020.

\bibitem[Choudhary et~al.(2020)Choudhary, Mishra, Goswami, and Sarangapani]{choudhary2020comprehensive}
Tejalal Choudhary, Vipul Mishra, Anurag Goswami, and Jagannathan Sarangapani.
\newblock A comprehensive survey on model compression and acceleration.
\newblock \emph{Artificial Intelligence Review}, 53:\penalty0 5113--5155, 2020.

\bibitem[Chowdhery et~al.(2022)Chowdhery, Narang, Devlin, Bosma, Mishra, Roberts, Barham, Chung, Sutton, Gehrmann, et~al.]{chowdhery2022palm}
Aakanksha Chowdhery, Sharan Narang, Jacob Devlin, Maarten Bosma, Gaurav Mishra, Adam Roberts, Paul Barham, Hyung~Won Chung, Charles Sutton, Sebastian Gehrmann, et~al.
\newblock Palm: Scaling language modeling with pathways.
\newblock \emph{arXiv preprint arXiv:2204.02311}, 2022.

\bibitem[Clark et~al.(2019)Clark, Lee, Chang, Kwiatkowski, Collins, and Toutanova]{clark2019boolq}
Christopher Clark, Kenton Lee, Ming-Wei Chang, Tom Kwiatkowski, Michael Collins, and Kristina Toutanova.
\newblock Boolq: Exploring the surprising difficulty of natural yes/no questions.
\newblock \emph{arXiv preprint arXiv:1905.10044}, 2019.

\bibitem[Clark et~al.(2018)Clark, Cowhey, Etzioni, Khot, Sabharwal, Schoenick, and Tafjord]{clark2018think}
Peter Clark, Isaac Cowhey, Oren Etzioni, Tushar Khot, Ashish Sabharwal, Carissa Schoenick, and Oyvind Tafjord.
\newblock Think you have solved question answering? try arc, the ai2 reasoning challenge.
\newblock \emph{arXiv preprint arXiv:1803.05457}, 2018.

\bibitem[Cohen et~al.(2017)Cohen, Afshar, Tapson, and Van~Schaik]{EMNIST}
Gregory Cohen, Saeed Afshar, Jonathan Tapson, and Andre Van~Schaik.
\newblock Emnist: Extending mnist to handwritten letters.
\newblock In \emph{2017 international joint conference on neural networks (IJCNN)}, pages 2921--2926. IEEE, 2017.

\bibitem[Condat and {Richt\'arik}(2022)]{con22m}
L.~Condat and P.~{Richt\'arik}.
\newblock {MURANA: A} generic framework for stochastic variance-reduced optimization.
\newblock In \emph{Proc. of the Mathematical and Scientific Machine Learning (MSML) conference}, 2022.

\bibitem[Condat and Richt{\'a}rik(2023)]{con22rp}
L.~Condat and P.~Richt{\'a}rik.
\newblock {RandProx}: {P}rimal-dual optimization algorithms with randomized proximal updates.
\newblock In \emph{{Proc.\ of} Int. Conf. Learning Representations (ICLR)}, 2023.

\bibitem[Condat et~al.(2022{\natexlab{a}})Condat, Agarsk{\'y}, and Richt{\'a}rik]{CompressedScaffnew}
L.~Condat, I.~Agarsk{\'y}, and P.~Richt{\'a}rik.
\newblock Provably doubly accelerated federated learning: The first theoretically successful combination of local training and compressed communication.
\newblock preprint arXiv:2210.13277, 2022{\natexlab{a}}.

\bibitem[Condat et~al.(2022{\natexlab{b}})Condat, Kitahara, Contreras, and Hirabayashi]{con19}
L.~Condat, D.~Kitahara, A.~Contreras, and A.~Hirabayashi.
\newblock Proximal splitting algorithms for convex optimization: {A} tour of recent advances, with new twists.
\newblock \emph{SIAM Review}, 2022{\natexlab{b}}.
\newblock to appear.

\bibitem[Condat et~al.(2022{\natexlab{c}})Condat, Malinovsky, and Richt{\'a}rik]{con22}
L.~Condat, G.~Malinovsky, and P.~Richt{\'a}rik.
\newblock Distributed proximal splitting algorithms with rates and acceleration.
\newblock \emph{Frontiers in Signal Processing}, 1, January 2022{\natexlab{c}}.

\bibitem[Condat et~al.(2023)Condat, {Agarsk\'y}, Malinovsky, and Richt{\'a}rik]{con23tam2}
L.~Condat, I.~{Agarsk\'y}, G.~Malinovsky, and P.~Richt{\'a}rik.
\newblock {TAMUNA}: {D}oubly accelerated federated learning with local training, compression, and partial participation.
\newblock preprint arXiv:2302.09832, 2023.

\bibitem[Cunha et~al.(2022)Cunha, Gidel, Pedregosa, Scieur, and Paquette]{cunha2022only}
Leonardo Cunha, Gauthier Gidel, Fabian Pedregosa, Damien Scieur, and Courtney Paquette.
\newblock Only tails matter: Average-case universality and robustness in the convex regime.
\newblock In \emph{International Conference on Machine Learning}, pages 4474--4491. PMLR, 2022.

\bibitem[Dean et~al.(2012)Dean, Corrado, Monga, Chen, Devin, Mao, Ranzato, Senior, Tucker, Yang, et~al.]{dean2012large}
Jeffrey Dean, Greg Corrado, Rajat Monga, Kai Chen, Matthieu Devin, Mark Mao, Marc'aurelio Ranzato, Andrew Senior, Paul Tucker, Ke~Yang, et~al.
\newblock Large scale distributed deep networks.
\newblock \emph{Advances in neural information processing systems}, 25, 2012.

\bibitem[Diao et~al.(2021)Diao, Ding, and Tarokh]{HeteroFL}
Enmao Diao, Jie Ding, and Vahid Tarokh.
\newblock Heterofl: Computation and communication efficient federated learning for heterogeneous clients.
\newblock In \emph{International Conference on Learning Representations}, 2021.

\bibitem[Ding et~al.(2021)Ding, Liang, Bi, and Pan]{ding2021differentially}
Jiahao Ding, Guannan Liang, Jinbo Bi, and Miao Pan.
\newblock Differentially private and communication efficient collaborative learning.
\newblock In \emph{Proceedings of the AAAI Conference on Artificial Intelligence}, volume~35, pages 7219--7227, 2021.

\bibitem[Dinh et~al.(2020)Dinh, Tran, and Nguyen]{pFedMe}
C.~T. Dinh, N.~H. Tran, and T.~D. Nguyen.
\newblock Personalized federated learning with {Moreau} envelopes.
\newblock In \emph{Proc. of Conf. Neural Information Processing Systems (NeurIPS)}, volume~33, pages 21394--21405, 2020.

\bibitem[Dubey et~al.(2024)Dubey, Jauhri, Pandey, Kadian, Al-Dahle, Letman, Mathur, Schelten, Yang, Fan, et~al.]{LlaMA3}
Abhimanyu Dubey, Abhinav Jauhri, Abhinav Pandey, Abhishek Kadian, Ahmad Al-Dahle, Aiesha Letman, Akhil Mathur, Alan Schelten, Amy Yang, Angela Fan, et~al.
\newblock The llama 3 herd of models.
\newblock \emph{arXiv preprint arXiv:2407.21783}, 2024.

\bibitem[Dun et~al.(2022)Dun, Wolfe, Jermaine, and Kyrillidis]{dun2022resist}
Chen Dun, Cameron~R Wolfe, Christopher~M Jermaine, and Anastasios Kyrillidis.
\newblock Resist: Layer-wise decomposition of resnets for distributed training.
\newblock In \emph{Uncertainty in Artificial Intelligence}, pages 610--620. PMLR, 2022.

\bibitem[Dun et~al.(2023)Dun, Hipolito, Jermaine, Dimitriadis, and Kyrillidis]{dun2023efficient}
Chen Dun, Mirian Hipolito, Chris Jermaine, Dimitrios Dimitriadis, and Anastasios Kyrillidis.
\newblock Efficient and light-weight federated learning via asynchronous distributed dropout.
\newblock In \emph{International Conference on Artificial Intelligence and Statistics}, pages 6630--6660. PMLR, 2023.

\bibitem[Dwork et~al.(2006)Dwork, McSherry, Nissim, and Smith]{dwork2006calibrating}
Cynthia Dwork, Frank McSherry, Kobbi Nissim, and Adam Smith.
\newblock Calibrating noise to sensitivity in private data analysis.
\newblock In \emph{Theory of Cryptography: Third Theory of Cryptography Conference, TCC 2006, New York, NY, USA, March 4-7, 2006. Proceedings 3}, pages 265--284. Springer, 2006.

\bibitem[Dwork et~al.(2014)Dwork, Roth, et~al.]{dwork2014algorithmic}
Cynthia Dwork, Aaron Roth, et~al.
\newblock The algorithmic foundations of differential privacy.
\newblock \emph{Foundations and Trends{\textregistered} in Theoretical Computer Science}, 9\penalty0 (3--4):\penalty0 211--407, 2014.

\bibitem[Egiazarian et~al.(2024)Egiazarian, Panferov, Kuznedelev, Frantar, Babenko, and Alistarh]{AQLM}
Vage Egiazarian, Andrei Panferov, Denis Kuznedelev, Elias Frantar, Artem Babenko, and Dan Alistarh.
\newblock Extreme compression of large language models via additive quantization.
\newblock In \emph{Forty-first International Conference on Machine Learning}, 2024.

\bibitem[Elhoseiny et~al.(2021)Elhoseiny, Yi, and Elfeki]{elhoseiny2021cizsl++}
Mohamed Elhoseiny, Kai Yi, and Mohamed Elfeki.
\newblock Cizsl++: Creativity inspired generative zero-shot learning.
\newblock \emph{T-PAMI major revision}, 2021.

\bibitem[Evci et~al.(2020)Evci, Gale, Menick, Castro, and Elsen]{RigL}
Utku Evci, Trevor Gale, Jacob Menick, Pablo~Samuel Castro, and Erich Elsen.
\newblock Rigging the lottery: Making all tickets winners.
\newblock In \emph{International conference on machine learning}, pages 2943--2952. PMLR, 2020.

\bibitem[Fallah et~al.(2020)Fallah, Mokhtari, and Ozdaglar]{fallah2020personalized}
Alireza Fallah, Aryan Mokhtari, and Asuman Ozdaglar.
\newblock Personalized federated learning with theoretical guarantees: A model-agnostic meta-learning approach.
\newblock \emph{Advances in neural information processing systems}, 33:\penalty0 3557--3568, 2020.

\bibitem[Fatkhullin et~al.(2021)Fatkhullin, Sokolov, Gorbunov, Li, and Richt{\'a}rik]{fat21}
I.~Fatkhullin, I.~Sokolov, E.~Gorbunov, Z.~Li, and P.~Richt{\'a}rik.
\newblock {EF21} with bells {\&} whistles: Practical algorithmic extensions of modern error feedback.
\newblock preprint arXiv:2110.03294, 2021.

\bibitem[Feldman et~al.(2020)Feldman, Koren, and Talwar]{feldman2020private}
Vitaly Feldman, Tomer Koren, and Kunal Talwar.
\newblock Private stochastic convex optimization: optimal rates in linear time.
\newblock In \emph{Proceedings of the 52nd Annual ACM SIGACT Symposium on Theory of Computing}, pages 439--449, 2020.

\bibitem[Fletcher(1970)]{fletcher1970new}
R.~Fletcher.
\newblock {{A new approach to variable metric algorithms}}.
\newblock \emph{The Computer Journal}, 13\penalty0 (3):\penalty0 317--322, 1970.

\bibitem[Folland(1984)]{Folland1984RealAM}
Gerald~B. Folland.
\newblock {Real Analysis: Modern Techniques and Their Applications}.
\newblock 1984.

\bibitem[Frankle and Carbin(2018)]{frankle2018lottery}
Jonathan Frankle and Michael Carbin.
\newblock The lottery ticket hypothesis: Finding sparse, trainable neural networks.
\newblock \emph{arXiv preprint arXiv:1803.03635}, 2018.

\bibitem[Frantar and Alistarh(2023)]{SparseGPT}
Elias Frantar and Dan Alistarh.
\newblock Sparsegpt: Massive language models can be accurately pruned in one-shot.
\newblock In \emph{International Conference on Machine Learning}, pages 10323--10337. PMLR, 2023.

\bibitem[Frantar et~al.(2023)Frantar, Ashkboos, Hoefler, and Alistarh]{GPTQ}
Elias Frantar, Saleh Ashkboos, Torsten Hoefler, and Dan Alistarh.
\newblock {OPTQ}: Accurate quantization for generative pre-trained transformers.
\newblock In \emph{The Eleventh International Conference on Learning Representations}, 2023.
\newblock URL \url{https://openreview.net/forum?id=tcbBPnfwxS}.

\bibitem[Gandikota et~al.(2019)Gandikota, Kane, Maity, and Mazumdar]{gan19}
Venkata Gandikota, Daniel Kane, Raj~Kumar Maity, and Arya Mazumdar.
\newblock {vqSGD}: {V}ector quantized stochastic gradient descent.
\newblock preprint arXiv:1911.07971, 2019.

\bibitem[Gao et~al.(2022{\natexlab{a}})Gao, Yao, and Yang]{gao2022survey}
Dashan Gao, Xin Yao, and Qiang Yang.
\newblock A survey on heterogeneous federated learning.
\newblock \emph{arXiv preprint arXiv:2210.04505}, 2022{\natexlab{a}}.

\bibitem[Gao et~al.(2021)Gao, Tow, Biderman, Black, DiPofi, Foster, Golding, Hsu, McDonell, Muennighoff, et~al.]{gao2021framework}
Leo Gao, Jonathan Tow, Stella Biderman, Sid Black, Anthony DiPofi, Charles Foster, Laurence Golding, Jeffrey Hsu, Kyle McDonell, Niklas Muennighoff, et~al.
\newblock A framework for few-shot language model evaluation.
\newblock \emph{Version v0. 0.1. Sept}, 10:\penalty0 8--9, 2021.

\bibitem[Gao et~al.(2022{\natexlab{b}})Gao, Fu, Li, Chen, Xu, and Xu]{Gao_2022_CVPR}
Liang Gao, Huazhu Fu, Li~Li, Yingwen Chen, Ming Xu, and Cheng-Zhong Xu.
\newblock Feddc: Federated learning with non-iid data via local drift decoupling and correction.
\newblock In \emph{Proceedings of the IEEE/CVF Conference on Computer Vision and Pattern Recognition (CVPR)}, pages 10112--10121, June 2022{\natexlab{b}}.

\bibitem[Gasanov et~al.(2022)Gasanov, Khaled, Horv\'{a}th, and Richt\'{a}rik]{FLIX}
E.~Gasanov, A.~Khaled, S.~Horv\'{a}th, and P.~Richt\'{a}rik.
\newblock Flix: A simple and communication-efficient alternative to local methods in federated learning.
\newblock In \emph{{Proc.\ of} 24th Int. Conf. Artificial Intelligence and Statistics (AISTATS)}, 2022.

\bibitem[Ghadimi and Lan(2013)]{ghadimi2013stochastic}
Saeed Ghadimi and Guanghui Lan.
\newblock Stochastic first-and zeroth-order methods for nonconvex stochastic programming.
\newblock \emph{SIAM Journal on Optimization}, 23\penalty0 (4):\penalty0 2341--2368, 2013.

\bibitem[Ghosh et~al.(2020)Ghosh, Chung, Yin, and Ramchandran]{ghosh2020efficient}
Avishek Ghosh, Jichan Chung, Dong Yin, and Kannan Ramchandran.
\newblock An efficient framework for clustered federated learning.
\newblock \emph{Advances in Neural Information Processing Systems}, 33:\penalty0 19586--19597, 2020.

\bibitem[Goldfarb(1970)]{goldfarb1970family}
Donald Goldfarb.
\newblock {A family of variable-metric methods derived by variational means}.
\newblock \emph{Mathematics of Computation}, 24\penalty0 (109):\penalty0 23--26, 1970.

\bibitem[Gorbunov et~al.(2020{\natexlab{a}})Gorbunov, Hanzely, and Richt\'{a}rik]{LSGDunified2020}
E.~Gorbunov, F.~Hanzely, and P.~Richt\'{a}rik.
\newblock Local {SGD}: {Unified theory} and new efficient methods.
\newblock In \emph{Proc. of Conf. Neural Information Processing Systems (NeurIPS)}, 2020{\natexlab{a}}.

\bibitem[Gorbunov et~al.(2020{\natexlab{b}})Gorbunov, Hanzely, and {Richt\'{a}rik}]{gor202}
E.~Gorbunov, F.~Hanzely, and P.~{Richt\'{a}rik}.
\newblock A unified theory of {SGD: Variance} reduction, sampling, quantization and coordinate descent.
\newblock In \emph{Proc. of 23rd Int. Conf. Artificial Intelligence and Statistics (AISTATS)}, 2020{\natexlab{b}}.

\bibitem[Gorbunov et~al.(2020{\natexlab{c}})Gorbunov, Kovalev, Makarenko, and {Richt\'arik}]{gor20}
E.~Gorbunov, D.~Kovalev, D.~Makarenko, and P.~{Richt\'arik}.
\newblock Linearly converging error compensated {SGD}.
\newblock In \emph{Proc. of 34th Conf. Neural Information Processing Systems (NeurIPS)}, 2020{\natexlab{c}}.

\bibitem[Goujaud et~al.(2022)Goujaud, Scieur, Dieuleveut, Taylor, and Pedregosa]{goujaud2022super}
Baptiste Goujaud, Damien Scieur, Aymeric Dieuleveut, Adrien~B Taylor, and Fabian Pedregosa.
\newblock Super-acceleration with cyclical step-sizes.
\newblock In \emph{International Conference on Artificial Intelligence and Statistics}, pages 3028--3065. PMLR, 2022.

\bibitem[Gower et~al.(2019{\natexlab{a}})Gower, Loizou, Qian, Sailanbayev, Shulgin, and Richt\'{a}rik]{gow19}
R.~M. Gower, N.~Loizou, X.~Qian, A.~Sailanbayev, E.~Shulgin, and P.~Richt\'{a}rik.
\newblock {SGD}: {G}eneral analysis and improved rates.
\newblock In \emph{{Proc.\ of} 36th Int. Conf. Machine Learning (ICML), PMLR 97}, pages 5200--5209, 2019{\natexlab{a}}.

\bibitem[Gower et~al.(2020)Gower, Schmidt, Bach, and {Richt\'arik}]{gow20a}
R.~M. Gower, M.~Schmidt, F.~Bach, and P.~{Richt\'arik}.
\newblock Variance-reduced methods for machine learning.
\newblock \emph{Proc. of the IEEE}, 108\penalty0 (11):\penalty0 1968--1983, November 2020.

\bibitem[Gower et~al.(2021)Gower, {Richt\'arik}, and Bach]{gow20}
R.~M. Gower, P.~{Richt\'arik}, and F.~Bach.
\newblock Stochastic quasi-gradient methods: {V}ariance reduction via {J}acobian sketching.
\newblock \emph{Math. Program.}, 188:\penalty0 135--192, July 2021.

\bibitem[Gower et~al.(2019{\natexlab{b}})Gower, Loizou, Qian, Sailanbayev, Shulgin, and Richt{\'a}rik]{gower2019sgd}
Robert~Mansel Gower, Nicolas Loizou, Xun Qian, Alibek Sailanbayev, Egor Shulgin, and Peter Richt{\'a}rik.
\newblock Sgd: General analysis and improved rates.
\newblock In \emph{International conference on machine learning}, pages 5200--5209. PMLR, 2019{\natexlab{b}}.

\bibitem[Grudzie{\'n} et~al.(2023)Grudzie{\'n}, Malinovsky, and Richt{\'a}rik]{mal22b}
M.~Grudzie{\'n}, G.~Malinovsky, and P.~Richt{\'a}rik.
\newblock {Can 5th Generation Local Training Methods Support Client Sampling? Yes!}
\newblock In \emph{{Proc.\ of} Int. Conf. Artificial Intelligence and Statistics (AISTATS)}, April 2023.

\bibitem[Haddadpour and Mahdavi(2019)]{LocalDescent2019}
F.~Haddadpour and M.~Mahdavi.
\newblock On the convergence of local descent methods in federated learning.
\newblock preprint arXiv:1910.14425, 2019.

\bibitem[Han et~al.(2015)Han, Pool, Tran, and Dally]{han2015learning}
Song Han, Jeff Pool, John Tran, and William Dally.
\newblock Learning both weights and connections for efficient neural network.
\newblock \emph{Advances in neural information processing systems}, 28, 2015.

\bibitem[Hanzely and Richt{\'a}rik(2020)]{hanzely2020federated}
F.~Hanzely and P.~Richt{\'a}rik.
\newblock Federated learning of a mixture of global and local models.
\newblock preprint arXiv:2002.05516, 2020.

\bibitem[Hanzely and Richt{\'a}rik(2019)]{han19}
Filip Hanzely and Peter Richt{\'a}rik.
\newblock One method to rule them all: Variance reduction for data, parameters and many new methods.
\newblock \emph{preprint arXiv:1905.11266}, 2019.

\bibitem[Hanzely et~al.(2021)Hanzely, Zhao, and Kolar]{hanzely2021personalized}
Filip Hanzely, Boxin Zhao, and Mladen Kolar.
\newblock Personalized federated learning: A unified framework and universal optimization techniques.
\newblock \emph{arXiv preprint arXiv:2102.09743}, 2021.

\bibitem[Hanzely et~al.(2022)Hanzely, Kamzolov, Pasechnyuk, Gasnikov, Richt\'{a}rik, and Tak\'{a}\v{c}]{hanzely2022damped}
Slavom{\'\i}r Hanzely, Dmitry Kamzolov, Dmitry Pasechnyuk, Alexander Gasnikov, Peter Richt\'{a}rik, and Martin Tak\'{a}\v{c}.
\newblock A damped newton method achieves global $o(1/k^2)$ and local quadratic convergence rate.
\newblock \emph{Advances in Neural Information Processing Systems}, 35:\penalty0 25320--25334, 2022.

\bibitem[Hard et~al.(2018)Hard, Rao, Mathews, Ramaswamy, Beaufays, Augenstein, Eichner, Kiddon, and Ramage]{hard2018federated}
Andrew Hard, Kanishka Rao, Rajiv Mathews, Swaroop Ramaswamy, Fran{\c{c}}oise Beaufays, Sean Augenstein, Hubert Eichner, Chlo{\'e} Kiddon, and Daniel Ramage.
\newblock {Federated learning for mobile keyboard prediction}.
\newblock \emph{arXiv preprint arXiv:1811.03604}, 2018.

\bibitem[He et~al.(2021)He, Mushtaq, Ding, and Avestimehr]{he2021fednas}
Chaoyang He, Erum Mushtaq, Jie Ding, and Salman Avestimehr.
\newblock Fednas: Federated deep learning via neural architecture search.
\newblock 2021.

\bibitem[He et~al.(2016)He, Zhang, Ren, and Sun]{ResNet}
Kaiming He, Xiangyu Zhang, Shaoqing Ren, and Jian Sun.
\newblock Deep residual learning for image recognition.
\newblock In \emph{Proceedings of the IEEE conference on computer vision and pattern recognition}, pages 770--778, 2016.

\bibitem[Hestenes et~al.(1952)Hestenes, Stiefel, et~al.]{conjugate_gradients}
Magnus~Rudolph Hestenes, Eduard Stiefel, et~al.
\newblock \emph{{Methods of conjugate gradients for solving linear systems}}, volume~49.
\newblock NBS Washington, DC, 1952.

\bibitem[Hoefler et~al.(2021)Hoefler, Alistarh, Ben-Nun, Dryden, and Peste]{hoefler2021sparsity}
Torsten Hoefler, Dan Alistarh, Tal Ben-Nun, Nikoli Dryden, and Alexandra Peste.
\newblock Sparsity in deep learning: Pruning and growth for efficient inference and training in neural networks.
\newblock \emph{Journal of Machine Learning Research}, 22\penalty0 (241):\penalty0 1--124, 2021.

\bibitem[{Horv\'ath} et~al.(2022){Horv\'ath}, Kovalev, Mishchenko, Stich, and {Richt\'arik}]{hor22}
S.~{Horv\'ath}, D.~Kovalev, K.~Mishchenko, S.~Stich, and P.~{Richt\'arik}.
\newblock Stochastic distributed learning with gradient quantization and variance reduction.
\newblock \emph{Optimization Methods and Software}, 2022.

\bibitem[Horv\'{a}th et~al.(2019)Horv\'{a}th, Ho, Horv\'{a}th, Sahu, Canini, and Richt\'{a}rik]{Cnat}
Samuel Horv\'{a}th, Chen-Yu Ho, Ludov\'{i}t Horv\'{a}th, Atal~Narayan Sahu, Marco Canini, and Peter Richt\'{a}rik.
\newblock Natural compression for distributed deep learning.
\newblock preprint arXiv:1905.10988, 2019.

\bibitem[Horv\'{a}th et~al.(2021)Horv\'{a}th, Laskaridis, Almeida, Leontiadis, Venieris, and Lane]{FjORD}
Samuel Horv\'{a}th, Stefanos Laskaridis, Mario Almeida, Ilias Leontiadis, Stylianos Venieris, and Nicholas Lane.
\newblock Fjord: Fair and accurate federated learning under heterogeneous targets with ordered dropout.
\newblock \emph{Advances in Neural Information Processing Systems}, 34:\penalty0 12876--12889, 2021.

\bibitem[Huang et~al.(2022)Huang, Zhang, Sun, Fang, Yuan, and Wu]{FedTiny}
Hong Huang, Lan Zhang, Chaoyue Sun, Ruogu Fang, Xiaoyong Yuan, and Dapeng Wu.
\newblock Fedtiny: Pruned federated learning towards specialized tiny models.
\newblock \emph{arXiv preprint arXiv:2212.01977}, 2022.

\bibitem[Huang et~al.(2020)Huang, Su, Ravi, Song, Arora, and Li]{huang2020privacy}
Yangsibo Huang, Yushan Su, Sachin Ravi, Zhao Song, Sanjeev Arora, and Kai Li.
\newblock Privacy-preserving learning via deep net pruning.
\newblock \emph{arXiv preprint arXiv:2003.01876}, 2020.

\bibitem[Hubara et~al.(2018)Hubara, Courbariaux, Soudry, El-Yaniv, and Bengio]{hubara2018quantized}
Itay Hubara, Matthieu Courbariaux, Daniel Soudry, Ran El-Yaniv, and Yoshua Bengio.
\newblock Quantized neural networks: Training neural networks with low precision weights and activations.
\newblock \emph{Journal of Machine Learning Research}, 18\penalty0 (187):\penalty0 1--30, 2018.

\bibitem[Iyengar et~al.(2019)Iyengar, Near, Song, Thakkar, Thakurta, and Wang]{iyengar2019towards}
Roger Iyengar, Joseph~P Near, Dawn Song, Om~Thakkar, Abhradeep Thakurta, and Lun Wang.
\newblock Towards practical differentially private convex optimization.
\newblock In \emph{2019 IEEE Symposium on Security and Privacy (SP)}, pages 299--316. IEEE, 2019.

\bibitem[Jaggi et~al.(2014)Jaggi, Smith, Tak{\'a}c, Terhorst, Krishnan, Hofmann, and Jordan]{jaggi2014communication}
Martin Jaggi, Virginia Smith, Martin Tak{\'a}c, Jonathan Terhorst, Sanjay Krishnan, Thomas Hofmann, and Michael~I Jordan.
\newblock Communication-efficient distributed dual coordinate ascent.
\newblock \emph{Advances in neural information processing systems}, 27, 2014.

\bibitem[Jahani et~al.(2021)Jahani, Rusakov, Shi, Richt{\'a}rik, Mahoney, and Tak{\'a}{\v{c}}]{jahani2021doubly}
Majid Jahani, Sergey Rusakov, Zheng Shi, Peter Richt{\'a}rik, Michael~W Mahoney, and Martin Tak{\'a}{\v{c}}.
\newblock {Doubly adaptive scaled algorithm for machine learning using second-order information}.
\newblock \emph{arXiv preprint arXiv:2109.05198}, 2021.

\bibitem[Javaheripi et~al.(2023)Javaheripi, Bubeck, Abdin, Aneja, Bubeck, Mendes, Chen, Del~Giorno, Eldan, Gopi, et~al.]{Phi2}
Mojan Javaheripi, S{\'e}bastien Bubeck, Marah Abdin, Jyoti Aneja, Sebastien Bubeck, Caio C{\'e}sar~Teodoro Mendes, Weizhu Chen, Allie Del~Giorno, Ronen Eldan, Sivakanth Gopi, et~al.
\newblock Phi-2: The surprising power of small language models.
\newblock \emph{Microsoft Research Blog}, 2023.

\bibitem[Jha et~al.(2022)Jha, Yi, Skorokhodov, and Elhoseiny]{Jha2022CreativeWA}
Divyansh Jha, Kai Yi, Ivan Skorokhodov, and Mohamed Elhoseiny.
\newblock Creative walk adversarial networks: Novel art generation with probabilistic random walk deviation from style norms.
\newblock In \emph{International Conference on Innovative Computing and Cloud Computing}, 2022.
\newblock URL \url{https://api.semanticscholar.org/CorpusID:252440876}.

\bibitem[Jhunjhunwala et~al.(2023)Jhunjhunwala, Wang, and Joshi]{jhunjhunwala2023fedexp}
Divyansh Jhunjhunwala, Shiqiang Wang, and Gauri Joshi.
\newblock {FedExP: Speeding up federated averaging via extrapolation}.
\newblock \emph{arXiv preprint arXiv:2301.09604}, 2023.

\bibitem[Jiang et~al.(2020)Jiang, Kantarci, Oktug, and Soyata]{jiang2020federated}
Ji~Chu Jiang, Burak Kantarci, Sema Oktug, and Tolga Soyata.
\newblock Federated learning in smart city sensing: Challenges and opportunities.
\newblock \emph{Sensors}, 20\penalty0 (21):\penalty0 6230, 2020.

\bibitem[Jiang et~al.(2022{\natexlab{a}})Jiang, Wang, Valls, Ko, Lee, Leung, and Tassiulas]{PruneFL}
Yuang Jiang, Shiqiang Wang, Victor Valls, Bong~Jun Ko, Wei-Han Lee, Kin~K Leung, and Leandros Tassiulas.
\newblock Model pruning enables efficient federated learning on edge devices.
\newblock \emph{IEEE Transactions on Neural Networks and Learning Systems}, 2022{\natexlab{a}}.

\bibitem[Jiang et~al.(2022{\natexlab{b}})Jiang, Wang, Valls, Ko, Lee, Leung, and Tassiulas]{jiang2022model}
Yuang Jiang, Shiqiang Wang, Victor Valls, Bong~Jun Ko, Wei-Han Lee, Kin~K Leung, and Leandros Tassiulas.
\newblock Model pruning enables efficient federated learning on edge devices.
\newblock \emph{IEEE Transactions on Neural Networks and Learning Systems}, 2022{\natexlab{b}}.

\bibitem[Kairouz et~al.(2021)Kairouz, McMahan, Avent, Bellet, Bennis, Bhagoji, Bonawitz, Charles, Cormode, Cummings, et~al.]{kairouz2021advances}
Peter Kairouz, H~Brendan McMahan, Brendan Avent, Aur{\'e}lien Bellet, Mehdi Bennis, Arjun~Nitin Bhagoji, Kallista Bonawitz, Zachary Charles, Graham Cormode, Rachel Cummings, et~al.
\newblock Advances and open problems in federated learning.
\newblock \emph{Foundations and Trends{\textregistered} in Machine Learning}, 14\penalty0 (1--2):\penalty0 1--210, 2021.

\bibitem[{Kairouz et al.}(2019)]{FL-big}
P.~{Kairouz et al.}
\newblock Advances and open problems in federated learning.
\newblock \emph{Foundations and Trends in Machine Learning}, 14\penalty0 (1--2):\penalty0 1--210, 2019.

\bibitem[Karimi et~al.(2022)Karimi, Li, and Li]{karimi2022layerwise}
Belhal Karimi, Ping Li, and Xiaoyun Li.
\newblock {Layer-wise and dimension-wise locally adaptive federated learning}, 2022.

\bibitem[Karimi et~al.(2016)Karimi, Nutini, and Schmidt]{kar16}
Hamed Karimi, Julie Nutini, and Mark Schmidt.
\newblock Linear convergence of gradient and proximal-gradient methods under the {{P}olyak-{\L}ojasiewicz} condition.
\newblock In Paolo Frasconi, Niels Landwehr, Giuseppe Manco, and Jilles Vreeken, editors, \emph{Machine Learning and Knowledge Discovery in Databases}, pages 795--811, Cham, 2016. Springer International Publishing.

\bibitem[Karimireddy et~al.(2020{\natexlab{a}})Karimireddy, Kale, Mohri, Reddi, Stich, and Suresh]{SCAFFOLD}
S.~Karimireddy, S.~Kale, M.~Mohri, S.~Reddi, S.~Stich, and A.~Suresh.
\newblock {SCAFFOLD}: Stochastic controlled averaging for on-device federated learning.
\newblock In \emph{Proc. of Int. Conf. Machine Learning (ICML)}, 2020{\natexlab{a}}.

\bibitem[Karimireddy et~al.(2020{\natexlab{b}})Karimireddy, Jaggi, Kale, Mohri, Reddi, Stich, and Suresh]{MIME}
Sai~Praneeth Karimireddy, Martin Jaggi, Satyen Kale, Mehryar Mohri, Sashank~J Reddi, Sebastian~U Stich, and Ananda~Theertha Suresh.
\newblock {Mime: Mimicking centralized stochastic algorithms in federated learning}.
\newblock \emph{arXiv preprint arXiv:2008.03606}, 2020{\natexlab{b}}.

\bibitem[Karimireddy et~al.(2020{\natexlab{c}})Karimireddy, Kale, Mohri, Reddi, Stich, and Suresh]{karimireddy2020scaffold}
Sai~Praneeth Karimireddy, Satyen Kale, Mehryar Mohri, Sashank Reddi, Sebastian Stich, and Ananda~Theertha Suresh.
\newblock Scaffold: Stochastic controlled averaging for federated learning.
\newblock In \emph{International conference on machine learning}, pages 5132--5143. PMLR, 2020{\natexlab{c}}.

\bibitem[Khaled et~al.(2019)Khaled, Mishchenko, and Richt\'{a}rik]{localGD}
A.~Khaled, K.~Mishchenko, and P.~Richt\'{a}rik.
\newblock First analysis of local {GD} on heterogeneous data.
\newblock paper arXiv:1909.04715, presented at NeurIPS Workshop on Federated Learning for Data Privacy and Confidentiality, 2019.

\bibitem[Khaled et~al.(2020{\natexlab{a}})Khaled, Mishchenko, and Richt\'{a}rik]{localSGD-AISTATS2020}
A.~Khaled, K.~Mishchenko, and P.~Richt\'{a}rik.
\newblock Tighter theory for local {SGD} on identical and heterogeneous data.
\newblock In \emph{Proc. of 23rd Int. Conf. Artificial Intelligence and Statistics (AISTATS)}, 2020{\natexlab{a}}.

\bibitem[Khaled et~al.(2020{\natexlab{b}})Khaled, Sebbouh, Loizou, Gower, and {Richt\'{a}rik}]{kha20}
A.~Khaled, O.~Sebbouh, N.~Loizou, R.~M. Gower, and P.~{Richt\'{a}rik}.
\newblock Unified analysis of stochastic gradient methods for composite convex and smooth optimization.
\newblock preprint arXiv:2006.11573, 2020{\natexlab{b}}.

\bibitem[Khaled and Jin(2023)]{SPPM}
Ahmed Khaled and Chi Jin.
\newblock {Faster federated optimization under second-order similarity}.
\newblock In \emph{The Eleventh International Conference on Learning Representations}, 2023.

\bibitem[Khaled and Richt{\'a}rik(2020)]{khaled2020better}
Ahmed Khaled and Peter Richt{\'a}rik.
\newblock Better theory for sgd in the nonconvex world.
\newblock \emph{arXiv preprint arXiv:2002.03329}, 2020.

\bibitem[Kone{\v{c}}n{\`y} et~al.(2016)Kone{\v{c}}n{\`y}, McMahan, Ramage, and Richt{\'a}rik]{konevcny2016federated}
Jakub Kone{\v{c}}n{\`y}, H~Brendan McMahan, Daniel Ramage, and Peter Richt{\'a}rik.
\newblock Federated optimization: Distributed machine learning for on-device intelligence.
\newblock \emph{arXiv preprint arXiv:1610.02527}, 2016.

\bibitem[Kone\v{c}n\'{y} et~al.(2016)Kone\v{c}n\'{y}, McMahan, Yu, Richt\'{a}rik, Suresh, and Bacon]{ja2016}
Jakub Kone\v{c}n\'{y}, H.~Brendan McMahan, Felix~X. Yu, Peter Richt\'{a}rik, Ananda~Theertha Suresh, and Dave Bacon.
\newblock Federated learning: Strategies for improving communication efficiency.
\newblock In \emph{NIPS Workshop on Private Multi-Party Machine Learning}, 2016.

\bibitem[Krizhevsky et~al.(2009)Krizhevsky, Hinton, et~al.]{CIFAR}
Alex Krizhevsky, Geoffrey Hinton, et~al.
\newblock Learning multiple layers of features from tiny images.
\newblock 2009.

\bibitem[Kusupati et~al.(2020)Kusupati, Ramanujan, Somani, Wortsman, Jain, Kakade, and Farhadi]{kusupati2020soft}
Aditya Kusupati, Vivek Ramanujan, Raghav Somani, Mitchell Wortsman, Prateek Jain, Sham Kakade, and Ali Farhadi.
\newblock Soft threshold weight reparameterization for learnable sparsity.
\newblock In \emph{International Conference on Machine Learning}, pages 5544--5555. PMLR, 2020.

\bibitem[Lasby et~al.(2023)Lasby, Golubeva, Evci, Nica, and Ioannou]{SRigL}
Mike Lasby, Anna Golubeva, Utku Evci, Mihai Nica, and Yani Ioannou.
\newblock Dynamic sparse training with structured sparsity.
\newblock \emph{arXiv preprint arXiv:2305.02299}, 2023.

\bibitem[LeCun et~al.(1989)LeCun, Denker, and Solla]{lecun1989optimal}
Yann LeCun, John Denker, and Sara Solla.
\newblock Optimal brain damage.
\newblock \emph{Advances in neural information processing systems}, 2, 1989.

\bibitem[Lee et~al.(2018)Lee, Ajanthan, and Torr]{lee2018snip}
Namhoon Lee, Thalaiyasingam Ajanthan, and Philip~HS Torr.
\newblock Snip: Single-shot network pruning based on connection sensitivity.
\newblock \emph{arXiv preprint arXiv:1810.02340}, 2018.

\bibitem[Li and Wang(2019{\natexlab{a}})]{FedMD}
D.~Li and J.~Wang.
\newblock Fedmd: {H}eterogenous federated learning via model distillation.
\newblock preprint arXiv:1910.03581, 2019{\natexlab{a}}.

\bibitem[Li and Wang(2019{\natexlab{b}})]{li2019fedmd}
Daliang Li and Junpu Wang.
\newblock Fedmd: Heterogenous federated learning via model distillation.
\newblock \emph{arXiv preprint arXiv:1910.03581}, 2019{\natexlab{b}}.

\bibitem[Li et~al.(2021{\natexlab{a}})Li, He, and Song]{MOON}
Q.~Li, B.~He, and D.~Song.
\newblock Model-contrastive federated learning.
\newblock In \emph{Proc. of IEEE/CVF Conf. Computer Vision and Pattern Recognition}, pages 10713--10722, 2021{\natexlab{a}}.

\bibitem[Li et~al.(2021{\natexlab{b}})Li, He, and Song]{li2021model}
Qinbin Li, Bingsheng He, and Dawn Song.
\newblock Model-contrastive federated learning.
\newblock In \emph{Proceedings of the IEEE/CVF conference on computer vision and pattern recognition}, pages 10713--10722, 2021{\natexlab{b}}.

\bibitem[Li et~al.(2020{\natexlab{a}})Li, Sahu, Talwalkar, and Smith]{li2020federated2}
Tian Li, Anit~Kumar Sahu, Ameet Talwalkar, and Virginia Smith.
\newblock Federated learning: Challenges, methods, and future directions.
\newblock \emph{IEEE signal processing magazine}, 37\penalty0 (3):\penalty0 50--60, 2020{\natexlab{a}}.

\bibitem[Li et~al.(2020{\natexlab{b}})Li, Sahu, Zaheer, Sanjabi, Talwalkar, and Smith]{FedProx}
Tian Li, Anit~Kumar Sahu, Manzil Zaheer, Maziar Sanjabi, Ameet Talwalkar, and Virginia Smith.
\newblock {Federated optimization in heterogeneous networks}.
\newblock 2020{\natexlab{b}}.

\bibitem[Li et~al.(2020{\natexlab{c}})Li, Sahu, Zaheer, Sanjabi, Talwalkar, and Smith]{li2020federated}
Tian Li, Anit~Kumar Sahu, Manzil Zaheer, Maziar Sanjabi, Ameet Talwalkar, and Virginia Smith.
\newblock Federated optimization in heterogeneous networks.
\newblock \emph{Proceedings of Machine learning and systems}, 2:\penalty0 429--450, 2020{\natexlab{c}}.

\bibitem[Li et~al.(2020{\natexlab{d}})Li, Huang, Yang, Wang, and Zhang]{Li-local-bounded-grad-norms--ICLR2020}
X.~Li, K.~Huang, W.~Yang, S.~Wang, and Z.~Zhang.
\newblock On the convergence of {F}ed{A}vg on non-{IID} data.
\newblock In \emph{Proc. of Int. Conf. Learning Representations (ICLR)}, 2020{\natexlab{d}}.

\bibitem[Li et~al.(2023)Li, Niu, Zhang, Liu, Zhu, and Kang]{li2023sparse}
Yun Li, Lin Niu, Xipeng Zhang, Kai Liu, Jianchen Zhu, and Zhanhui Kang.
\newblock E-sparse: Boosting the large language model inference through entropy-based n: M sparsity.
\newblock \emph{arXiv preprint arXiv:2310.15929}, 2023.

\bibitem[Li et~al.(2020{\natexlab{e}})Li, Kovalev, Qian, and {Richt\'arik}]{li2020}
Z.~Li, D.~Kovalev, X.~Qian, and P.~{Richt\'arik}.
\newblock Acceleration for compressed gradient descent in distributed and federated optimization.
\newblock In \emph{Proc. of 37th Int. Conf. Machine Learning (ICML)}, 2020{\natexlab{e}}.

\bibitem[Li and Li(2022)]{li2022simple}
Zhize Li and Jian Li.
\newblock Simple and optimal stochastic gradient methods for nonsmooth nonconvex optimization.
\newblock \emph{The Journal of Machine Learning Research}, 23\penalty0 (1):\penalty0 10891--10951, 2022.

\bibitem[Li et~al.(2022)Li, Zhao, Li, and Chi]{li2022soteriafl}
Zhize Li, Haoyu Zhao, Boyue Li, and Yuejie Chi.
\newblock Soteriafl: A unified framework for private federated learning with communication compression.
\newblock \emph{Advances in Neural Information Processing Systems}, 35:\penalty0 4285--4300, 2022.

\bibitem[Liao et~al.(2023)Liao, Gao, Zhao, and Xu]{Flado}
Dongping Liao, Xitong Gao, Yiren Zhao, and Cheng-Zhong Xu.
\newblock Adaptive channel sparsity for federated learning under system heterogeneity.
\newblock In \emph{Proceedings of the IEEE/CVF Conference on Computer Vision and Pattern Recognition}, pages 20432--20441, 2023.

\bibitem[Liao and Kyrillidis(2022)]{liao2022convergence}
Fangshuo Liao and Anastasios Kyrillidis.
\newblock On the convergence of shallow neural network training with randomly masked neurons.
\newblock \emph{Transactions on Machine Learning Research}, 2022.

\bibitem[Lin et~al.(2022)Lin, Kostina, and Hassibi]{chu22}
Chung-Yi Lin, Victoria Kostina, and Babak Hassibi.
\newblock Differentially quantized gradient methods.
\newblock 68\penalty0 (9):\penalty0 6078--6097, September 2022.

\bibitem[Lin et~al.(2024)Lin, Han, Ye, and Zhang]{lin2024stochastic}
Dachao Lin, Yuze Han, Haishan Ye, and Zhihua Zhang.
\newblock {Stochastic distributed optimization under average second-order similarity: Algorithms and analysis}.
\newblock \emph{Advances in Neural Information Processing Systems}, 36, 2024.

\bibitem[Lin et~al.(2020)Lin, Kong, Stich, and Jaggi]{lin2020ensemble}
Tao Lin, Lingjing Kong, Sebastian~U Stich, and Martin Jaggi.
\newblock Ensemble distillation for robust model fusion in federated learning.
\newblock \emph{Advances in Neural Information Processing Systems}, 33:\penalty0 2351--2363, 2020.

\bibitem[Lin et~al.(2017)Lin, Han, Mao, Wang, and Dally]{lin2017deep}
Yujun Lin, Song Han, Huizi Mao, Yu~Wang, and William~J Dally.
\newblock Deep gradient compression: Reducing the communication bandwidth for distributed training.
\newblock \emph{arXiv preprint arXiv:1712.01887}, 2017.

\bibitem[Liu et~al.(2022)Liu, Huang, Zhou, Li, Ji, Xiong, and Dou]{liu2022distributed}
Ji~Liu, Jizhou Huang, Yang Zhou, Xuhong Li, Shilei Ji, Haoyi Xiong, and Dejing Dou.
\newblock From distributed machine learning to federated learning: A survey.
\newblock \emph{Knowledge and Information Systems}, 64\penalty0 (4):\penalty0 885--917, 2022.

\bibitem[Liu et~al.(2021)Liu, Yin, Mocanu, and Pechenizkiy]{liu2021we}
Shiwei Liu, Lu~Yin, Decebal~Constantin Mocanu, and Mykola Pechenizkiy.
\newblock Do we actually need dense over-parameterization? in-time over-parameterization in sparse training.
\newblock In \emph{International Conference on Machine Learning}, pages 6989--7000. PMLR, 2021.

\bibitem[Lowy et~al.(2023)Lowy, Ghafelebashi, and Razaviyayn]{lowy2023private}
Andrew Lowy, Ali Ghafelebashi, and Meisam Razaviyayn.
\newblock Private non-convex federated learning without a trusted server.
\newblock In \emph{International Conference on Artificial Intelligence and Statistics}, pages 5749--5786. PMLR, 2023.

\bibitem[Ma et~al.(2015)Ma, Smith, Jaggi, Jordan, Richt{\'a}rik, and Tak{\'a}c]{ma2015adding}
Chenxin Ma, Virginia Smith, Martin Jaggi, Michael Jordan, Peter Richt{\'a}rik, and Martin Tak{\'a}c.
\newblock Adding vs. averaging in distributed primal-dual optimization.
\newblock In \emph{International Conference on Machine Learning}, pages 1973--1982. PMLR, 2015.

\bibitem[Ma(2020)]{ma2020apollo}
Xuezhe Ma.
\newblock {Apollo: An adaptive parameter-wise diagonal quasi-Newton method for nonconvex stochastic optimization}.
\newblock \emph{arXiv preprint arXiv:2009.13586}, 2020.

\bibitem[Malinovskii et~al.(2024)Malinovskii, Mazur, Ilin, Kuznedelev, Burlachenko, Yi, Alistarh, and Richt{\'a}rik]{PV-Tuning}
Vladimir Malinovskii, Denis Mazur, Ivan Ilin, Denis Kuznedelev, Konstantin~Pavlovich Burlachenko, Kai Yi, Dan Alistarh, and Peter Richt{\'a}rik.
\newblock {PV}-tuning: Beyond straight-through estimation for extreme {LLM} compression.
\newblock In \emph{The Thirty-eighth Annual Conference on Neural Information Processing Systems}, 2024.
\newblock URL \url{https://openreview.net/forum?id=YvA8UF0I37}.

\bibitem[Malinovsky et~al.(2020)Malinovsky, Kovalev, Gasanov, Condat, and Richt{\'a}rik]{mal20}
G.~Malinovsky, D.~Kovalev, E.~Gasanov, L.~Condat, and P.~Richt{\'a}rik.
\newblock From local {SGD} to local fixed point methods for federated learning.
\newblock In \emph{{Proc.\ of} 37th Int. Conf. Machine Learning (ICML)}, 2020.

\bibitem[Malinovsky et~al.(2022)Malinovsky, Yi, and Richt{\'a}rik]{ProxSkip-VR}
G.~Malinovsky, K.~Yi, and P.~Richt{\'a}rik.
\newblock Variance reduced {Proxskip}: {A}lgorithm, theory and application to federated learning.
\newblock In \emph{Proc. of Conf. Neural Information Processing Systems (NeurIPS)}, 2022.

\bibitem[Malinovsky et~al.(2023)Malinovsky, Mishchenko, and Richt{\'a}rik]{malinovsky2023server}
Grigory Malinovsky, Konstantin Mishchenko, and Peter Richt{\'a}rik.
\newblock {Server-side stepsizes and sampling without replacement provably help in federated optimization}.
\newblock In \emph{Proceedings of the 4th International Workshop on Distributed Machine Learning}, pages 85--104, 2023.

\bibitem[Maranjyan et~al.(2022)Maranjyan, Safaryan, and Richt{\'a}rik]{GradSkip}
A.~Maranjyan, M.~Safaryan, and P.~Richt{\'a}rik.
\newblock Gradskip: {C}ommunication-accelerated local gradient methods with better computational complexity.
\newblock preprint arXiv:2210.16402, 2022.

\bibitem[Martinet(1970)]{martinet1970regularisation}
Bernard Martinet.
\newblock {Regularisation d'inequations variationelles par approximations successives}.
\newblock \emph{Revue Francaise d'informatique et de Recherche operationelle}, 4:\penalty0 154--159, 1970.

\bibitem[Mayekar and Tyagi(2021)]{may21}
Prathamesh Mayekar and Himanshu Tyagi.
\newblock {RATQ}: {A} universal fixed-length quantizer for stochastic optimization.
\newblock 67\penalty0 (5):\penalty0 3130--3154, 2021.

\bibitem[McMahan et~al.(2016{\natexlab{a}})McMahan, Moore, Ramage, and {Ag\"{u}era y Arcas}]{FedAvg2016}
B.~McMahan, E.~Moore, D.~Ramage, and B.~{Ag\"{u}era y Arcas}.
\newblock Federated learning of deep networks using model averaging.
\newblock preprint arXiv:1602.05629, 2016{\natexlab{a}}.

\bibitem[McMahan et~al.(2017{\natexlab{a}})McMahan, Moore, Ramage, Hampson, and y~Arcas]{mcmahan2017communication}
Brendan McMahan, Eider Moore, Daniel Ramage, Seth Hampson, and Blaise~Aguera y~Arcas.
\newblock Communication-efficient learning of deep networks from decentralized data.
\newblock In \emph{Artificial intelligence and statistics}, pages 1273--1282. PMLR, 2017{\natexlab{a}}.

\bibitem[McMahan et~al.(2016{\natexlab{b}})McMahan, Yu, Richtarik, Suresh, Bacon, et~al.]{mcmahan2016federated}
H~Brendan McMahan, FX~Yu, P~Richtarik, AT~Suresh, D~Bacon, et~al.
\newblock Federated learning: Strategies for improving communication efficiency.
\newblock In \emph{Proceedings of the 29th Conference on Neural Information Processing Systems (NIPS), Barcelona, Spain}, pages 5--10, 2016{\natexlab{b}}.

\bibitem[McMahan et~al.(2017{\natexlab{b}})McMahan, Moore, Ramage, Hampson, and Ag\"{u}era~y Arcas]{FL2017-AISTATS}
H~Brendan McMahan, Eider Moore, Daniel Ramage, Seth Hampson, and Blaise Ag\"{u}era~y Arcas.
\newblock Communication-efficient learning of deep networks from decentralized data.
\newblock In \emph{Proceedings of the 20th International Conference on Artificial Intelligence and Statistics (AISTATS)}, 2017{\natexlab{b}}.

\bibitem[McMahan et~al.(2017{\natexlab{c}})McMahan, Moore, Ramage, Hampson, and Ag\"{u}era~y Arcas]{FedAvg}
H~Brendan McMahan, Eider Moore, Daniel Ramage, Seth Hampson, and Blaise Ag\"{u}era~y Arcas.
\newblock {Communication-efficient learning of deep networks from decentralized data}.
\newblock In \emph{Proceedings of the 20th International Conference on Artificial Intelligence and Statistics (AISTATS)}, 2017{\natexlab{c}}.

\bibitem[Mei et~al.(2022)Mei, Guo, Zhou, and Patel]{FLANC}
Y.~Mei, P.~Guo, M.~Zhou, and V.~Patel.
\newblock Resource-adaptive federated learning with all-in-one neural composition.
\newblock In \emph{Proc. of Conf. Neural Information Processing Systems (NeurIPS)}, 2022.

\bibitem[Meinhardt et~al.(2024)Meinhardt, Yi, Condat, and Richt{\'a}rik]{SparseProxSkip}
Georg Meinhardt, Kai Yi, Laurent Condat, and Peter Richt{\'a}rik.
\newblock Prune at the clients, not the server: Accelerated sparse training in federated learning.
\newblock \emph{arXiv preprint arXiv:2405.20623}, 2024.

\bibitem[Mendieta et~al.(2022)Mendieta, Yang, Wang, Lee, Ding, and Chen]{Mendieta_2022_CVPR}
Matias Mendieta, Taojiannan Yang, Pu~Wang, Minwoo Lee, Zhengming Ding, and Chen Chen.
\newblock Local learning matters: Rethinking data heterogeneity in federated learning.
\newblock In \emph{Proceedings of the IEEE/CVF Conference on Computer Vision and Pattern Recognition (CVPR)}, pages 8397--8406, June 2022.

\bibitem[Merity et~al.(2016)Merity, Xiong, Bradbury, and Socher]{WikiText2}
Stephen Merity, Caiming Xiong, James Bradbury, and Richard Socher.
\newblock Pointer sentinel mixture models.
\newblock \emph{arXiv preprint arXiv:1609.07843}, 2016.

\bibitem[Mihaylov et~al.(2018)Mihaylov, Clark, Khot, and Sabharwal]{mihaylov2018can}
Todor Mihaylov, Peter Clark, Tushar Khot, and Ashish Sabharwal.
\newblock Can a suit of armor conduct electricity? a new dataset for open book question answering.
\newblock \emph{arXiv preprint arXiv:1809.02789}, 2018.

\bibitem[Mishchenko et~al.(2020)Mishchenko, Hanzely, and Richt{\'a}rik]{mis20}
Konstantin Mishchenko, Filip Hanzely, and Peter Richt{\'a}rik.
\newblock 99\% of worker-master communication in distributed optimization is not needed.
\newblock In \emph{Proc. of 36th Conf. on Uncertainty in Artificial Intelligence (UAI)}, volume 124, pages 979--988, 2020.

\bibitem[Mishchenko et~al.(2022{\natexlab{a}})Mishchenko, Khaled, and Richtarik]{mishchenko2022proximal}
Konstantin Mishchenko, Ahmed Khaled, and Peter Richtarik.
\newblock {Proximal and federated random reshuffling}.
\newblock In \emph{Proceedings of the 39th International Conference on Machine Learning}, volume 162, pages 15718--15749. PMLR, 2022{\natexlab{a}}.

\bibitem[Mishchenko et~al.(2022{\natexlab{b}})Mishchenko, Malinovsky, Stich, and Richt\'{a}rik]{ProxSkip}
Konstantin Mishchenko, Grigory Malinovsky, Sebastian Stich, and Peter Richt\'{a}rik.
\newblock Prox{S}kip: {Y}es! {L}ocal gradient steps provably lead to communication acceleration! {F}inally!
\newblock In \emph{39th International Conference on Machine Learning (ICML 2022)}, 2022{\natexlab{b}}.

\bibitem[Mishchenko et~al.(2024)Mishchenko, Gorbunov, Tak{\'a}{\v{c}}, and Richt{\'a}rik]{DIANA}
Konstantin Mishchenko, Eduard Gorbunov, Martin Tak{\'a}{\v{c}}, and Peter Richt{\'a}rik.
\newblock Distributed learning with compressed gradient differences.
\newblock \emph{Optimization Methods and Software}, pages 1--16, 2024.

\bibitem[Mitra et~al.(2021)Mitra, Jaafar, Pappas, and Hassani]{FEDLIN}
A.~Mitra, R.~Jaafar, G.~Pappas, and H.~Hassani.
\newblock Linear convergence in federated learning: {T}ackling client heterogeneity and sparse gradients.
\newblock In \emph{Proc. of Conf. Neural Information Processing Systems (NeurIPS)}, 2021.

\bibitem[Mocanu et~al.(2018)Mocanu, Mocanu, Stone, Nguyen, Gibescu, and Liotta]{mocanu2018scalable}
Decebal~Constantin Mocanu, Elena Mocanu, Peter Stone, Phuong~H Nguyen, Madeleine Gibescu, and Antonio Liotta.
\newblock Scalable training of artificial neural networks with adaptive sparse connectivity inspired by network science.
\newblock \emph{Nature communications}, 9\penalty0 (1):\penalty0 2383, 2018.

\bibitem[Moreau(1965)]{moreau1965proximite}
Jean-Jacques Moreau.
\newblock {Proximité et dualité dans un espace hilbertien}.
\newblock \emph{Bulletin de la Société Mathématique de France}, 93:\penalty0 273--299, 1965.

\bibitem[Moritz et~al.(2016)Moritz, Nishihara, Stoica, and Jordan]{SparkNet2016}
P.~Moritz, R.~Nishihara, I.~Stoica, and M.~I. Jordan.
\newblock Spark{N}et: Training deep networks in {S}park.
\newblock In \emph{Proc. of Int. Conf. Learning Representations (ICLR)}, 2016.

\bibitem[Mostafa and Wang(2019)]{mostafa2019parameter}
Hesham Mostafa and Xin Wang.
\newblock Parameter efficient training of deep convolutional neural networks by dynamic sparse reparameterization.
\newblock In \emph{International Conference on Machine Learning}, pages 4646--4655. PMLR, 2019.

\bibitem[Nesterov(2003)]{nesterov2003introductory}
Yurii Nesterov.
\newblock \emph{Introductory lectures on convex optimization: A basic course}, volume~87.
\newblock Springer Science \& Business Media, 2003.

\bibitem[Parikh and Boyd(2014)]{par14}
N.~Parikh and S.~Boyd.
\newblock Proximal algorithms.
\newblock \emph{Foundations and Trends in Optimization}, 3\penalty0 (1):\penalty0 127--239, 2014.

\bibitem[Patrascu and Necoara(2018)]{patrascu2018nonasymptotic}
Andrei Patrascu and Ion Necoara.
\newblock {Nonasymptotic convergence of stochastic proximal point methods for constrained convex optimization}.
\newblock \emph{Journal of Machine Learning Research}, 18\penalty0 (198):\penalty0 1--42, 2018.

\bibitem[Philippenko and Dieuleveut(2020)]{phi20}
C.~Philippenko and A.~Dieuleveut.
\newblock Bidirectional compression in heterogeneous settings for distributed or federated learning with partial participation: tight convergence guarantees.
\newblock arXiv:2006.14591, 2020.

\bibitem[Povey et~al.(2014)Povey, Zhang, and Khudanpur]{Povey2015}
D.~Povey, X.~Zhang, and S.~Khudanpur.
\newblock Parallel training of {DNN}s with natural gradient and parameter averaging.
\newblock preprint arXiv:1410.7455, 2014.

\bibitem[Qian et~al.(2019)Qian, Sailanbayev, Mishchenko, and Richt{\'a}rik]{qia19}
X.~Qian, A.~Sailanbayev, K.~Mishchenko, and P.~Richt{\'a}rik.
\newblock {MISO} is making a comeback with better proofs and rates.
\newblock arXiv:1906.01474, June 2019.

\bibitem[Raffel et~al.(2020)Raffel, Shazeer, Roberts, Lee, Narang, Matena, Zhou, Li, and Liu]{C4}
Colin Raffel, Noam Shazeer, Adam Roberts, Katherine Lee, Sharan Narang, Michael Matena, Yanqi Zhou, Wei Li, and Peter~J Liu.
\newblock Exploring the limits of transfer learning with a unified text-to-text transformer.
\newblock \emph{Journal of machine learning research}, 21\penalty0 (140):\penalty0 1--67, 2020.

\bibitem[Ramaswamy et~al.(2019)Ramaswamy, Mathews, Rao, and Beaufays]{ramaswamy2019federated}
Swaroop Ramaswamy, Rajiv Mathews, Kanishka Rao, and Fran{\c{c}}oise Beaufays.
\newblock {Federated learning for emoji prediction in a mobile keyboard}.
\newblock \emph{arXiv preprint arXiv:1906.04329}, 2019.

\bibitem[Reddi et~al.(2020)Reddi, Charles, Zaheer, Garrett, Rush, Kone{\v{c}}n{\`y}, Kumar, and McMahan]{reddi2020adaptive}
Sashank Reddi, Zachary Charles, Manzil Zaheer, Zachary Garrett, Keith Rush, Jakub Kone{\v{c}}n{\`y}, Sanjiv Kumar, and H~Brendan McMahan.
\newblock {Adaptive federated optimization}.
\newblock \emph{arXiv preprint arXiv:2003.00295}, 2020.

\bibitem[{Richt\'arik} and {Tak\'a\v{c}}(2016)]{ric16}
P.~{Richt\'arik} and M.~{Tak\'a\v{c}}.
\newblock Parallel coordinate descent methods for big data optimization.
\newblock \emph{Math. Program.}, 156:\penalty0 433--484, 2016.

\bibitem[Richt{\'a}rik et~al.(2021{\natexlab{a}})Richt{\'a}rik, Sokolov, and Fatkhullin]{EF21}
Peter Richt{\'a}rik, Igor Sokolov, and Ilyas Fatkhullin.
\newblock Ef21: A new, simpler, theoretically better, and practically faster error feedback.
\newblock \emph{Advances in Neural Information Processing Systems}, 34:\penalty0 4384--4396, 2021{\natexlab{a}}.

\bibitem[Richt{\'a}rik et~al.(2021{\natexlab{b}})Richt{\'a}rik, Sokolov, and Fatkhullin]{ric21}
Peter Richt{\'a}rik, Igor Sokolov, and Ilyas Fatkhullin.
\newblock {EF21}: {A} new, simpler, theoretically better, and practically faster error feedback.
\newblock In \emph{Proc. of 35th Conf. Neural Information Processing Systems (NeurIPS)}, 2021{\natexlab{b}}.

\bibitem[Ro et~al.(2021)Ro, Suresh, and Wu]{ro2021fedjax}
J.~H. Ro, A.~T. Suresh, and K.~Wu.
\newblock {FedJAX: Federated learning simulation with JAX}.
\newblock preprint arXiv:2108.02117, 2021.

\bibitem[Ryu and Boyd(2016)]{RyuBoy:16}
Ernest Ryu and Stephen Boyd.
\newblock {Stochastic proximal iteration: A non-asymptotic improvement upon stochastic gradient descent}.
\newblock Technical report, Stanford University, 2016.

\bibitem[Safaryan et~al.(2021{\natexlab{a}})Safaryan, Hanzely, and Richt{\'a}rik]{safaryan2021smoothness}
Mher Safaryan, Filip Hanzely, and Peter Richt{\'a}rik.
\newblock Smoothness matrices beat smoothness constants: Better communication compression techniques for distributed optimization.
\newblock \emph{Advances in Neural Information Processing Systems}, 34:\penalty0 25688--25702, 2021{\natexlab{a}}.

\bibitem[Safaryan et~al.(2021{\natexlab{b}})Safaryan, Shulgin, and Richt{\'a}rik]{saf21}
Mher Safaryan, Egor Shulgin, and Peter Richt{\'a}rik.
\newblock Uncertainty principle for communication compression in distributed and federated learning and the search for an optimal compressor.
\newblock \emph{Information and Inference: A Journal of the IMA}, 2021{\natexlab{b}}.

\bibitem[Saha et~al.(2021)Saha, Pilanci, and Goldsmith]{sah21}
Rajarshi Saha, Mert Pilanci, and Andrea~J. Goldsmith.
\newblock Democratic source coding: An optimal fixed-length quantization scheme for distributed optimization under communication constraints.
\newblock preprint arXiv:2103.07578, 2021.

\bibitem[Sakaguchi et~al.(2021)Sakaguchi, Bras, Bhagavatula, and Choi]{sakaguchi2021winogrande}
Keisuke Sakaguchi, Ronan~Le Bras, Chandra Bhagavatula, and Yejin Choi.
\newblock Winogrande: An adversarial winograd schema challenge at scale.
\newblock \emph{Communications of the ACM}, 64\penalty0 (9):\penalty0 99--106, 2021.

\bibitem[Seide et~al.(2014)Seide, Fu, Droppo, Li, and Yu]{sei14}
Frank Seide, Hao Fu, Jasha Droppo, Gang Li, and Dong Yu.
\newblock 1-bit stochastic gradient descent and application to data-parallel distributed training of speech {DNNs}.
\newblock In \emph{Proc. of Annual Conf. of Int. Speech Communication Association (Interspeech)}, 2014.

\bibitem[Shanno(1970)]{shanno1970conditioning}
David~F Shanno.
\newblock {Conditioning of quasi-Newton methods for function minimization}.
\newblock \emph{Mathematics of Computation}, 24\penalty0 (111):\penalty0 647--656, 1970.

\bibitem[Sheller et~al.(2020)Sheller, Edwards, Reina, Martin, Pati, Kotrotsou, Milchenko, Xu, Marcus, Colen, et~al.]{sheller2020federated}
Micah~J Sheller, Brandon Edwards, G~Anthony Reina, Jason Martin, Sarthak Pati, Aikaterini Kotrotsou, Mikhail Milchenko, Weilin Xu, Daniel Marcus, Rivka~R Colen, et~al.
\newblock Federated learning in medicine: facilitating multi-institutional collaborations without sharing patient data.
\newblock \emph{Scientific reports}, 10\penalty0 (1):\penalty0 12598, 2020.

\bibitem[Shtoff(2022)]{shtoff2022efficient}
Alex Shtoff.
\newblock {Efficient implementation of incremental proximal-point methods}.
\newblock \emph{arXiv preprint arXiv:2205.01457}, 2022.

\bibitem[Shulgin and Richt{\'a}rik(2023)]{shulgin2023towards}
Egor Shulgin and Peter Richt{\'a}rik.
\newblock Towards a better theoretical understanding of independent subnetwork training.
\newblock \emph{arXiv preprint arXiv:2306.16484}, 2023.

\bibitem[Smith et~al.(2017)Smith, Chiang, Sanjabi, and Talwalkar]{smith2017federated}
Virginia Smith, Chao-Kai Chiang, Maziar Sanjabi, and Ameet~S Talwalkar.
\newblock Federated multi-task learning.
\newblock \emph{Advances in neural information processing systems}, 30, 2017.

\bibitem[Stich(2018)]{stich2018local}
Sebastian~U Stich.
\newblock Local sgd converges fast and communicates little.
\newblock \emph{arXiv preprint arXiv:1805.09767}, 2018.

\bibitem[Sun et~al.(2024)Sun, Wu, Huang, and Zhang]{sun2024role}
Jianhui Sun, Xidong Wu, Heng Huang, and Aidong Zhang.
\newblock {On the role of server momentum in federated learning}.
\newblock In \emph{Proceedings of the AAAI Conference on Artificial Intelligence}, volume~38, pages 15164--15172, 2024.

\bibitem[Sun et~al.(2023{\natexlab{a}})Sun, Liu, Bair, and Kolter]{Wanda}
Mingjie Sun, Zhuang Liu, Anna Bair, and J~Zico Kolter.
\newblock A simple and effective pruning approach for large language models.
\newblock In \emph{The Twelfth International Conference on Learning Representations}, 2023{\natexlab{a}}.

\bibitem[Sun et~al.(2023{\natexlab{b}})Sun, Shen, Huang, Ding, and Tao]{FedSpeed}
Yan Sun, Li~Shen, Tiansheng Huang, Liang Ding, and Dacheng Tao.
\newblock {Fedspeed: Larger local interval, less communication round, and higher generalization accuracy}.
\newblock \emph{arXiv preprint arXiv:2302.10429}, 2023{\natexlab{b}}.

\bibitem[Szlendak et~al.(2022)Szlendak, Tyurin, and Richt{\'a}rik]{sze22}
R.~Szlendak, A.~Tyurin, and P.~Richt{\'a}rik.
\newblock Permutation compressors for provably faster distributed nonconvex optimization.
\newblock In \emph{Proc. of Int. Conf. on Learning Representations (ICLR)}, 2022.

\bibitem[Szlendak et~al.(2021)Szlendak, Tyurin, and Richt{\'a}rik]{szlendak2021permutation}
Rafa{\l} Szlendak, Alexander Tyurin, and Peter Richt{\'a}rik.
\newblock Permutation compressors for provably faster distributed nonconvex optimization.
\newblock \emph{arXiv preprint arXiv:2110.03300}, 2021.

\bibitem[Tan et~al.(2022)Tan, Long, Liu, Zhou, Lu, Jiang, and Zhang]{tan2022fedproto}
Yue Tan, Guodong Long, Lu~Liu, Tianyi Zhou, Qinghua Lu, Jing Jiang, and Chengqi Zhang.
\newblock Fedproto: Federated prototype learning across heterogeneous clients.
\newblock In \emph{Proceedings of the AAAI Conference on Artificial Intelligence}, volume~36, pages 8432--8440, 2022.

\bibitem[Touvron et~al.(2023{\natexlab{a}})Touvron, Lavril, Izacard, Martinet, Lachaux, Lacroix, Rozi{\`e}re, Goyal, Hambro, Azhar, et~al.]{LlaMA}
Hugo Touvron, Thibaut Lavril, Gautier Izacard, Xavier Martinet, Marie-Anne Lachaux, Timoth{\'e}e Lacroix, Baptiste Rozi{\`e}re, Naman Goyal, Eric Hambro, Faisal Azhar, et~al.
\newblock Llama: Open and efficient foundation language models.
\newblock \emph{arXiv preprint arXiv:2302.13971}, 2023{\natexlab{a}}.

\bibitem[Touvron et~al.(2023{\natexlab{b}})Touvron, Lavril, Izacard, Martinet, Lachaux, Lacroix, Rozi{\`e}re, Goyal, Hambro, Azhar, et~al.]{touvron2023llama}
Hugo Touvron, Thibaut Lavril, Gautier Izacard, Xavier Martinet, Marie-Anne Lachaux, Timoth{\'e}e Lacroix, Baptiste Rozi{\`e}re, Naman Goyal, Eric Hambro, Faisal Azhar, et~al.
\newblock Llama: Open and efficient foundation language models.
\newblock \emph{arXiv preprint arXiv:2302.13971}, 2023{\natexlab{b}}.

\bibitem[Touvron et~al.(2023{\natexlab{c}})Touvron, Martin, Stone, Albert, Almahairi, Babaei, Bashlykov, Batra, Bhargava, Bhosale, et~al.]{LlaMA2}
Hugo Touvron, Louis Martin, Kevin Stone, Peter Albert, Amjad Almahairi, Yasmine Babaei, Nikolay Bashlykov, Soumya Batra, Prajjwal Bhargava, Shruti Bhosale, et~al.
\newblock Llama 2: Open foundation and fine-tuned chat models.
\newblock \emph{arXiv preprint arXiv:2307.09288}, 2023{\natexlab{c}}.

\bibitem[Verbraeken et~al.(2020)Verbraeken, Wolting, Katzy, Kloppenburg, Verbelen, and Rellermeyer]{verbraeken2020survey}
Joost Verbraeken, Matthijs Wolting, Jonathan Katzy, Jeroen Kloppenburg, Tim Verbelen, and Jan~S Rellermeyer.
\newblock A survey on distributed machine learning.
\newblock \emph{Acm computing surveys (csur)}, 53\penalty0 (2):\penalty0 1--33, 2020.

\bibitem[Wang et~al.(2018)Wang, Singh, Michael, Hill, Levy, and Bowman]{wang2018glue}
A~Wang, A~Singh, J~Michael, F~Hill, O~Levy, and SR~Bowman.
\newblock Glue: A multi-task benchmark and analysis platform for natural language understanding. arxiv preprint arxiv: 180407461, 2018.

\bibitem[Wang et~al.(2022)Wang, Safaryan, and Richt{\'a}rik]{wang2022theoretically}
Bokun Wang, Mher Safaryan, and Peter Richt{\'a}rik.
\newblock Theoretically better and numerically faster distributed optimization with smoothness-aware quantization techniques.
\newblock \emph{Advances in Neural Information Processing Systems}, 35:\penalty0 9841--9852, 2022.

\bibitem[Wang et~al.(2017)Wang, Ye, and Xu]{wang2017differentially}
Di~Wang, Minwei Ye, and Jinhui Xu.
\newblock Differentially private empirical risk minimization revisited: Faster and more general.
\newblock \emph{Advances in Neural Information Processing Systems}, 30, 2017.

\bibitem[Wang et~al.(2021{\natexlab{a}})Wang, Liu, Liang, Joshi, and Poor]{wang2021novel}
Jianyu Wang, Qinghua Liu, Hao Liang, Gauri Joshi, and H~Vincent Poor.
\newblock {A novel framework for the analysis and design of heterogeneous federated learning}.
\newblock \emph{IEEE Transactions on Signal Processing}, 69:\penalty0 5234--5249, 2021{\natexlab{a}}.

\bibitem[Wang et~al.(2021{\natexlab{b}})Wang, Xu, Garrett, Charles, Liu, and Joshi]{wang2021local}
Jianyu Wang, Zheng Xu, Zachary Garrett, Zachary Charles, Luyang Liu, and Gauri Joshi.
\newblock {Local adaptivity in federated learning: Convergence and consistency}.
\newblock \emph{arXiv preprint arXiv:2106.02305}, 2021{\natexlab{b}}.

\bibitem[{Wang et al.}(2021)]{FieldGuide2021}
J.~{Wang et al.}
\newblock A field guide to federated optimization.
\newblock preprint arXiv:2107.06917, 2021.

\bibitem[Wolf et~al.(2020)Wolf, Debut, Sanh, Chaumond, Delangue, Moi, Cistac, Rault, Louf, Funtowicz, et~al.]{wolf2020transformers}
Thomas Wolf, Lysandre Debut, Victor Sanh, Julien Chaumond, Clement Delangue, Anthony Moi, Pierric Cistac, Tim Rault, R{\'e}mi Louf, Morgan Funtowicz, et~al.
\newblock Transformers: State-of-the-art natural language processing.
\newblock \emph{EMNLP 2020}, page~38, 2020.

\bibitem[Wolfe et~al.(2023)Wolfe, Yang, Liao, Chowdhury, Dun, Bayer, Segarra, and Kyrillidis]{wolfe2023gist}
Cameron~R Wolfe, Jingkang Yang, Fangshuo Liao, Arindam Chowdhury, Chen Dun, Artun Bayer, Santiago Segarra, and Anastasios Kyrillidis.
\newblock Gist: Distributed training for large-scale graph convolutional networks.
\newblock \emph{Journal of Applied and Computational Topology}, pages 1--53, 2023.

\bibitem[Wu et~al.(2023)Wu, Zhang, Yu, Liu, Gu, Zhou, Chen, and Cheng]{FedGMM}
Y.~Wu, S.~Zhang, W.~Yu, Y.~Liu, Q.~Gu, D.~Zhou, H.~Chen, and W.~Cheng.
\newblock Personalized federated learning under mixture of distributions.
\newblock preprint arXiv:2305.01068, 2023.

\bibitem[Xiao et~al.(2023)Xiao, Lin, Seznec, Wu, Demouth, and Han]{SmoothQuant}
Guangxuan Xiao, Ji~Lin, Mickael Seznec, Hao Wu, Julien Demouth, and Song Han.
\newblock Smoothquant: Accurate and efficient post-training quantization for large language models.
\newblock In \emph{International Conference on Machine Learning}, pages 38087--38099. PMLR, 2023.

\bibitem[Xiao et~al.(2017)Xiao, Rasul, and Vollgraf]{FashionMNIST}
Han Xiao, Kashif Rasul, and Roland Vollgraf.
\newblock Fashion-mnist: a novel image dataset for benchmarking machine learning algorithms.
\newblock \emph{arXiv preprint arXiv:1708.07747}, 2017.

\bibitem[Xu et~al.(2020)Xu, Ho, Abdelmoniem, Dutta, Bergou, Karatsenidis, Canini, and Kalnis]{xu20}
H.~Xu, C.-Y. Ho, A.~M. Abdelmoniem, A.~Dutta, E.~H. Bergou, K.~Karatsenidis, M.~Canini, and P.~Kalnis.
\newblock Compressed communication for distributed deep learning: {S}urvey and quantitative evaluation.
\newblock Technical report, KAUST, 2020.

\bibitem[Xu et~al.(2021)Xu, Wang, Wang, and Yao]{xu2021fedcm}
Jing Xu, Sen Wang, Liwei Wang, and Andrew Chi-Chih Yao.
\newblock {FedCM: Federated learning with client-level momentum}.
\newblock \emph{arXiv preprint arXiv:2106.10874}, 2021.

\bibitem[Yang et~al.(2020)Yang, He, Zhang, and Cao]{FedSteg}
H.~Yang, H.~He, W.~Zhang, and X.~Cao.
\newblock Fedsteg: A federated transfer learning framework for secure image steganalysis.
\newblock \emph{IEEE Trans. Network Science and Engineering}, 8\penalty0 (2):\penalty0 1084--1094, 2020.

\bibitem[Yang et~al.(2018)Yang, Andrew, Eichner, Sun, Li, Kong, Ramage, and Beaufays]{yang2018applied}
Timothy Yang, Galen Andrew, Hubert Eichner, Haicheng Sun, Wei Li, Nicholas Kong, Daniel Ramage, and Fran{\c{c}}oise Beaufays.
\newblock {Applied federated learning: Improving Google keyboard query suggestions}.
\newblock \emph{arXiv preprint arXiv:1812.02903}, 2018.

\bibitem[Ye et~al.(2023{\natexlab{a}})Ye, Fang, Du, Yuen, and Tao]{ye2023heterogeneous}
Mang Ye, Xiuwen Fang, Bo~Du, Pong~C Yuen, and Dacheng Tao.
\newblock Heterogeneous federated learning: State-of-the-art and research challenges.
\newblock \emph{ACM Computing Surveys}, 56\penalty0 (3):\penalty0 1--44, 2023{\natexlab{a}}.

\bibitem[Ye et~al.(2023{\natexlab{b}})Ye, Ni, Wu, Chen, and Wang]{pFedGraph}
R.~Ye, Z.~Ni, F.~Wu, S.~Chen, and Y.~Wang.
\newblock Personalized federated learning with inferred collaboration graphs.
\newblock In \emph{{Proc.\ of} 40th Int. Conf. Machine Learning (ICML), PMLR 202}, 2023{\natexlab{b}}.

\bibitem[Ye et~al.(2024)Ye, Ge, Zhu, Chai, Du, Liu, Wang, and Chen]{ye2024fedllm}
Rui Ye, Rui Ge, Xinyu Zhu, Jingyi Chai, Yaxin Du, Yang Liu, Yanfeng Wang, and Siheng Chen.
\newblock Fedllm-bench: Realistic benchmarks for federated learning of large language models.
\newblock \emph{arXiv preprint arXiv:2406.04845}, 2024.

\bibitem[Yi et~al.(2021{\natexlab{a}})Yi, Janson, Zhang, and Elhoseiny]{DACZSL}
Kai Yi, Paul Janson, Wenxuan Zhang, and Mohamed Elhoseiny.
\newblock Domain-aware continual zero-shot learning.
\newblock \emph{arXiv preprint arXiv:2112.12989}, 2021{\natexlab{a}}.

\bibitem[Yi et~al.(2021{\natexlab{b}})Yi, Pang, Zhang, Zeng, and Xu]{3DSpVAE}
Kai Yi, Jianye Pang, Yungeng Zhang, Xiangrui Zeng, and Min Xu.
\newblock Disentangling semantic features of macromolecules in cryo-electron tomography.
\newblock \emph{arXiv preprint arXiv:2106.14192}, 2021{\natexlab{b}}.

\bibitem[Yi et~al.(2022)Yi, Shen, Gou, and Elhoseiny]{HGR-Net}
Kai Yi, Xiaoqian Shen, Yunhao Gou, and Mohamed Elhoseiny.
\newblock Exploring hierarchical graph representation for large-scale zero-shot image classification.
\newblock In \emph{European Conference on Computer Vision}, pages 116--132. Springer, 2022.

\bibitem[Yi et~al.(2023)Yi, Condat, and Richt{\'a}rik]{Scafflix}
Kai Yi, Laurent Condat, and Peter Richt{\'a}rik.
\newblock Explicit personalization and local training: Double communication acceleration in federated learning.
\newblock \emph{arXiv preprint arXiv:2305.13170}, 2023.

\bibitem[Yi et~al.(2024)Yi, Gazagnadou, Richtarik, and Lyu]{kai2023fedp3}
Kai Yi, Nidham Gazagnadou, Peter Richtarik, and Lingjuan Lyu.
\newblock Fedp3: Federated personalized and privacy-friendly network pruning under model heterogeneity.
\newblock \emph{ICLR}, 2024.

\bibitem[Yuan et~al.(2022)Yuan, Wolfe, Dun, Tang, Kyrillidis, and Jermaine]{yuan2022distributed}
Binhang Yuan, Cameron~R Wolfe, Chen Dun, Yuxin Tang, Anastasios Kyrillidis, and Chris Jermaine.
\newblock Distributed learning of fully connected neural networks using independent subnet training.
\newblock \emph{Proceedings of the VLDB Endowment}, 15\penalty0 (8):\penalty0 1581--1590, 2022.

\bibitem[Yuan and Li(2023)]{yuan2023sharper}
Xiao-Tong Yuan and Ping Li.
\newblock {Sharper analysis for minibatch stochastic proximal point methods: Stability, smoothness, and deviation}.
\newblock \emph{Journal of Machine Learning Research}, 24\penalty0 (270):\penalty0 1--52, 2023.

\bibitem[Yuan and Li(2022)]{yuan2022convergence}
Xiaotong Yuan and Ping Li.
\newblock {On convergence of {FedProx}: Local dissimilarity invariant bounds, non-smoothness and beyond}.
\newblock \emph{Advances in Neural Information Processing Systems}, 35:\penalty0 10752--10765, 2022.

\bibitem[Zeiler and Fergus(2014)]{zeiler2014visualizing}
Matthew~D Zeiler and Rob Fergus.
\newblock Visualizing and understanding convolutional networks.
\newblock In \emph{Computer Vision--ECCV 2014: 13th European Conference, Zurich, Switzerland, September 6-12, 2014, Proceedings, Part I 13}, pages 818--833. Springer, 2014.

\bibitem[Zellers et~al.(2019)Zellers, Holtzman, Bisk, Farhadi, and Choi]{zellers2019hellaswag}
Rowan Zellers, Ari Holtzman, Yonatan Bisk, Ali Farhadi, and Yejin Choi.
\newblock Hellaswag: Can a machine really finish your sentence?
\newblock \emph{arXiv preprint arXiv:1905.07830}, 2019.

\bibitem[Zeng et~al.(2023)Zeng, Liang, Hu, Wang, and Xu]{FedLab}
Dun Zeng, Siqi Liang, Xiangjing Hu, Hui Wang, and Zenglin Xu.
\newblock Fedlab: A flexible federated learning framework.
\newblock \emph{Journal of Machine Learning Research}, 24\penalty0 (100):\penalty0 1--7, 2023.

\bibitem[Zeng et~al.(2021)Zeng, Howe, Yi, Zeng, Zhang, Chang, and Xu]{MLUDA}
Yuchen Zeng, Gregory Howe, Kai Yi, Xiangrui Zeng, Jing Zhang, Yi-Wei Chang, and Min Xu.
\newblock Unsupervised domain alignment based open set structural recognition of macromolecules captured by cryo-electron tomography.
\newblock In \emph{2021 IEEE International Conference on Image Processing (ICIP)}, pages 106--110. IEEE, 2021.

\bibitem[Zhang et~al.(2024{\natexlab{a}})Zhang, Wang, Deng, Li, Yang, and Yin]{zhang2024magr}
Aozhong Zhang, Naigang Wang, Yanxia Deng, Xin Li, Zi~Yang, and Penghang Yin.
\newblock Magr: Weight magnitude reduction for enhancing post-training quantization.
\newblock \emph{Advances in neural information processing systems}, 2024{\natexlab{a}}.

\bibitem[Zhang et~al.(2019)Zhang, Li, Nado, Martens, Sachdeva, Dahl, Shallue, and Grosse]{zhang2019algorithmic}
Guodong Zhang, Lala Li, Zachary Nado, James Martens, Sushant Sachdeva, George Dahl, Chris Shallue, and Roger~B Grosse.
\newblock Which algorithmic choices matter at which batch sizes? insights from a noisy quadratic model.
\newblock \emph{Advances in neural information processing systems}, 32, 2019.

\bibitem[Zhang et~al.(2023{\natexlab{a}})Zhang, Li, Dai, Zou, and Xiong]{FedCR}
Hao Zhang, Chenglin Li, Wenrui Dai, Junni Zou, and Hongkai Xiong.
\newblock Fedcr: Personalized federated learning based on across-client common representation with conditional mutual information regularization.
\newblock 2023{\natexlab{a}}.

\bibitem[Zhang et~al.(2021)Zhang, You, and Xie]{zha21}
Jiaqi Zhang, Keyou You, and Lihua Xie.
\newblock Innovation compression for communication-efficient distributed optimization with linear convergence.
\newblock preprint arXiv:2105.06697, 2021.

\bibitem[Zhang et~al.(2022{\natexlab{a}})Zhang, Roller, Goyal, Artetxe, Chen, Chen, Dewan, Diab, Li, Lin, et~al.]{OPT}
Susan Zhang, Stephen Roller, Naman Goyal, Mikel Artetxe, Moya Chen, Shuohui Chen, Christopher Dewan, Mona Diab, Xian Li, Xi~Victoria Lin, et~al.
\newblock Opt: Open pre-trained transformer language models.
\newblock \emph{arXiv preprint arXiv:2205.01068}, 2022{\natexlab{a}}.

\bibitem[Zhang et~al.(2023{\natexlab{b}})Zhang, Janson, Yi, Skorokhodov, and Elhoseiny]{IGCZSL}
Wenxuan Zhang, Paul Janson, Kai Yi, Ivan Skorokhodov, and Mohamed Elhoseiny.
\newblock Continual zero-shot learning through semantically guided generative random walks.
\newblock In \emph{Proceedings of the IEEE/CVF International Conference on Computer Vision}, pages 11574--11585, 2023{\natexlab{b}}.

\bibitem[Zhang et~al.(2020)Zhang, Fang, Liu, and Zhu]{zhang2020private}
Xin Zhang, Minghong Fang, Jia Liu, and Zhengyuan Zhu.
\newblock Private and communication-efficient edge learning: a sparse differential gaussian-masking distributed sgd approach.
\newblock In \emph{Proceedings of the Twenty-First International Symposium on Theory, Algorithmic Foundations, and Protocol Design for Mobile Networks and Mobile Computing}, pages 261--270, 2020.

\bibitem[Zhang et~al.(2024{\natexlab{b}})Zhang, Bai, Lin, Zhao, Hou, and Cannistraci]{RIA}
Yingtao Zhang, Haoli Bai, Haokun Lin, Jialin Zhao, Lu~Hou, and Carlo~Vittorio Cannistraci.
\newblock Plug-and-play: An efficient post-training pruning method for large language models.
\newblock In \emph{The Twelfth International Conference on Learning Representations}, 2024{\natexlab{b}}.

\bibitem[Zhang et~al.(2022{\natexlab{b}})Zhang, Lin, Lin, Luo, Li, Chao, Wu, and Ji]{zhang2022learning}
Yuxin Zhang, Mingbao Lin, Zhihang Lin, Yiting Luo, Ke~Li, Fei Chao, Yongjian Wu, and Rongrong Ji.
\newblock Learning best combination for efficient n: M sparsity.
\newblock \emph{Advances in Neural Information Processing Systems}, 35:\penalty0 941--953, 2022{\natexlab{b}}.

\bibitem[Zhang et~al.(2023{\natexlab{c}})Zhang, Zhao, Lin, Sun, Yao, Han, Tanner, Liu, and Ji]{zhang2023dynamic}
Yuxin Zhang, Lirui Zhao, Mingbao Lin, Yunyun Sun, Yiwu Yao, Xingjia Han, Jared Tanner, Shiwei Liu, and Rongrong Ji.
\newblock Dynamic sparse no training: Training-free fine-tuning for sparse llms.
\newblock \emph{arXiv preprint arXiv:2310.08915}, 2023{\natexlab{c}}.

\bibitem[Zhao et~al.(2020)Zhao, Zhao, Yang, Wang, Wang, Lyu, Niyato, and Lam]{zhao2020local}
Yang Zhao, Jun Zhao, Mengmeng Yang, Teng Wang, Ning Wang, Lingjuan Lyu, Dusit Niyato, and Kwok-Yan Lam.
\newblock Local differential privacy-based federated learning for internet of things.
\newblock \emph{IEEE Internet of Things Journal}, 8\penalty0 (11):\penalty0 8836--8853, 2020.

\bibitem[Zhao et~al.(2018)Zhao, Li, Lai, Suda, Civin, and Chandra]{zhao2018federated}
Yue Zhao, Meng Li, Liangzhen Lai, Naveen Suda, Damon Civin, and Vikas Chandra.
\newblock Federated learning with non-iid data.
\newblock \emph{arXiv preprint arXiv:1806.00582}, 2018.

\bibitem[Zhou et~al.(2021)Zhou, Ma, Zhu, Liu, Zhang, Yuan, Sun, and Li]{zhou2021learning}
Aojun Zhou, Yukun Ma, Junnan Zhu, Jianbo Liu, Zhijie Zhang, Kun Yuan, Wenxiu Sun, and Hongsheng Li.
\newblock Learning n: m fine-grained structured sparse neural networks from scratch.
\newblock \emph{arXiv preprint arXiv:2102.04010}, 2021.

\bibitem[Zhu et~al.(2022)Zhu, Liu, Radhakrishnan, and Belkin]{zhu2022quadratic}
Libin Zhu, Chaoyue Liu, Adityanarayanan Radhakrishnan, and Mikhail Belkin.
\newblock Quadratic models for understanding neural network dynamics.
\newblock \emph{arXiv preprint arXiv:2205.11787}, 2022.

\bibitem[Zhu et~al.(2019)Zhu, Liu, and Han]{zhu2019deep}
Ligeng Zhu, Zhijian Liu, and Song Han.
\newblock Deep leakage from gradients.
\newblock \emph{Advances in neural information processing systems}, 32, 2019.

\bibitem[Zhu et~al.(2024)Zhu, Li, Liu, Ma, and Wang]{zhu2024survey}
Xunyu Zhu, Jian Li, Yong Liu, Can Ma, and Weiping Wang.
\newblock A survey on model compression for large language models.
\newblock \emph{Transactions of the Association for Computational Linguistics}, 12:\penalty0 1556--1577, 2024.

\end{thebibliography}
